\documentclass{article}

\usepackage{arxiv}

\usepackage[utf8]{inputenc} 
\usepackage[T1]{fontenc}    
\usepackage{hyperref}       
\usepackage{url}            
\usepackage{booktabs}       
\usepackage{amsfonts}       
\usepackage{nicefrac}       
\usepackage{amsmath}        
\usepackage{amsthm}         
\usepackage{microtype}      
\usepackage{graphicx}
\usepackage{subcaption}
\usepackage{float}
\usepackage{enumitem}
\usepackage{multirow}
\usepackage{tikz}
\usetikzlibrary{matrix, positioning, calc}
\graphicspath{ {./fig/} }

\newtheorem{theorem}{Theorem}[section]

\theoremstyle{definition}

\theoremstyle{remark}
\newtheorem{remark}[theorem]{Remark}

\title{Analytic Distribution of Classifier-Free Guidance for Schedule Design}

\author{
 Enze Jiang \\
  School of Mathematical Sciences\\
  Shanghai Jiao Tong University\\
  Shanghai, China \\
  \texttt{geguojia@sjtu.edu.cn} \\
   \And
 Zheng Ma \\
  School of Mathematical Sciences\\
  Shanghai Jiao Tong University\\
  Shanghai, China \\
  \texttt{zhengma@sjtu.edu.cn} \\
}

\begin{document}
\maketitle
\begin{abstract}
Classifier-free guidance (CFG) is the default mechanism for conditional generation in diffusion models, but the distribution sampled by its deterministic guided dynamics is not captured by the usual product-distribution heuristic $p_0^\omega q_0^{1-\omega}$. We analyze CFG through the probability flow ODE and derive exact analytic path-integral representations of the induced distributions for both constant and time-dependent guidance. The resulting formulas show that CFG modifies $p_{t_0}$ by an exponential path-integral correction, and that a time-dependent schedule enters this correction through the weight $\omega(t)-1$. This characterization explains how score discrepancies accumulate along sampling trajectories and motivates Distribution-Guided CFG (DG-CFG), a schedule that balances timestep contributions while accounting for signal strength and low-noise score-error amplification. A toy model with analytic scores closely verifies the predicted distributions. Across Stable Diffusion~1.5, Stable Diffusion~2.1, and Stable Diffusion~XL, DG-CFG yields a stronger diversity--fidelity trade-off and robustly mitigates the saturation and quality degradation caused by strong constant or heuristic guidance. Complete NFE experiments on Stable Diffusion~1.5 and Stable Diffusion~2.1 confirm that these gains persist across sampling budgets, while fixed-quality experiments on both backbones show that DG-CFG reaches target metrics with fewer sampling steps.
\end{abstract}


\section{Introduction}

Diffusion models~\cite{sohl2015deep,ho2020denoising,song2021scorebased} have emerged as a leading paradigm in generative modeling, achieving state-of-the-art results across image synthesis~\cite{dhariwal2021diffusion,rombach2022high}, video generation~\cite{ho2022video}, audio synthesis~\cite{kong2021diffwave}, and molecular conformation generation~\cite{xu2022geodiff}. Their operating principle is conceptually elegant: a forward stochastic differential equation (SDE) progressively corrupts data with noise, and a learned reverse process---parameterized by the score function---removes this noise to reconstruct realistic samples. This framework, unified by Song et al.~\cite{song2021scorebased}, admits two equivalent reverse-time descriptions: a stochastic SDE and a deterministic probability flow ODE that preserves the same marginal distributions at every timestep.

For conditional generation, \emph{classifier-free guidance} (CFG)~\cite{ho2021classifierfree} has become the de facto standard. Unlike classifier-based guidance~\cite{dhariwal2021diffusion}, which requires an auxiliary classifier, CFG blends the score functions of a conditional model $p_{\mathrm{data}}$ and a reference model $q_{\mathrm{data}}$ (typically the unconditional model) through a guidance strength parameter $\omega$:
\begin{equation}
    s_t^\omega(x) = \omega \nabla_x \log p_t(x) + (1-\omega) \nabla_x \log q_t(x).
\end{equation}
When $\omega > 1$, the conditional score is extrapolated beyond the reference score, often improving visual quality and prompt alignment but reducing sample diversity and causing supersaturation. Because it requires neither auxiliary networks nor additional training, this simple extrapolation underpins modern text-to-image systems~\cite{rombach2022high,saharia2022photorealistic,ramesh2022hierarchical,balaji2022ediff} and has also been adopted in video~\cite{ho2022video}, audio~\cite{liu2023audioldm}, and 3D synthesis~\cite{poole2023dreamfusion}.

Despite its empirical success, CFG leaves a fundamental distributional question: \emph{what distribution does it actually sample from?} Because $s_t^\omega$ is the score of $p_t^\omega q_t^{1-\omega}$ at each fixed time $t$, CFG is often assumed to sample from the normalized product distribution $\tilde{r}_0 \propto p_0^\omega q_0^{1-\omega}$. This pointwise score identity, however, does not determine the distribution transported by the full reverse-time dynamics. The blended score acts along the entire sampling trajectory, and its cumulative effect is not captured by the endpoint product alone. An exact characterization would clarify how CFG departs from this na\"{i}ve product picture and inform principled time-dependent guidance schedules. 

In this work, we characterize the distribution produced by CFG within the deterministic probability flow ODE framework. By deriving and solving the Fokker--Planck equations associated with the guided reverse process, we obtain exact analytic path-integral representations of the induced distribution under both constant and time-dependent guidance. These formulas reveal a simple organizing principle: guidance modifies the target distribution through an exponential path-integral correction whose effective temporal weight is governed by $\omega(t)-1$. We then use this principle to design a guidance schedule and validate both the theory and the schedule experimentally. Our main contributions are:

\begin{enumerate}[leftmargin=*]
    \item \textbf{Theory: exact analytic distribution of CFG.} We derive exact analytic path-integral representations of the CFG-induced distribution under constant guidance (Theorem~\ref{thm:ddim}) and time-dependent guidance (Theorem~\ref{thm:time-dep}). Both representations have the same structural form: the reference distribution multiplied by an exponential path-integral correction. In the time-dependent case, the correction is weighted by $\omega(t)-1$, with constant CFG recovered as the degenerate case $\omega(t)\equiv\omega$. The analysis also identifies an interaction potential determined by the discrepancy between the conditional and reference scores, showing how this discrepancy accumulates along sampling trajectories.

    \item \textbf{Schedule: a principled guidance schedule.} Under the VPSDE formulation, the temporal weight in the path integral contains the factor $\beta(t)(\omega(t)-1)$. Because the standard noise rate $\beta(t)$ varies by roughly two orders of magnitude, constant guidance assigns substantially greater weight to high-noise timesteps. Motivated by this temporal imbalance, we propose Distribution-Guided Classifier-Free Guidance (DG-CFG):
    \begin{equation*}
        \omega(t) = 1 + C \cdot (\bar{\omega}-1) \cdot \frac{(1-\bar{\alpha}_t)\sqrt{\bar{\alpha}_t}}{\beta(t)},
    \end{equation*}
    where $1/\beta(t)$ balances the diffusion-rate contribution, $\sqrt{\bar{\alpha}_t}$ emphasizes timesteps with stronger signal content, and $1-\bar{\alpha}_t$ suppresses score-error amplification in the low-noise regime.

    \item \textbf{Experiments: verification and validation.} In a toy model with analytic scores, empirical DDIM samples closely match the distributions predicted by Theorems~\ref{thm:ddim} and~\ref{thm:time-dep} under both constant and time-dependent guidance. Fixed-budget experiments on Stable Diffusion~1.5, Stable Diffusion~2.1, and Stable Diffusion~XL show that, at high guidance strengths, DG-CFG reduces saturation while providing a more favorable diversity--fidelity trade-off across all three backbones. On SD1.5 and SD2.1, comprehensive component ablations validate the schedule design, while evaluations across NFE settings and fixed-quality comparisons show that DG-CFG remains effective across sampling budgets and reaches target quality with fewer steps.
\end{enumerate}


\section{Preliminaries}
\label{sec:preliminaries}
\subsection{Diffusion Models}

Diffusion models generate data by learning to reverse a gradual noising process. In the forward direction, data are transported toward a simple prior---typically a standard Gaussian---according to a prescribed SDE
\begin{equation}
    dx = f(x,t)\,dt + g(t)\,dw,
    \label{eq:forward-sde}
\end{equation}
over $[0,T]$, where $w$ is a standard Wiener process and $p_t$ denotes the marginal density at time $t$. The associated reverse-time SDE~\cite{anderson1982reverse} is parameterized by the \emph{score function} $\nabla_x \log p_t(x)$:
\begin{equation}
    dx = \bigl(f(x,t) - g^2(t)\,\nabla_x \log p_t(x)\bigr)\,dt + g(t)\,d\bar{w},
    \label{eq:reverse-sde}
\end{equation}
where $\bar{w}$ is a reverse-time Wiener process. Song et al.~\cite{song2021scorebased} showed that the same marginal densities are also realized by the \emph{deterministic probability flow ODE}
\begin{equation}
    dx = \Bigl(f(x,t) - \frac{1}{2}g^2(t)\,\nabla_x \log p_t(x)\Bigr)\,dt,
    \label{eq:pf-ode}
\end{equation}
which preserves the same marginal distributions at every $t$.

Two canonical parameterizations are widely used. The \textbf{variance-preserving} SDE (VPSDE) formulation~\cite{song2021scorebased} sets
\begin{equation}
    f(x,t) = -\frac{1}{2}\beta(t)\,x, \qquad g(t) = \sqrt{\beta(t)},
\end{equation}
where $\beta(t) > 0$ is a predefined noise schedule. The corresponding perturbation kernel is
\begin{equation}
    p_t(x_t \mid x_0) = \mathcal{N}\bigl(x_t; \sqrt{\bar{\alpha}_t}\,x_0, \sigma_t^2 I\bigr),
    \qquad \bar{\alpha}_t = \exp\!\Bigl(-\int_0^t \beta(s)\,ds\Bigr), \quad \sigma_t^2 = 1 - \bar{\alpha}_t.
    \label{eq:perturbation}
\end{equation}
The associated deterministic sampler is DDIM~\cite{song2021denoising}. The \textbf{variance-exploding} SDE (VESDE) formulation sets $f \equiv 0$, yielding $p_t(x_t \mid x_0) = \mathcal{N}(x_t; x_0, \sigma_t^2 I)$; its deterministic counterpart underlies EDM-style samplers~\cite{karras2022elucidating}. Both are instances of~\eqref{eq:pf-ode} with different choices of $(f, g)$. Our theoretical analysis is developed within the probability flow ODE framework~\eqref{eq:pf-ode}, which admits a characteristic-based analysis and does not depend on a particular choice of $(f,g)$.

Near the data endpoint, the score may become singular as $t\to0$, particularly when the data distribution is concentrated near a low-dimensional manifold. We therefore state our results at an arbitrarily small cutoff time $t_0>0$, consistent with practical discrete samplers that sample to the smallest positive noise level followed by a final denoising step. 

\subsection{Classifier-Free Guidance}

Classifier-free guidance (CFG)~\cite{ho2021classifierfree} steers sampling without a separately trained classifier. Given a target distribution $p_{\mathrm{data}}$ and a reference distribution $q_{\mathrm{data}}$ (typically the unconditional distribution), CFG replaces the score in the reverse process with the affine combination
\begin{equation}
    s_t^\omega(x) = \omega \nabla_x \log p_t(x) + (1-\omega) \nabla_x \log q_t(x),
\end{equation}
where $\omega \in \mathbb{R}$ is the guidance strength and $p_t,q_t$ denote the forward-noised marginals of $p_{\mathrm{data}},q_{\mathrm{data}}$. Substituting $s_t^\omega$ into the probability flow ODE~\eqref{eq:pf-ode} yields the \emph{guided deterministic sampler}
\begin{equation}
    dx = \Bigl(f(x,t) - \frac{1}{2}g^2(t)\,s_t^\omega(x)\Bigr)\,dt.
    \label{eq:cfg-pf-ode}
\end{equation}
This deterministic sampler is the primary object of our study. When $\omega = 1$, it reduces to ordinary sampling from the target distribution. When $\omega > 1$, the target score is extrapolated away from the reference score, increasing conditional fidelity but often reducing diversity and causing color oversaturation.

\subsection{Fokker--Planck Equation}

The Fokker--Planck equation (FPE) governs the temporal evolution of the probability density associated with a diffusion process. For the forward SDE $dx = f(x,t)\,dt + g(t)\,dw$, the marginal density $p_t(x)$ satisfies
\begin{equation}
    \frac{\partial p_t(x)}{\partial t} = -\nabla \cdot \big(f(x,t)\,p_t(x)\big) + \frac{1}{2}g^2(t)\,\Delta p_t(x).
\end{equation}
The first term describes deterministic transport under the drift $f$, while the second describes stochastic diffusion. We use the FPE to compare the evolution of the normalized product reference density with that of the density transported by the guided probability flow ODE, leading to the exact path-integral representations in Section~\ref{sec:dist-cfg}.

\section{Related Work}
\label{sec:related}

\subsection{Diffusion Models}

Diffusion models generate data by reversing a progressive noising process~\cite{sohl2015deep,ho2020denoising}. Song et al.~\cite{song2021scorebased} unified diffusion probabilistic models and score-based generative models through continuous-time SDEs and derived the probability flow ODE, which follows deterministic dynamics while preserving the same time marginals. Closely related samplers such as DDIM enable efficient deterministic generation~\cite{song2021denoising}. Subsequent advances have improved sample quality~\cite{dhariwal2021diffusion,nichol2021improved}, accelerated inference~\cite{salimans2022progressive,lu2022dpm}, and extended diffusion models to latent spaces~\cite{rombach2022high,vahdat2021score}.

\subsection{Classifier-Free Guidance and Variants}

Ho and Salimans introduced classifier-free guidance (CFG)~\cite{ho2021classifierfree} as an alternative to classifier-based guidance~\cite{dhariwal2021diffusion}. CFG learns conditional and unconditional score estimates within a single model, eliminating the need for an external classifier. It is widely used in text-to-image systems~\cite{rombach2022high,saharia2022photorealistic,ramesh2022hierarchical,balaji2022ediff}, with extensions to video~\cite{ho2022video}, audio~\cite{liu2023audioldm}, and 3D synthesis~\cite{poole2023dreamfusion}. Several variants address its failure modes: CFG++ applies guidance to the denoised estimate while using the unconditional prediction for re-noising~\cite{chung2025cfg++}; autoguidance replaces the unconditional reference with a weaker version of the model~\cite{karras2024guiding}; and adaptive projected guidance decomposes the CFG update relative to the conditional prediction and down-weights its parallel component~\cite{sadat2024eliminating}.

\subsection{Theoretical Understanding of CFG}

Recent work has analyzed CFG from several complementary perspectives. Bradley and Nakkiran~\cite{bradley2024classifier} showed that CFG does not directly sample a gamma-powered product distribution under DDPM or DDIM and interpreted its SDE limit as a predictor-corrector procedure. Lehman Pavasovic et al.~\cite{lehman2025understanding} established recovery of the target conditional distribution in an infinite-dimensional limit and characterized finite-dimensional mean and variance corrections. Jin et al.~\cite{jin2026stagewise} identified direction-shift, mode-separation, and concentration stages under multimodal conditionals. Skreta et al.~\cite{skreta2025feynmankac} introduced Feynman--Kac correctors for product-distribution sampling.

\subsection{Guidance Schedule Design}

Time-dependent CFG schedules redistribute guidance across the sampling trajectory. Interval guidance applies CFG only within a selected noise range, motivated by its limited utility at the highest and lowest noise levels~\cite{kynkaanniemi2024applying}. $\beta$--CFG instead uses a beta-shaped schedule that suppresses guidance near both endpoints~\cite{malarz2025classifier}. C$^2$FG derives bounds on the conditional--unconditional score discrepancy and uses an exponential-decay control function~\cite{gao2026c2fg}.

\section{Distribution of Classifier-Free Guidance}
\label{sec:dist-cfg}

We now analyze the distribution produced by CFG under the deterministic probability flow ODE~\eqref{eq:cfg-pf-ode}, without committing to a specific parameterization of $f(x,t)$ and $g(t)$. We first formulate the sampling problem (Section~\ref{sec:problem}), then derive the Fokker--Planck equation for the normalized product reference distribution (Section~\ref{sec:fp-analysis}). We next obtain an exact analytic path-integral representation under constant guidance (Section~\ref{sec:exact-dist}) and extend the result to time-dependent guidance schedules (Section~\ref{sec:time-dep}).

\subsection{Problem Formulation}
\label{sec:problem}

Let $p_{\mathrm{data}}$ and $q_{\mathrm{data}}$ be two probability densities on $\mathbb{R}^n$. In the CFG setting, $p_{\mathrm{data}}$ represents the target conditional distribution and $q_{\mathrm{data}}$ the reference distribution, typically the unconditional distribution. Consider the forward noising SDE~\eqref{eq:forward-sde} on $[0,T]$. Let $p_t$ denote the marginal density at time $t$ initialized from $p_0 = p_{\mathrm{data}}$; it evolves according to
\begin{equation}
    \frac{\partial p_t}{\partial t} = -\nabla \cdot (f p_t) + \frac{1}{2} g^2 \Delta p_t.
    \label{eq:fp-p}
\end{equation}
Similarly, let $q_t$ be the marginal density initialized from $q_0=q_{\mathrm{data}}$ and governed by the same Fokker--Planck equation. In practical diffusion sampling, reverse processes associated with different data distributions are typically initialized from the same Gaussian prior at time $T$. 

As noted in Section~\ref{sec:preliminaries}, the score may become singular near $t=0$. We therefore introduce a small stopping time $t_0>0$ and analyze the reverse-time horizon $[t_0,T]$. The guided reverse process maps the terminal reference density at time $T$ to a density $\hat{r}_{t_0}$ at time $t_0$. All theoretical results are stated for fixed $t_0>0$, with $p_{t_0}$ and $q_{t_0}$ serving as small-noise approximations to their respective data distributions.

Classifier-free guidance modifies the deterministic probability flow ODE by replacing the target score with the blended score $s_t^\omega$:
\begin{equation}
    dx = \Bigl(f(x,t) - \frac{1}{2}g^2(t)\,
          \bigl(\omega \nabla_x \log p_t(x) + (1-\omega) \nabla_x \log q_t(x)\bigr)\Bigr)\,dt.
    \label{eq:cfg-pf-ode2}
\end{equation}

To analyze~\eqref{eq:cfg-pf-ode2}, define the normalized product reference distribution
\begin{equation}
    \tilde{r}_t(x) := \frac{1}{Z_t} p_t^\omega(x)\,q_t^{1-\omega}(x) = \frac{1}{Z_t} r_t(x),
    \label{eq:rt}
\end{equation}
where $r_t(x)=p_t^\omega(x)q_t^{1-\omega}(x)$ and $Z_t=\int_{\mathbb{R}^n}r_t(x)\,dx$ is the normalizing constant which is assumed to be finite. A key identity follows immediately:
\begin{equation}
    \nabla_x \log \tilde{r}_t = \nabla_x \log r_t = \omega \nabla_x \log p_t + (1-\omega) \nabla_x \log q_t,
    \label{eq:score-id}
\end{equation}
Consequently, the guided sampler~\eqref{eq:cfg-pf-ode2} can be written as $dx = (f - \frac{1}{2}g^2\,\nabla_x \log \tilde{r}_t)\,dt$.

Let $\hat{r}_t$ denote the marginal density transported by the CFG reverse process on $[t_0,T]$, initialized from the terminal reference density $\tilde{r}_T$ at time $T$. Under the usual terminal condition $p_T\simeq q_T$, this density coincides with the common diffusion prior. Our goal is to characterize $\hat{r}_{t_0}$, the distribution from which deterministic CFG actually samples.

\subsection{Fokker--Planck Equation for the Reference Distribution}
\label{sec:fp-analysis}

We first derive the evolution equation for the normalized product reference distribution $\tilde{r}_t$. Although its score coincides with the blended CFG score by~\eqref{eq:score-id}, its density evolution is not the standard forward Fokker--Planck equation; an additional potential term appears from the interaction between the two score functions.

\begin{theorem}[Fokker--Planck equation for $\tilde{r}_t$]
\label{thm:fp}
Fix $0<t_0<T$ and $\omega\in\mathbb{R}$. Suppose $p,q\in C^{1,2}([t_0,T]\times\mathbb{R}^n)$ with $p_t(x)>0$ and $q_t(x)>0$, and that both densities satisfy the Fokker--Planck equation~\eqref{eq:fp-p}. Let $r_t=p_t^\omega q_t^{1-\omega}$ and $Z_t=\int_{\mathbb{R}^n}r_t(x)\,dx$, where $0<Z_t<\infty$ on $[t_0,T]$. For $B_R=\{x:\|x\|\le R\}$, assume the boundary decay condition $\lim_{R\to\infty}\int_{\partial B_R}\big(\|f(x,t)\|\,r_t(x)+\|\nabla r_t(x)\|\big)\,dS(x)=0$ for every $t\in[t_0,T]$.
Then the normalized product distribution $\tilde{r}_t=Z_t^{-1}p_t^\omega q_t^{1-\omega}$ satisfies
\begin{equation}
    \frac{\partial \tilde{r}_t}{\partial t}
    = -\nabla \cdot (f \tilde{r}_t) + \frac{1}{2} g^2 \Delta \tilde{r}_t
    - \frac{1}{2} g^2 \big(V_t - \mathbb{E}_{\tilde{r}_t}[V_t]\big) \tilde{r}_t,
    \label{eq:fp-tilde}
\end{equation}
where the \emph{interaction potential} is
\begin{equation}
    V_t(x) = \omega(\omega-1)\,\big\|\nabla_x \log p_t(x) - \nabla_x \log q_t(x)\big\|_2^2.
    \label{eq:V}
\end{equation}
\end{theorem}

\begin{proof}
See Appendix~\ref{sec:proof-fp}.
\end{proof}

The interaction potential $V_t$ is proportional to the squared local discrepancy between the target and reference scores. It vanishes for $\omega\in\{0,1\}$, corresponding to $\tilde{r}_t=q_t$ and $\tilde{r}_t=p_t$, respectively. For $\omega>1$, the centered term $V_t-\mathbb{E}_{\tilde{r}_t}[V_t]$ introduces a spatially varying source--sink correction to the Fokker--Planck dynamics of $\tilde{r}_t$. This already indicates that matching the instantaneous CFG score to $\nabla\log\tilde{r}_t$ is not enough to conclude that deterministic CFG samples from the normalized product distribution $\tilde{r}_{t_0}$ at sampling time.

\subsection{Exact Sampling Distribution}
\label{sec:exact-dist}

We now turn to the central result: an exact analytic path-integral characterization of the distribution induced by deterministic CFG.

\begin{theorem}[CFG distribution under deterministic sampling]
\label{thm:ddim}
Suppose $p,q\in C^{1,2}([t_0,T]\times\mathbb{R}^n)$ with $p_t(x)>0$ and $q_t(x)>0$, and that both densities satisfy the Fokker--Planck equation~\eqref{eq:fp-p}. Assume that the characteristic ODE below generates a unique $C^1$ flow and that all path integrals appearing in the statement are finite. Let $\hat{r}_{t_0}$ be the distribution obtained by evolving the terminal density $\tilde{r}_T$ backward to $t_0$ via the deterministic ODE $dx = (f - \frac{1}{2}g^2\nabla_x\log\tilde{r}_t)\,dt$. Then
\begin{equation}
    \hat{r}_{t_0}(x) = \frac{1}{Z'}\, p_{t_0}^\omega(x)\,q_{t_0}^{1-\omega}(x)\,
    \exp\!\left(-\frac{1}{2}\,\omega(\omega-1)\int_{t_0}^T g^2(s)\,
    \big\|\nabla_x \log p_s(X_s^x) - \nabla_x \log q_s(X_s^x)\big\|_2^2\,ds\right),
    \label{eq:ddim-main}
\end{equation}
where $X_s^x$ solves the characteristic ODE
\begin{equation}
    \frac{dX_s^x}{ds} = f(X_s^x, s) - \frac{1}{2}g^2(s)\,\nabla_x \log \tilde{r}_s(X_s^x),
    \qquad X_{t_0}^x = x,
    \label{eq:char-ode}
\end{equation}
and $Z'$ is a normalization constant.

Equivalently, in terms of the log-ratio $L_t(x) = \log p_t(x) - \log q_t(x)$, the same distribution can be written as
\begin{equation}
    \begin{aligned}
    \hat{r}_{t_0}(x)
    &= \frac{1}{Z'}\, p_{t_0}(x)\,
    \left(\frac{p_T(X_T^x)}{q_T(X_T^x)}\right)^{\omega-1} \\
    &\quad \times
    \exp\!\left(-\frac{1}{2}(\omega-1)\int_{t_0}^T g^2(s)\,
    \Big(\Delta L_s(X_s^x) + \nabla L_s(X_s^x) \cdot \nabla\log p_s(X_s^x)\Big)\,ds\right).
    \end{aligned}
    \label{eq:ddim-laplace-general}
\end{equation}

Under the usual terminal condition $\log p_T=\log q_T$, the terminal ratio in~\eqref{eq:ddim-laplace-general} is equal to one, and the distribution reduces to
\begin{equation}
    \hat{r}_{t_0}(x) = \frac{1}{Z'}\, p_{t_0}(x)\,
    \exp\!\left(-\frac{1}{2}(\omega-1)\int_{t_0}^T g^2(s)\,
    \Big(\Delta L_s(X_s^x) + \nabla L_s(X_s^x) \cdot \nabla\log p_s(X_s^x)\Big)\,ds\right).
    \label{eq:ddim-laplace}
\end{equation}
\end{theorem}

\begin{proof}
See Appendix~\ref{sec:proof-ddim}.
\end{proof}

Theorem~\ref{thm:ddim} gives an exact analytic representation of the CFG-induced distribution at the stopping time $t_0$. The first form~\eqref{eq:ddim-main} shows that the na\"{i}ve product distribution is corrected by an exponential path integral of the score discrepancy along the characteristic trajectory. For $\omega>1$, trajectories that pass through regions where the target and reference scores disagree strongly are penalized by this factor. The equivalent form~\eqref{eq:ddim-laplace-general} gives a complementary view: after rewriting relative to $p_{t_0}$, the departure from the target distribution is governed linearly by the guidance deviation $\omega-1$, up to the terminal boundary factor. Under the usual terminal condition $\log p_T=\log q_T$, this boundary factor vanishes and the simplified form~\eqref{eq:ddim-laplace} is obtained.

Several consequences follow. First, deterministic CFG does \emph{not} generally sample from the normalized product distribution $\tilde{r}_{t_0}\propto p_{t_0}^\omega q_{t_0}^{1-\omega}$; an additional path-integral correction is required. This correction vanishes in the limiting cases $\omega\in \{0,1\}$, and it becomes distributionally irrelevant in degenerate settings where the score discrepancy is spatially constant, so that its contribution can be absorbed into the normalization. Second, after accounting for the terminal boundary factor, the representation expresses the induced distribution as $p_{t_0}$ multiplied by an exponential path-integral factor weighted by $\omega-1$; under the usual terminal condition the boundary factor is absent. This representation separates the base target distribution from the cumulative effect of guidance, and will be the basis for the schedule design in Section~\ref{sec:schedule-design}.

\subsection{Time-Dependent Guidance}
\label{sec:time-dep}

Theorem~\ref{thm:ddim} considers constant guidance. We next extend the analysis to time-dependent guidance $\omega(t)$.

\begin{theorem}[CFG with time-dependent guidance]
\label{thm:time-dep}
Suppose $p,q\in C^{1,2}([t_0,T]\times\mathbb{R}^n)$ with $p_t(x)>0$ and $q_t(x)>0$, and that both densities satisfy the Fokker--Planck equation~\eqref{eq:fp-p}. Let $\omega\in W^{1,1}([t_0,T])$, define $r_t(x)=p_t^{\omega(t)}(x)q_t^{1-\omega(t)}(x)$ and $Z_t=\int_{\mathbb{R}^n}r_t(x)\,dx$, and set $\tilde{r}_t=Z_t^{-1}r_t$, where $0<Z_t<\infty$ on $[t_0,T]$. Assume that the characteristic ODE below generates a unique $C^1$ flow and that all path integrals appearing in the statement are finite. Let $L_t(x)=\log p_t(x)-\log q_t(x)$. The distribution obtained by evolving the terminal density $\tilde{r}_T$ backward to $t_0$ via the deterministic ODE is
\begin{equation}
    \begin{aligned}
    \hat{r}_{t_0}(x)
    &= \frac{1}{Z'}\, p_{t_0}(x)\,
    \left(\frac{p_T(X_T^x)}{q_T(X_T^x)}\right)^{\omega(T)-1} \\
    &\quad \times
    \exp\!\left(-\int_{t_0}^T \frac{1}{2}g^2(s)\,(\omega(s)-1)\,
    \Big(\Delta L_s(X_s^x) + \nabla L_s(X_s^x)\cdot\nabla\log p_s(X_s^x)\Big)\,ds\right),
    \end{aligned}
    \label{eq:time-dep-simple-general}
\end{equation}
where $Z'$ is a normalization constant and $X_s^x$ is the solution of
\begin{equation}
    \frac{dX_s^x}{ds} = f(X_s^x, s) - \frac{1}{2}g^2(s)\,\nabla_x \log \tilde{r}_s(X_s^x),
    \qquad X_{t_0}^x = x.
    \label{eq:char-ode-time}
\end{equation}

Under the usual terminal condition $\log p_T=\log q_T$, the terminal ratio in~\eqref{eq:time-dep-simple-general} is equal to one, and the distribution reduces to
\begin{equation}
    \hat{r}_{t_0}(x) = \frac{1}{Z'}\, p_{t_0}(x)\,
    \exp\!\left(-\int_{t_0}^T \frac{1}{2}g^2(s)\,(\omega(s)-1)\,
    \Big(\Delta L_s(X_s^x) + \nabla L_s(X_s^x)\cdot\nabla\log p_s(X_s^x)\Big)\,ds\right),
    \label{eq:time-dep-simple}
\end{equation}
\end{theorem}

\begin{proof}
See Appendix~\ref{sec:proof-time-dep}.
\end{proof}

Equation~\eqref{eq:time-dep-simple-general} shows that time-dependent guidance has the same structural form as constant guidance: after accounting for the terminal boundary factor, the induced distribution is represented as $p_{t_0}$ multiplied by an exponential path-integral correction. The only schedule-dependent factor in this simplified expression is $\omega(t)-1$, which controls how strongly each timestep contributes to the correction. Under the usual terminal condition, the boundary factor disappears and~\eqref{eq:time-dep-simple} is recovered.

Constant guidance is recovered as the special case $\omega(t)\equiv\bar{\omega}$, for which~\eqref{eq:time-dep-simple} reduces to Theorem~\ref{thm:ddim}. This observation places standard CFG within a broader time-varying framework. As we shall argue in Section~\ref{sec:schedule-design} and validate experimentally in Section~\ref{sec:experiments}, constant guidance can be suboptimal because the path-integral weight is also modulated by the noise schedule.

\begin{remark}
    The preceding analysis does \emph{not} rely on specific choices of $f(x,t)$, $g(t)$, or the data distributions $p_{\mathrm{data}}$ and $q_{\mathrm{data}}$. The derivation uses only the forward SDE~\eqref{eq:forward-sde} and a deterministic guided reverse process of the form
\begin{equation}
    dx = \Bigl(f(x,t) - \frac{1}{2}g^2(t)\,\bigl(\omega(t)\nabla_x\log p_t + (1-\omega(t))\nabla_x\log q_t\bigr)\Bigr)\,dt.
    \label{eq:general-sampler}
\end{equation}
In the usual CFG setting, $q_{\mathrm{data}}$ is the unconditional distribution and $p_{\mathrm{data}}=q_{\mathrm{data}}(\cdot\mid c)$ is the conditional distribution for condition $c$. The form~\eqref{eq:general-sampler} therefore covers deterministic guided samplers with possibly time-dependent guidance strength. In particular, it includes two widely used parameterizations:
\begin{itemize}[leftmargin=*]
    \item \textbf{VPSDE:} $f(x,t) = -\frac{1}{2}\beta(t)\,x$, $g(t) = \sqrt{\beta(t)}$. $\beta(t)$ is the continuous variance rate of the forward process, typically linear in $t$~\cite{song2021denoising}.
    \item \textbf{VESDE:} $f \equiv 0$, with $g(t)=\sqrt{2\sigma(t)\dot{\sigma}(t)}$, where $\sigma(t)$ specifies the noise scale of the perturbation kernel $p_t(x_t \mid x_0)=\mathcal{N}(x_t; x_0, \sigma_t^2 I)$~\cite{karras2022elucidating}.
\end{itemize}
\end{remark}

To maintain a coherent narrative, the remainder of the paper---schedule design (Section~\ref{sec:schedule-design}) and experiments (Section~\ref{sec:experiments})---is presented within the VPSDE framework. The schedule-design principle, however, is not tied to this parameterization: the same path-integral perspective can be naturally adapted to VESDE.

\section{Our Method}
\label{sec:schedule-design}

We now translate the path-integral characterization into a practical guidance schedule. Section~\ref{sec:dist-cfg} shows, through the path-integral representation, that $\hat{r}_{t_0}$ differs from $p_{t_0}$ by an exponential path-integral correction. Under VPSDE, $g^2(s)=\beta(s)$, and the schedule-dependent contribution to the exponent contains
\begin{equation}
    \beta(s)\,(\omega(s)-1)\,
    \Big(\Delta L_s(X_s^x) + \nabla L_s(X_s^x)\cdot\nabla\log p_s(X_s^x)\Big).
    \label{eq:integrand}
\end{equation}
This expression exposes three sources of temporal non-uniformity: the known diffusion coefficient $\beta(s)$, which can vary substantially along the trajectory; the explicit guidance deviation $\omega(s)-1$, which determines how guidance is allocated over time; and the geometry-dependent term $\Delta L_s+\nabla L_s\cdot\nabla\log p_s$.

\subsection{Motivation: Why Constant Guidance Is Suboptimal.}

With constant guidance $\omega(s) \equiv \bar{\omega}$, the schedule-dependent weight reduces to $\beta(s)(\bar{\omega}-1)$. In the standard VPSDE with a linear continuous noise rate $\beta(s) = \beta_{\min} + s(\beta_{\max}-\beta_{\min})/T$ (typical values: $\beta_{\min} = 0.1$ and $\beta_{\max} = 20$), the high-noise end ($s \approx T$) receives a coefficient roughly $200$ times larger than the low-noise end ($s \approx t_0$). Thus, a nominally constant guidance strength does not produce a constant contribution to the path integral; the effective per-timestep guidance weight is modulated by two orders of magnitude.

This imbalance has visible consequences. The high-noise phase primarily determines coarse semantic layout, global structure, and color, whereas the low-noise phase refines details and textures. When constant CFG is combined with a rapidly increasing $\beta(s)$, guidance is concentrated disproportionately in the high-noise regime. This can encourage repeated coarse layouts and color bias, reducing global diversity and increasing oversaturation. Meanwhile, the low-noise regime receives comparatively little effective guidance, which can leave fine details under-refined. These effects become more obvious at large guidance strengths because the entire path-integral correction is amplified by $\bar{\omega}-1$.

\subsection{Distribution-Guided Classifier-Free Guidance (DG-CFG).}

We construct the guidance schedule in three incremental steps. Each step addresses a specific source of imbalance suggested by the path-integral form, and the resulting components are validated separately in the ablation study (Section~\ref{sec:ablation}).

\medskip\noindent
\textbf{Step 1: Time balancing.}
As established above, $\beta(s)$ varies by two orders of magnitude and causes constant CFG to overweight the high-noise regime. We first compensate for this known coefficient by scaling the guidance deviation inversely with $\beta(t)$:
$\omega(t) = 1 + C_1 \cdot (\bar{\omega}-1) / \beta(t)$.
This makes the schedule-dependent factor $\beta(t)(\omega(t)-1)$ approximately uniform in the path integral, reducing the dominance of high-noise stages and helping preserve diversity while suppressing saturation.

\medskip\noindent
\textbf{Step 2: Signal-content weighting.}
Time balancing treats every timestep as equally informative. In reality, the signal content of $\mathbf{x}_t = \sqrt{\bar{\alpha}_t}\,\mathbf{x}_0 + \sigma_t\,\boldsymbol{\epsilon}$ changes systematically: near $t=T$, $\sqrt{\bar{\alpha}_t}\approx0$ and the sample is nearly pure noise, whereas near $t=t_0$, $\sqrt{\bar{\alpha}_t}\approx1$ and the sample already contains most of the clean signal. Multiplying by $\sqrt{\bar{\alpha}_t}$ shifts guidance toward timesteps where the clean signal is present and the conditional score is more informative. The schedule becomes
$\omega(t) = 1 + C_2 \cdot (\bar{\omega}-1) \cdot \sqrt{\bar{\alpha}_t} / \beta(t)$.

\medskip\noindent
\textbf{Step 3: Score approximation error mitigation.}
The first two steps increase the relative influence of the low-noise regime. In practical diffusion models, however, the score is approximated by a neural network through $\nabla_{\mathbf{x}}\log p_t(\mathbf{x})\approx -\boldsymbol{\epsilon}_\theta(\mathbf{x},t)/\sigma_t$. Since $\sigma_t=\sqrt{1-\bar{\alpha}_t}$ becomes small near $t_0$, prediction errors in $\boldsymbol{\epsilon}_\theta$ are amplified in the score estimate; in squared score quantities, the amplification scales as $1/\sigma_t^2$. Overweighting this regime can therefore push samples away from the desired distribution and degrade realism. We suppress this low-noise error amplification by multiplying the schedule by $\sigma_t^2=1-\bar{\alpha}_t$, yielding the final correction factor.

\medskip\noindent
\textbf{Final schedule.}
Combining all three factors yields Distribution-Guided Classifier-Free Guidance (DG-CFG):
\begin{equation}
    \boxed{\;\omega(t) = 1 + C \cdot (\bar{\omega} - 1) \cdot
           \frac{(1-\bar{\alpha}_t)\,\sqrt{\bar{\alpha}_t}}{\beta(t)}\;},
    \label{eq:final-schedule}
\end{equation}
where $\bar{\omega} \geq 1$ is the nominal guidance strength and $C$ normalizes the integrated guidance deviation to match constant CFG:
\begin{equation}
    C = \frac{T}{\int_{t_0}^T \frac{(1-\bar{\alpha}_s)\sqrt{\bar{\alpha}_s}}{\beta(s)}\,ds}.
    \label{eq:C-def-final}
\end{equation}

The resulting factor $(1-\bar{\alpha}_t)\sqrt{\bar{\alpha}_t}/\beta(t)$ vanishes near both endpoints: near $t=t_0$, where $1-\bar{\alpha}_t\approx0$, and near $t=T$, where $\sqrt{\bar{\alpha}_t}\approx0$. DG-CFG therefore concentrates guidance at intermediate noise levels, where the sample contains meaningful signal while the score estimate remains sufficiently stable. This middle-focused profile is consistent with prior empirical observations that allocating stronger guidance to intermediate stages can improve generation quality~\cite{kynkaanniemi2024applying,malarz2025classifier}; here, it emerges from correcting the temporal weighting in the path-integral representation.

\section{Experiments}
\label{sec:experiments}

We evaluate our theory and the resulting schedule in two stages. First, we use a controlled toy model in which all relevant scores and path-integral terms admit analytic closed forms, enabling direct verification of Theorems~\ref{thm:ddim} and~\ref{thm:time-dep} (Section~\ref{sec:exp-toy}). Second, we test whether the same path-integral insight leads to practical gains in large-scale neural samplers by evaluating DG-CFG on Stable Diffusion~1.5, Stable Diffusion~2.1, and Stable Diffusion~XL (Section~\ref{sec:exp-sd}).

\subsection{Discrete Distribution on a Circle}
\label{sec:exp-toy}

\subsubsection{Setup}

We use the standard DDPM noise schedule~\cite{ho2020denoising} on the time interval $[0,T]$ with $T=1$, discretized into $1000$ steps. The continuous noise rate $\beta(t)$ increases linearly from $0.1$ to $20$, and we set $t_0=0.001$.

The unconditional data distribution $q_{\mathrm{data}}$ is uniform over $K=10$ equally spaced point masses on a circle of radius $R=2$. The conditional distribution $p_{\mathrm{data}}$ assigns likelihood weights $\boldsymbol{\ell}=[2,2,2,2,1,1,1,1,1,1]$ to the same support and then renormalizes, so that points $0$--$3$ receive twice the mass of points $4$--$9$. Under the forward Gaussian perturbation, both $p_t$ and $q_t$ become Gaussian mixtures, and their scores and score divergences admit closed-form expressions.

We compare two guidance strategies at $\bar{\omega} \in \{3, 5, 7, 9\}$:
\begin{itemize}[leftmargin=*]
    \item \textbf{Constant CFG:} $\omega(t) \equiv \bar{\omega}$.
    \item \textbf{Time-dependent CFG:}
    $\omega(t) = 1 + C \cdot (\bar{\omega}-1) \cdot \sqrt{\bar{\alpha}_t} / \beta(t)$, with $C = 1/0.351$. (The $(1-\bar{\alpha}_t)$ correction is omitted here because the toy model uses analytic score functions; the $1/\sigma_t^2$ error amplification that motivates Step~3 does not arise.)
\end{itemize}

\medskip\noindent
For each setting, we draw $N=200{,}000$ samples using 50-step DDIM sampling from $\mathcal{N}(\mathbf{0},I)$. Each sample is assigned to its nearest data point, yielding the empirical label distribution $\hat{P}^{\mathrm{samp}}_k$. The theoretical label probabilities $\hat{P}^{\mathrm{th}}_k$ are computed from Theorem~\ref{thm:ddim} for constant guidance and from Theorem~\ref{thm:time-dep} for time-dependent guidance, evaluated at each support point $\mathbf{x}^k$. Agreement is measured by total variation (TV) distance and mean absolute error (MAE). 

As an additional diagnostic, we compute the \emph{Min-Max-ratio} from the empirical sample distribution over the four high-weight points ($k=0,1,2,3$), defined as $\min_{k\in\{0,1,2,3\}}\hat{P}^{\mathrm{samp}}_k \,/\, \max_{k\in\{0,1,2,3\}}\hat{P}^{\mathrm{samp}}_k$. This metric captures mode imbalance within the conditioned subset: a value near $1$ indicates balanced coverage of the four preferred modes, whereas a value near $0$ indicates that at least one preferred mode has nearly disappeared. More details are shown in Appendix~\ref{sec:appendix-toy}.

\subsubsection{Results}

Table~\ref{tab:verification} and Figure~\ref{fig:toy-label-compare} report the results. Across both guidance strategies and all guidance strengths, the distributions of CFG closely match the theoretical predictions, with TV distances typically around $1$--$2\%$ and small MAE. This agreement directly validates the distributional formulas in Theorems~\ref{thm:ddim} and~\ref{thm:time-dep}.

Beyond verifying the formulas, the toy model also illustrates the effect of time-dependent weighting. The time-dependent schedule yields substantially higher Min-Max-ratios and more balanced per-label distributions than constant CFG. The difference is most pronounced at strong guidance ($\bar{\omega}=9$), where constant CFG nearly extinguishes the peripheral preferred modes (points 0 and 3), while the time-dependent schedule preserves visibly more balanced coverage. This behavior is consistent with the path-integral insight used to design the schedule in Section~\ref{sec:schedule-design}.

\begin{table}[ht]
\centering
\caption{Toy-model verification of the predicted induced distributions. TV distance and MAE compare the empirical label distribution ($N=200{,}000$, 50-step DDIM) with the theoretical predictions from Theorems~\ref{thm:ddim} and~\ref{thm:time-dep}. The Min-Max-ratio, defined as $\min_{k\in\{0,1,2,3\}}\hat{P}^{\mathrm{samp}}_k / \max_{k\in\{0,1,2,3\}}\hat{P}^{\mathrm{samp}}_k$, measures balance among the four preferred modes.}
\label{tab:verification}
\begin{tabular}{lcccc}
\toprule
Schedule & $\bar{\omega}$ & TV & MAE & Min-Max-ratio \\
\midrule
Constant CFG          & 3 & 0.0135 & 0.00269 & 0.665 \\
Constant CFG          & 5 & 0.0189 & 0.00377 & 0.368 \\
Constant CFG          & 7 & 0.0159 & 0.00317 & 0.193 \\
Constant CFG          & 9 & 0.0212 & 0.00425 & 0.099 \\
\midrule
Time-dependent CFG   & 3 & 0.0053 & 0.00106 & 0.959 \\
Time-dependent CFG   & 5 & 0.0076 & 0.00153 & 0.981 \\
Time-dependent CFG   & 7 & 0.0116 & 0.00232 & 0.898 \\
Time-dependent CFG   & 9 & 0.0159 & 0.00317 & 0.808 \\
\bottomrule
\end{tabular}
\end{table}

\begin{figure}[ht]
\centering
\small
\begin{tikzpicture}[
    img/.style={inner sep=0pt, outer sep=0pt},
    mat/.style={matrix of nodes, nodes=img, inner sep=0pt, outer sep=0pt, column sep=4pt, row sep=8pt}
]

\node[img] (top_all) {\includegraphics[width=0.35\textwidth]{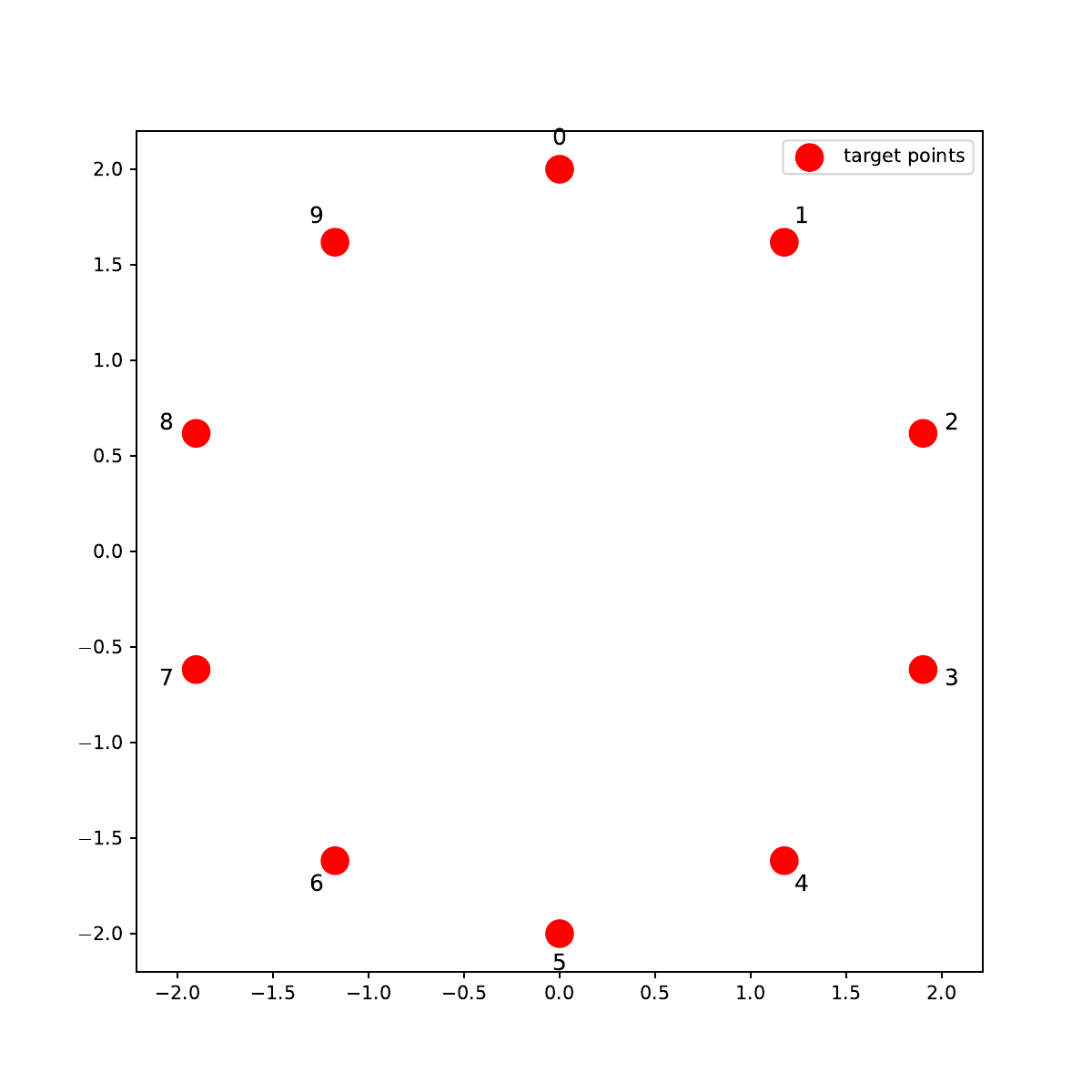}};
\node[anchor=north, yshift=5pt] at (top_all.south) {Discrete point layout};

\matrix (plots) [mat, below=12pt of top_all]
{
\includegraphics[width=0.40\textwidth]{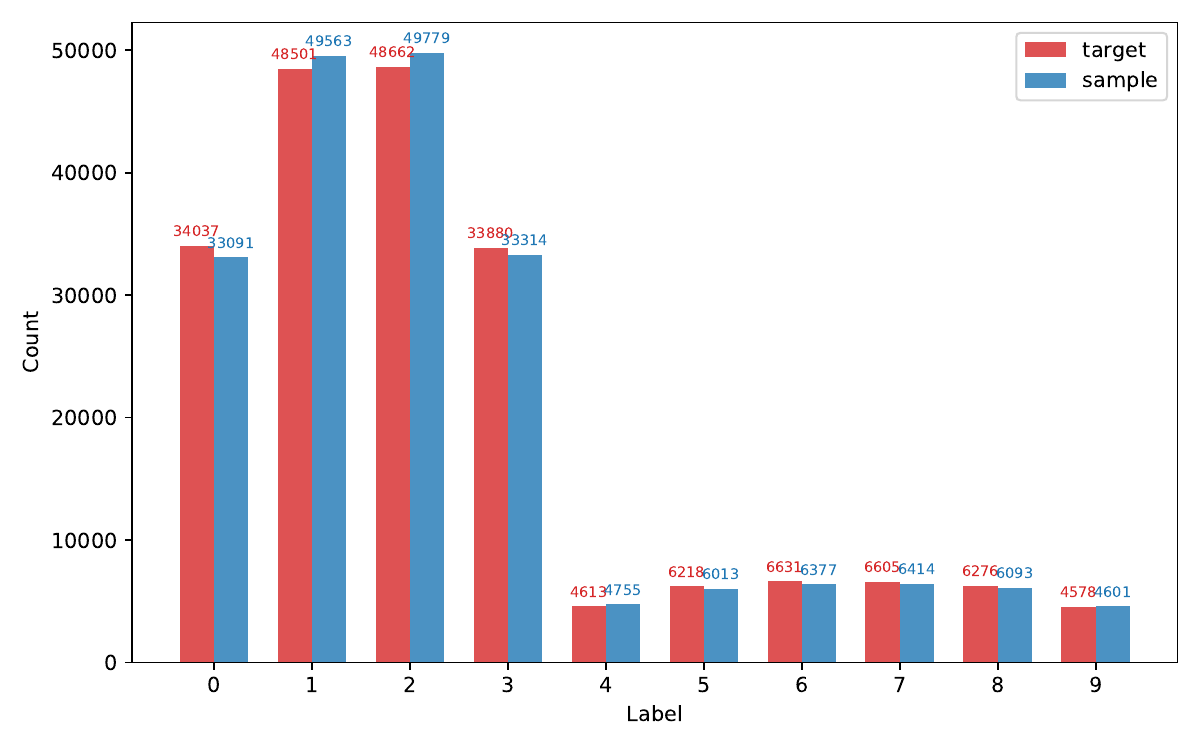} &
\includegraphics[width=0.40\textwidth]{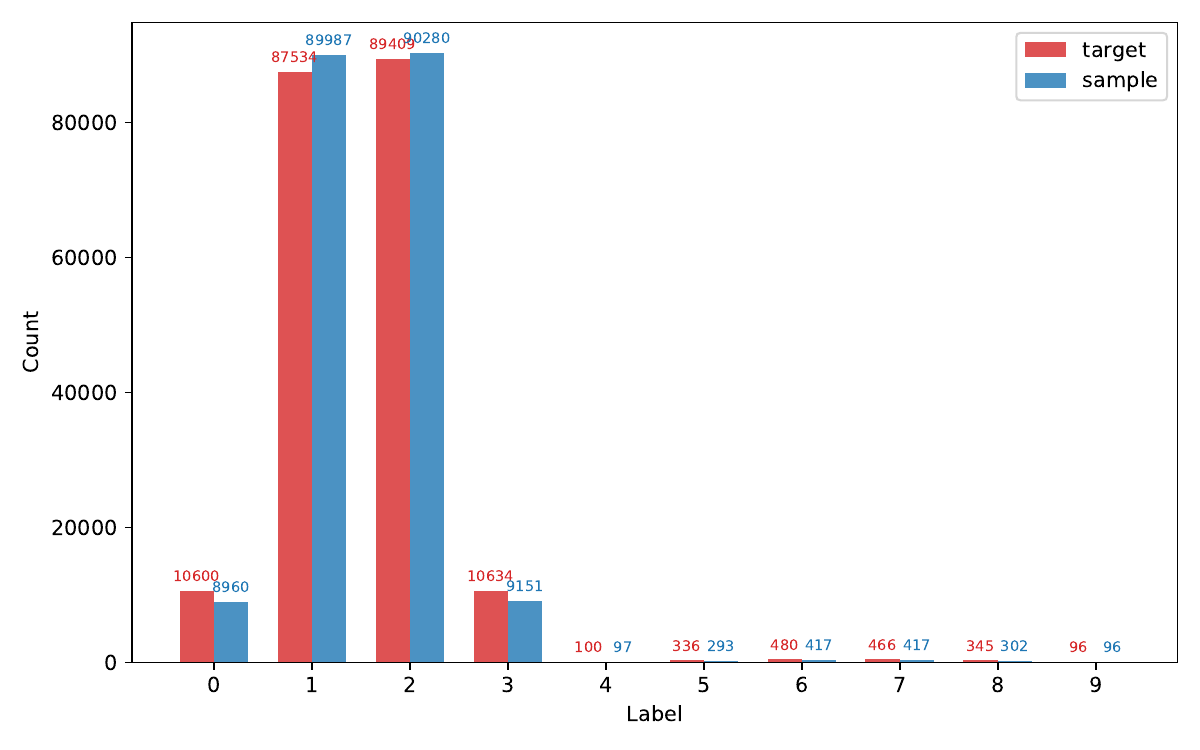} \\
\includegraphics[width=0.40\textwidth]{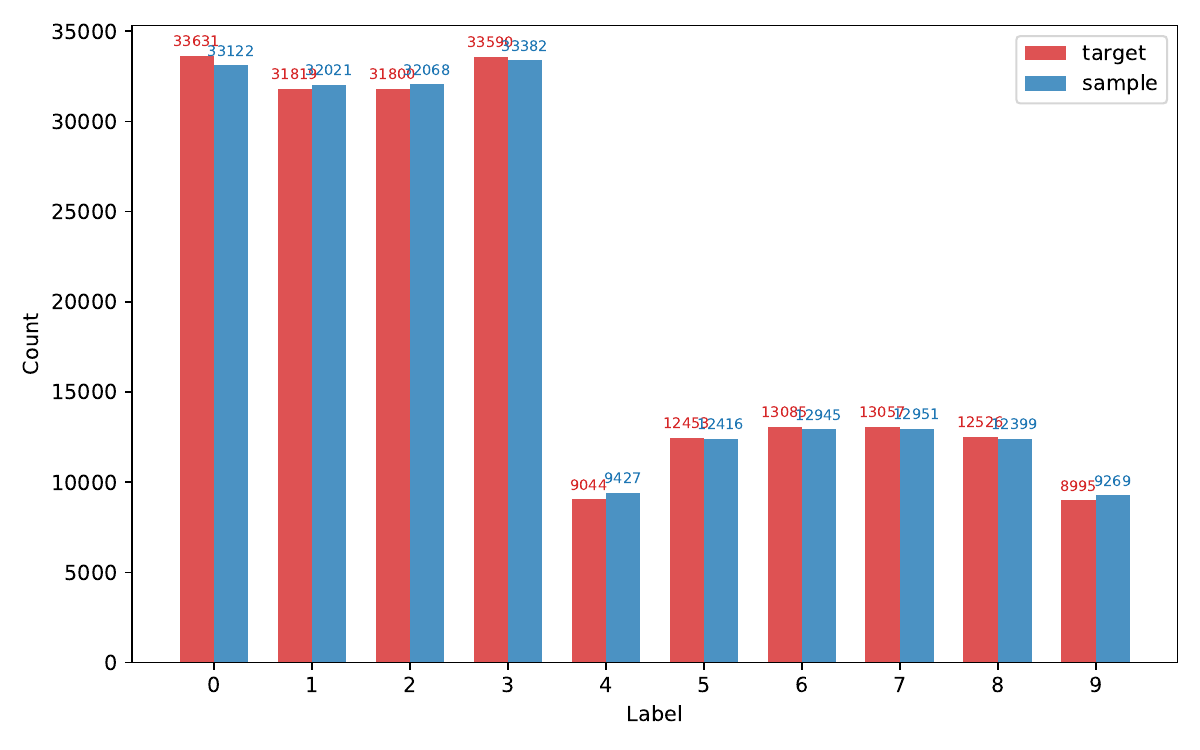} &
\includegraphics[width=0.40\textwidth]{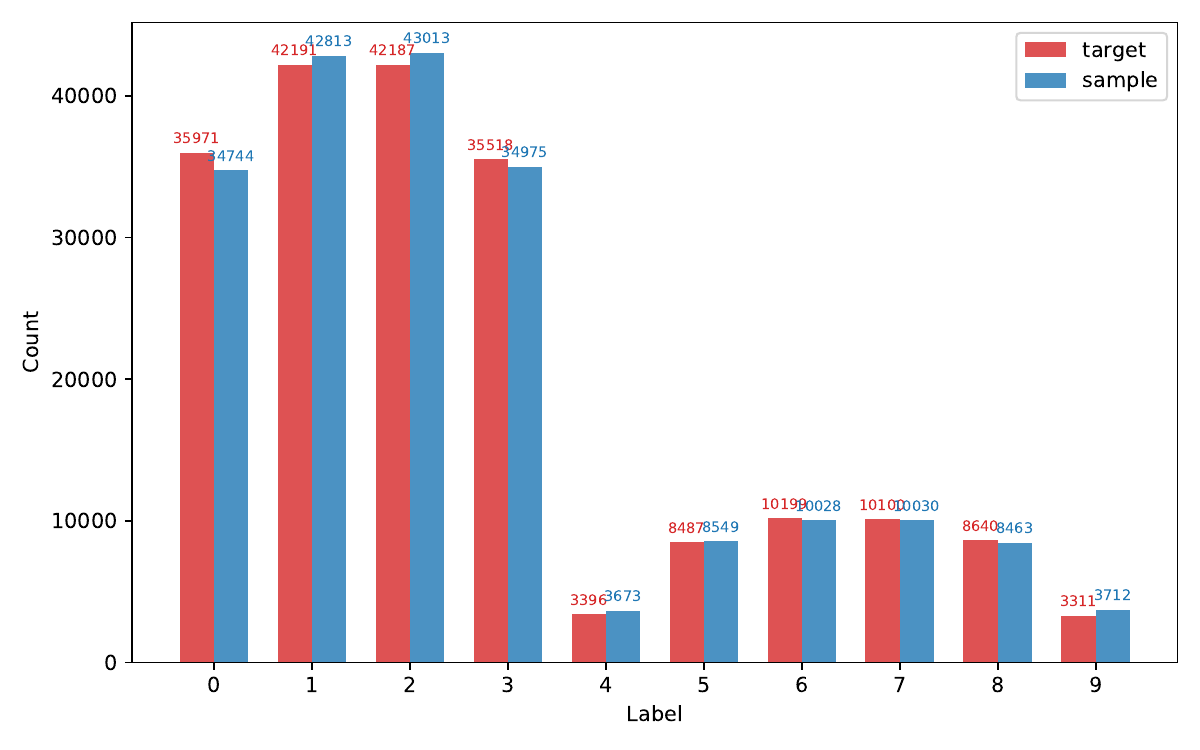} \\
};

\node[anchor=north, yshift=3pt] at (plots-1-1.south) {Constant CFG, $\bar{\omega}=3$};
\node[anchor=north, yshift=3pt] at (plots-1-2.south) {Constant CFG, $\bar{\omega}=9$};
\node[anchor=north, yshift=3pt] at (plots-2-1.south) {Time-dependent CFG, $\bar{\omega}=3$};
\node[anchor=north, yshift=3pt] at (plots-2-2.south) {Time-dependent CFG, $\bar{\omega}=9$};

\end{tikzpicture}
\caption{\textbf{Toy-model label distributions under constant and time-dependent guidance.}
\textbf{Top:} Discrete support points ($K=10$ on a circle of radius $R=2$), where points $0$--$3$ carry twice the target weight of points $4$--$9$.
\textbf{Middle:} Constant CFG.
\textbf{Bottom:} Time-dependent CFG.
Bars show target distribution (theoretical) and sample distribution (CFG) with $N=200{,}000$ samples and 50-step DDIM. }
\label{fig:toy-label-compare}
\end{figure}

\subsection{Stable Diffusion Models}
\label{sec:exp-sd}

We next evaluate DG-CFG on three latent diffusion backbones~\cite{rombach2022high}: Stable Diffusion~1.5 (SD1.5), Stable Diffusion~2.1 (SD2.1), and Stable Diffusion~XL (SDXL). These backbones span three model generations and two native output resolutions, allowing us to assess transfer across architectures and scales. At a fixed sampling budget, we compare guidance strategies over a range of nominal guidance strengths $\bar{\omega}$ on all three backbones. On SD1.5 and SD2.1, we further conduct full component ablations to isolate each design factor and assess its consistency across backbones. We also evaluate these models across NFE settings under strong guidance to assess robustness across sampling budgets.

\subsubsection{Setup}

We use the RunwayML SD1.5 checkpoint,\footnote{\url{https://huggingface.co/runwayml/stable-diffusion-v1-5}} the SD2.1-base checkpoint,\footnote{\url{https://huggingface.co/sd2-community/stable-diffusion-2-1-base}} and the SDXL-base-1.0 checkpoint.\footnote{\url{https://huggingface.co/stabilityai/stable-diffusion-xl-base-1.0}} The unconditional score is obtained with an empty text prompt, and the conditional score uses prompts from the MS-COCO 2017 validation set~\cite{lin2014microsoft}. All images are generated with the DDIM sampler~\cite{song2021denoising}.

For each method and backbone, we use the same randomly selected $N_{gen}=1{,}000$ prompt--image pairs from MS-COCO 2017 and generate one image per prompt at the model's native evaluation resolution: $512\times512$ for SD1.5 and SD2.1, and $1024\times1024$ for SDXL. We evaluate overall image quality with Fr\'{e}chet Inception Distance (FID)~\cite{heusel2017gans}, prompt alignment with CLIP score~\cite{radford2021learning}, and perceptual quality with ImageReward (IR). We also report a saturation diagnostic (Sat) to detect the color over-saturation pathology that commonly appears under strong guidance. To measure diversity, we additionally sample $N_{\mathrm{div}}=100$ prompts and generate 16 images per prompt with different initializations, then compute the Vendi score~\cite{friedman2023the}. CLIP, IR, and Sat are averaged over the $N_{gen}$ generated images, while Vendi score is averaged over the $N_{\mathrm{div}}$ prompts. For reference, real MS-COCO images achieve $\text{CLIP}=30.57$, $\text{IR}=0.38$, and $\text{Sat}=0.33$ under the same evaluators.

\subsubsection{Main Results}

We compare four guidance strategies using the same DDIM sampler with $8$ NFE. All three backbones are evaluated at $\bar{\omega}\in\{3,7,11,15,19,23\}$:

\begin{itemize}[leftmargin=*]
    \item \textbf{Constant CFG:} $\omega(t) \equiv \bar{\omega}$.
    \item \textbf{Interval CFG~\cite{kynkaanniemi2024applying}:}
    $\omega(t) = 1 + (\bar{\omega}-1)/0.6$ for $t \in [200, 800]$, $\omega(t) = 1$ otherwise.
    \item \textbf{$\beta$--CFG~\cite{malarz2025classifier}:}
    $\omega(t) = \bar{\omega} \cdot 6 \cdot (t/T) \cdot (1 - t/T)$.
    \item \textbf{DG-CFG (ours):}
    $\omega(t) = 1 + C \cdot (\bar{\omega}-1) \cdot (1-\bar{\alpha}_t)\sqrt{\bar{\alpha}_t} / \beta(t)$, where $C$ is computed from the noise schedule according to Eq.~\eqref{eq:C-def-final}.
\end{itemize}

All non-constant schedules are normalized so that $\int_{t_0}^T(\omega(s)-1)\,ds=(\bar{\omega}-1)T$. Thus, differences across methods reflect how guidance is distributed over time rather than differences in the total guidance budget.

Tables~\ref{tab:sd15-main}--\ref{tab:sdxl-main} report the quantitative results, and Figures~\ref{fig:sd15-grid}--\ref{fig:sdxl-grid} provide representative samples. We summarize the main observations below.

\textbf{Diversity--fidelity trade-off.}
Across all three backbones, DG-CFG provides a favorable diversity--fidelity trade-off. In the low-to-moderate guidance regime, where outputs remain broadly comparable in quality, it generally preserves more diversity than constant CFG while maintaining competitive FID and ImageReward. Interval CFG can produce higher Vendi scores in some settings, but often with weaker fidelity or perceptual quality. At stronger guidance, Vendi alone becomes less informative because color distortion, structural failures, and texture artifacts can disperse feature embeddings and spuriously inflate the score without adding meaningful semantic variation. Under the quality-aware assessment, DG-CFG more consistently preserves meaningful variation while avoiding the degradation observed in the baselines.

\textbf{Image-quality improvement at high guidance.}
The advantage of DG-CFG is most pronounced in the high-guidance regime. As the nominal guidance strength increases, constant and heuristic schedules progressively lose perceptual quality, prompt alignment, or distributional fidelity, whereas DG-CFG remains comparatively stable across all three backbones. The contrast is strongest on SDXL: baseline performance deteriorates sharply under strong guidance, while DG-CFG maintains high ImageReward and CLIP scores together with competitive FID. 

\textbf{Saturation control.}
The Sat metric captures the color over-saturation commonly induced by strong guidance. Across all three backbones, Sat increases steadily with the guidance strength under constant and interval CFG, approaching or exceeding the real-image reference at high guidance. By contrast, DG-CFG exhibits a substantially weaker increase and remains at or below the reference level throughout the evaluated guidance range.

\textbf{Qualitative visualization.}
Figures~\ref{fig:sd15-grid}--\ref{fig:sdxl-grid} show representative samples from all three backbones using the same prompts and random seeds. As guidance becomes stronger, the baselines progressively develop posterized colors, washed-out textures, and less coherent object boundaries. DG-CFG preserves more natural color balance, texture detail, and object structure throughout the evaluated guidance range. These visual results are consistent with the quantitative findings, confirming that the stability of DG-CFG translates into more natural and coherent samples rather than merely improved aggregate metrics.

\begin{table}[ht]
\centering
\caption{Main comparison on Stable Diffusion~1.5 with DDIM sampling at 8 NFE. Metrics are computed over 1K MS-COCO validation prompts, except Div (Vendi score), which is averaged over 100 prompts with 16 samples per prompt. The real-image reference values are CLIP=30.57, IR=0.38, and Sat=0.33.}
\label{tab:sd15-main}
\small
\begin{tabular}{llcccccc}
\toprule
Metric & Method & $\bar{\omega}=3$ & $\bar{\omega}=7$ & $\bar{\omega}=11$ & $\bar{\omega}=15$ & $\bar{\omega}=19$ & $\bar{\omega}=23$ \\
\midrule
\multirow{4}{*}{CLIP$\uparrow$}
& Constant CFG      & 30.15 & \textbf{31.29} & 31.38 & 31.29 & 31.01 & 30.64 \\
& Interval CFG      & \textbf{30.39} & 30.95 & 30.75 & 30.32 & 29.78 & 29.10 \\
& $\beta$--CFG      & 30.34 & 31.21 & \textbf{31.42} & 31.28 & 31.18 & 30.88 \\
& \textbf{DG-CFG (ours)} & 29.95 & 31.08 & 31.27 & \textbf{31.32} & \textbf{31.29} & \textbf{31.39} \\
\midrule
\multirow{4}{*}{IR$\uparrow$}
& Constant CFG      & $-$0.503 & $-$0.104 & $-$0.032 & $-$0.061 & $-$0.177 & $-$0.330 \\
& Interval CFG      & $-$0.463 & $-$0.255 & $-$0.308 & $-$0.518 & $-$0.758 & $-$0.986 \\
& $\beta$--CFG      & \textbf{$-$0.437} & \textbf{$-$0.103} & \textbf{$-$0.023} & $-$0.022 & $-$0.104 & $-$0.239 \\
& \textbf{DG-CFG (ours)} & $-$0.522 & $-$0.115 & $-$0.026 & \textbf{+0.013} & \textbf{+0.024} & \textbf{+0.029} \\
\midrule
\multirow{4}{*}{Sat}
& Constant CFG      & 0.212 & 0.258 & 0.317 & 0.379 & 0.430 & 0.469 \\
& Interval CFG      & 0.220 & 0.275 & 0.338 & 0.392 & 0.431 & 0.462 \\
& $\beta$--CFG      & 0.214 & 0.256 & 0.308 & 0.359 & 0.401 & 0.438 \\
& \textbf{DG-CFG (ours)} & 0.206 & 0.229 & 0.256 & 0.284 & 0.310 & 0.332 \\
\midrule
\multirow{4}{*}{FID$\downarrow$}
& Constant CFG      & \textbf{74.44} & \textbf{67.67} & \textbf{68.73} & 71.57 & 74.94 & 82.34 \\
& Interval CFG      & 72.46 & 70.83 & 76.84 & 86.60 & 99.07 & 110.89 \\
& $\beta$--CFG      & 73.14 & 68.14 & 70.73 & 73.07 & 77.92 & 85.38 \\
& \textbf{DG-CFG (ours)} & 77.26 & 67.87 & 68.79 & \textbf{70.62} & \textbf{72.54} & \textbf{74.53} \\
\midrule
\multirow{4}{*}{Div$\uparrow$}
& Constant CFG      & 6.54 & 5.40 & 5.18 & 5.27 & 5.43 & 5.67 \\
& Interval CFG      & 6.47 & 5.95 & 6.04 & 6.28 & 6.61 & 6.80 \\
& $\beta$--CFG      & 6.47 & 5.44 & 5.17 & 5.26 & 5.37 & 5.57 \\
& \textbf{DG-CFG (ours)} & 6.81 & 5.71 & 5.34 & 5.22 & 5.17 & 5.24 \\
\bottomrule
\end{tabular}
\end{table}

\begin{table}[ht]
\centering
\caption{Main comparison on Stable Diffusion~2.1 with DDIM sampling at 8 NFE. Metrics are computed over 1K MS-COCO validation prompts, except Div (Vendi score), which is averaged over 100 prompts with 16 samples per prompt. The real-image reference values are CLIP=30.57, IR=0.38, and Sat=0.33.}
\label{tab:sd21-main}
\small
\begin{tabular}{llcccccc}
\toprule
Metric & Method & $\bar{\omega}=3$ & $\bar{\omega}=7$ & $\bar{\omega}=11$ & $\bar{\omega}=15$ & $\bar{\omega}=19$ & $\bar{\omega}=23$ \\
\midrule
\multirow{4}{*}{CLIP$\uparrow$}
& Constant CFG      & 30.54 & \textbf{31.48} & 31.48 & 31.48 & 31.25 & 30.98 \\
& Interval CFG      & \textbf{30.71} & 31.02 & 30.93 & 30.54 & 30.12 & 29.80 \\
& $\beta$--CFG      & 30.62 & 31.36 & \textbf{31.51} & \textbf{31.55} & 31.40 & 31.14 \\
& \textbf{DG-CFG (ours)} & 30.36 & 31.23 & 31.36 & 31.38 & \textbf{31.49} & \textbf{31.43} \\
\midrule
\multirow{4}{*}{IR$\uparrow$}
& Constant CFG      & $-$0.264 & \textbf{+0.130} & +0.203 & +0.176 & +0.078 & $-$0.092 \\
& Interval CFG      & $-$0.235 & $-$0.055 & $-$0.101 & $-$0.275 & $-$0.456 & $-$0.630 \\
& $\beta$--CFG      & \textbf{$-$0.212} & +0.119 & +0.203 & +0.214 & +0.142 & +0.060 \\
& \textbf{DG-CFG (ours)} & $-$0.322 & +0.091 & \textbf{+0.209} & \textbf{+0.238} & \textbf{+0.245} & \textbf{+0.244} \\
\midrule
\multirow{4}{*}{Sat}
& Constant CFG      & 0.207 & 0.248 & 0.299 & 0.351 & 0.396 & 0.431 \\
& Interval CFG      & 0.214 & 0.262 & 0.315 & 0.361 & 0.396 & 0.422 \\
& $\beta$--CFG      & 0.209 & 0.248 & 0.291 & 0.334 & 0.372 & 0.398 \\
& \textbf{DG-CFG (ours)} & 0.203 & 0.222 & 0.247 & 0.272 & 0.295 & 0.315 \\
\midrule
\multirow{4}{*}{FID$\downarrow$}
& Constant CFG      & 73.40 & \textbf{67.45} & 70.10 & 72.63 & 74.79 & 79.94 \\
& Interval CFG      & \textbf{71.29} & 72.60 & 77.43 & 86.10 & 95.66 & 105.19 \\
& $\beta$--CFG      & 71.79 & 68.64 & 71.36 & 73.27 & 76.73 & 83.11 \\
& \textbf{DG-CFG (ours)} & 75.54 & 68.70 & \textbf{70.08} & \textbf{71.98} & \textbf{74.05} & \textbf{76.43} \\
\midrule
\multirow{4}{*}{Div$\uparrow$}
& Constant CFG      & 5.99 & 5.01 & 4.81 & 4.87 & 5.08 & 5.33 \\
& Interval CFG      & 6.04 & 5.49 & 5.66 & 5.86 & 6.18 & 6.38 \\
& $\beta$--CFG      & 5.98 & 4.99 & 4.85 & 4.89 & 5.01 & 5.24 \\
& \textbf{DG-CFG (ours)} & 6.26 & 5.20 & 4.97 & 4.88 & 4.85 & 4.83 \\
\bottomrule
\end{tabular}
\end{table}

\begin{table}[ht]
\centering
\caption{Main comparison on Stable Diffusion~XL with DDIM sampling at 8 NFE. Metrics are computed over 1K MS-COCO validation prompts, except Div (Vendi score), which is averaged over 100 prompts with 16 samples per prompt. The real-image reference values are CLIP=30.57, IR=0.38, and Sat=0.33.}
\label{tab:sdxl-main}
\small
\begin{tabular}{llcccccc}
\toprule
Metric & Method & $\bar{\omega}=3$ & $\bar{\omega}=7$ & $\bar{\omega}=11$ & $\bar{\omega}=15$ & $\bar{\omega}=19$ & $\bar{\omega}=23$ \\
\midrule
\multirow{4}{*}{CLIP$\uparrow$}
& Constant CFG      & 30.68 & 31.82 & 31.75 & 31.34 & 30.51 & 29.46 \\
& Interval CFG      & 30.73 & 31.32 & 30.82 & 29.97 & 29.04 & 28.17 \\
& $\beta$--CFG      & \textbf{30.86} & \textbf{31.96} & \textbf{32.06} & 31.90 & 31.44 & 30.76 \\
& \textbf{DG-CFG (ours)} & 30.23 & 31.65 & 32.03 & \textbf{32.10} & \textbf{32.10} & \textbf{32.00} \\
\midrule
\multirow{4}{*}{IR$\uparrow$}
& Constant CFG      & $-$0.349 & +0.205 & +0.159 & $-$0.045 & $-$0.383 & $-$0.717 \\
& Interval CFG      & $-$0.344 & $-$0.115 & $-$0.304 & $-$0.604 & $-$0.899 & $-$1.144 \\
& $\beta$--CFG      & \textbf{$-$0.289} & \textbf{+0.279} & \textbf{+0.350} & +0.247 & +0.053 & $-$0.222 \\
& \textbf{DG-CFG (ours)} & $-$0.530 & +0.102 & +0.321 & \textbf{+0.392} & \textbf{+0.417} & \textbf{+0.401} \\
\midrule
\multirow{4}{*}{Sat}
& Constant CFG      & 0.147 & 0.180 & 0.221 & 0.267 & 0.313 & 0.351 \\
& Interval CFG      & 0.153 & 0.183 & 0.220 & 0.259 & 0.293 & 0.326 \\
& $\beta$--CFG      & 0.150 & 0.182 & 0.219 & 0.259 & 0.292 & 0.318 \\
& \textbf{DG-CFG (ours)} & 0.146 & 0.161 & 0.179 & 0.200 & 0.221 & 0.241 \\
\midrule
\multirow{4}{*}{FID$\downarrow$}
& Constant CFG      & 91.33 & 72.86 & 73.18 & 79.01 & 91.75 & 105.39 \\
& Interval CFG      & \textbf{88.45} & 79.47 & 85.78 & 100.31 & 113.81 & 126.27 \\
& $\beta$--CFG      & 88.83 & \textbf{72.29} & \textbf{70.65} & 74.59 & 82.42 & 92.56 \\
& \textbf{DG-CFG (ours)} & 99.47 & 76.79 & 73.46 & \textbf{73.49} & \textbf{73.57} & \textbf{75.02} \\
\midrule
\multirow{4}{*}{Div$\uparrow$}
& Constant CFG      & 5.53 & 4.84 & 5.04 & 5.58 & 6.14 & 6.70 \\
& Interval CFG      & 5.77 & 5.62 & 6.08 & 6.56 & 6.87 & 7.06 \\
& $\beta$--CFG      & 5.49 & 4.75 & 4.72 & 5.01 & 5.51 & 5.94 \\
& \textbf{DG-CFG (ours)} & 5.79 & 5.00 & 4.74 & 4.63 & 4.62 & 4.70 \\
\bottomrule
\end{tabular}
\end{table}

\begin{figure}[t]
\centering
\small
\begin{tikzpicture}[
    img/.style={inner sep=0pt, outer sep=0pt, anchor=center},
    rowgrid/.style={
        matrix of nodes,
        nodes=img,
        inner sep=0pt, outer sep=0pt,
        column sep=1pt
    }
]

\matrix (c1top) [rowgrid]
{
\node[name=p1-1]{\includegraphics[width=0.132\textwidth]{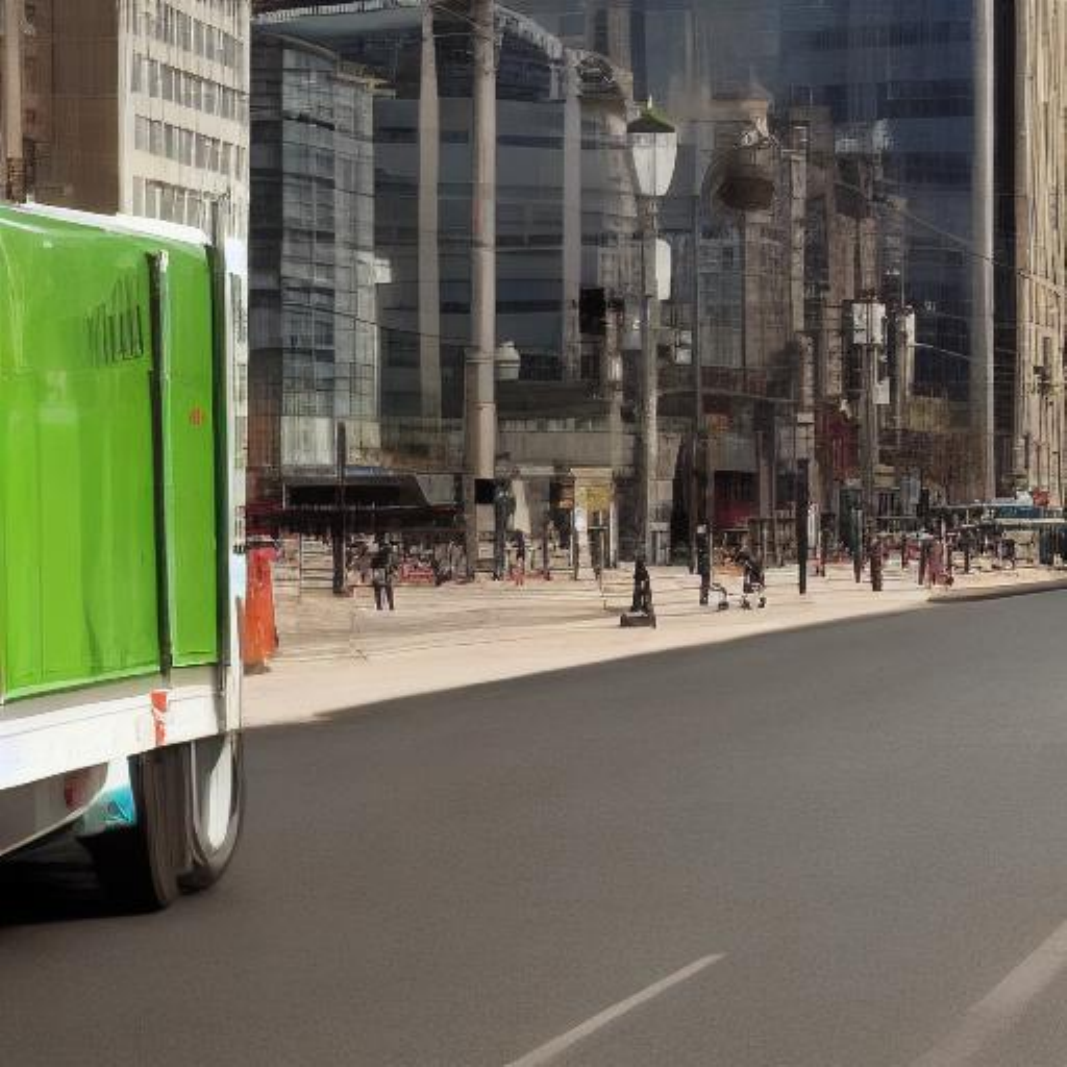}}; &
\node[name=p1-2]{\includegraphics[width=0.132\textwidth]{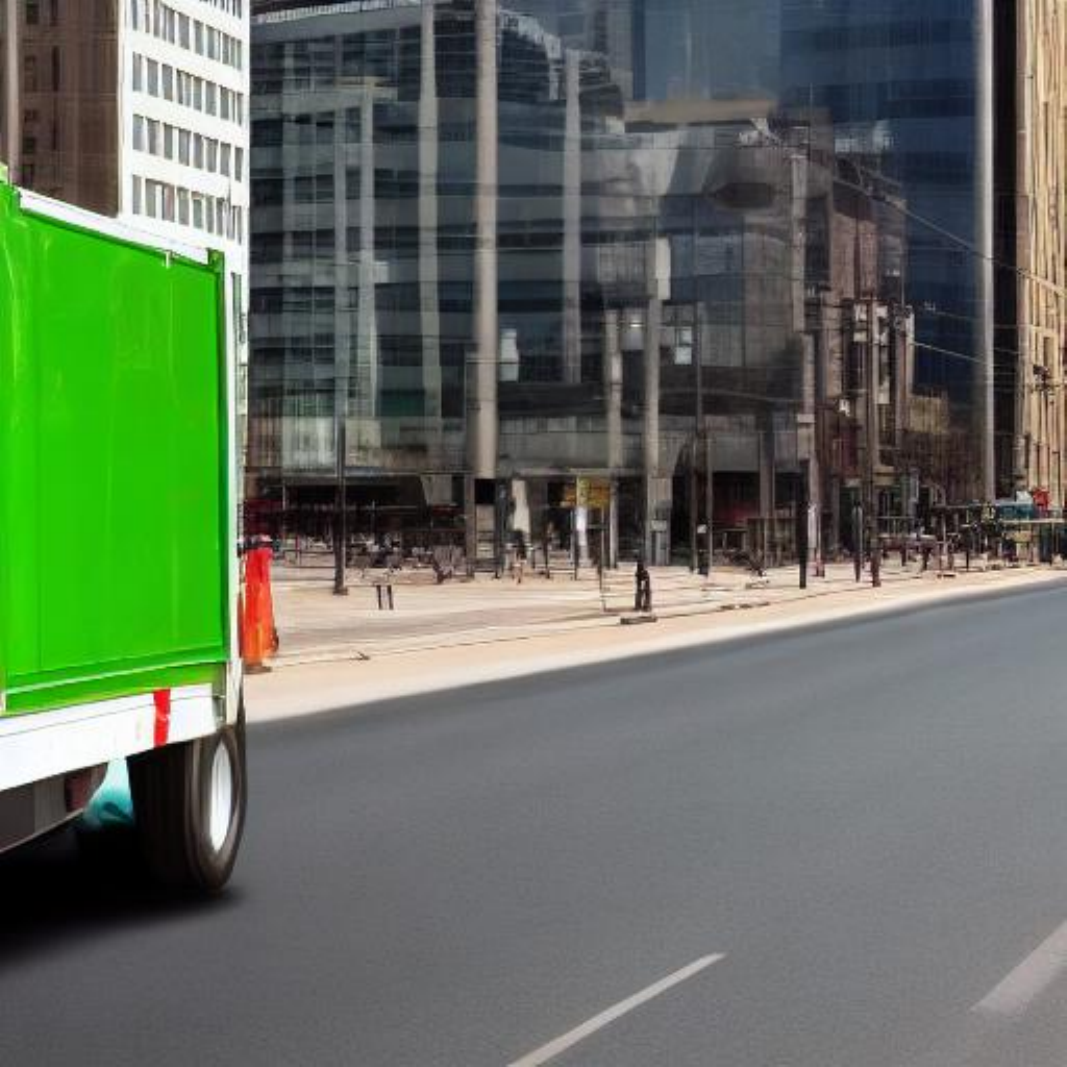}}; &
\node[name=p1-3]{\includegraphics[width=0.132\textwidth]{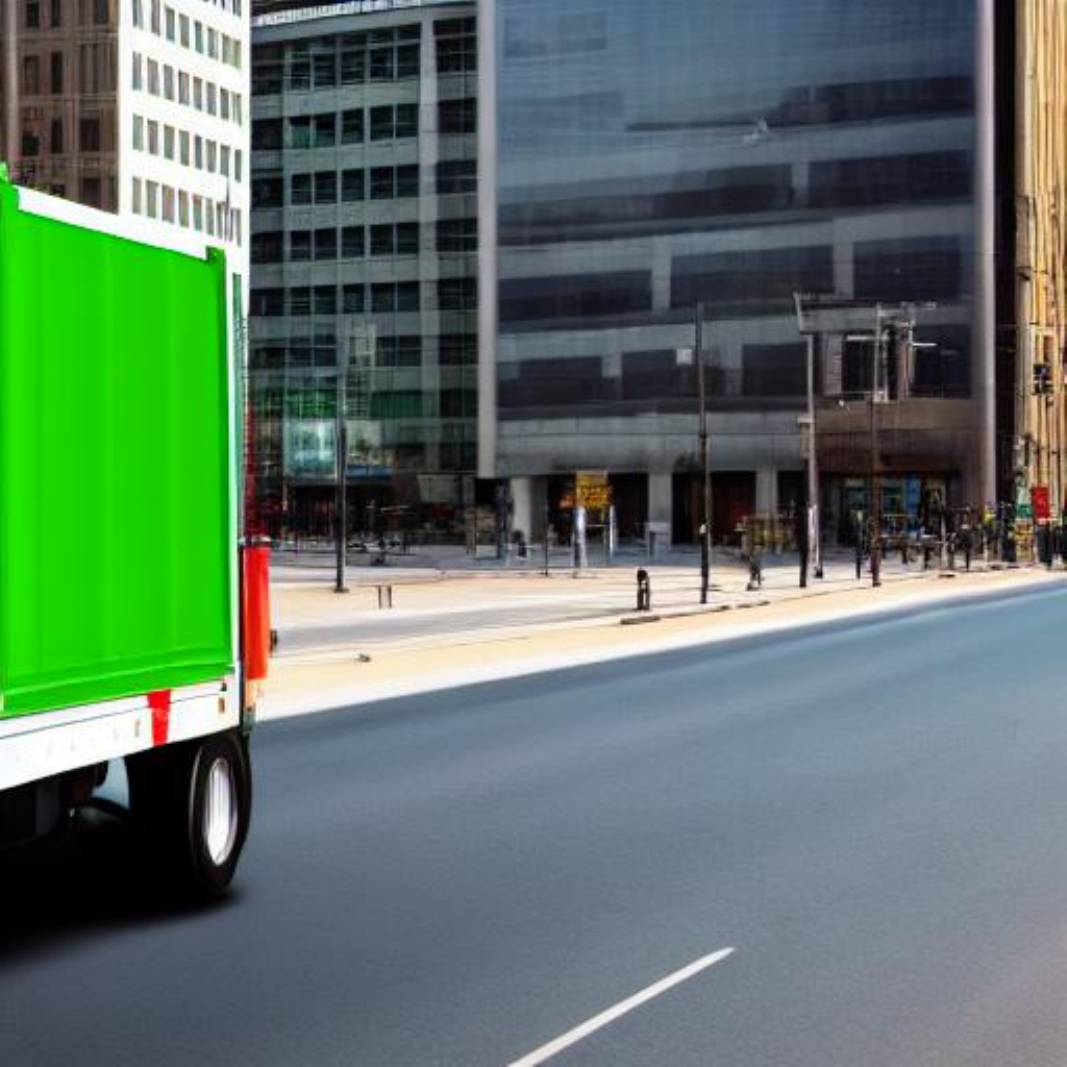}}; &
\node[name=p1-4]{\includegraphics[width=0.132\textwidth]{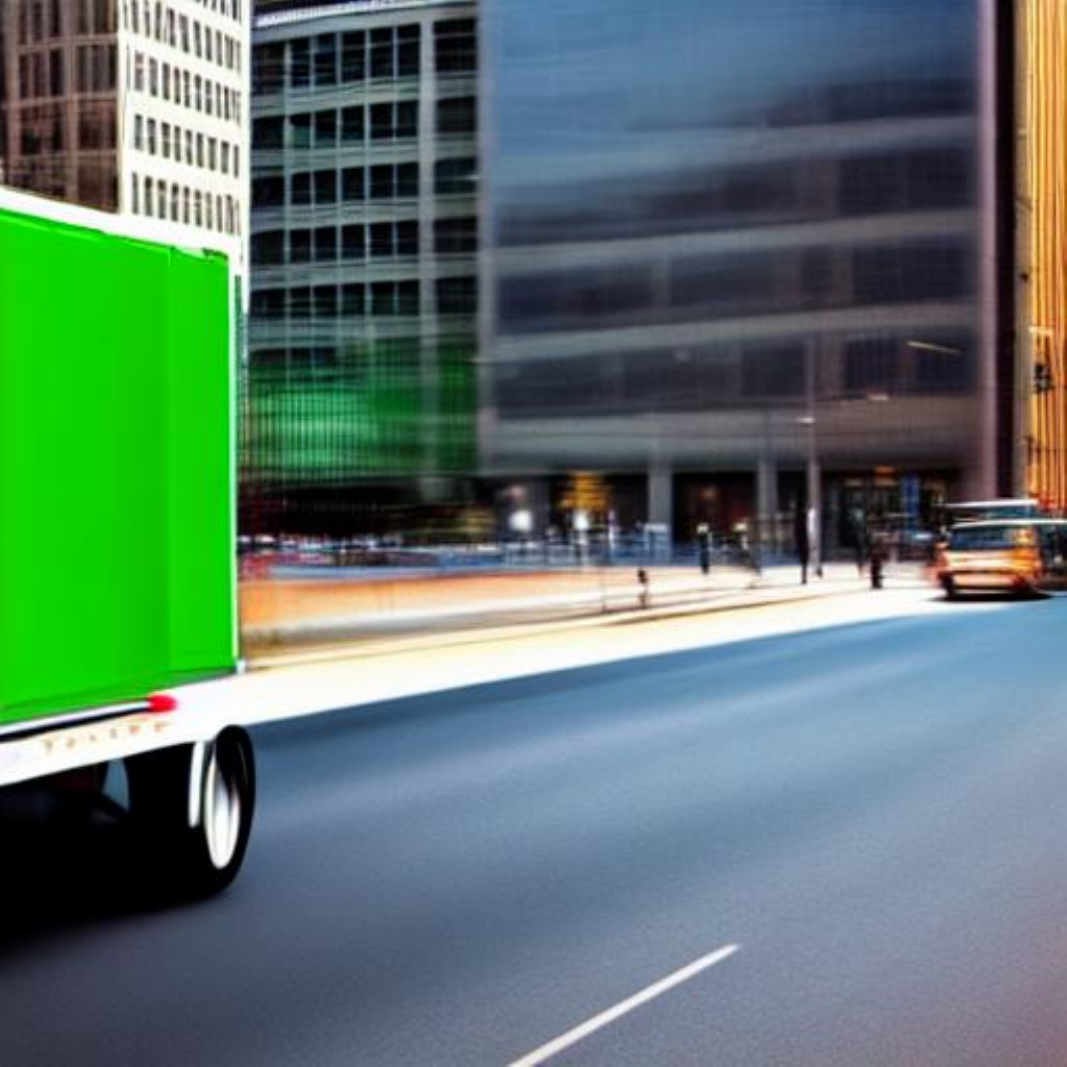}}; &
\node[name=p1-5]{\includegraphics[width=0.132\textwidth]{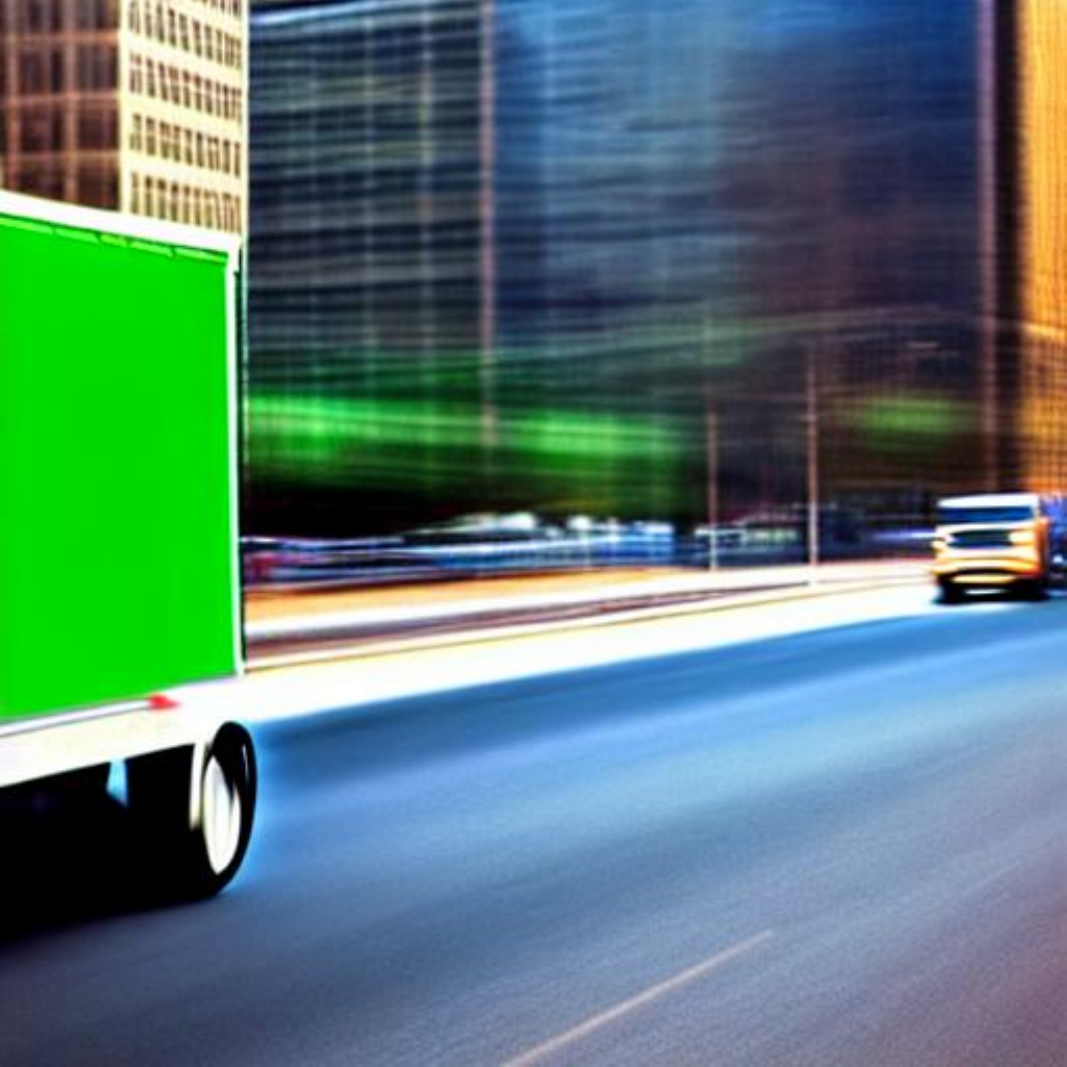}}; &
\node[name=p1-6]{\includegraphics[width=0.132\textwidth]{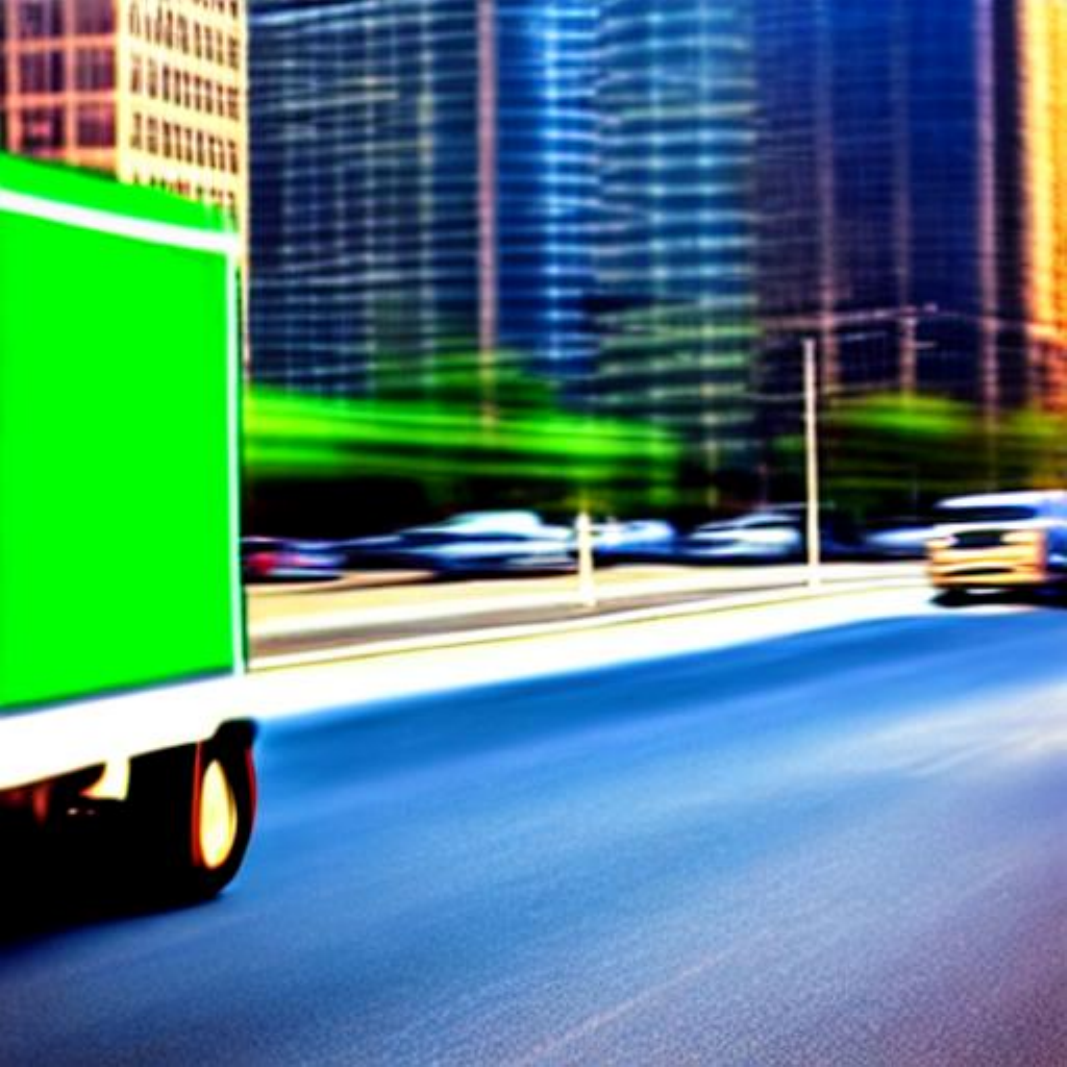}}; \\
};
\matrix (c1bot) [rowgrid, below=1pt of c1top]
{
\includegraphics[width=0.132\textwidth]{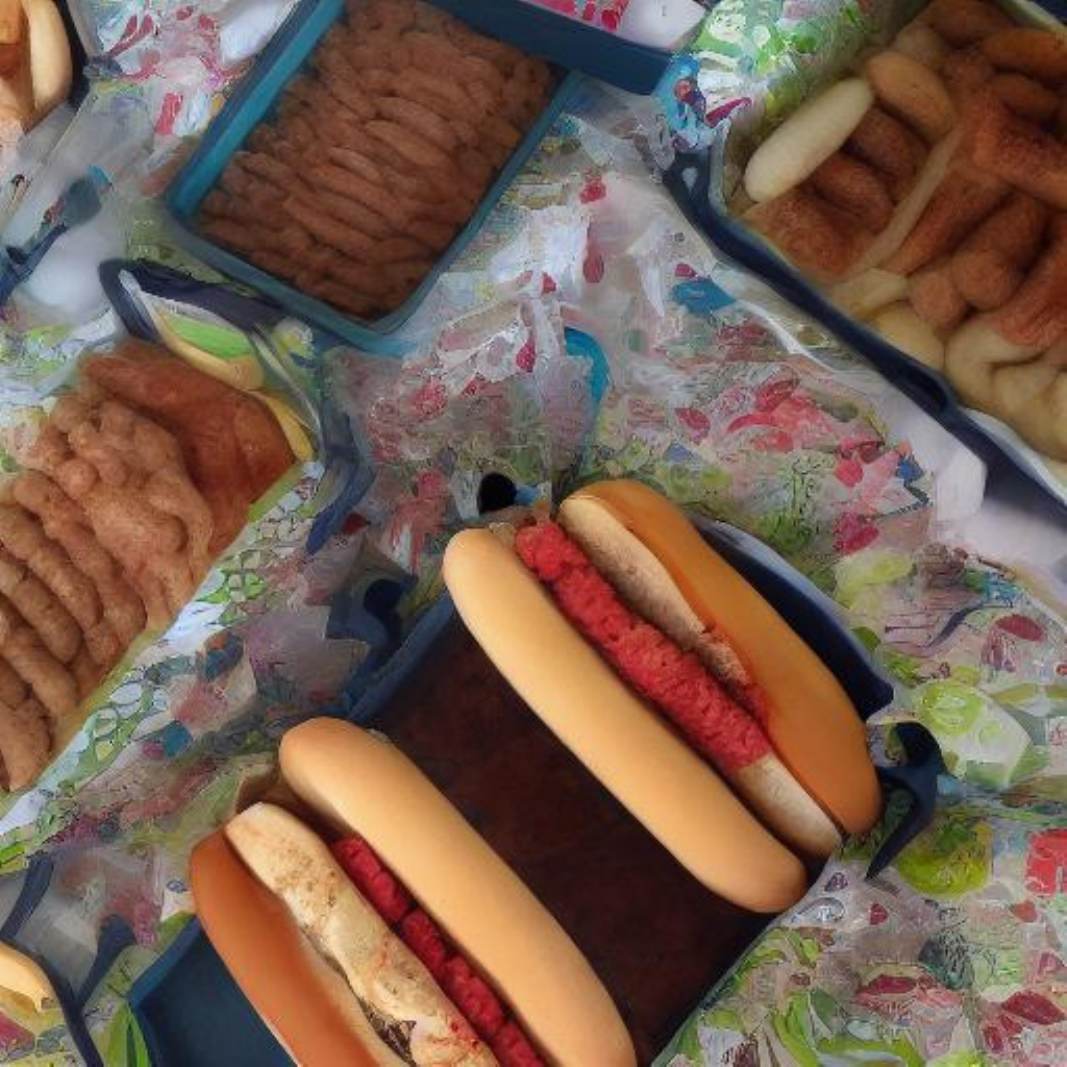} &
\includegraphics[width=0.132\textwidth]{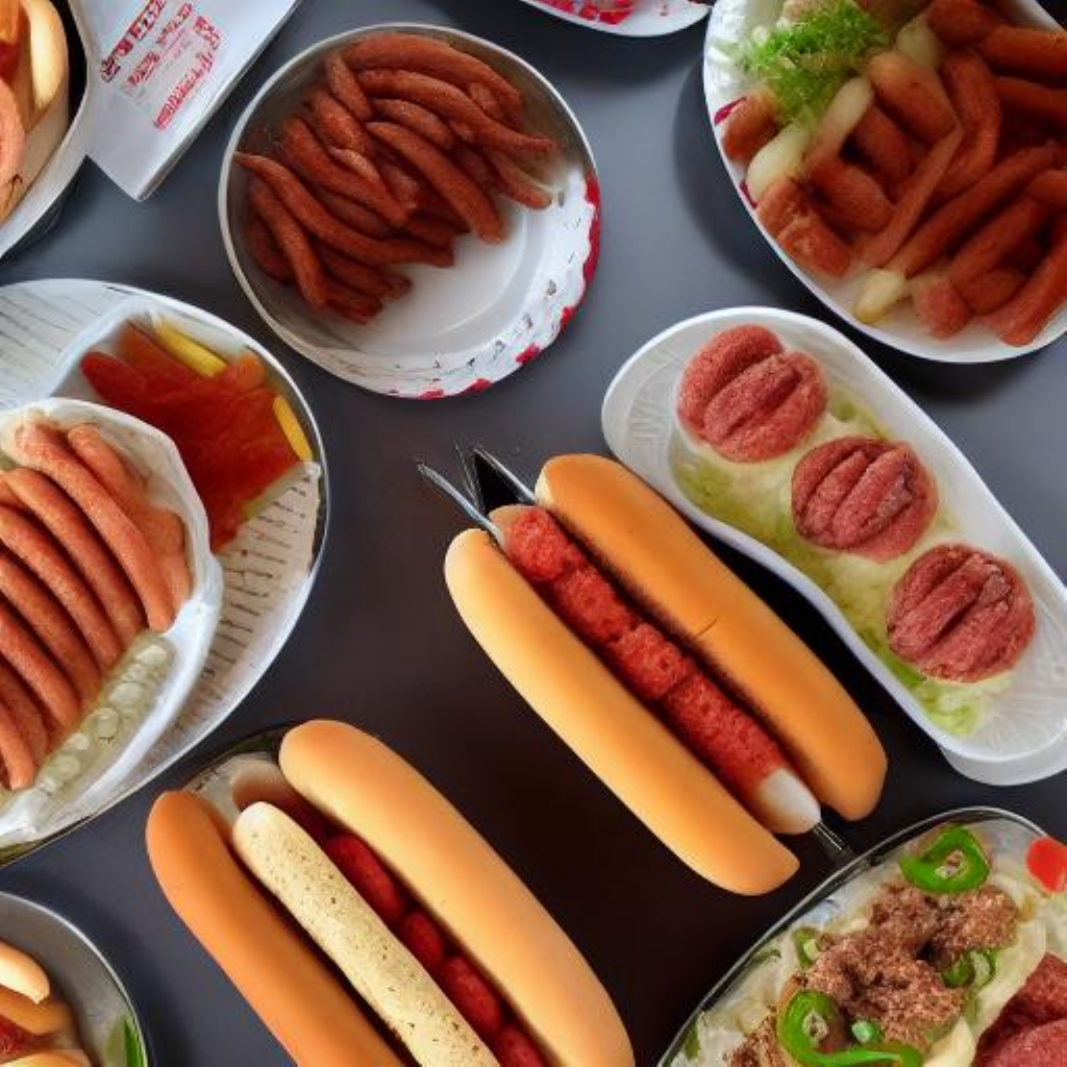} &
\includegraphics[width=0.132\textwidth]{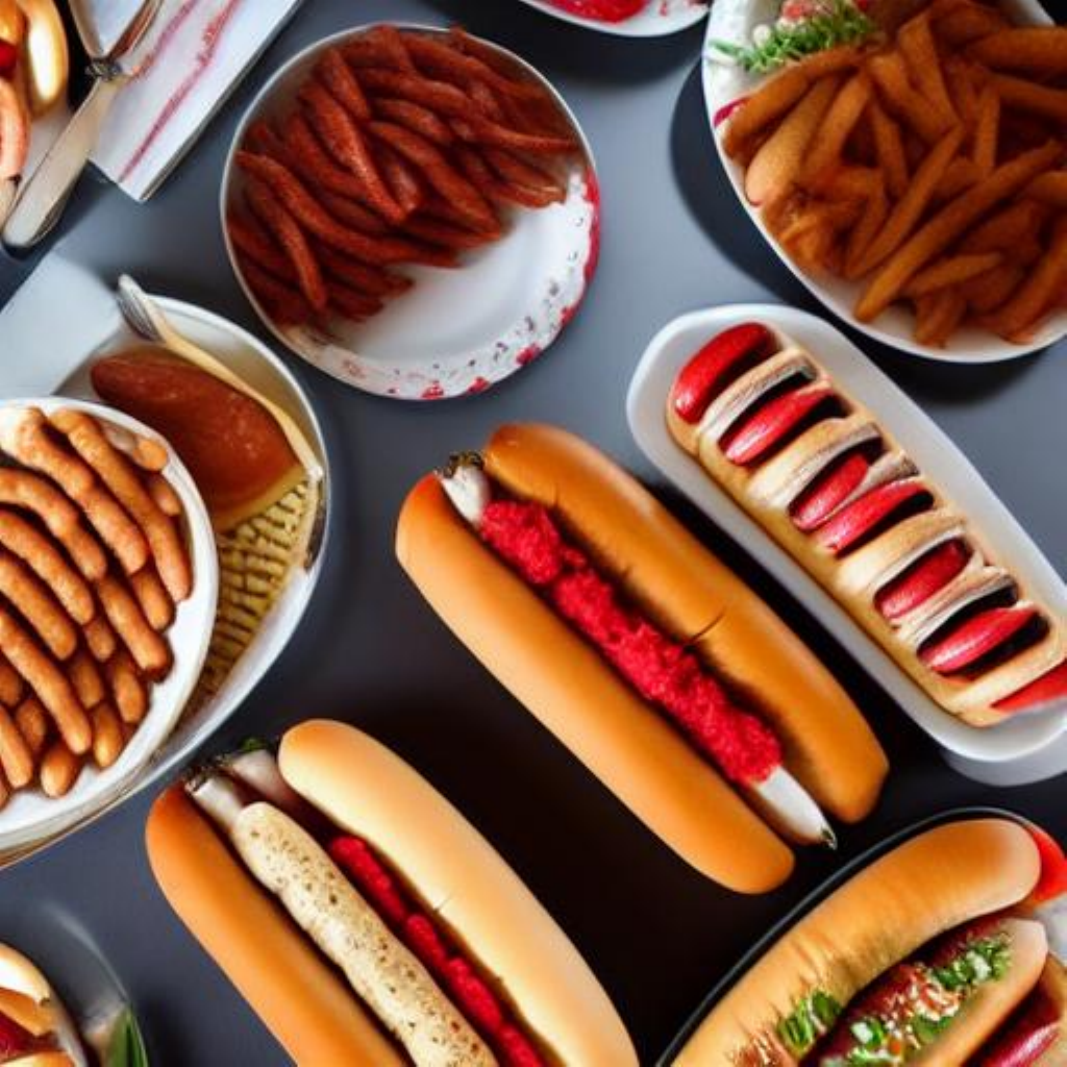} &
\includegraphics[width=0.132\textwidth]{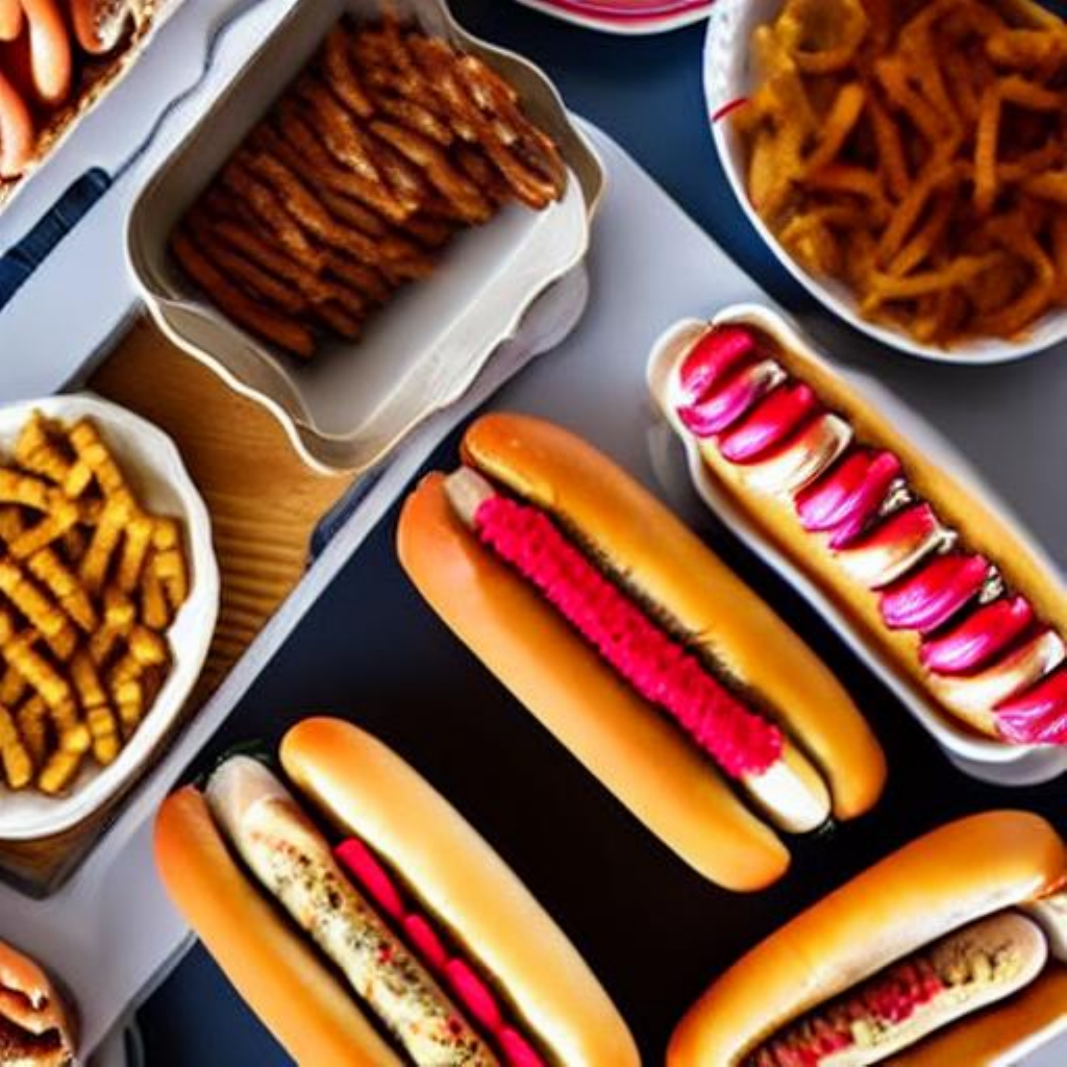} &
\includegraphics[width=0.132\textwidth]{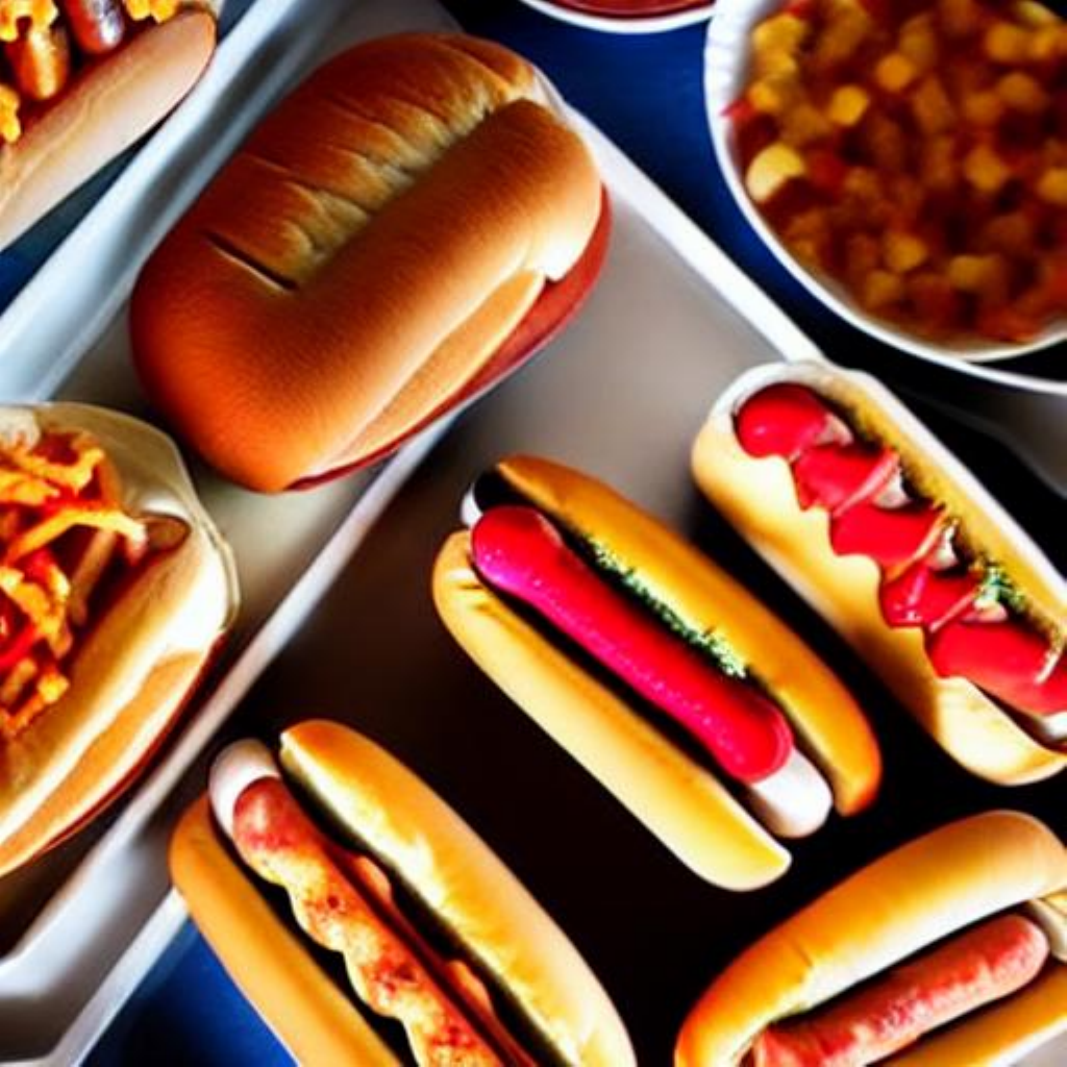} &
\includegraphics[width=0.132\textwidth]{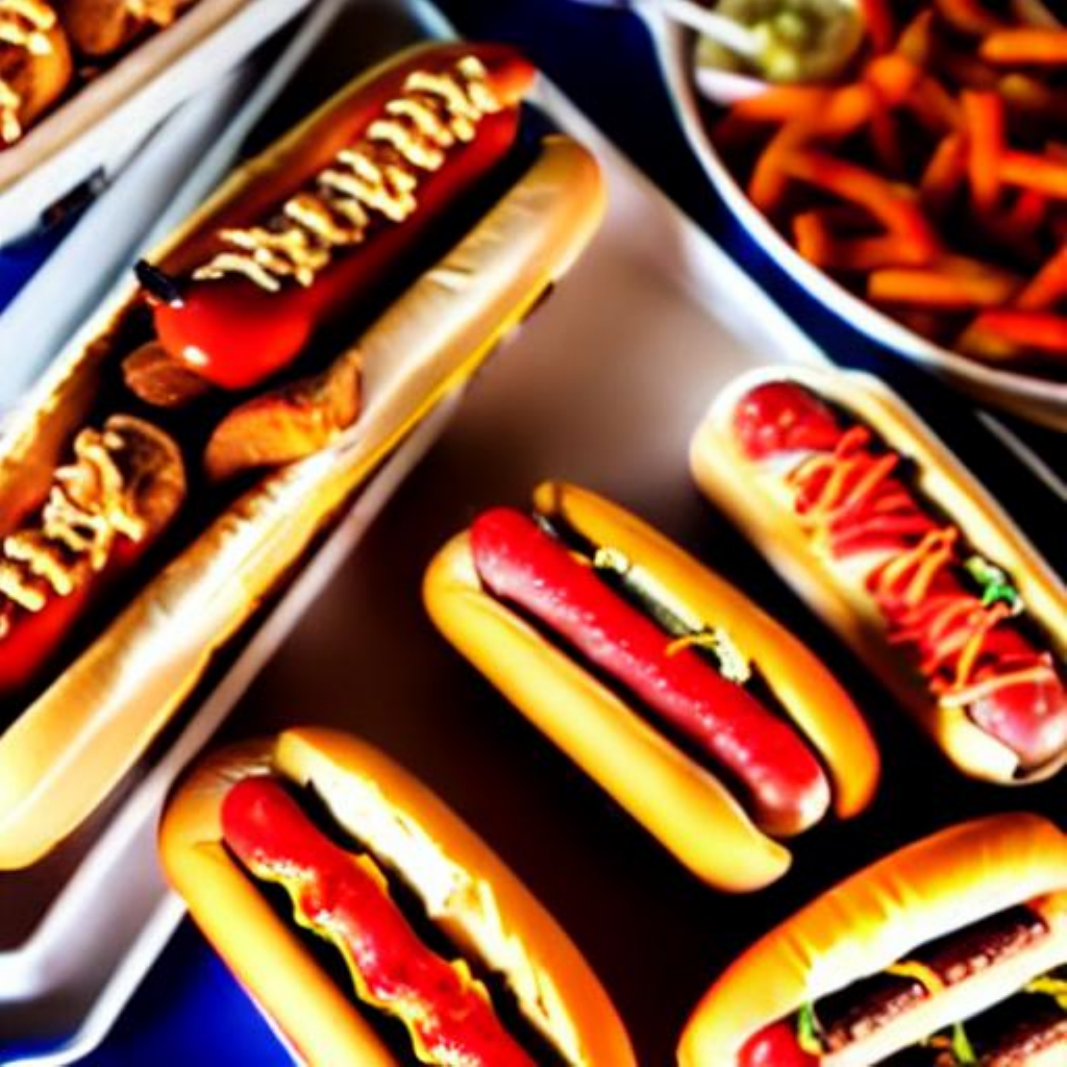} \\
};
\node[rotate=90, anchor=center, yshift=8pt] at ($(c1top.west)!0.5!(c1bot.west)$) {Constant CFG};

\matrix (c2top) [rowgrid, below=3pt of c1bot]
{
\includegraphics[width=0.132\textwidth]{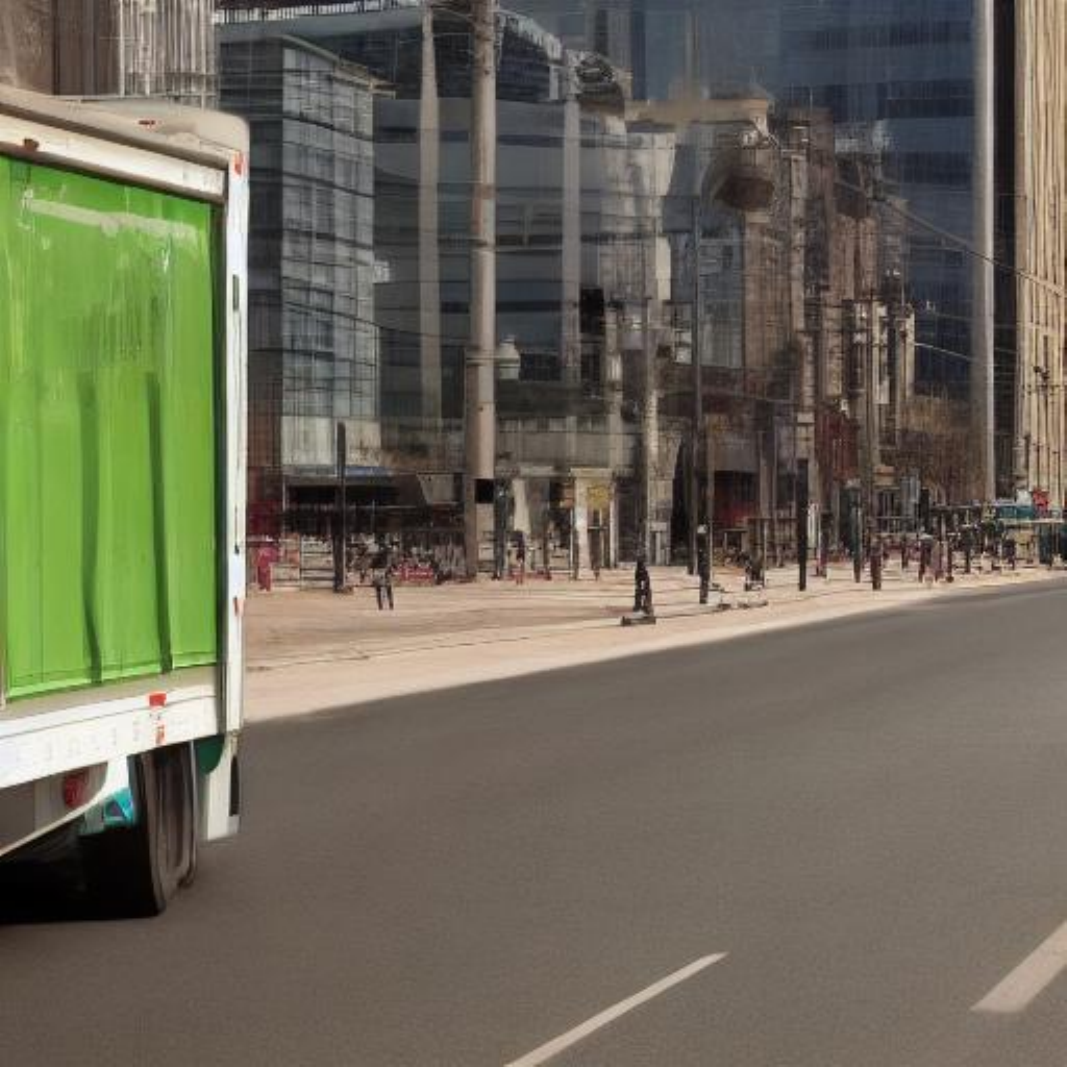} &
\includegraphics[width=0.132\textwidth]{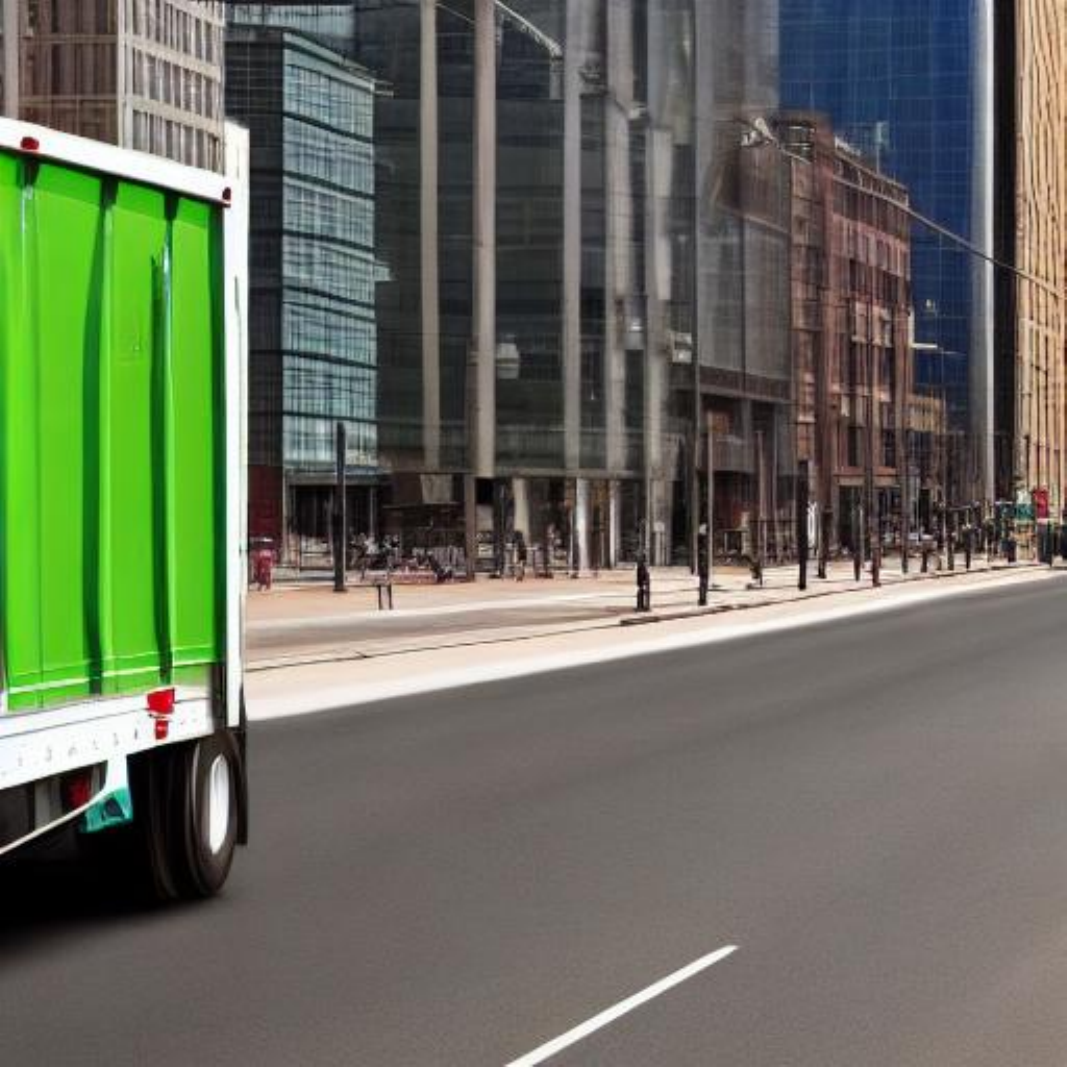} &
\includegraphics[width=0.132\textwidth]{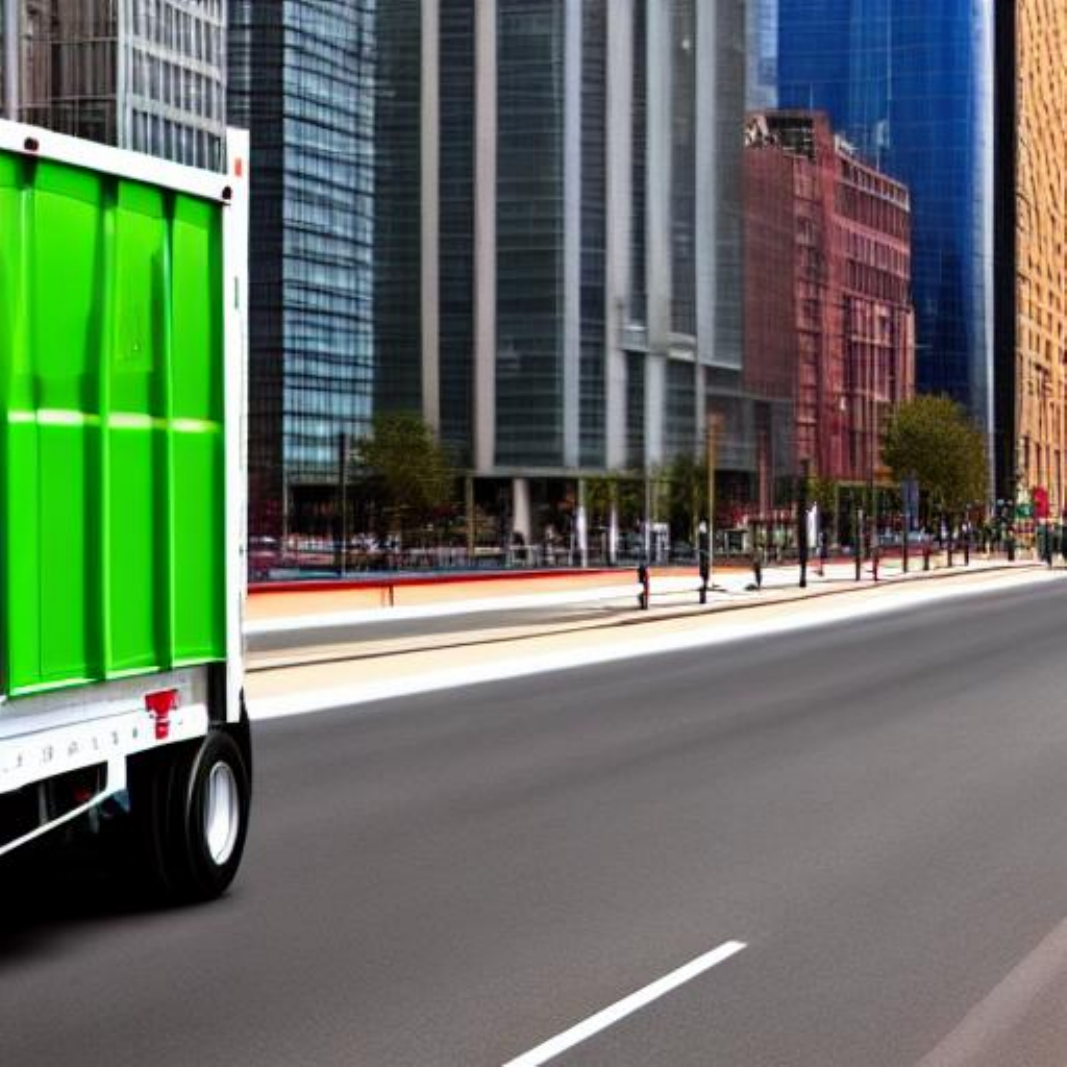} &
\includegraphics[width=0.132\textwidth]{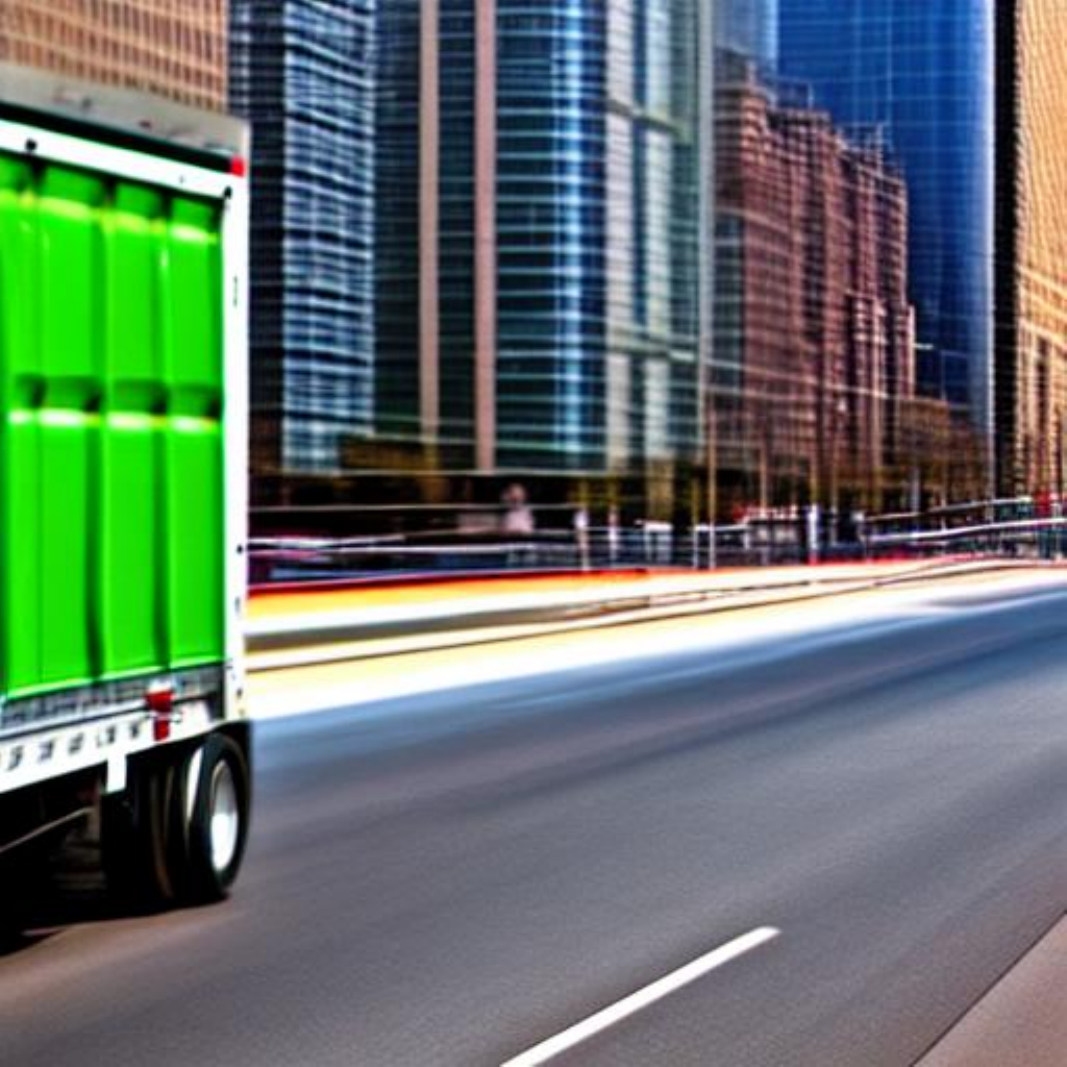} &
\includegraphics[width=0.132\textwidth]{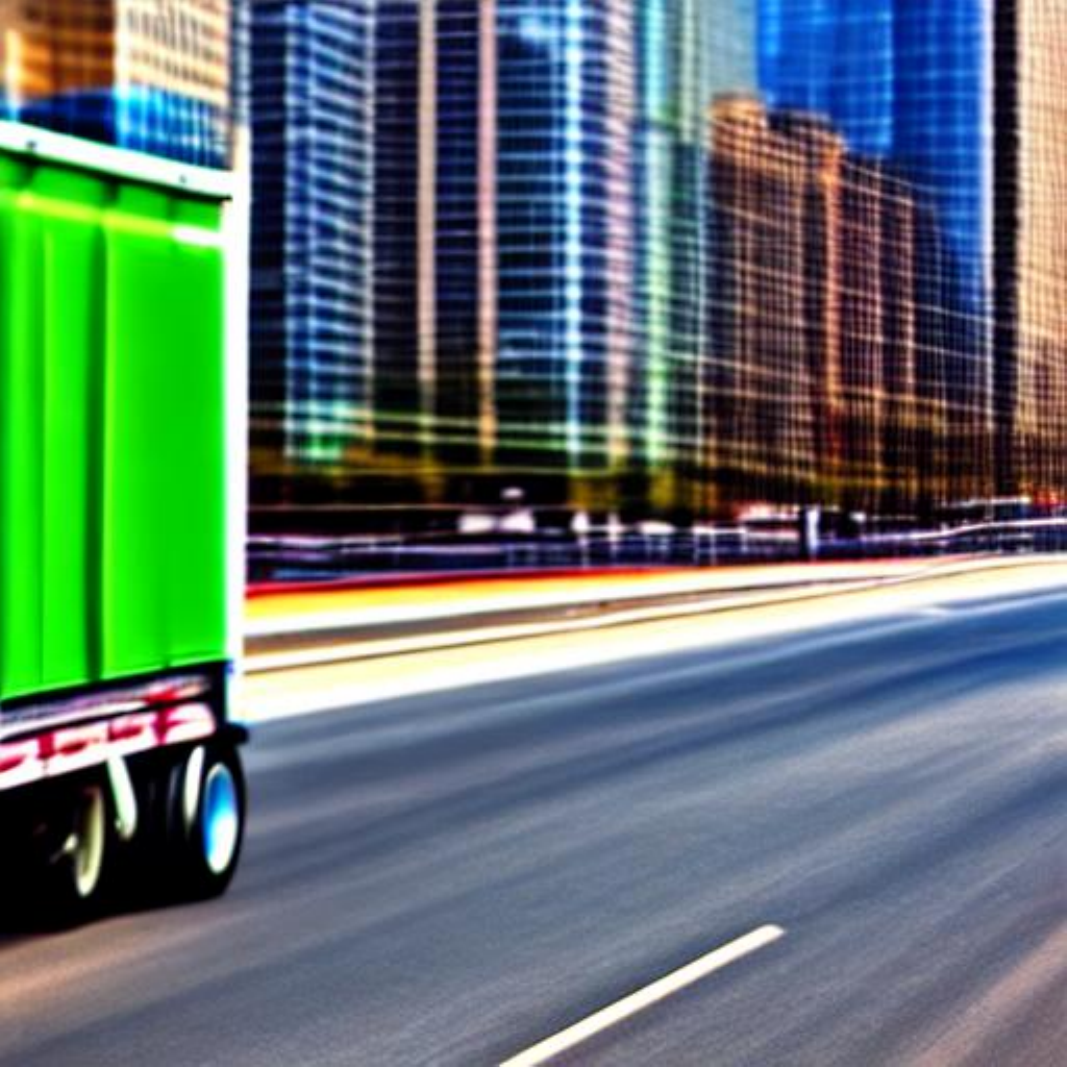} &
\includegraphics[width=0.132\textwidth]{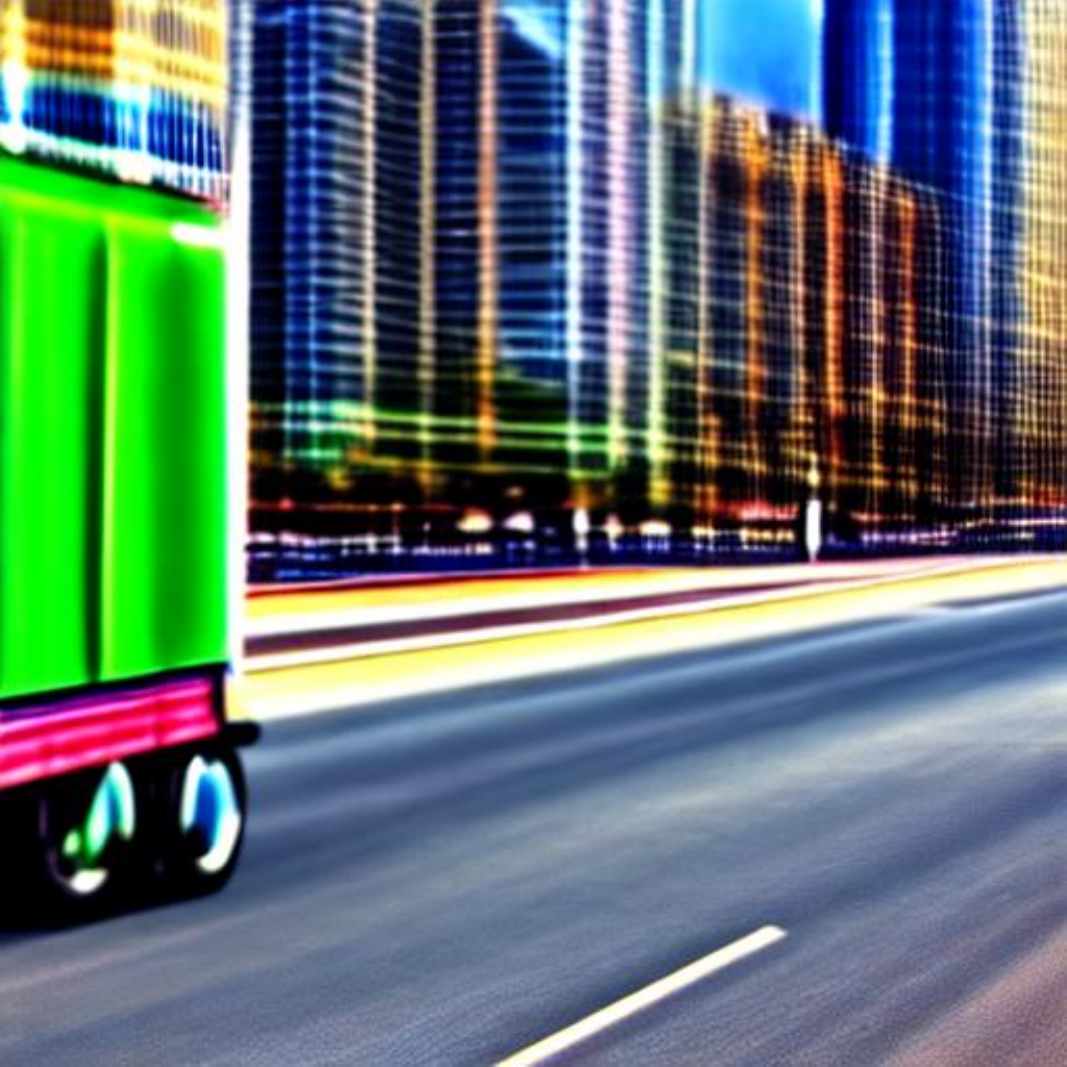} \\
};
\matrix (c2bot) [rowgrid, below=1pt of c2top]
{
\includegraphics[width=0.132\textwidth]{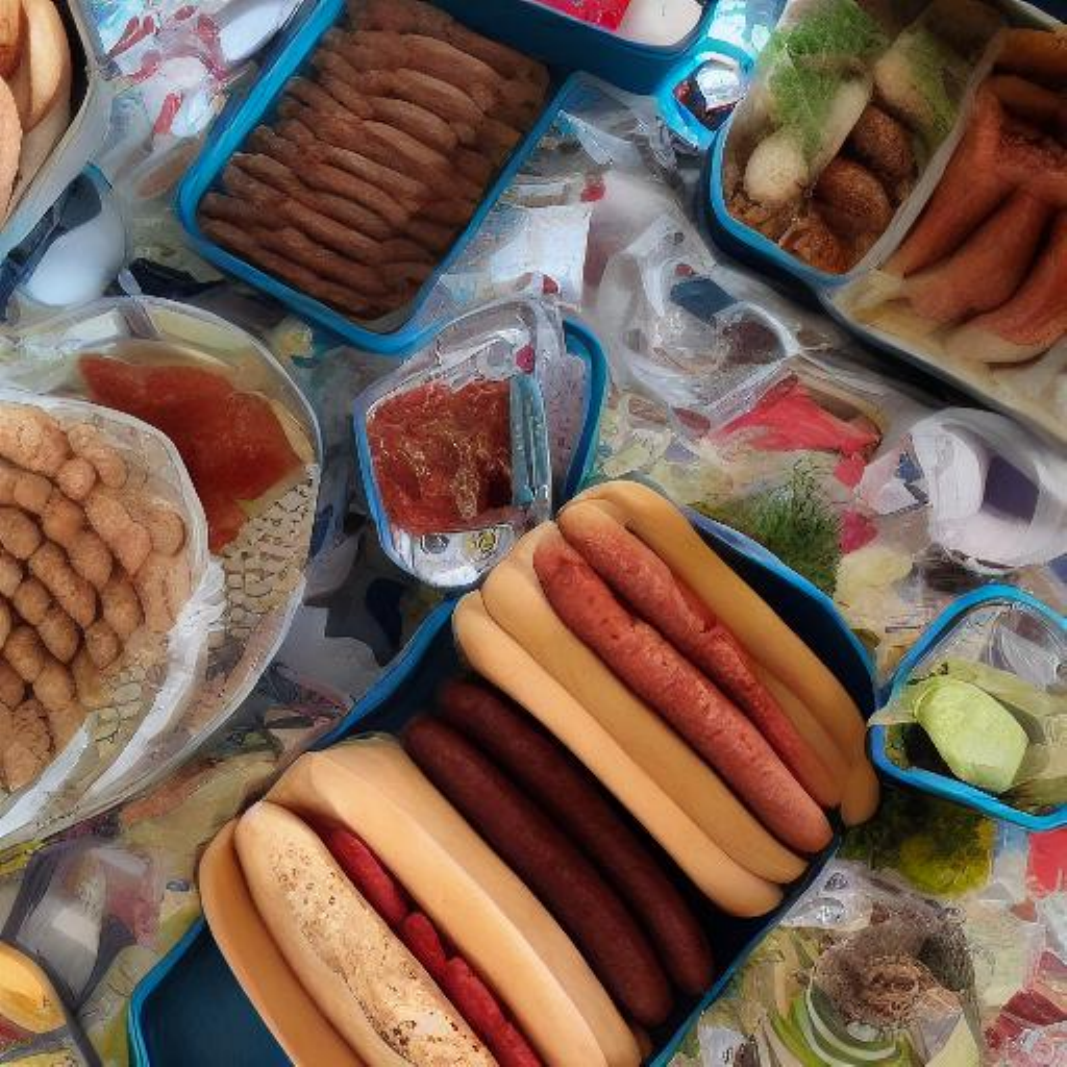} &
\includegraphics[width=0.132\textwidth]{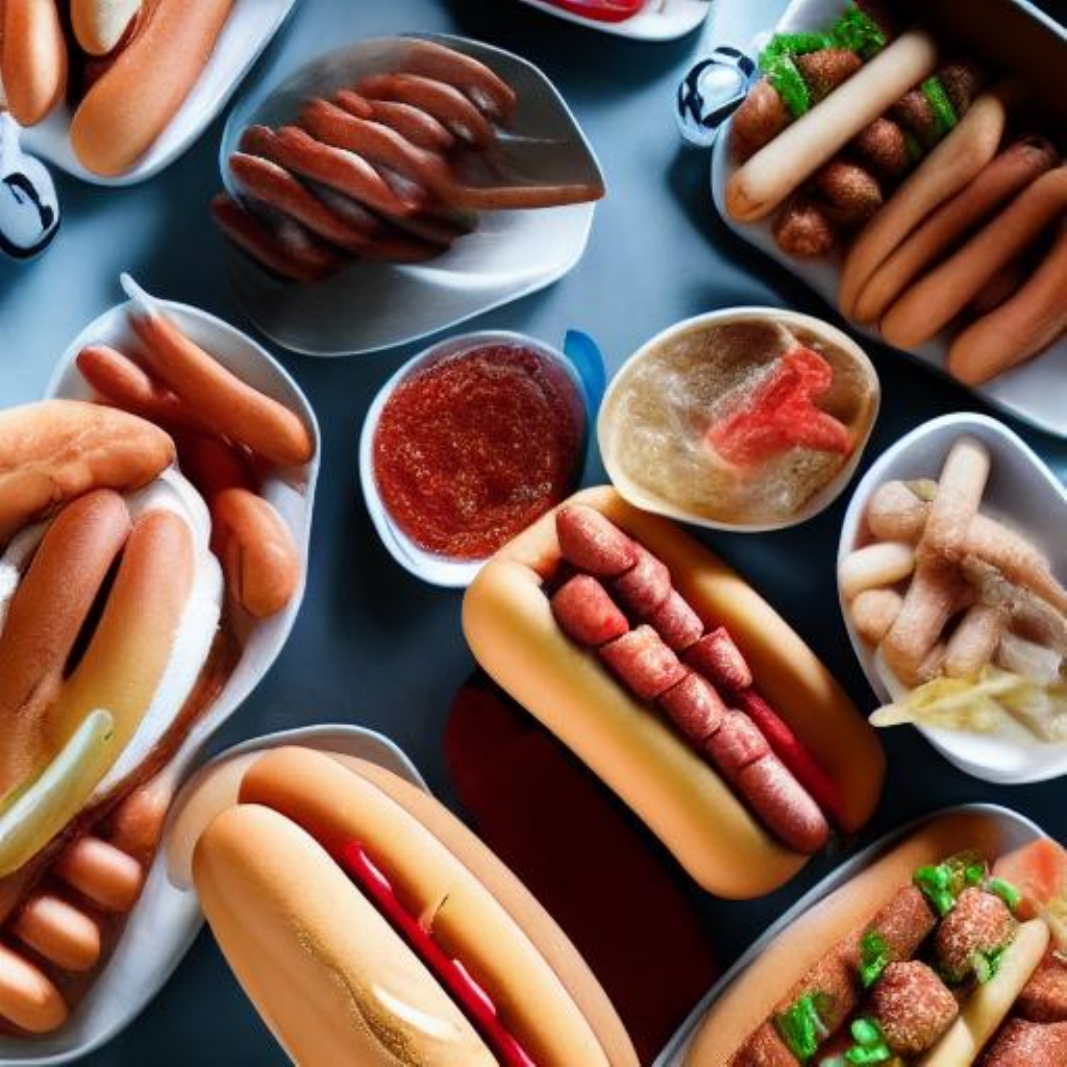} &
\includegraphics[width=0.132\textwidth]{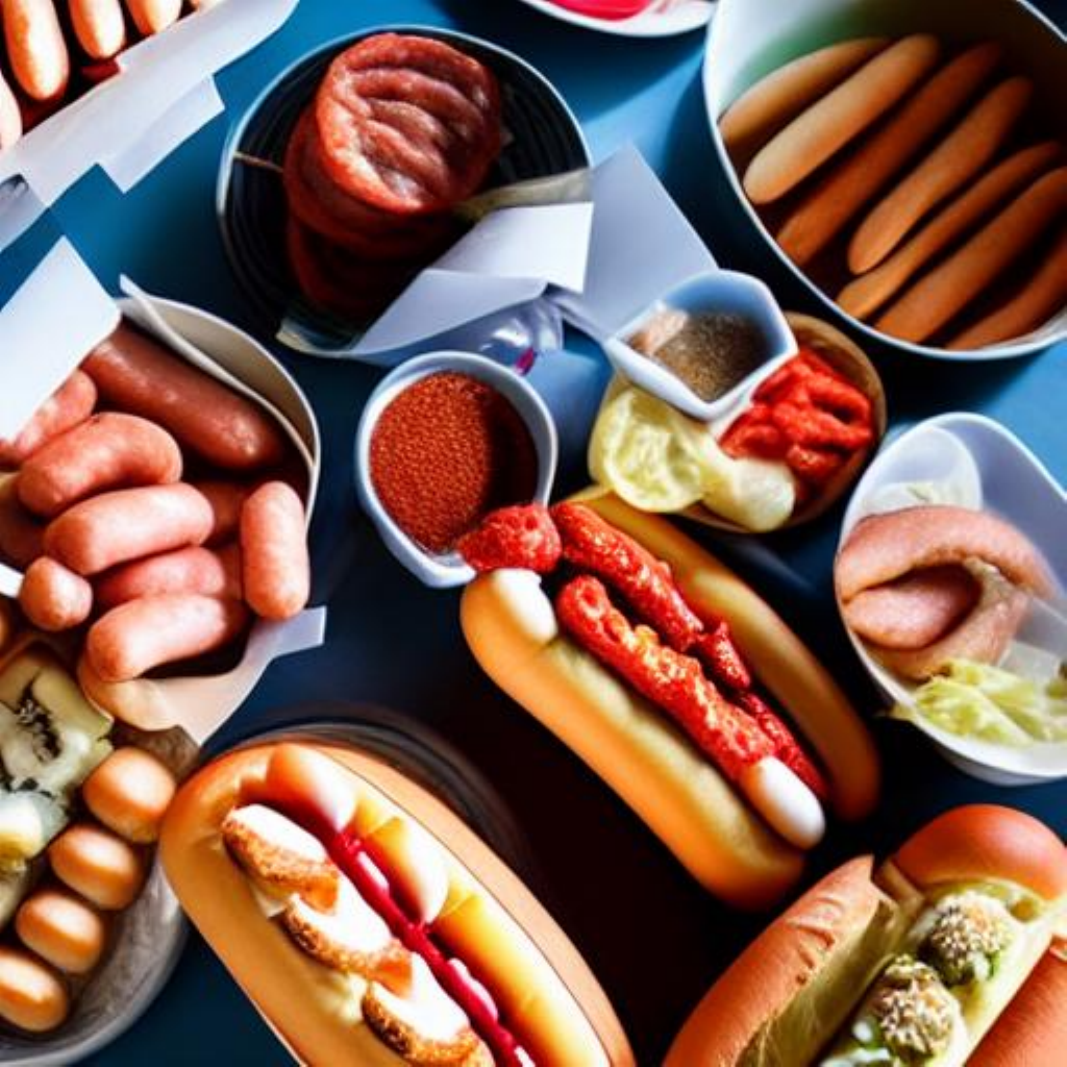} &
\includegraphics[width=0.132\textwidth]{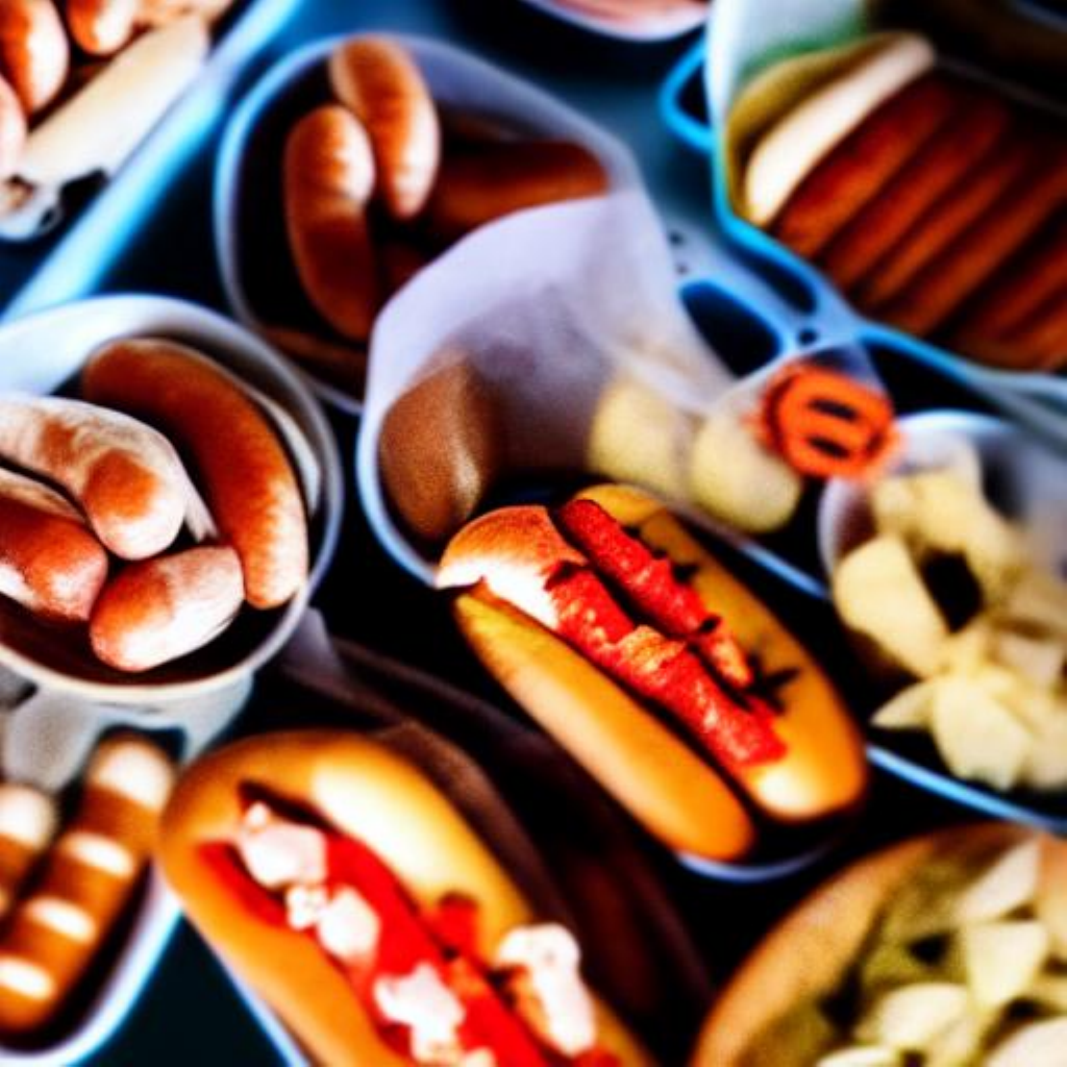} &
\includegraphics[width=0.132\textwidth]{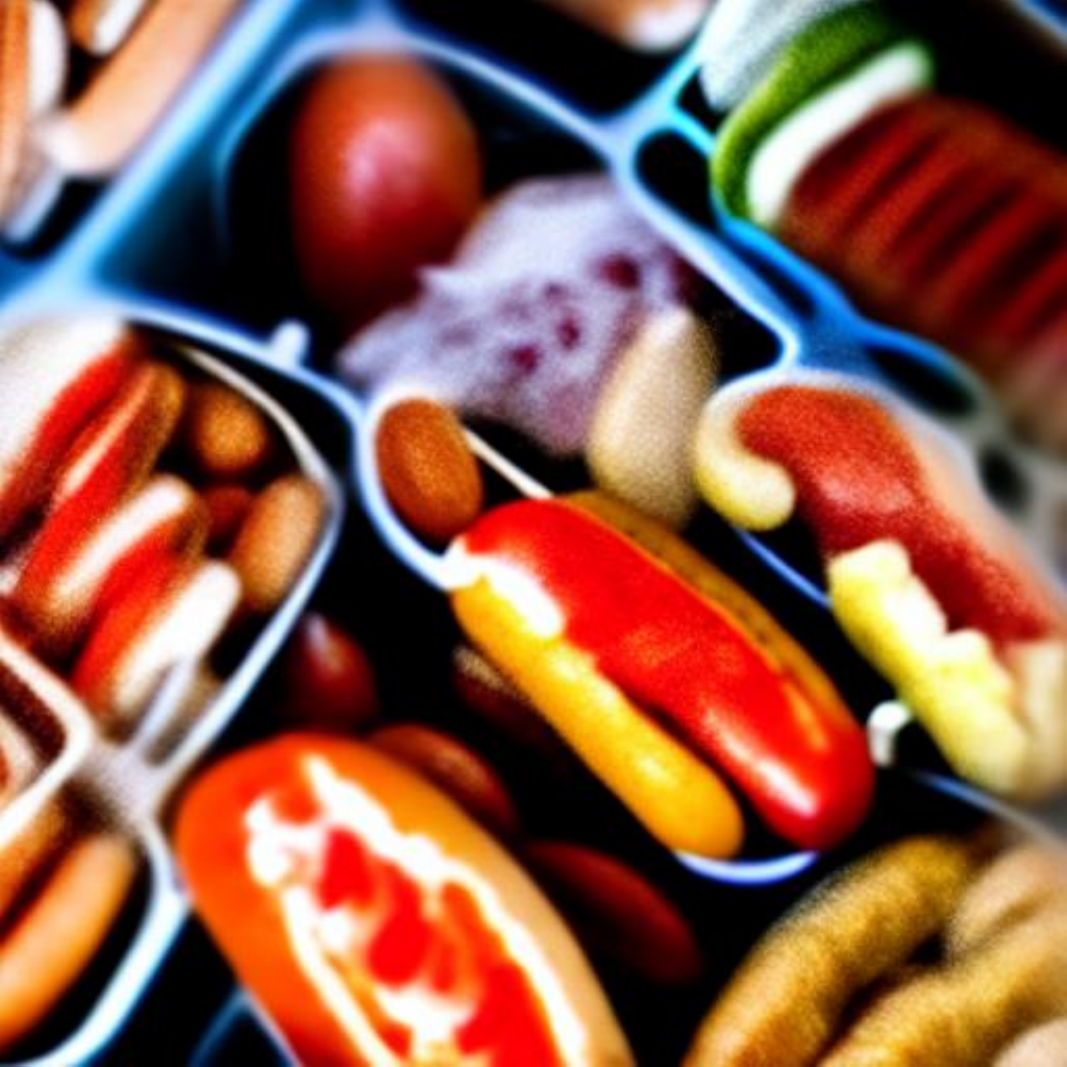} &
\includegraphics[width=0.132\textwidth]{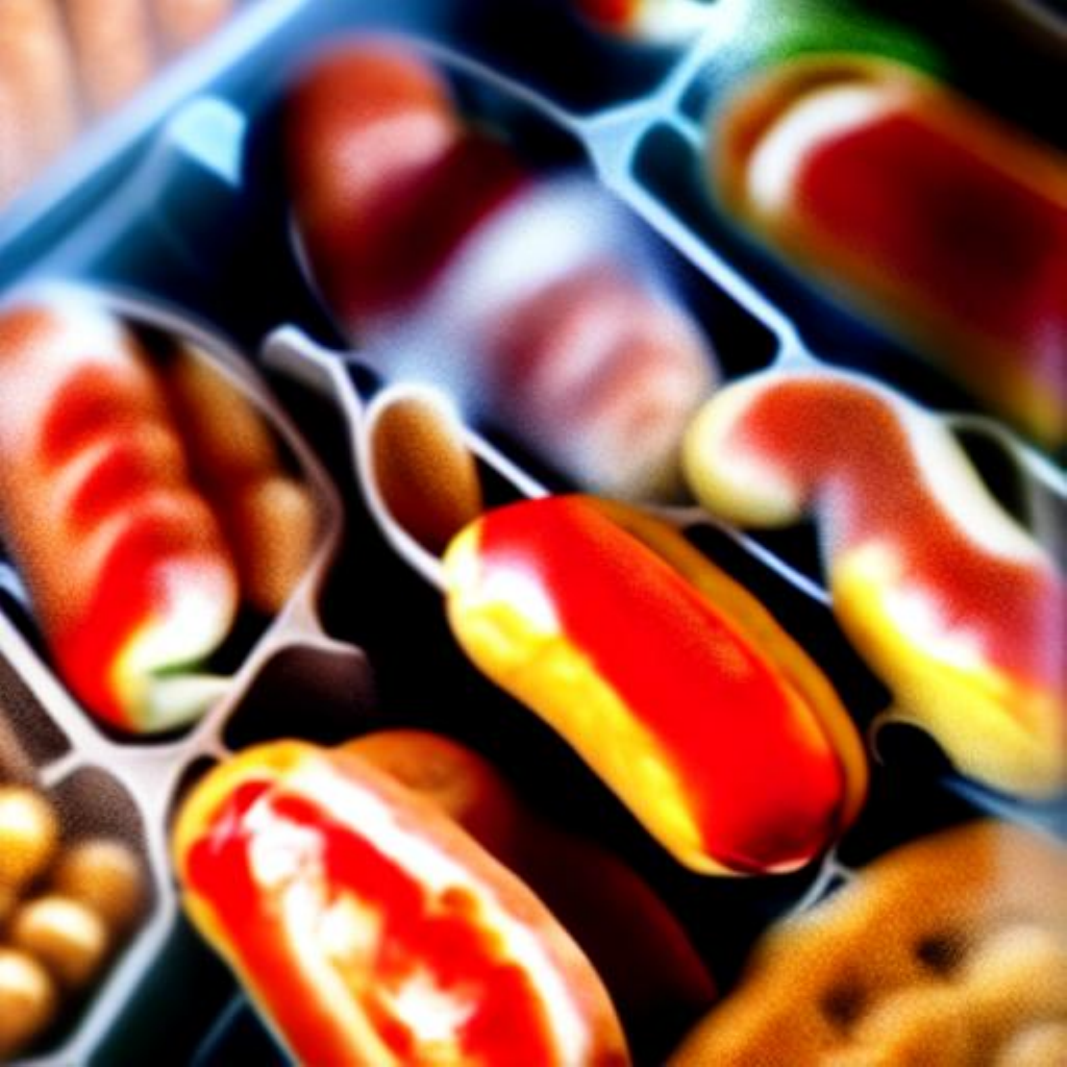} \\
};
\node[rotate=90, anchor=center, yshift=8pt] at ($(c2top.west)!0.5!(c2bot.west)$) {Interval CFG};

\matrix (c3top) [rowgrid, below=3pt of c2bot]
{
\includegraphics[width=0.132\textwidth]{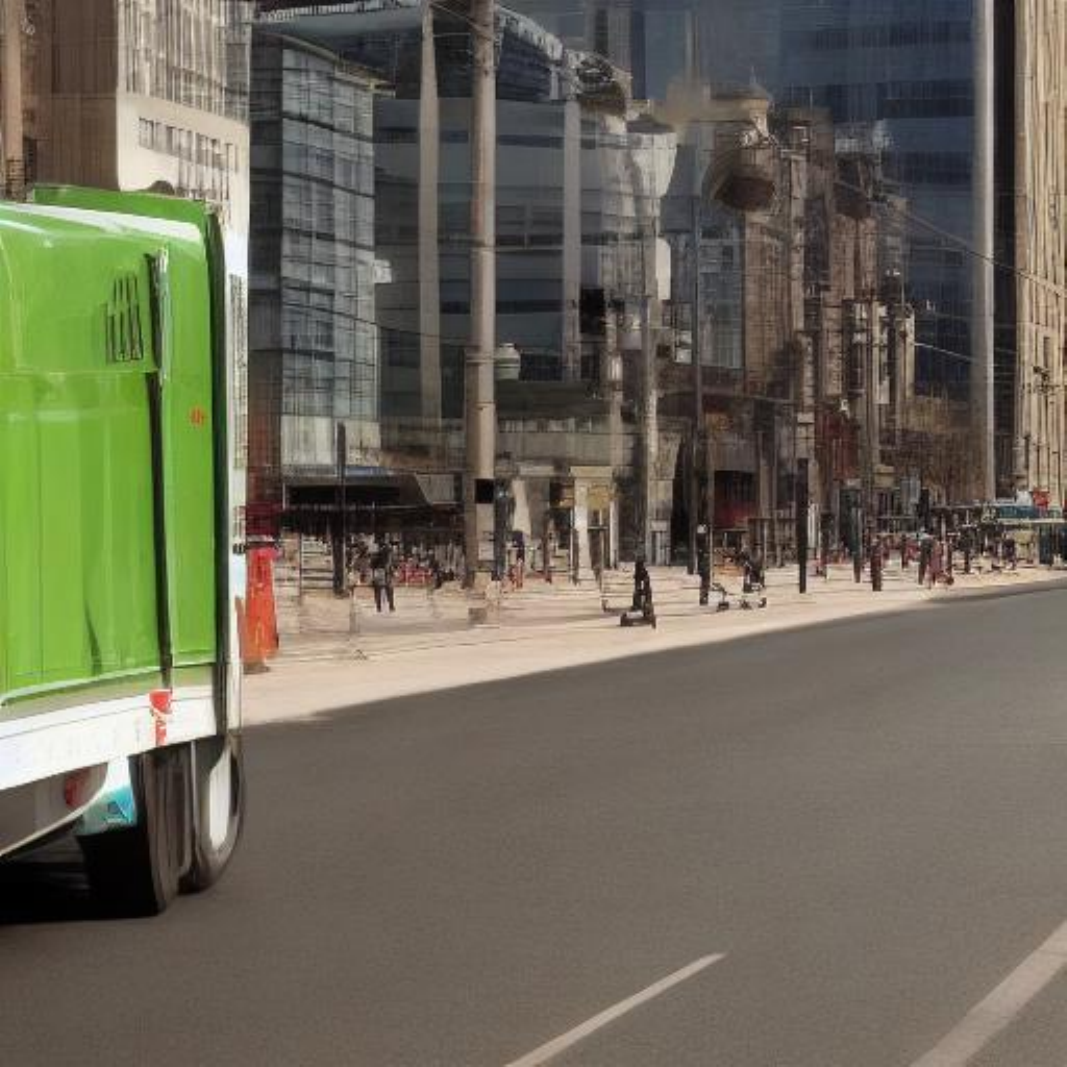} &
\includegraphics[width=0.132\textwidth]{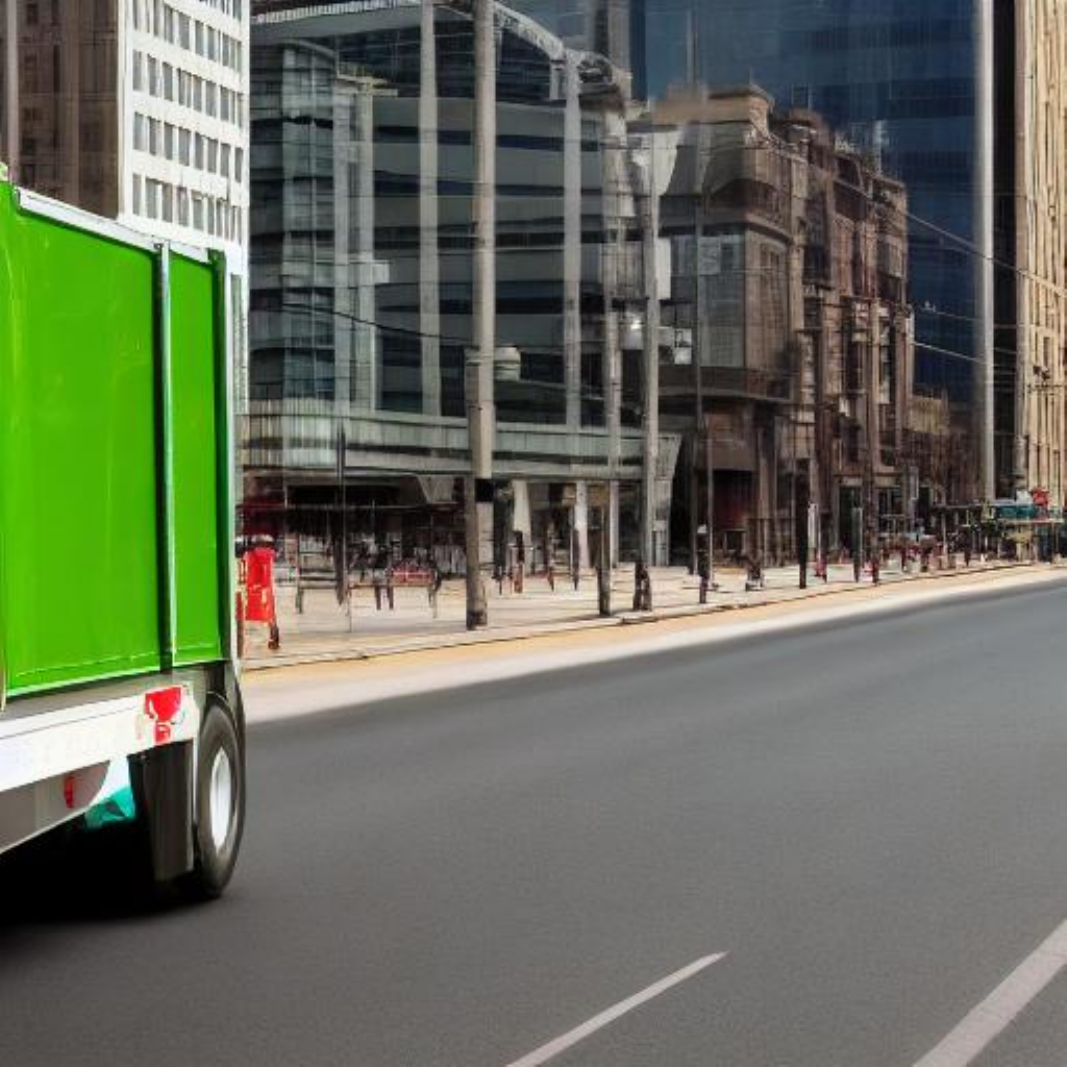} &
\includegraphics[width=0.132\textwidth]{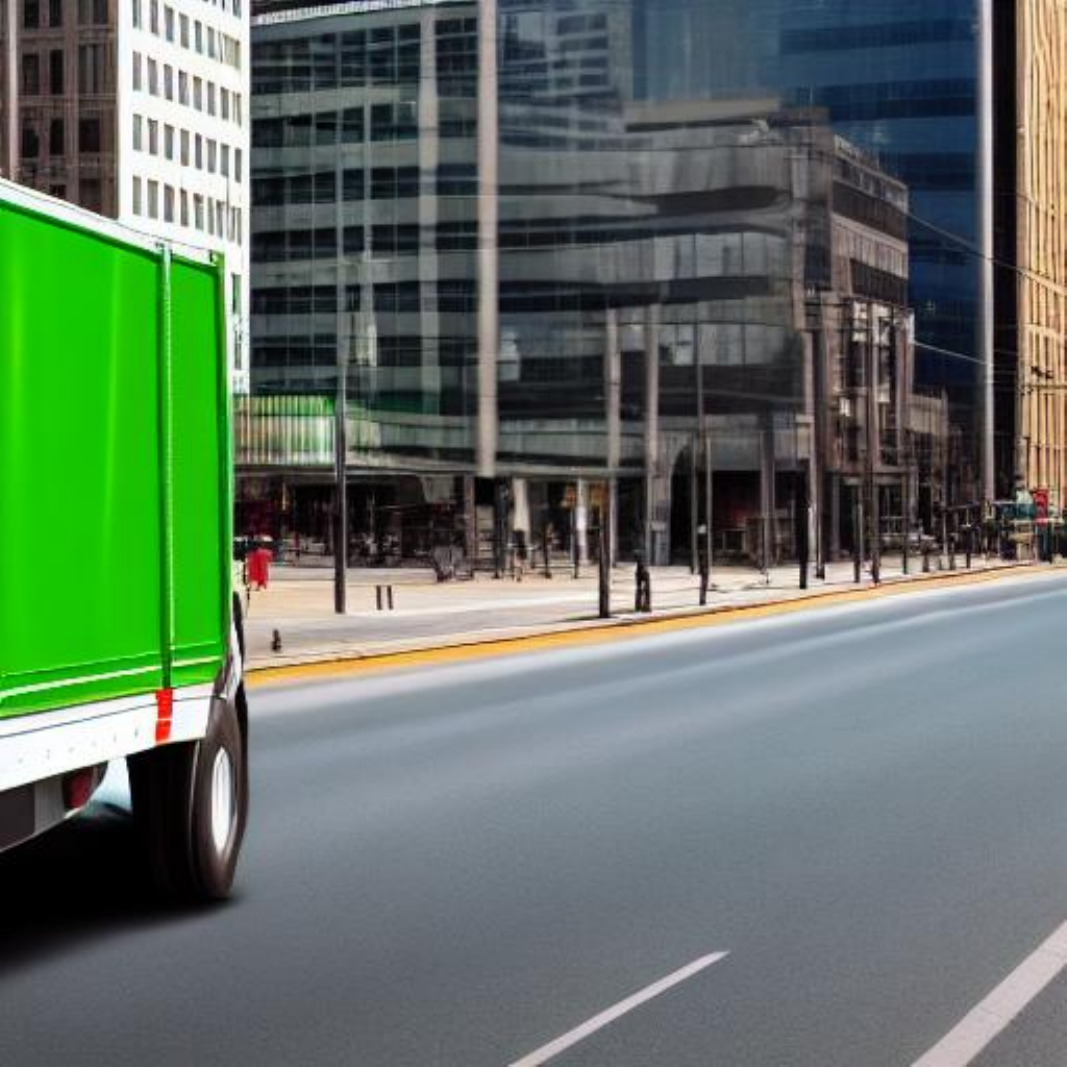} &
\includegraphics[width=0.132\textwidth]{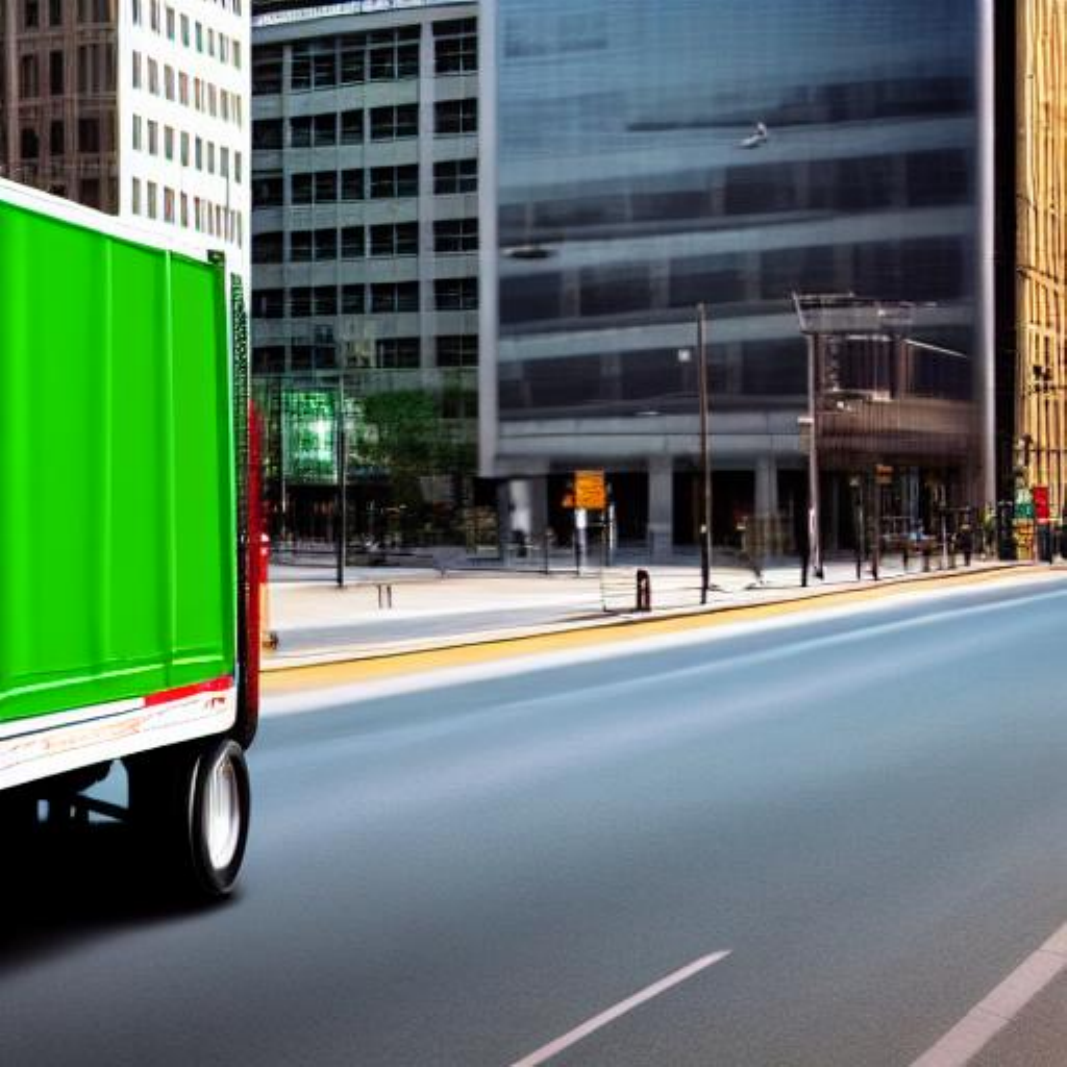} &
\includegraphics[width=0.132\textwidth]{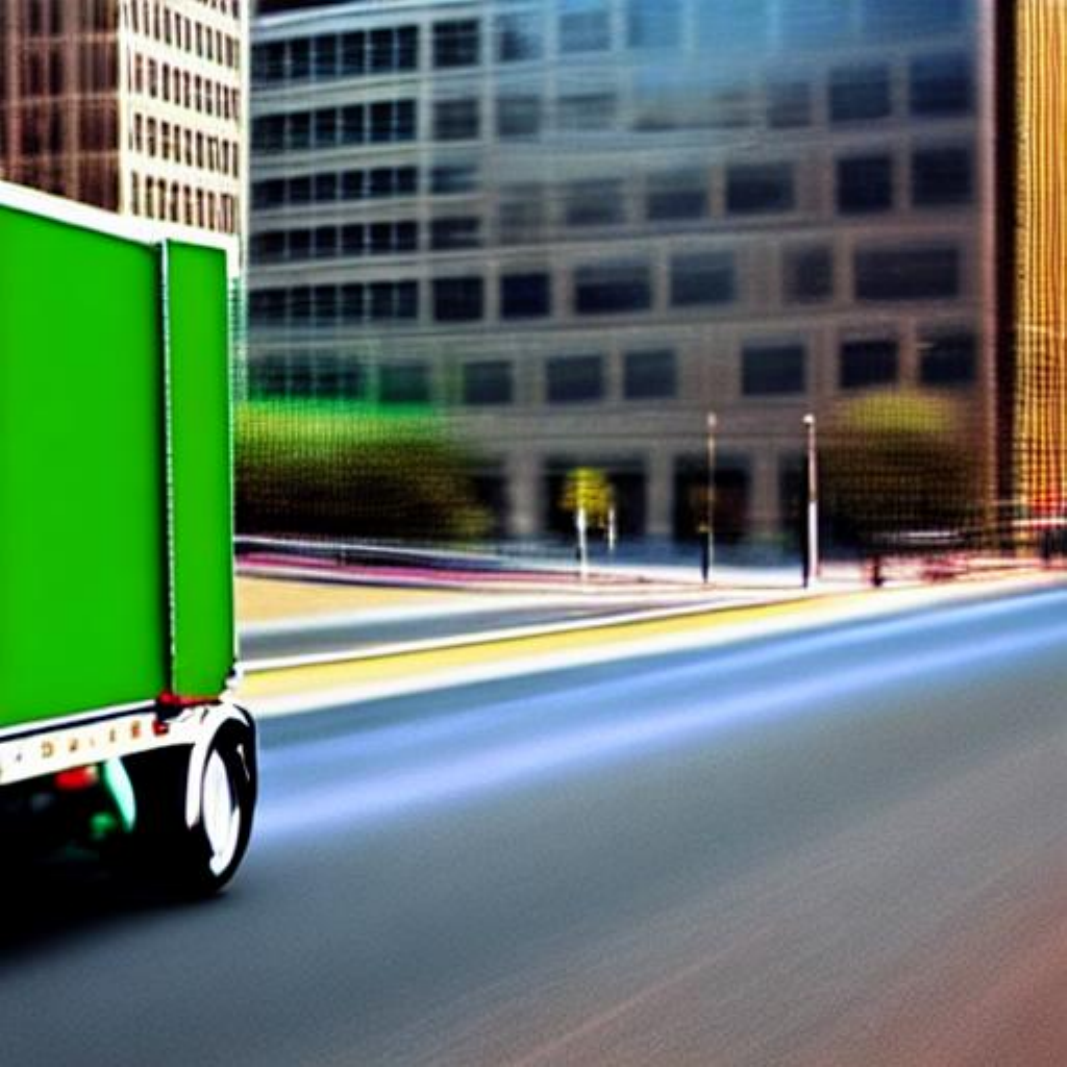} &
\includegraphics[width=0.132\textwidth]{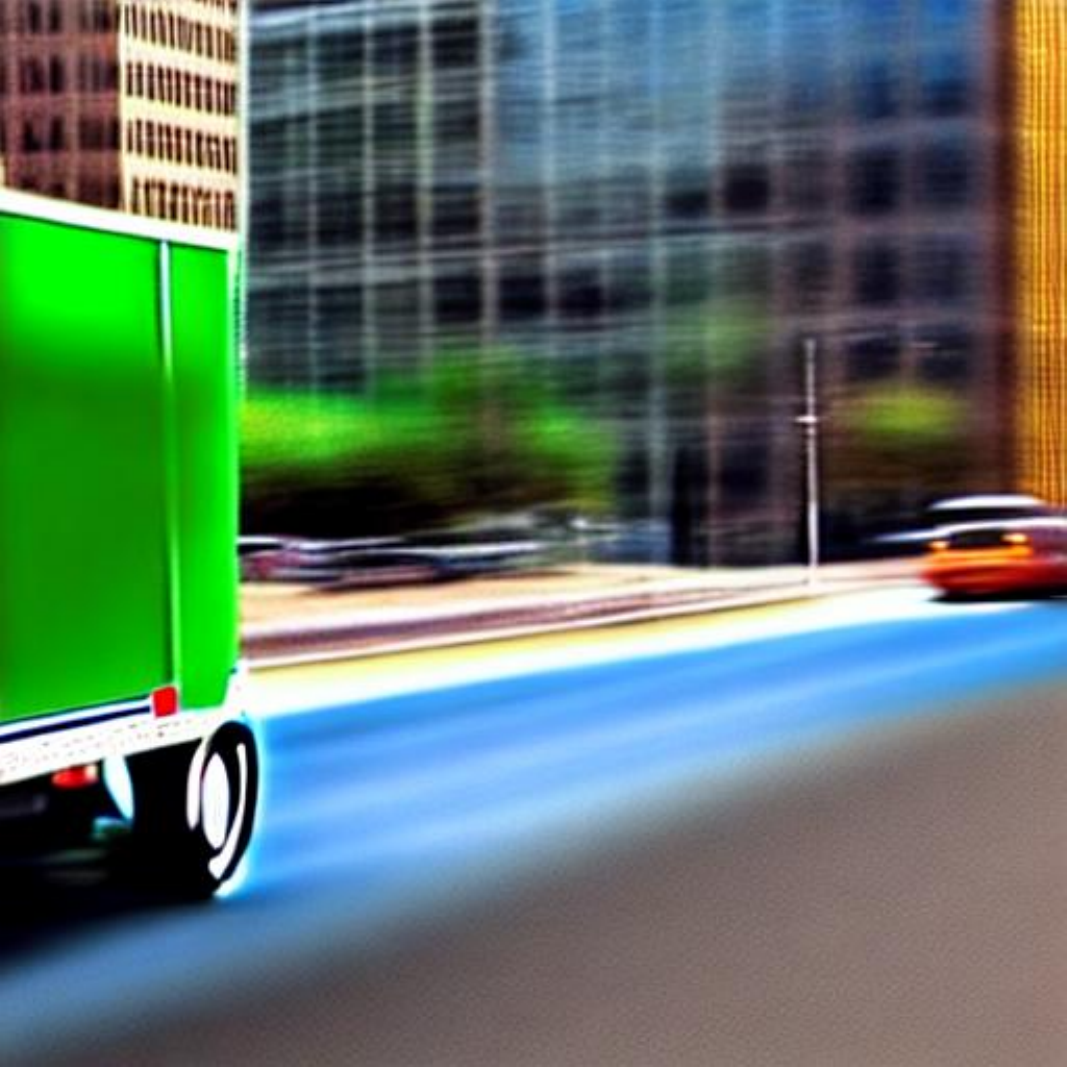} \\
};
\matrix (c3bot) [rowgrid, below=1pt of c3top]
{
\includegraphics[width=0.132\textwidth]{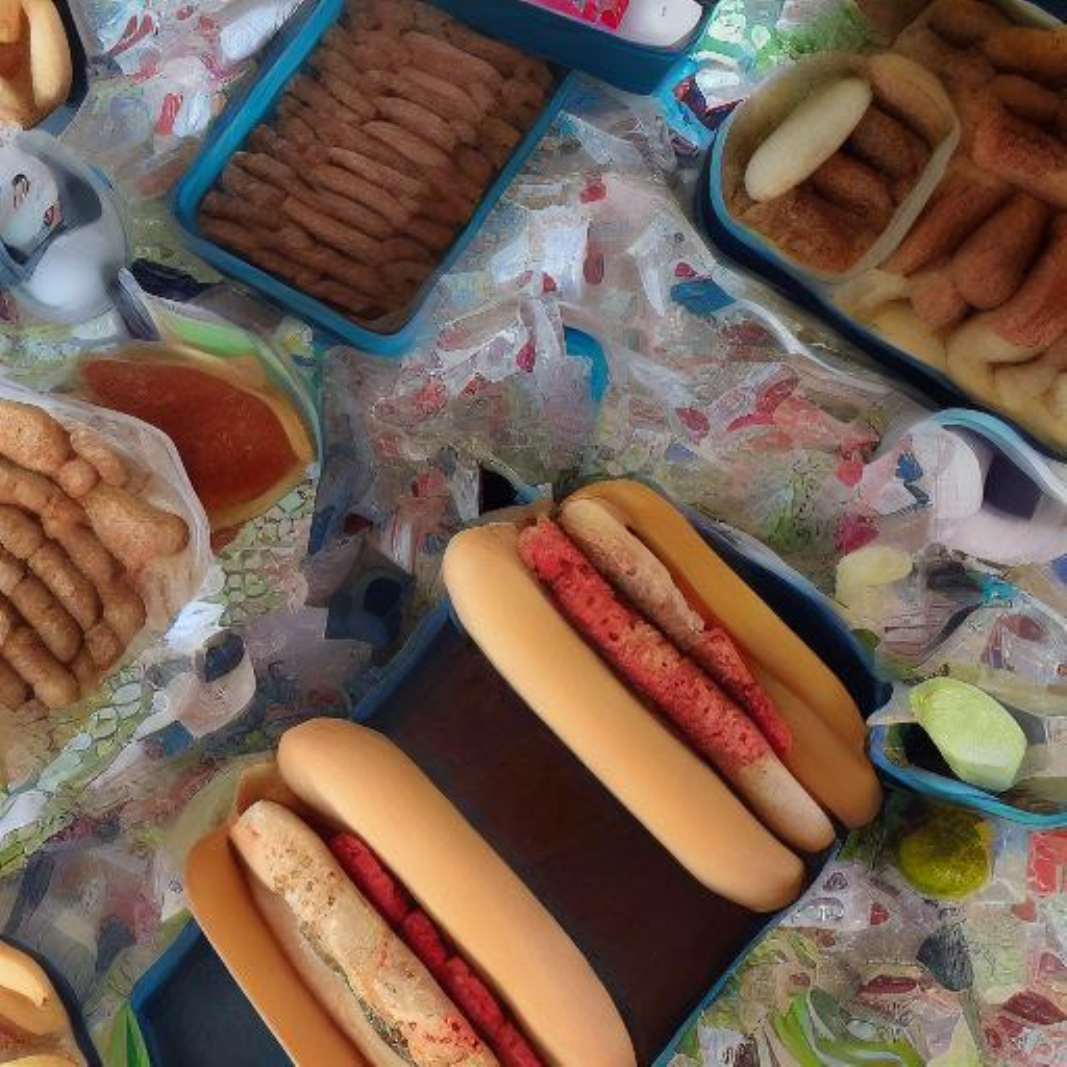} &
\includegraphics[width=0.132\textwidth]{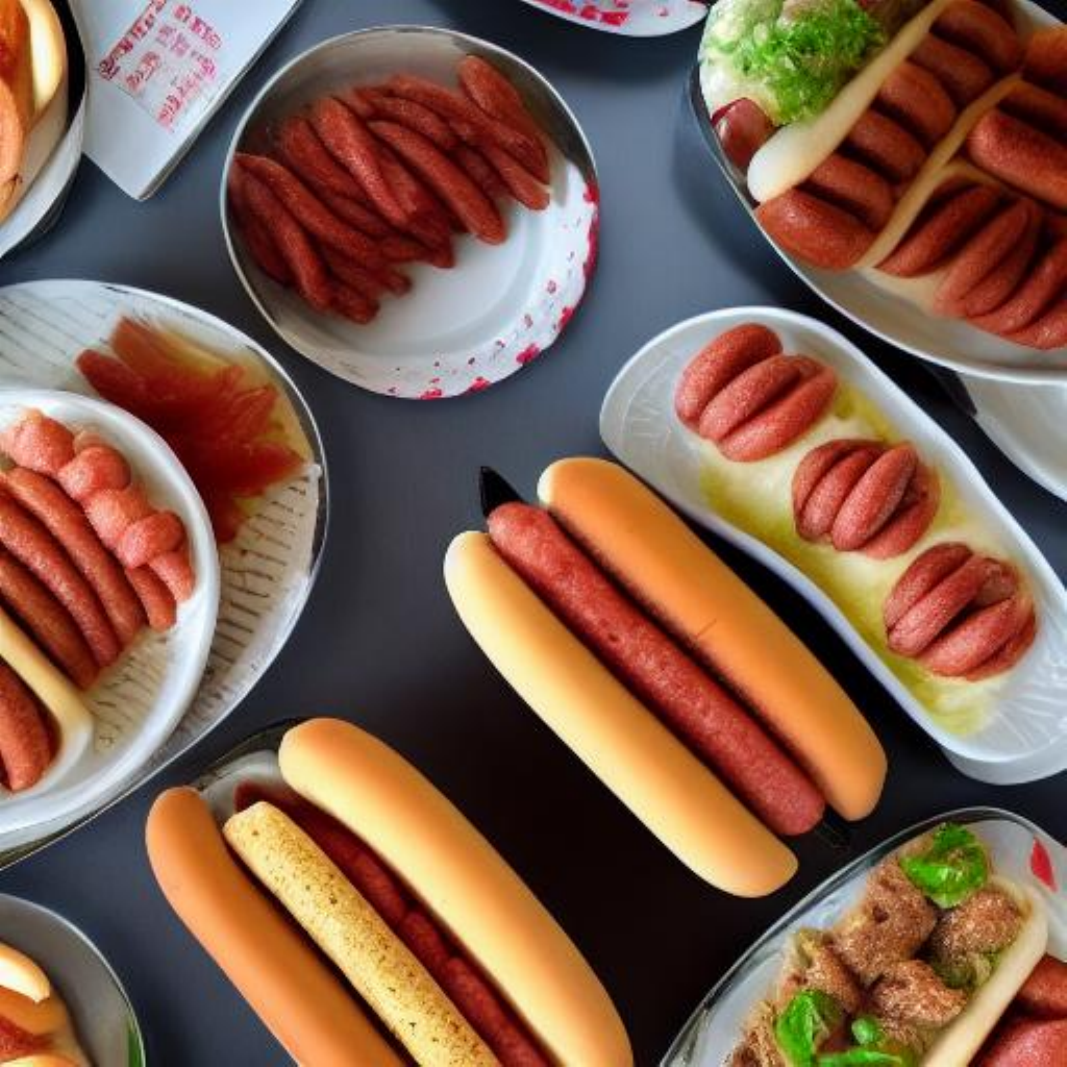} &
\includegraphics[width=0.132\textwidth]{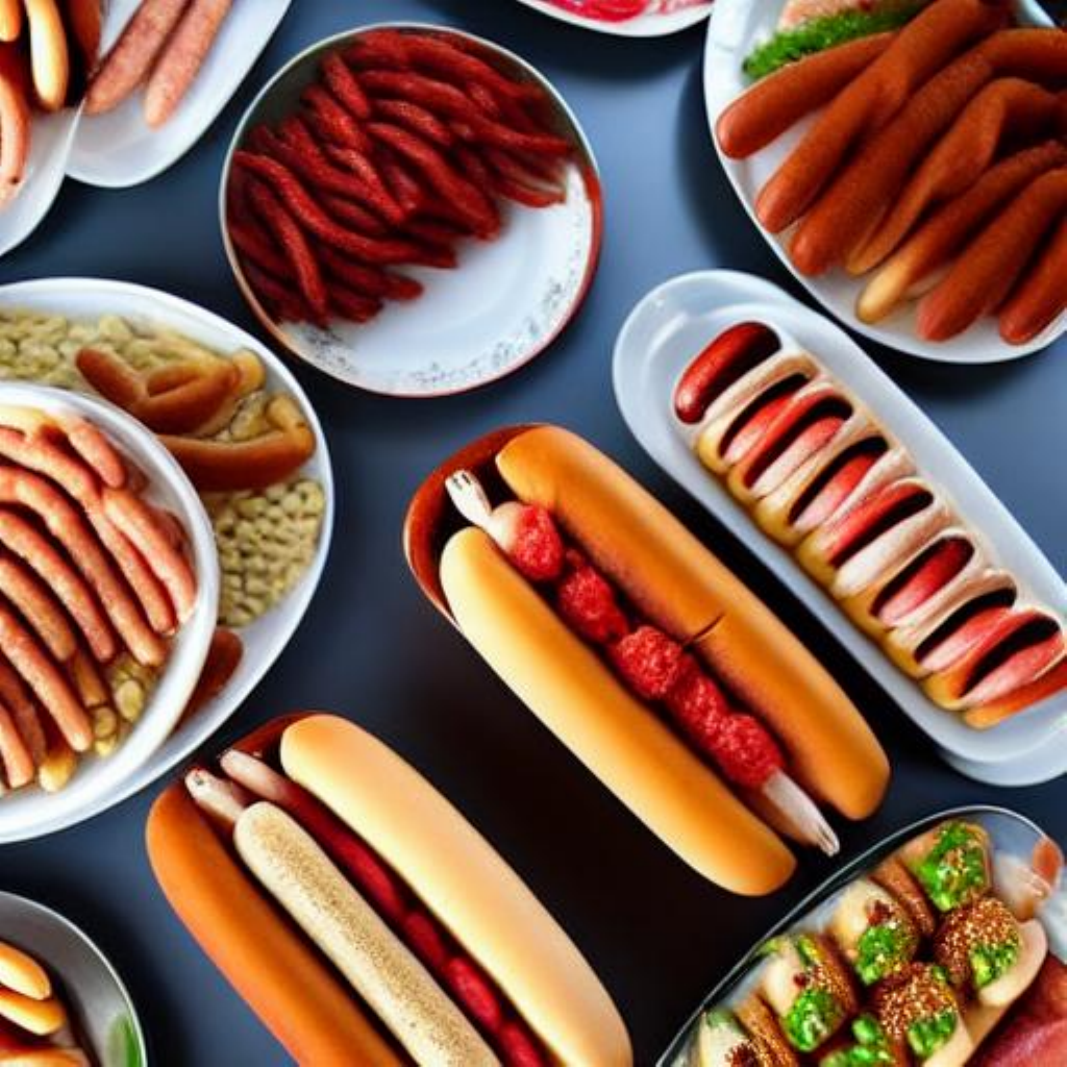} &
\includegraphics[width=0.132\textwidth]{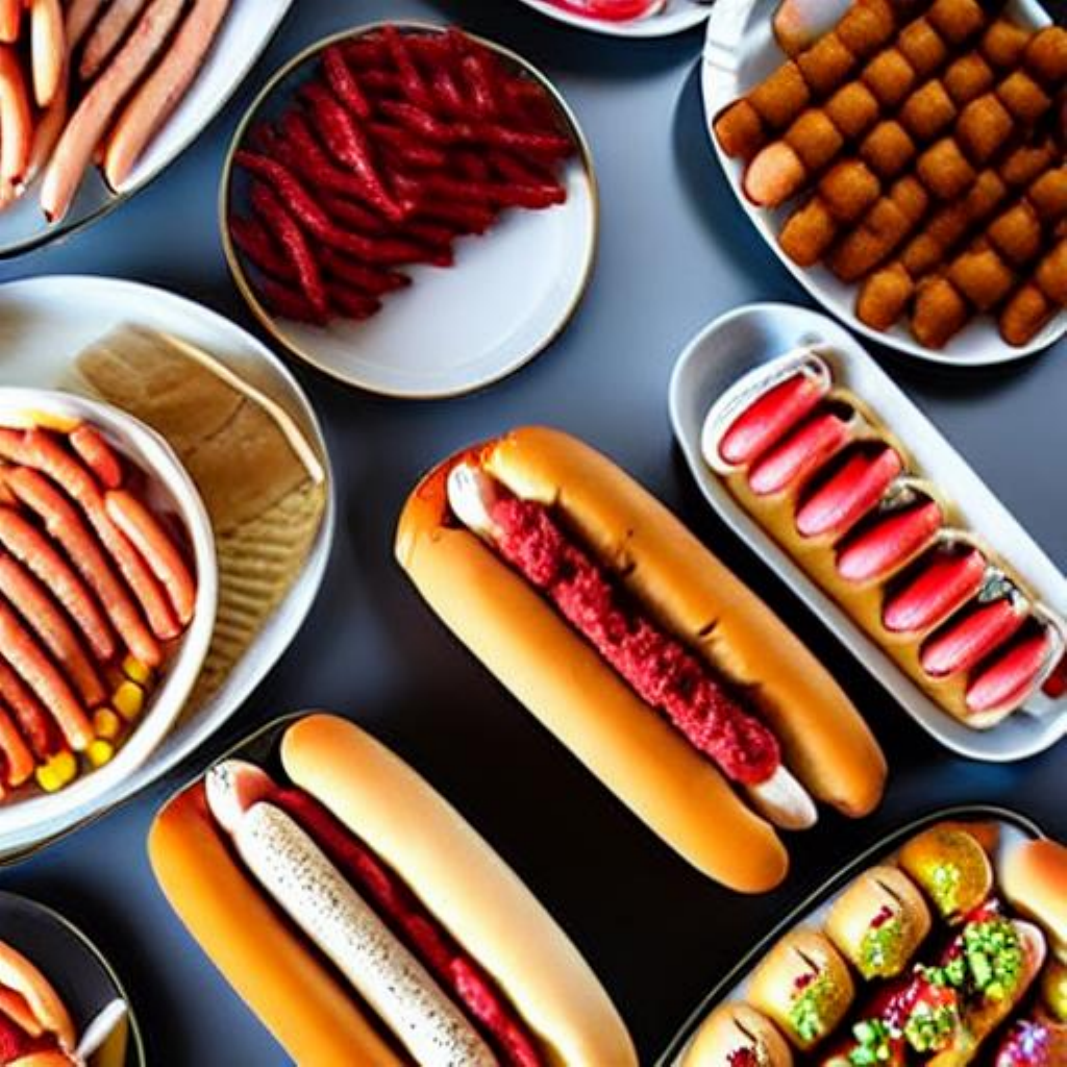} &
\includegraphics[width=0.132\textwidth]{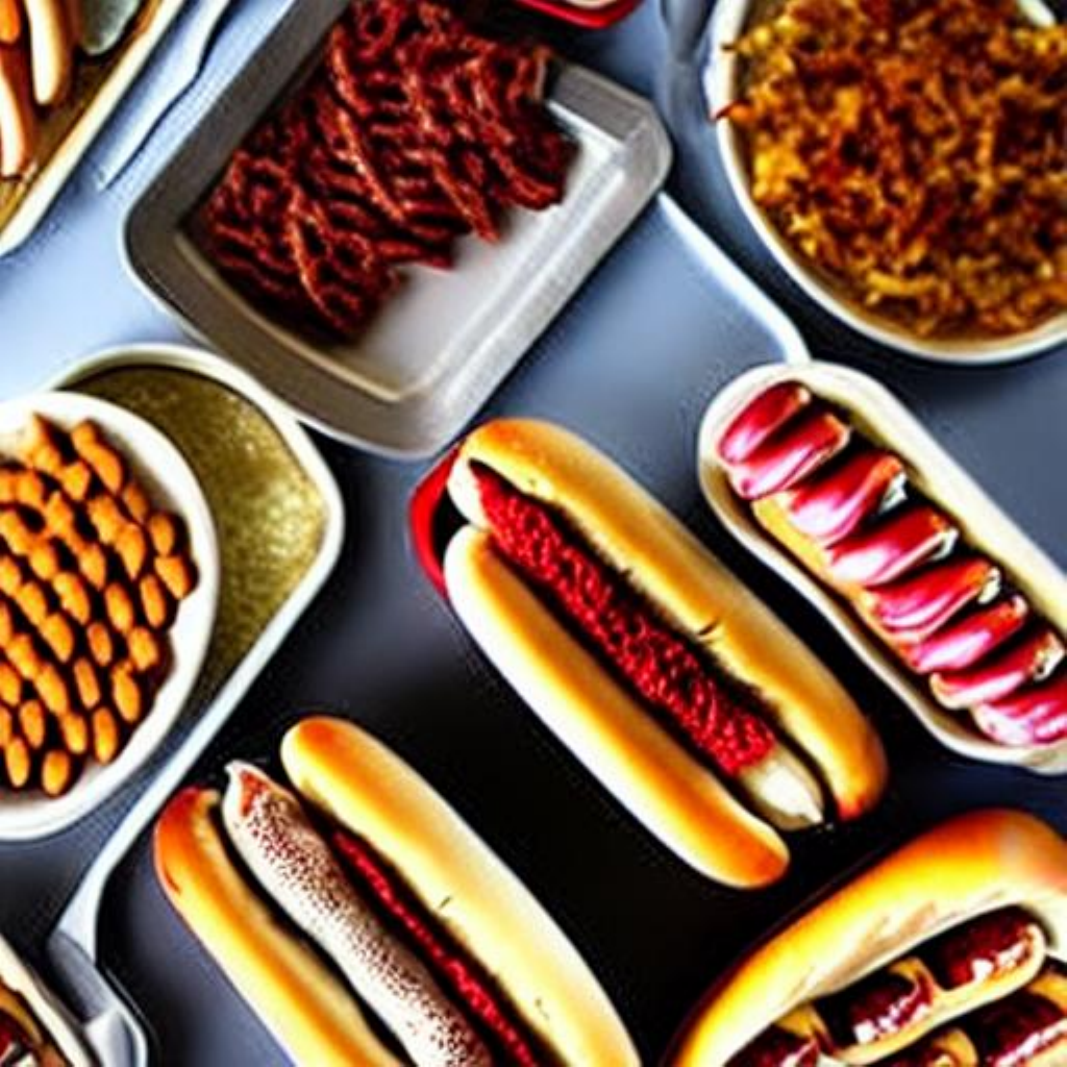} &
\includegraphics[width=0.132\textwidth]{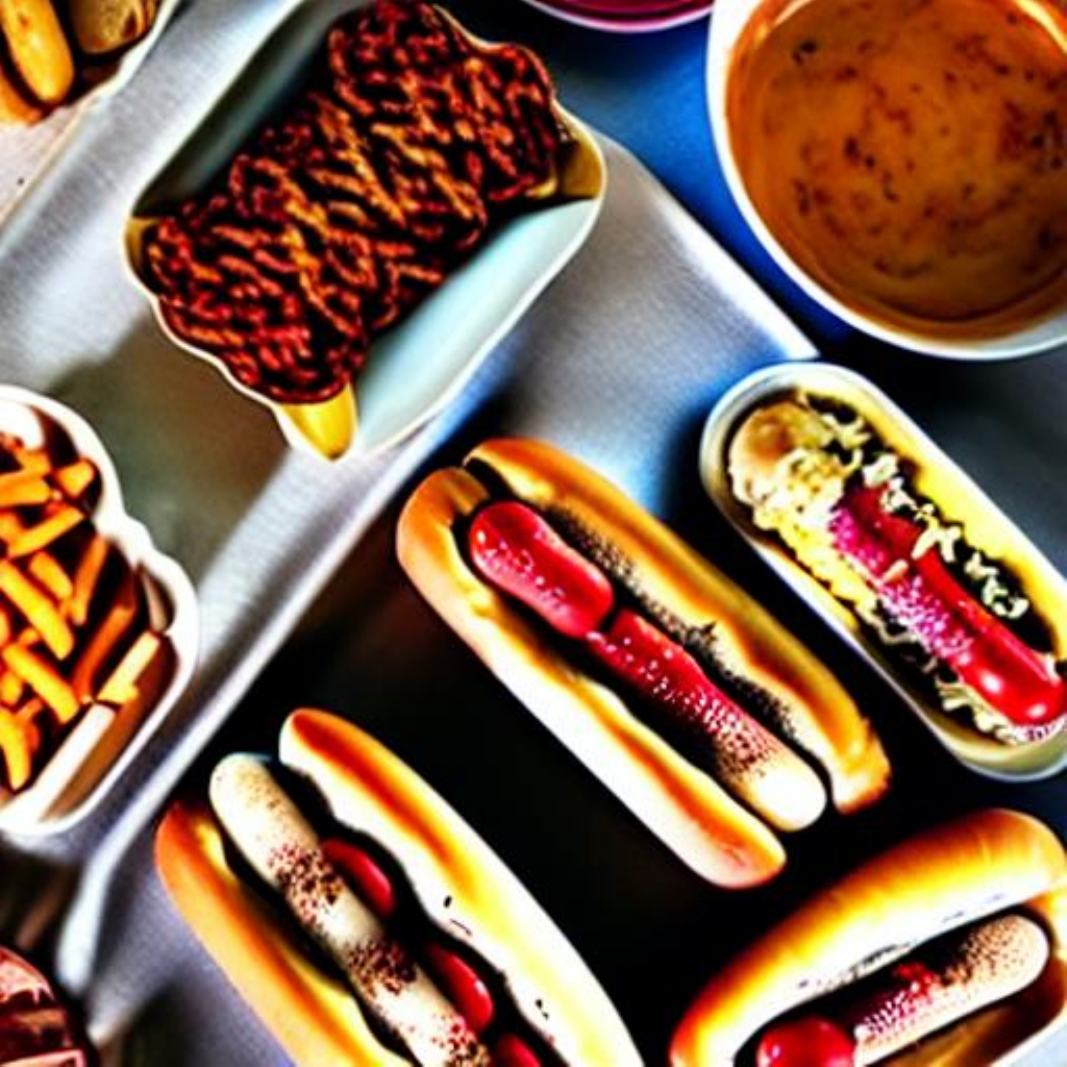} \\
};
\node[rotate=90, anchor=center, yshift=8pt] at ($(c3top.west)!0.5!(c3bot.west)$) {$\beta$--CFG};

\matrix (c4top) [rowgrid, below=3pt of c3bot]
{
\includegraphics[width=0.132\textwidth]{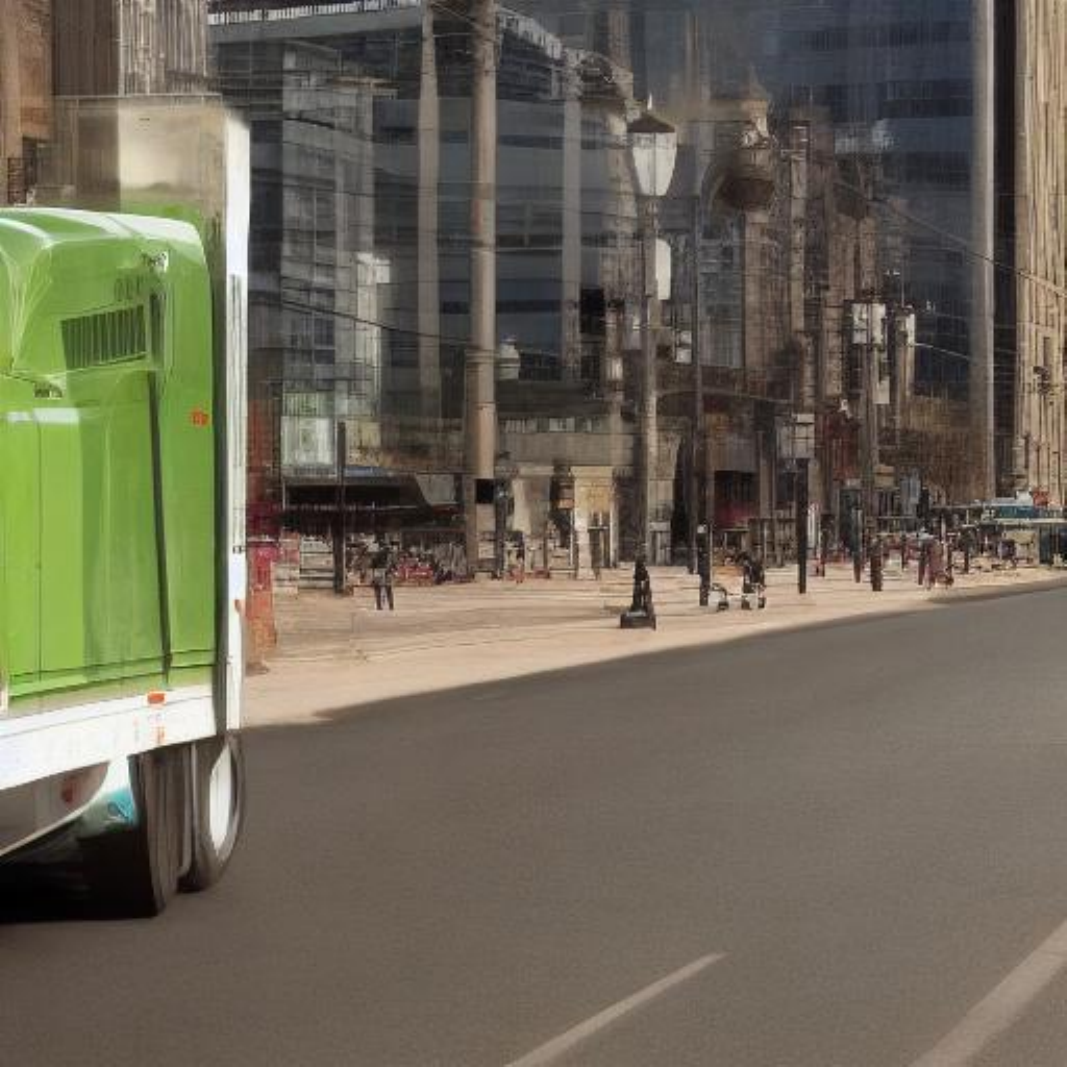} &
\includegraphics[width=0.132\textwidth]{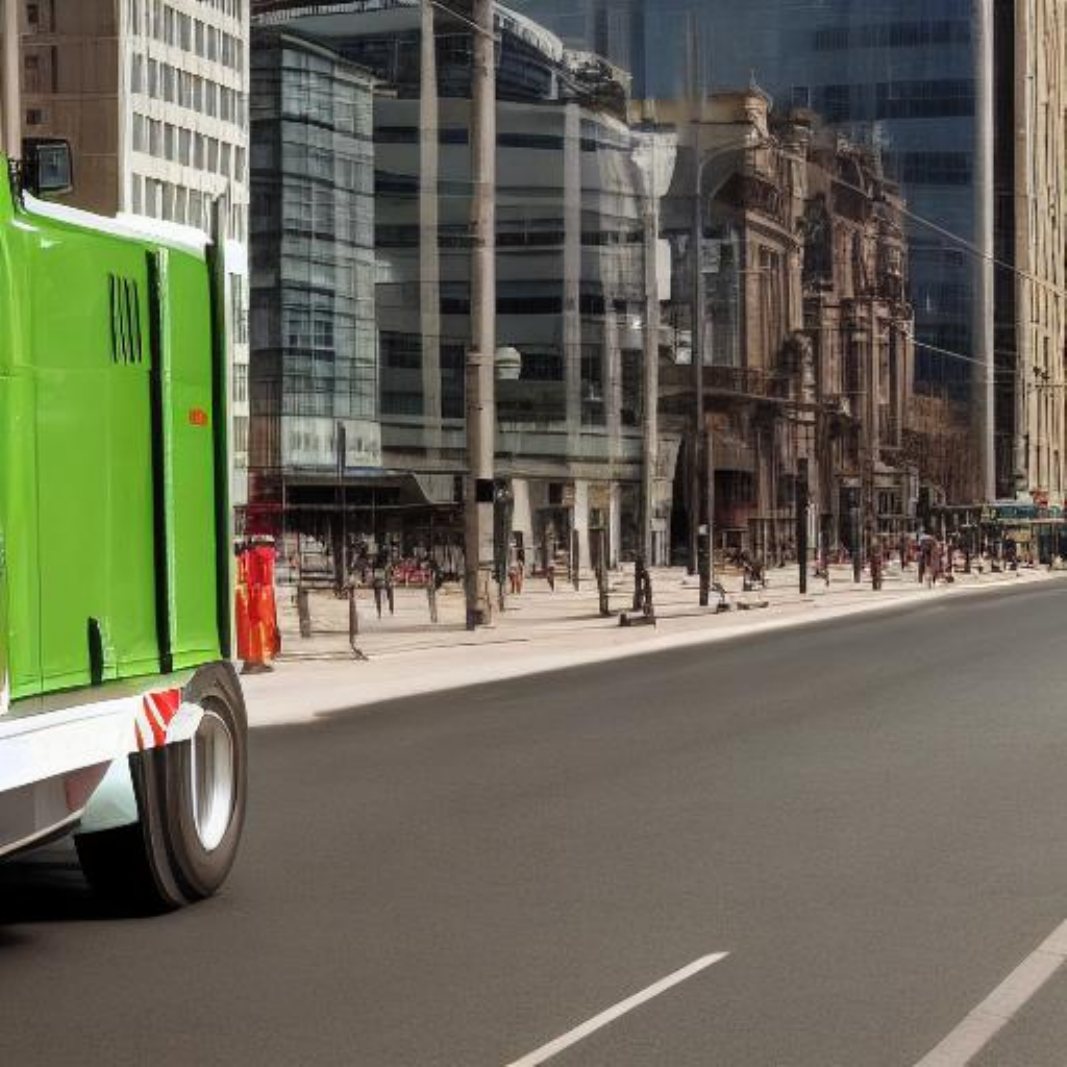} &
\includegraphics[width=0.132\textwidth]{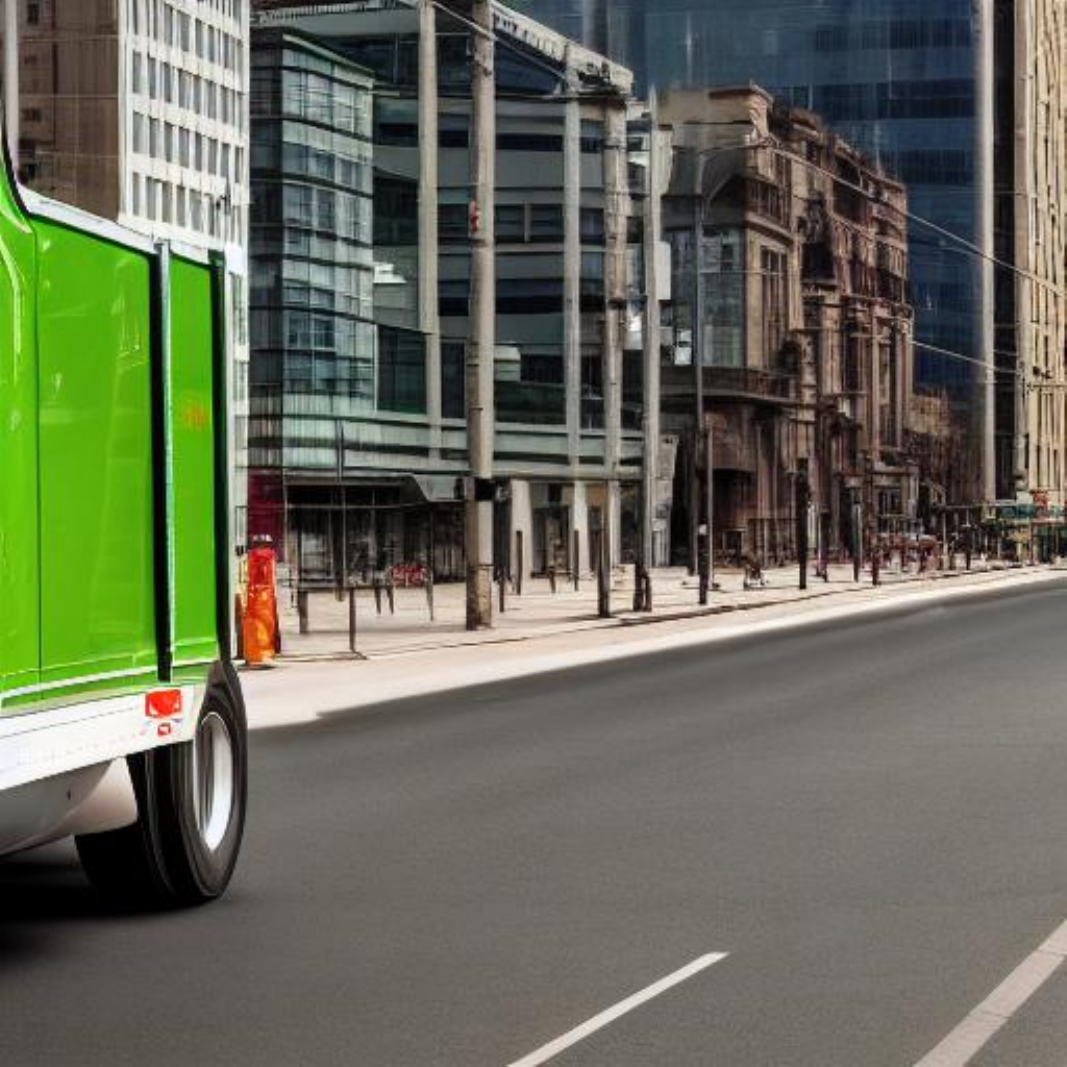} &
\includegraphics[width=0.132\textwidth]{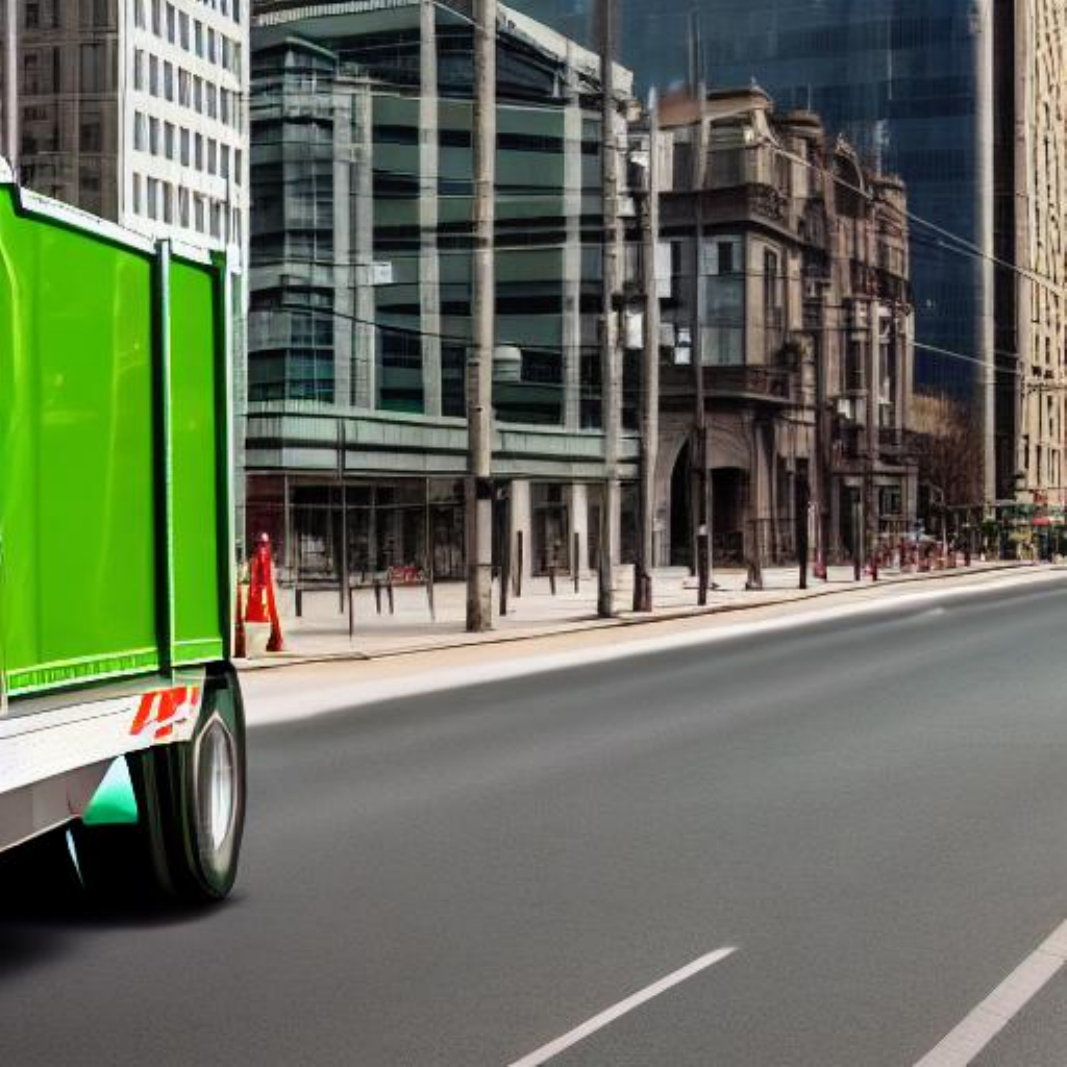} &
\includegraphics[width=0.132\textwidth]{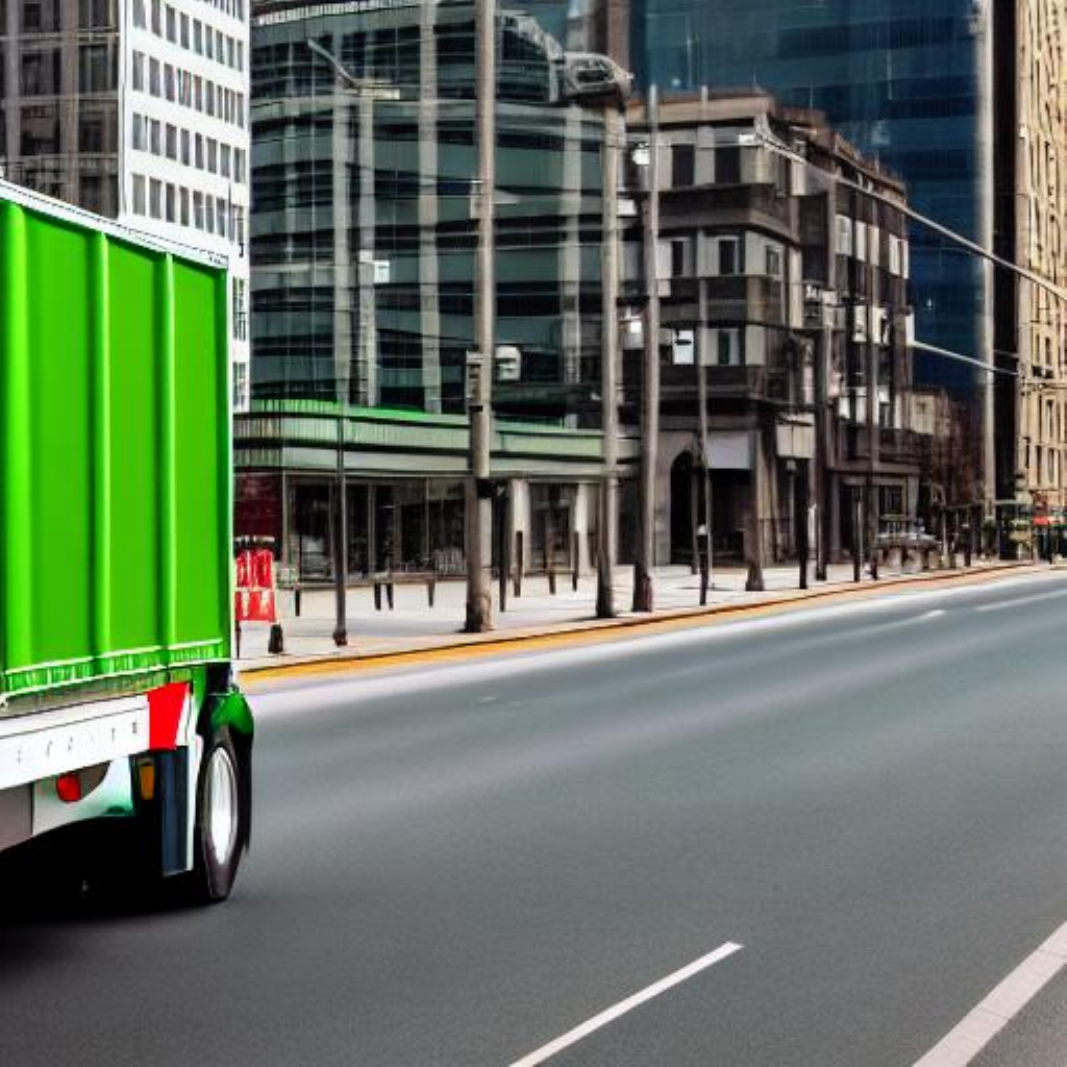} &
\includegraphics[width=0.132\textwidth]{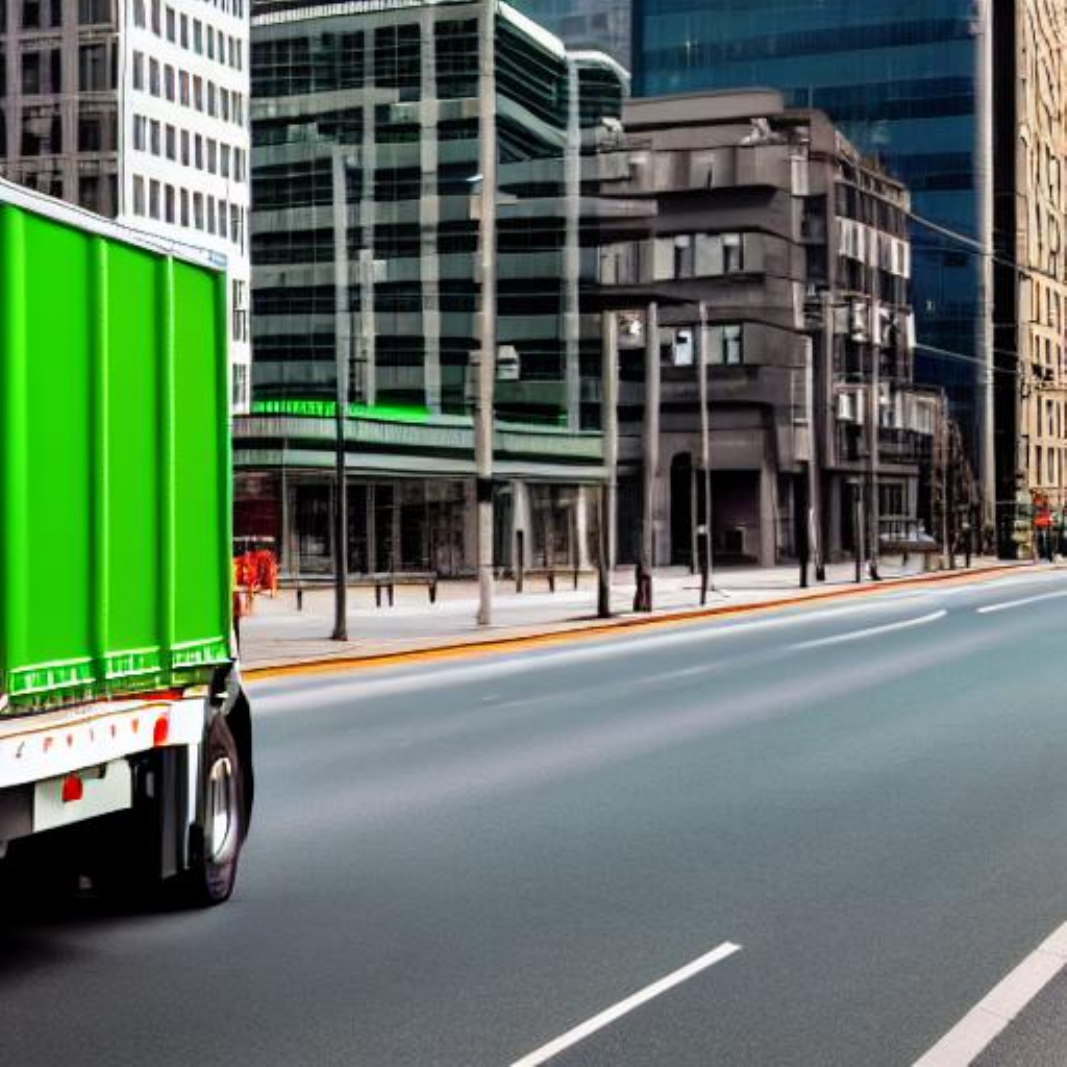} \\
};
\matrix (c4bot) [rowgrid, below=1pt of c4top]
{
\includegraphics[width=0.132\textwidth]{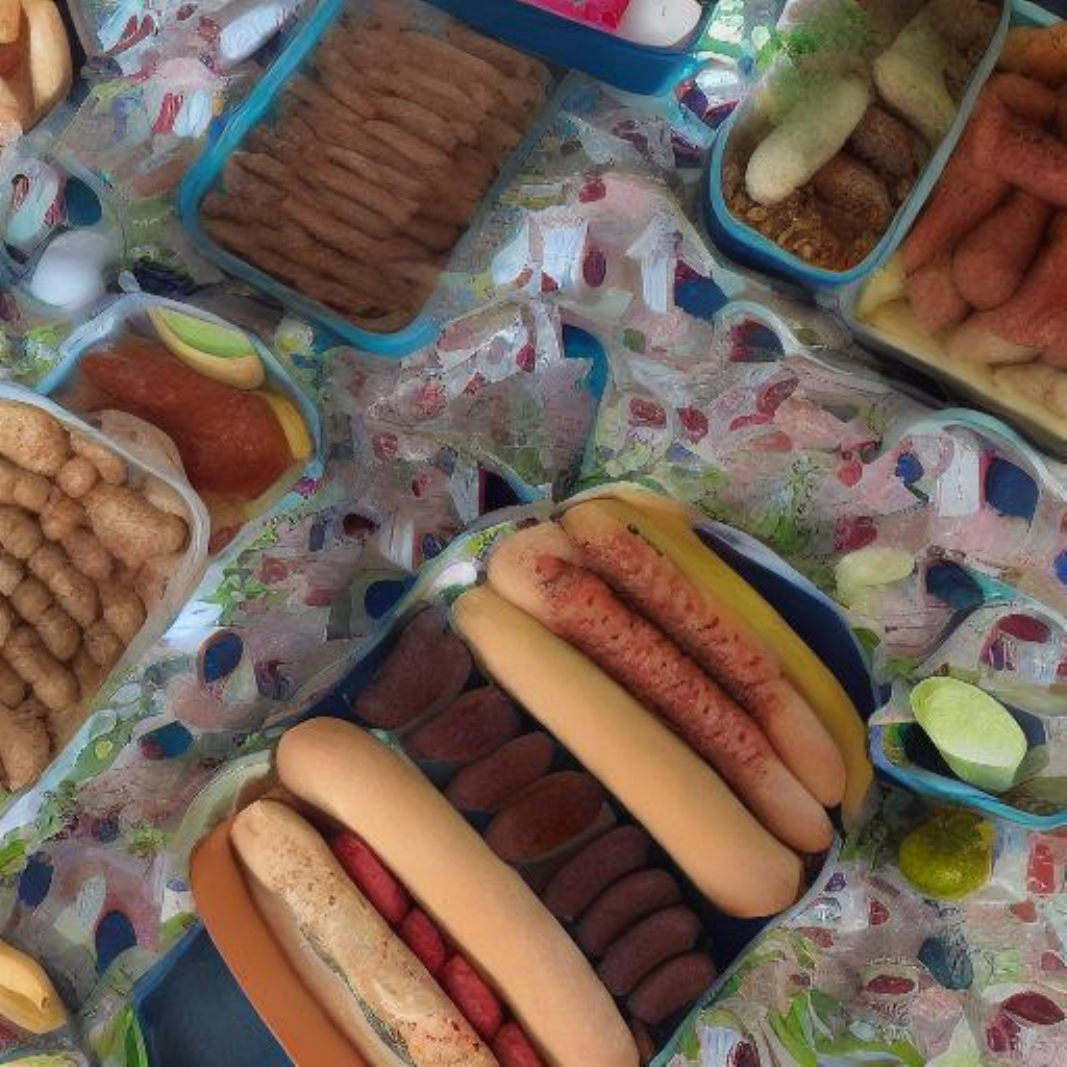} &
\includegraphics[width=0.132\textwidth]{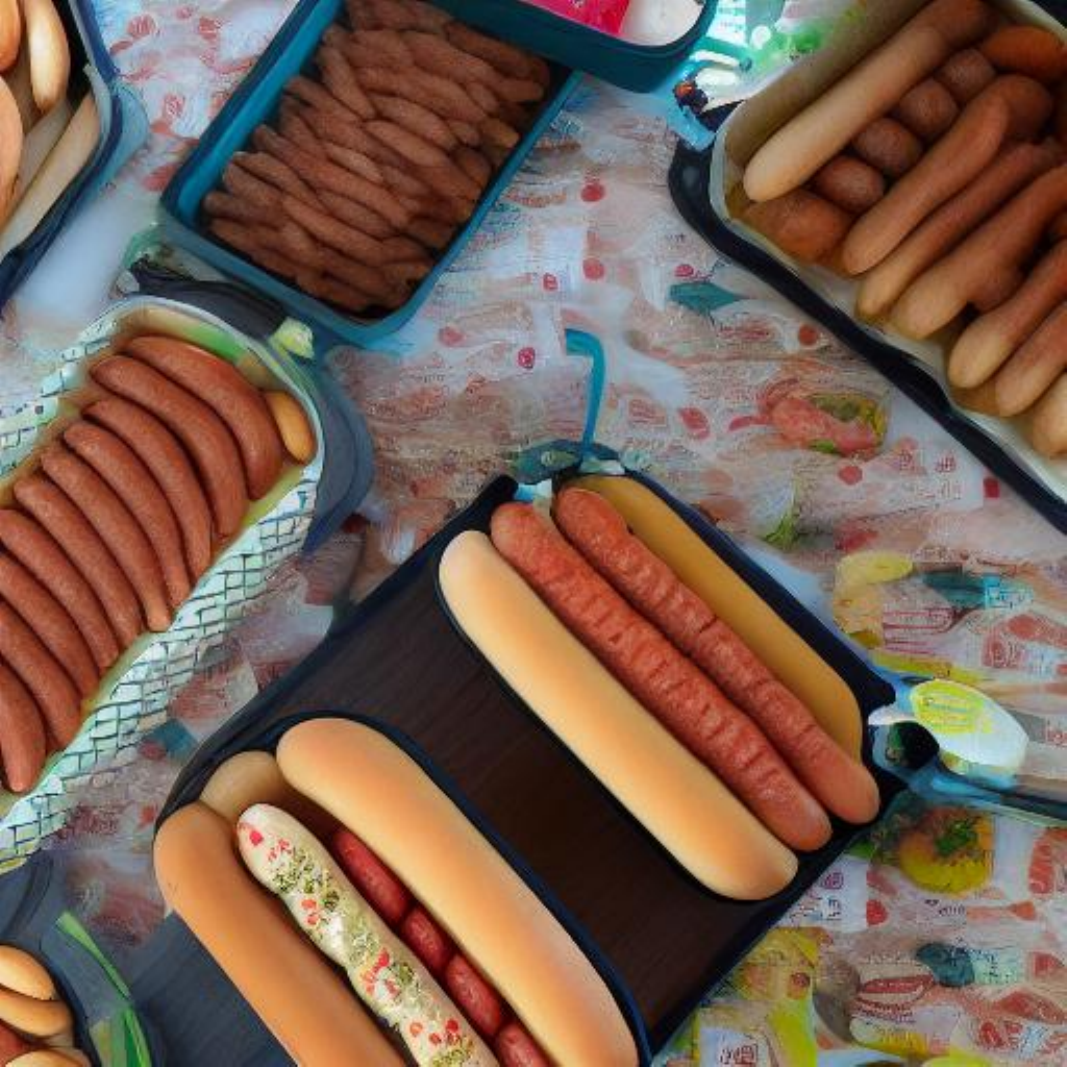} &
\includegraphics[width=0.132\textwidth]{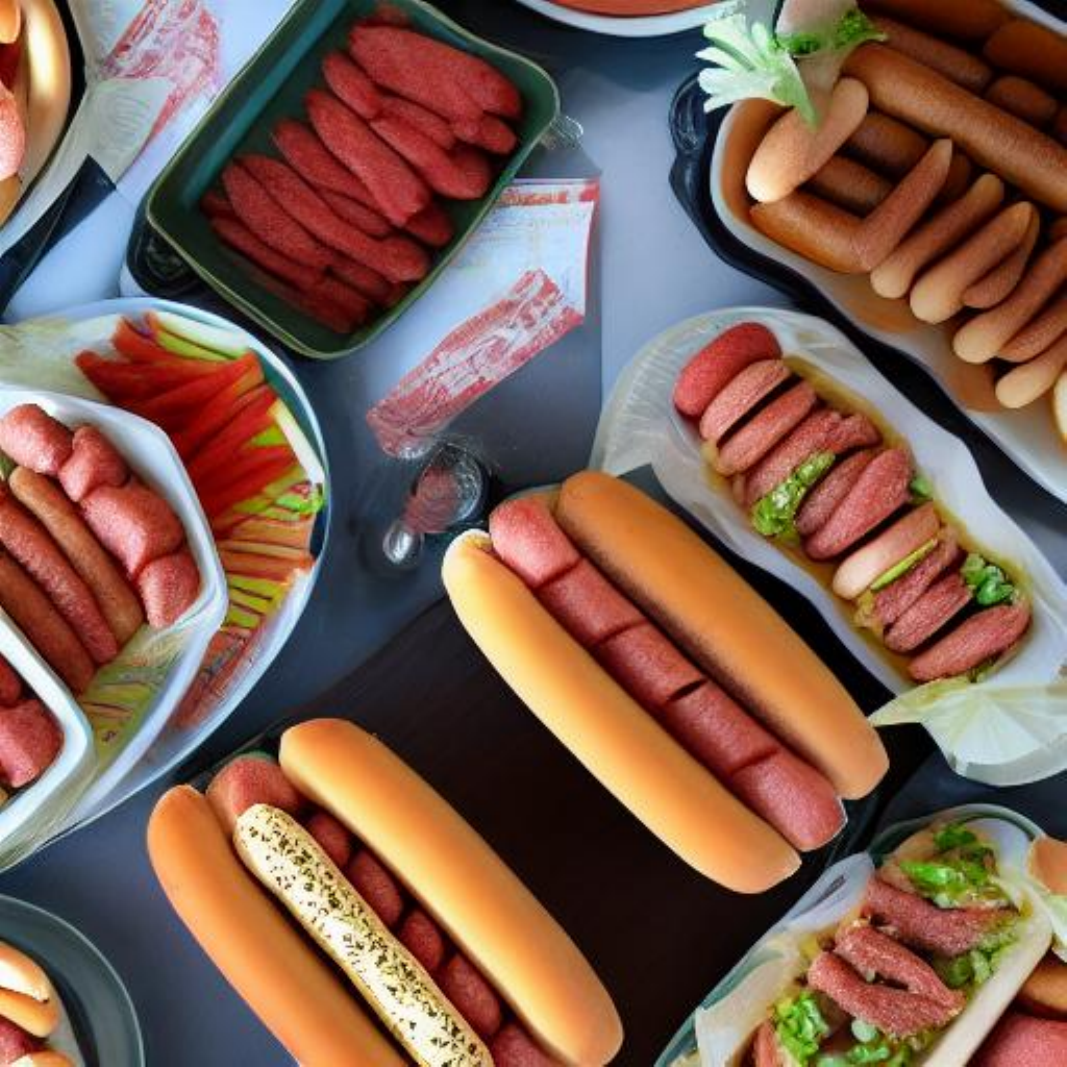} &
\includegraphics[width=0.132\textwidth]{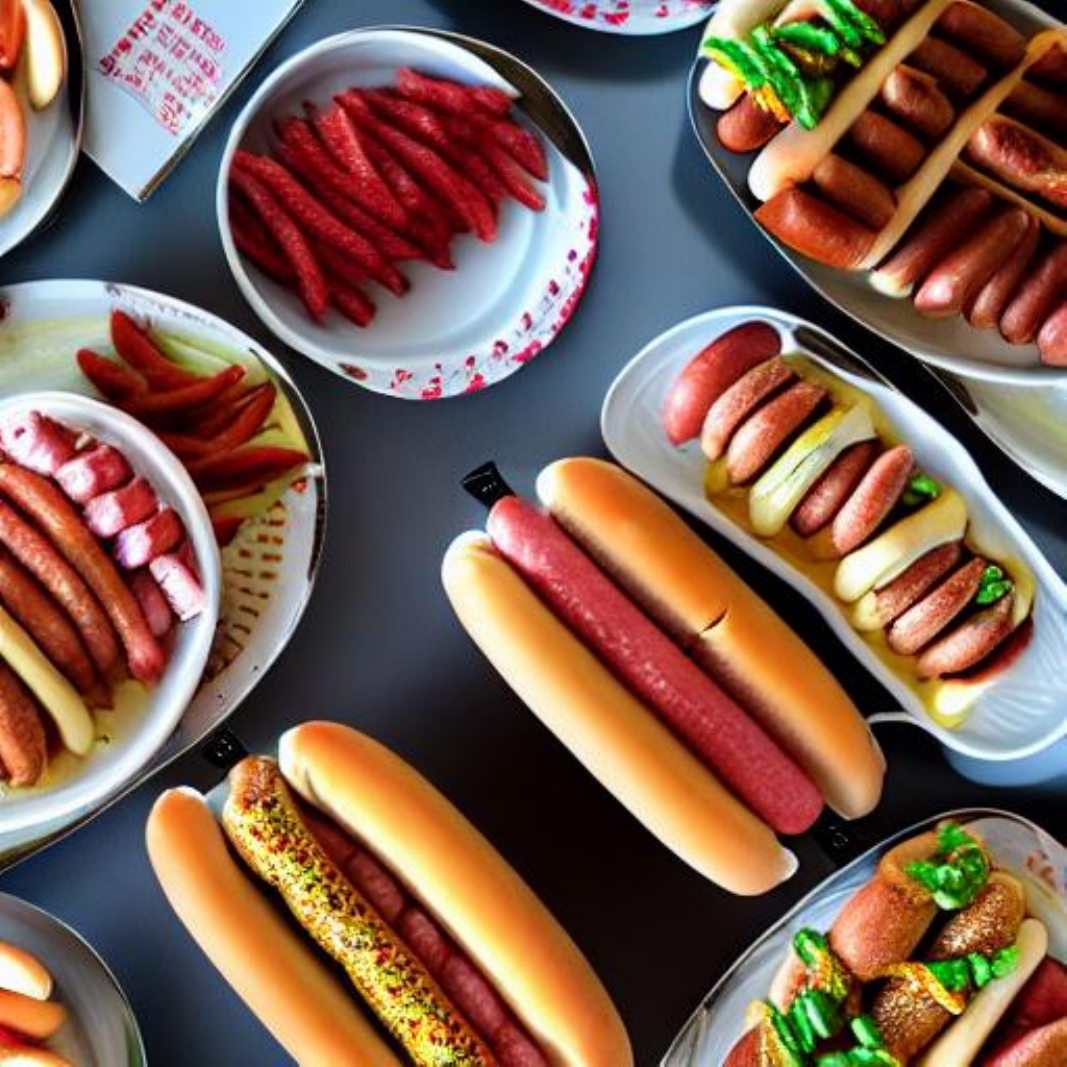} &
\includegraphics[width=0.132\textwidth]{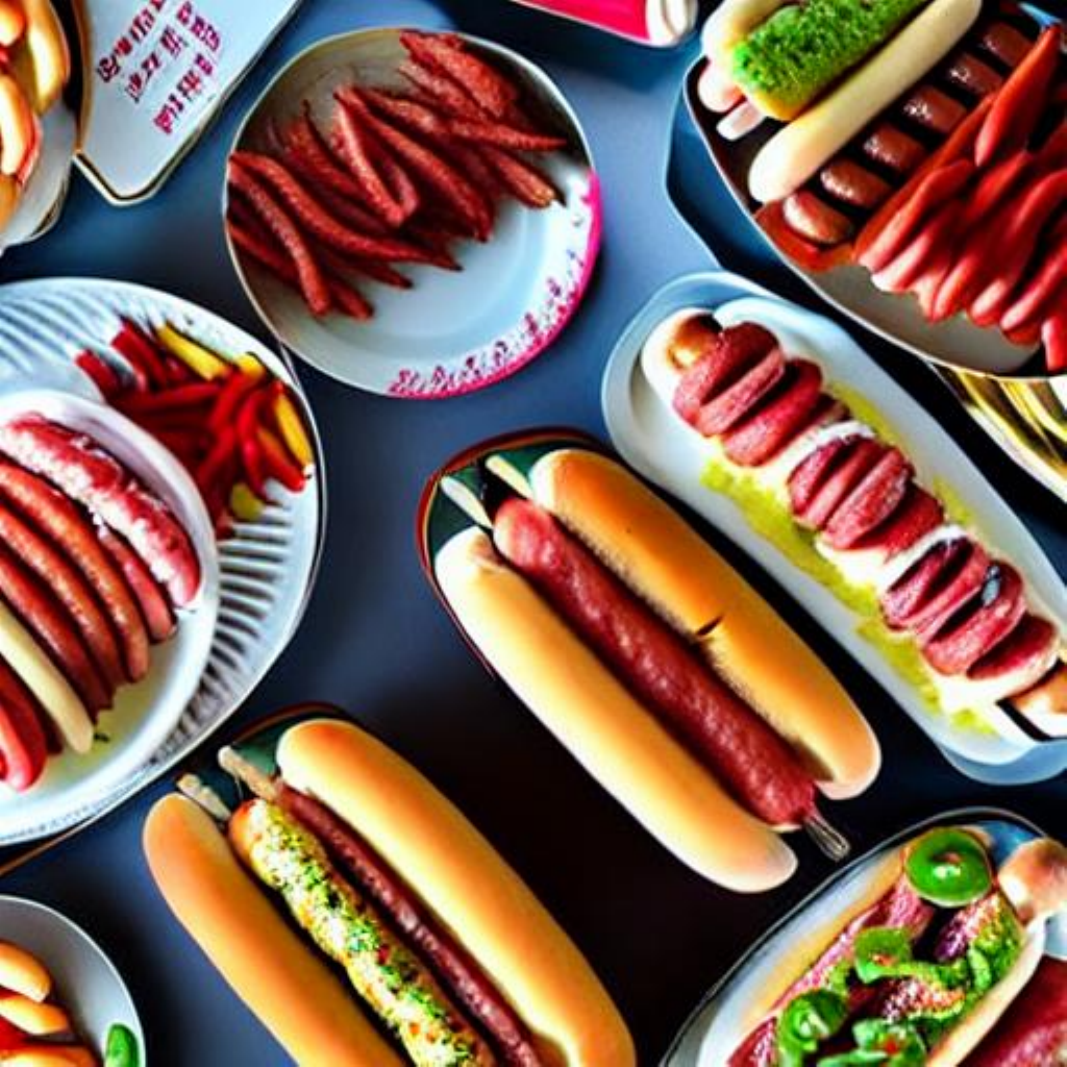} &
\includegraphics[width=0.132\textwidth]{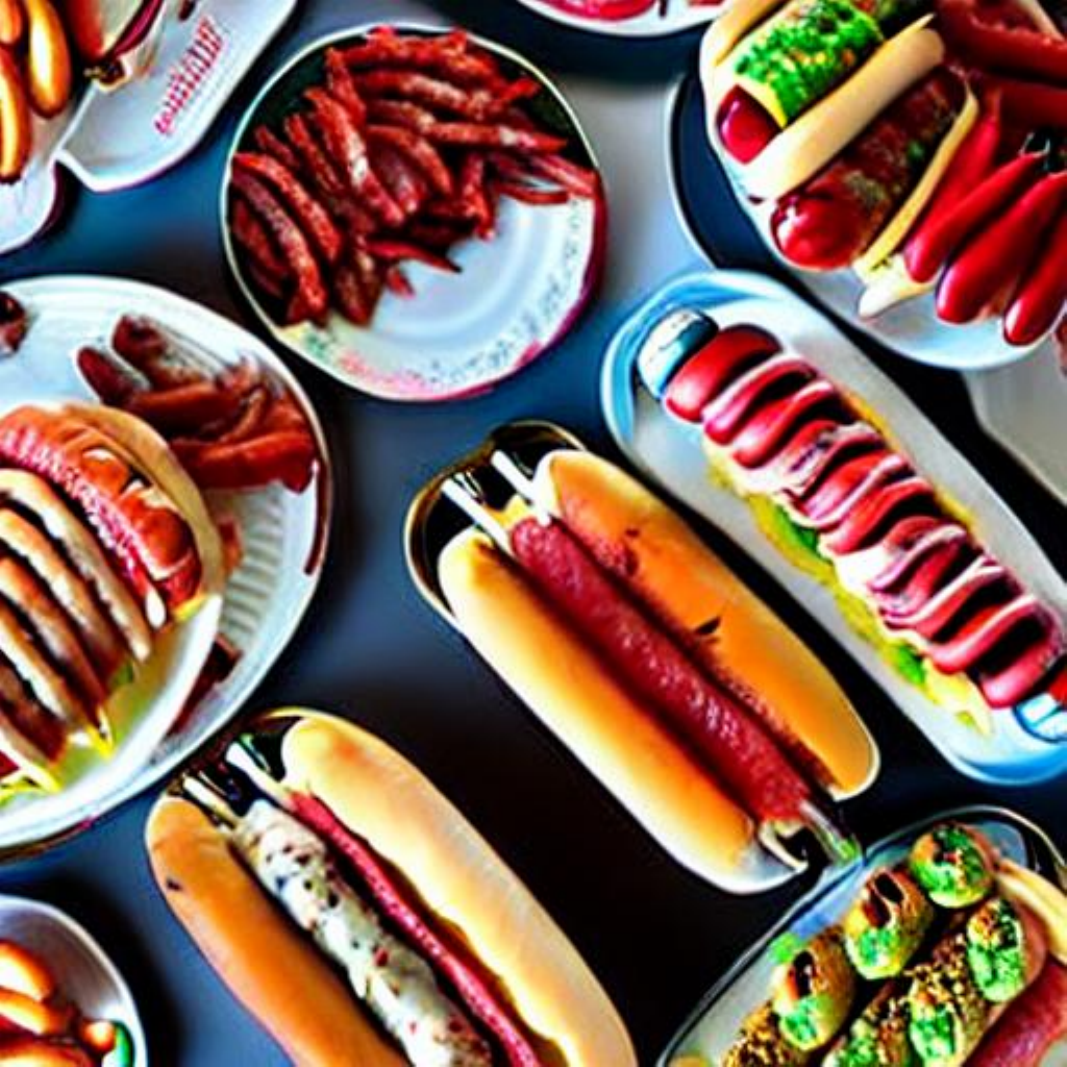} \\
};
\node[rotate=90, anchor=center, yshift=8pt] at ($(c4top.west)!0.5!(c4bot.west)$) {DG-CFG};

\node[above=2pt] at (p1-1.north) {$\bar{\omega}=3$};
\node[above=2pt] at (p1-2.north) {$\bar{\omega}=7$};
\node[above=2pt] at (p1-3.north) {$\bar{\omega}=11$};
\node[above=2pt] at (p1-4.north) {$\bar{\omega}=15$};
\node[above=2pt] at (p1-5.north) {$\bar{\omega}=19$};
\node[above=2pt] at (p1-6.north) {$\bar{\omega}=23$};

\end{tikzpicture}
\caption{Qualitative comparison on Stable Diffusion~1.5 (DDIM, 8 NFE). Columns vary the nominal guidance strength $\bar{\omega}\in\{3,7,11,15,19,23\}$; rows are grouped by method, with two prompts per method. The upper prompt is \textit{a large green truck on a city street}, and the lower prompt is \textit{a table topped with trays full of hot dogs}.}
\label{fig:sd15-grid}
\end{figure}

\newcommand{\sdqualrow}[3]{%
\rotatebox[origin=c]{90}{#1} &
\begin{tabular}{@{}*{6}{c}@{}}
\includegraphics[width=0.12\textwidth]{fig/grid_#2_#3_w3_p03.pdf} &
\includegraphics[width=0.12\textwidth]{fig/grid_#2_#3_w7_p03.pdf} &
\includegraphics[width=0.12\textwidth]{fig/grid_#2_#3_w11_p03.pdf} &
\includegraphics[width=0.12\textwidth]{fig/grid_#2_#3_w15_p03.pdf} &
\includegraphics[width=0.12\textwidth]{fig/grid_#2_#3_w19_p03.pdf} &
\includegraphics[width=0.12\textwidth]{fig/grid_#2_#3_w23_p03.pdf} \\
\includegraphics[width=0.12\textwidth]{fig/grid_#2_#3_w3_p05.pdf} &
\includegraphics[width=0.12\textwidth]{fig/grid_#2_#3_w7_p05.pdf} &
\includegraphics[width=0.12\textwidth]{fig/grid_#2_#3_w11_p05.pdf} &
\includegraphics[width=0.12\textwidth]{fig/grid_#2_#3_w15_p05.pdf} &
\includegraphics[width=0.12\textwidth]{fig/grid_#2_#3_w19_p05.pdf} &
\includegraphics[width=0.12\textwidth]{fig/grid_#2_#3_w23_p05.pdf}
\end{tabular} \\[3pt]
}

\newcommand{\sdqualgrid}[1]{%
\begingroup
\setlength{\tabcolsep}{1pt}
\begin{tabular}{cc}
& \begin{tabular}{@{}*{6}{c}@{}}
\makebox[0.12\textwidth][c]{$\bar{\omega}=3$} &
\makebox[0.12\textwidth][c]{$\bar{\omega}=7$} &
\makebox[0.12\textwidth][c]{$\bar{\omega}=11$} &
\makebox[0.12\textwidth][c]{$\bar{\omega}=15$} &
\makebox[0.12\textwidth][c]{$\bar{\omega}=19$} &
\makebox[0.12\textwidth][c]{$\bar{\omega}=23$}
\end{tabular} \\[2pt]
\sdqualrow{Constant CFG}{#1}{constant}
\sdqualrow{Interval CFG}{#1}{step_threshold}
\sdqualrow{$\beta$--CFG}{#1}{time_parabolic}
\sdqualrow{DG-CFG}{#1}{one_minus_alpha_sqrt_alpha_divide_beta}
\end{tabular}
\endgroup
}

\begin{figure}[p]
\centering
\small
\sdqualgrid{sd21}
\caption{Qualitative comparison on Stable Diffusion~2.1 (DDIM, 8 NFE). Columns vary the nominal guidance strength $\bar{\omega}\in\{3,7,11,15,19,23\}$; rows are grouped by method, with two prompts per method. The prompts and random seeds are identical to Figure~\ref{fig:sd15-grid}: the upper prompt is \textit{a large green truck on a city street}, and the lower prompt is \textit{a table topped with trays full of hot dogs}.}
\label{fig:sd21-grid}
\end{figure}

\begin{figure}[p]
\centering
\small
\sdqualgrid{sdxl}
\caption{Qualitative comparison on Stable Diffusion~XL (DDIM, 8 NFE). Columns vary the nominal guidance strength $\bar{\omega}\in\{3,7,11,15,19,23\}$; rows are grouped by method, with two prompts per method. The prompts and random seeds are identical to Figure~\ref{fig:sd15-grid}: the upper prompt is \textit{a large green truck on a city street}, and the lower prompt is \textit{a table topped with trays full of hot dogs}.}
\label{fig:sdxl-grid}
\end{figure}

\subsubsection{Ablation Study}
\label{sec:ablation}

To isolate the contribution of each design component, we compare four progressively refined schedules at $8$ NFE. We perform complete ablations on SD1.5 and SD2.1 over $\bar{\omega}\in\{3,7,11,15,19,23\}$, allowing us to assess whether the component-wise effects remain consistent across model generations. Each variant is normalized independently so that its integrated guidance deviation matches that of constant CFG.

\begin{itemize}[leftmargin=*]
    \item \textbf{Constant CFG:} $\omega(t) \equiv \bar{\omega}$ (baseline).
    \item \textbf{Medium 1:}
    $\omega(t) = 1 + C_1 \cdot (\bar{\omega}-1) / \beta(t)$ (time-balance only).
    \item \textbf{Medium 2:}
    $\omega(t) = 1 + C_2 \cdot (\bar{\omega}-1) \cdot \sqrt{\bar{\alpha}_t} / \beta(t)$
    ($+\,$signal-content weighting).
    \item \textbf{DG-CFG (ours):}
    $\omega(t) = 1 + C_3 \cdot (\bar{\omega}-1) \cdot (1-\bar{\alpha}_t)\sqrt{\bar{\alpha}_t} / \beta(t)$ (all three factors).
\end{itemize}

Tables~\ref{tab:ablation} and~\ref{tab:ablation-sd21} report the complete SD1.5 and SD2.1 results, respectively. We analyze the components in the same order as the schedule construction in Section~\ref{sec:schedule-design}, focusing on how each factor changes diversity, fidelity, perceptual quality, and saturation as guidance becomes stronger.

\emph{Step 1: time-balance (Medium 1 vs.\ constant).}
Normalizing by $\beta(t)$ improves the behavior of the sampler across the evaluated guidance ranges on both backbones. Medium 1 generally preserves more sample variation than constant CFG in the quality-comparable regime, while also reducing the saturation drift that becomes increasingly visible under stronger guidance. This supports the diagnosis in Section~\ref{sec:schedule-design}: without time balancing, the high-noise part of the trajectory can dominate the path-integral correction, leading to repeated coarse layouts and color over-saturation. Reweighting by the noise rate weakens this dominance, yielding more valid variation and better control of saturation-driven degradation.

\emph{Step 2: Signal-content weighting (Medium 2 vs.\ Medium 1).}
Adding $\sqrt{\bar{\alpha}_t}$ further shifts guidance toward timesteps with stronger signal content. This increases raw diversity scores and suppresses saturation more aggressively, but the improvement is not quality preserving. Across both backbones, Medium 2 tends to lose prompt alignment, perceptual quality, and distributional fidelity relative to Medium 1. This is consistent with the analysis in Section~\ref{sec:schedule-design}: emphasizing low-noise timesteps exposes the sampler to score-approximation errors amplified by $1/\sigma_t^2$. Signal-content weighting therefore changes the sample spread, but it is incomplete without a mechanism that protects image quality.

\emph{Step 3: Score-error mitigation (DG-CFG vs.\ Medium 2).}
Adding the $(1-\bar{\alpha}_t)$ factor mitigates the low-noise score-error amplification introduced by the previous step. DG-CFG restores image quality relative to Medium 2 while retaining much of the saturation control gained from the earlier reweighting. It also remains stable as guidance becomes stronger, providing greater meaningful variation in the high-guidance regime while preserving perceptual quality and distributional fidelity.

The qualitative results for SD1.5 and SD2.1 in Figures~\ref{fig:ablation} and~\ref{fig:ablation-sd21}, respectively, reinforce these findings. Constant CFG produces pronounced color over-saturation and limited fine-grained detail. Medium 1 improves color balance, fine-detail quality, and sample diversity. Medium 2 further amplifies these gains, but at the cost of unnatural textures and structural artifacts. DG-CFG achieves the most favorable balance, preserving meaningful variation, natural colors, and fine detail while producing more realistic textures and coherent structures.

\begin{table}[H]
\centering
\caption{Component-wise ablation of the DG-CFG schedule on Stable Diffusion~1.5 at 8 NFE. Medium 1 adds time balancing, Medium 2 adds signal-content weighting, and DG-CFG adds score-error mitigation. Metrics are computed over 1K MS-COCO validation prompts, except Div (Vendi score), which is averaged over 100 prompts with 16 samples per prompt. The real-image reference values are CLIP=30.57, IR=0.38, and Sat=0.33.}
\label{tab:ablation}
\small
\begin{tabular}{llcccccc}
\toprule
Metric & Method & $\bar{\omega}=3$ & $\bar{\omega}=7$ & $\bar{\omega}=11$ & $\bar{\omega}=15$ & $\bar{\omega}=19$ & $\bar{\omega}=23$ \\
\midrule
\multirow{4}{*}{CLIP$\uparrow$}
& Constant CFG      & \textbf{30.15} & \textbf{31.29} & \textbf{31.38} & 31.29 & 31.01 & 30.64 \\
& Medium 1          & 29.53 & 30.88 & 31.25 & 31.22 & 31.25 & 31.19 \\
& Medium 2          & 28.94 & 30.29 & 30.47 & 30.51 & 30.31 & 30.23 \\
& \textbf{DG-CFG}   & 29.95 & 31.08 & 31.27 & \textbf{31.32} & \textbf{31.29} & \textbf{31.39} \\
\midrule
\multirow{4}{*}{IR$\uparrow$}
& Constant CFG      & \textbf{$-$0.503} & \textbf{$-$0.104} & $-$0.032 & $-$0.061 & $-$0.177 & $-$0.330 \\
& Medium 1          & $-$0.655 & $-$0.203 & $-$0.051 & $-$0.021 & +0.001 & $-$0.000 \\
& Medium 2          & $-$0.823 & $-$0.377 & $-$0.263 & $-$0.232 & $-$0.257 & $-$0.283 \\
& \textbf{DG-CFG}   & $-$0.522 & $-$0.115 & \textbf{$-$0.026} & \textbf{+0.013} & \textbf{+0.024} & \textbf{+0.029} \\
\midrule
\multirow{4}{*}{Sat}
& Constant CFG      & 0.212 & 0.258 & 0.317 & 0.379 & 0.430 & 0.469 \\
& Medium 1          & 0.202 & 0.219 & 0.239 & 0.260 & 0.282 & 0.305 \\
& Medium 2          & 0.199 & 0.209 & 0.220 & 0.234 & 0.247 & 0.259 \\
& \textbf{DG-CFG}   & 0.206 & 0.229 & 0.256 & 0.284 & 0.310 & 0.332 \\
\midrule
\multirow{4}{*}{FID$\downarrow$}
& Constant CFG      & \textbf{74.44} & \textbf{67.67} & 68.73 & 71.57 & 74.94 & 82.34 \\
& Medium 1          & 82.34 & 69.39 & \textbf{68.24} & \textbf{69.19} & \textbf{71.15} & \textbf{73.24} \\
& Medium 2          & 90.77 & 75.87 & 75.56 & 77.64 & 82.97 & 87.65 \\
& \textbf{DG-CFG}   & 77.26 & 67.87 & 68.79 & 70.62 & 72.54 & 74.53 \\
\midrule
\multirow{4}{*}{Div$\uparrow$}
& Constant CFG      & 6.54 & 5.40 & 5.18 & 5.27 & 5.43 & 5.67 \\
& Medium 1          & 7.08 & 5.86 & 5.47 & 5.33 & 5.25 & 5.26 \\
& Medium 2          & \textbf{7.54} & \textbf{6.53} & \textbf{6.10} & \textbf{5.98} & \textbf{5.82} & \textbf{5.80} \\
& \textbf{DG-CFG}   & 6.81 & 5.71 & 5.34 & 5.22 & 5.17 & 5.24 \\
\bottomrule
\end{tabular}
\end{table}

\begin{table}[H]
\centering
\caption{Component-wise ablation of the DG-CFG schedule on Stable Diffusion~2.1 at 8 NFE. Medium 1 adds time balancing, Medium 2 adds signal-content weighting, and DG-CFG adds score-error mitigation. Metrics are computed over 1K MS-COCO validation prompts, except Div (Vendi score), which is averaged over 100 prompts with 16 samples per prompt. The real-image reference values are CLIP=30.57, IR=0.38, and Sat=0.33.}
\label{tab:ablation-sd21}
\small
\begin{tabular}{llcccccc}
\toprule
Metric & Method & $\bar{\omega}=3$ & $\bar{\omega}=7$ & $\bar{\omega}=11$ & $\bar{\omega}=15$ & $\bar{\omega}=19$ & $\bar{\omega}=23$ \\
\midrule
\multirow{4}{*}{CLIP$\uparrow$}
& Constant CFG      & \textbf{30.54} & \textbf{31.48} & \textbf{31.48} & \textbf{31.48} & 31.25 & 30.98 \\
& Medium 1          & 30.03 & 31.14 & 31.33 & 31.35 & 31.34 & 31.28 \\
& Medium 2          & 29.47 & 30.51 & 30.63 & 30.60 & 30.43 & 30.25 \\
& \textbf{DG-CFG}   & 30.36 & 31.23 & 31.36 & 31.38 & \textbf{31.49} & \textbf{31.43} \\
\midrule
\multirow{4}{*}{IR$\uparrow$}
& Constant CFG      & \textbf{$-$0.264} & \textbf{+0.130} & +0.203 & +0.176 & +0.078 & $-$0.092 \\
& Medium 1          & $-$0.457 & +0.022 & +0.151 & +0.198 & +0.210 & +0.172 \\
& Medium 2          & $-$0.610 & $-$0.187 & $-$0.066 & $-$0.052 & $-$0.108 & $-$0.158 \\
& \textbf{DG-CFG}   & $-$0.322 & +0.091 & \textbf{+0.209} & \textbf{+0.238} & \textbf{+0.245} & \textbf{+0.244} \\
\midrule
\multirow{4}{*}{Sat}
& Constant CFG      & 0.207 & 0.248 & 0.299 & 0.351 & 0.396 & 0.431 \\
& Medium 1          & 0.200 & 0.213 & 0.231 & 0.251 & 0.271 & 0.290 \\
& Medium 2          & 0.198 & 0.205 & 0.214 & 0.226 & 0.238 & 0.250 \\
& \textbf{DG-CFG}   & 0.203 & 0.222 & 0.247 & 0.272 & 0.295 & 0.315 \\
\midrule
\multirow{4}{*}{FID$\downarrow$}
& Constant CFG      & \textbf{73.40} & \textbf{67.45} & 70.10 & 72.63 & 74.79 & 79.94 \\
& Medium 1          & 80.42 & 68.73 & \textbf{68.58} & \textbf{70.21} & \textbf{72.70} & \textbf{74.93} \\
& Medium 2          & 86.85 & 74.51 & 74.71 & 78.63 & 84.63 & 91.22 \\
& \textbf{DG-CFG}   & 75.54 & 68.70 & 70.08 & 71.98 & 74.05 & 76.43 \\
\midrule
\multirow{4}{*}{Div$\uparrow$}
& Constant CFG      & 5.99 & 5.01 & 4.81 & 4.87 & 5.08 & 5.33 \\
& Medium 1          & 6.57 & 5.38 & 5.07 & 4.94 & 4.94 & 4.92 \\
& Medium 2          & \textbf{7.01} & \textbf{6.00} & \textbf{5.67} & \textbf{5.59} & \textbf{5.49} & \textbf{5.46} \\
& \textbf{DG-CFG}   & 6.26 & 5.20 & 4.97 & 4.88 & 4.85 & 4.83 \\
\bottomrule
\end{tabular}
\end{table}

\begin{figure}[p]
\centering
\small
\begin{tikzpicture}[
    subimg/.style={
        inner sep=0pt, outer sep=0pt,
        anchor=center
    },
    gridmat/.style={
        matrix of nodes,
        nodes=subimg,
        inner sep=0pt, outer sep=0pt,
        column sep=1pt,
        row sep=1pt
    },
    outer/.style={
        matrix of nodes,
        nodes={inner sep=0pt,outer sep=0pt},
        inner sep=0pt, outer sep=0pt,
        column sep=5pt,
        row sep=5pt
    }
]

\matrix (tl) [gridmat]
{
    \includegraphics[width=0.10\textwidth]{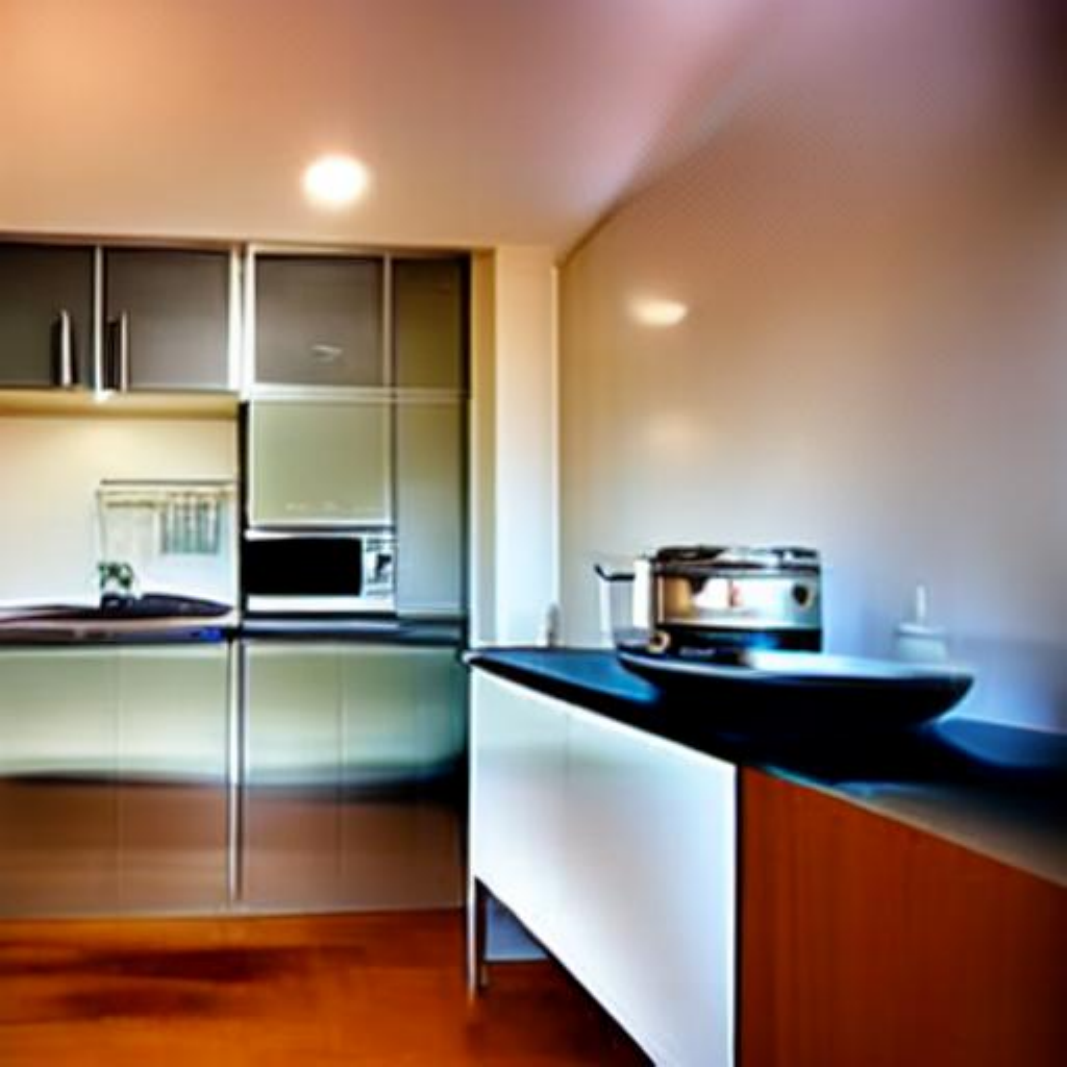} &
    \includegraphics[width=0.10\textwidth]{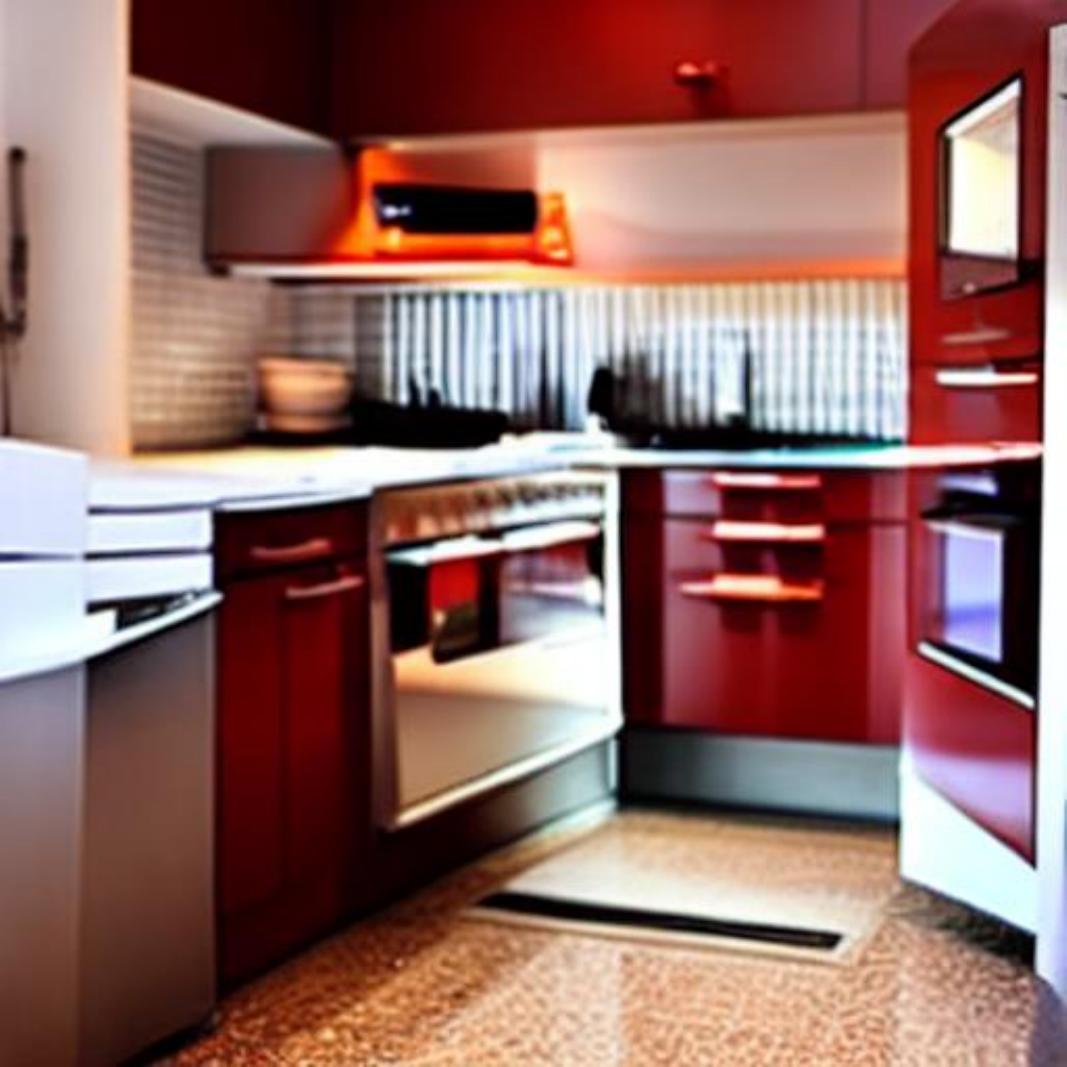} &
    \includegraphics[width=0.10\textwidth]{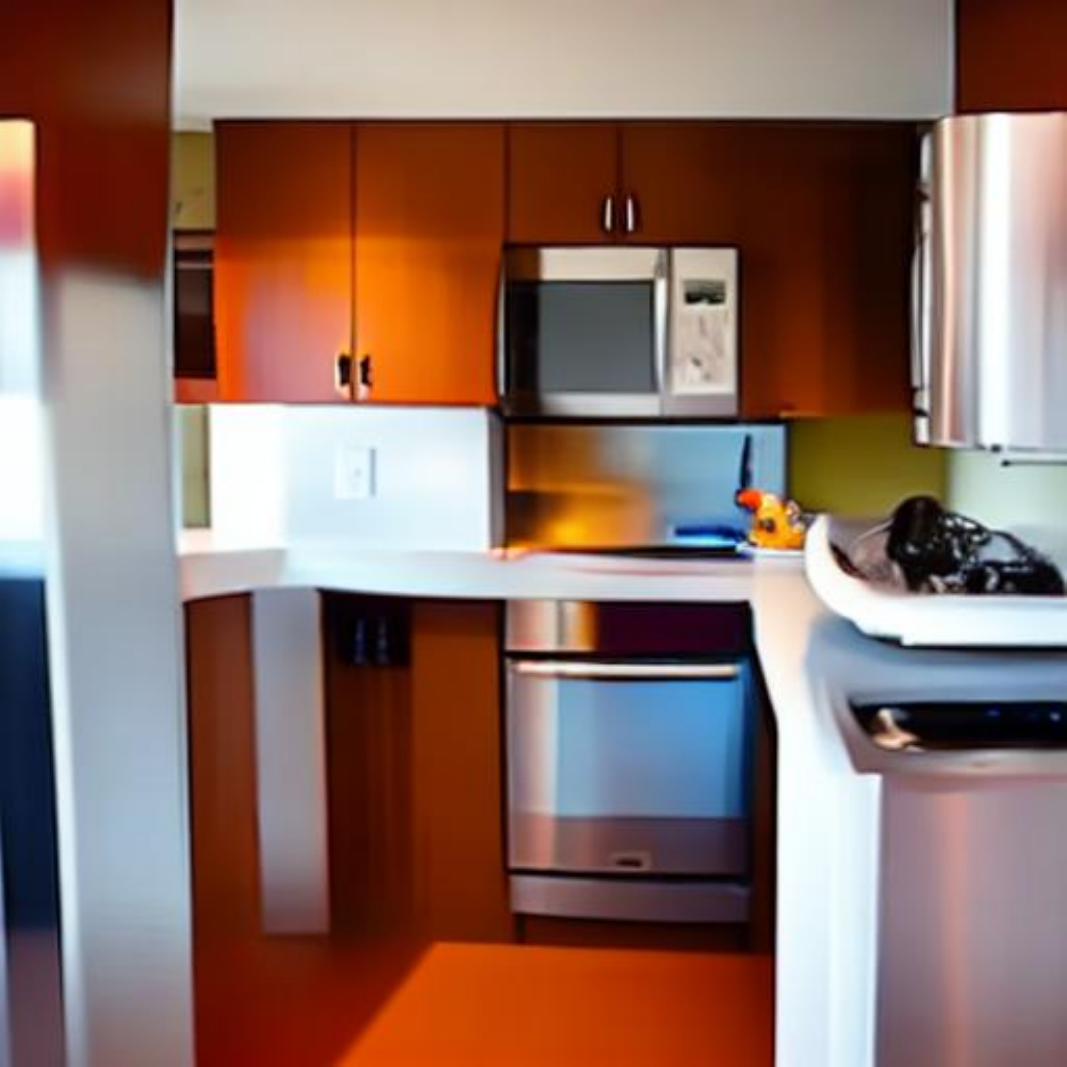} &
    \includegraphics[width=0.10\textwidth]{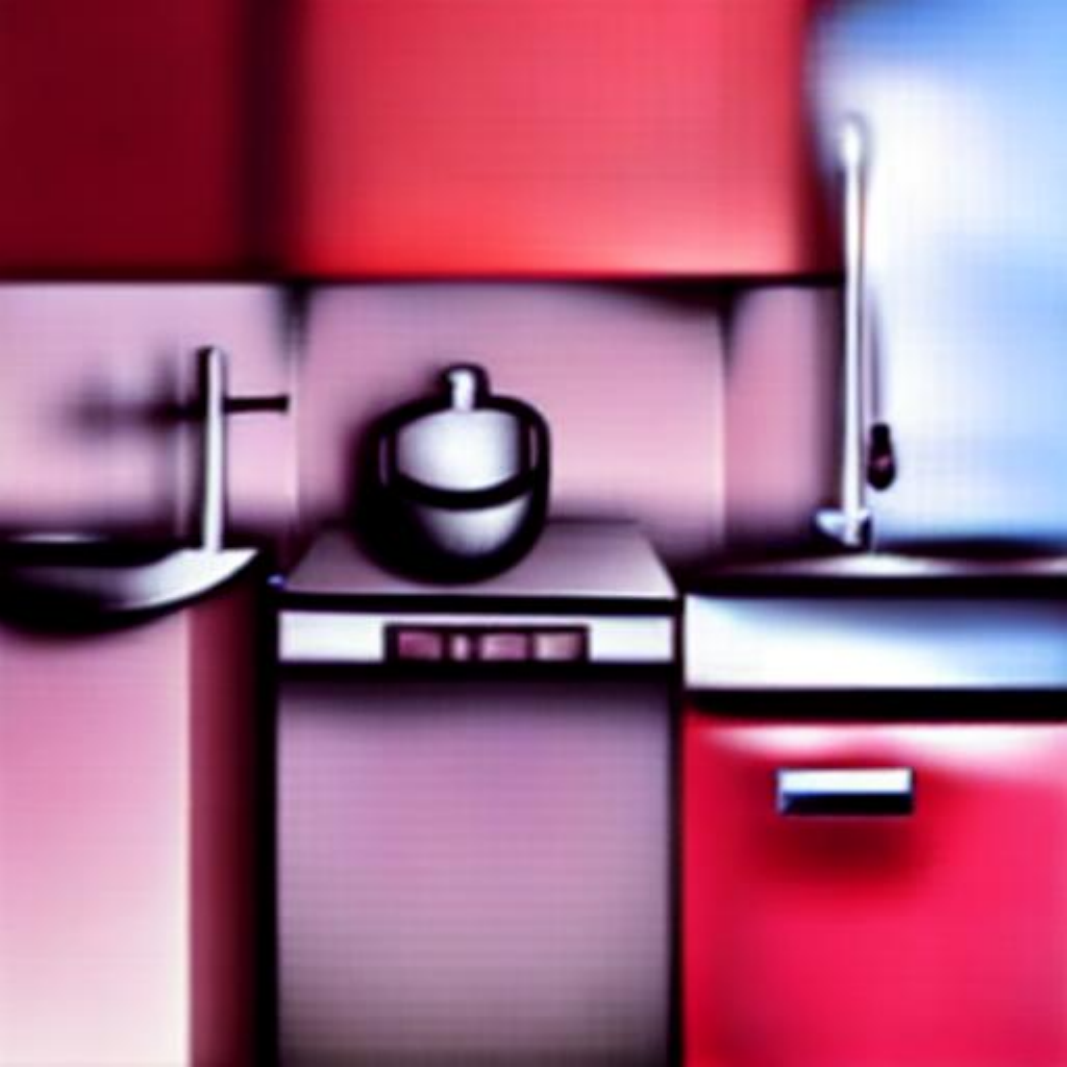} \\
    \includegraphics[width=0.10\textwidth]{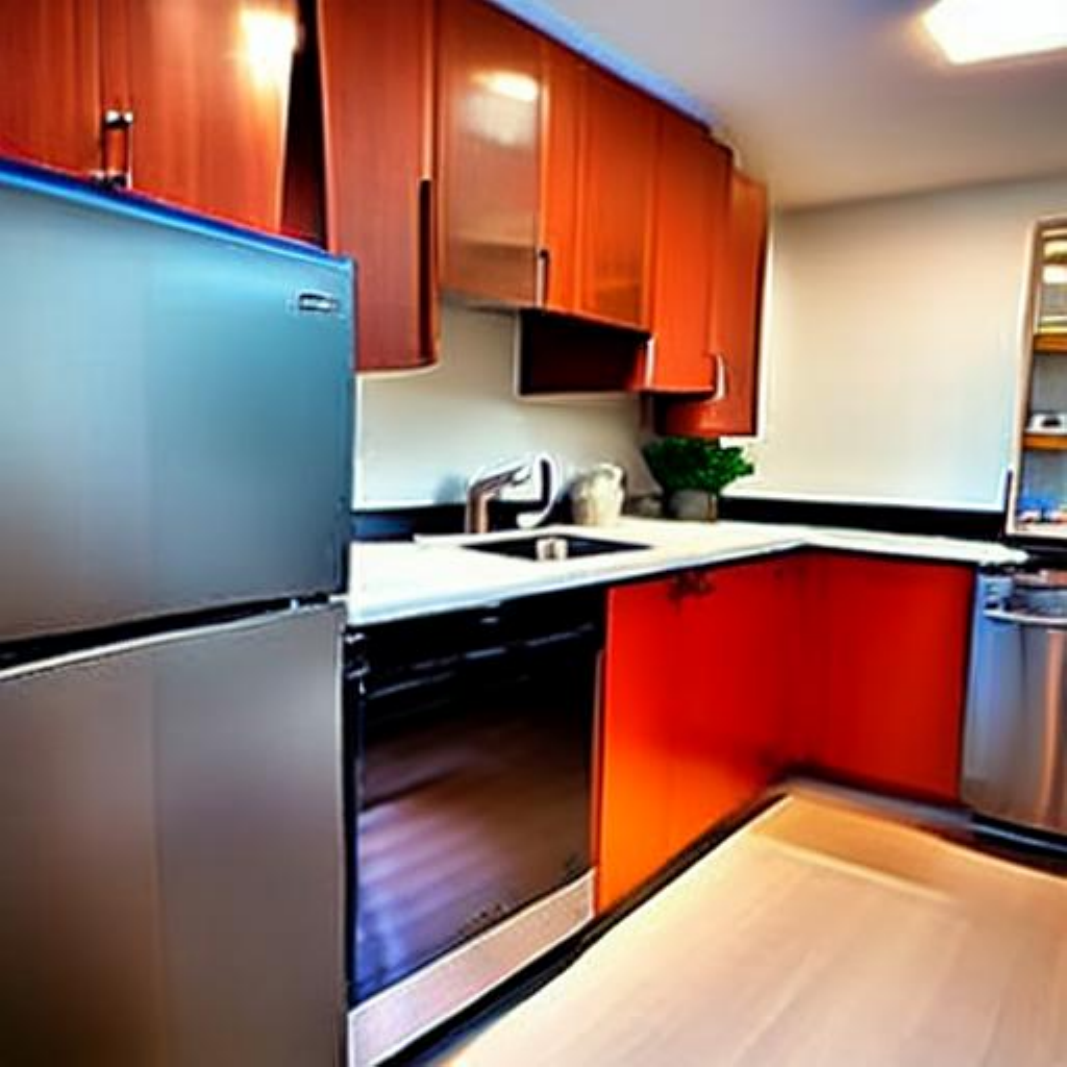} &
    \includegraphics[width=0.10\textwidth]{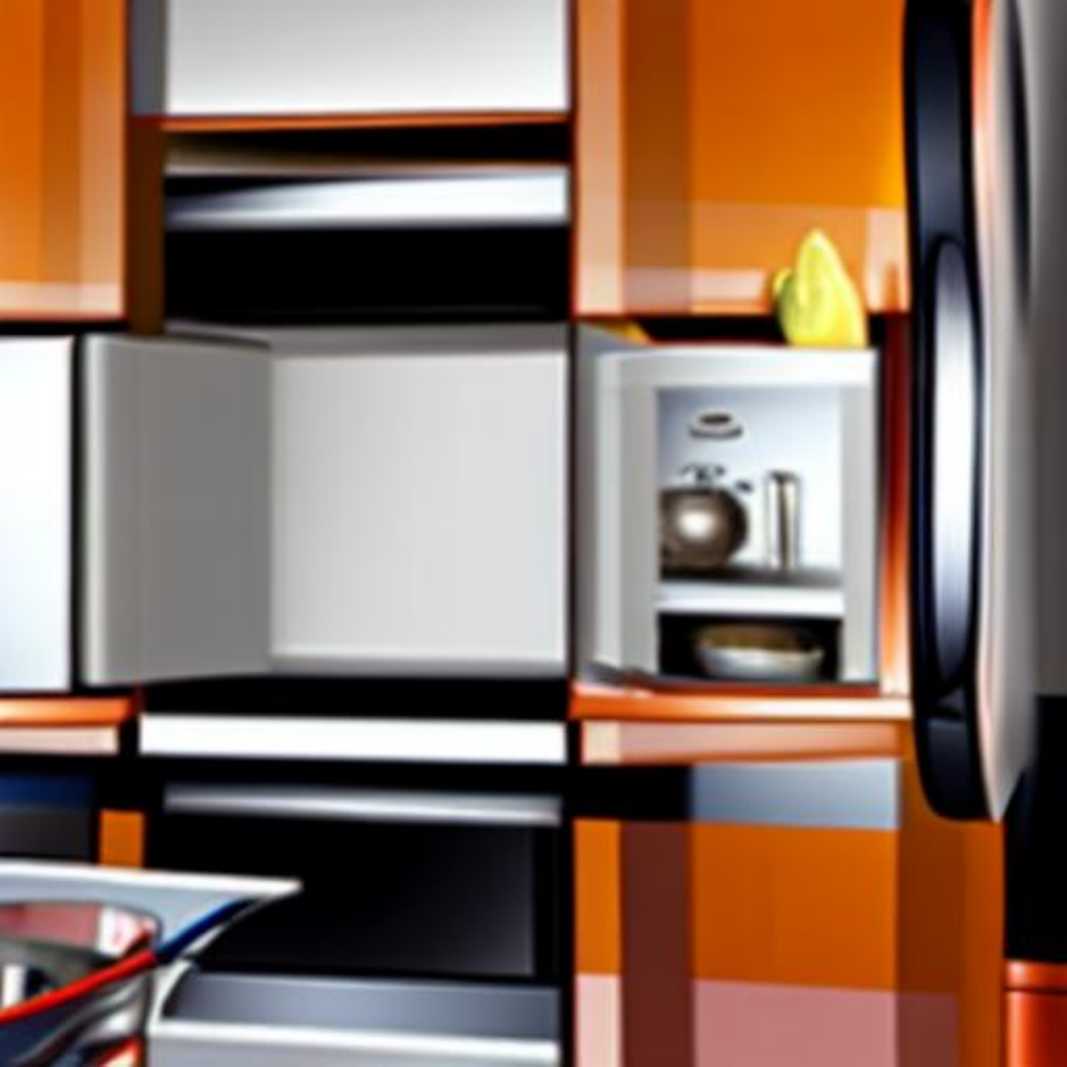} &
    \includegraphics[width=0.10\textwidth]{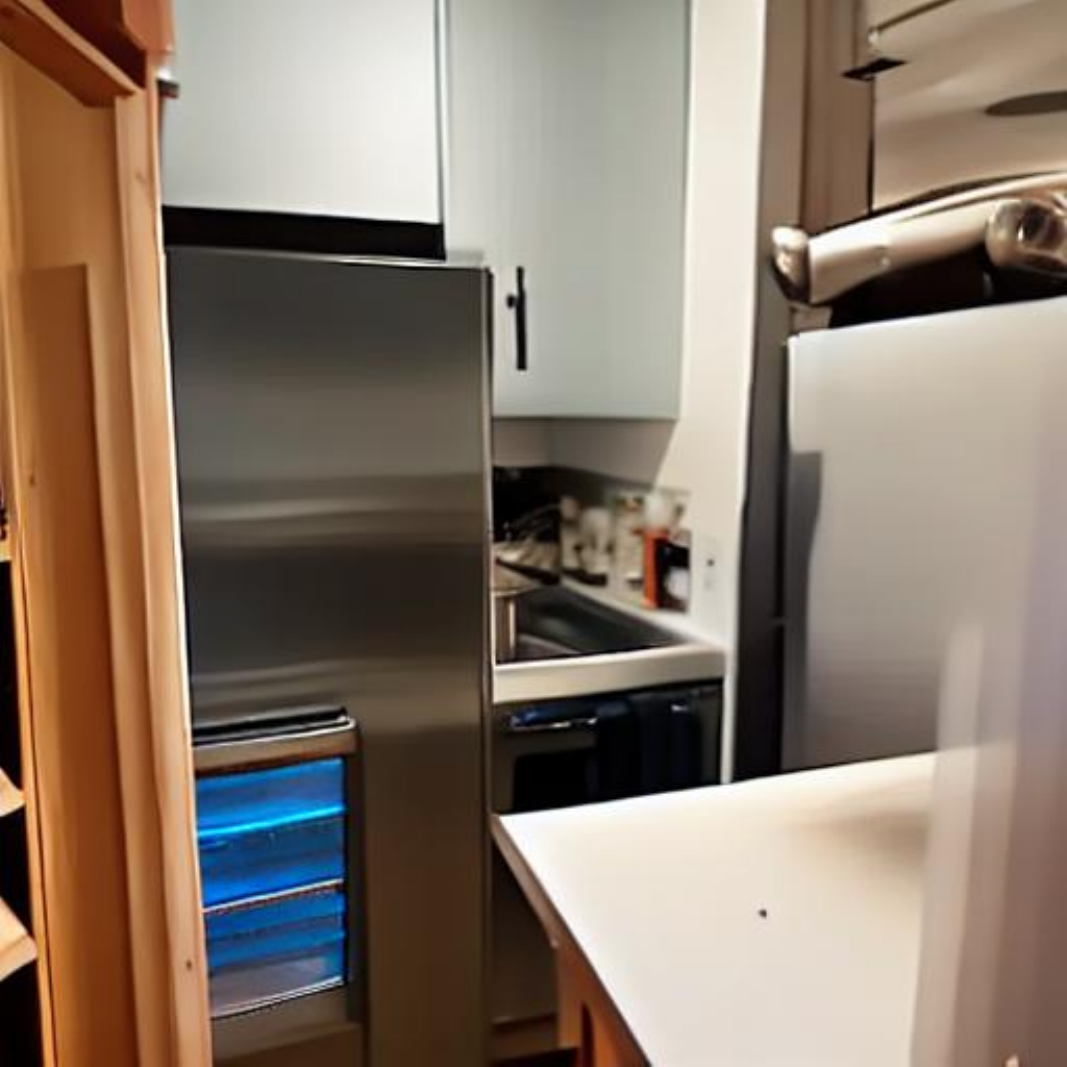} &
    \includegraphics[width=0.10\textwidth]{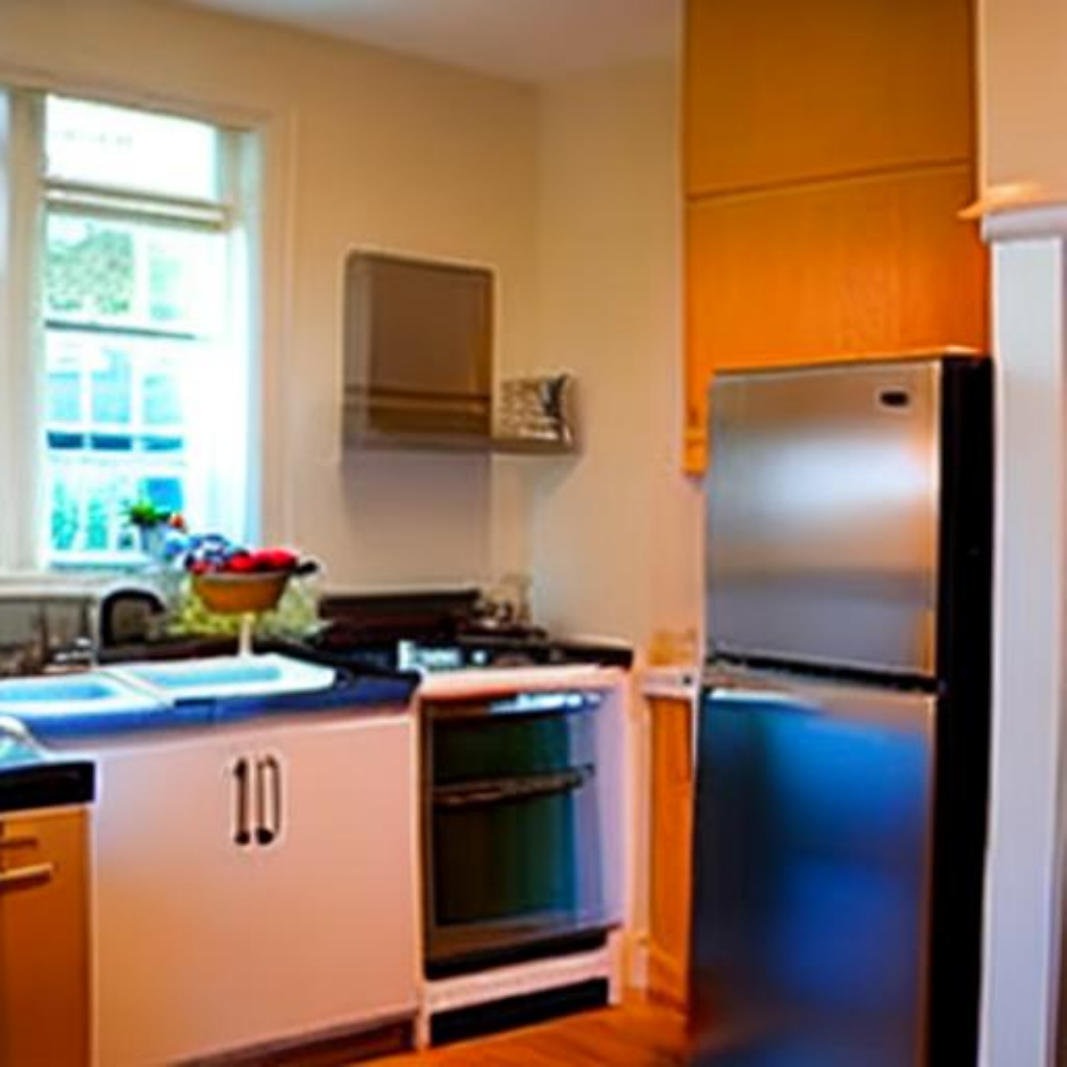} \\
    \includegraphics[width=0.10\textwidth]{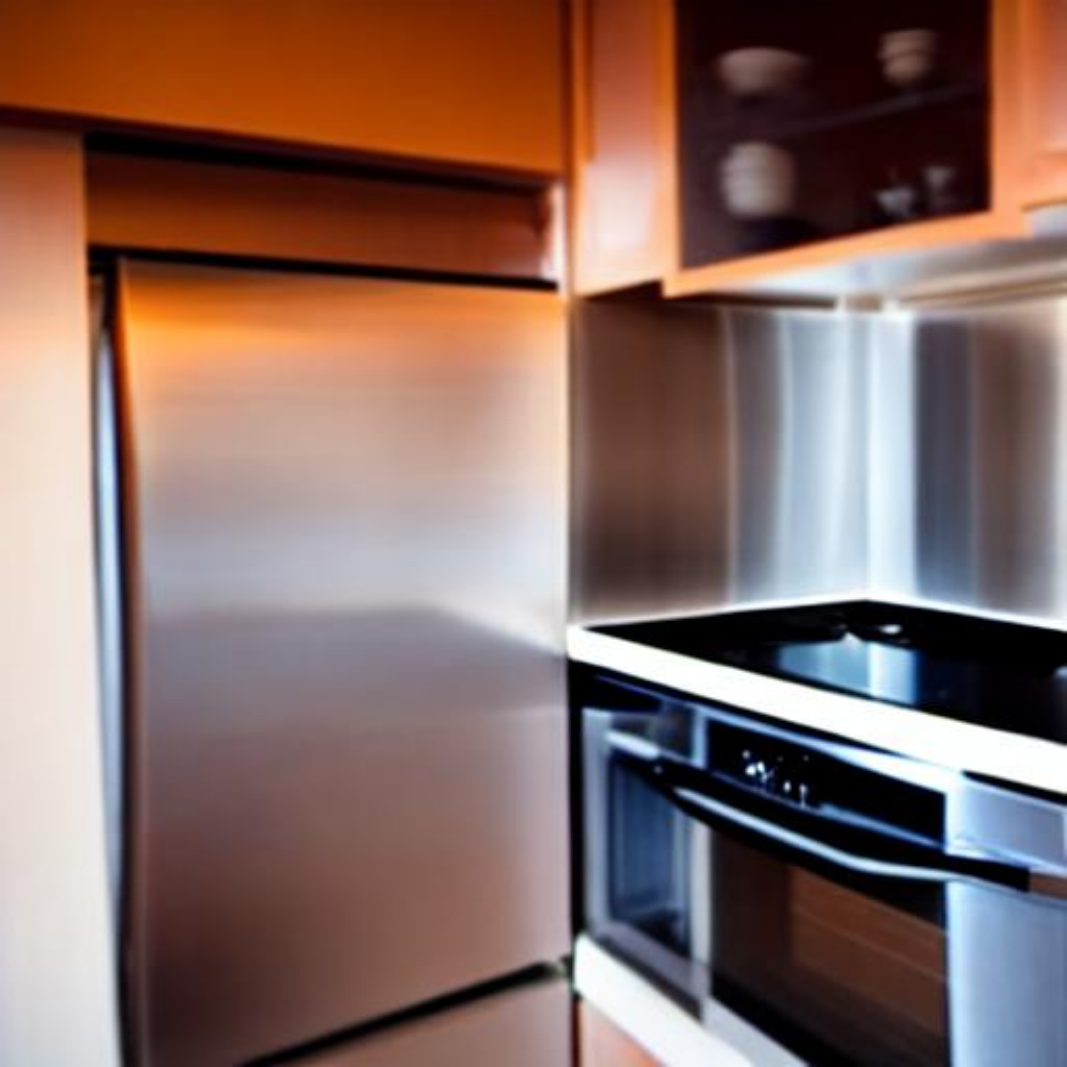} &
    \includegraphics[width=0.10\textwidth]{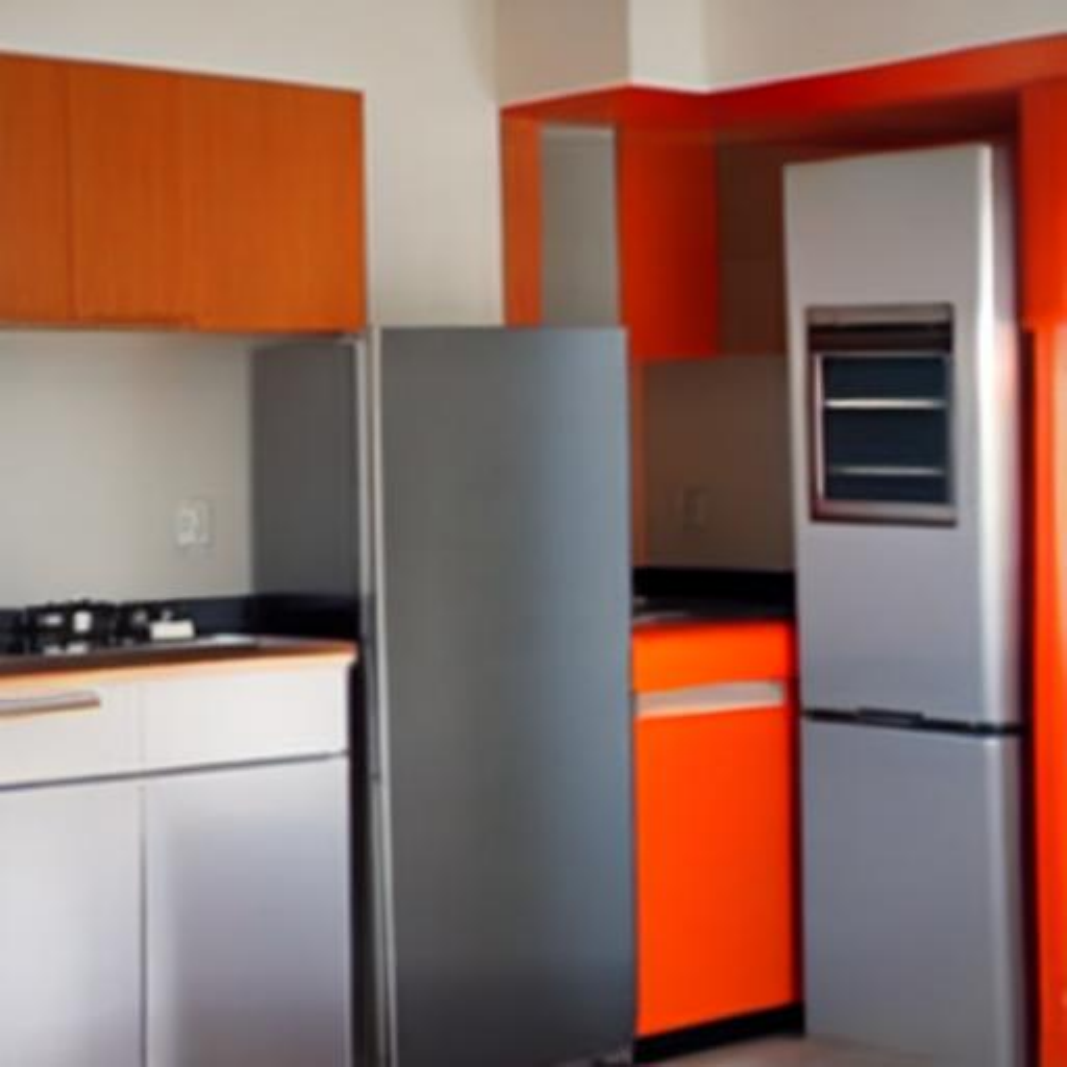} &
    \includegraphics[width=0.10\textwidth]{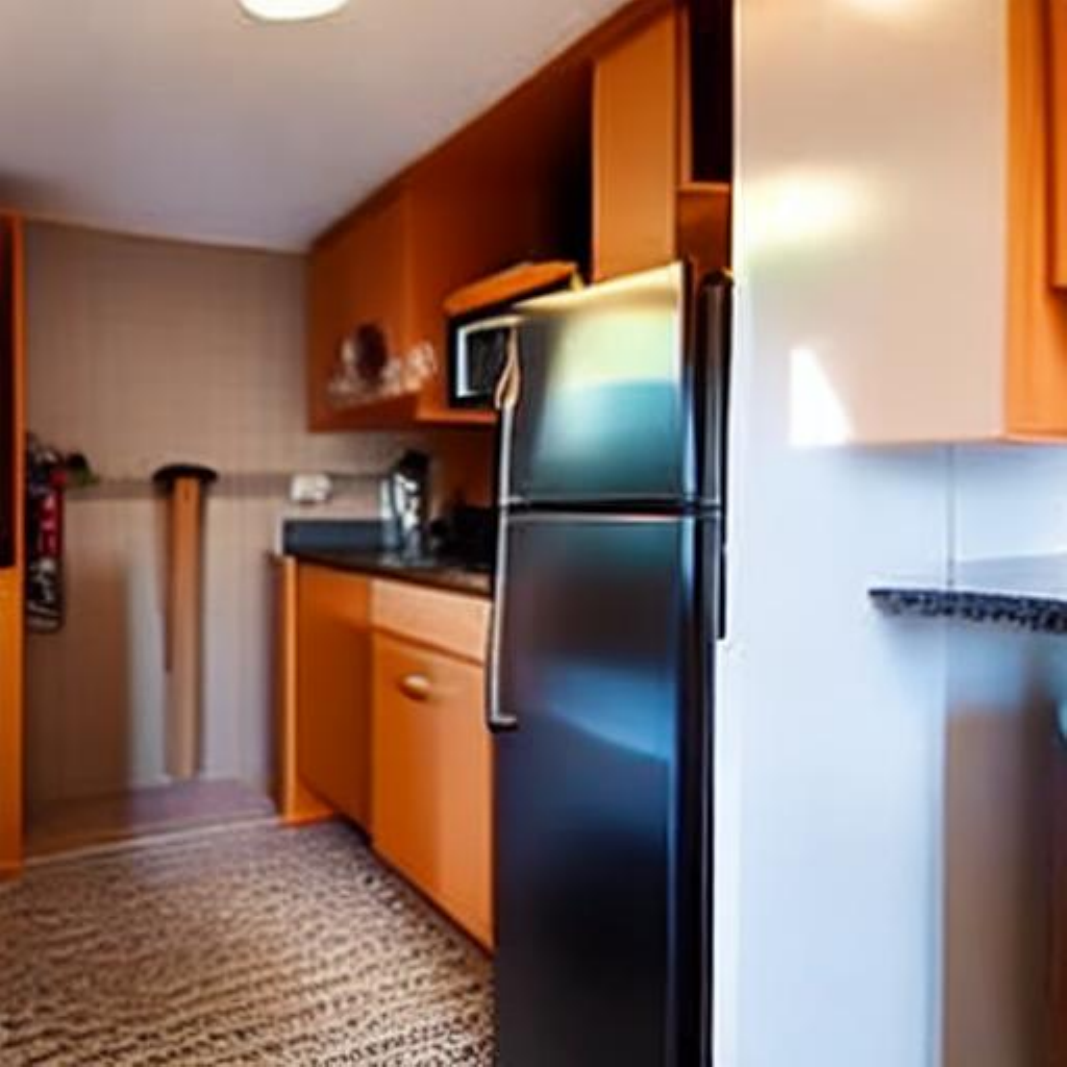} &
    \includegraphics[width=0.10\textwidth]{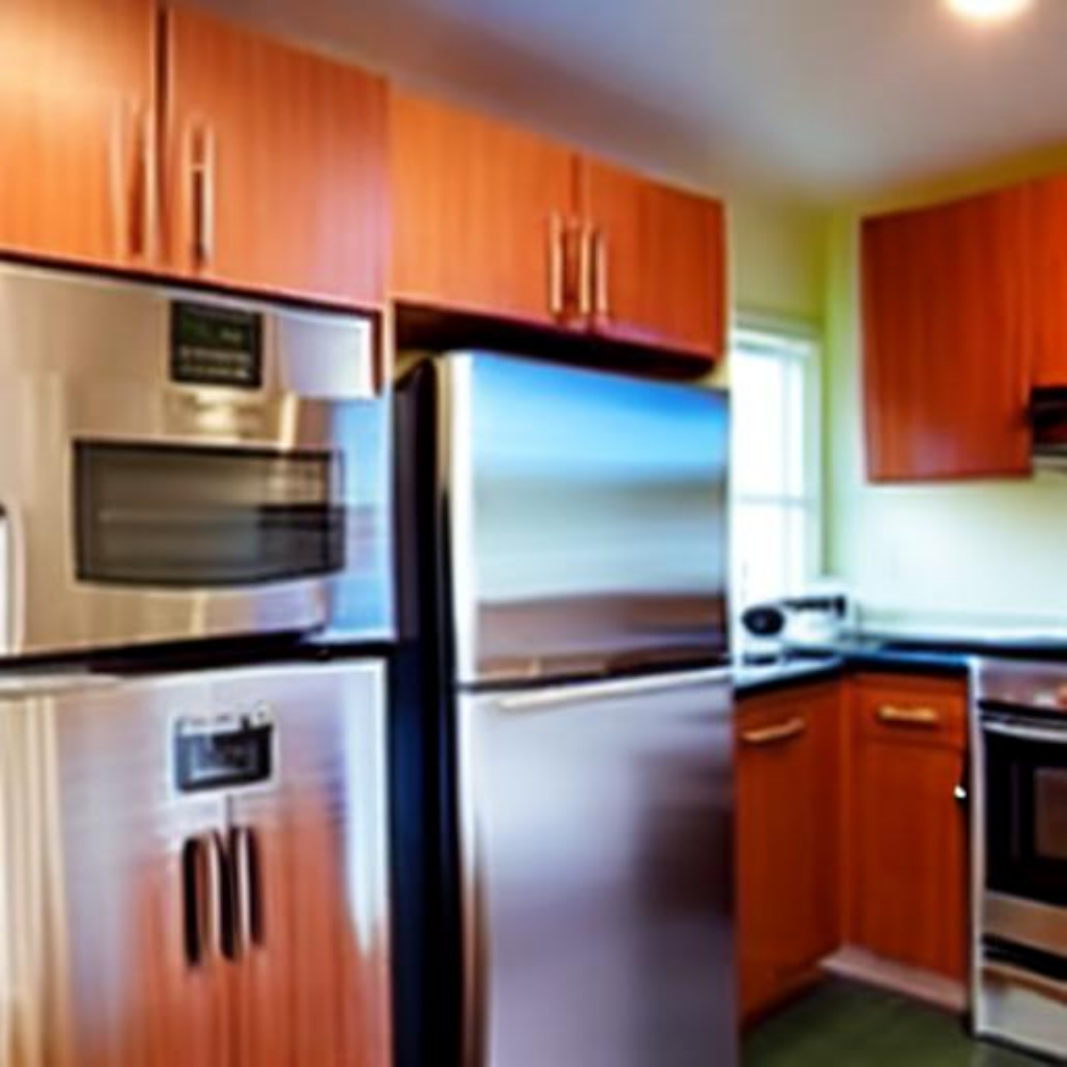} \\
    \includegraphics[width=0.10\textwidth]{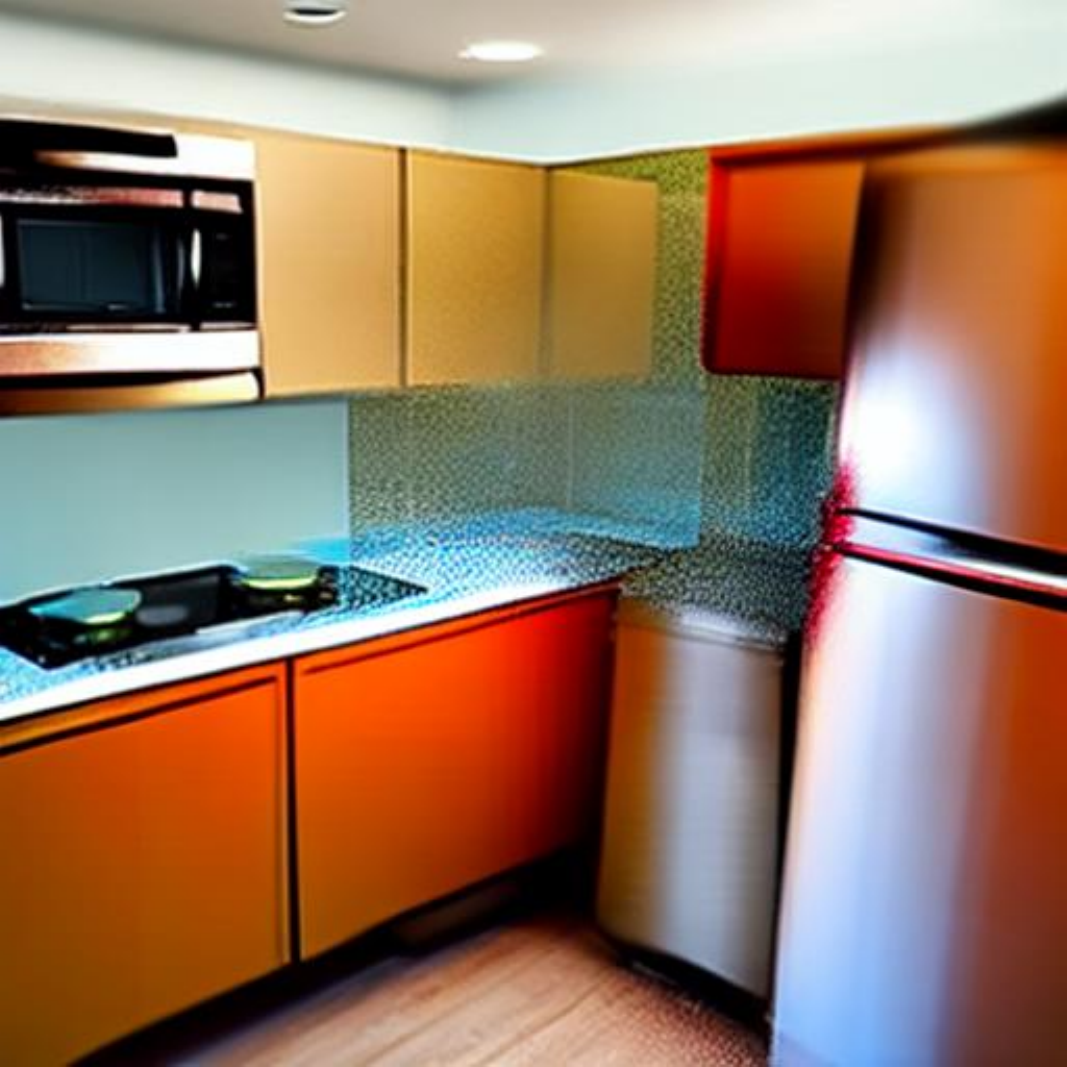} &
    \includegraphics[width=0.10\textwidth]{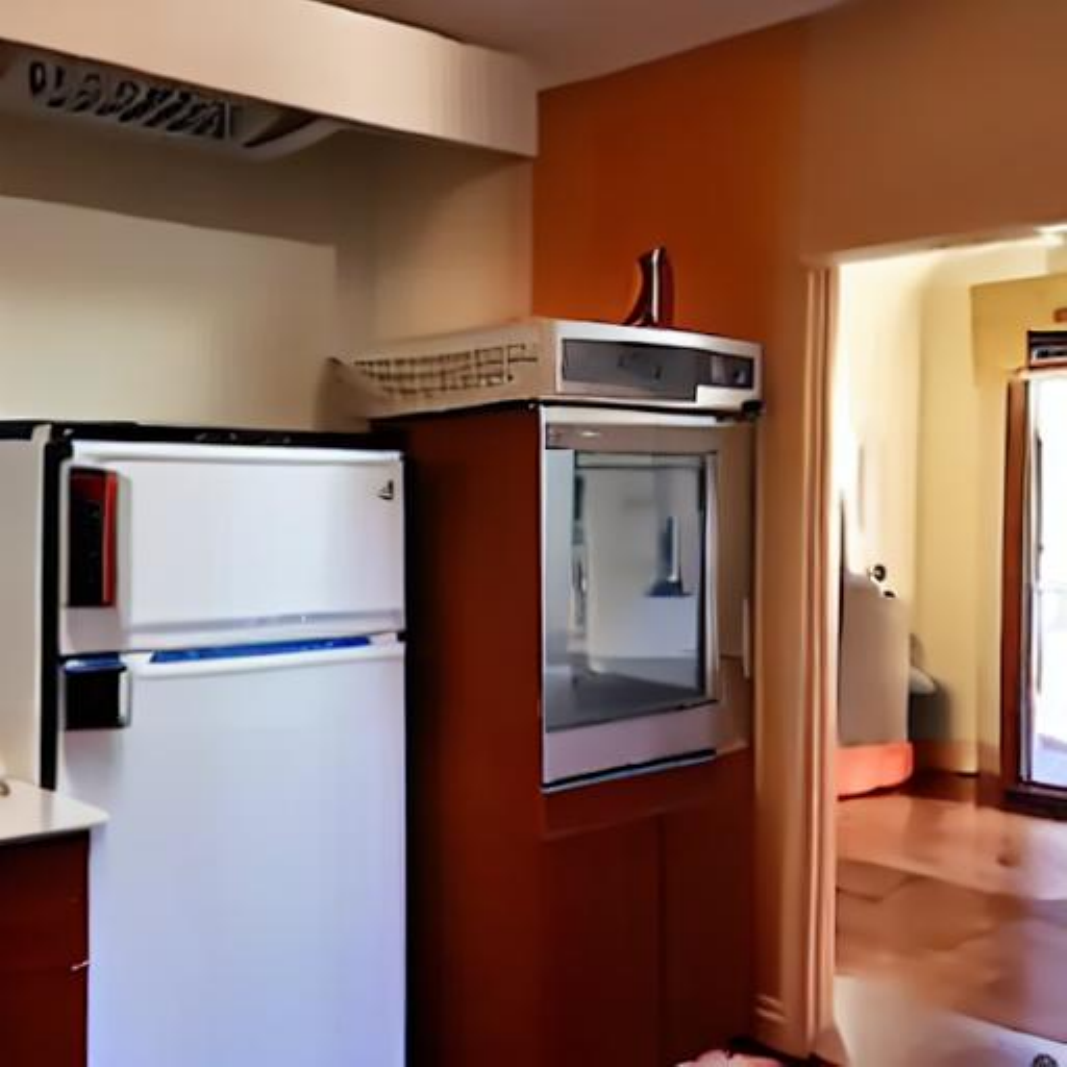} &
    \includegraphics[width=0.10\textwidth]{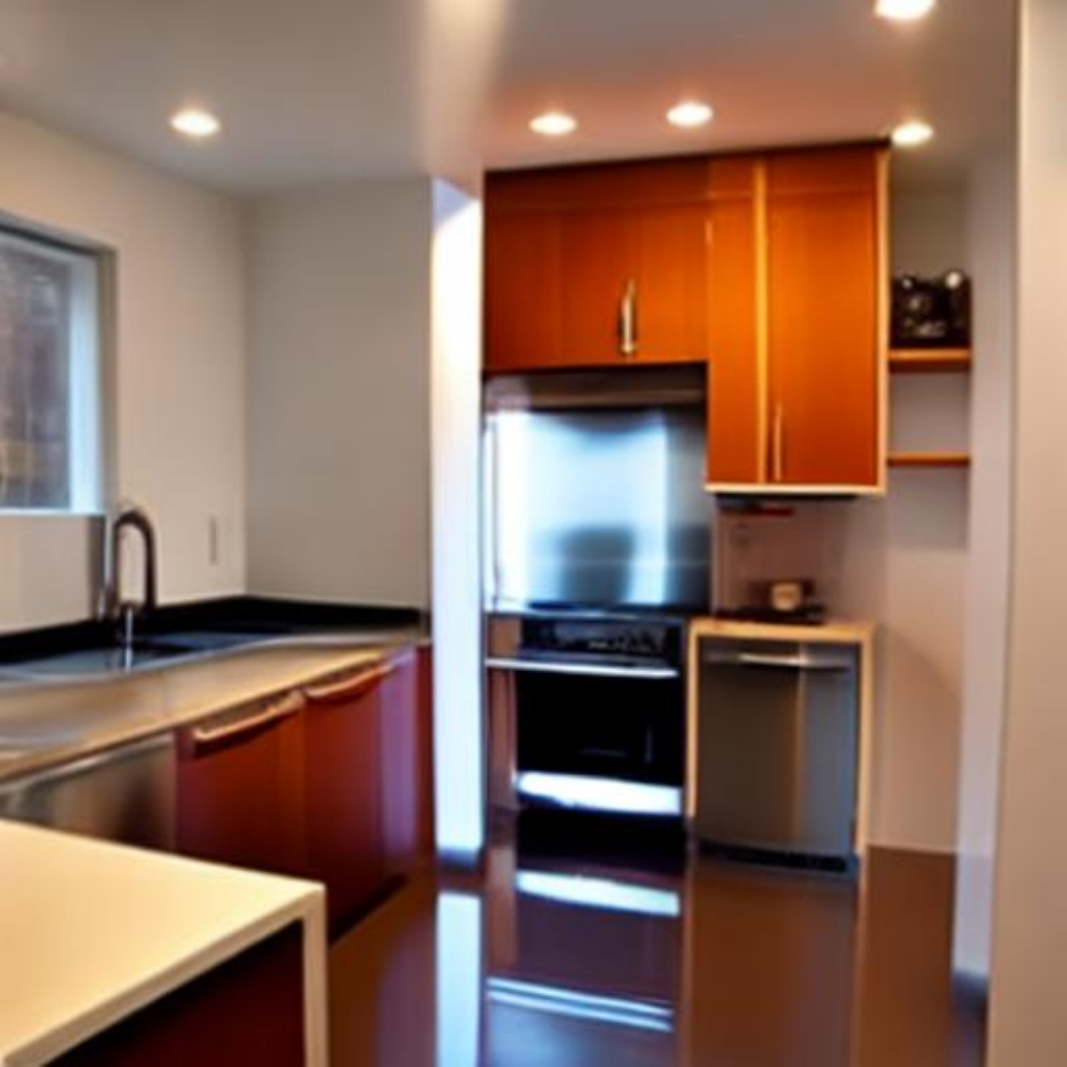} &
    \includegraphics[width=0.10\textwidth]{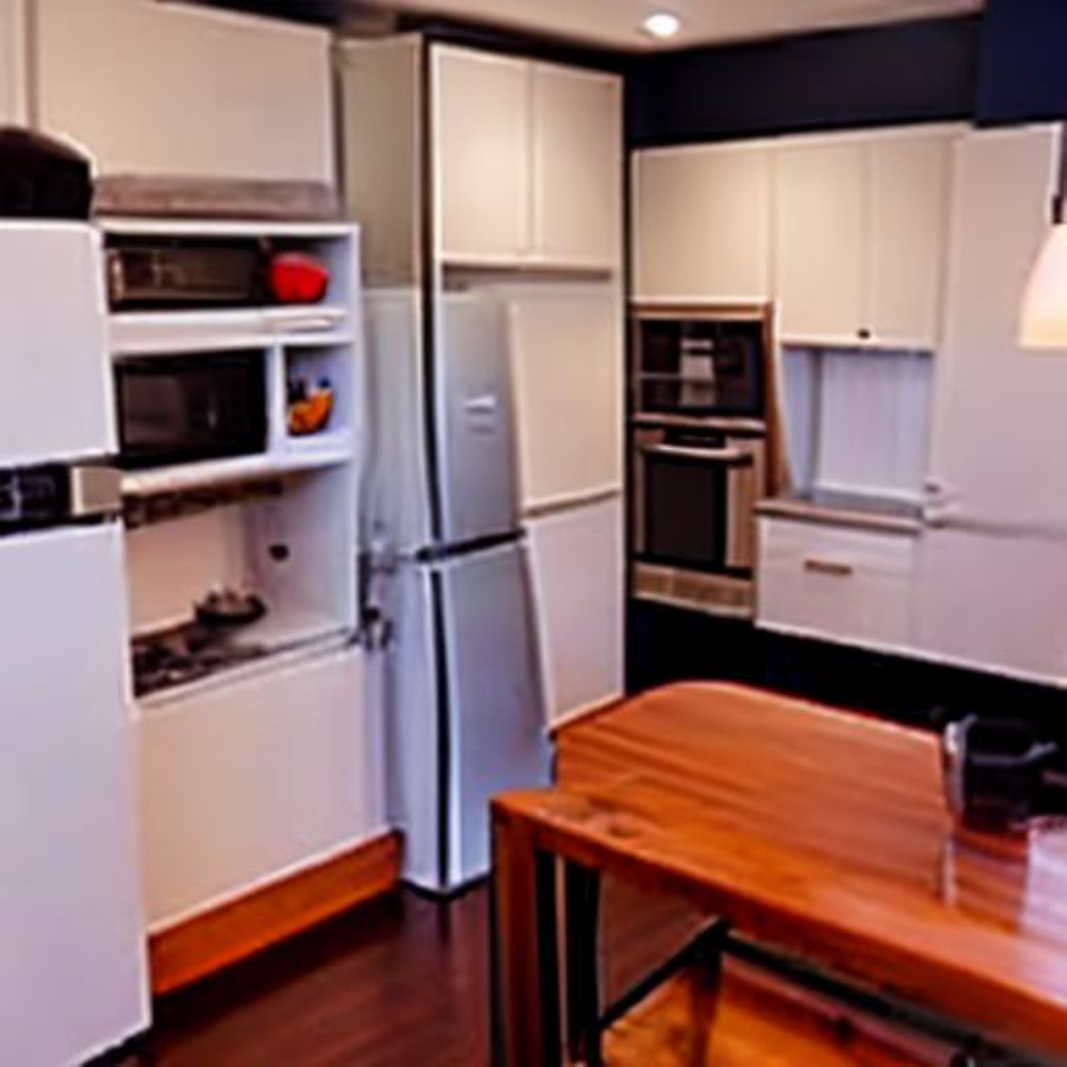} \\
};
\node[below=0pt of tl] {Constant CFG};

\matrix (tr) [gridmat, right=12pt of tl]
{
    \includegraphics[width=0.10\textwidth]{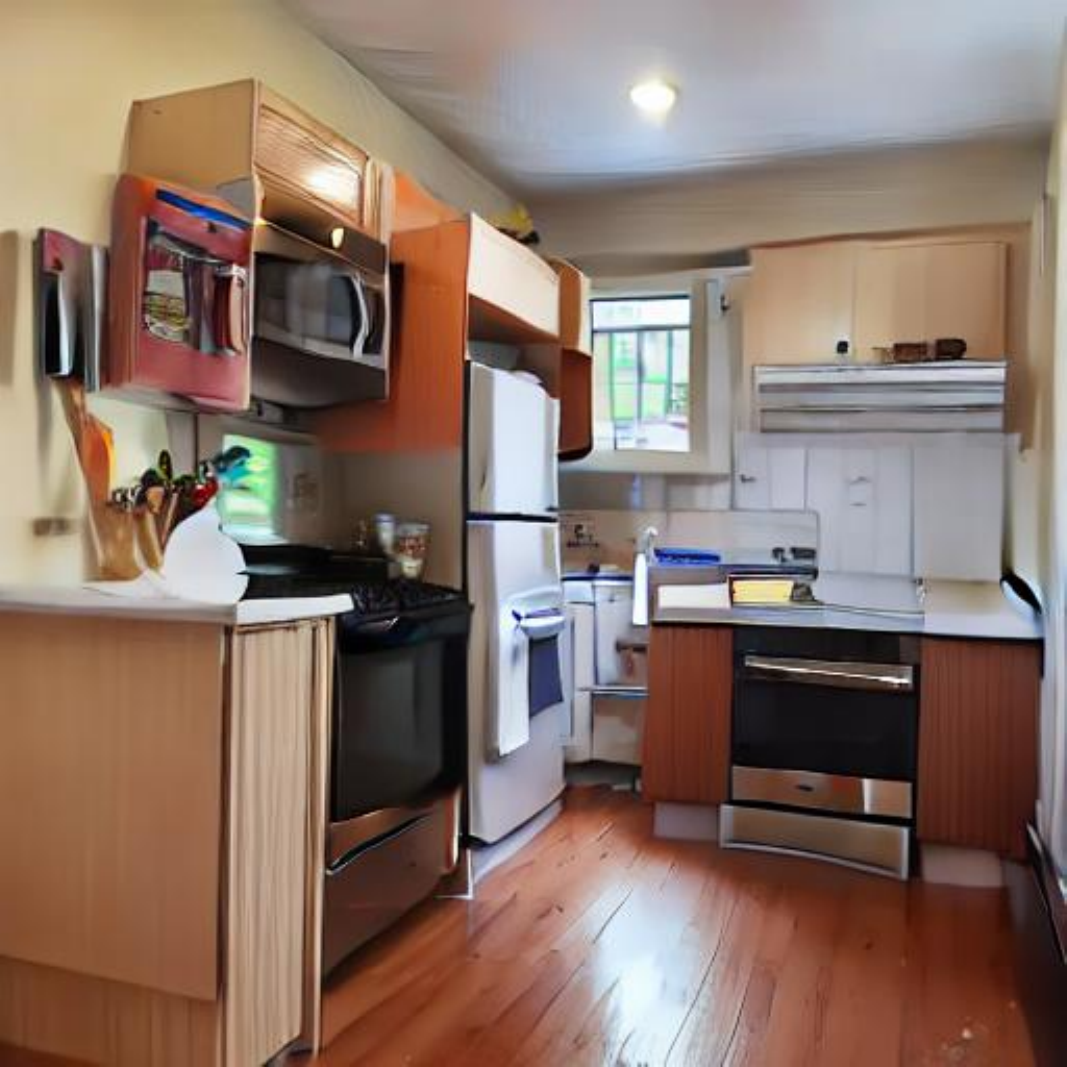} &
    \includegraphics[width=0.10\textwidth]{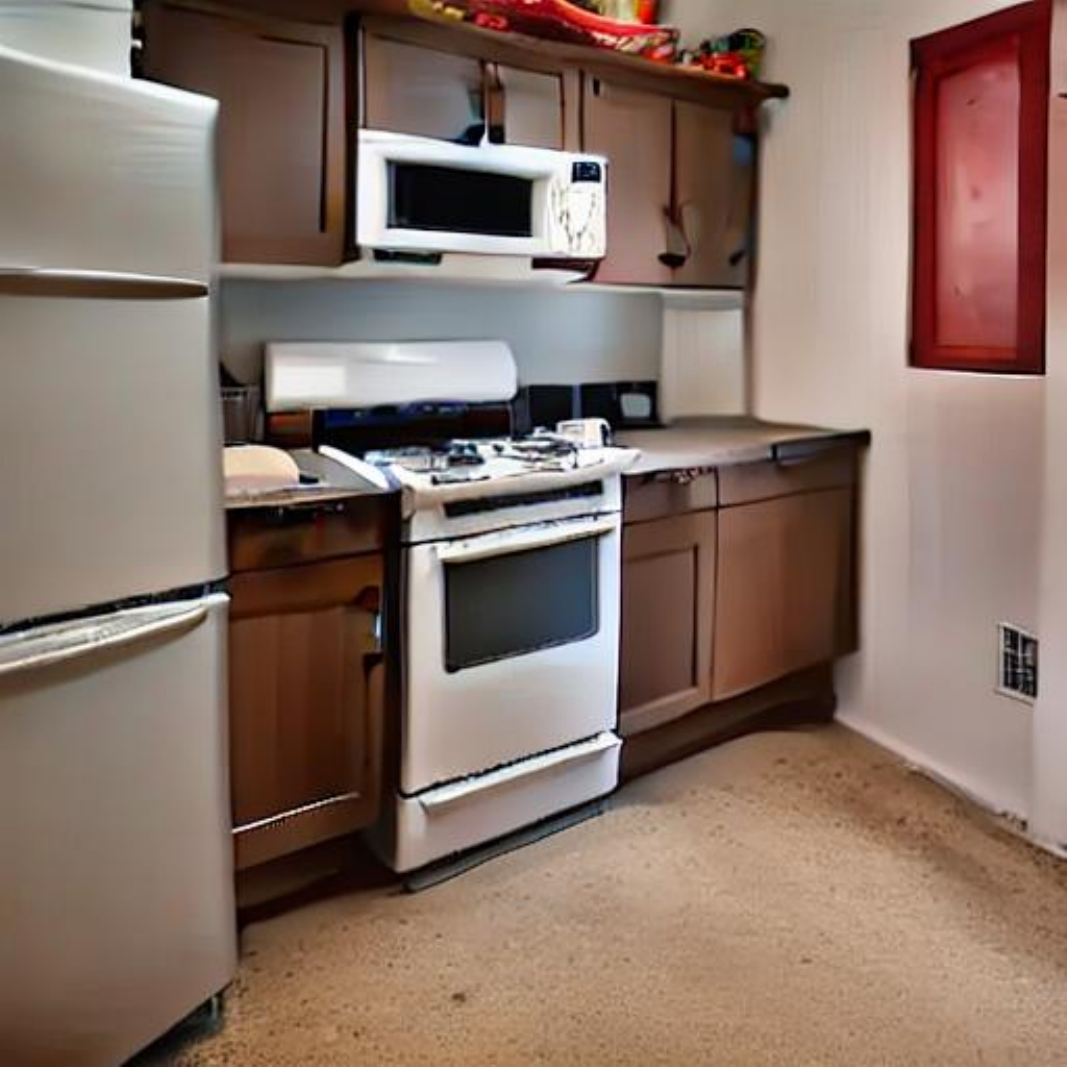} &
    \includegraphics[width=0.10\textwidth]{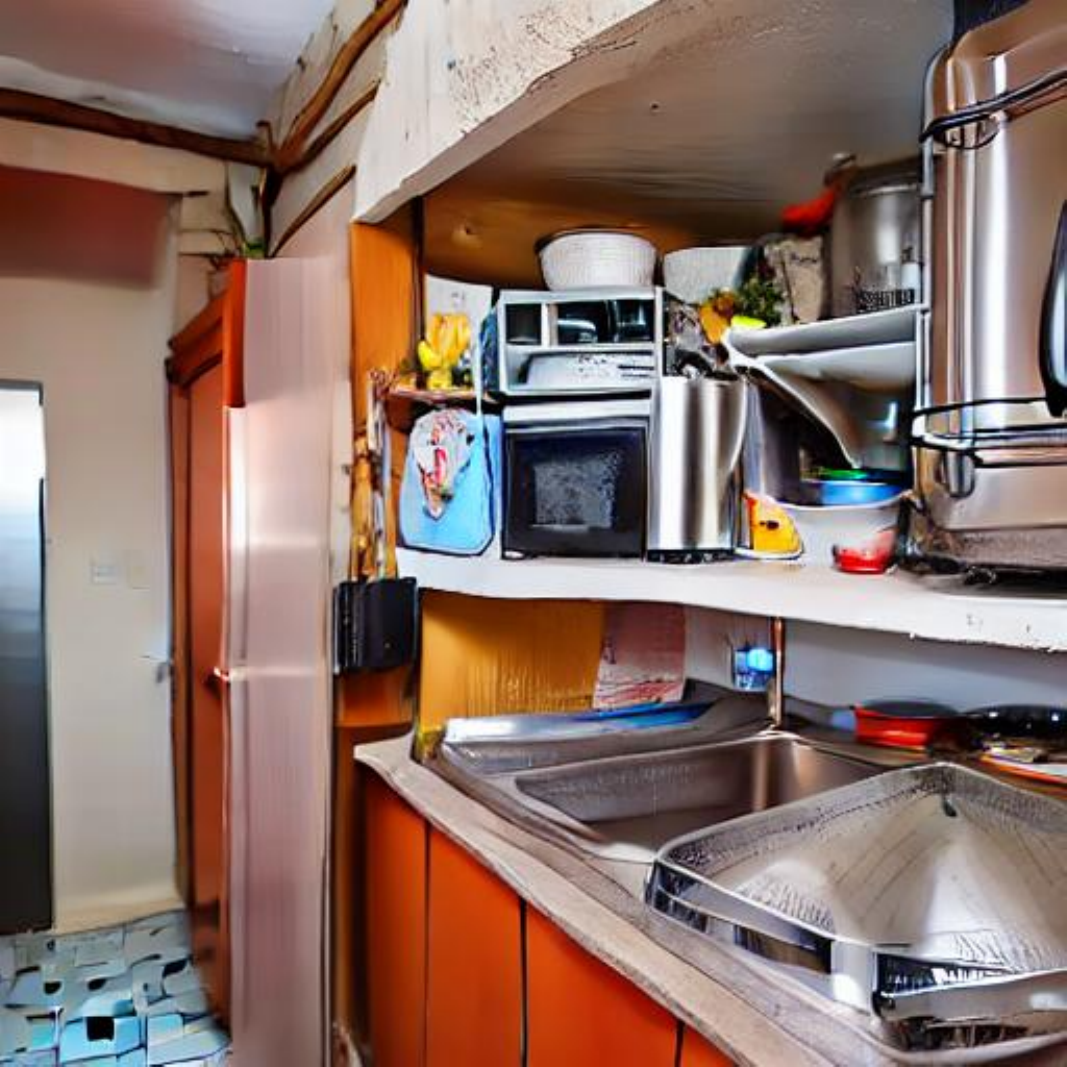} &
    \includegraphics[width=0.10\textwidth]{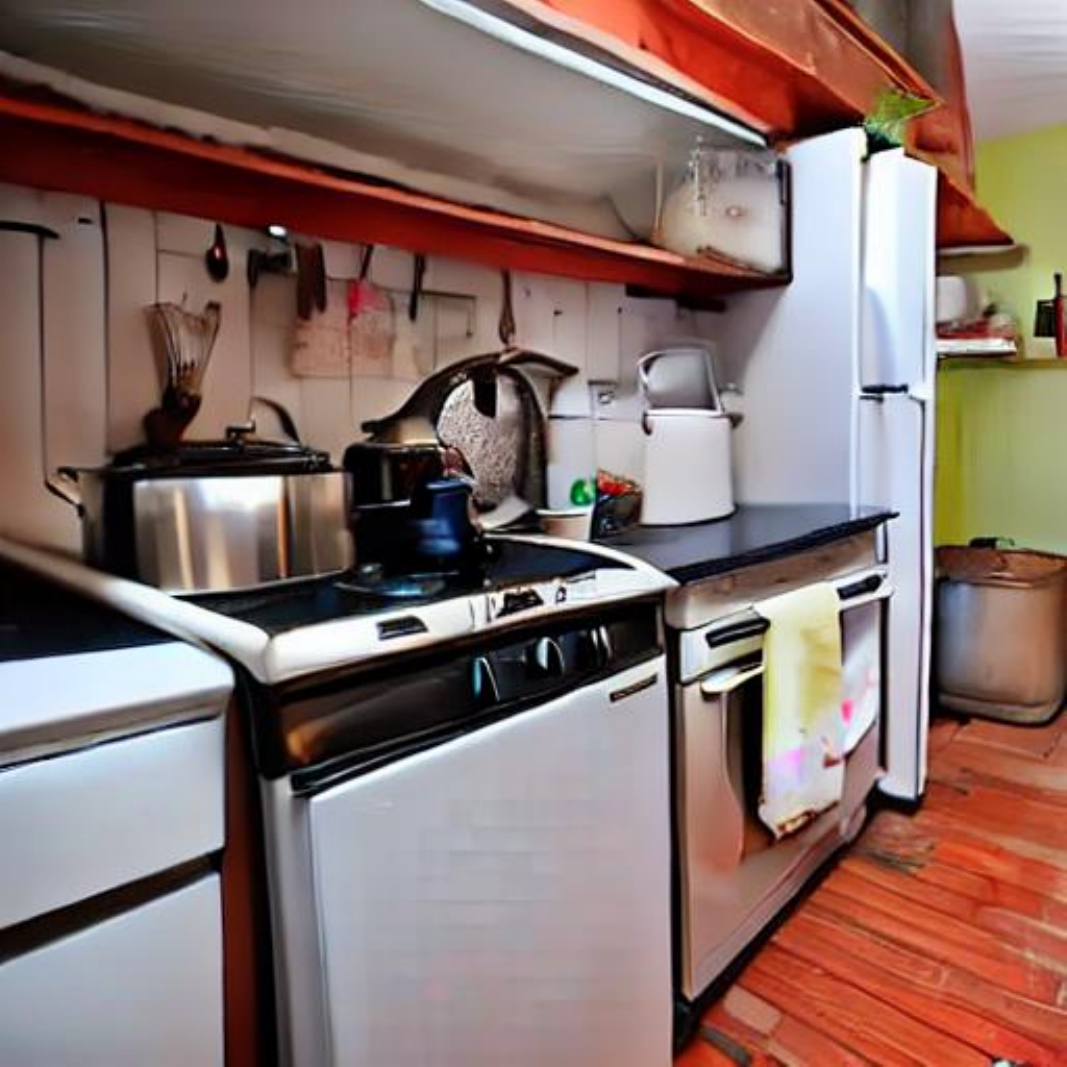} \\
    \includegraphics[width=0.10\textwidth]{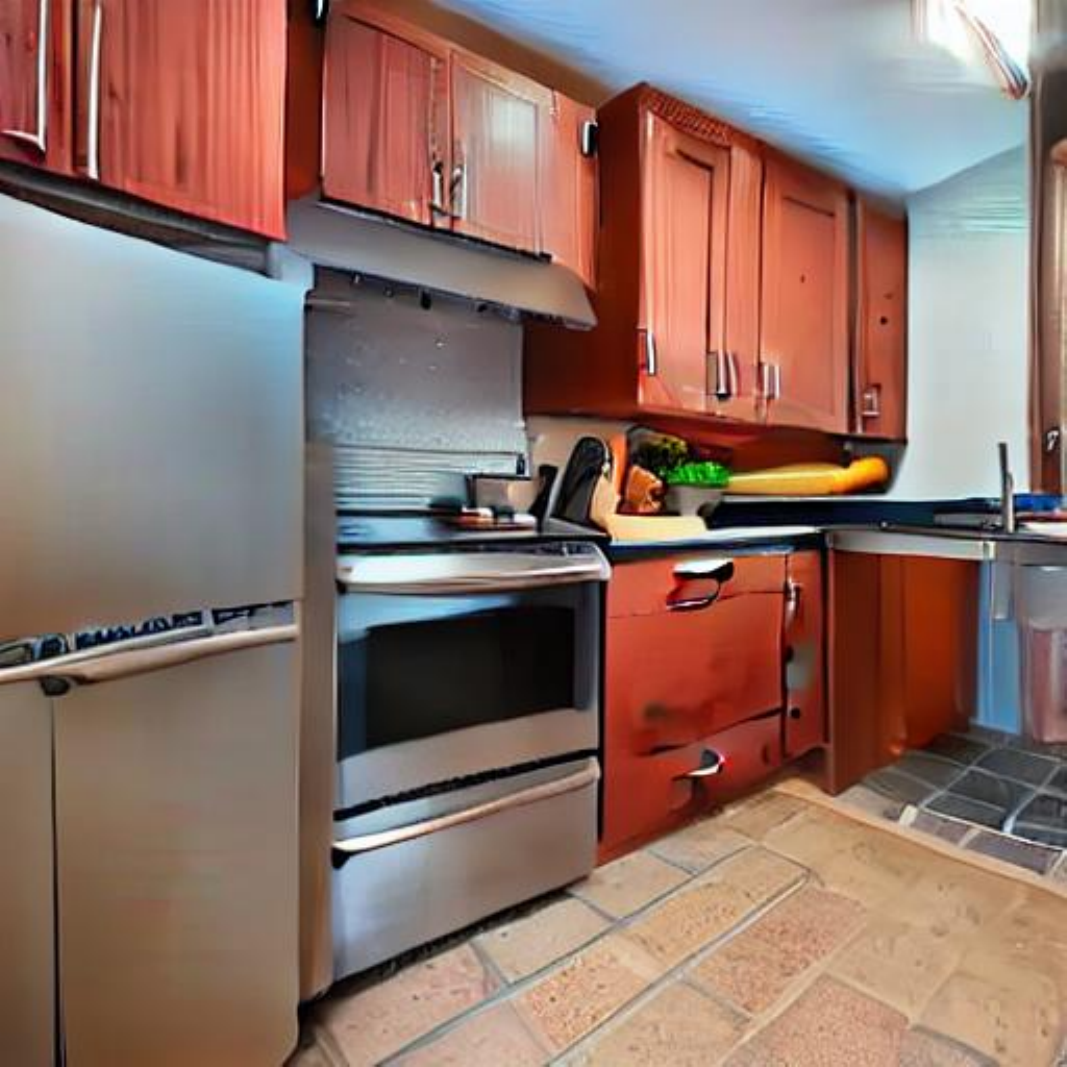} &
    \includegraphics[width=0.10\textwidth]{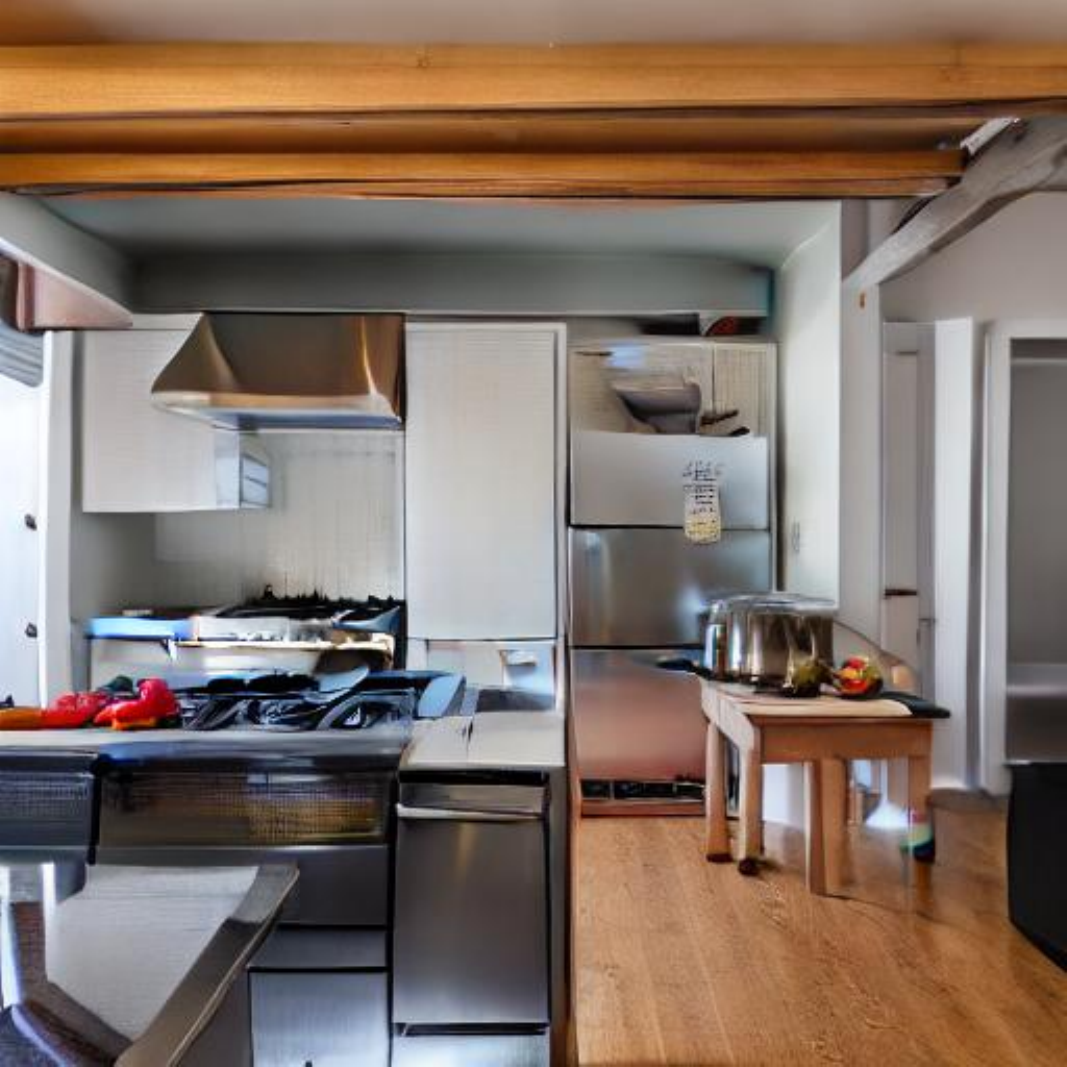} &
    \includegraphics[width=0.10\textwidth]{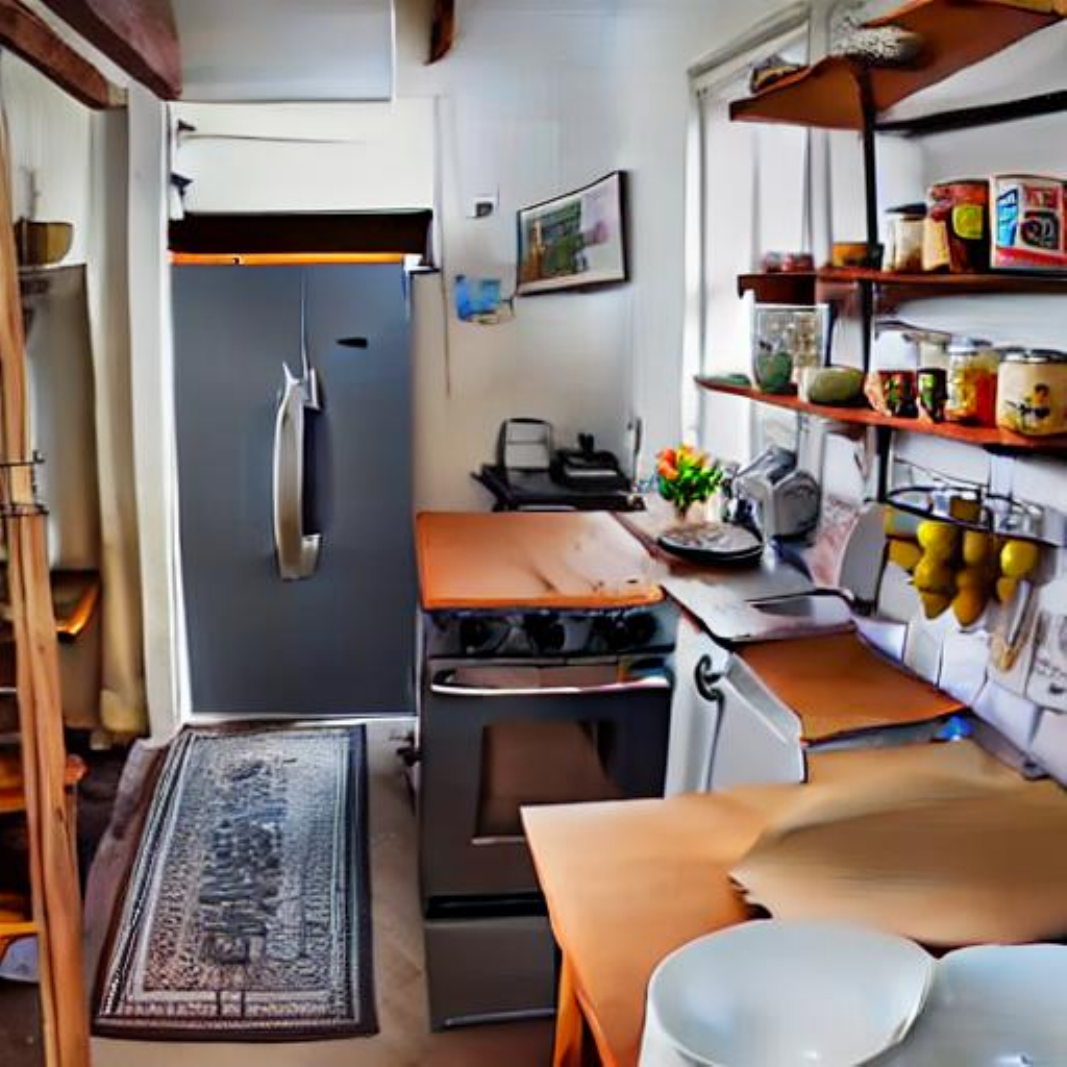} &
    \includegraphics[width=0.10\textwidth]{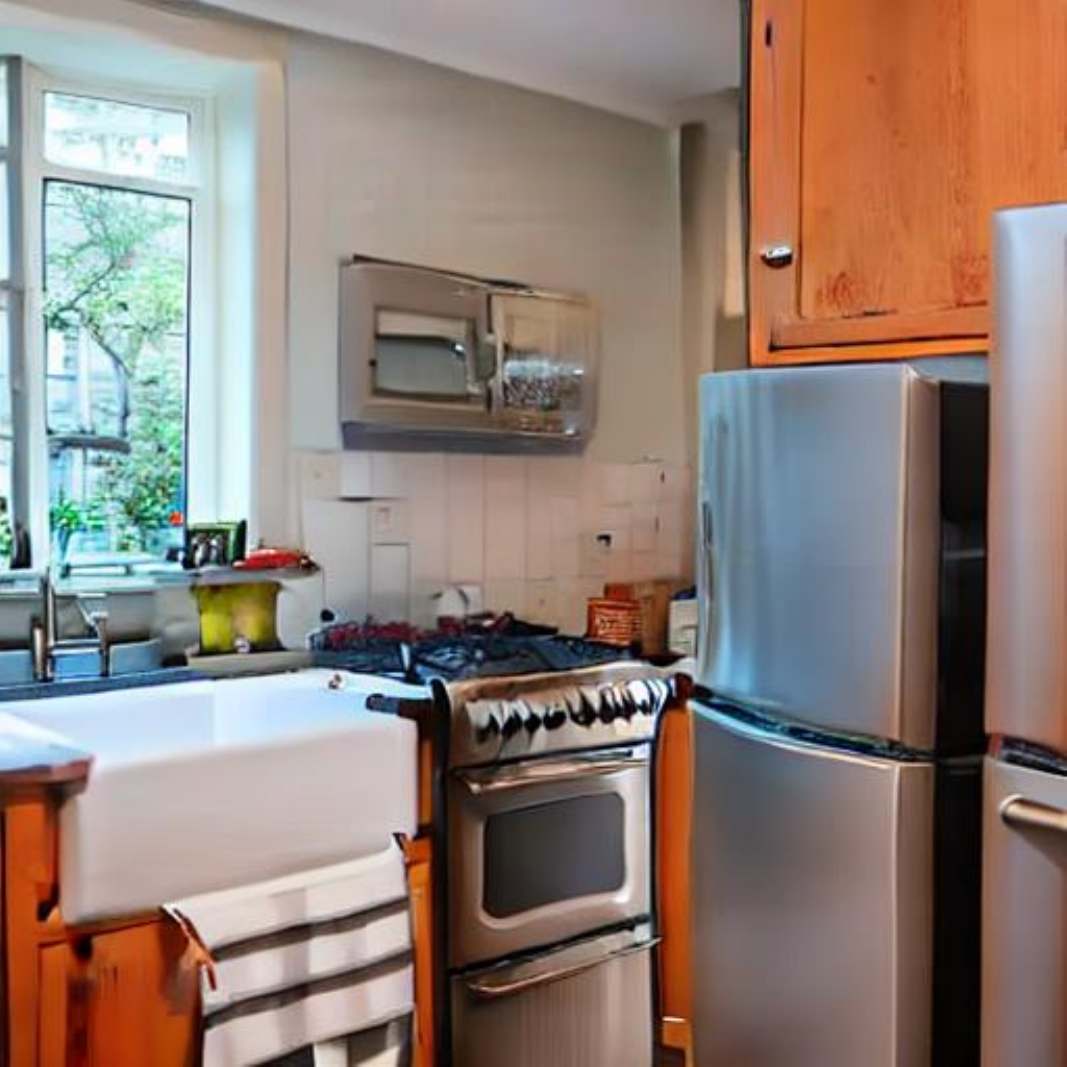} \\
    \includegraphics[width=0.10\textwidth]{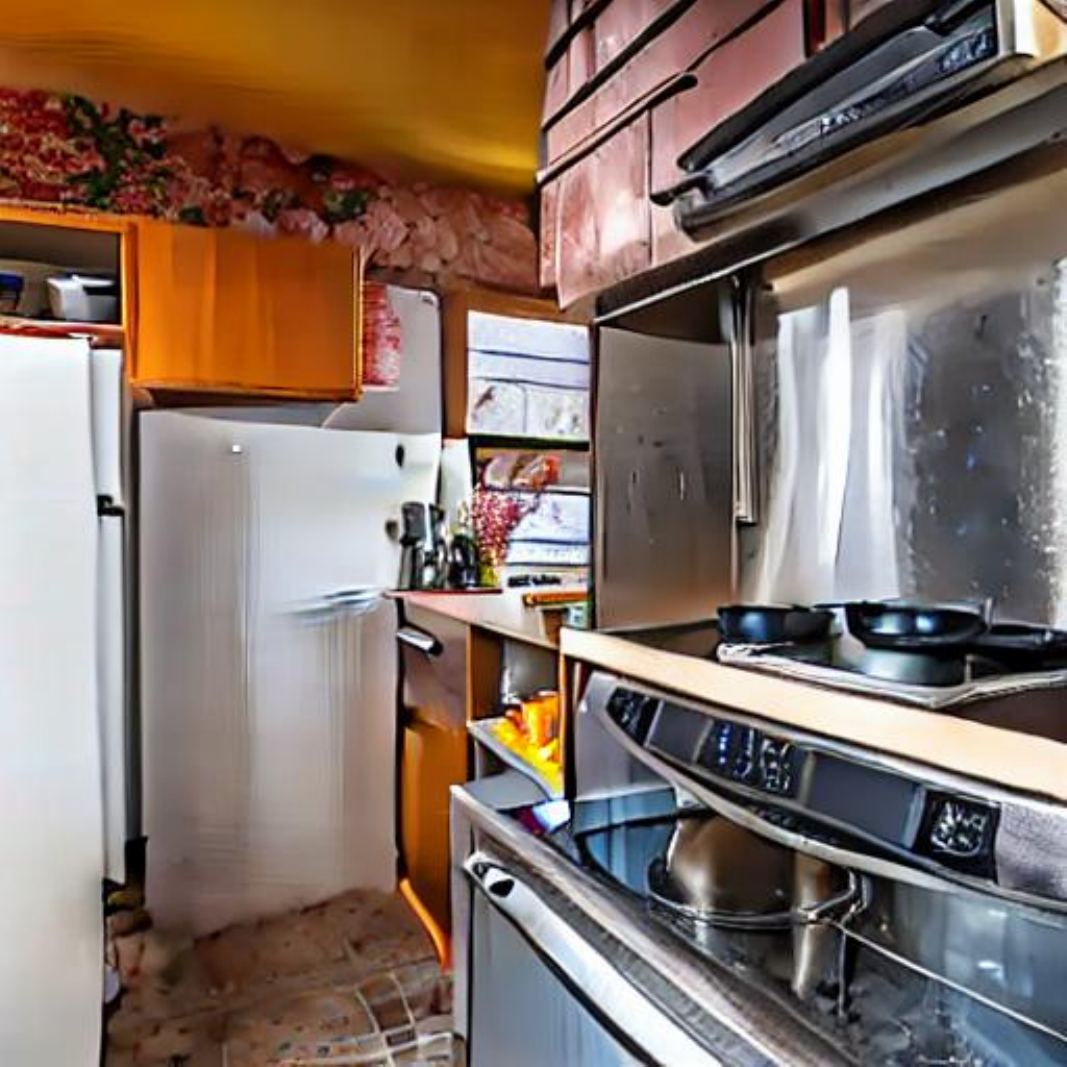} &
    \includegraphics[width=0.10\textwidth]{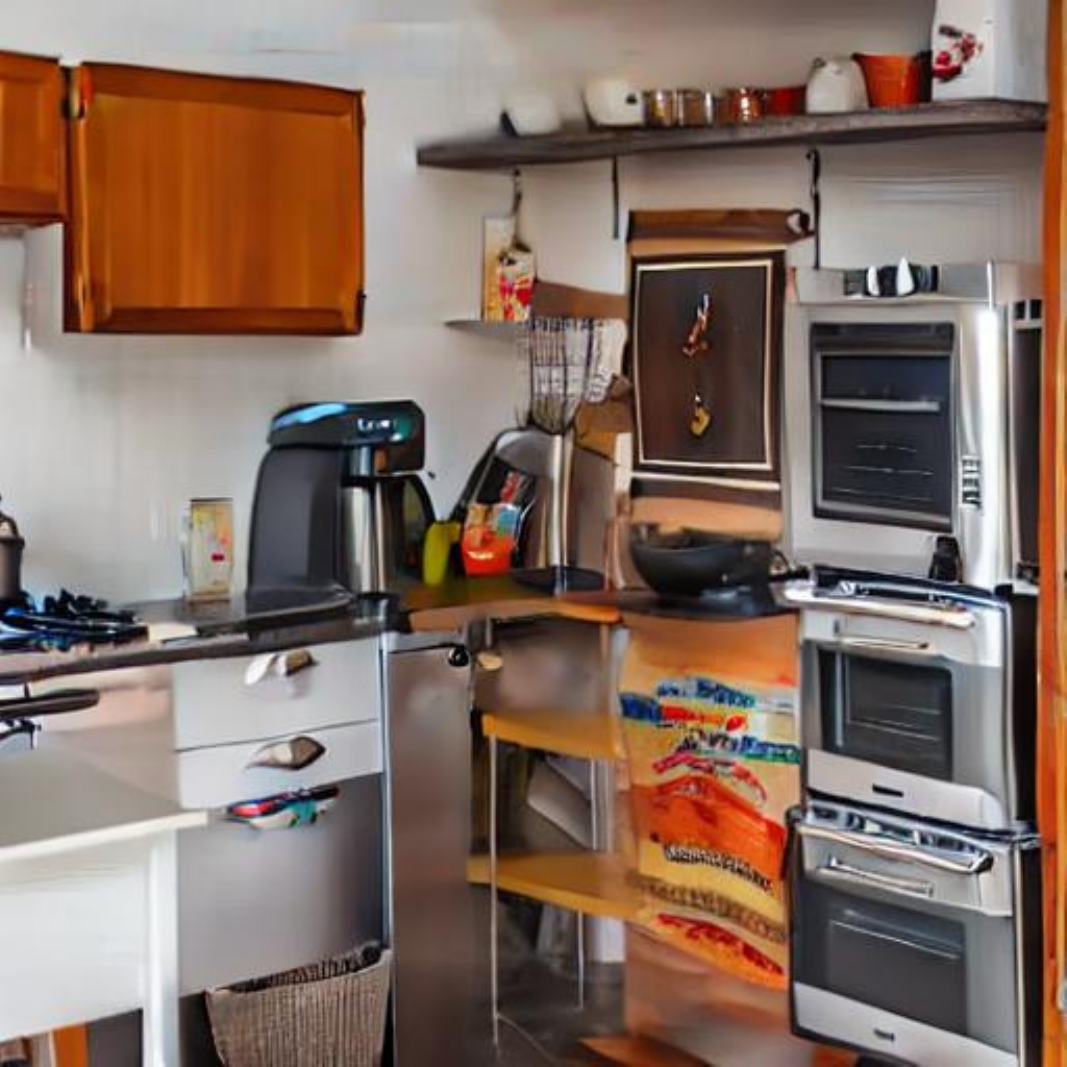} &
    \includegraphics[width=0.10\textwidth]{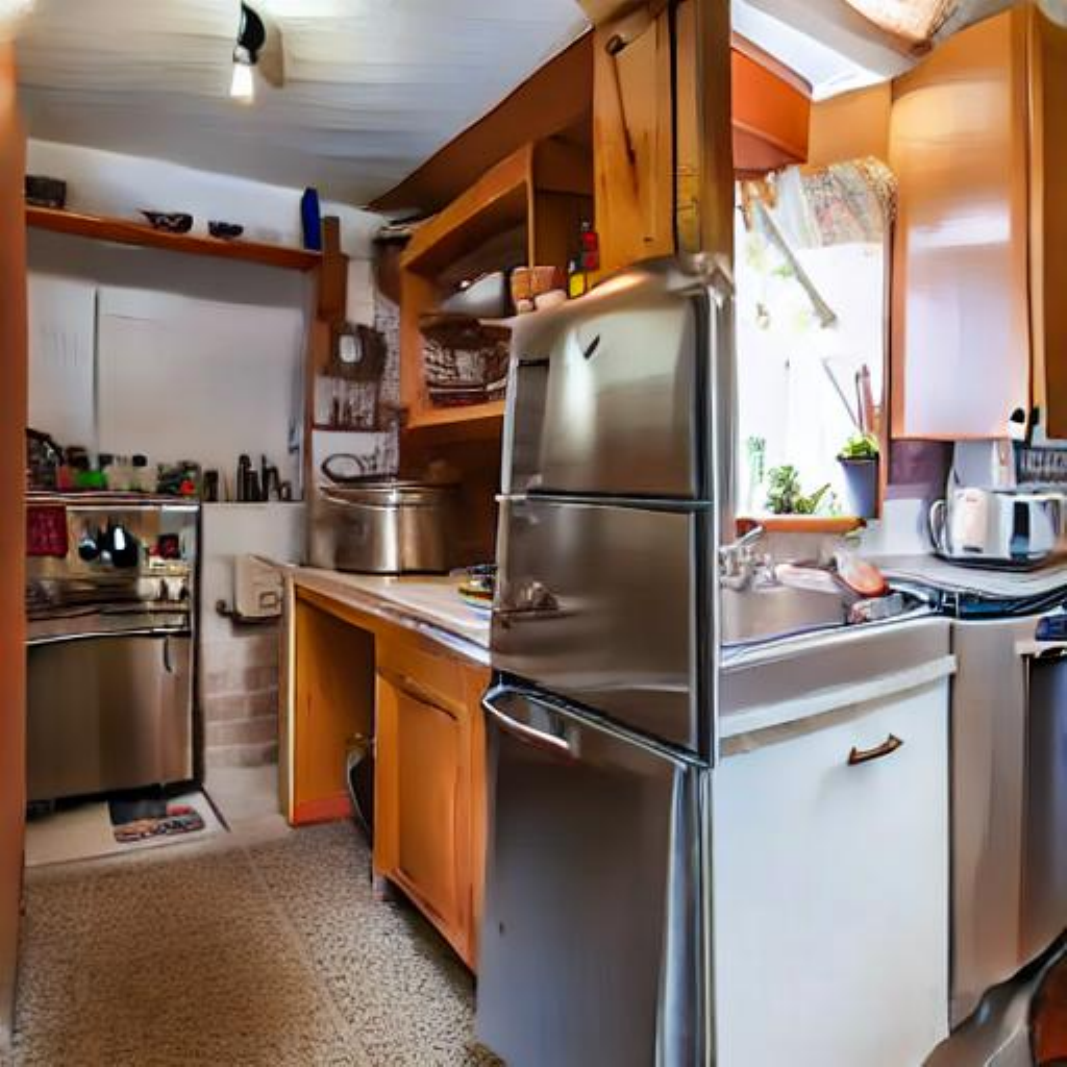} &
    \includegraphics[width=0.10\textwidth]{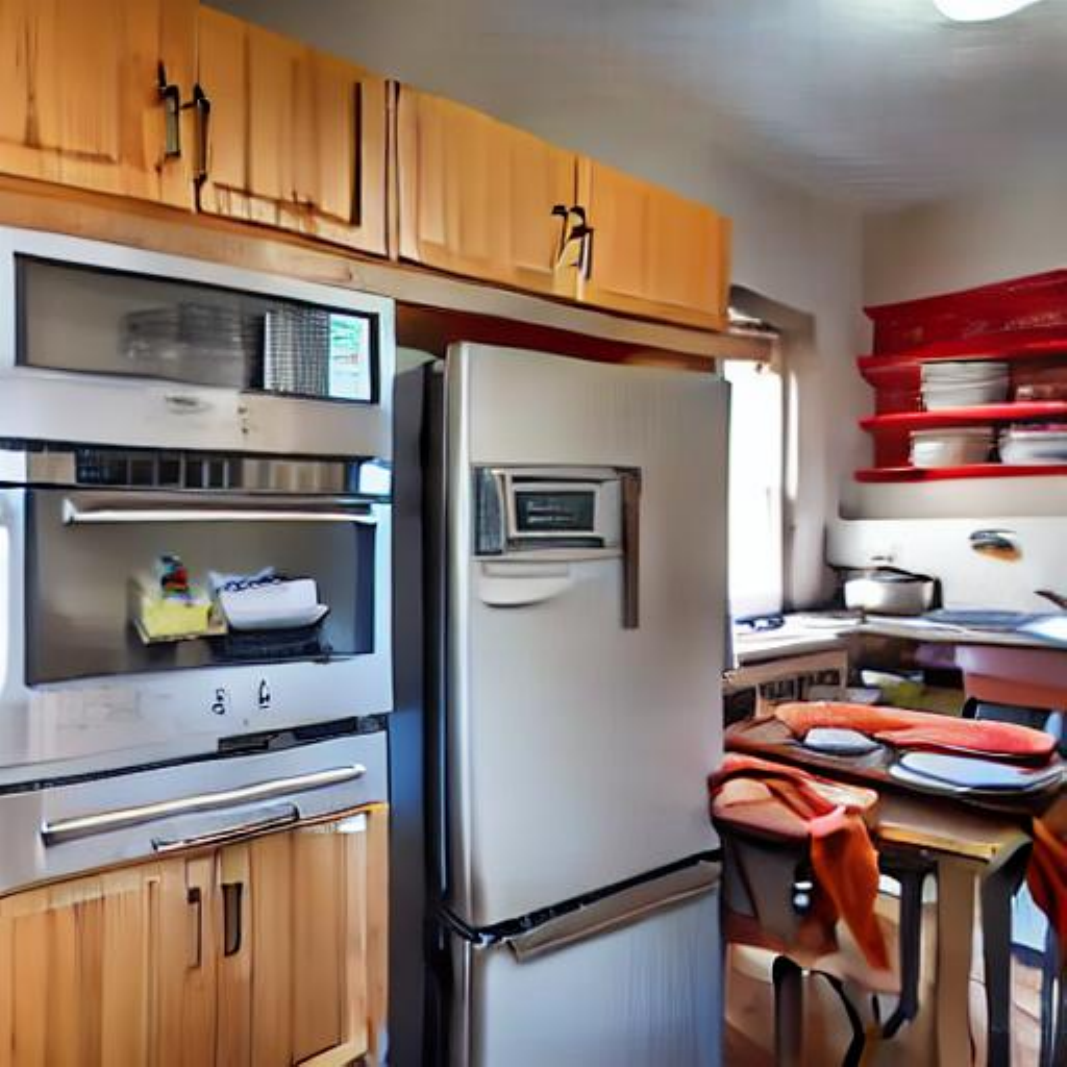} \\
    \includegraphics[width=0.10\textwidth]{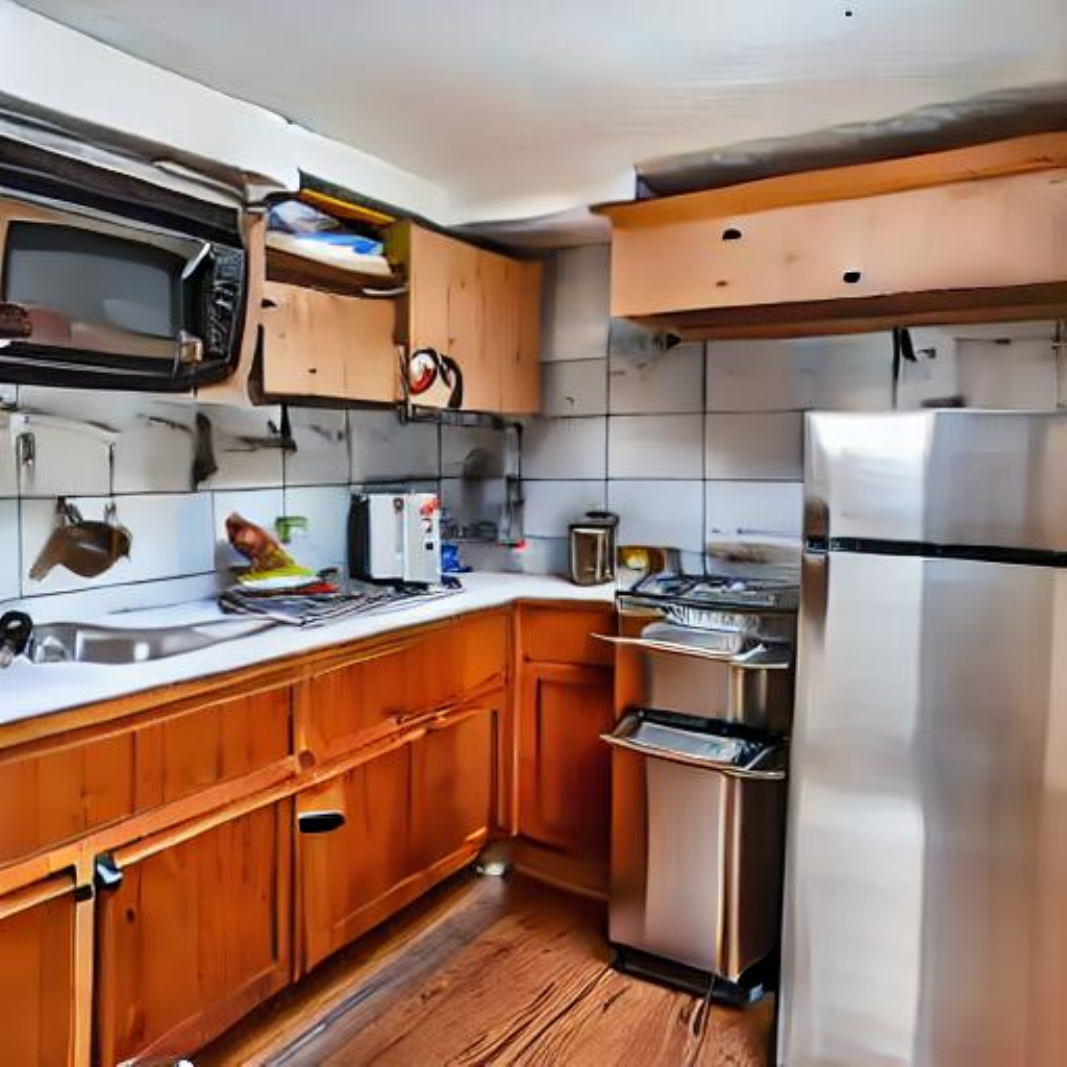} &
    \includegraphics[width=0.10\textwidth]{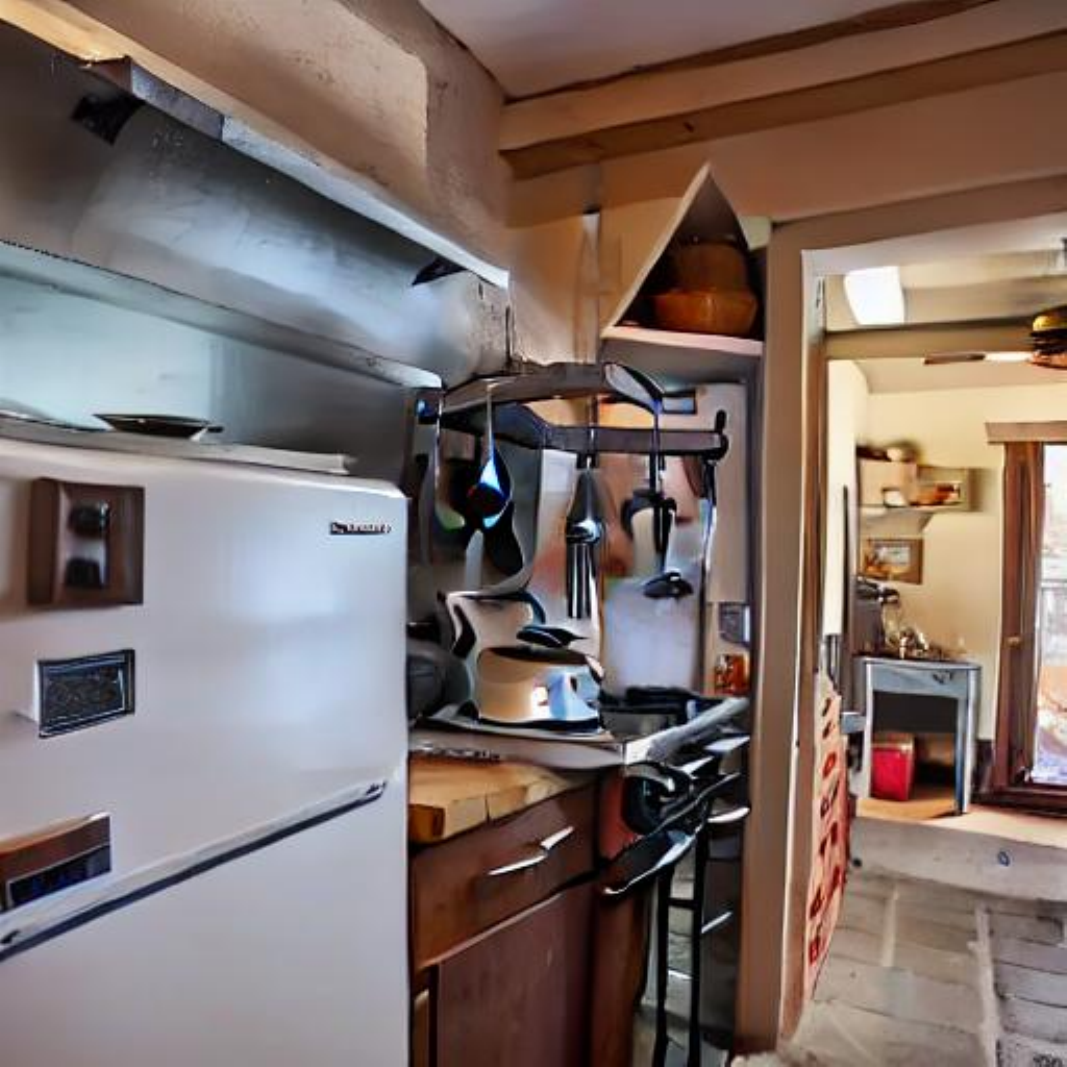} &
    \includegraphics[width=0.10\textwidth]{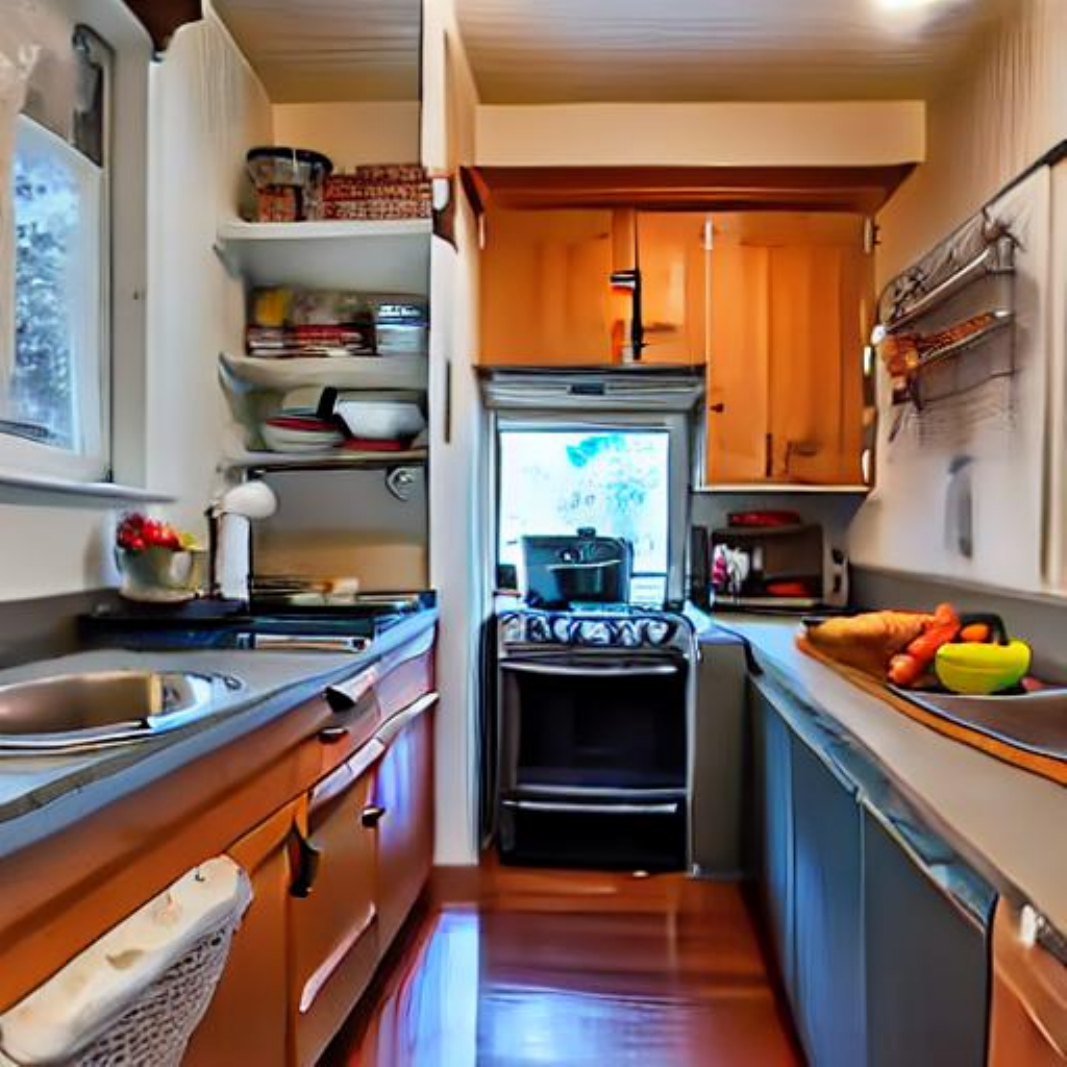} &
    \includegraphics[width=0.10\textwidth]{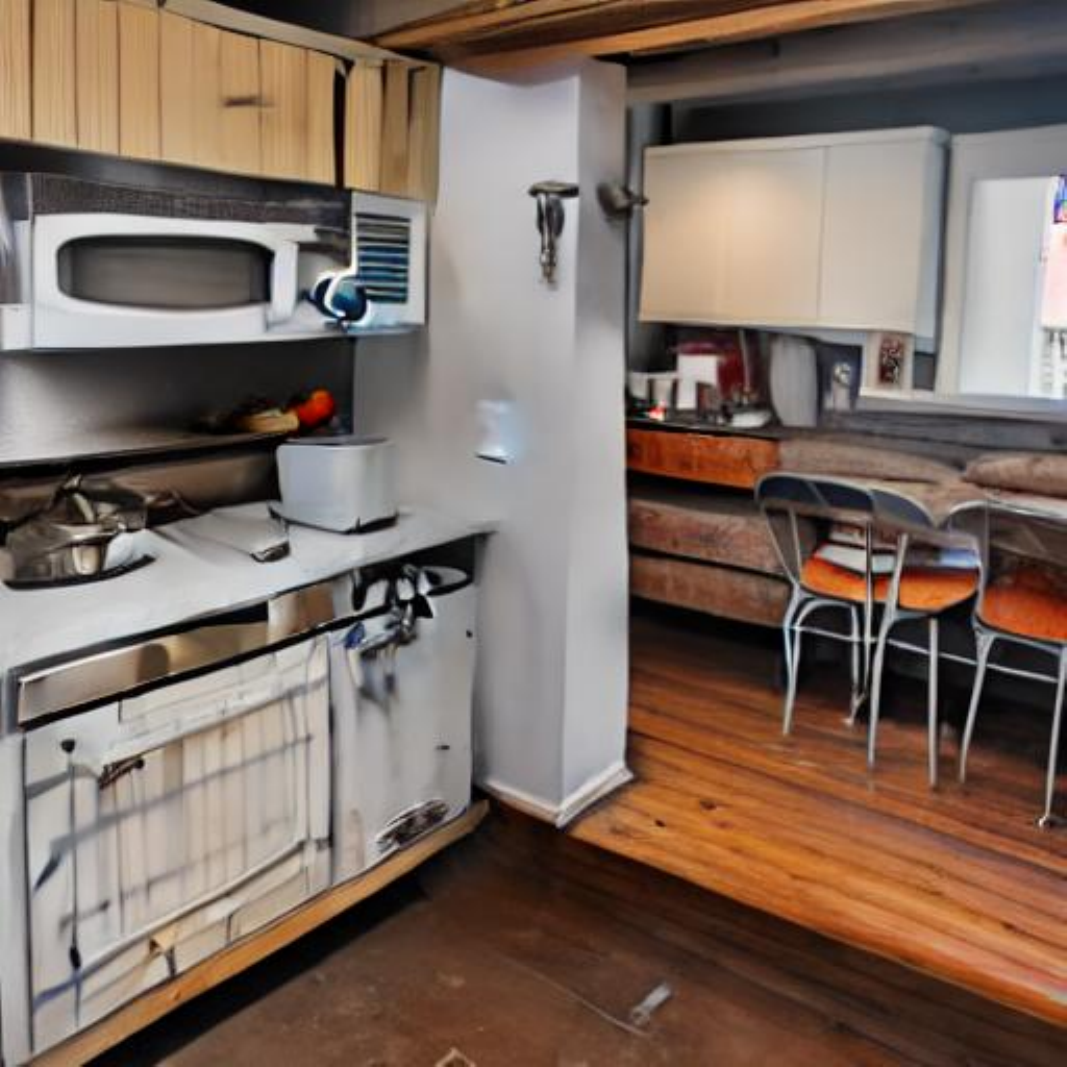} \\
};
\node[below=0pt of tr] {Medium 1};

\matrix (bl) [gridmat, below=12pt of tl]
{
    \includegraphics[width=0.10\textwidth]{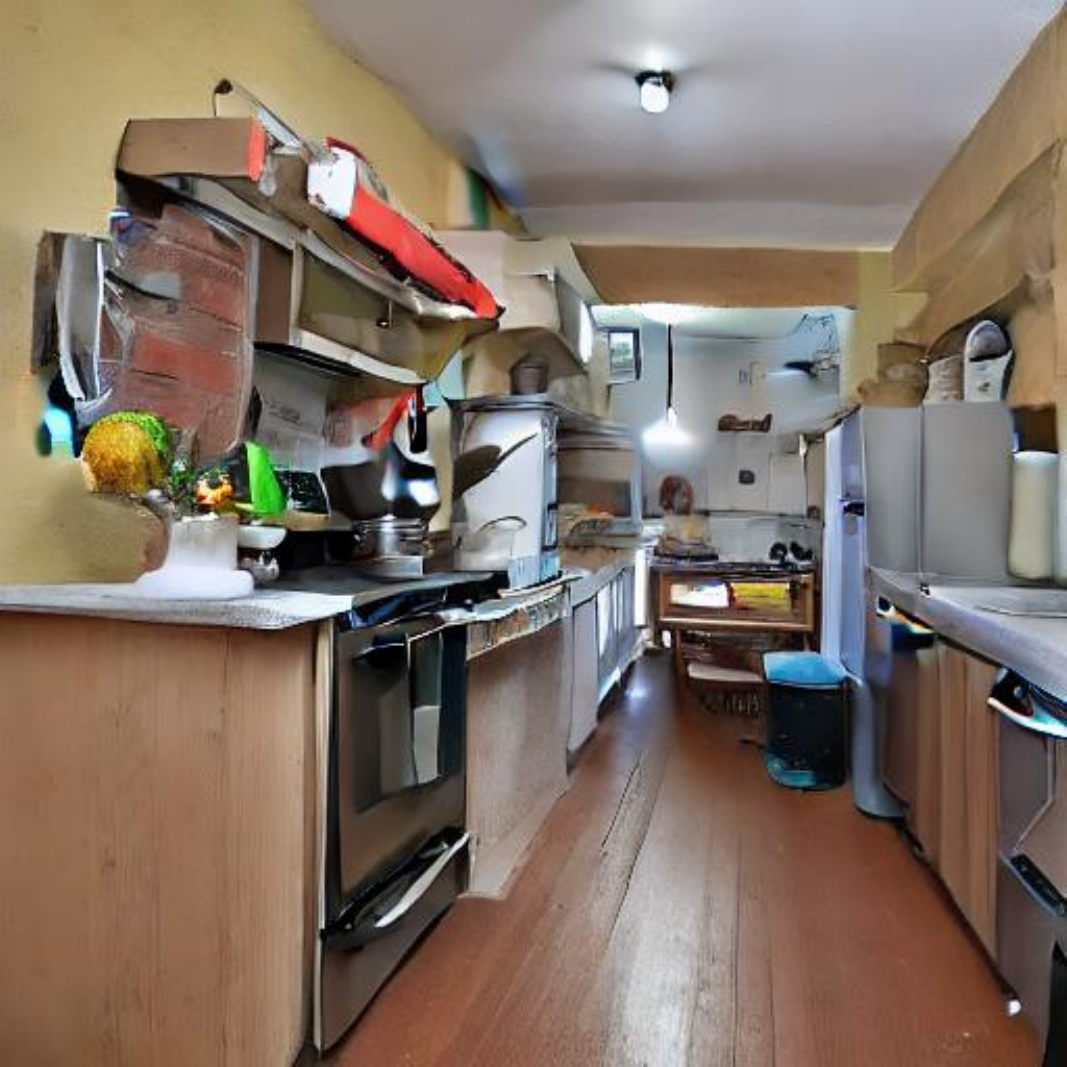} &
    \includegraphics[width=0.10\textwidth]{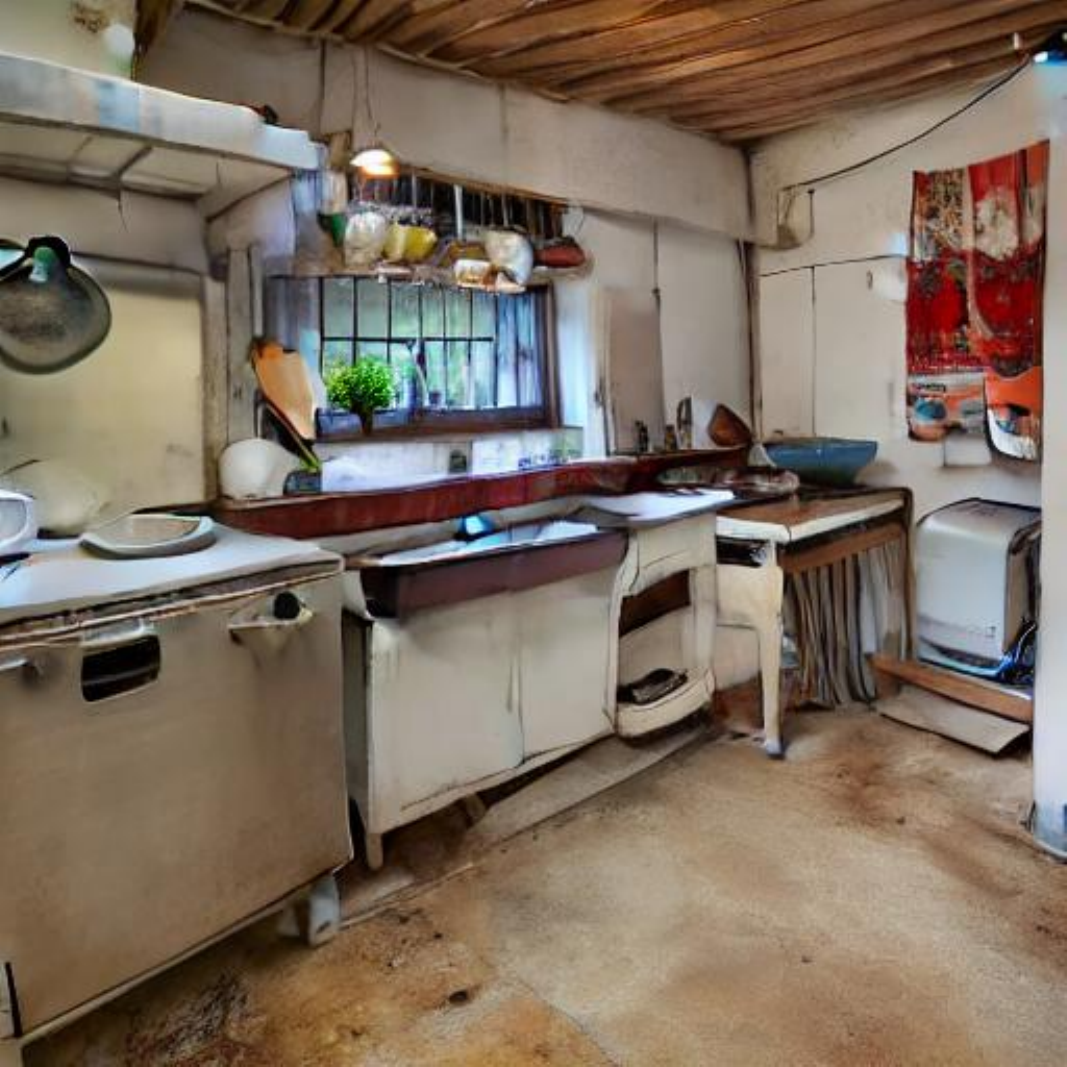} &
    \includegraphics[width=0.10\textwidth]{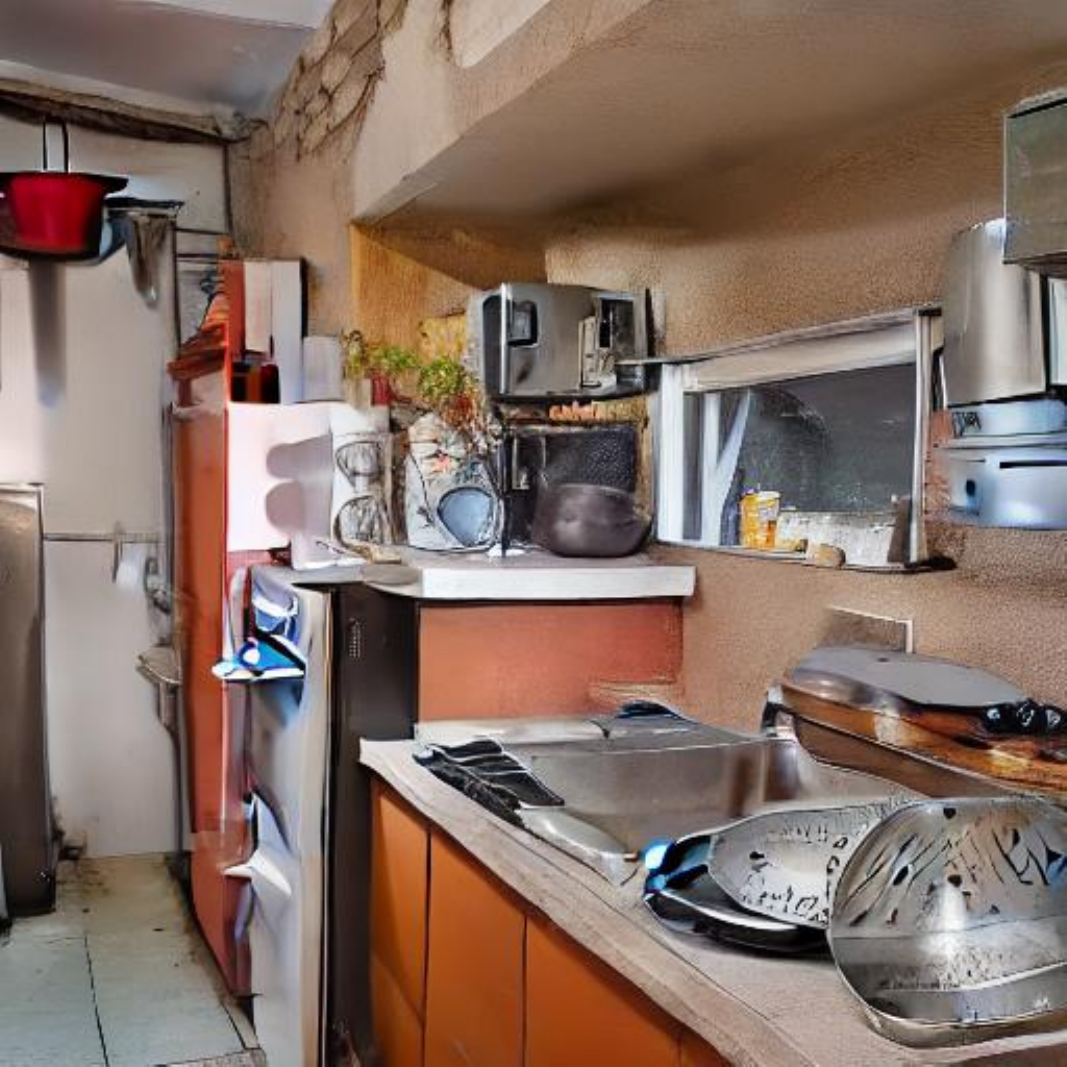} &
    \includegraphics[width=0.10\textwidth]{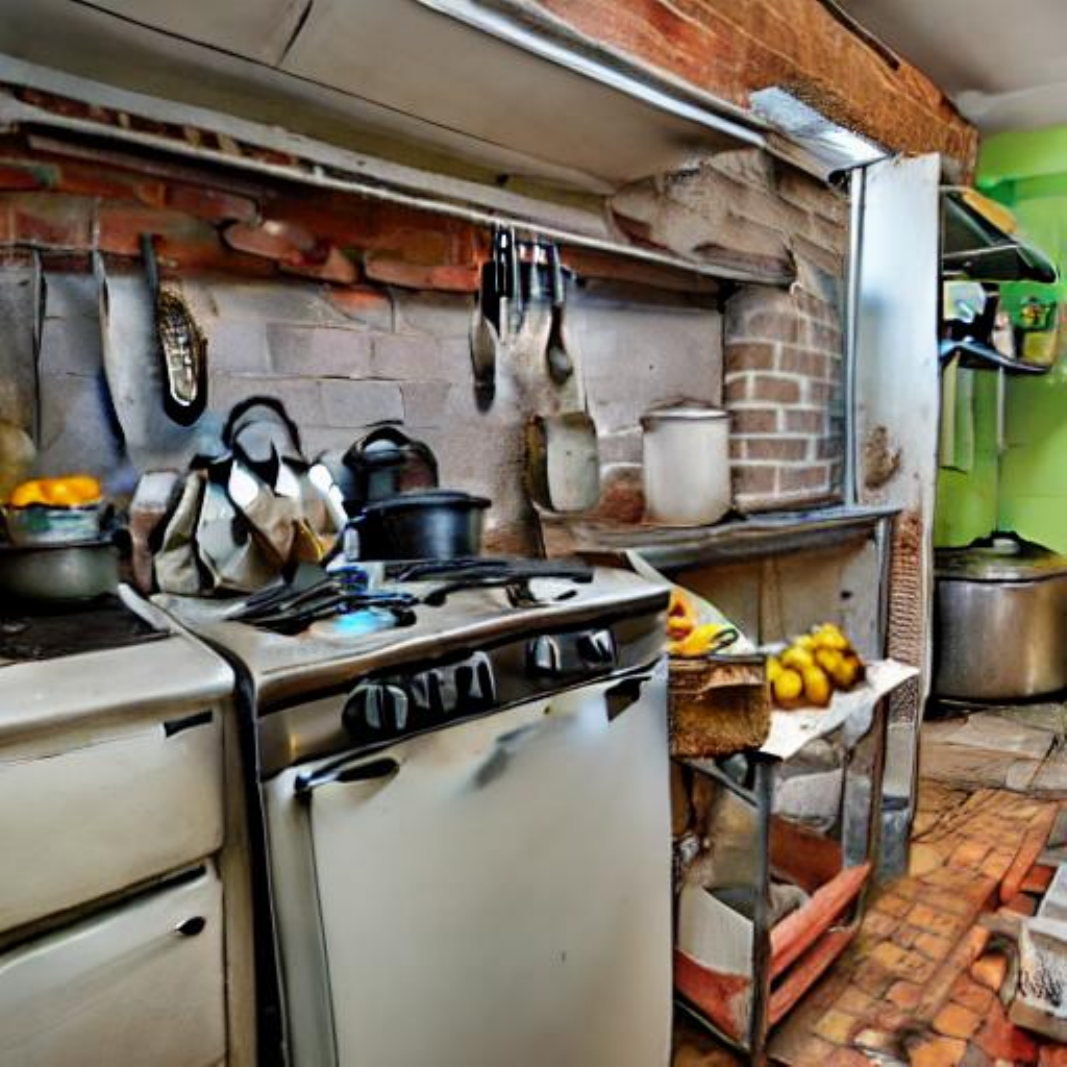} \\
    \includegraphics[width=0.10\textwidth]{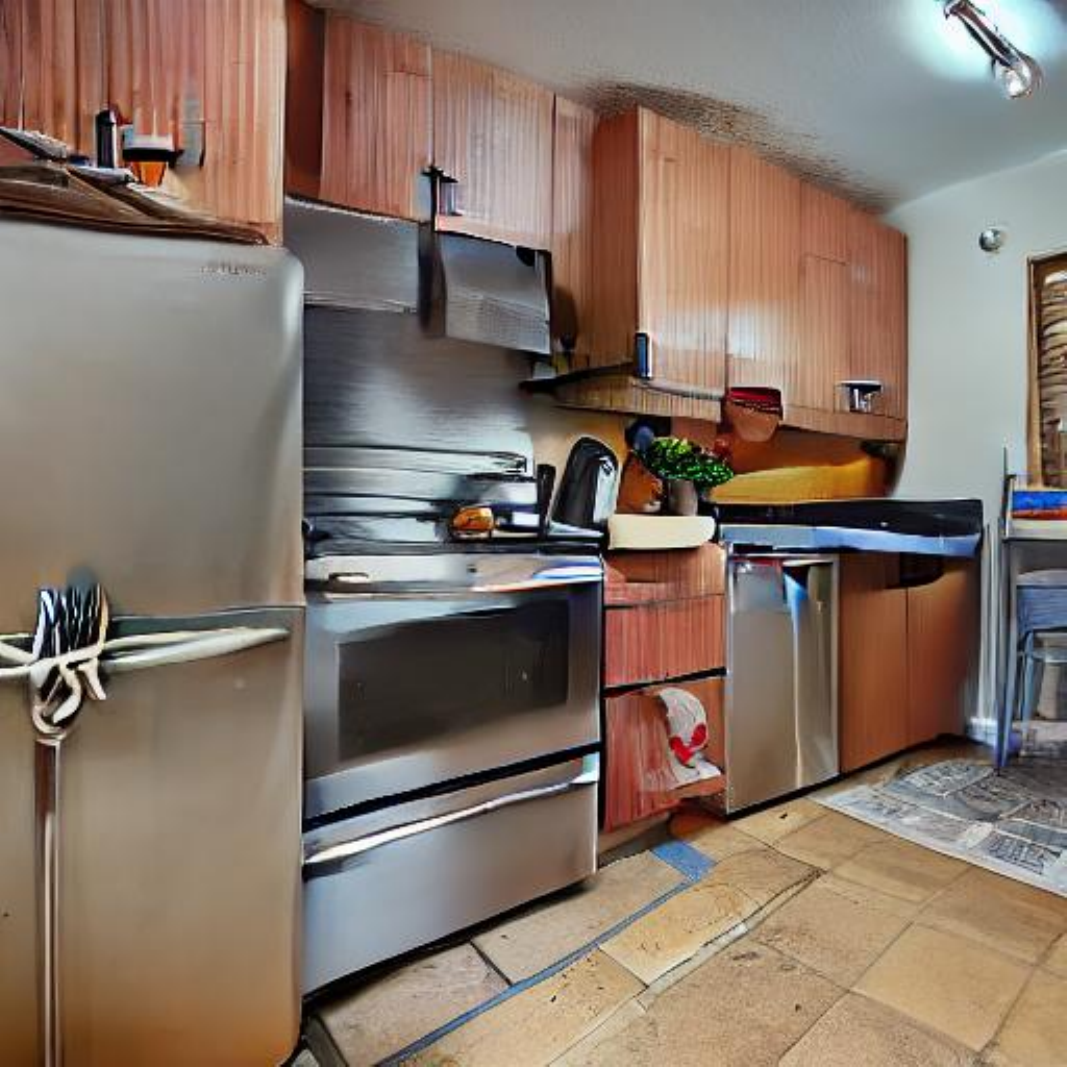} &
    \includegraphics[width=0.10\textwidth]{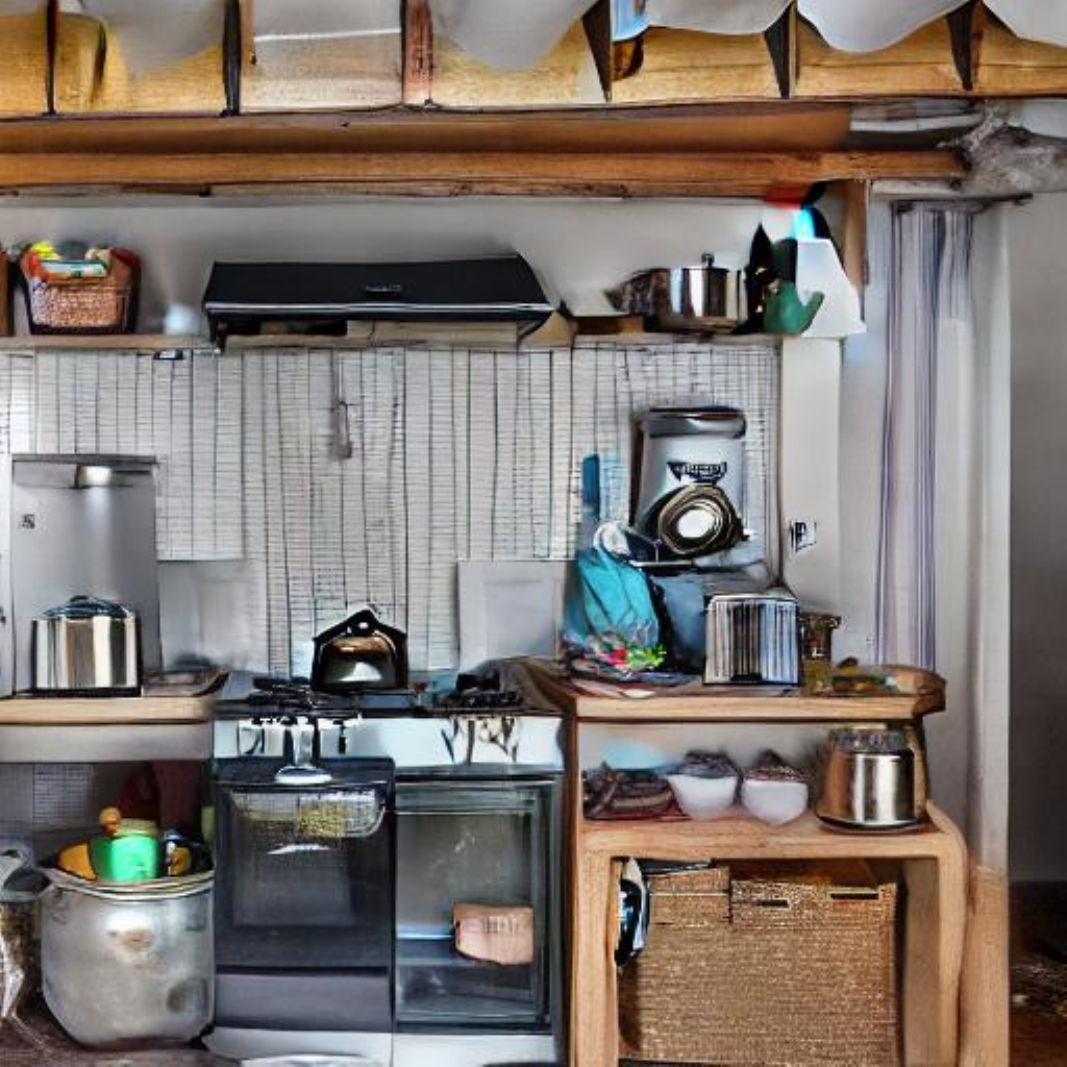} &
    \includegraphics[width=0.10\textwidth]{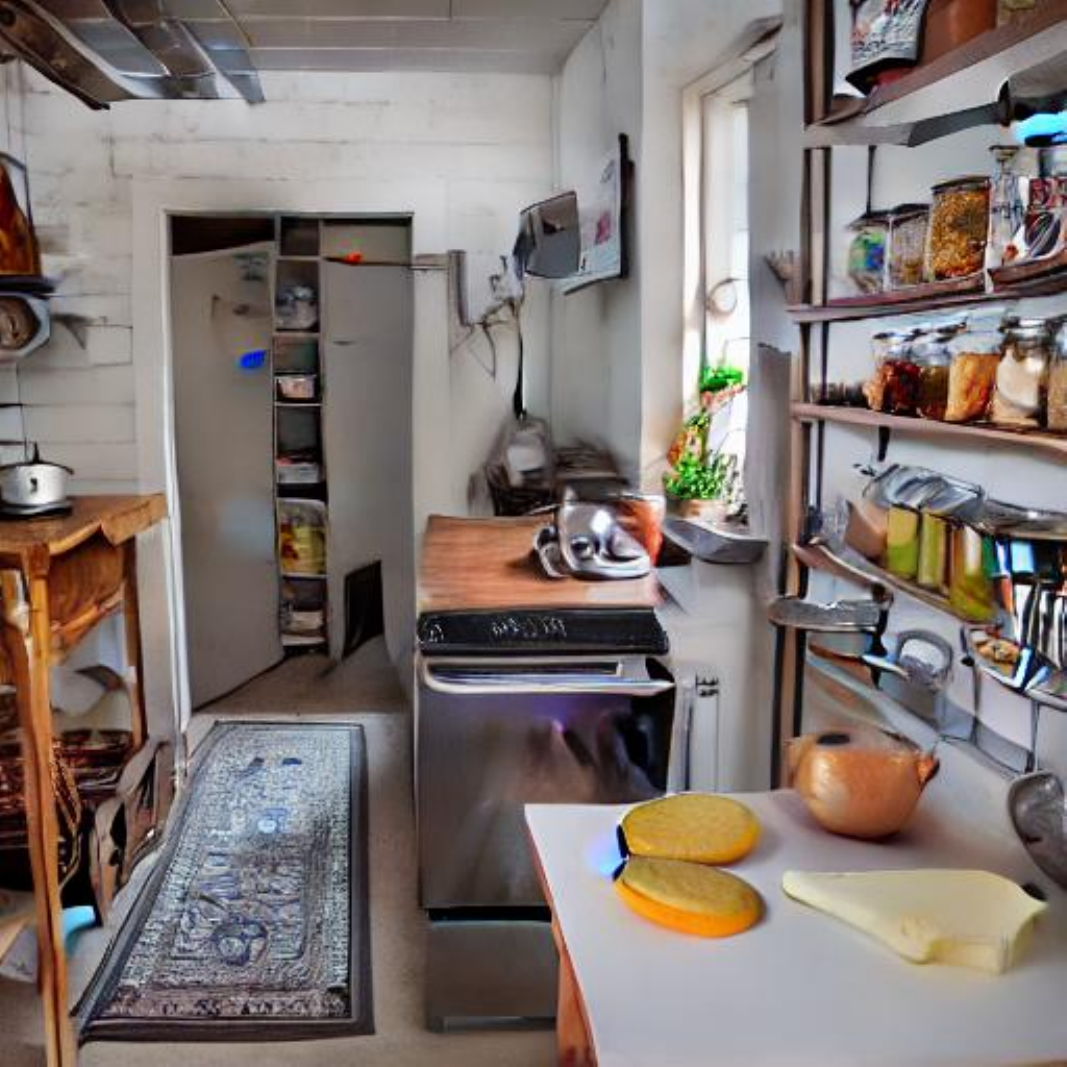} &
    \includegraphics[width=0.10\textwidth]{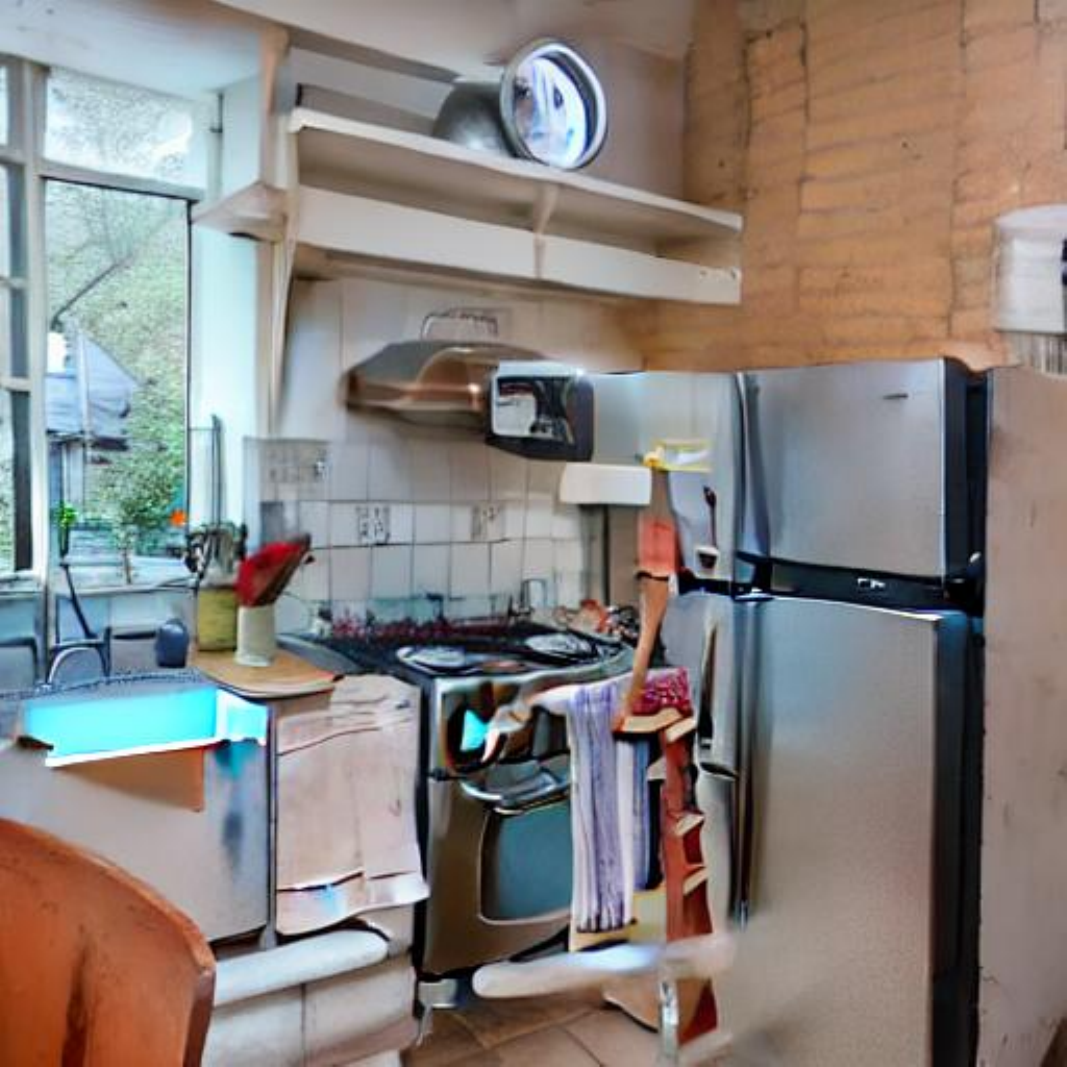} \\
    \includegraphics[width=0.10\textwidth]{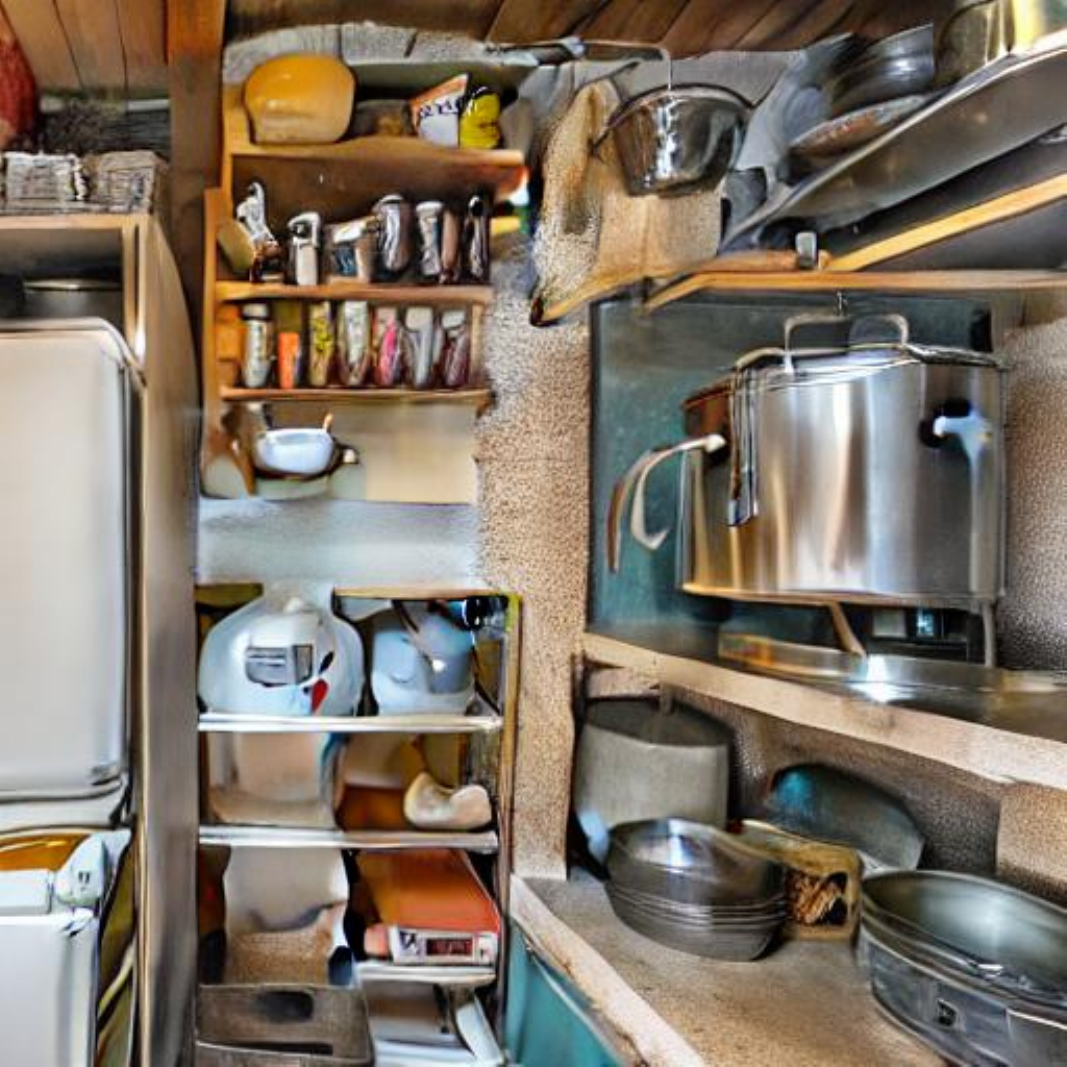} &
    \includegraphics[width=0.10\textwidth]{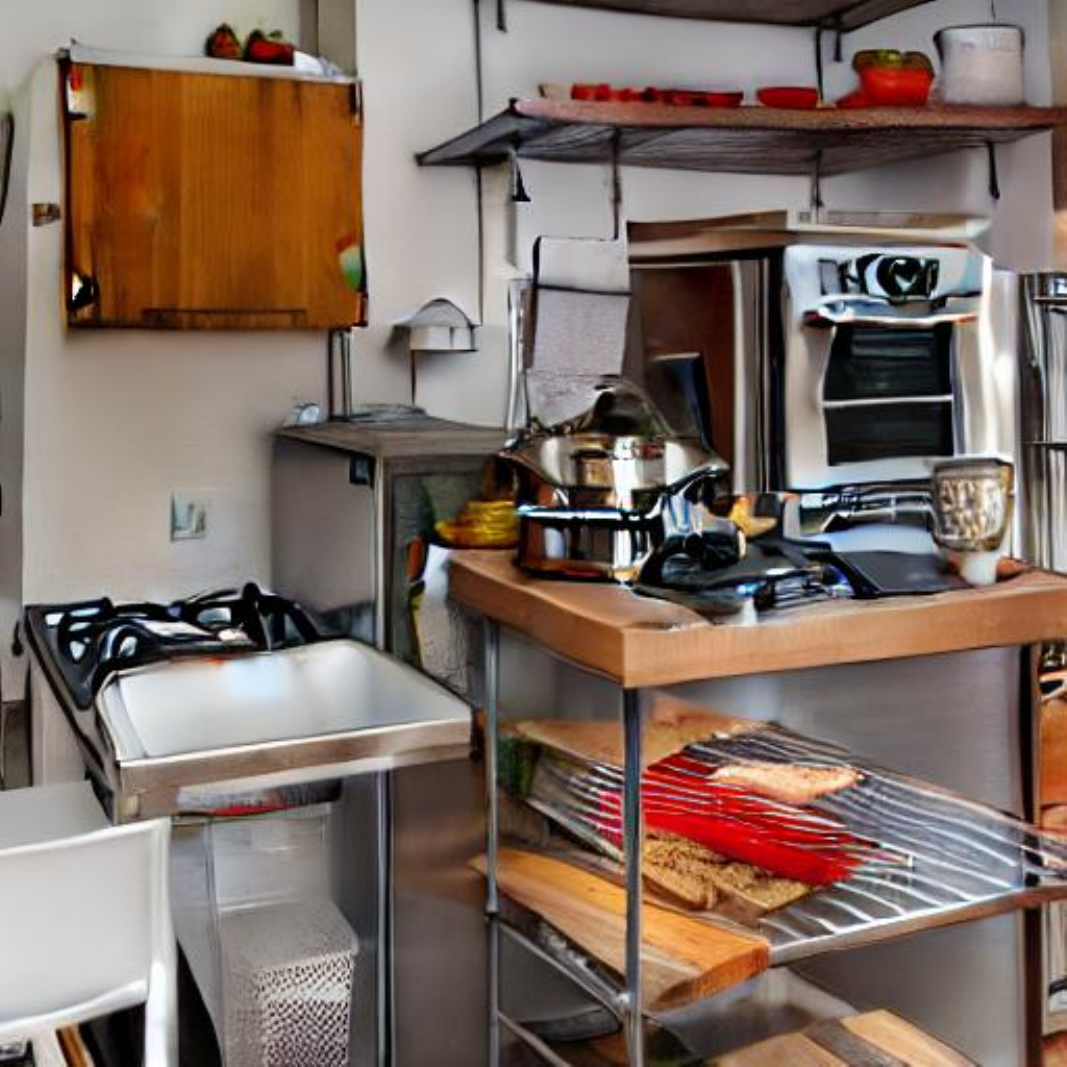} &
    \includegraphics[width=0.10\textwidth]{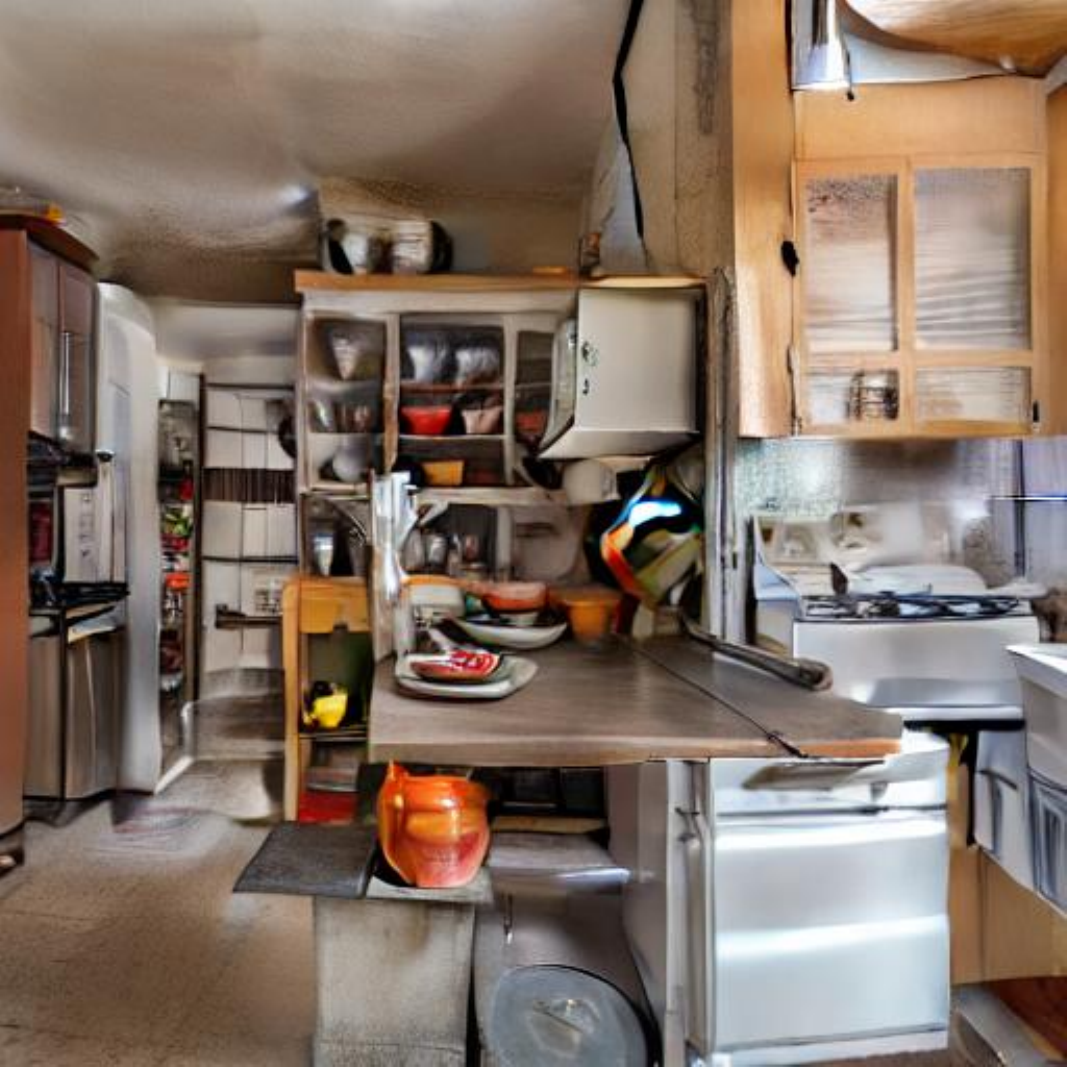} &
    \includegraphics[width=0.10\textwidth]{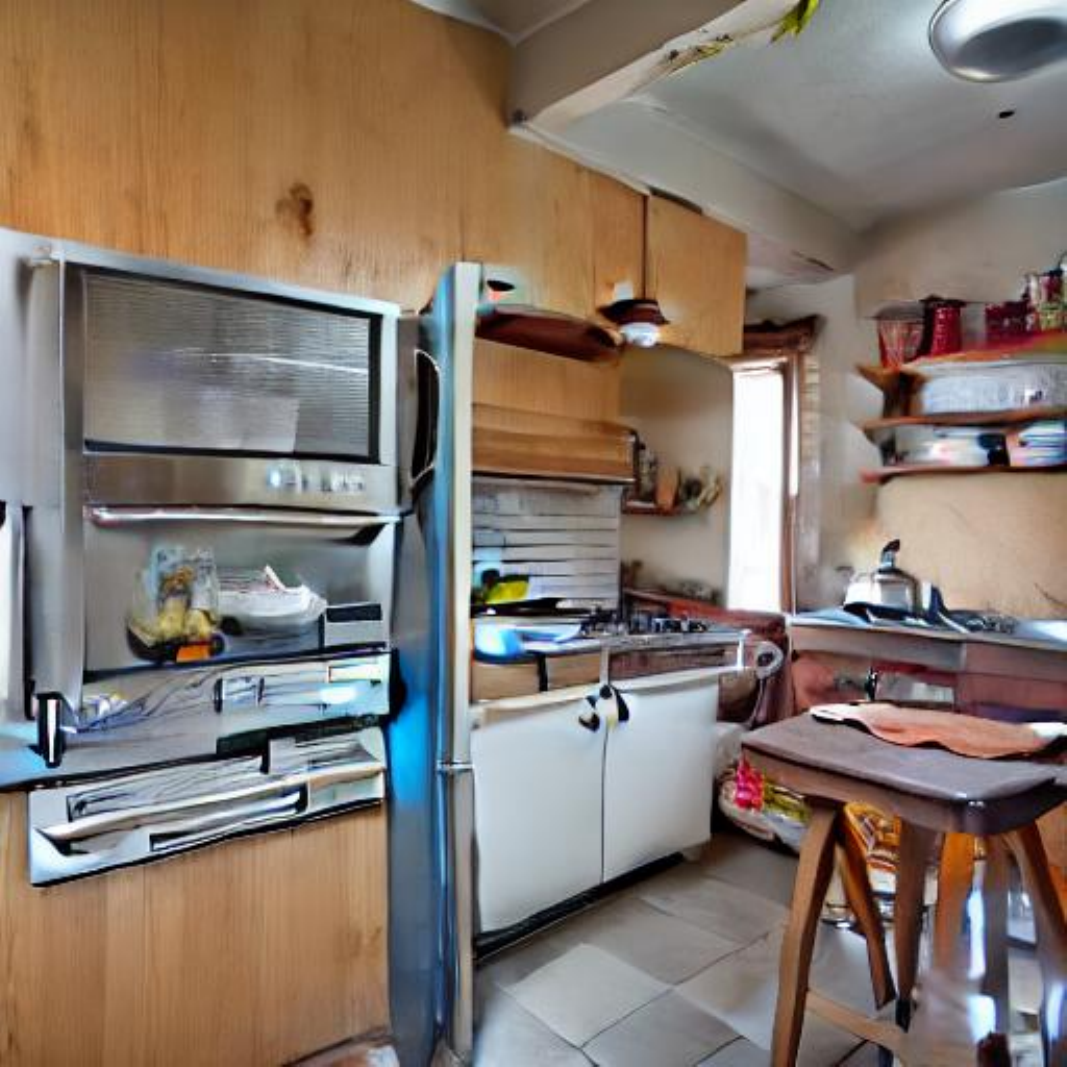} \\
    \includegraphics[width=0.10\textwidth]{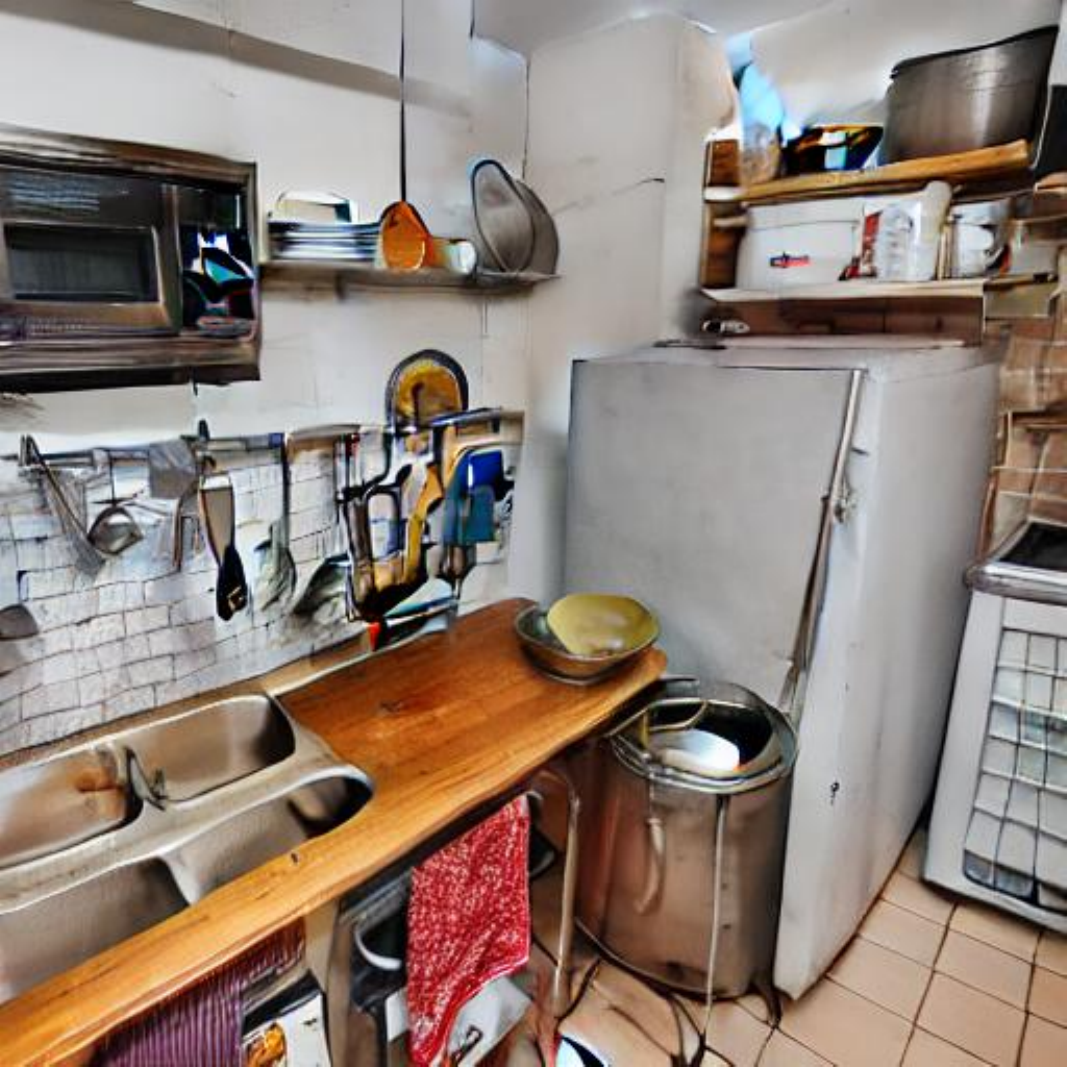} &
    \includegraphics[width=0.10\textwidth]{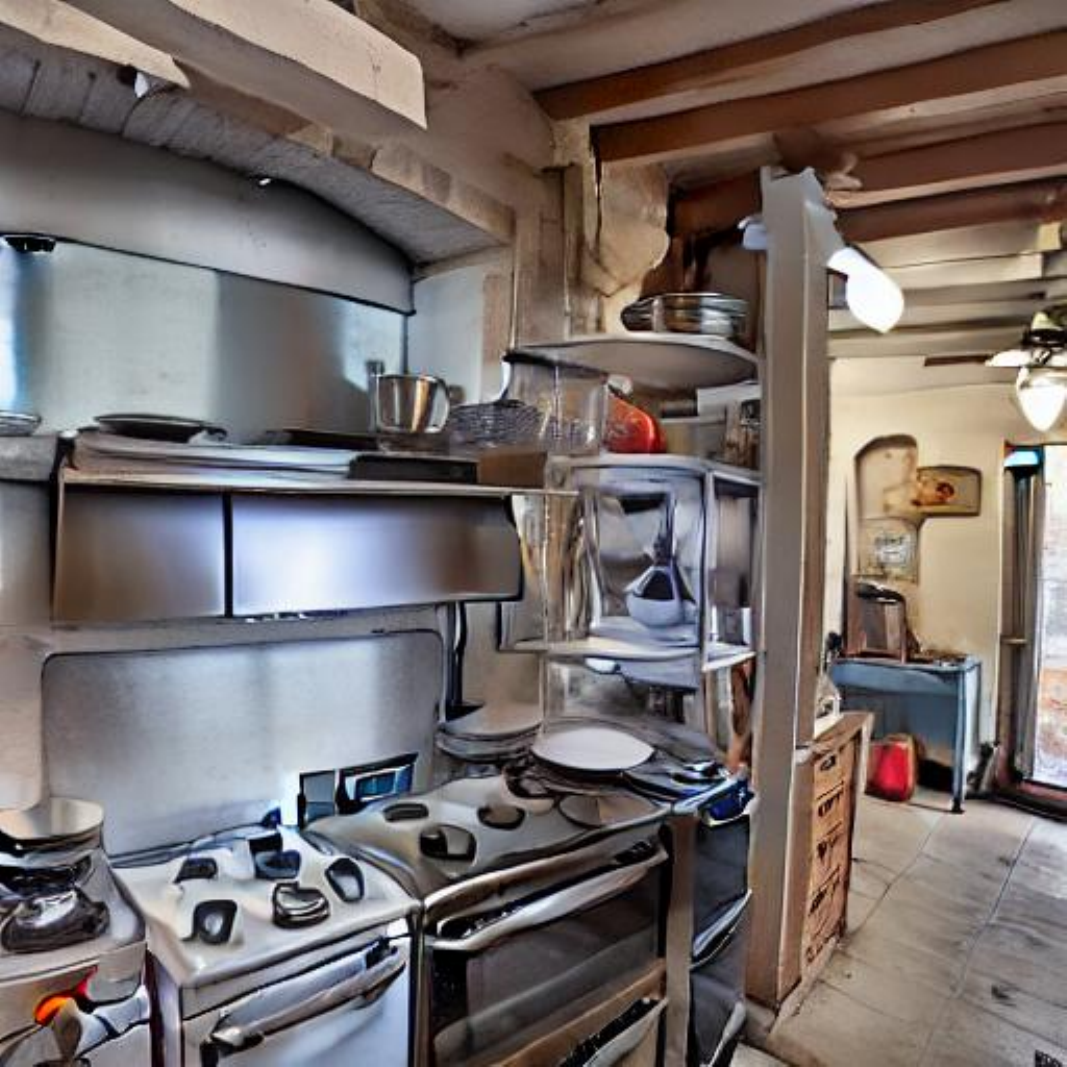} &
    \includegraphics[width=0.10\textwidth]{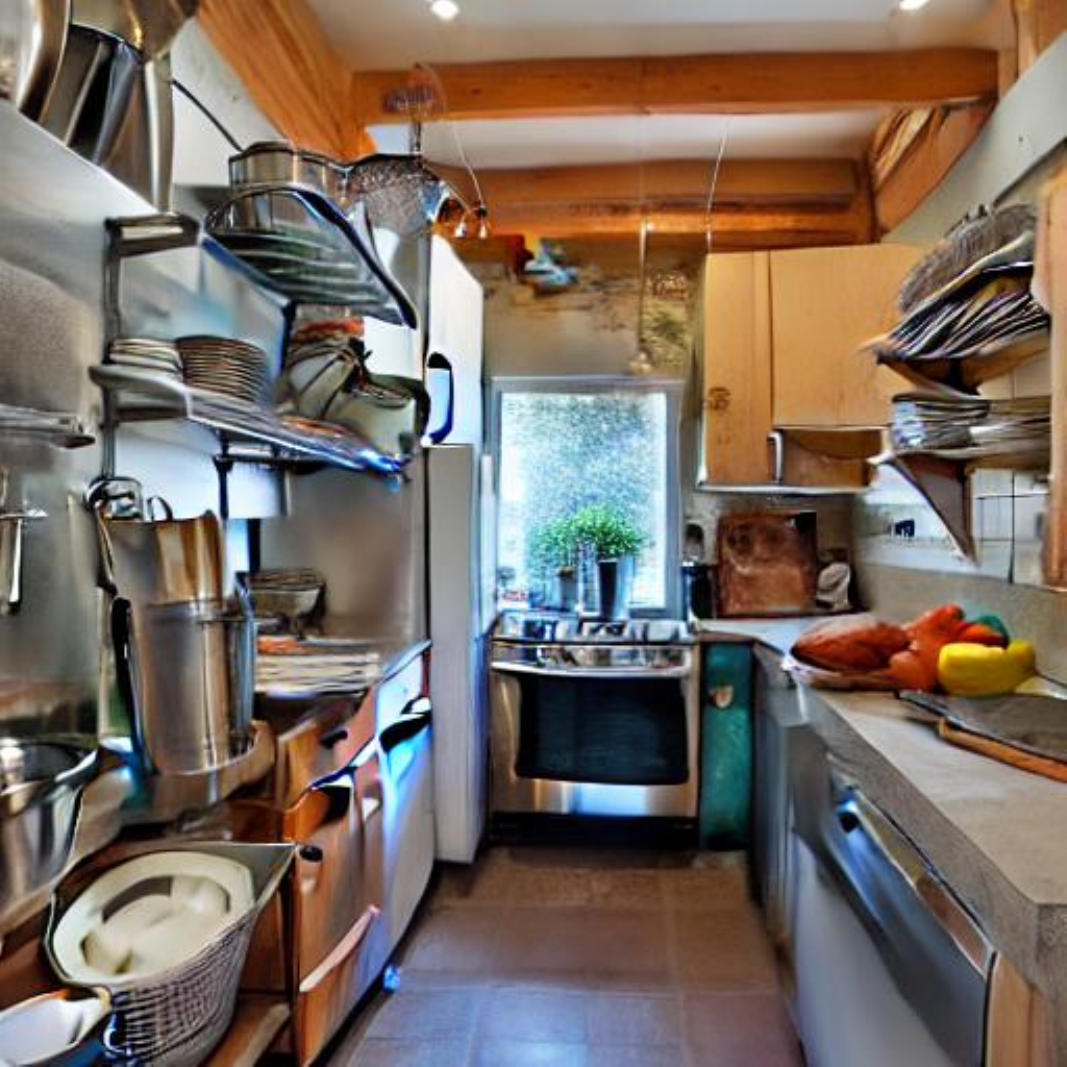} &
    \includegraphics[width=0.10\textwidth]{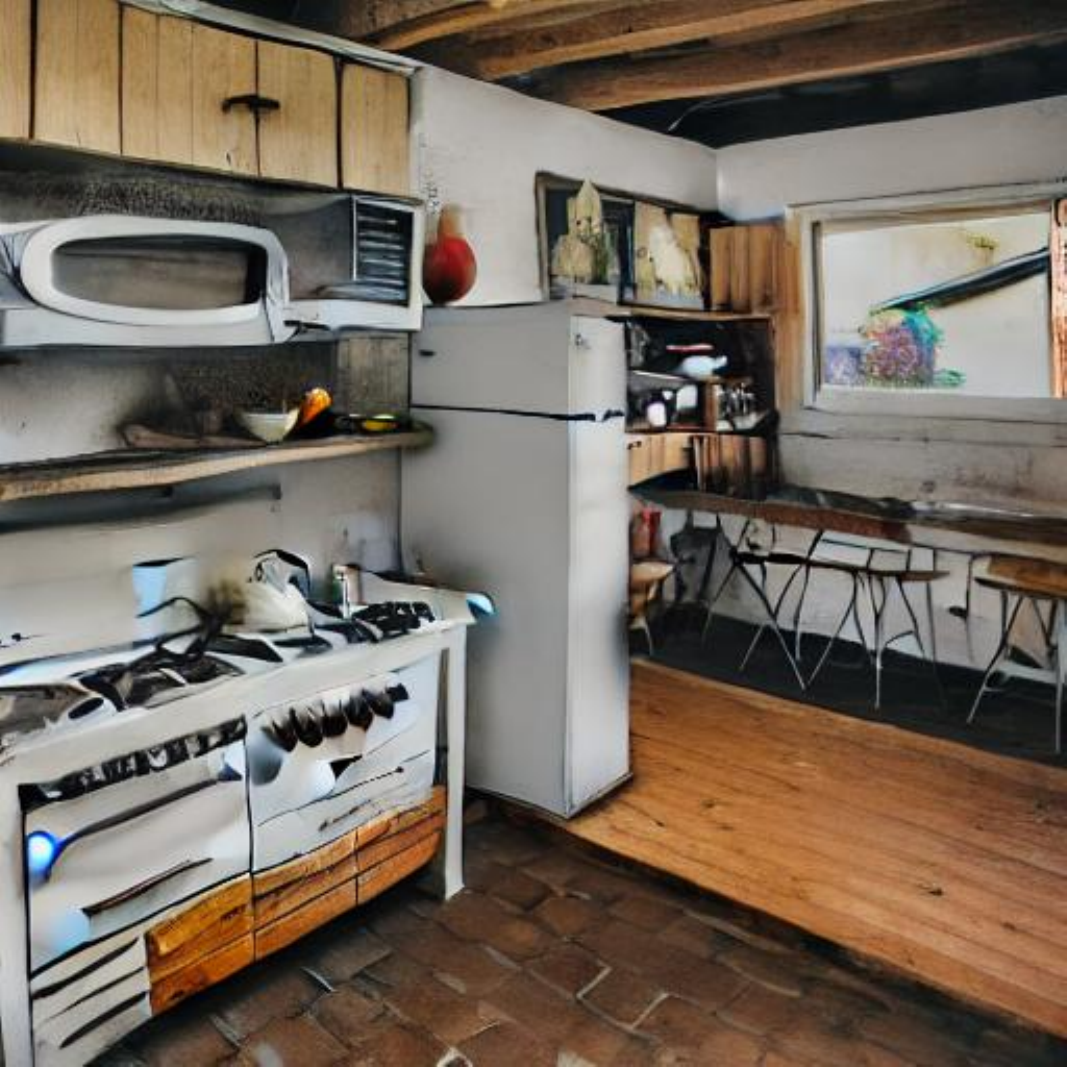} \\
};
\node[below=0pt of bl] {Medium 2};

\matrix (br) [gridmat, right=12pt of bl]
{
    \includegraphics[width=0.10\textwidth]{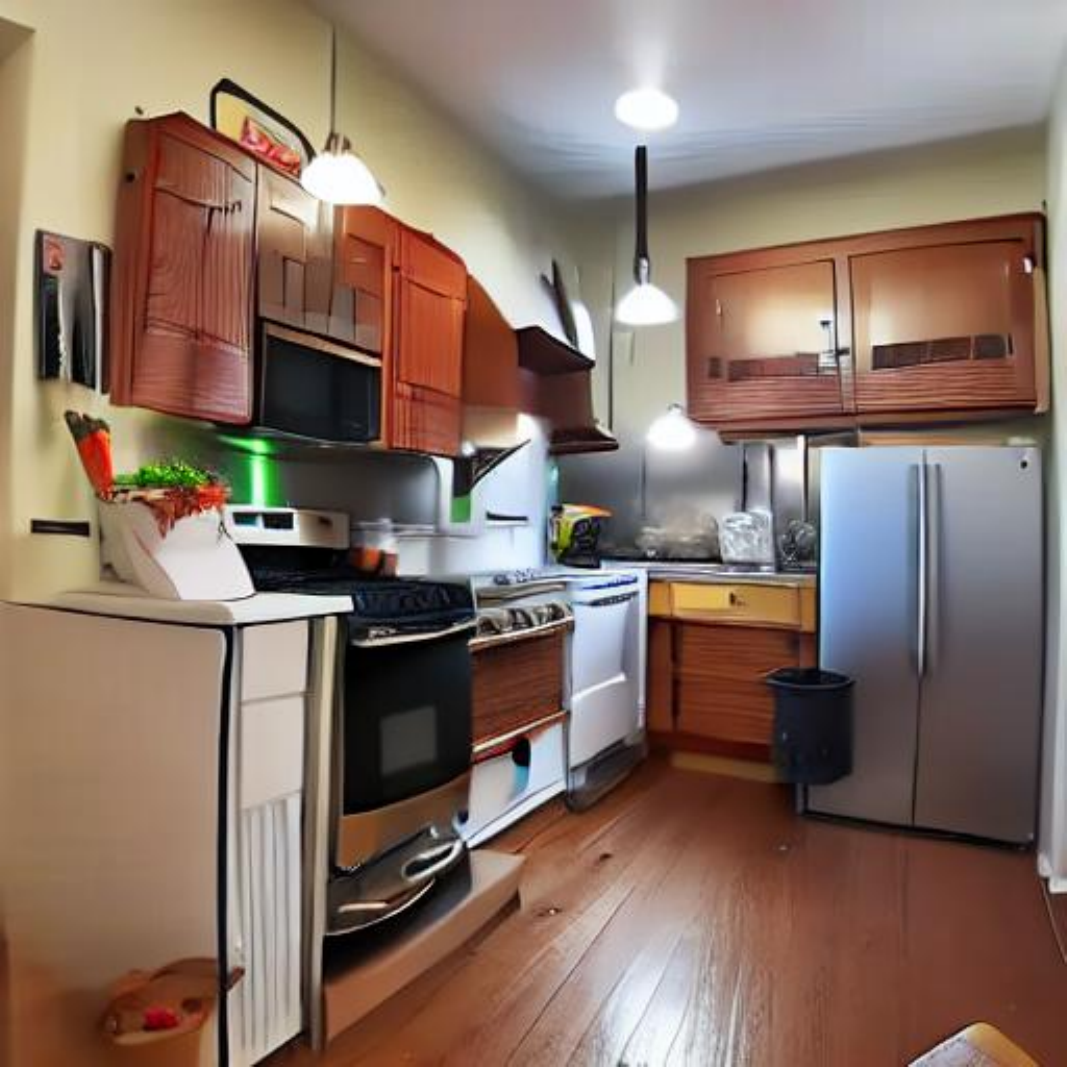} &
    \includegraphics[width=0.10\textwidth]{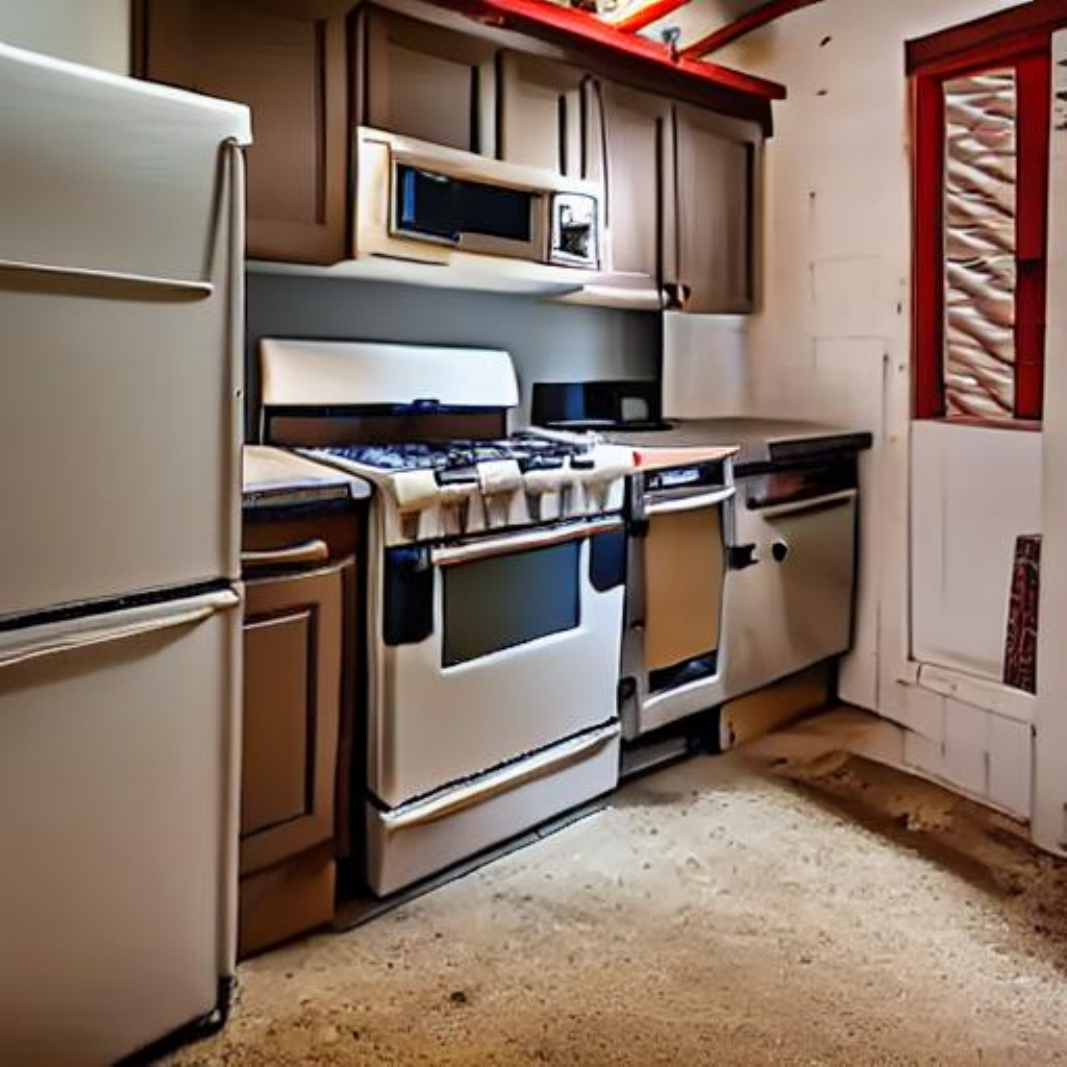} &
    \includegraphics[width=0.10\textwidth]{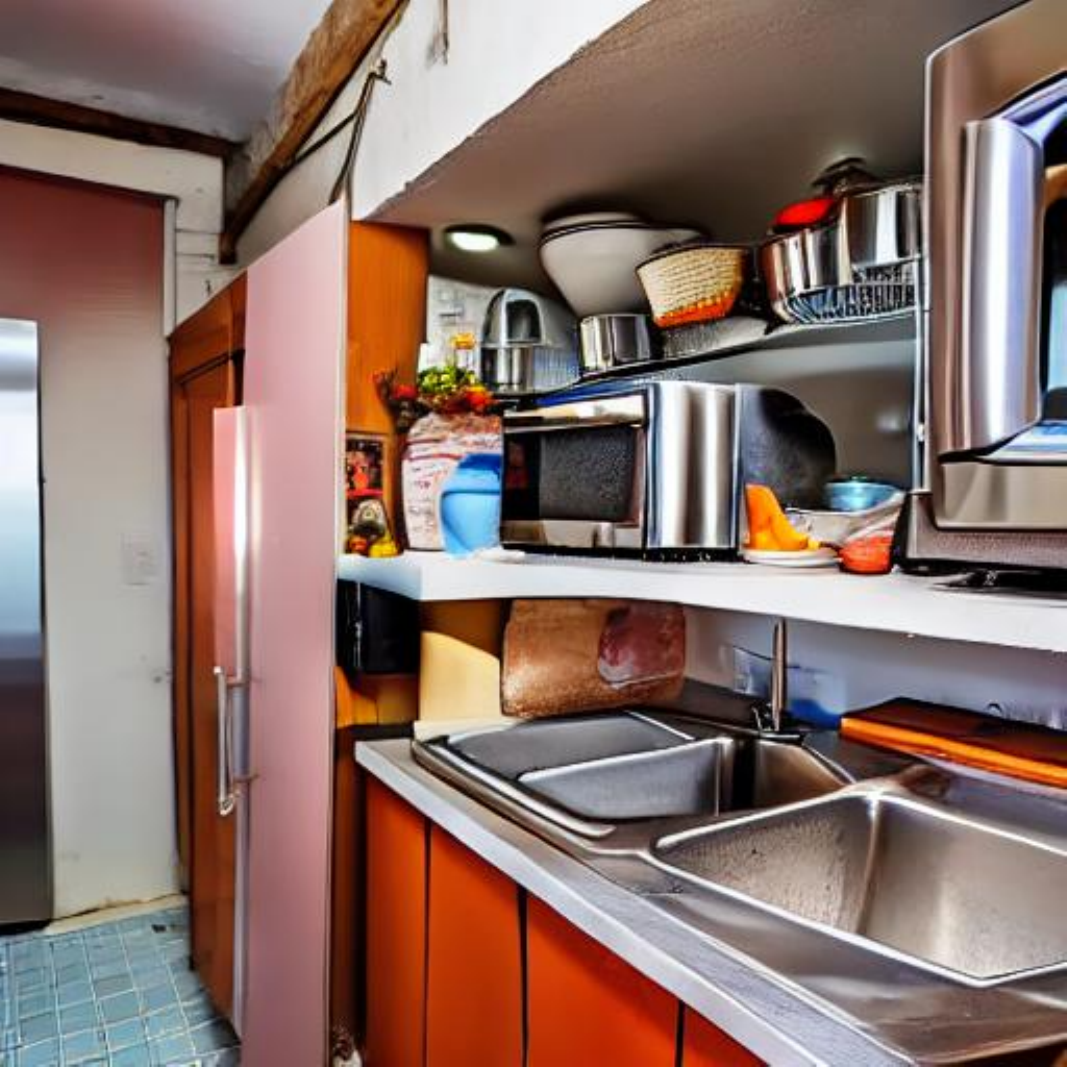} &
    \includegraphics[width=0.10\textwidth]{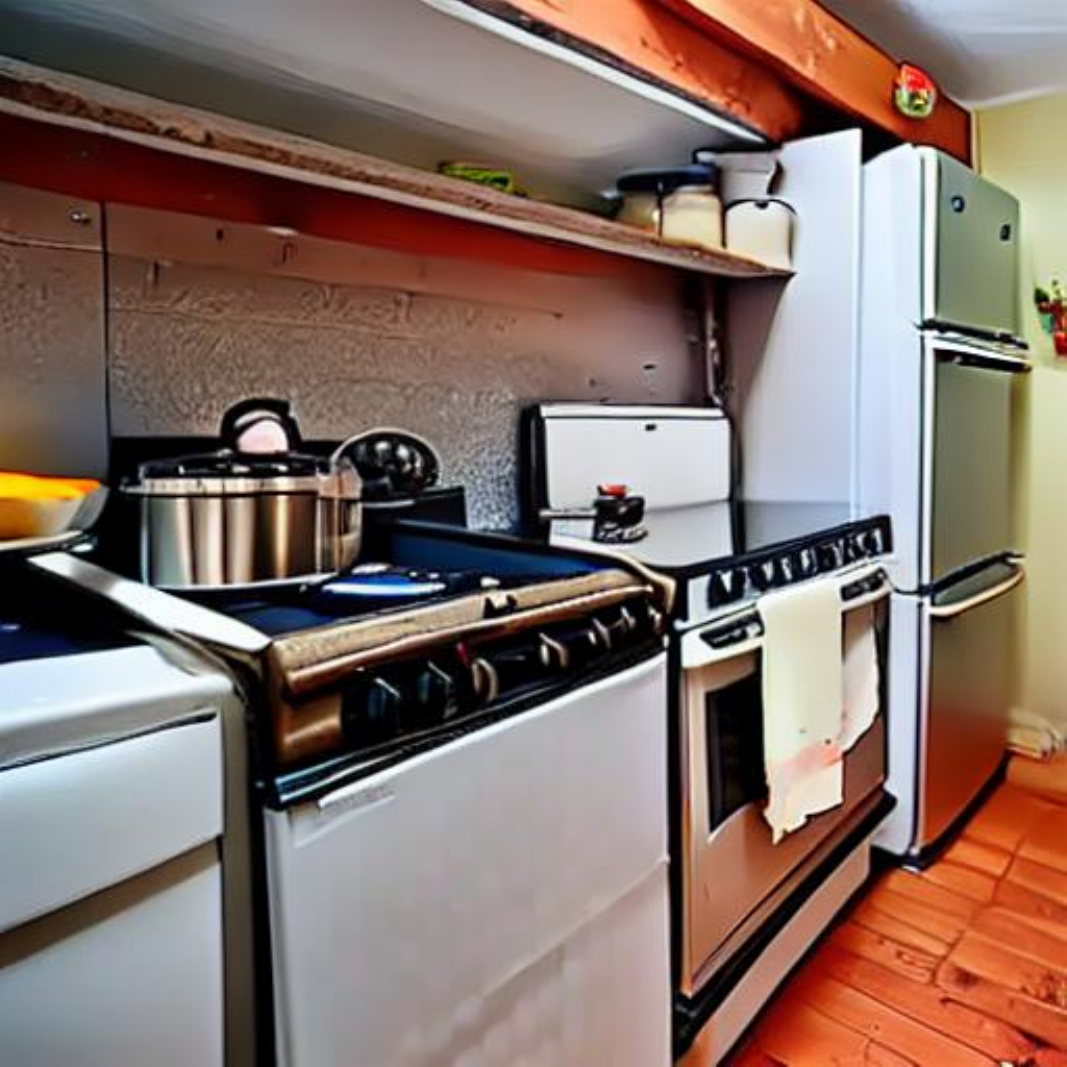} \\
    \includegraphics[width=0.10\textwidth]{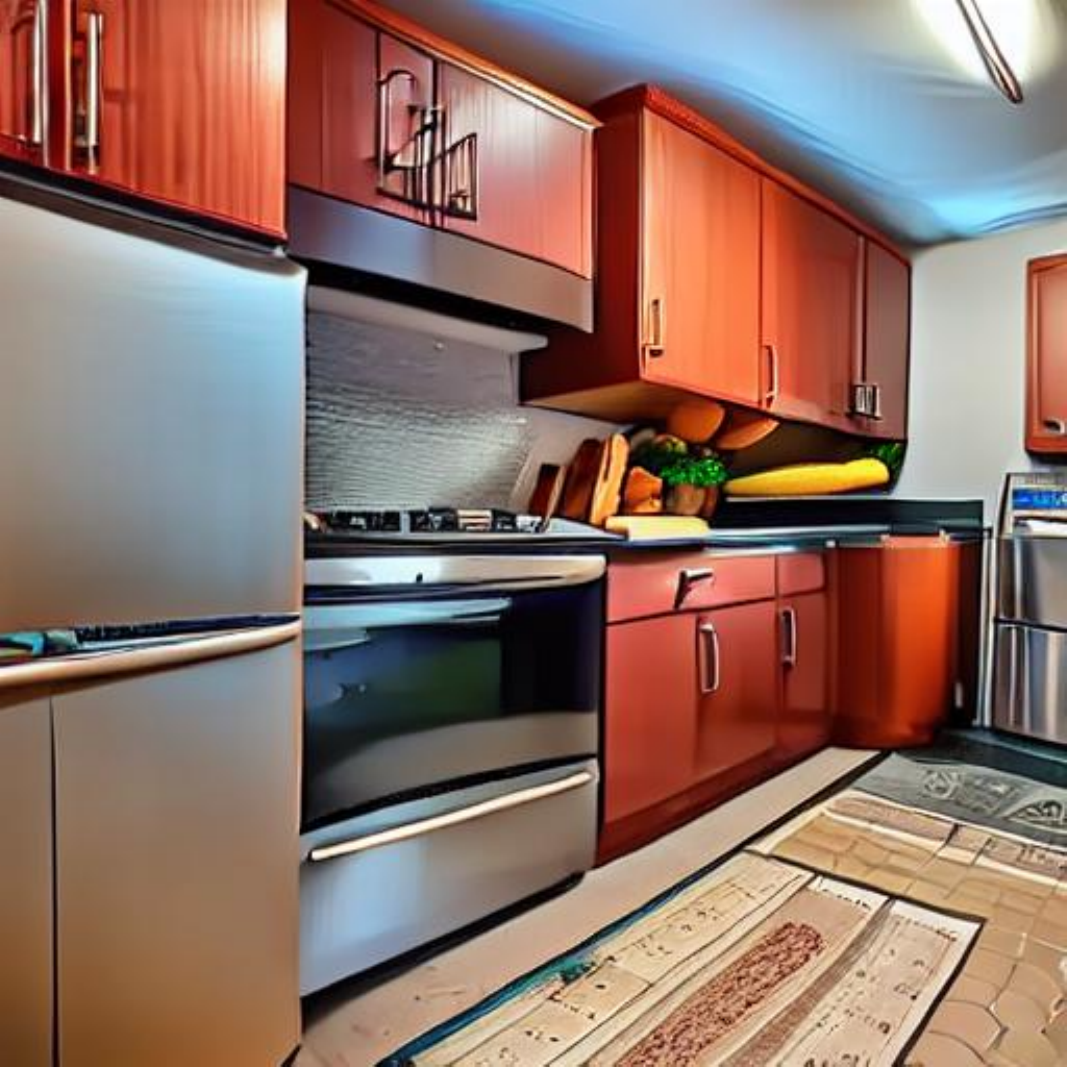} &
    \includegraphics[width=0.10\textwidth]{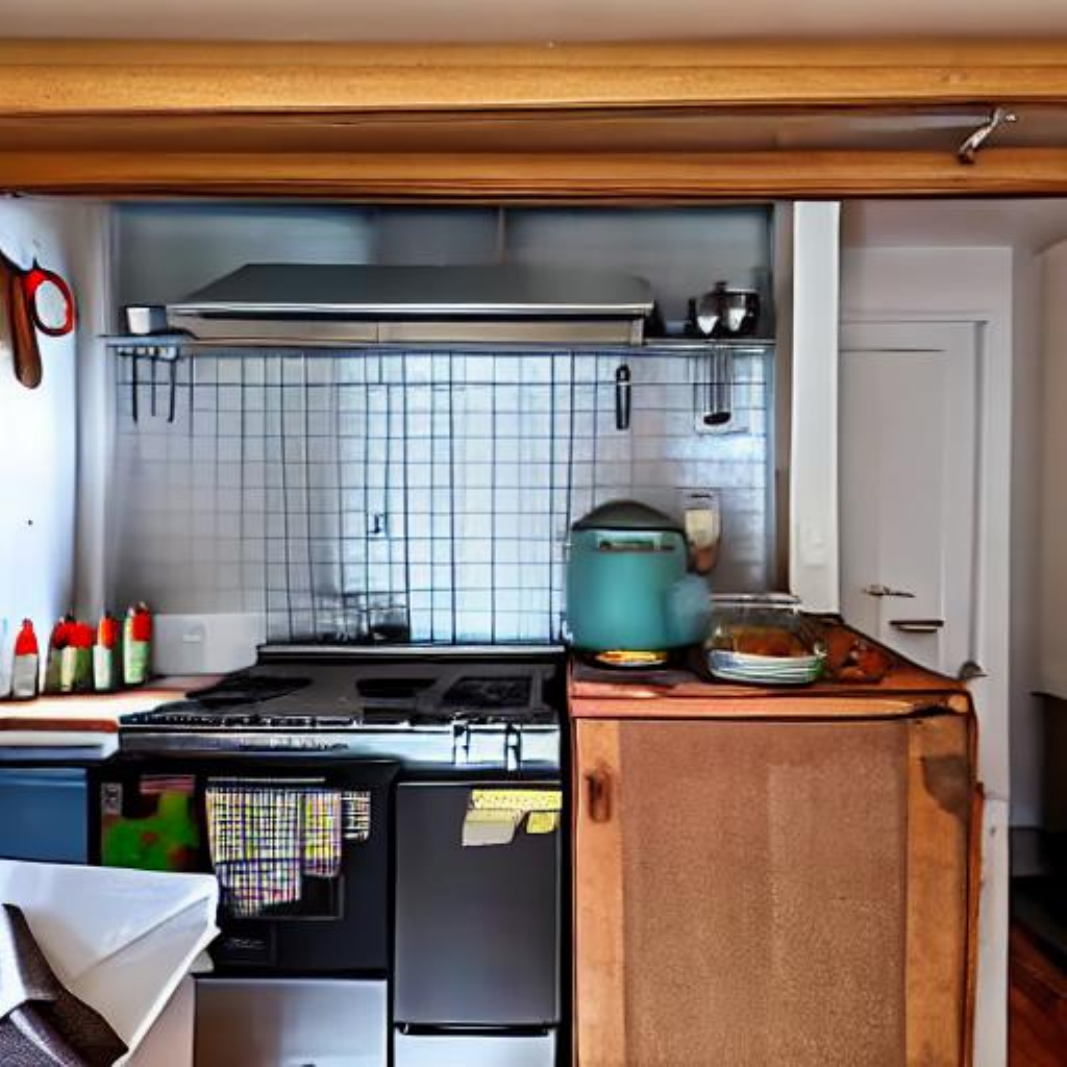} &
    \includegraphics[width=0.10\textwidth]{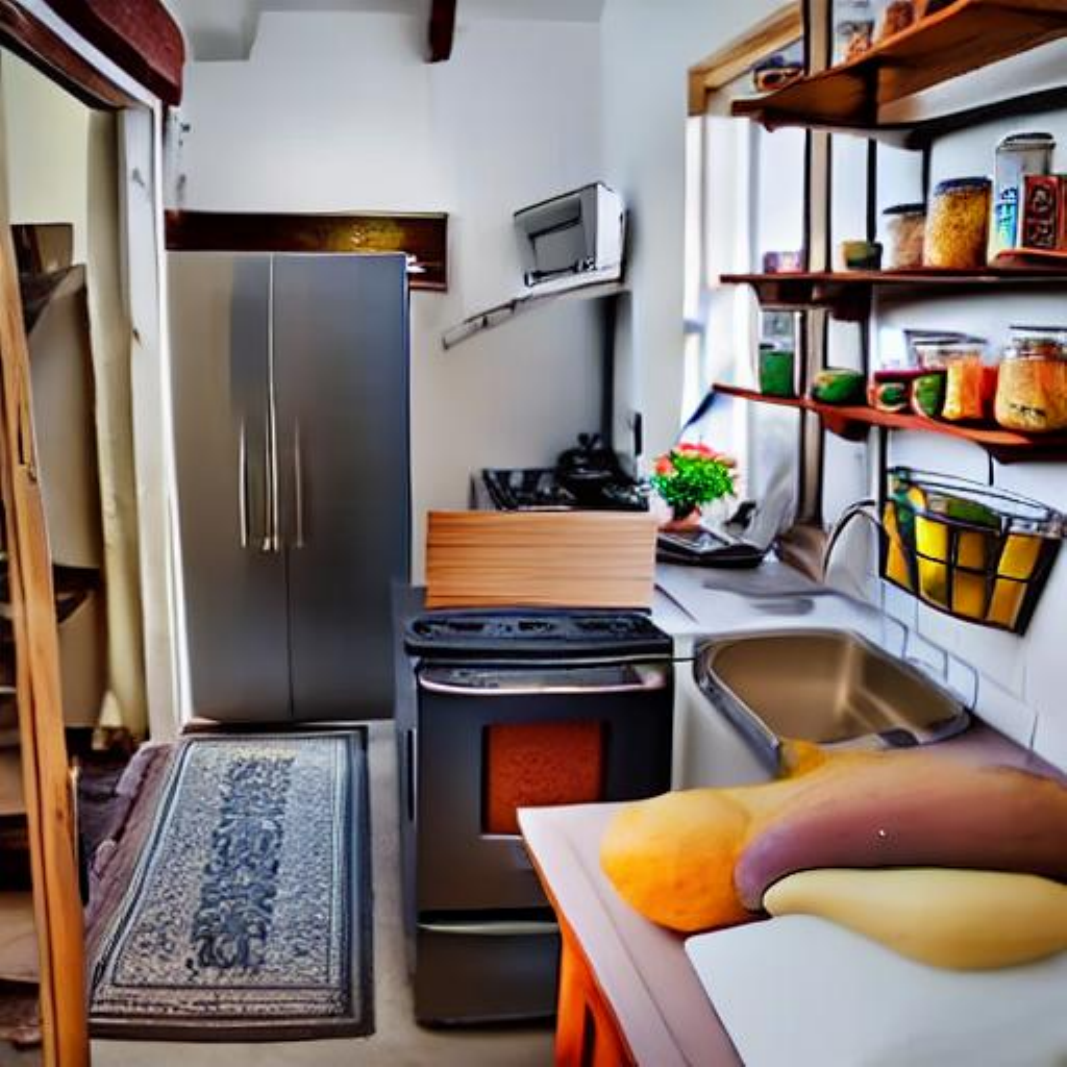} &
    \includegraphics[width=0.10\textwidth]{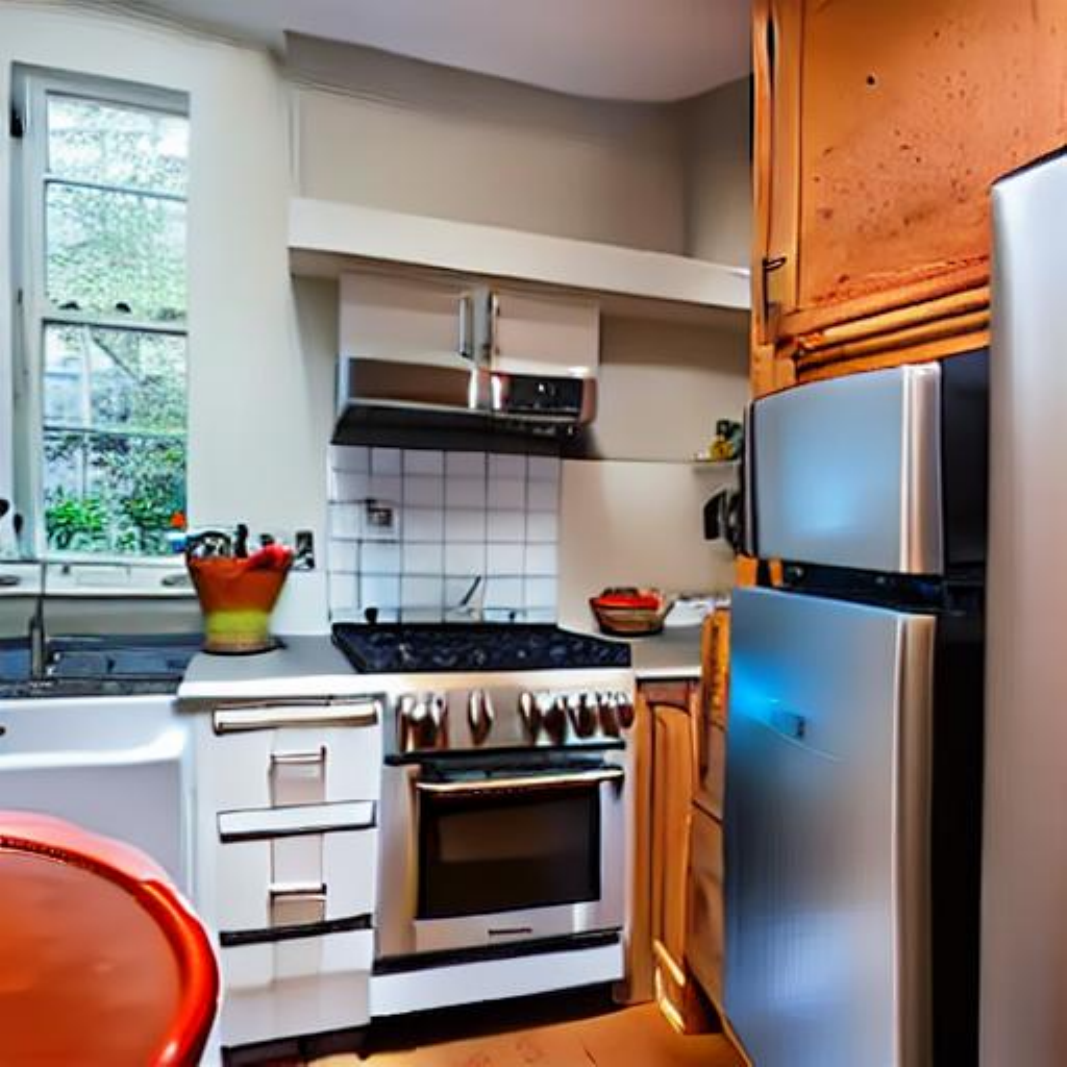} \\
    \includegraphics[width=0.10\textwidth]{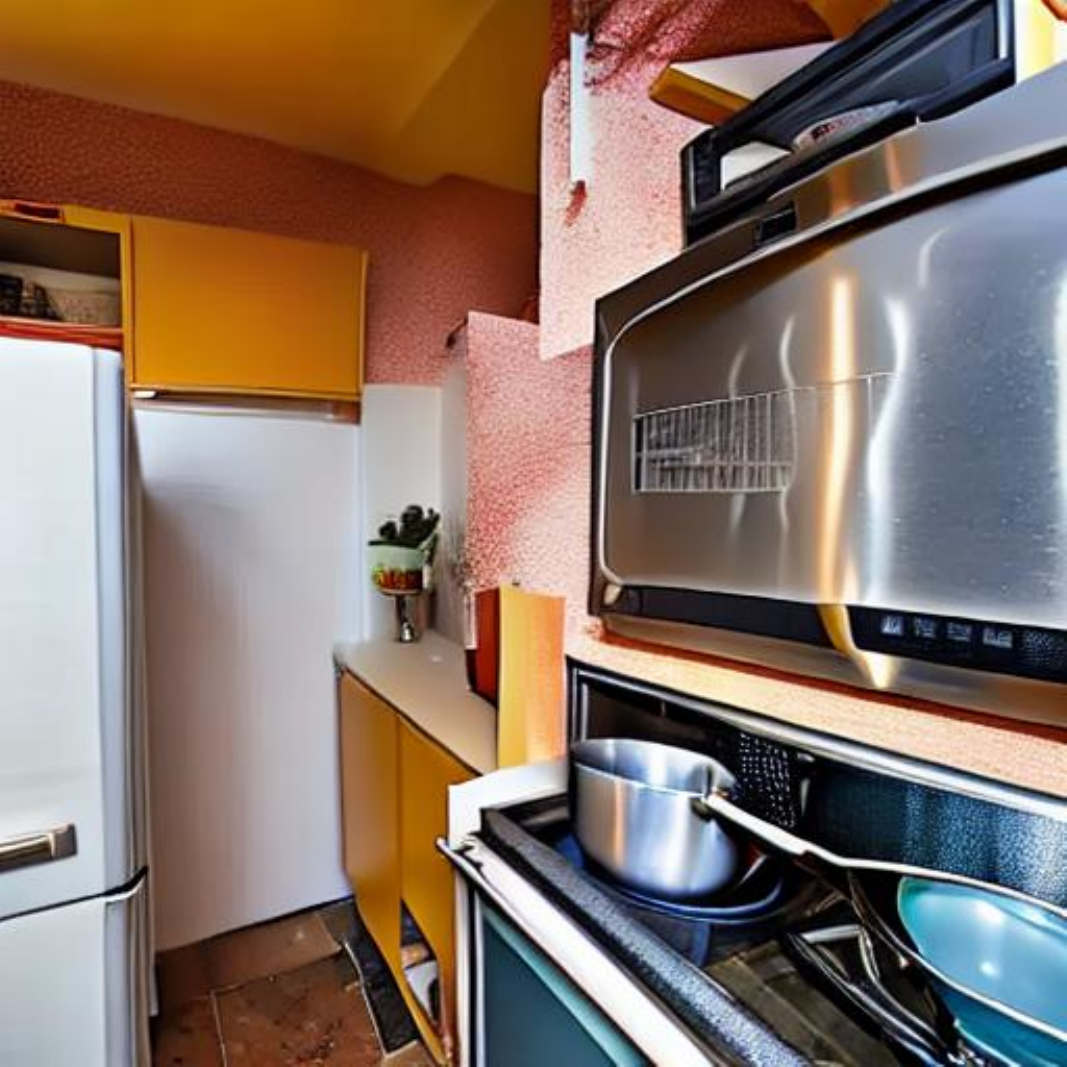} &
    \includegraphics[width=0.10\textwidth]{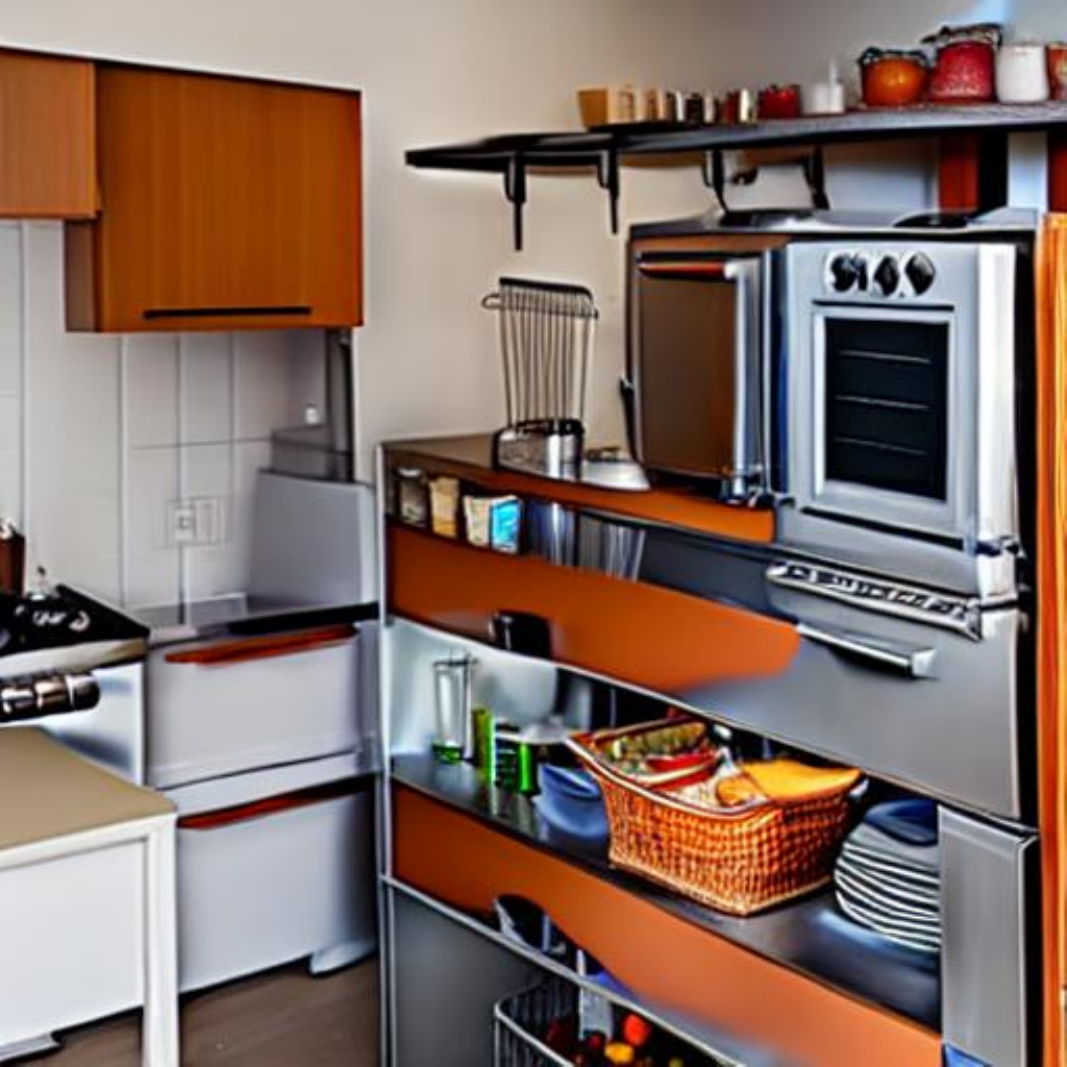} &
    \includegraphics[width=0.10\textwidth]{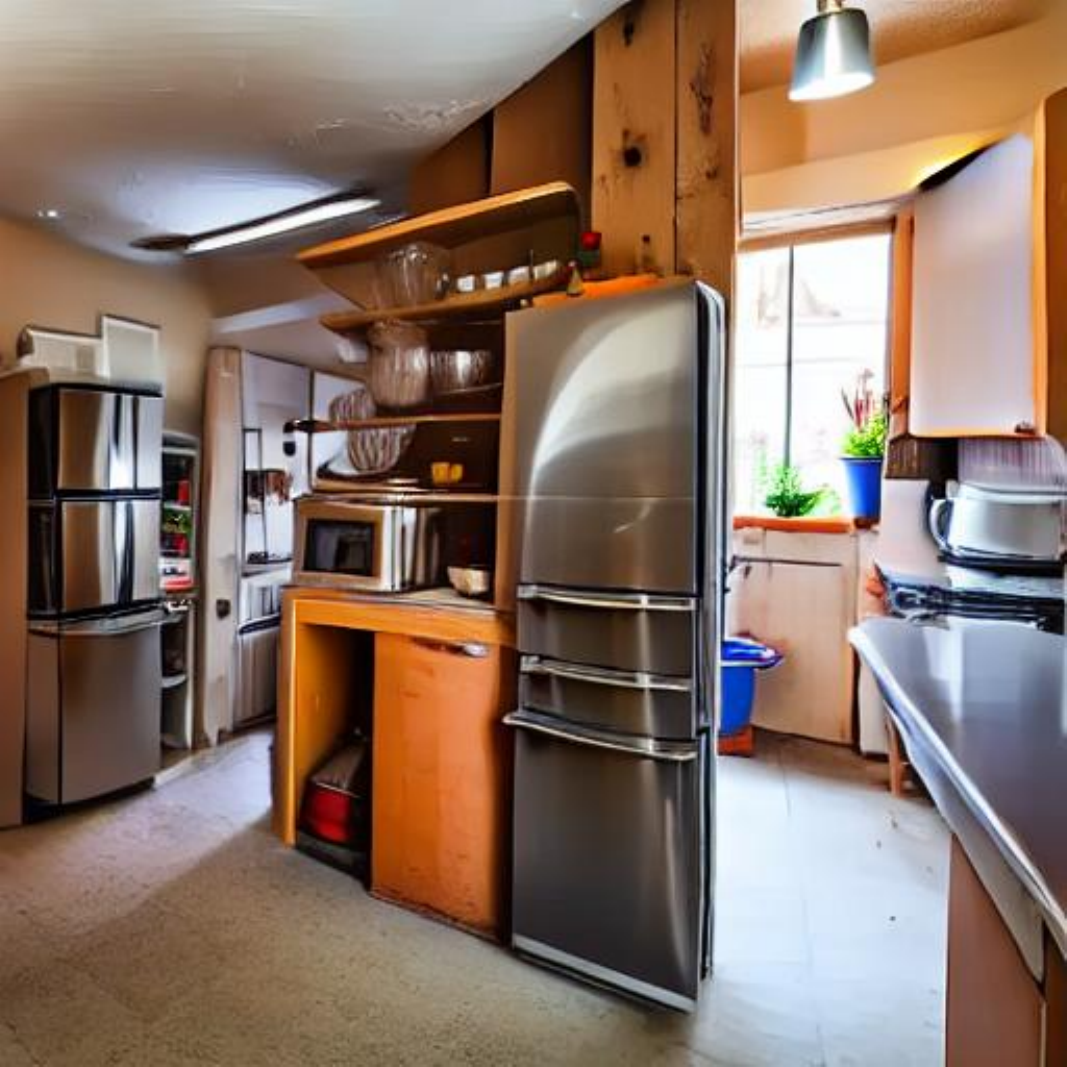} &
    \includegraphics[width=0.10\textwidth]{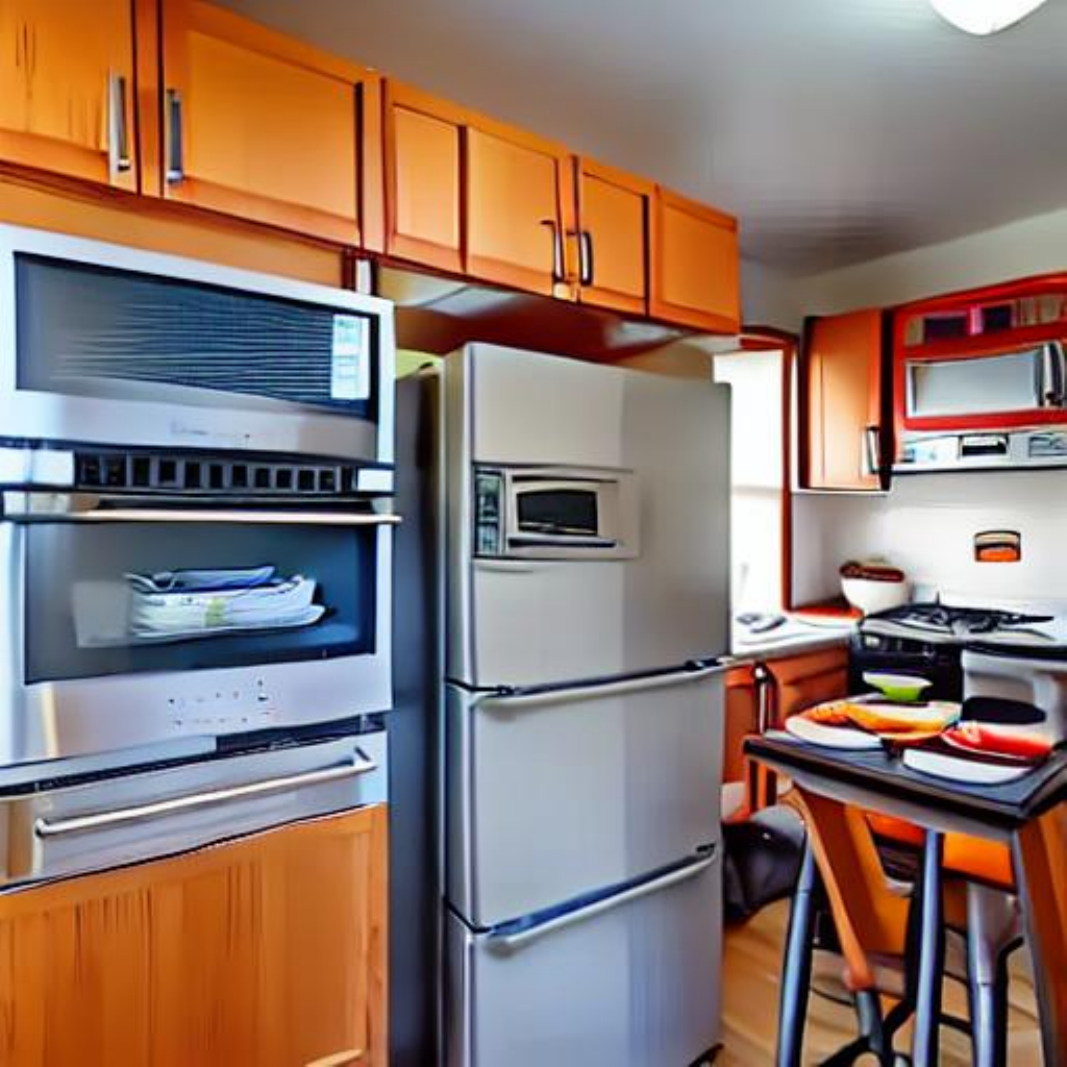} \\
    \includegraphics[width=0.10\textwidth]{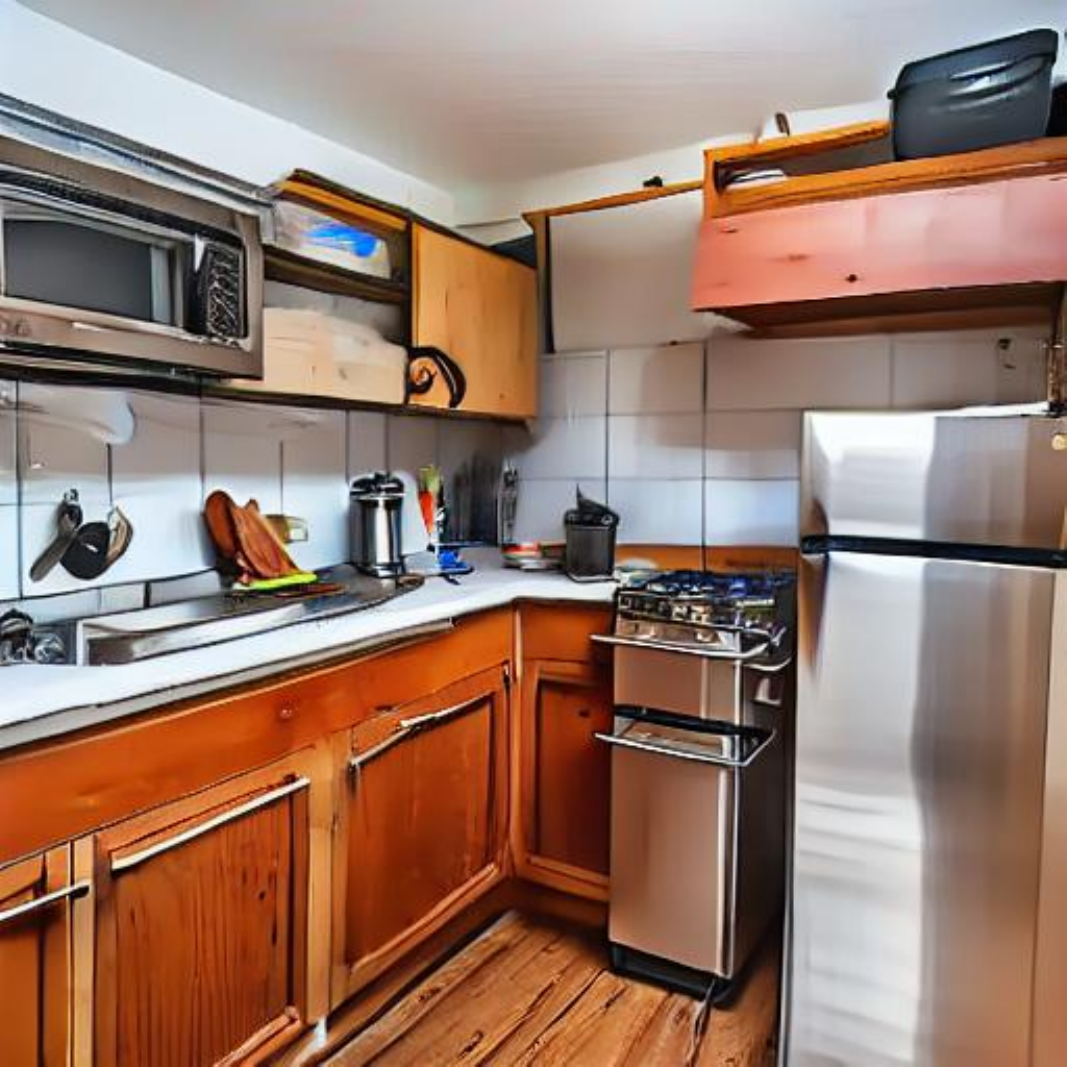} &
    \includegraphics[width=0.10\textwidth]{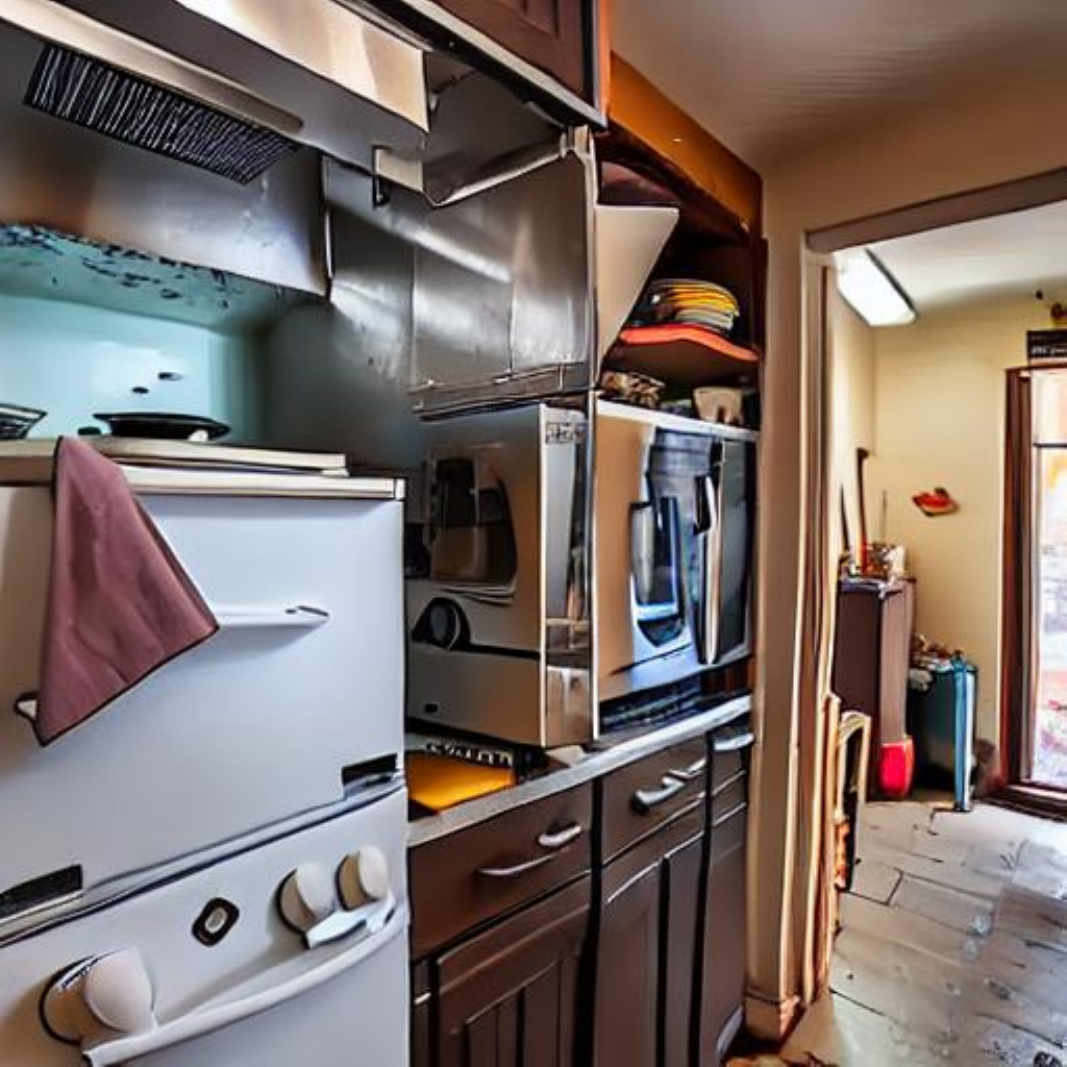} &
    \includegraphics[width=0.10\textwidth]{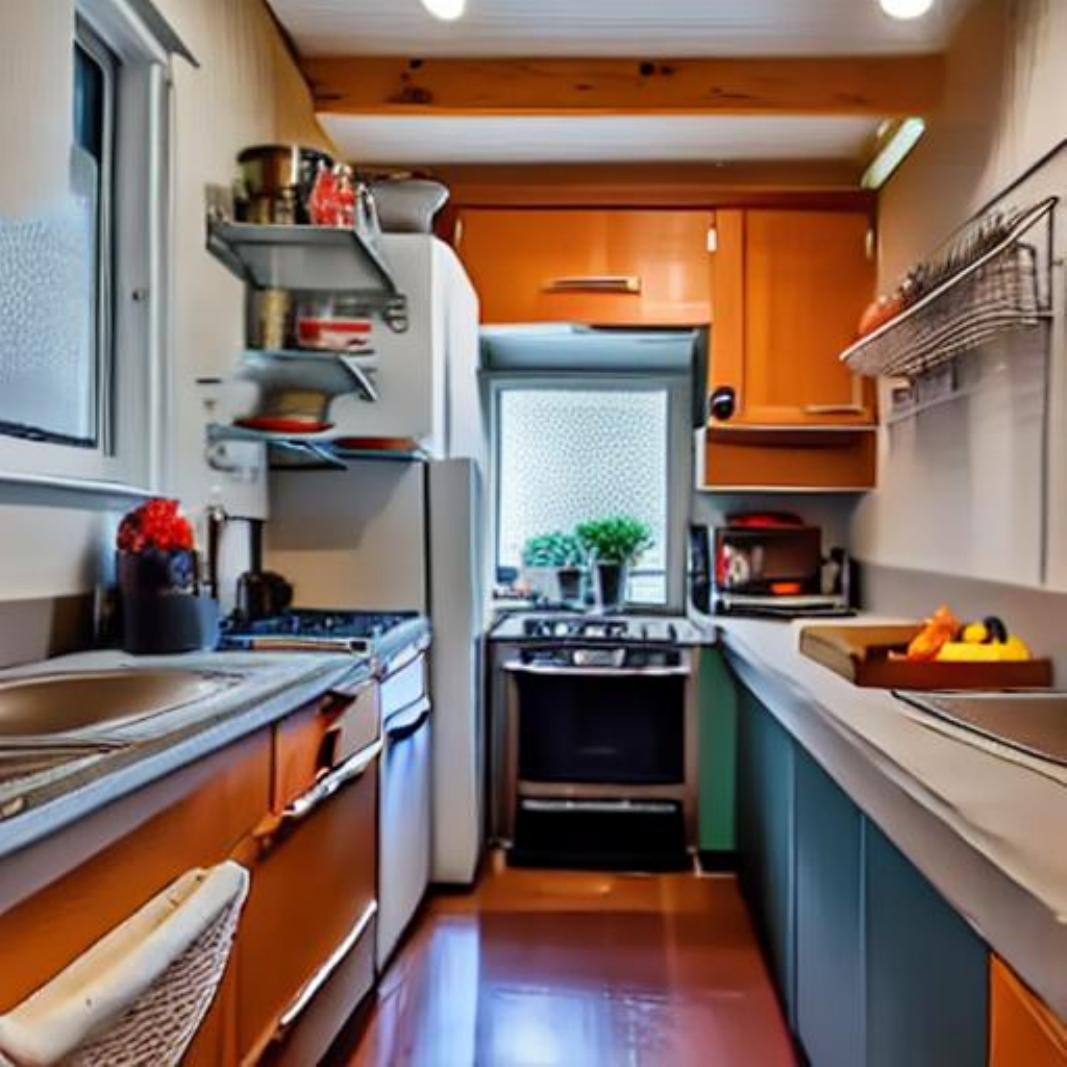} &
    \includegraphics[width=0.10\textwidth]{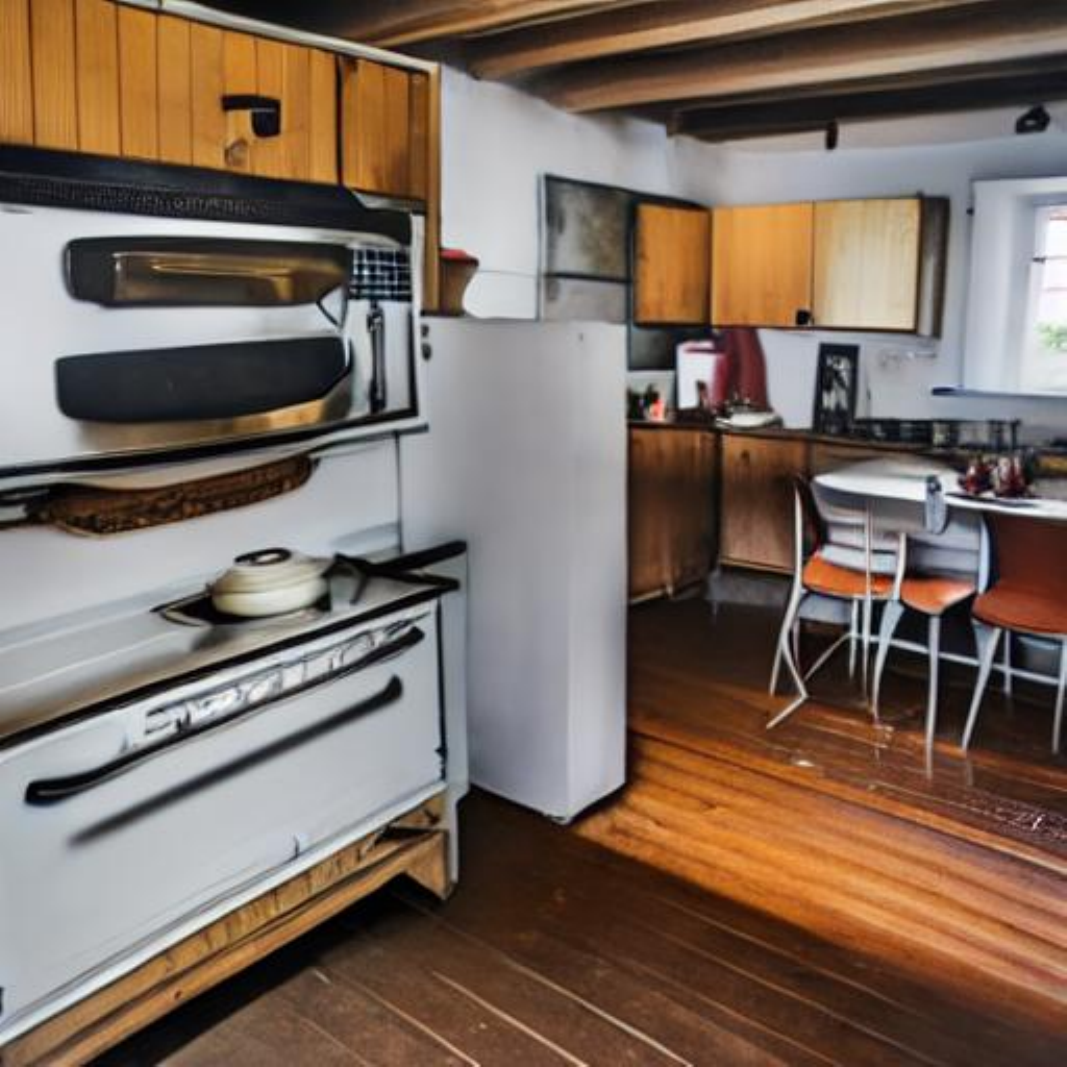} \\
};
\node[below=0pt of br] {DG-CFG (Ours)};

\end{tikzpicture}
\caption{\textbf{Qualitative ablation on Stable Diffusion~1.5 at $\bar{\omega}=23$.}
Each quadrant shows 16 samples generated from the same prompt with different random seeds. \textbf{Top-left:} constant CFG. \textbf{Top-right:} Medium 1 (time balancing). \textbf{Bottom-left:} Medium 2 (time balancing plus signal-content weighting). \textbf{Bottom-right:} DG-CFG (all three components). The prompt is \textit{a narrow kitchen filled with appliances and cooking utensils}.}
\label{fig:ablation}
\end{figure}

\begin{figure}[p]
\centering
\small
\begin{tikzpicture}[
    subimg/.style={
        inner sep=0pt, outer sep=0pt,
        anchor=center
    },
    gridmat/.style={
        matrix of nodes,
        nodes=subimg,
        inner sep=0pt, outer sep=0pt,
        column sep=1pt,
        row sep=1pt
    },
    outer/.style={
        matrix of nodes,
        nodes={inner sep=0pt,outer sep=0pt},
        inner sep=0pt, outer sep=0pt,
        column sep=5pt,
        row sep=5pt
    }
]

\matrix (tl) [gridmat]
{
    \includegraphics[width=0.10\textwidth]{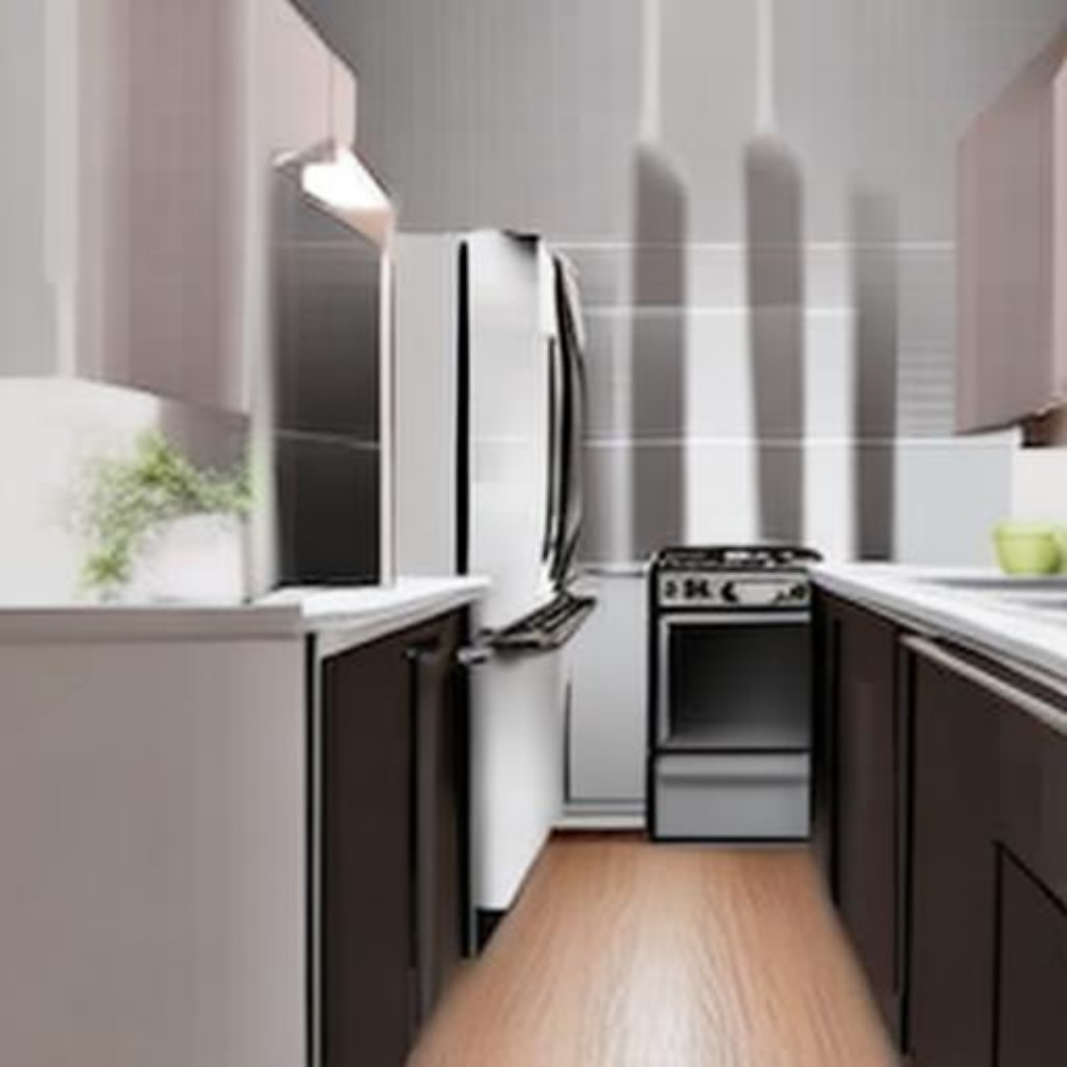} &
    \includegraphics[width=0.10\textwidth]{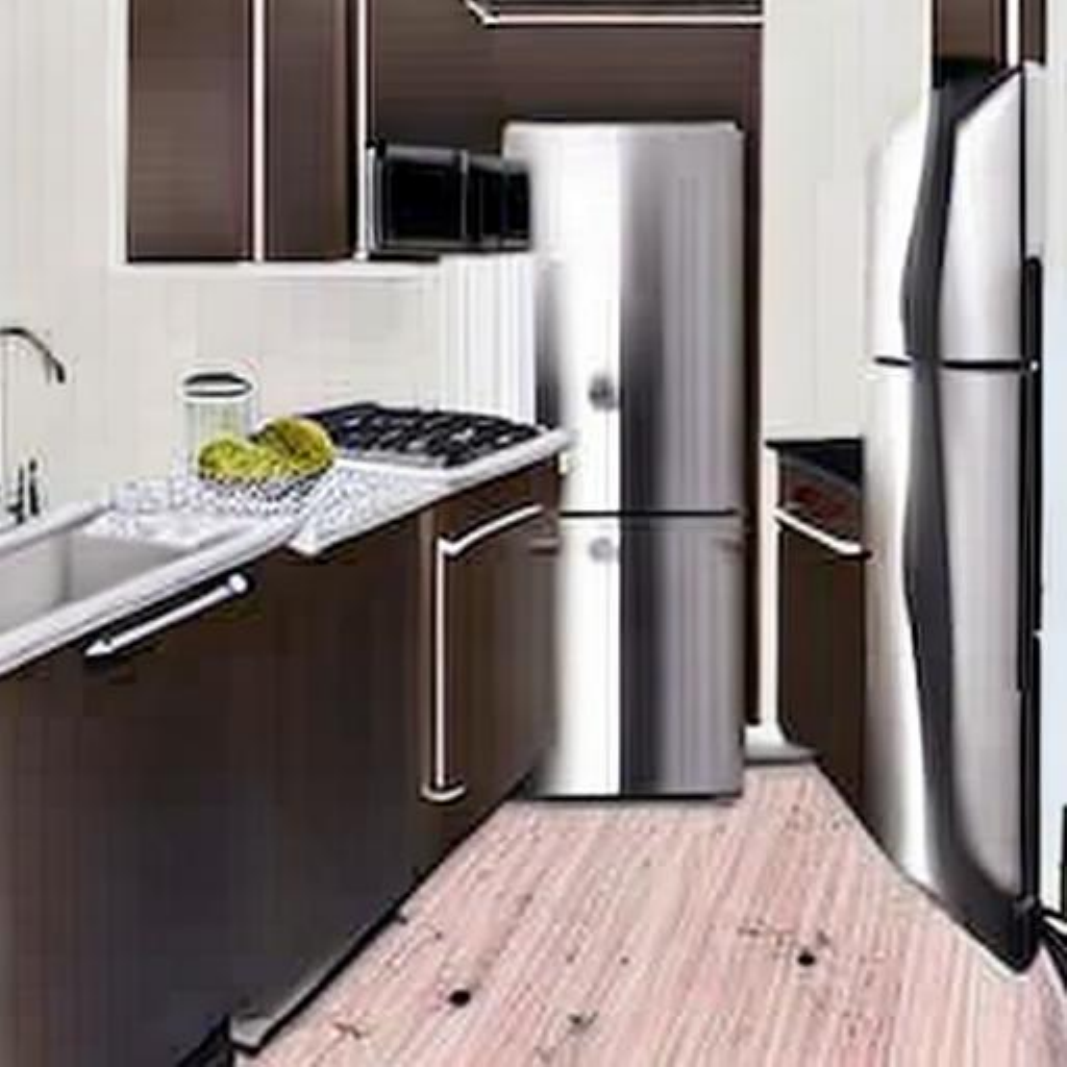} &
    \includegraphics[width=0.10\textwidth]{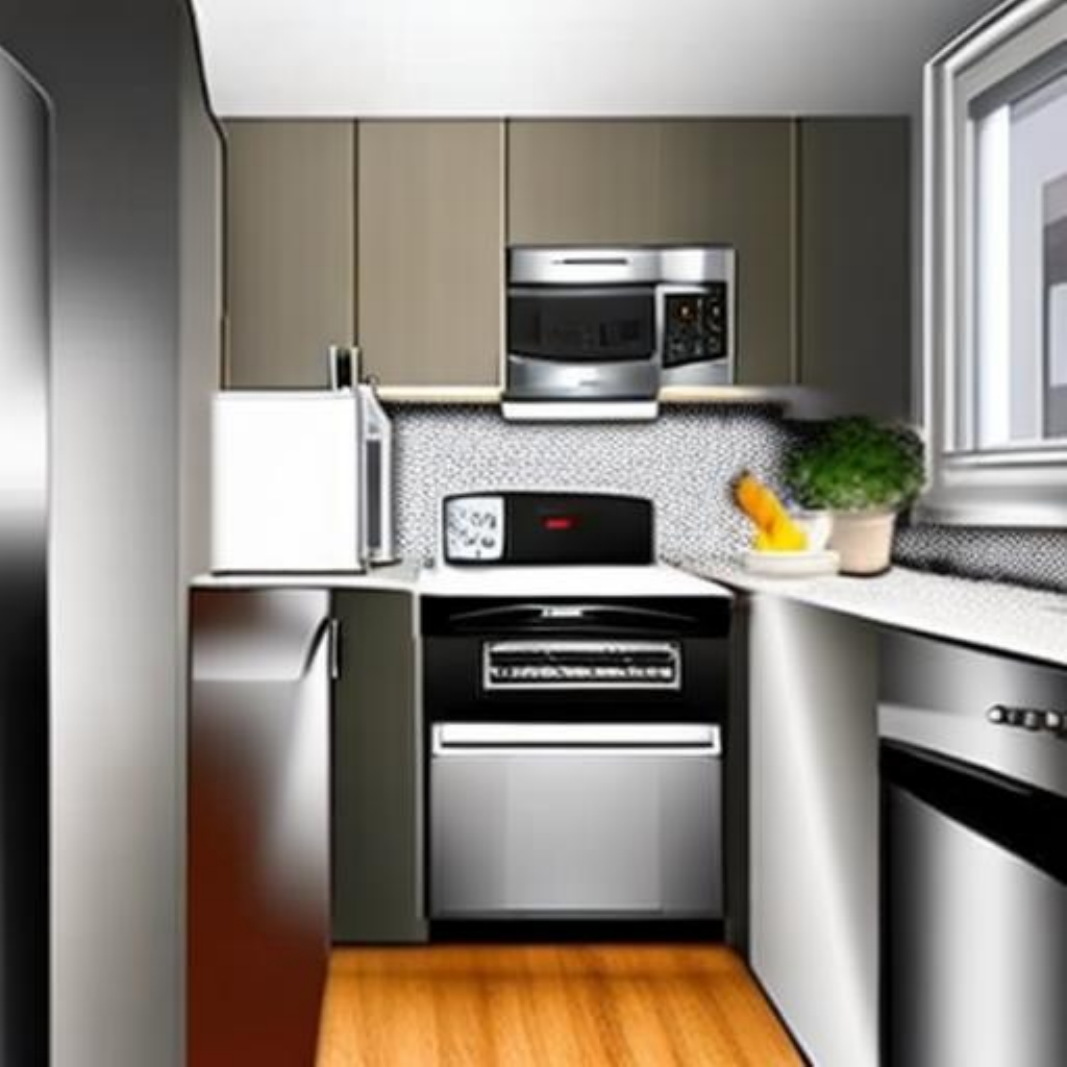} &
    \includegraphics[width=0.10\textwidth]{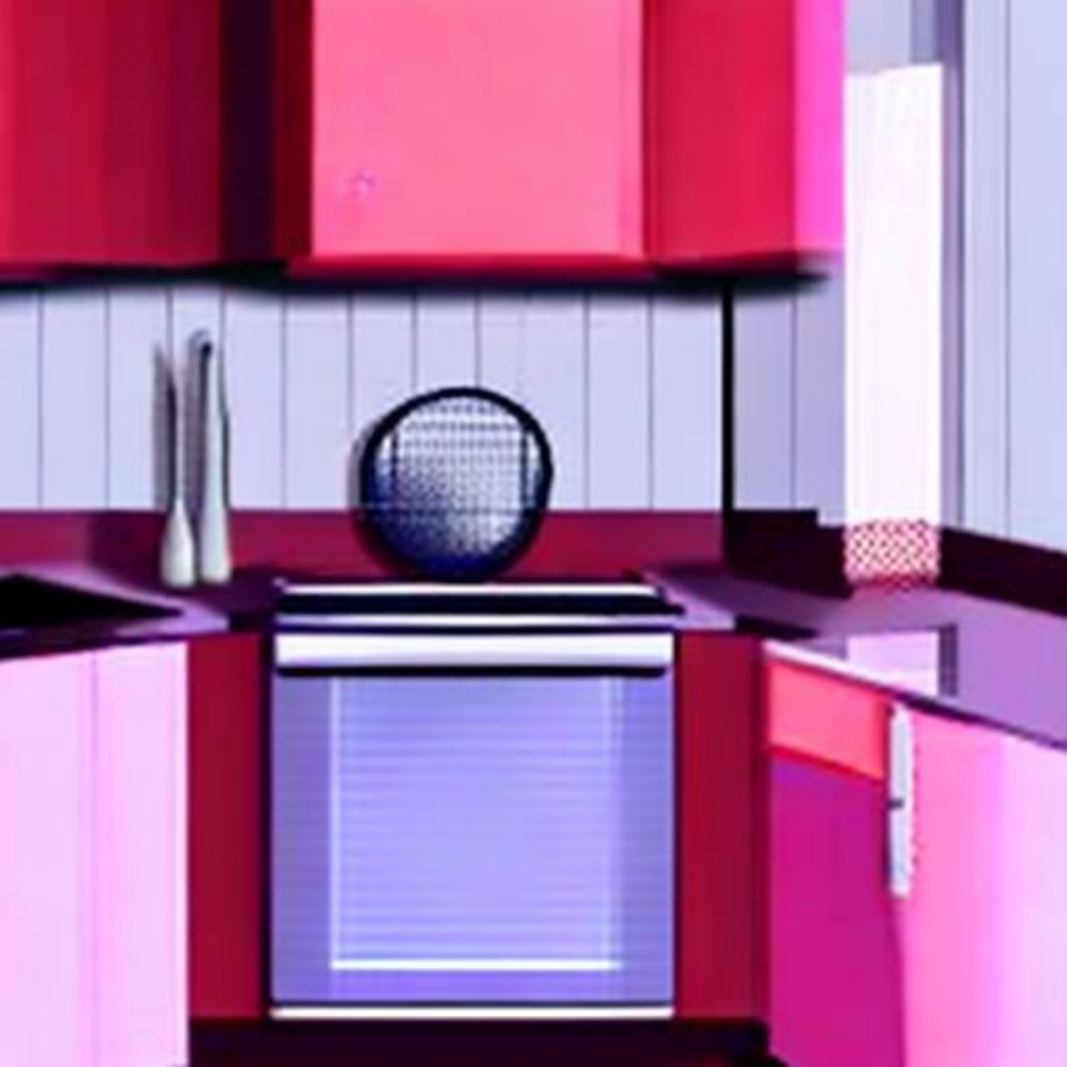} \\
    \includegraphics[width=0.10\textwidth]{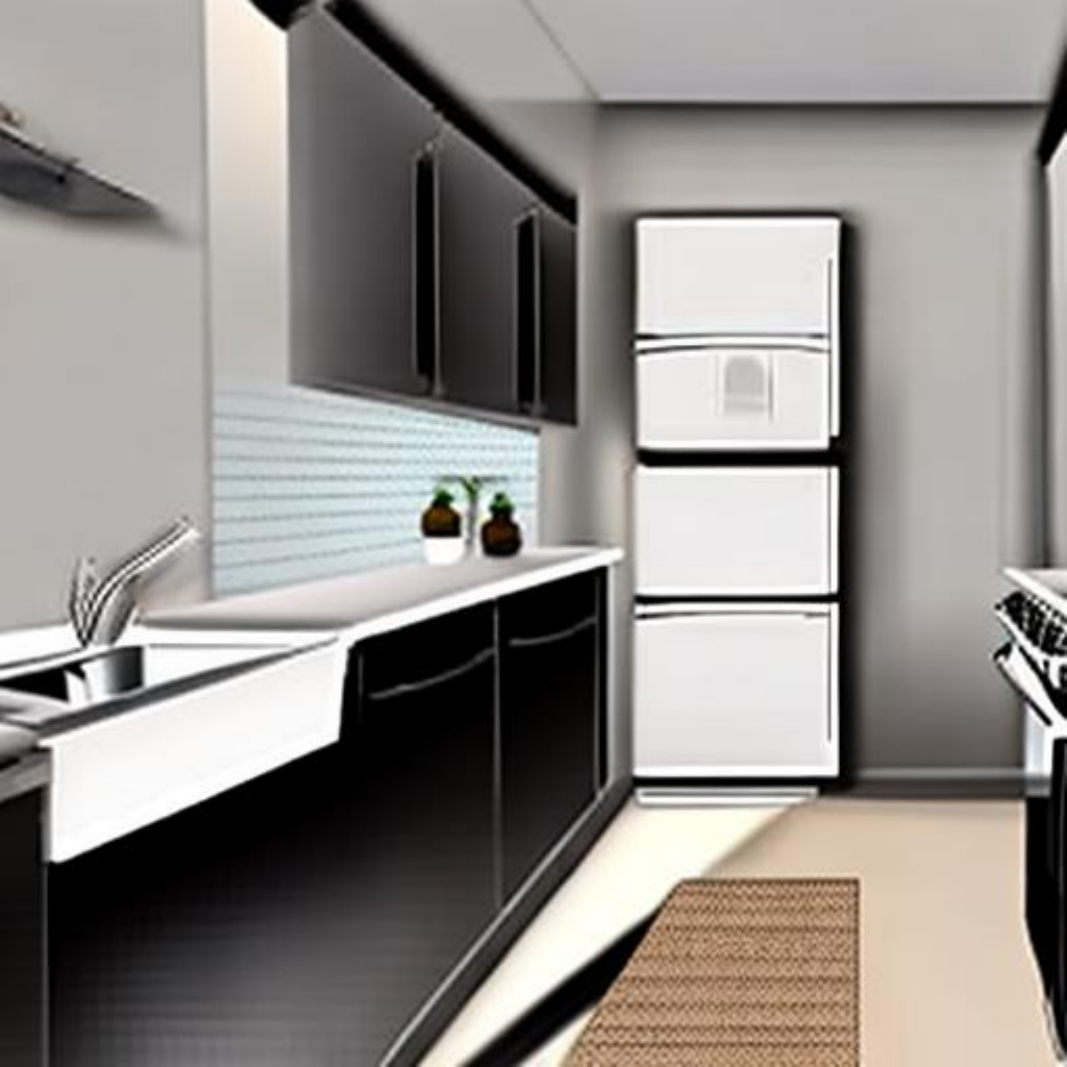} &
    \includegraphics[width=0.10\textwidth]{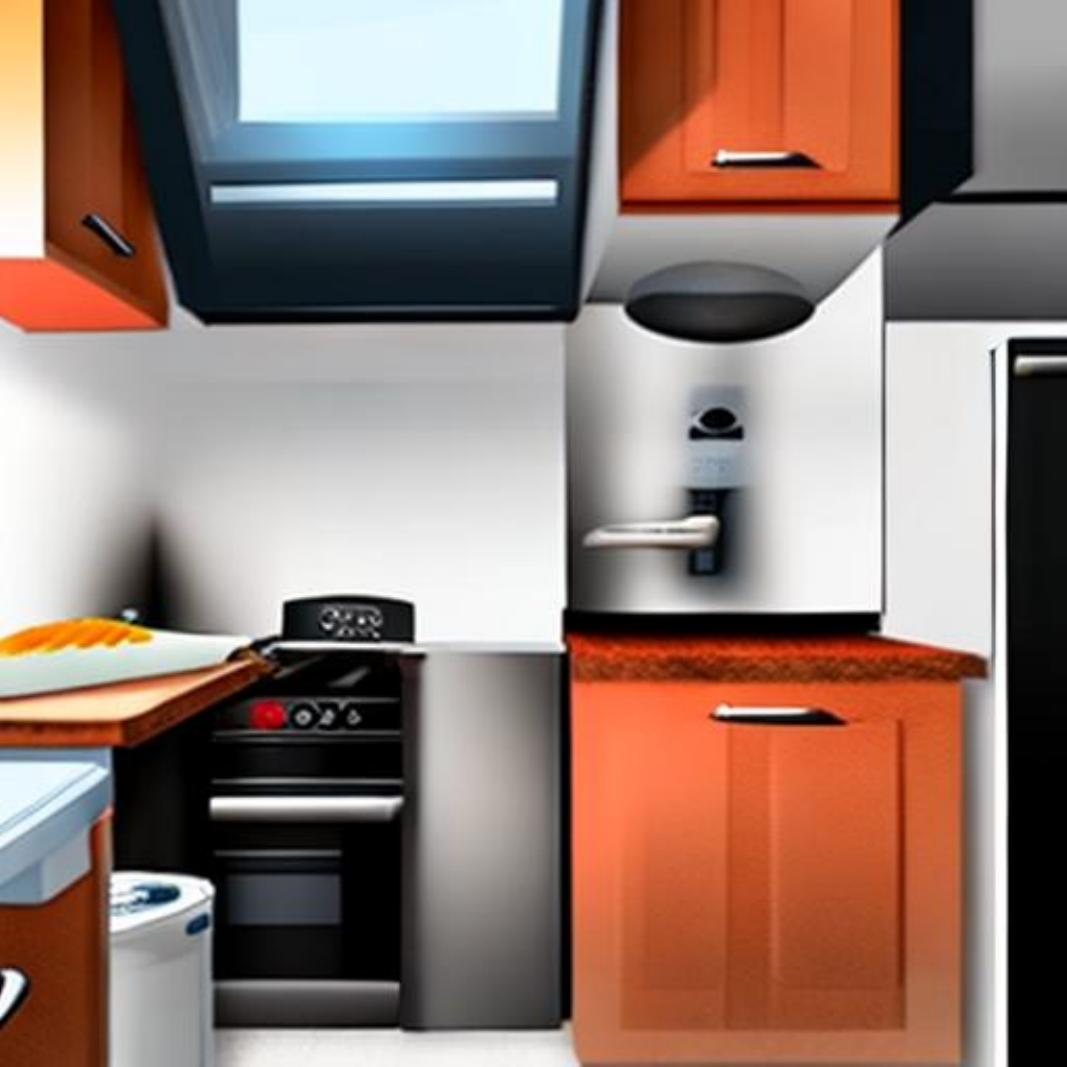} &
    \includegraphics[width=0.10\textwidth]{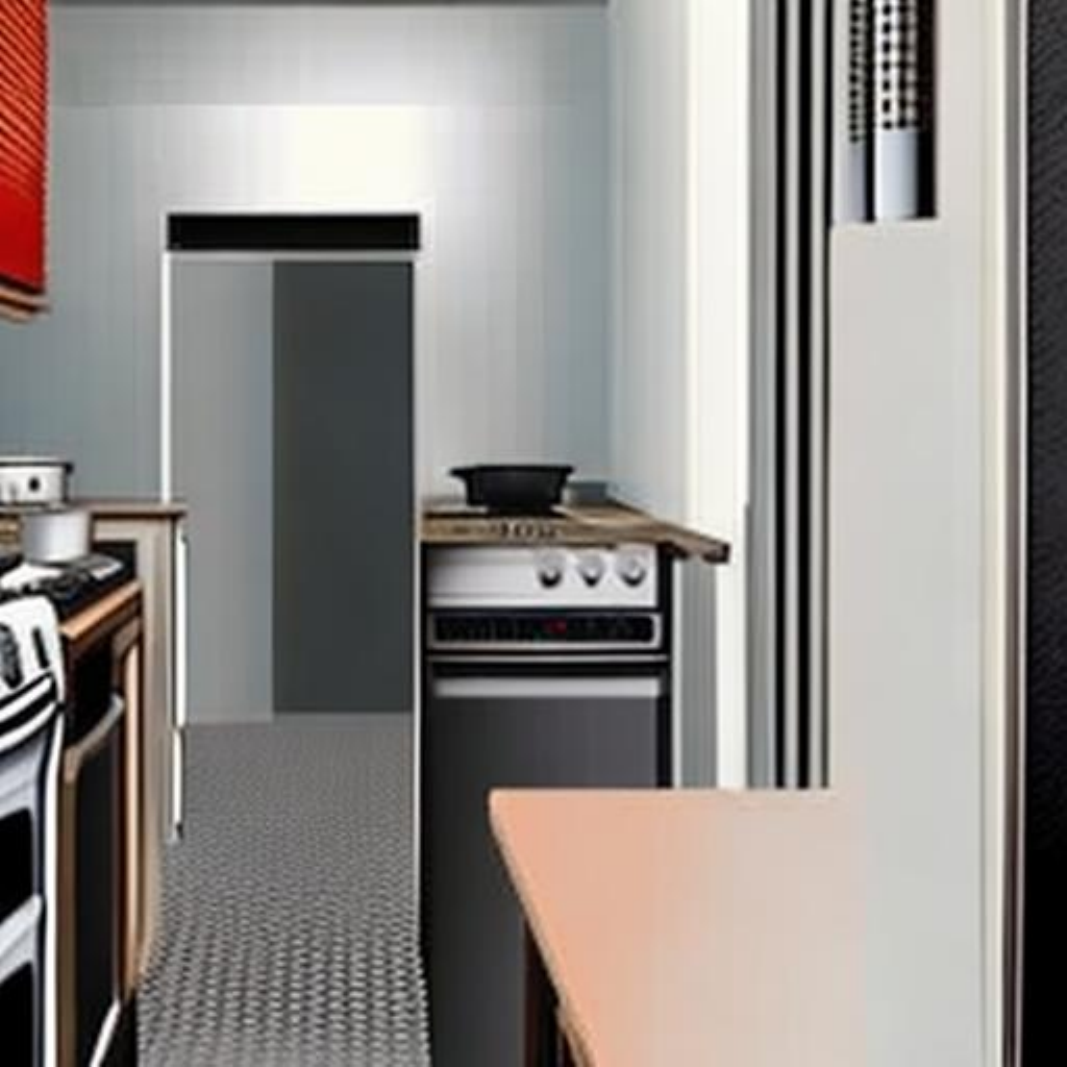} &
    \includegraphics[width=0.10\textwidth]{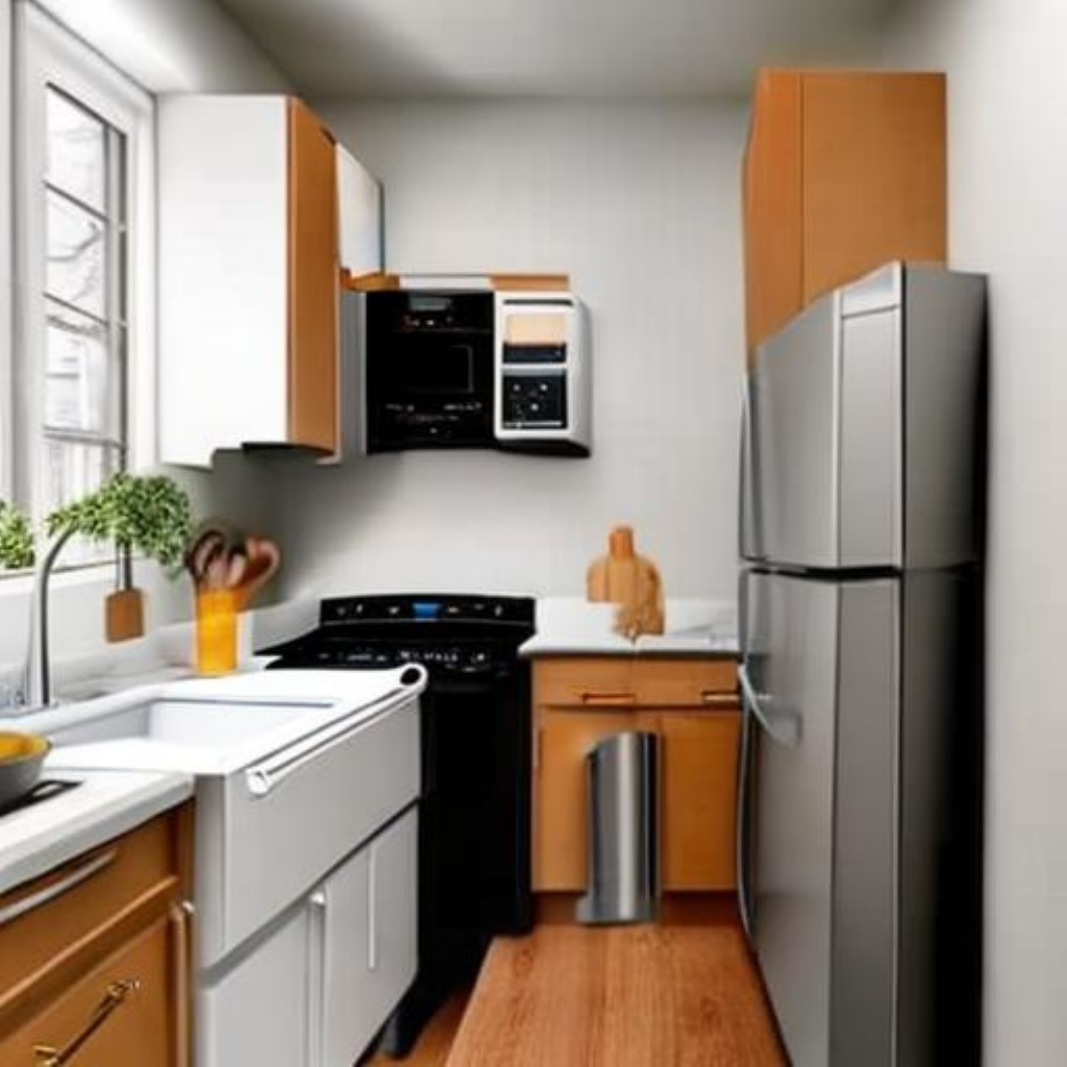} \\
    \includegraphics[width=0.10\textwidth]{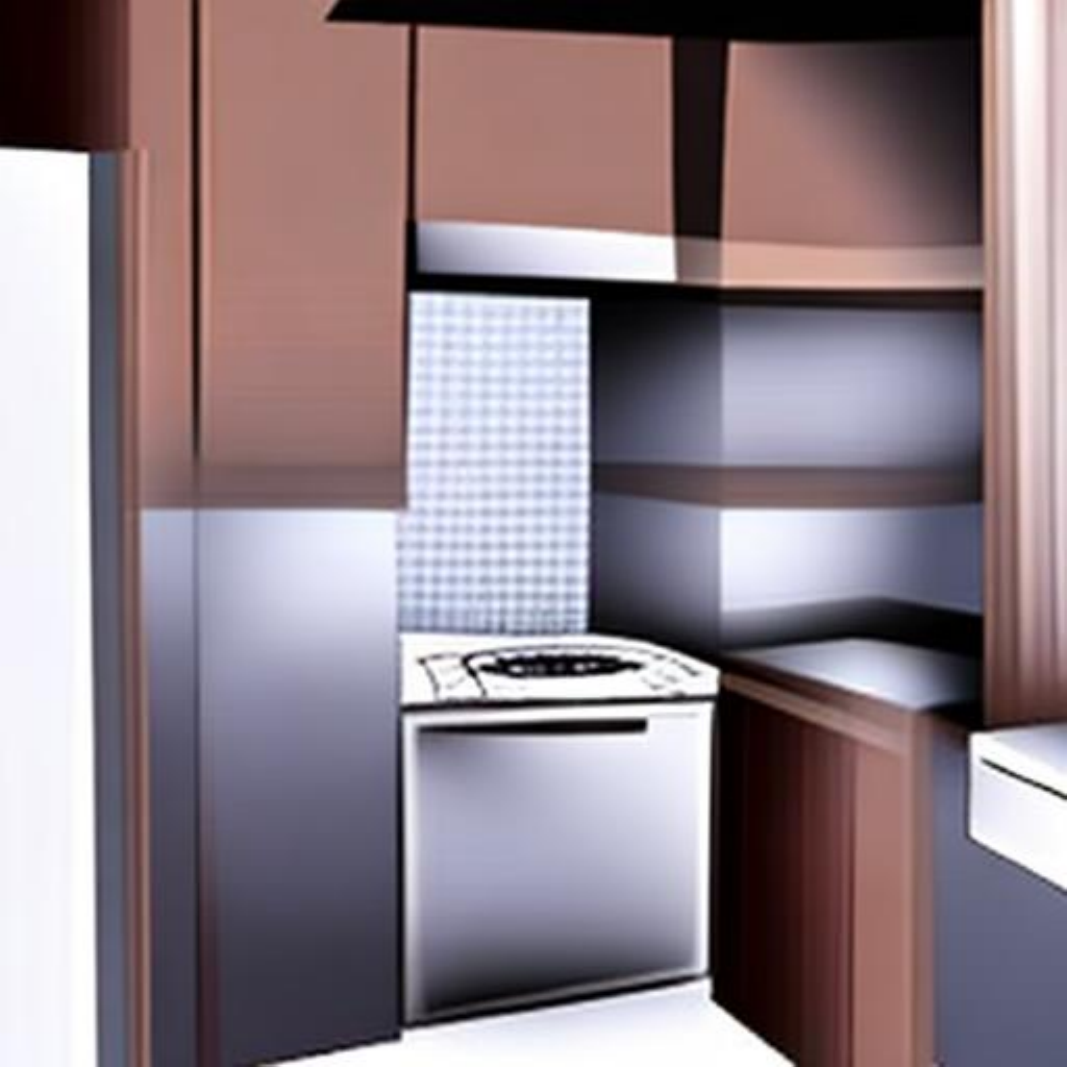} &
    \includegraphics[width=0.10\textwidth]{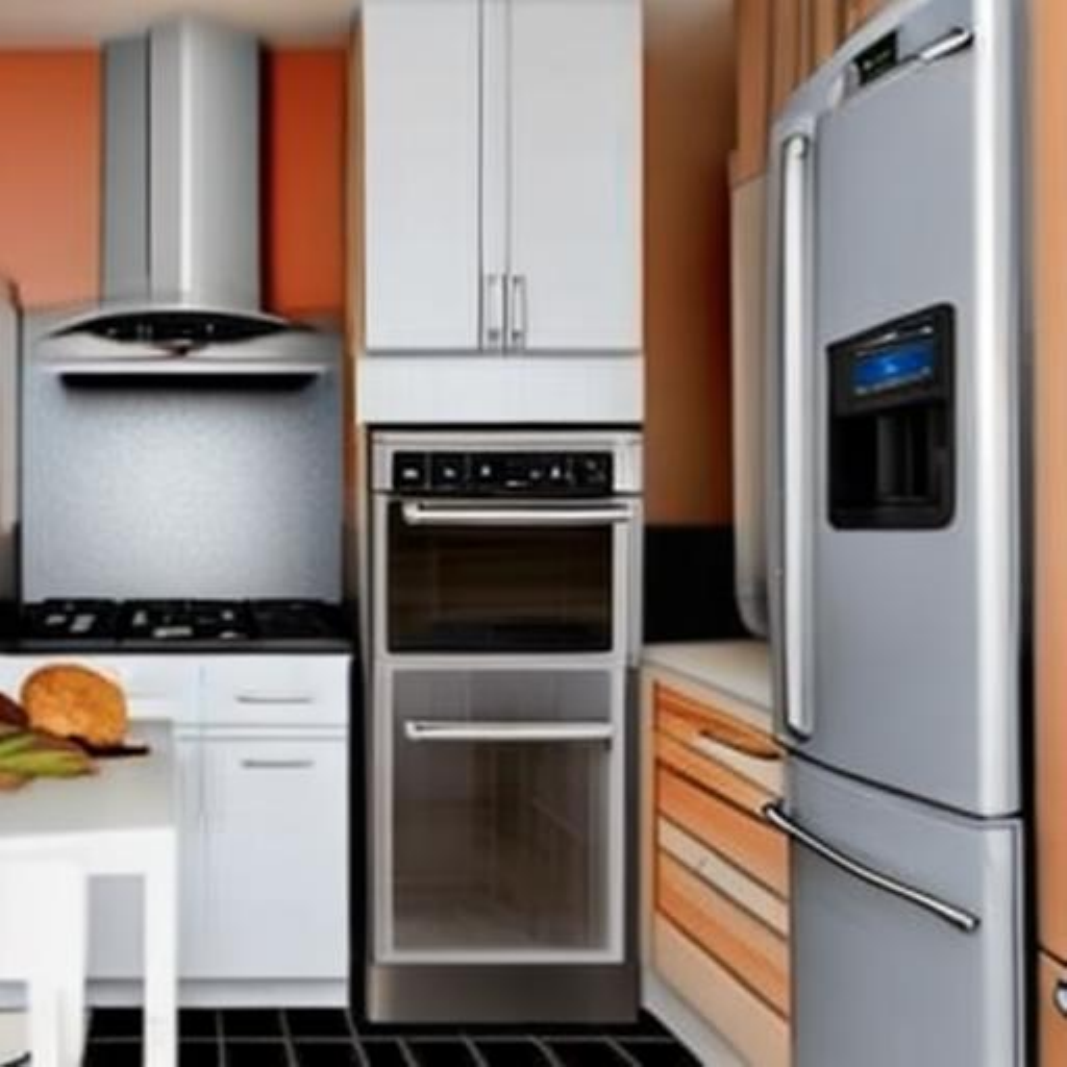} &
    \includegraphics[width=0.10\textwidth]{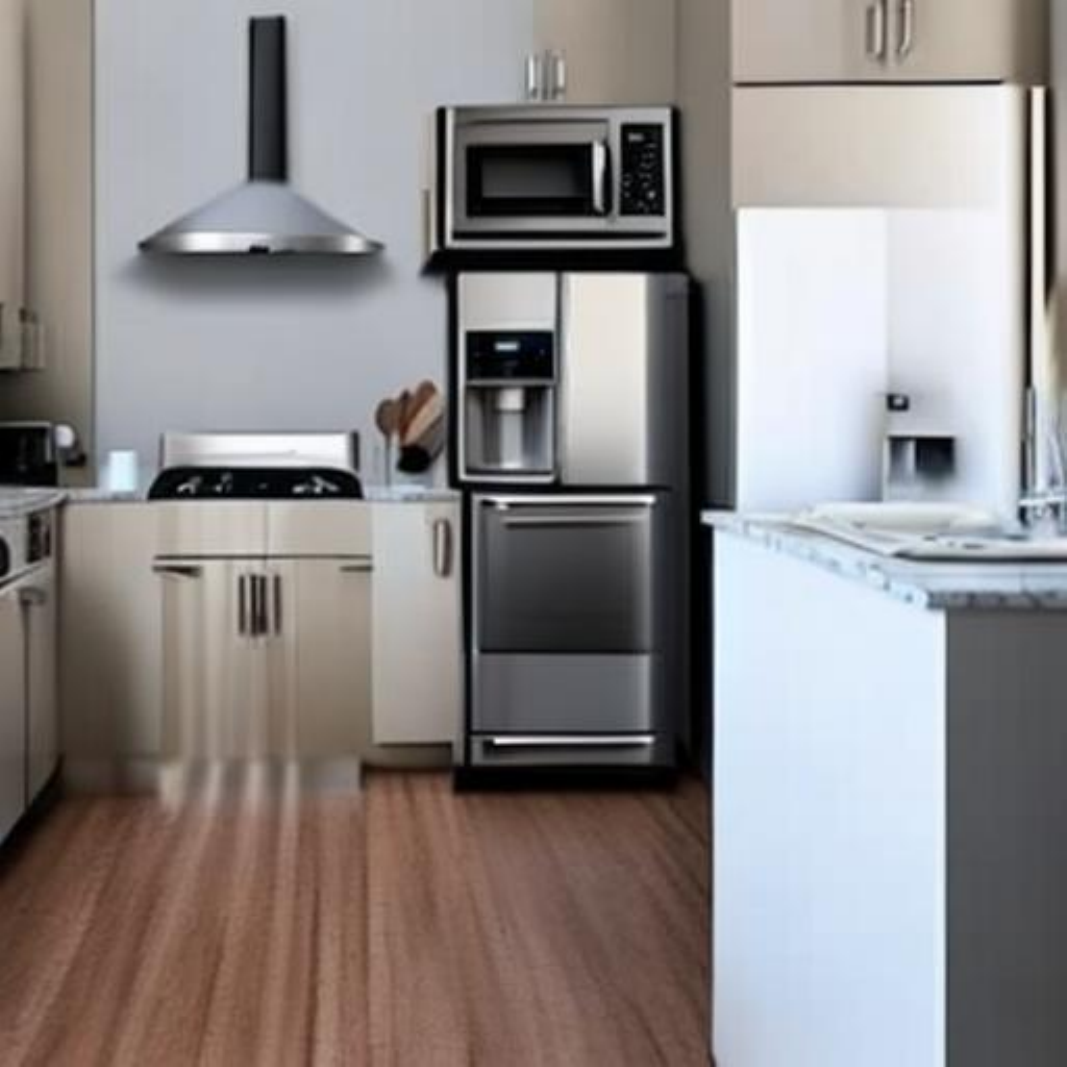} &
    \includegraphics[width=0.10\textwidth]{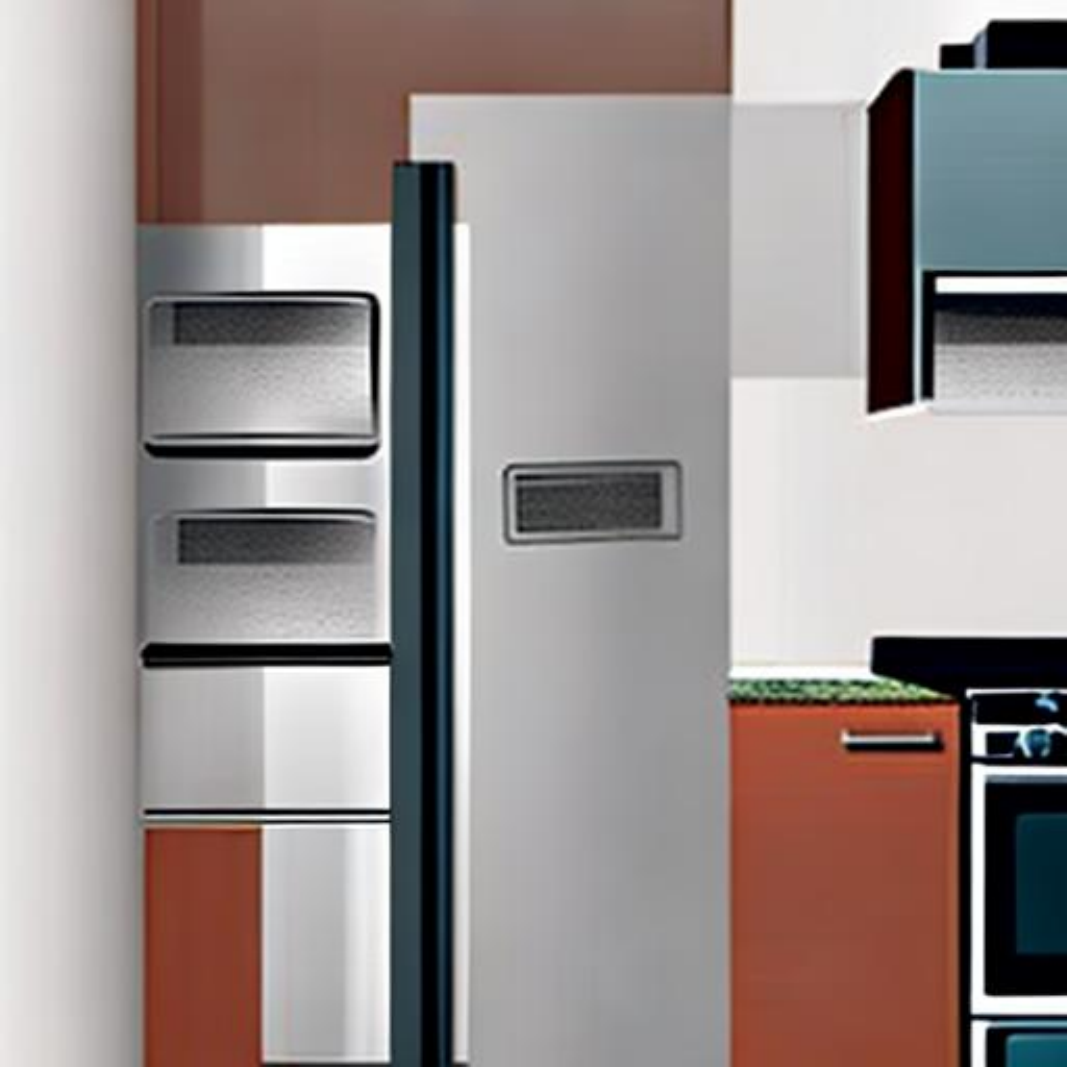} \\
    \includegraphics[width=0.10\textwidth]{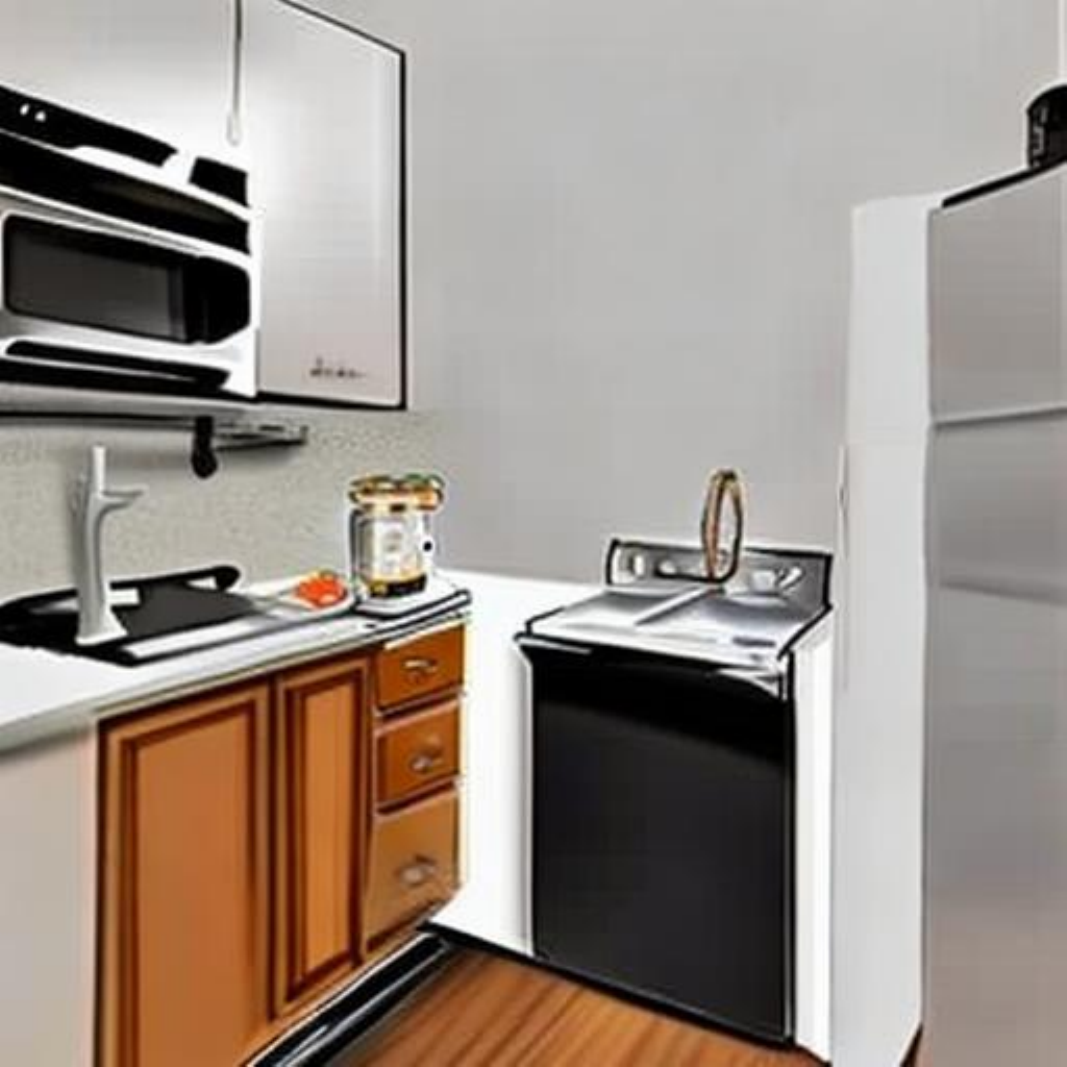} &
    \includegraphics[width=0.10\textwidth]{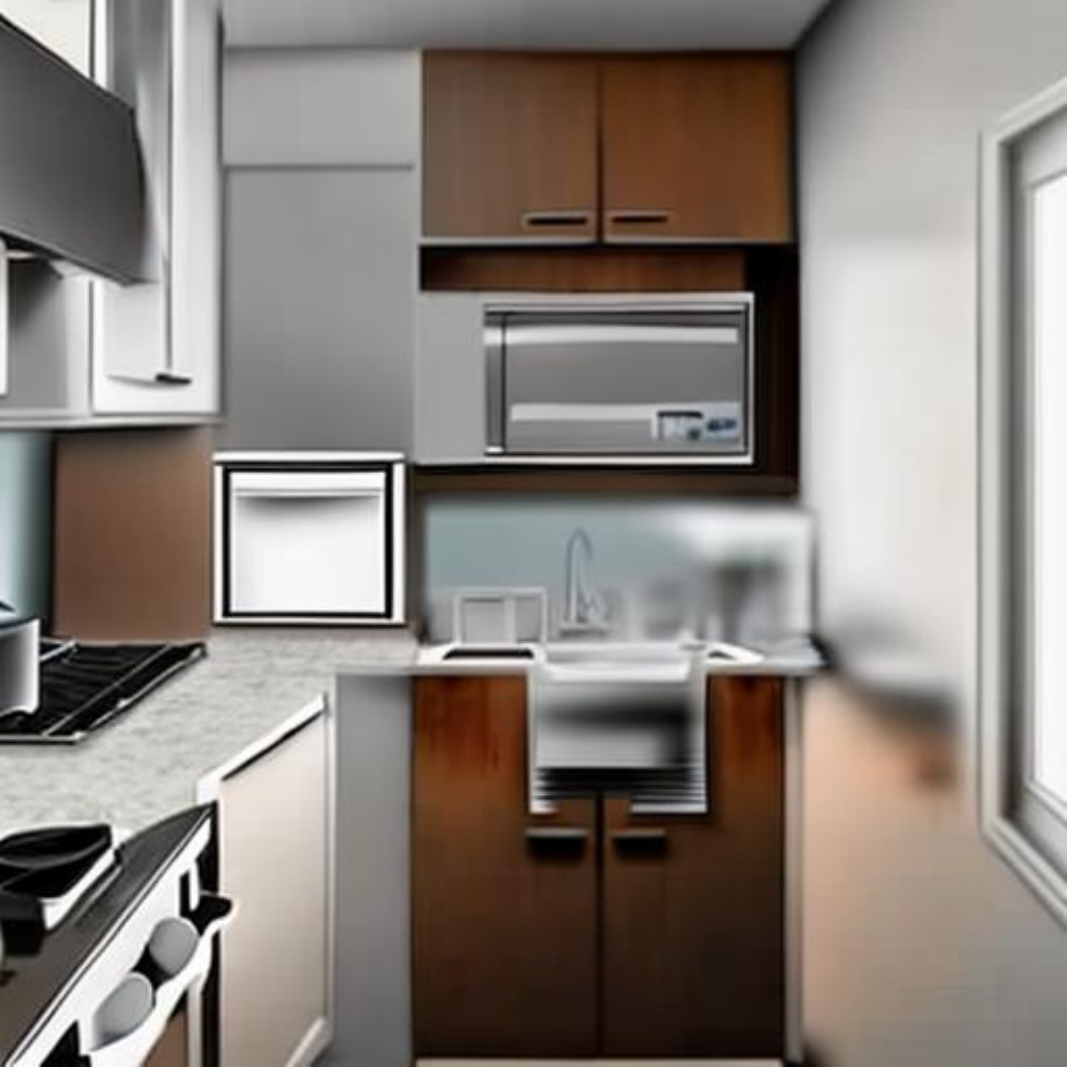} &
    \includegraphics[width=0.10\textwidth]{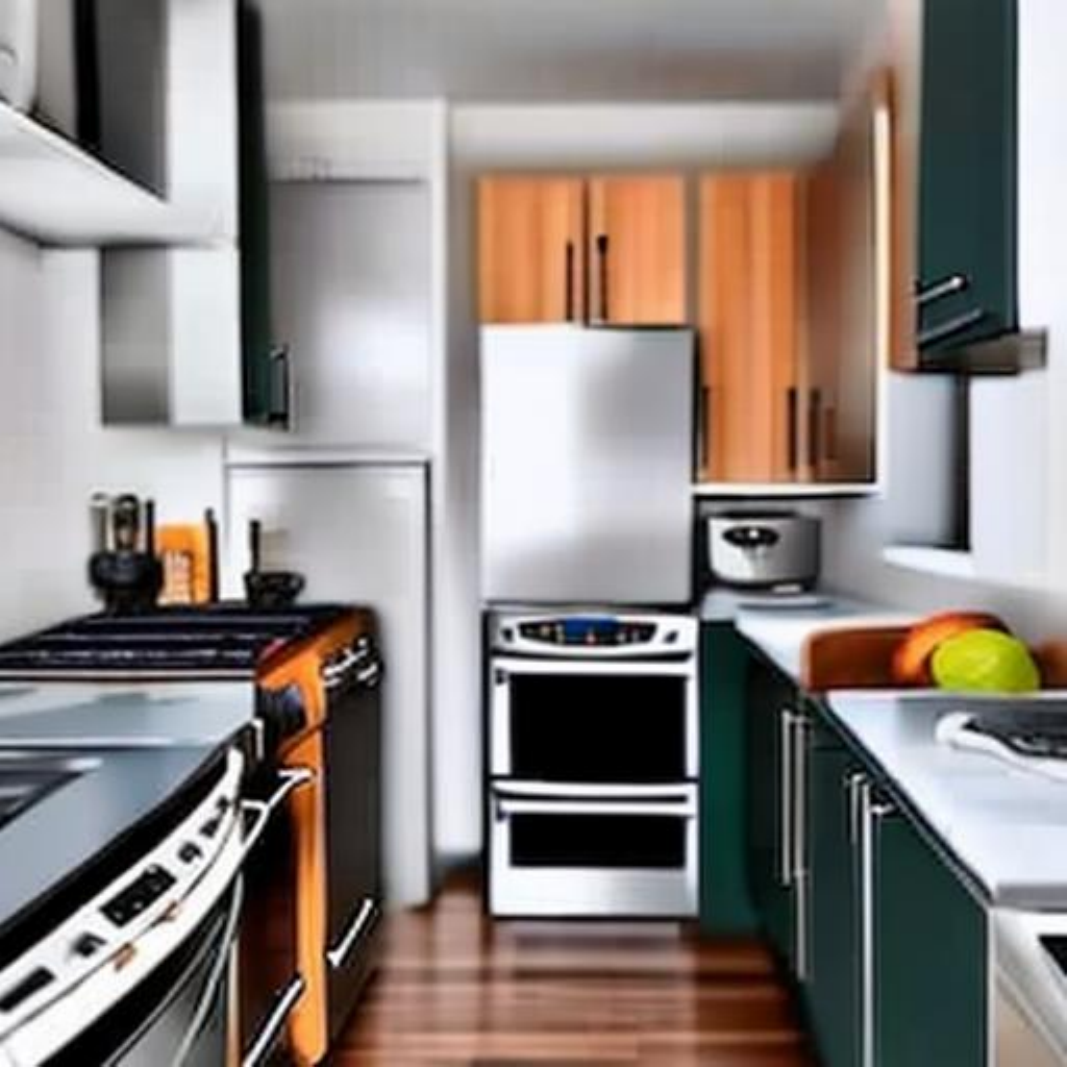} &
    \includegraphics[width=0.10\textwidth]{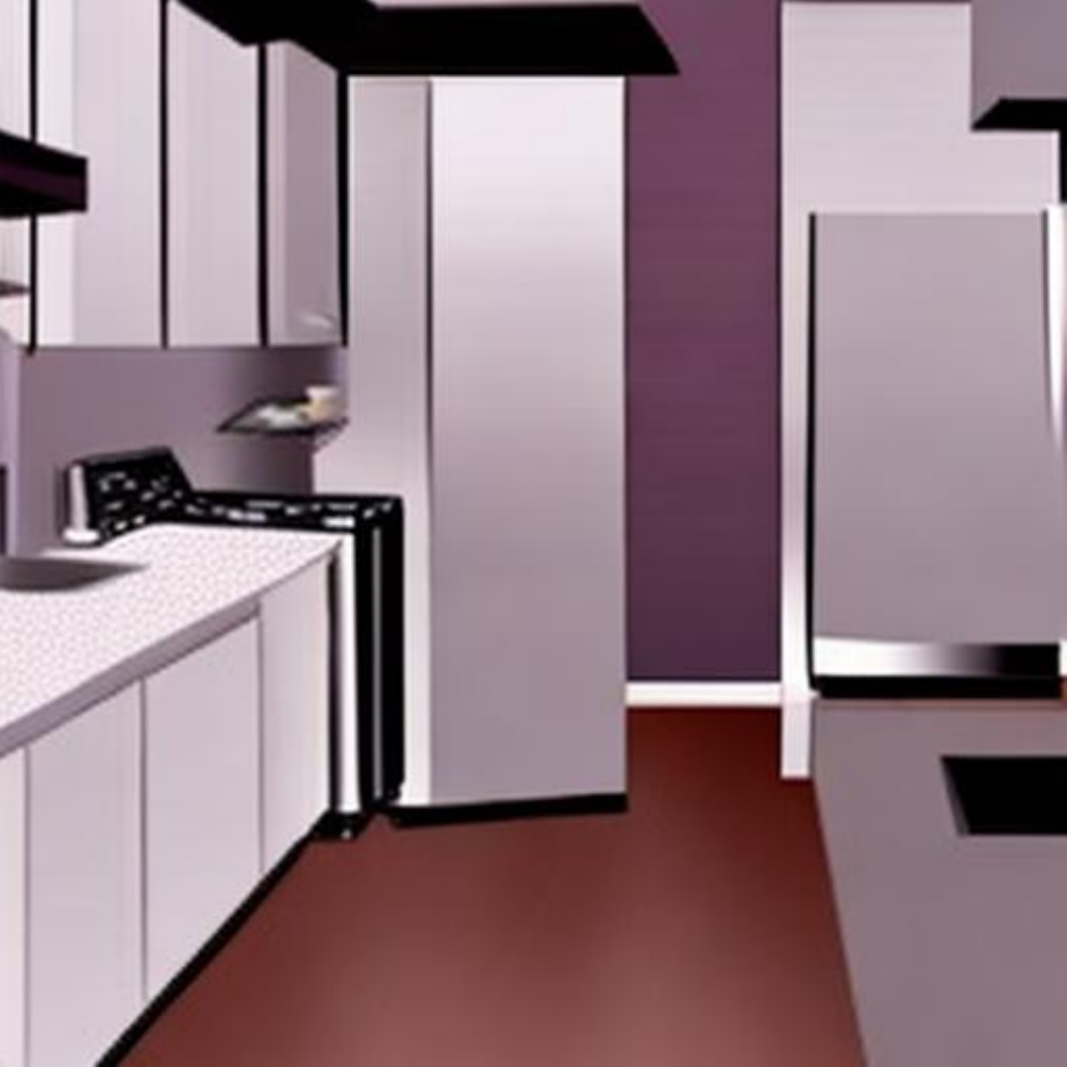} \\
};
\node[below=0pt of tl] {Constant CFG};

\matrix (tr) [gridmat, right=12pt of tl]
{
    \includegraphics[width=0.10\textwidth]{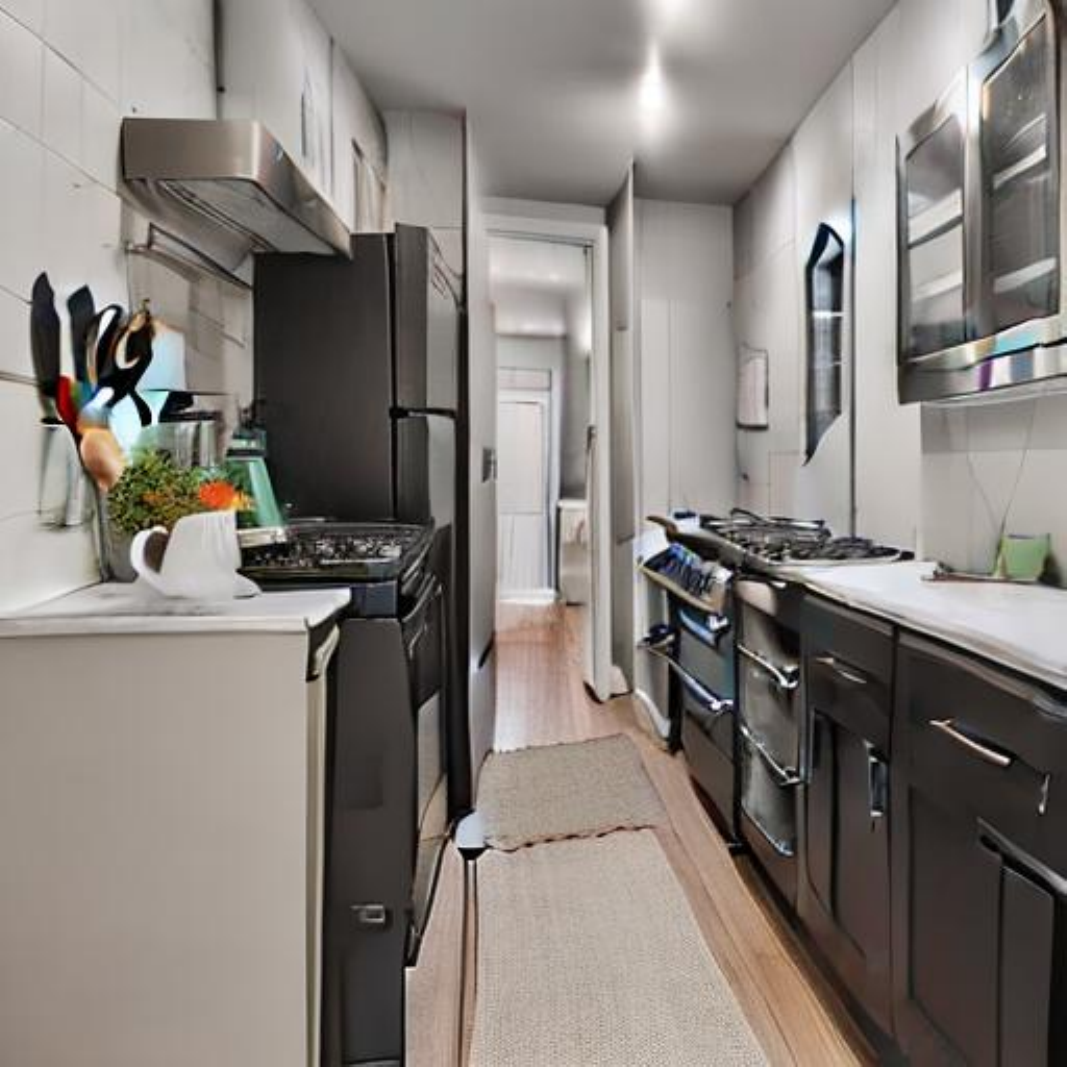} &
    \includegraphics[width=0.10\textwidth]{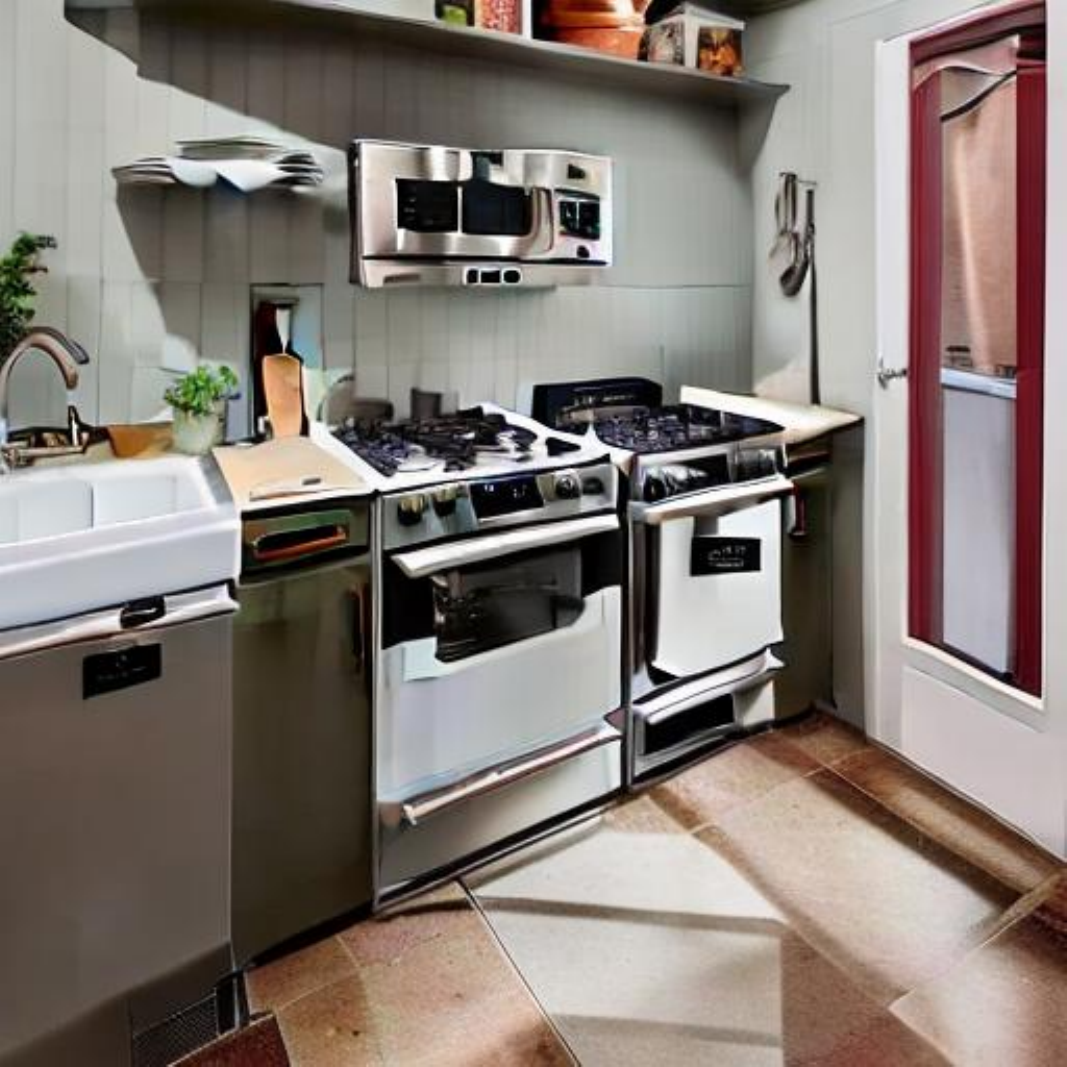} &
    \includegraphics[width=0.10\textwidth]{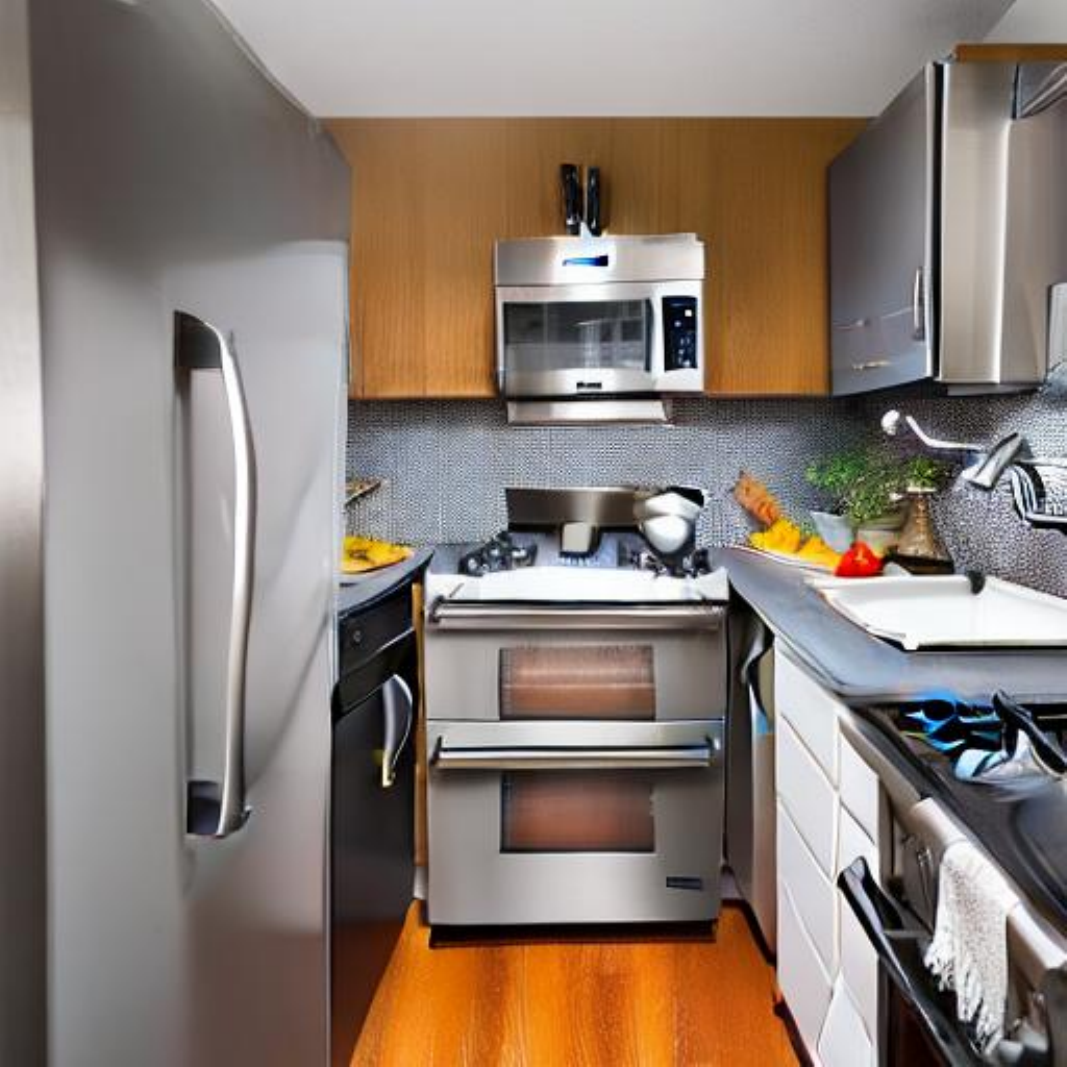} &
    \includegraphics[width=0.10\textwidth]{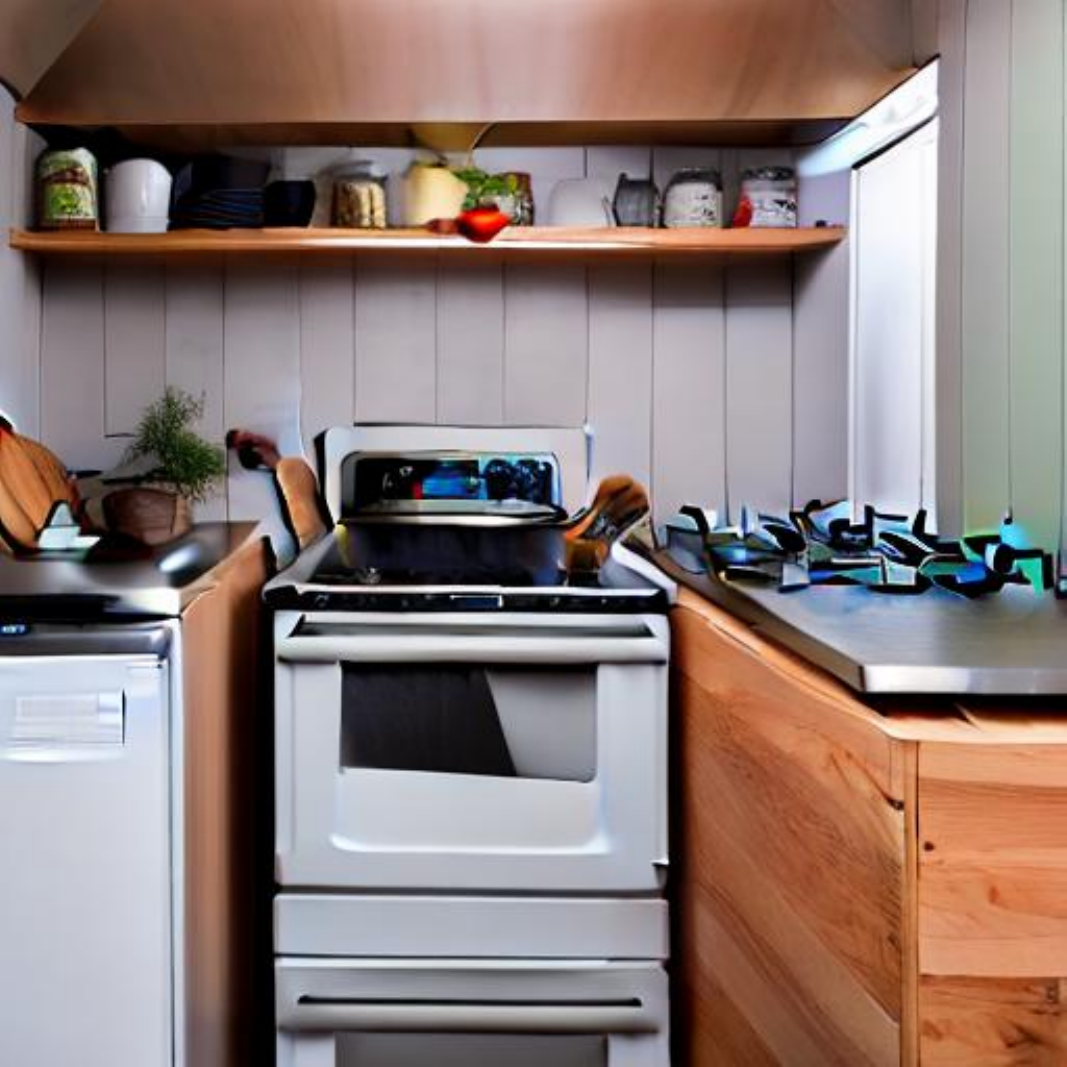} \\
    \includegraphics[width=0.10\textwidth]{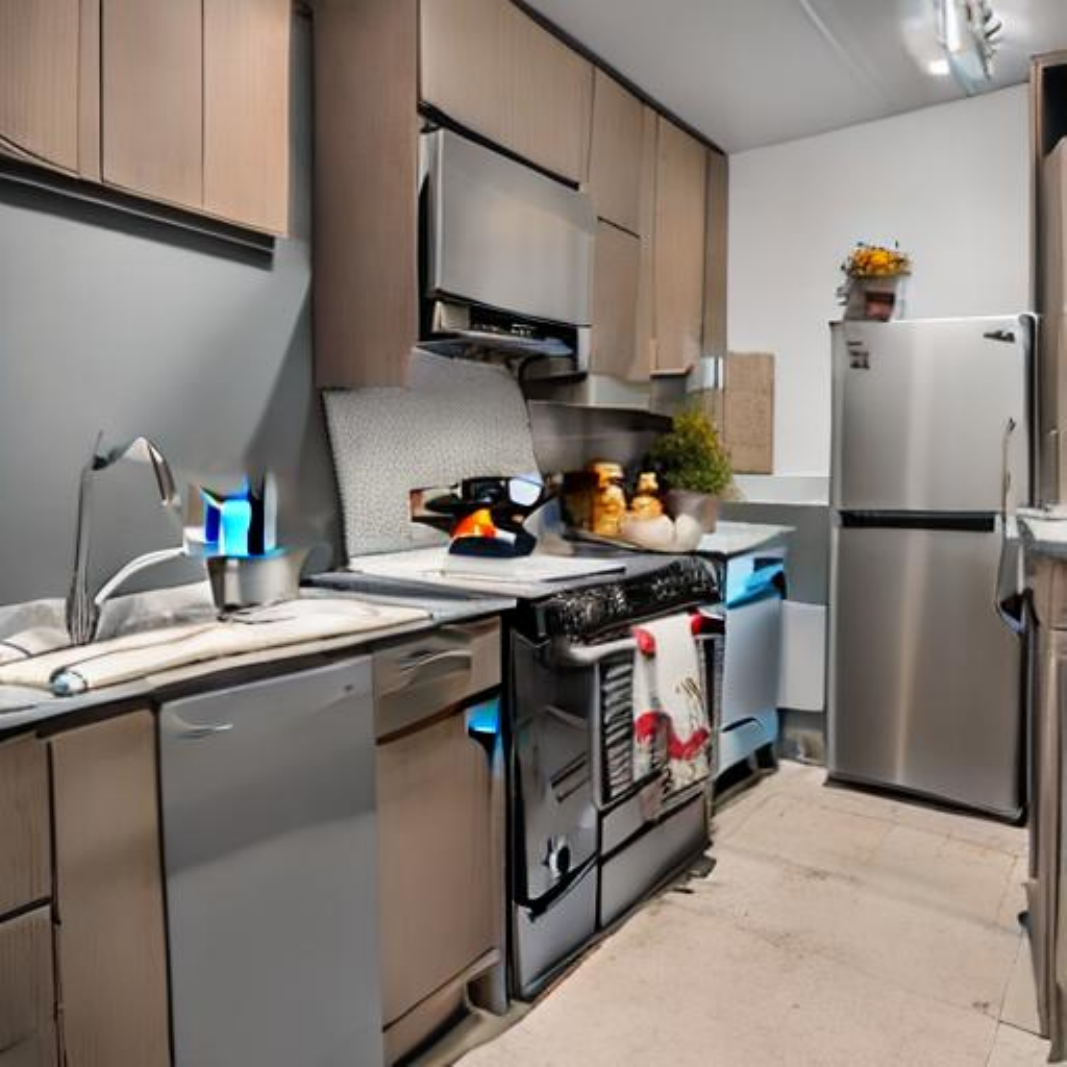} &
    \includegraphics[width=0.10\textwidth]{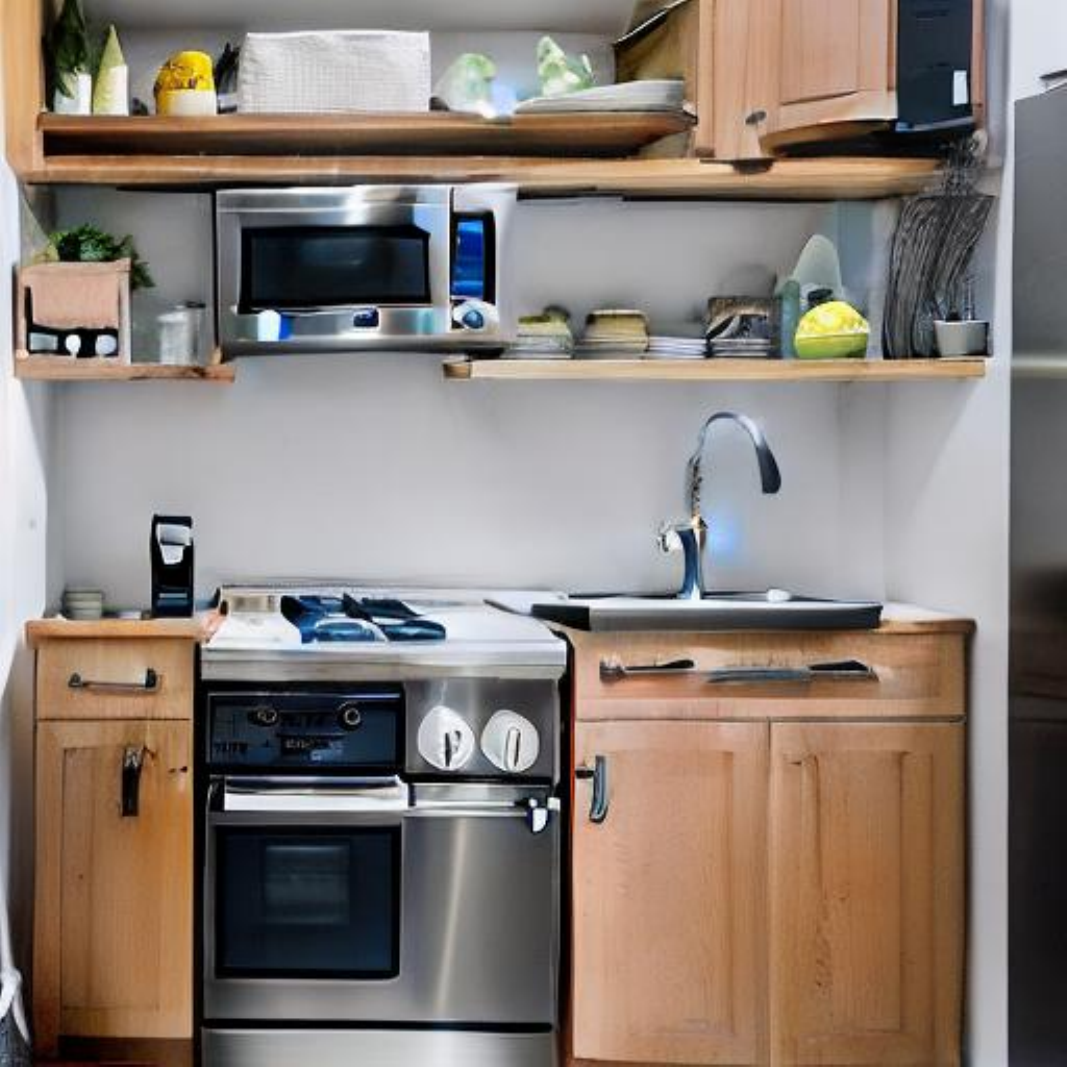} &
    \includegraphics[width=0.10\textwidth]{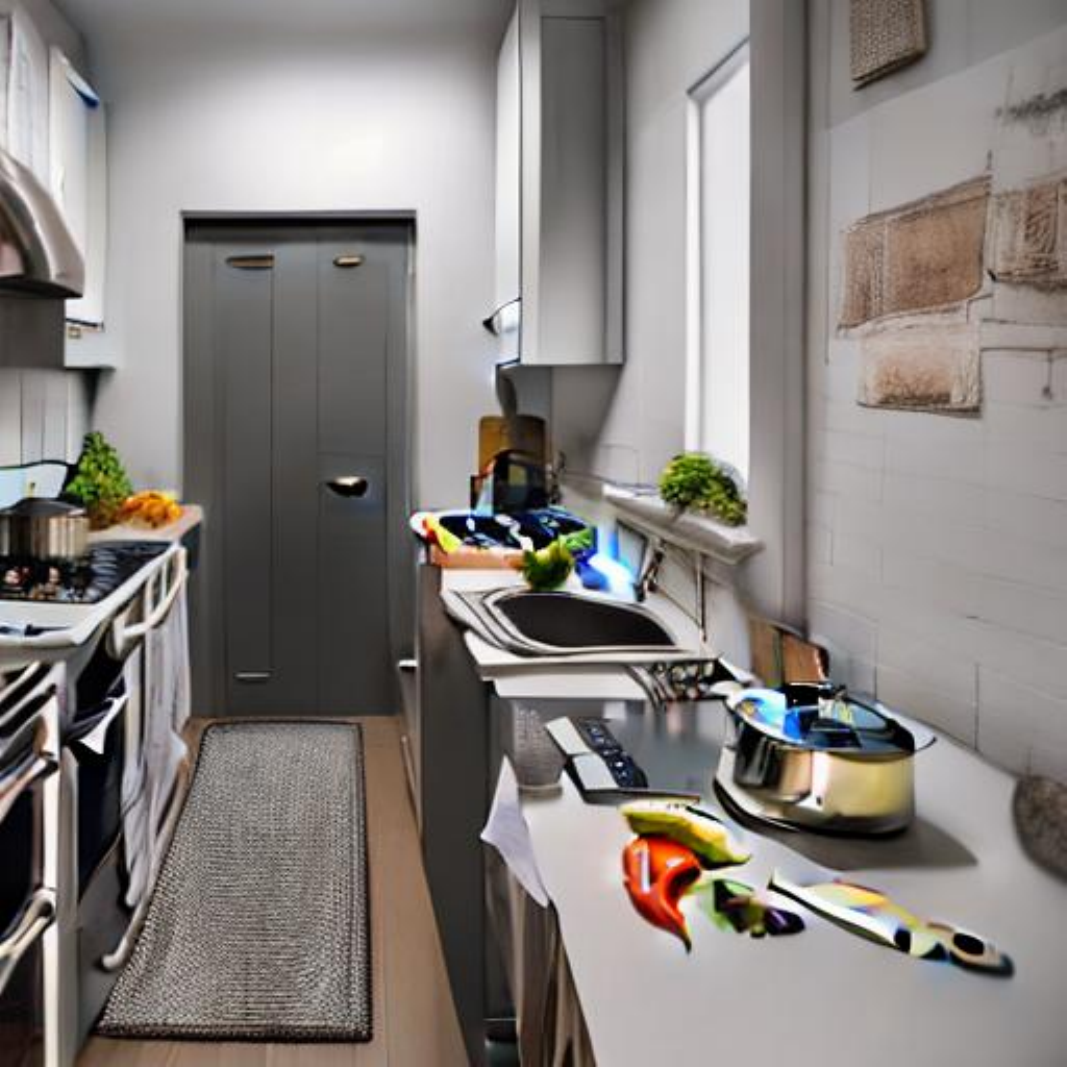} &
    \includegraphics[width=0.10\textwidth]{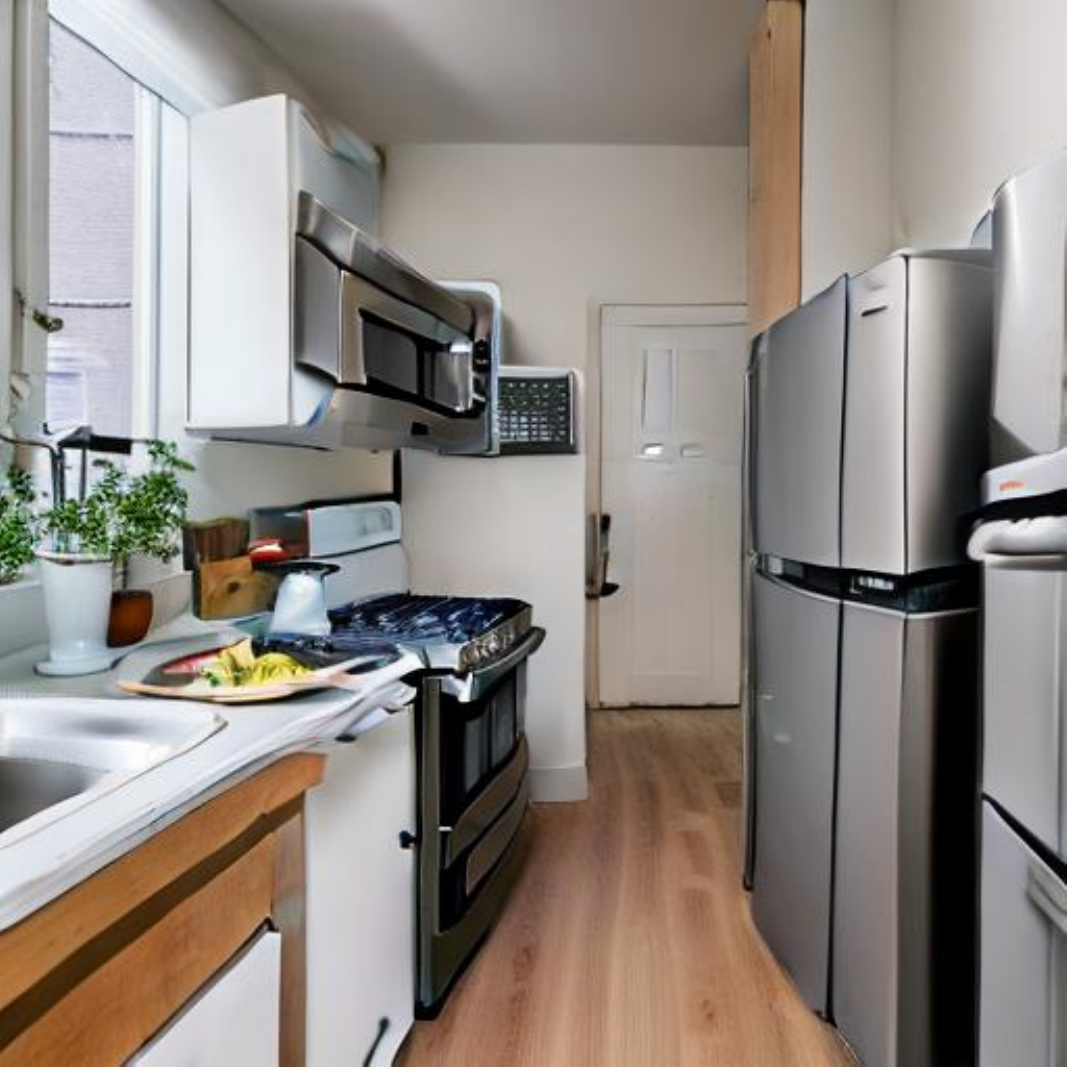} \\
    \includegraphics[width=0.10\textwidth]{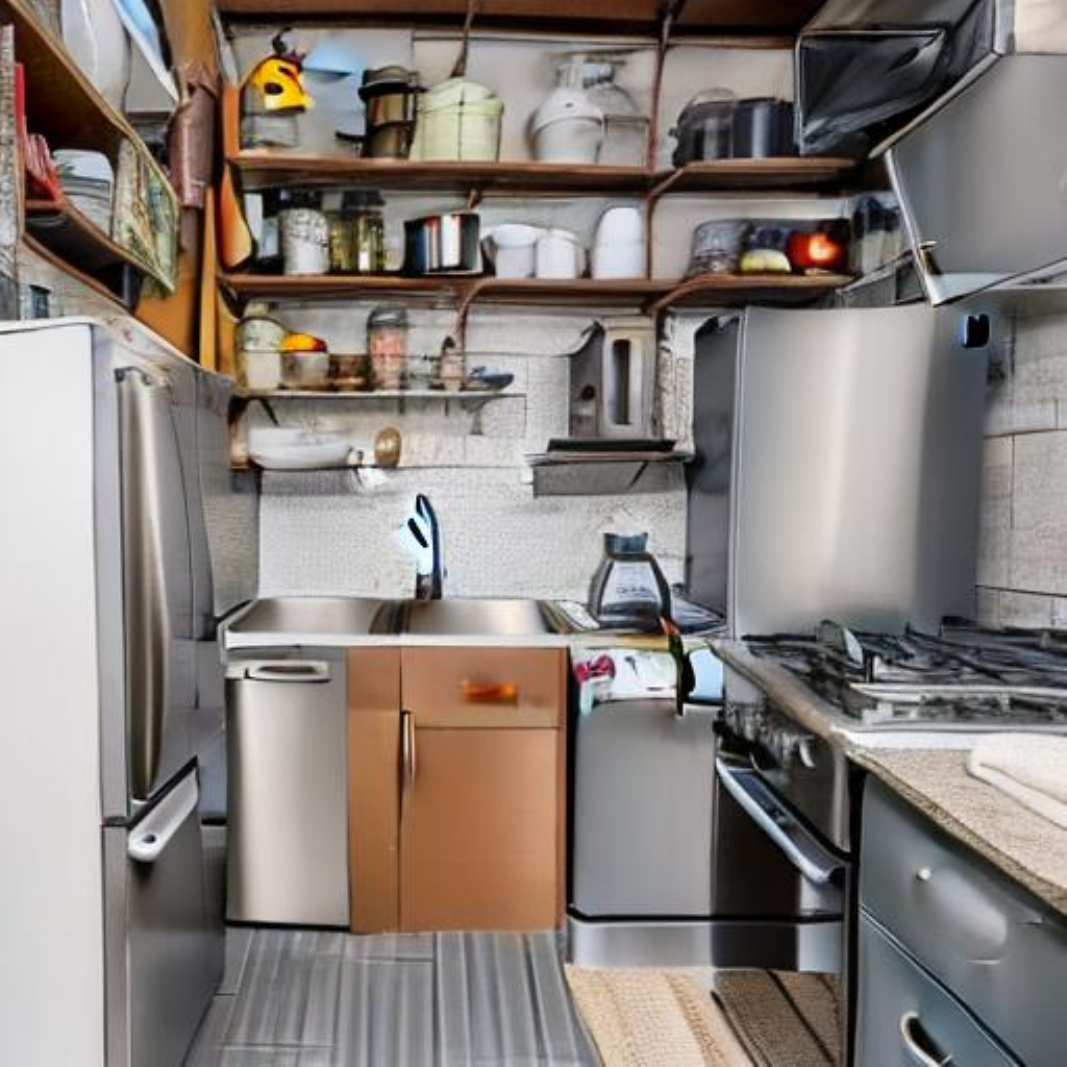} &
    \includegraphics[width=0.10\textwidth]{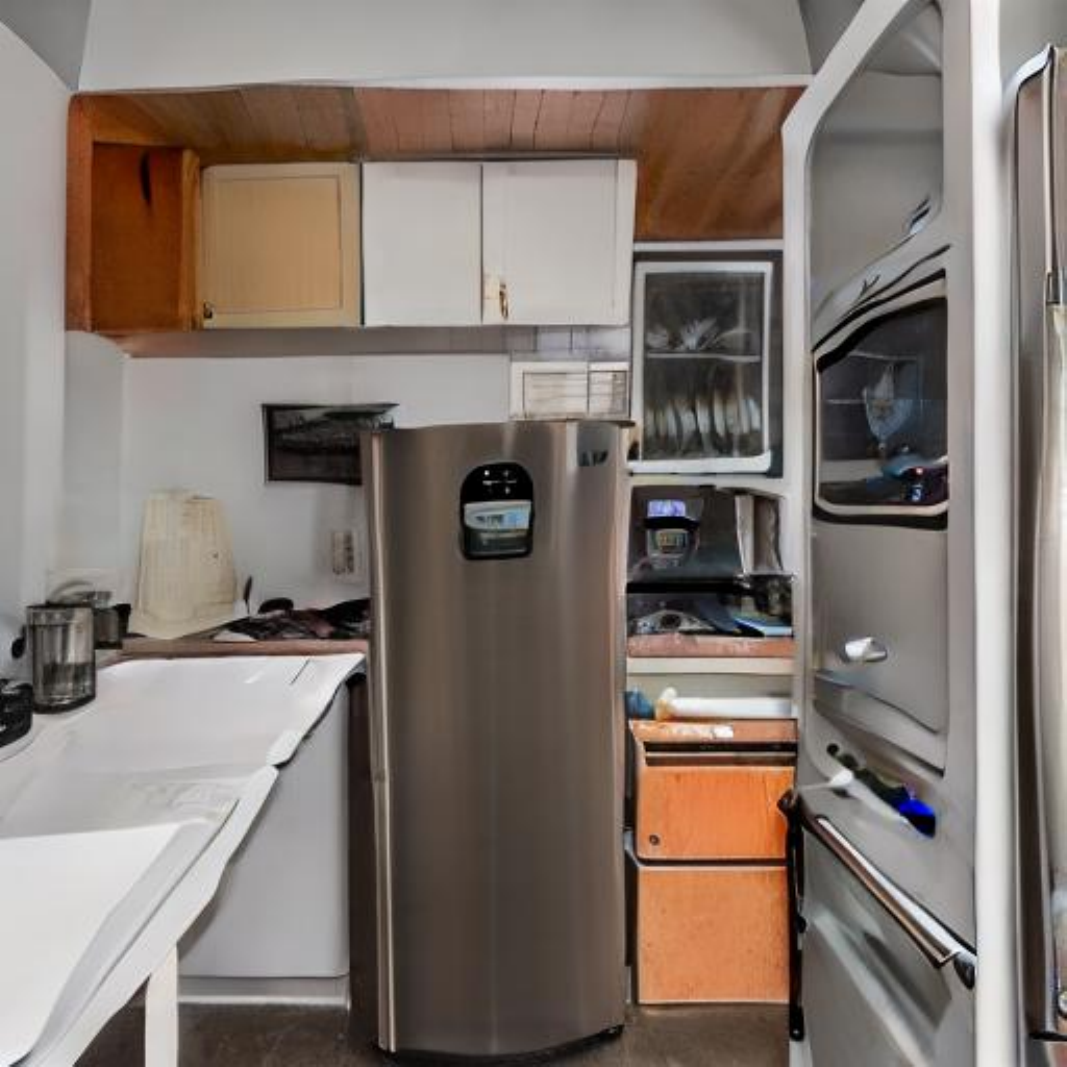} &
    \includegraphics[width=0.10\textwidth]{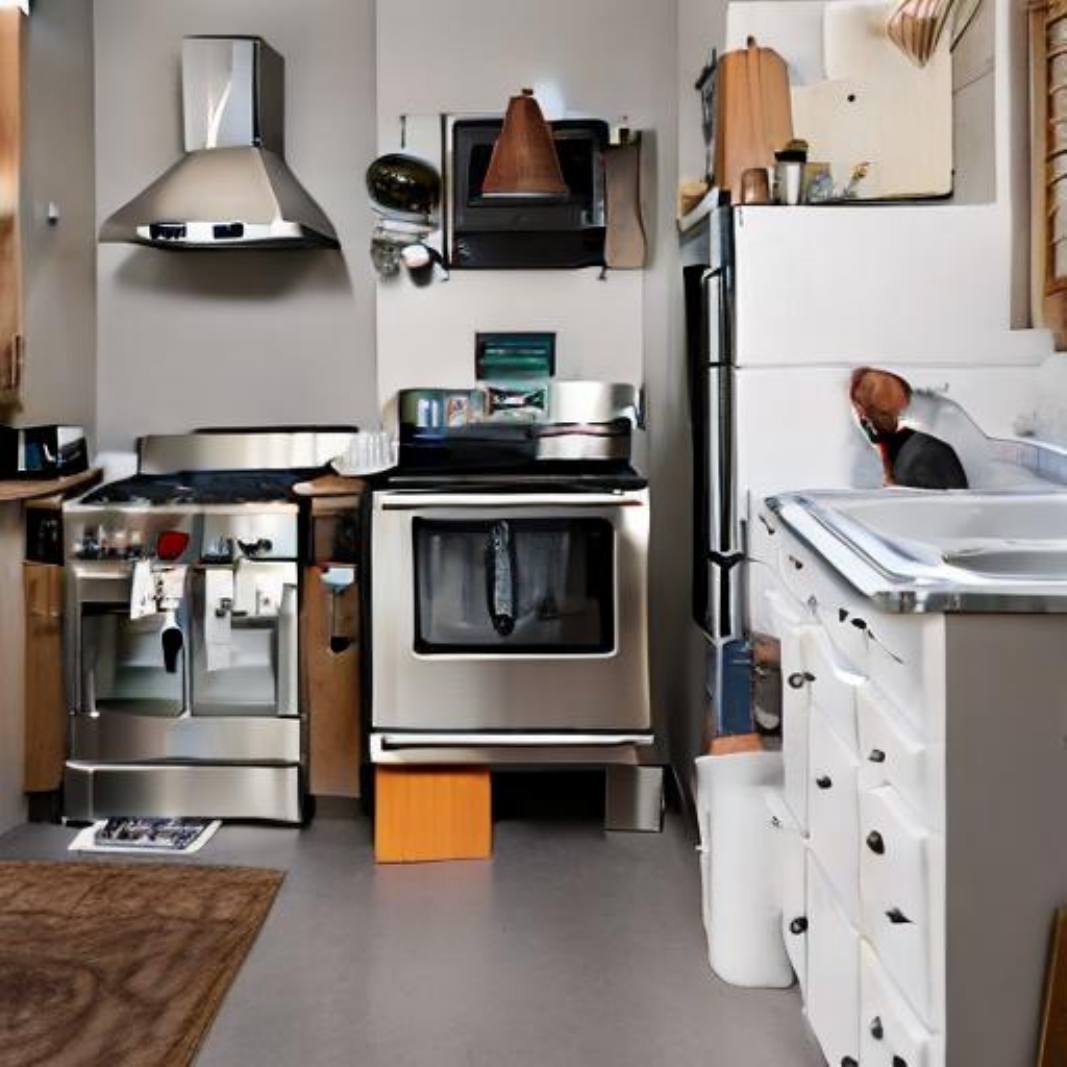} &
    \includegraphics[width=0.10\textwidth]{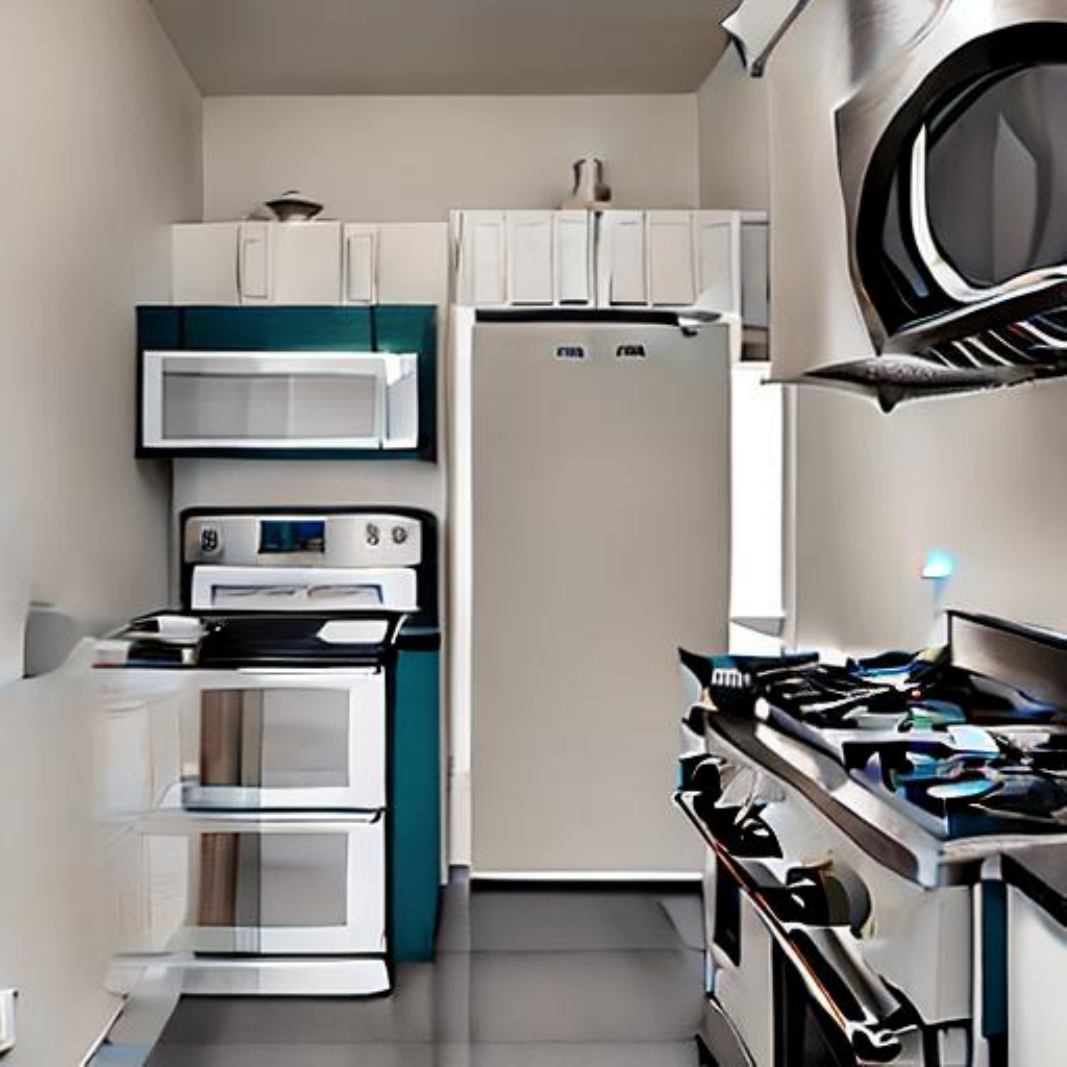} \\
    \includegraphics[width=0.10\textwidth]{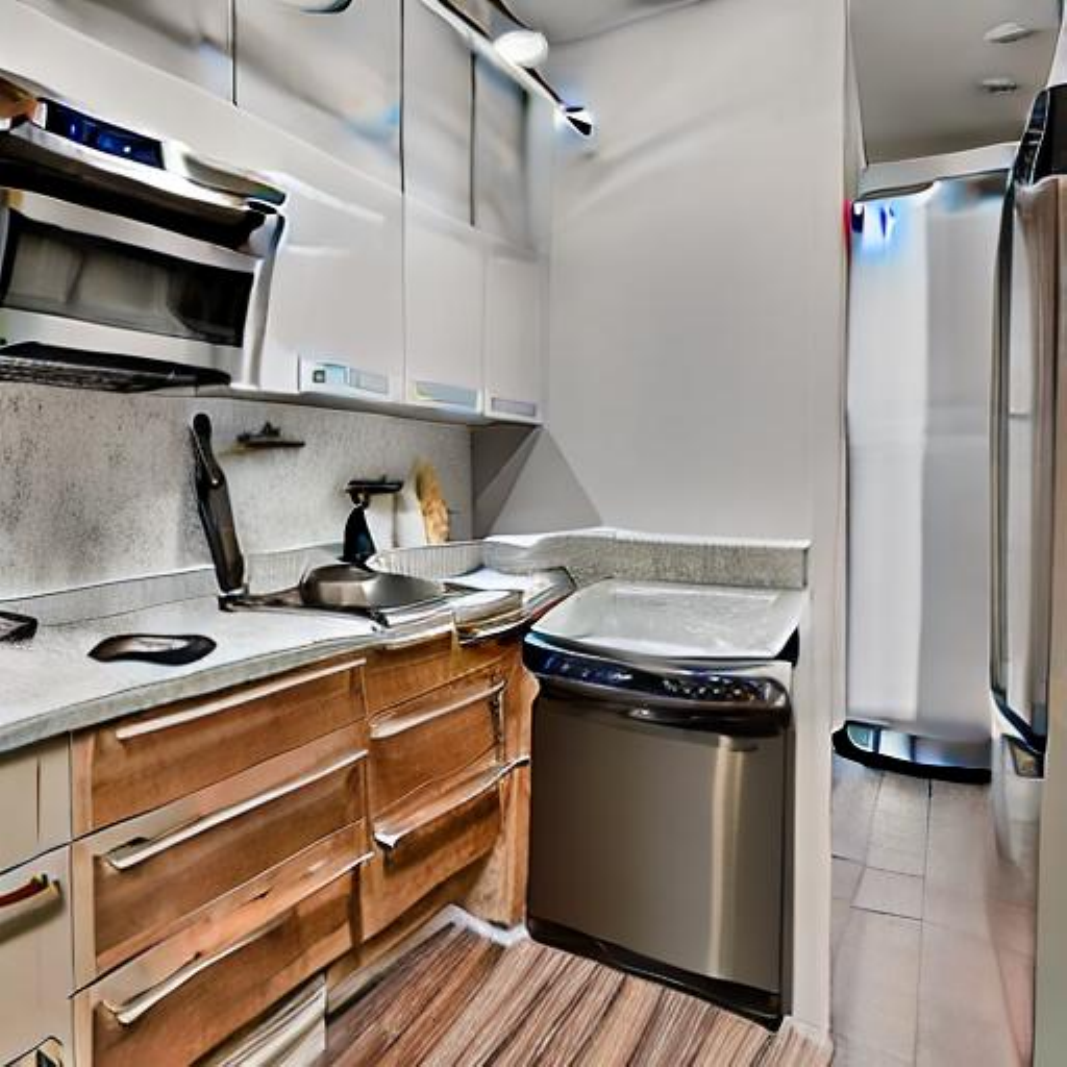} &
    \includegraphics[width=0.10\textwidth]{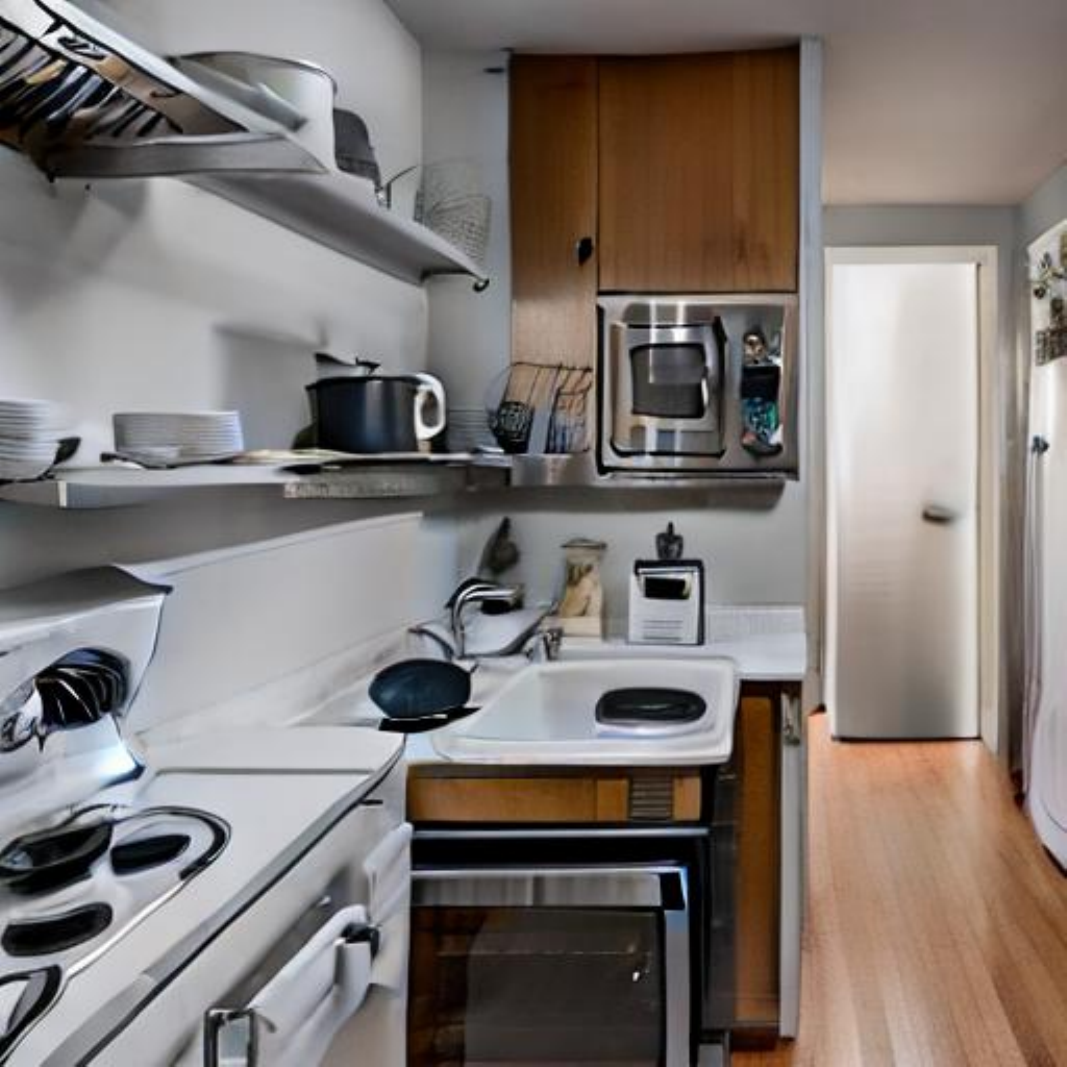} &
    \includegraphics[width=0.10\textwidth]{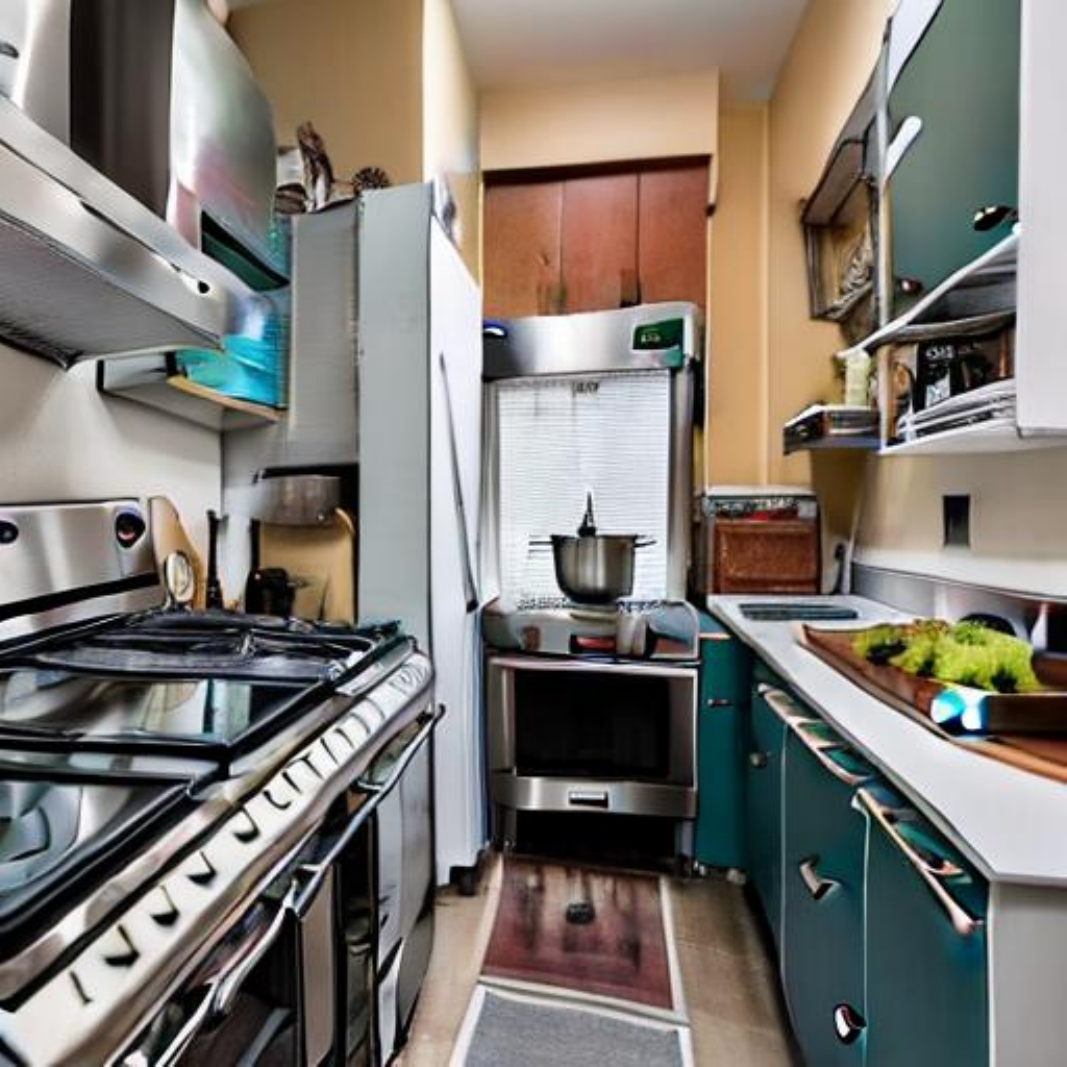} &
    \includegraphics[width=0.10\textwidth]{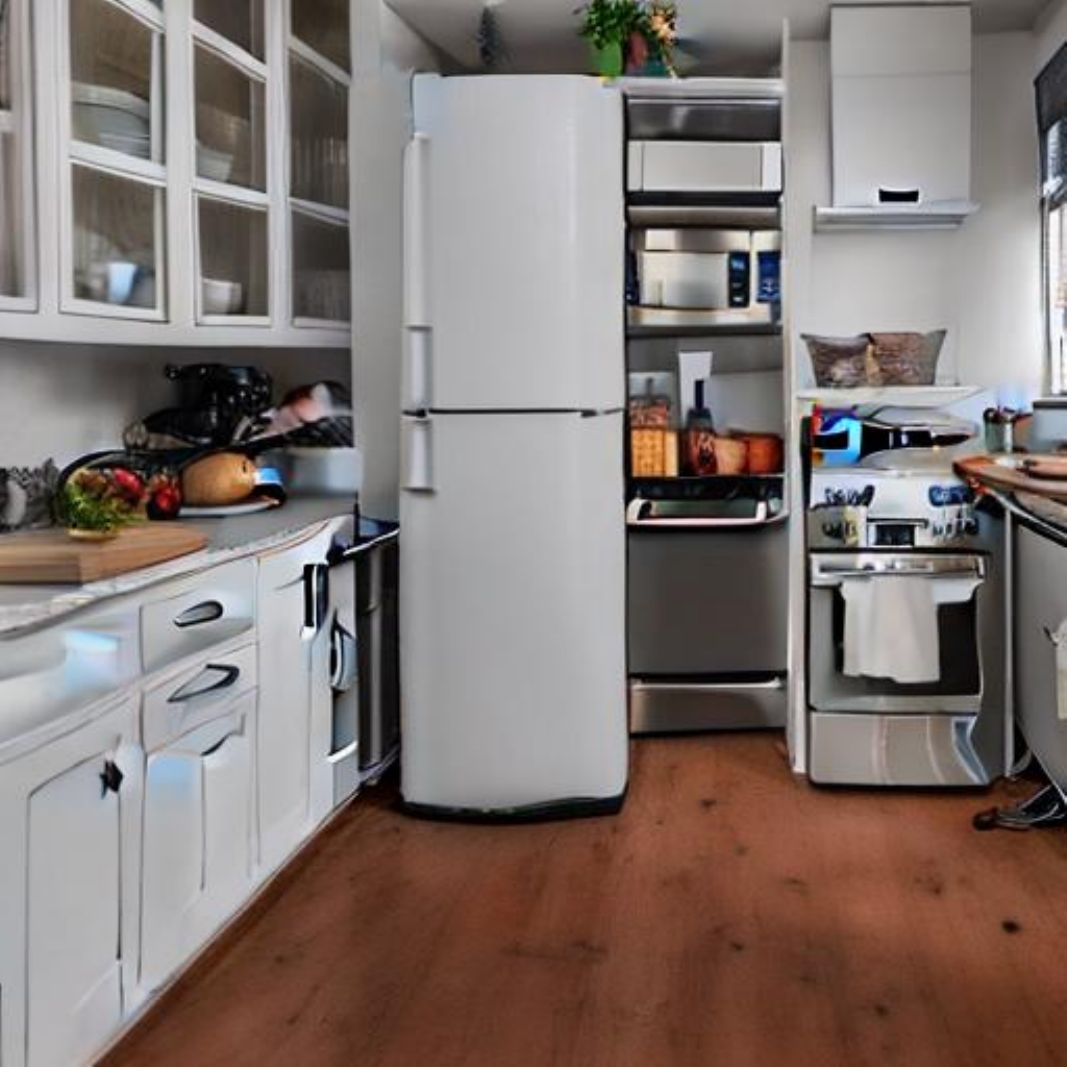} \\
};
\node[below=0pt of tr] {Medium 1};

\matrix (bl) [gridmat, below=12pt of tl]
{
    \includegraphics[width=0.10\textwidth]{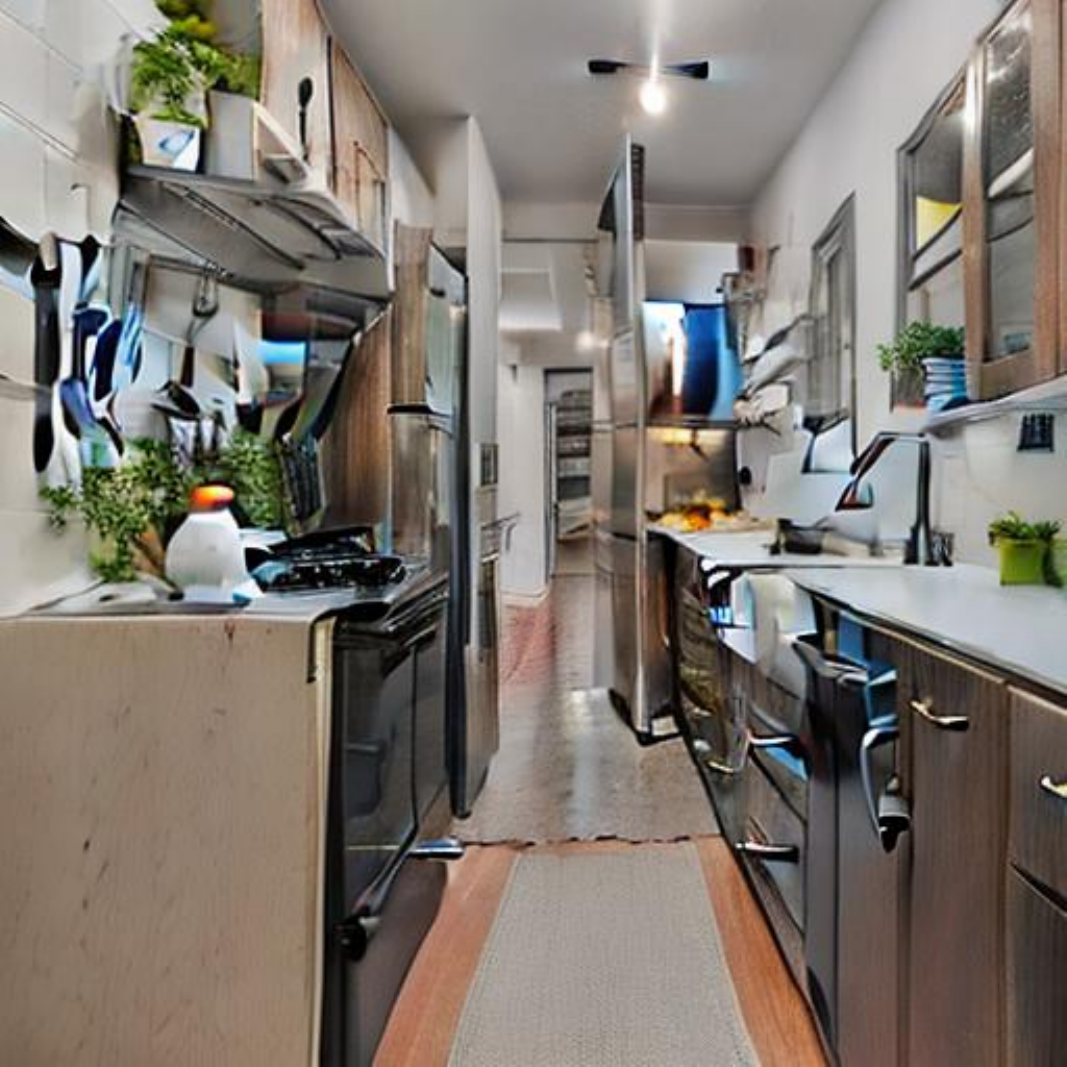} &
    \includegraphics[width=0.10\textwidth]{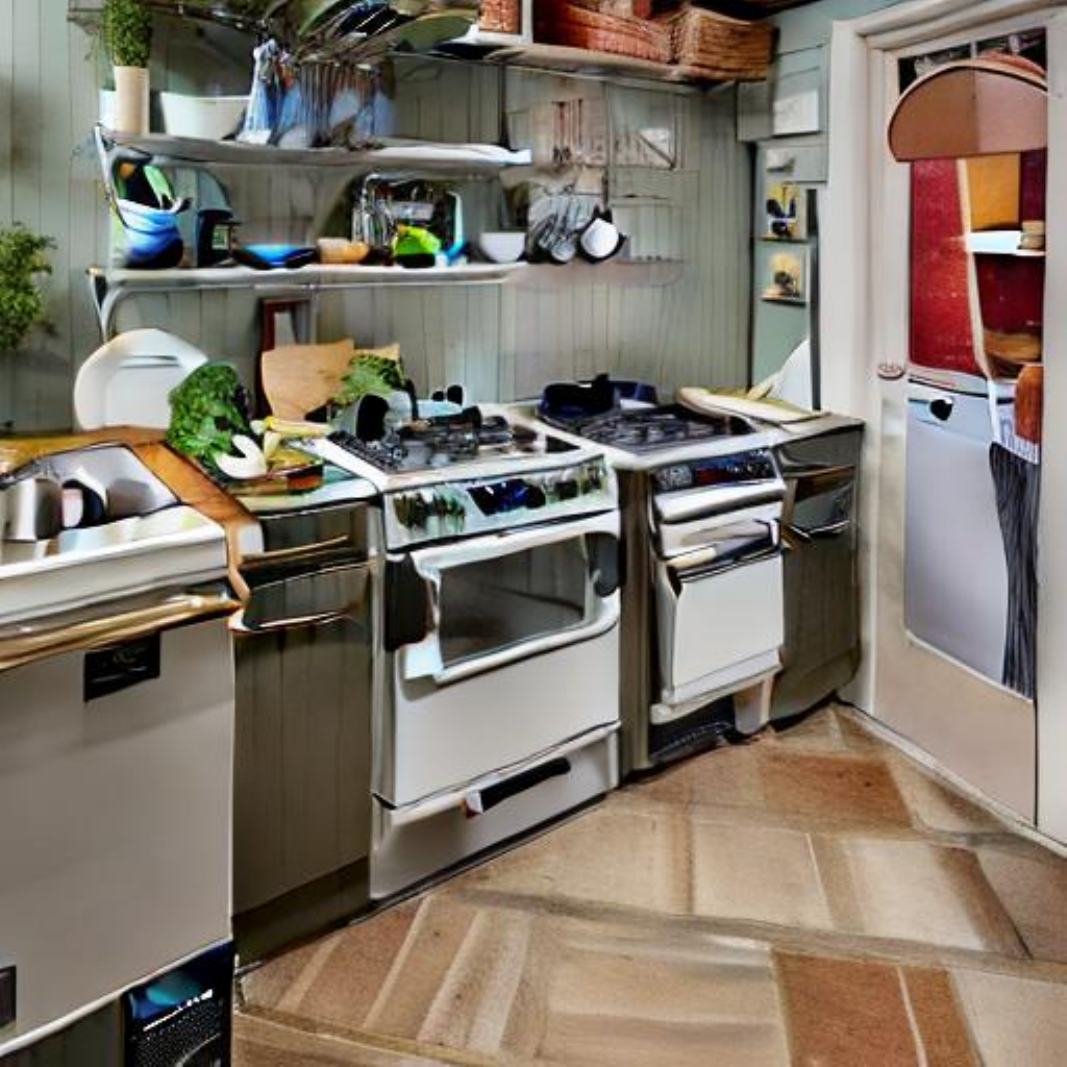} &
    \includegraphics[width=0.10\textwidth]{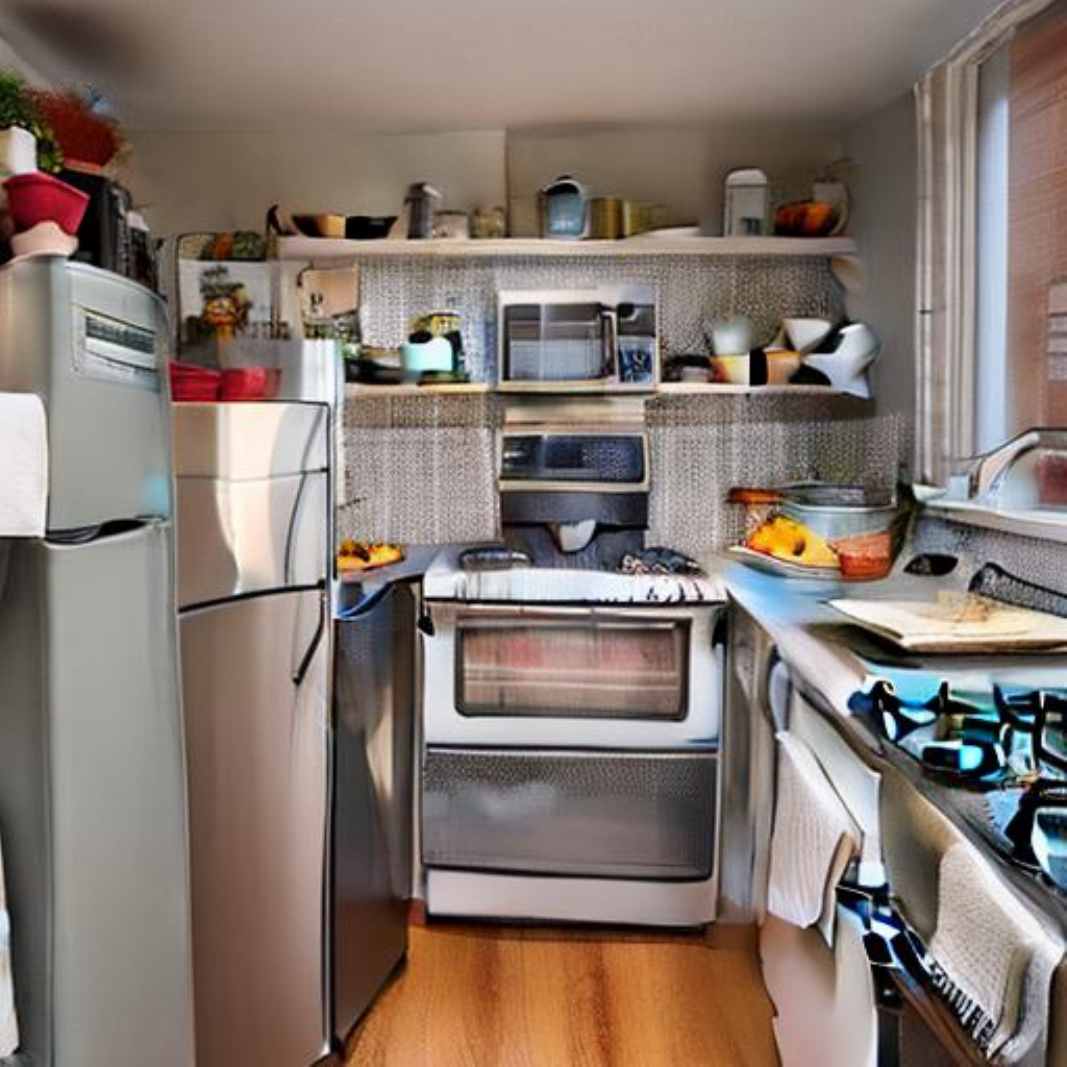} &
    \includegraphics[width=0.10\textwidth]{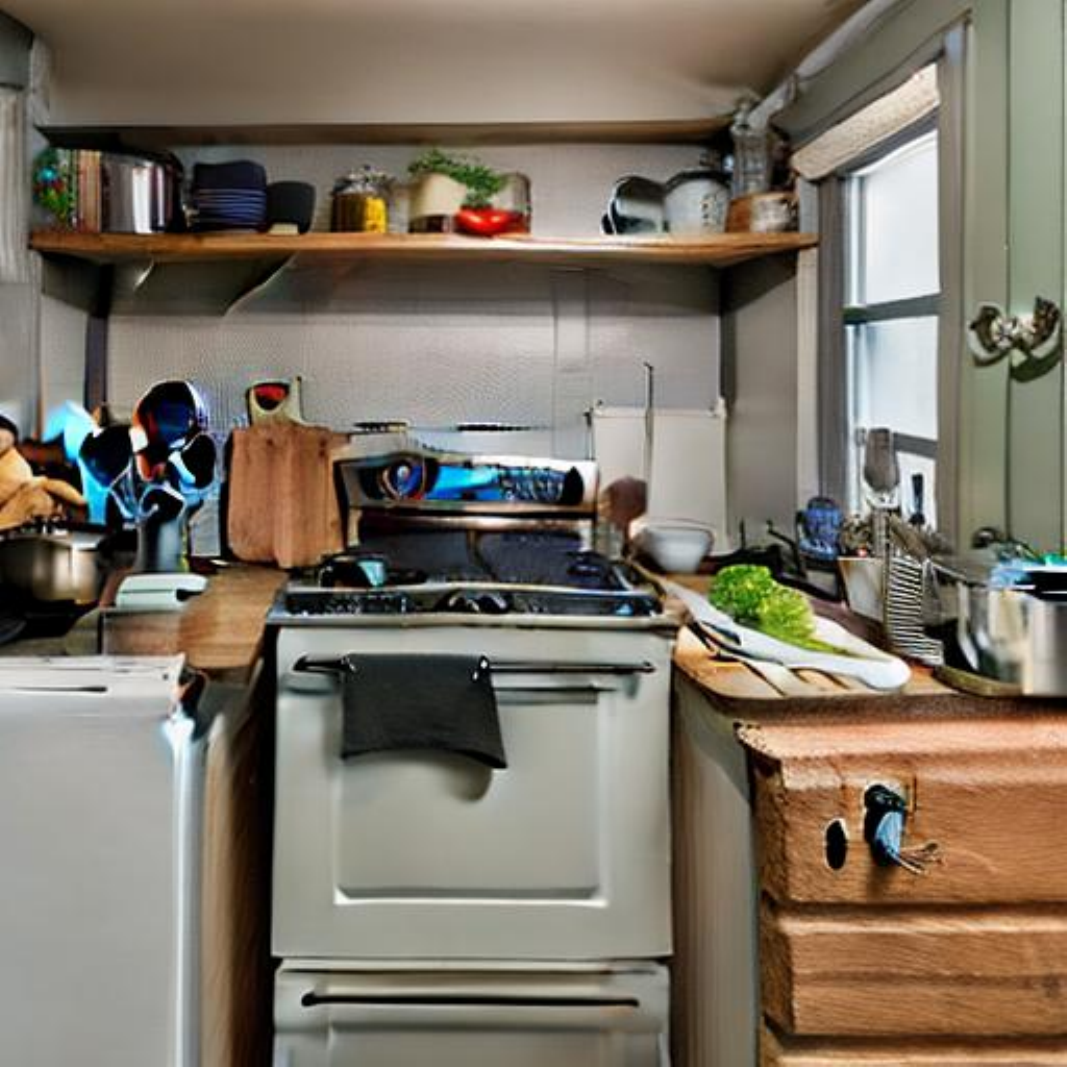} \\
    \includegraphics[width=0.10\textwidth]{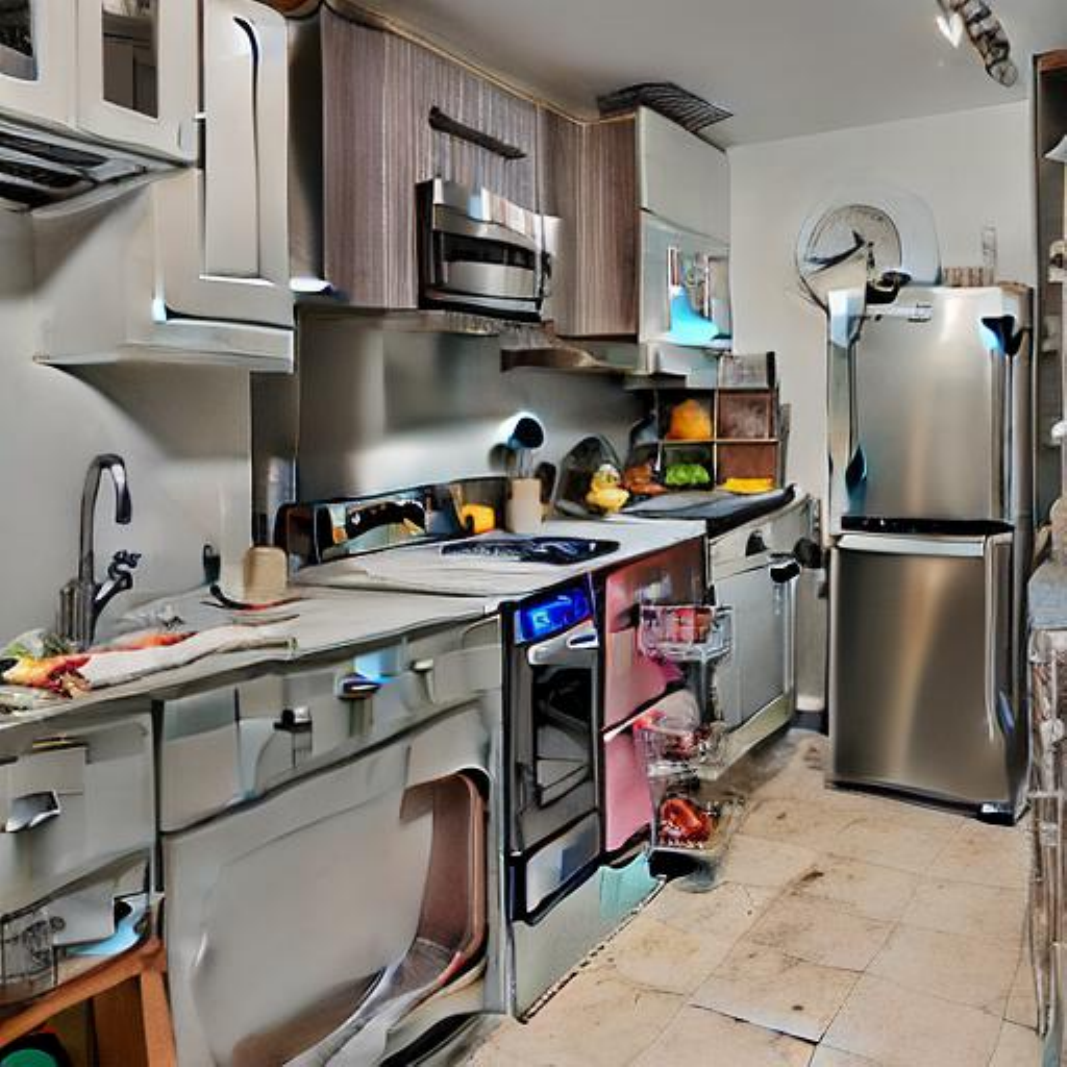} &
    \includegraphics[width=0.10\textwidth]{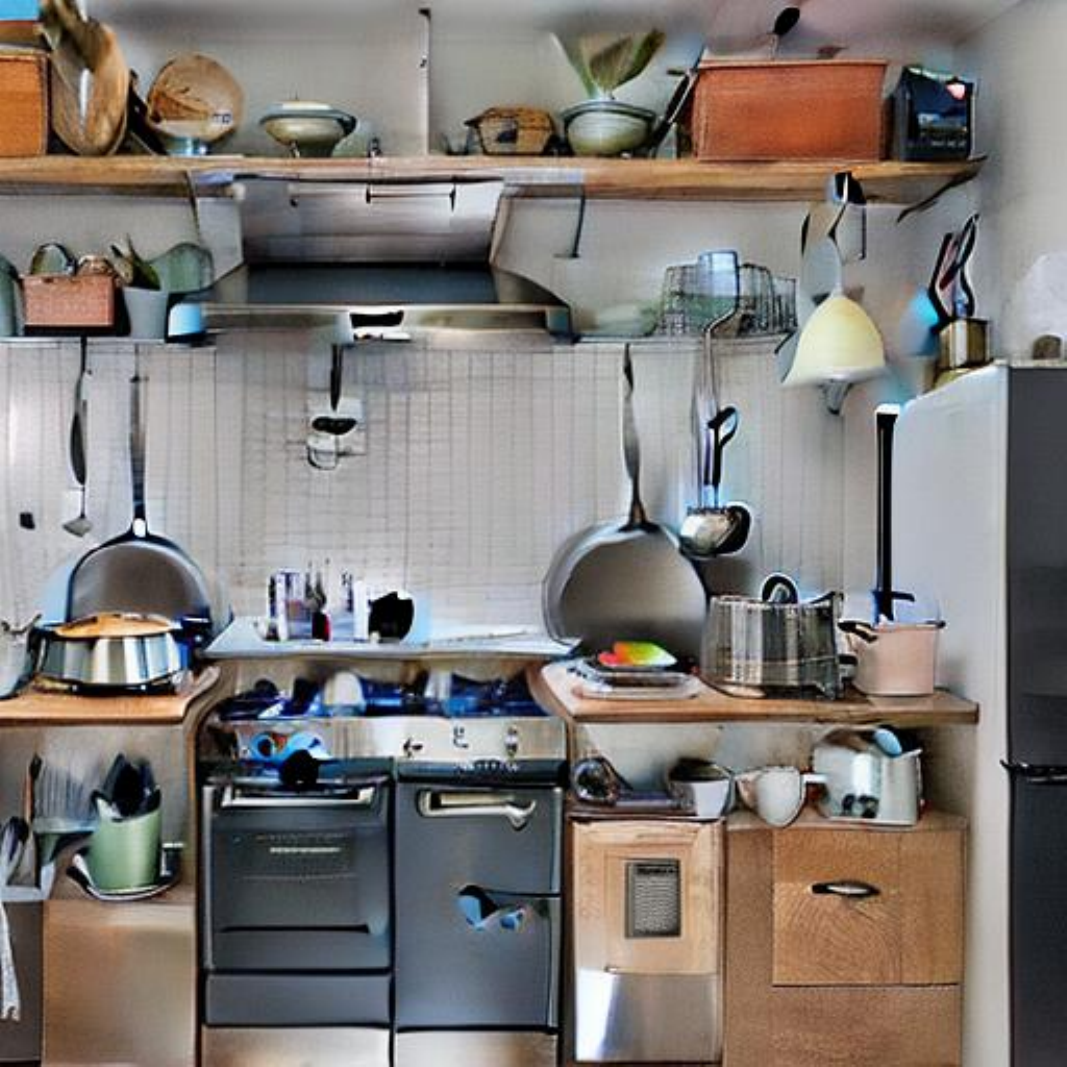} &
    \includegraphics[width=0.10\textwidth]{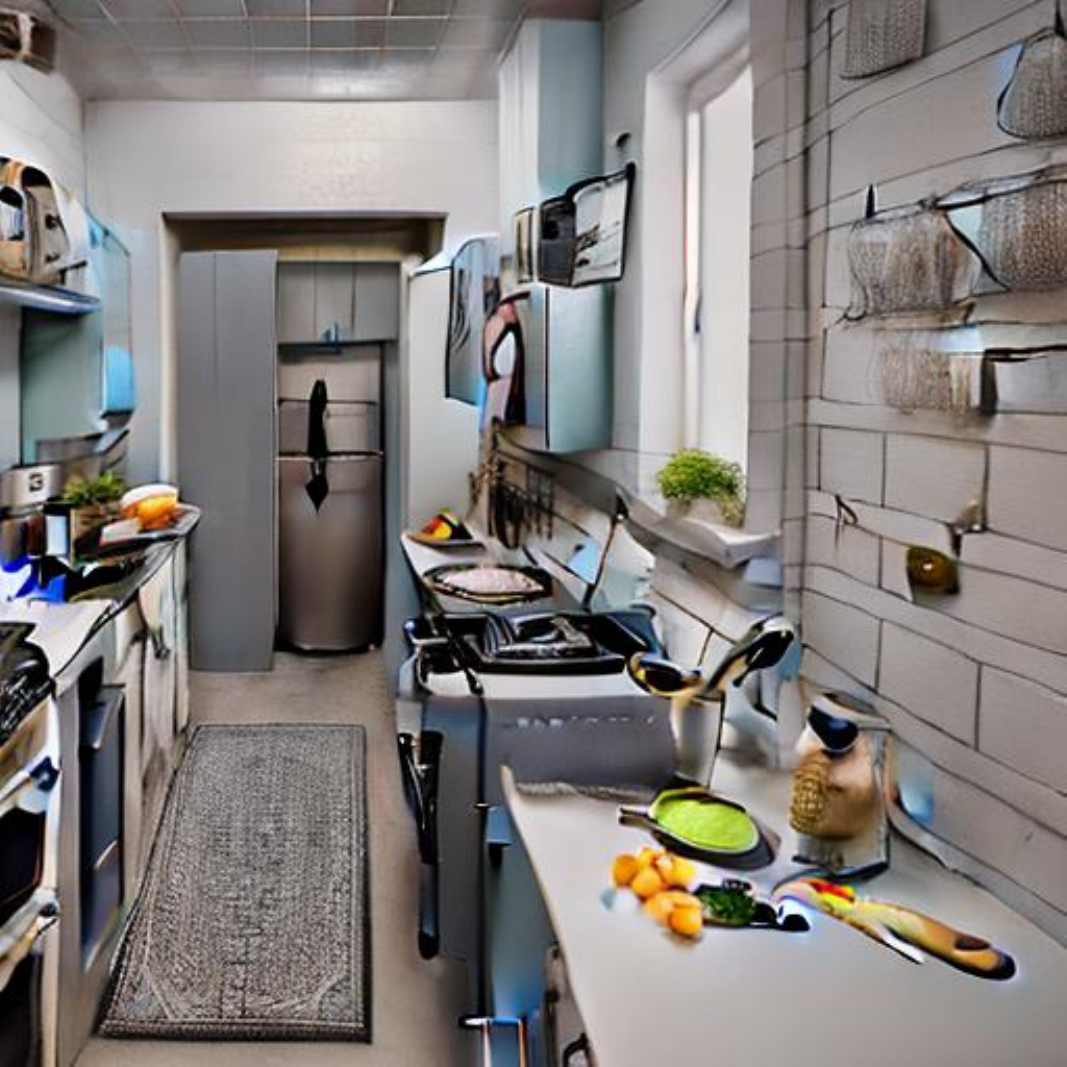} &
    \includegraphics[width=0.10\textwidth]{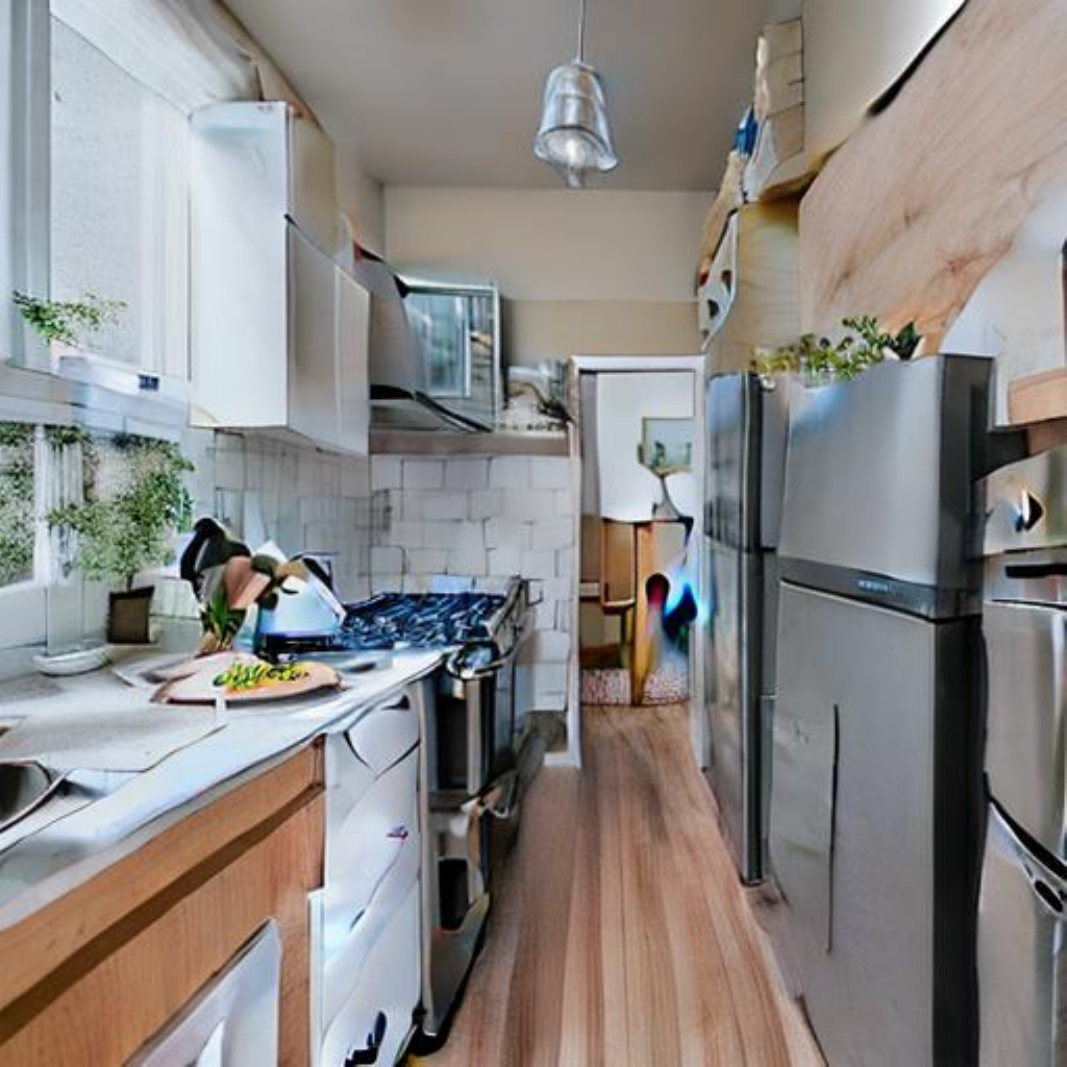} \\
    \includegraphics[width=0.10\textwidth]{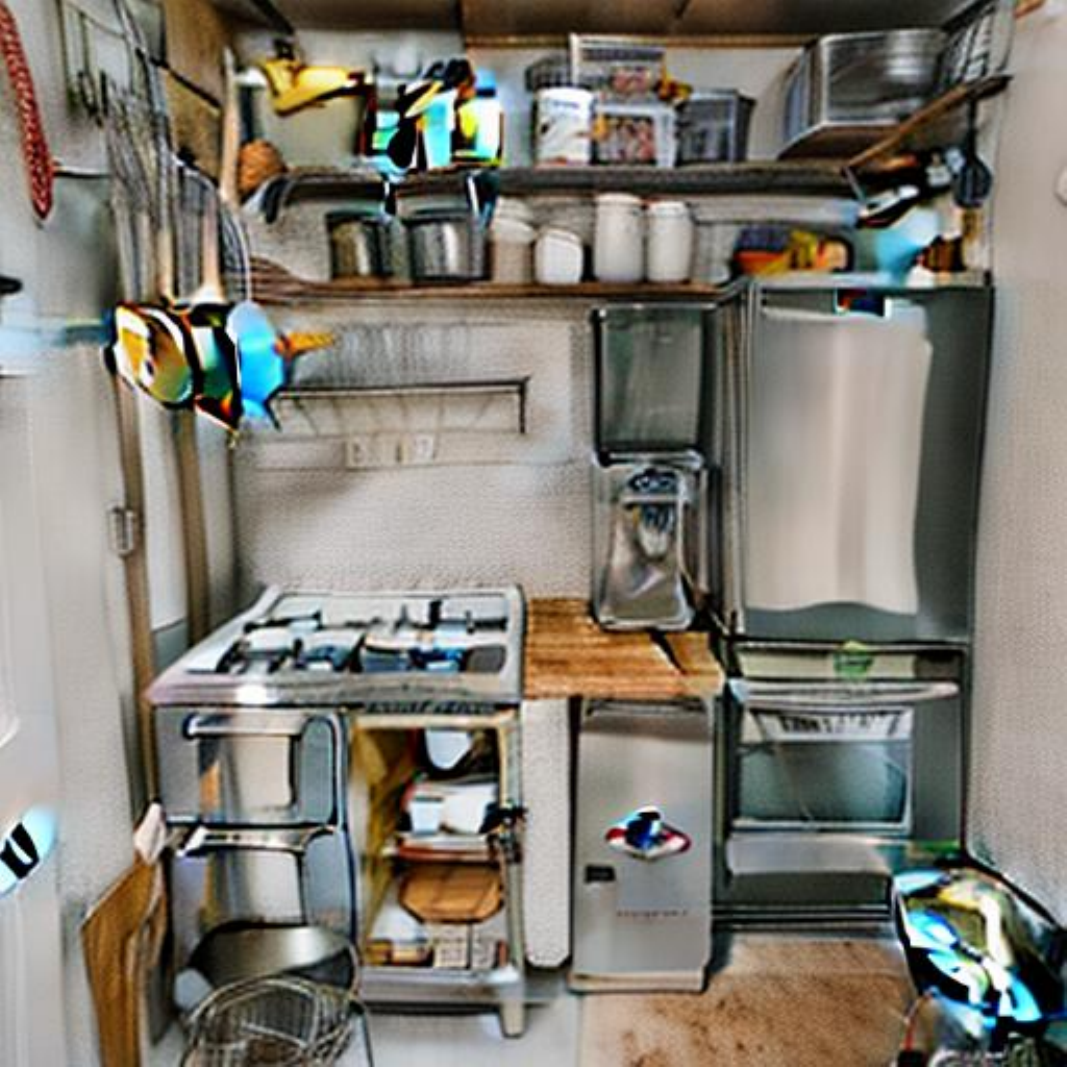} &
    \includegraphics[width=0.10\textwidth]{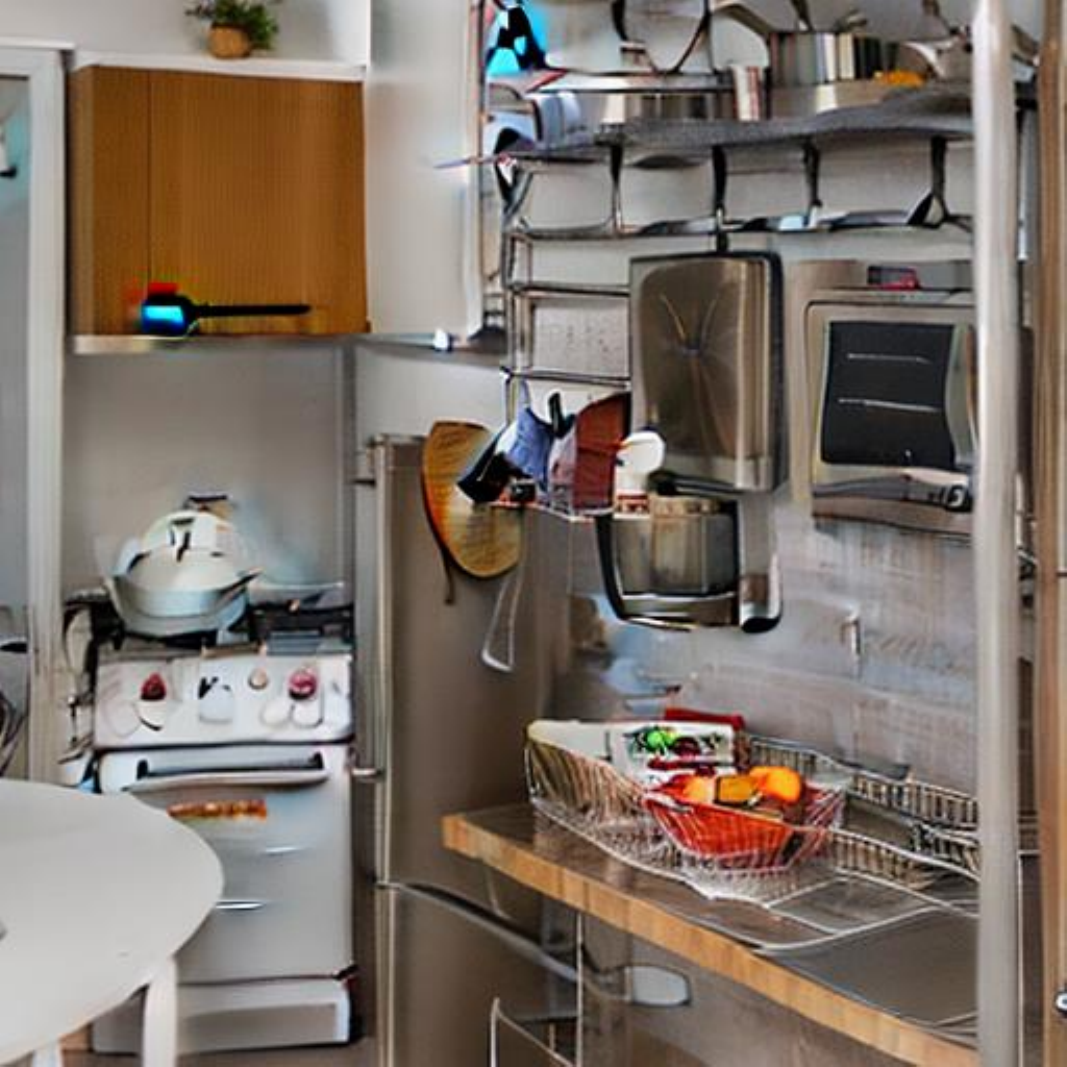} &
    \includegraphics[width=0.10\textwidth]{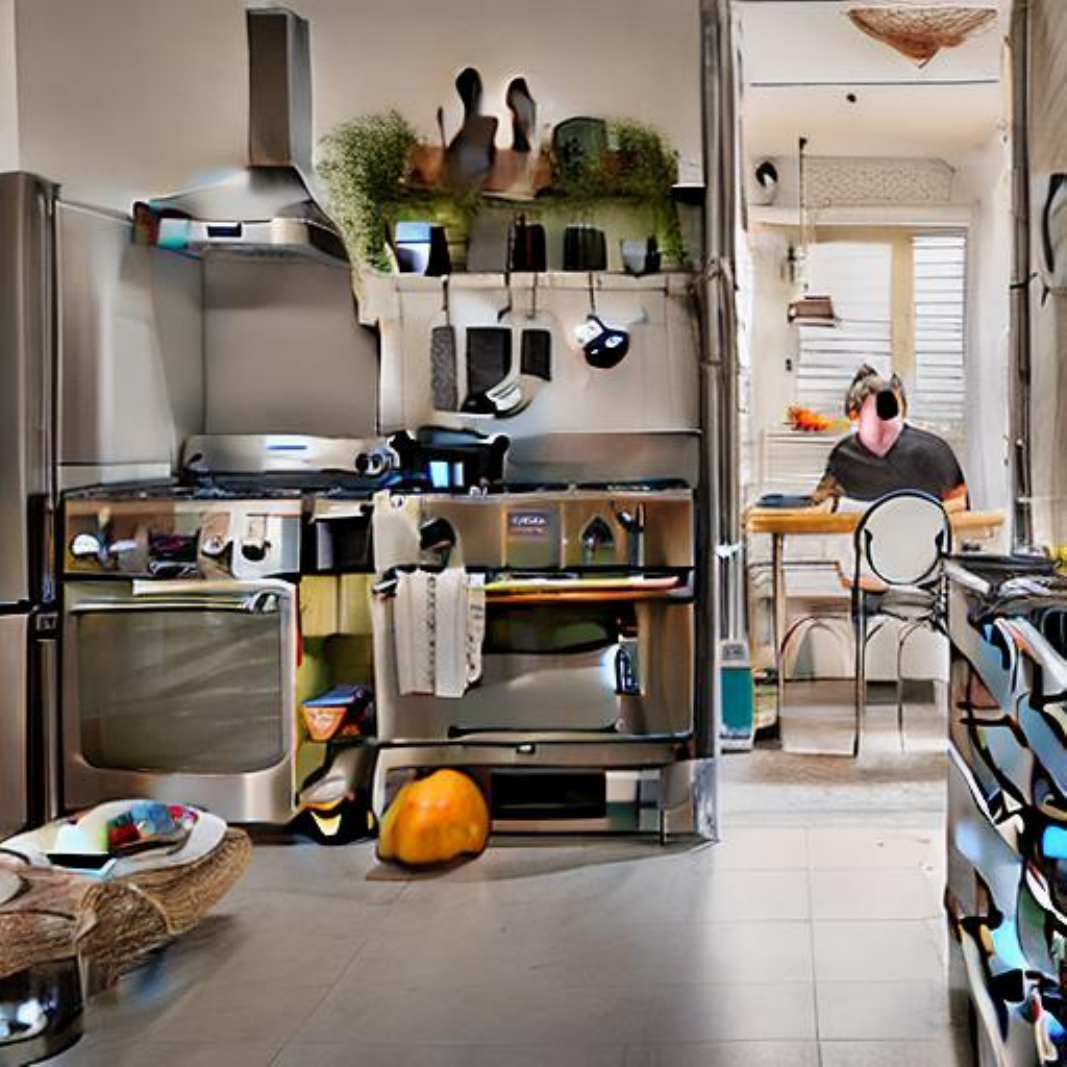} &
    \includegraphics[width=0.10\textwidth]{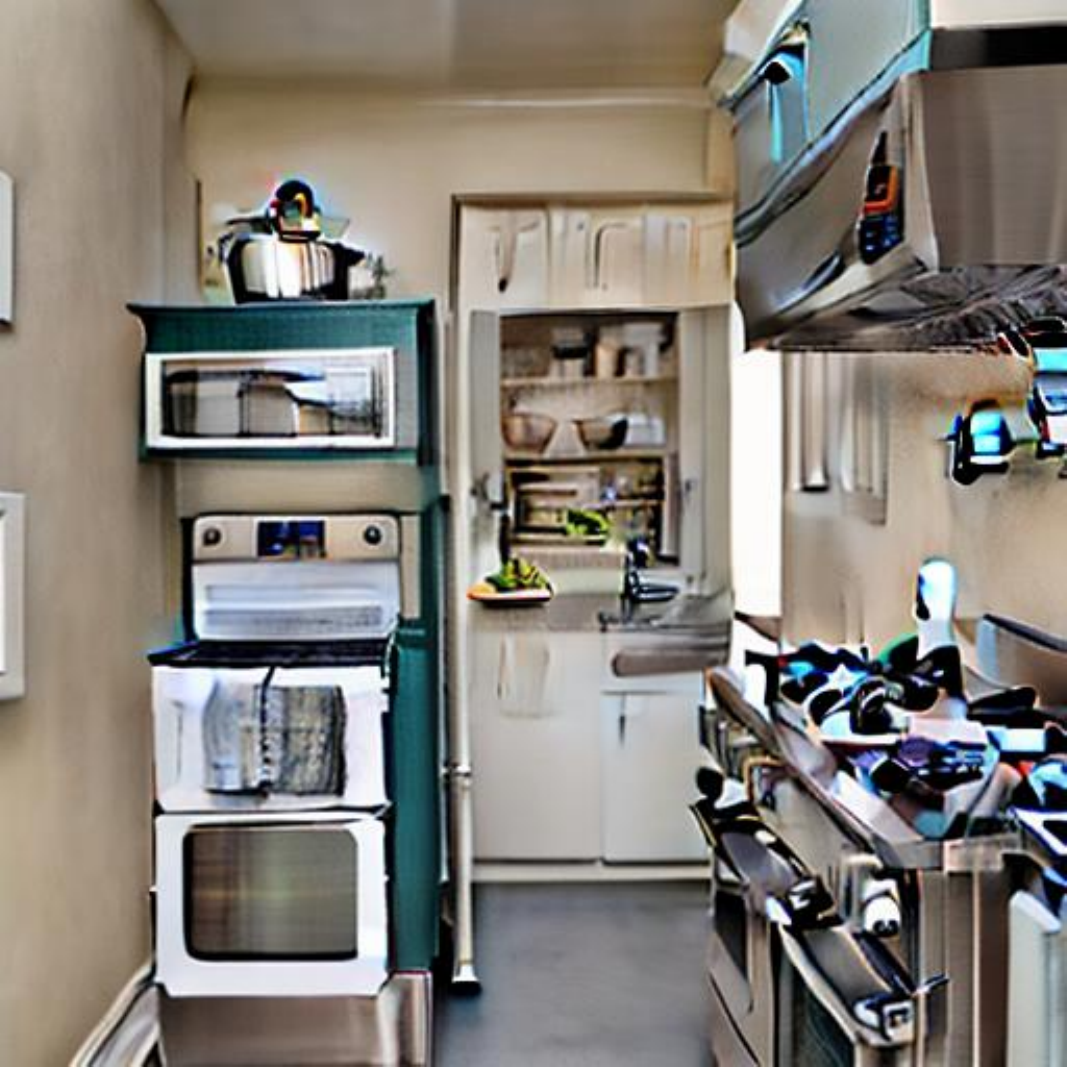} \\
    \includegraphics[width=0.10\textwidth]{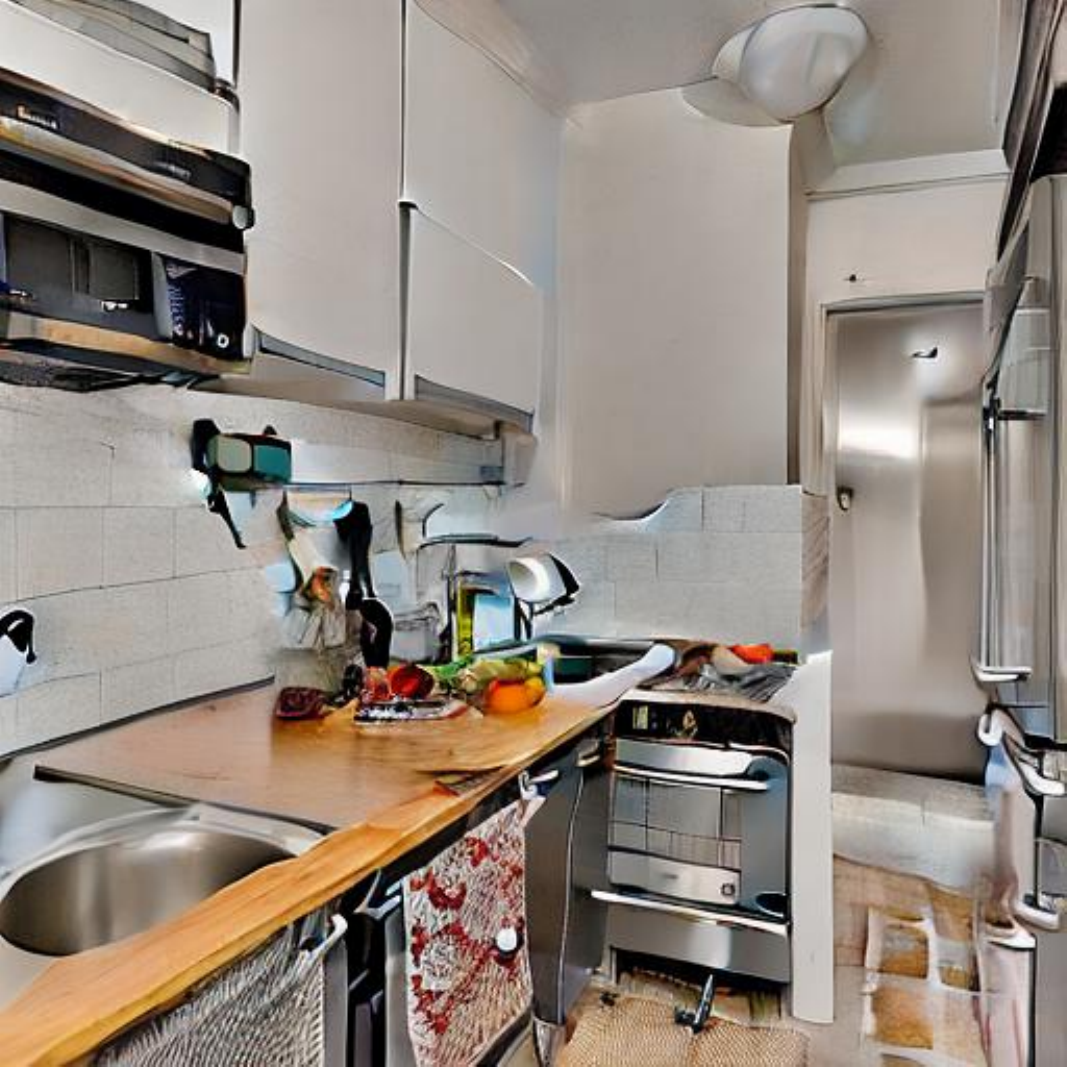} &
    \includegraphics[width=0.10\textwidth]{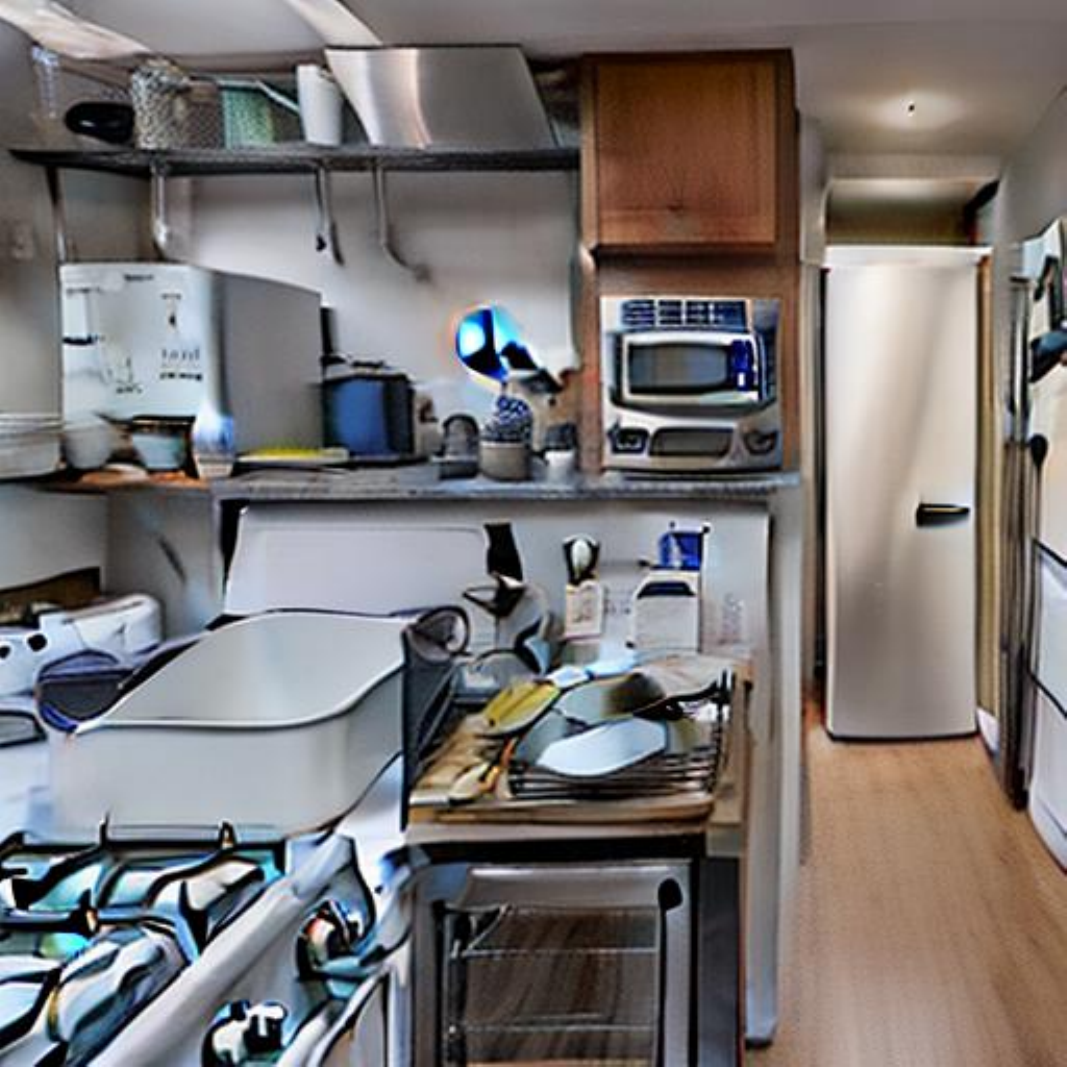} &
    \includegraphics[width=0.10\textwidth]{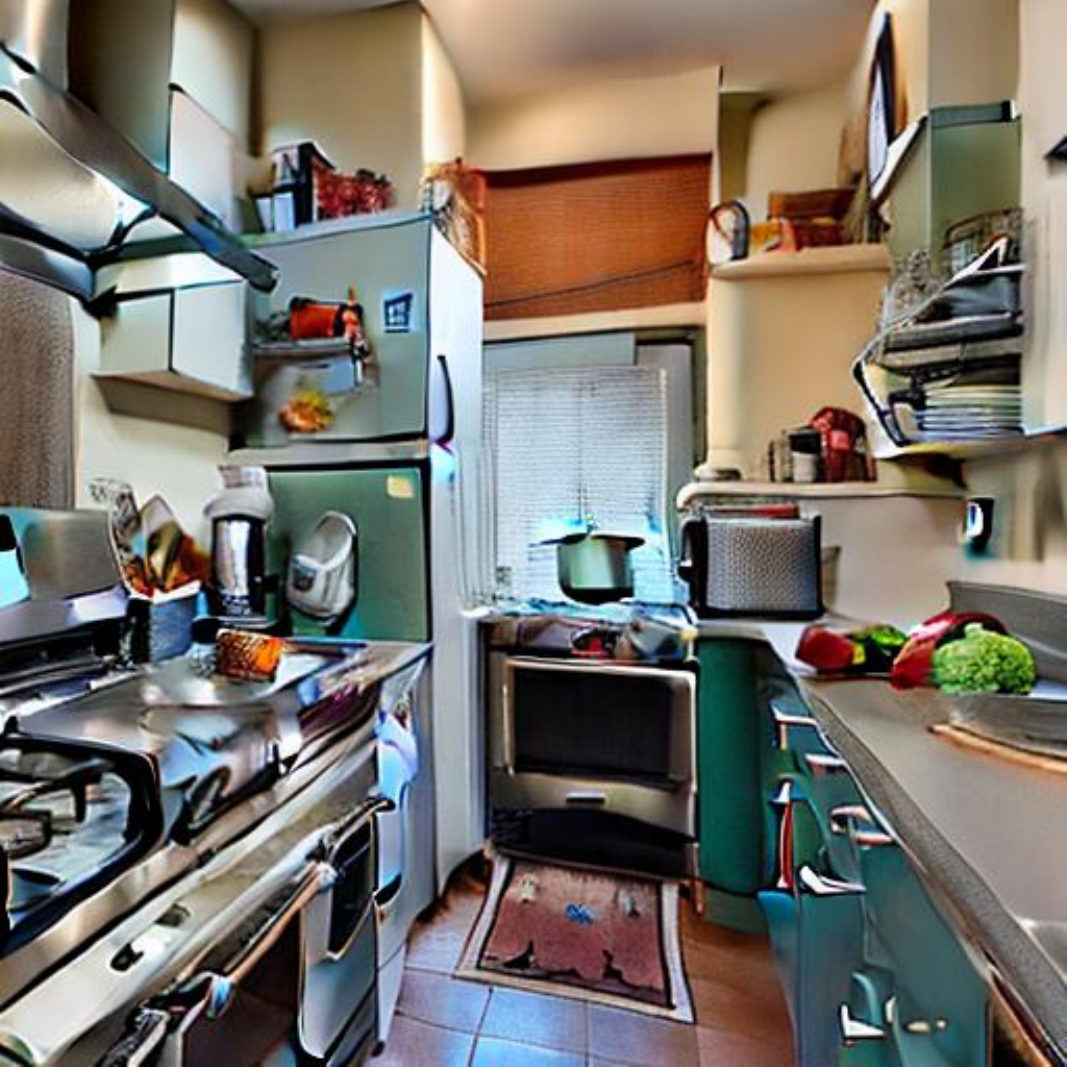} &
    \includegraphics[width=0.10\textwidth]{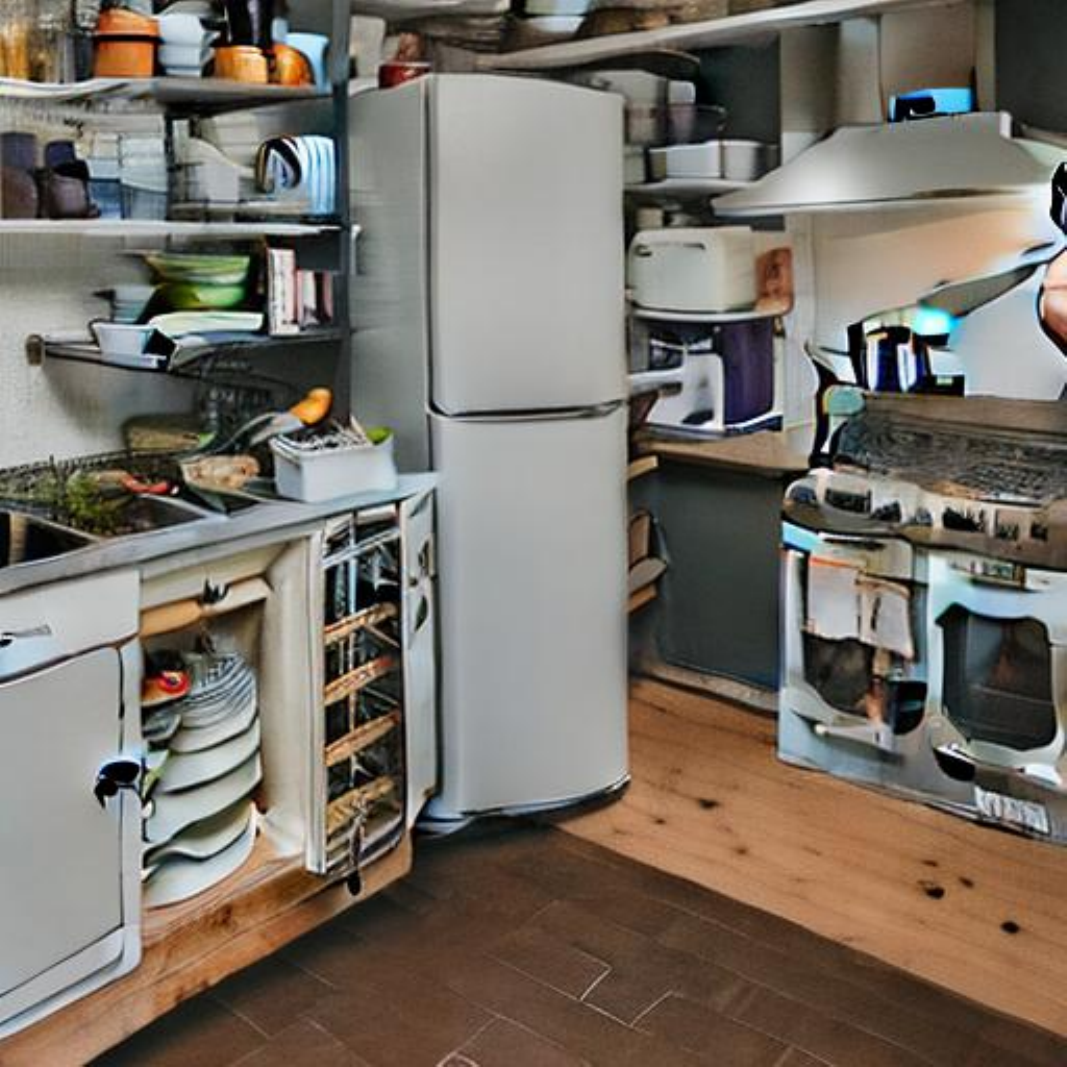} \\
};
\node[below=0pt of bl] {Medium 2};

\matrix (br) [gridmat, right=12pt of bl]
{
    \includegraphics[width=0.10\textwidth]{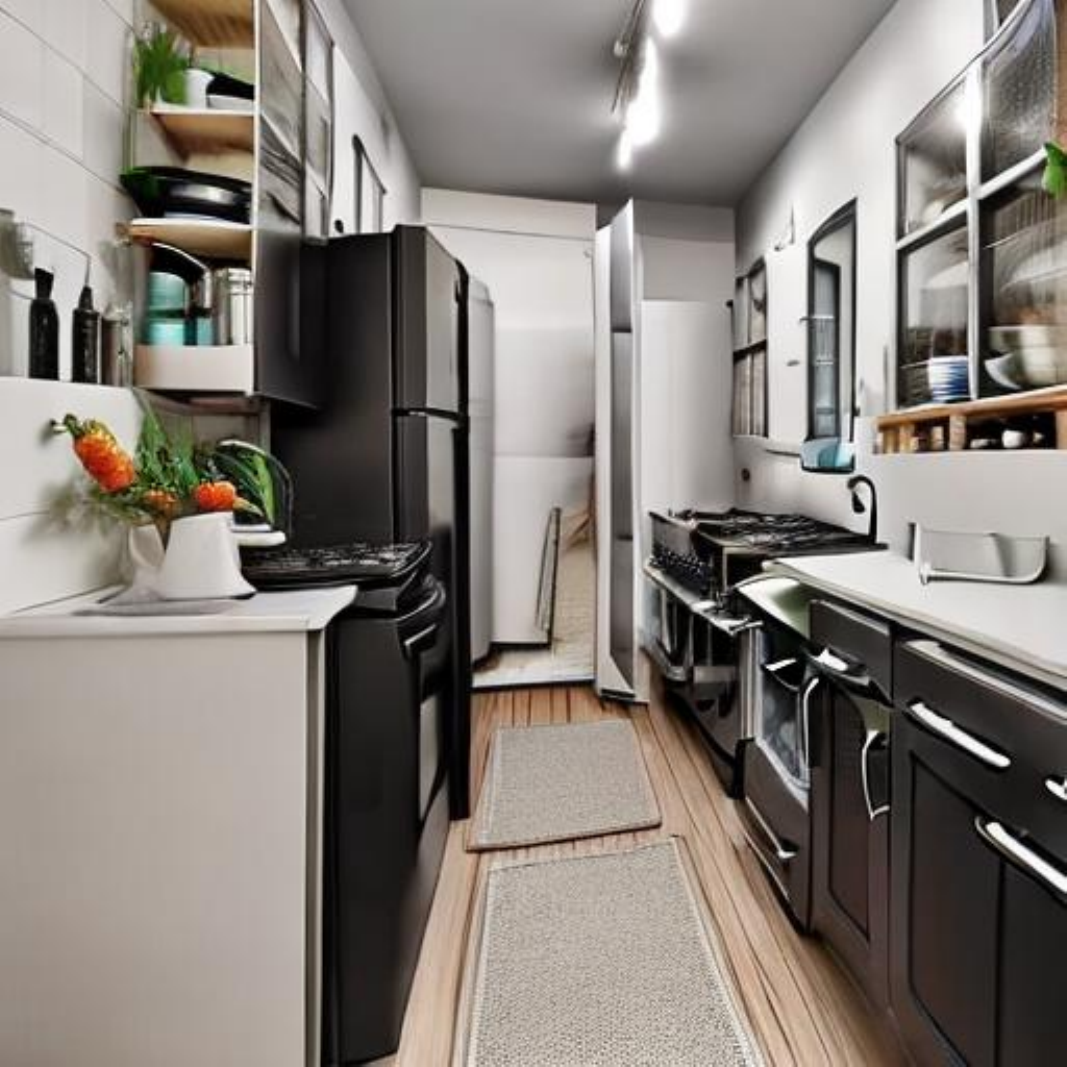} &
    \includegraphics[width=0.10\textwidth]{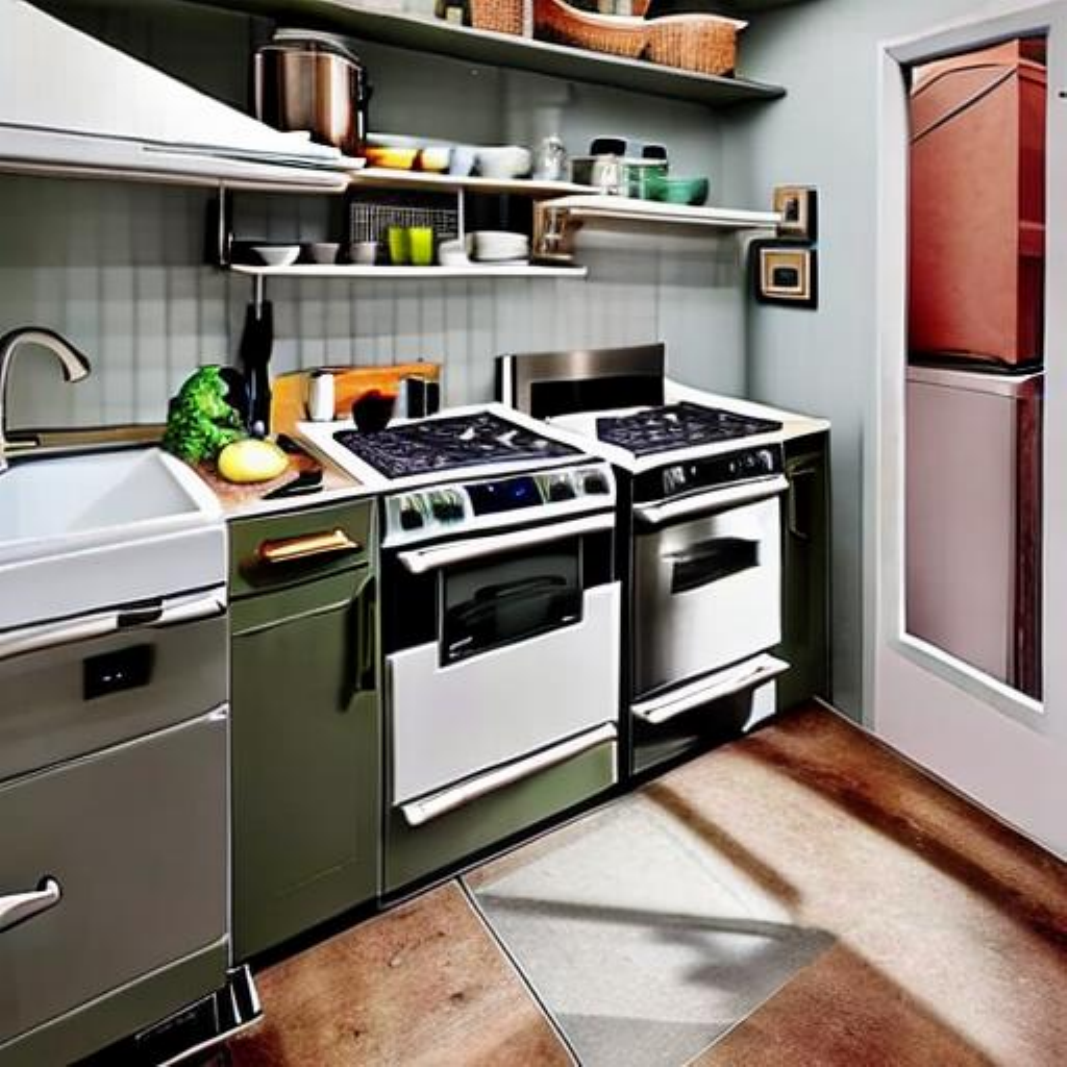} &
    \includegraphics[width=0.10\textwidth]{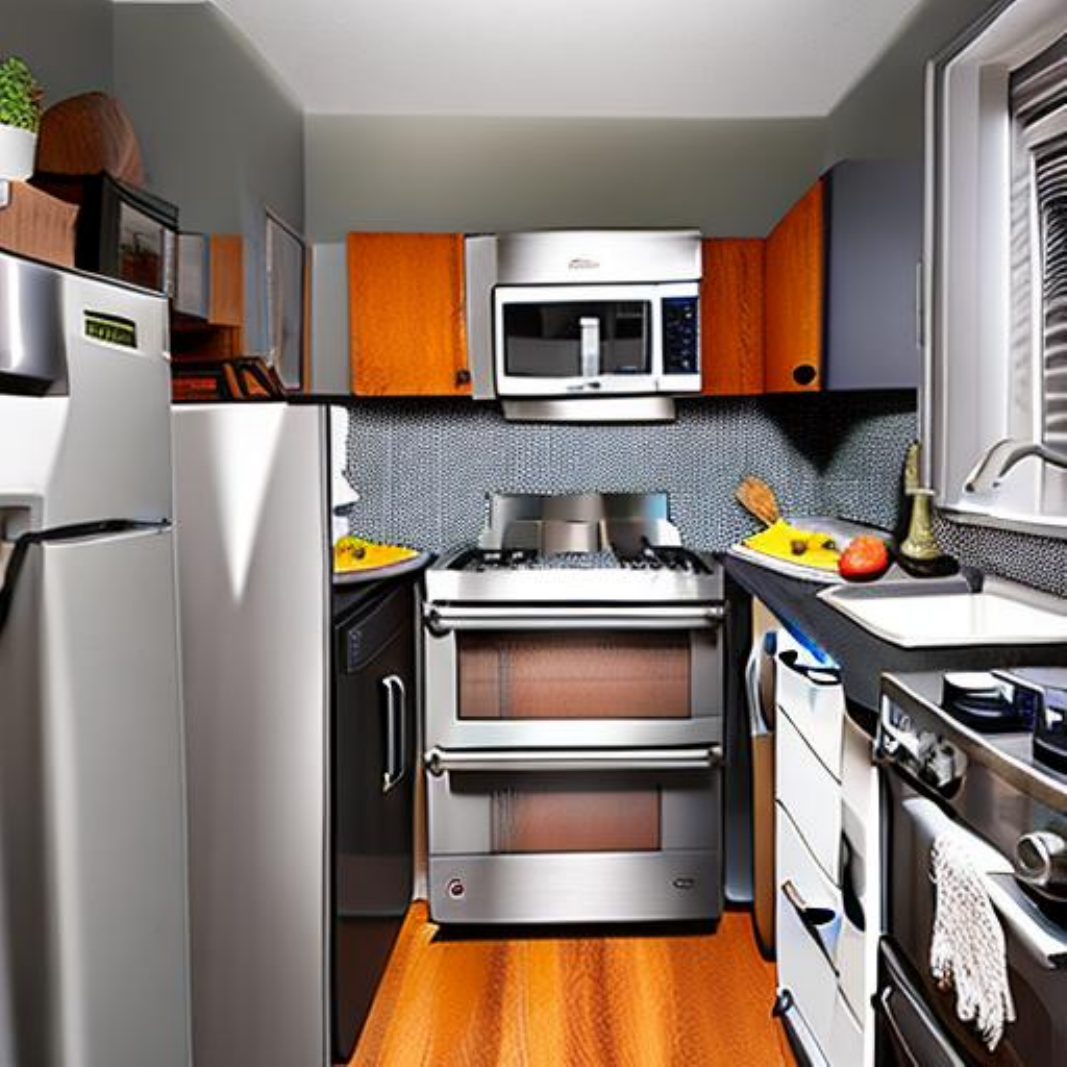} &
    \includegraphics[width=0.10\textwidth]{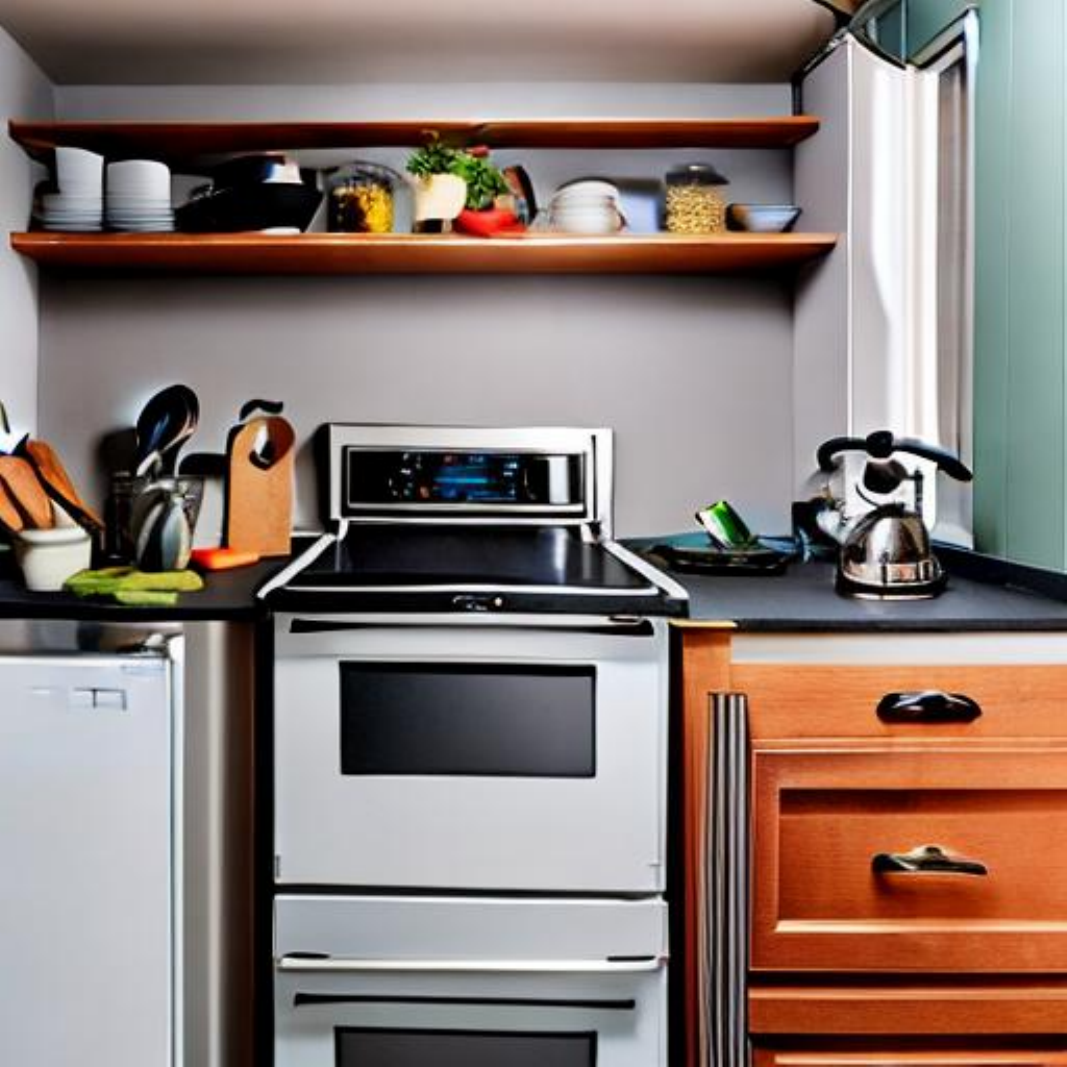} \\
    \includegraphics[width=0.10\textwidth]{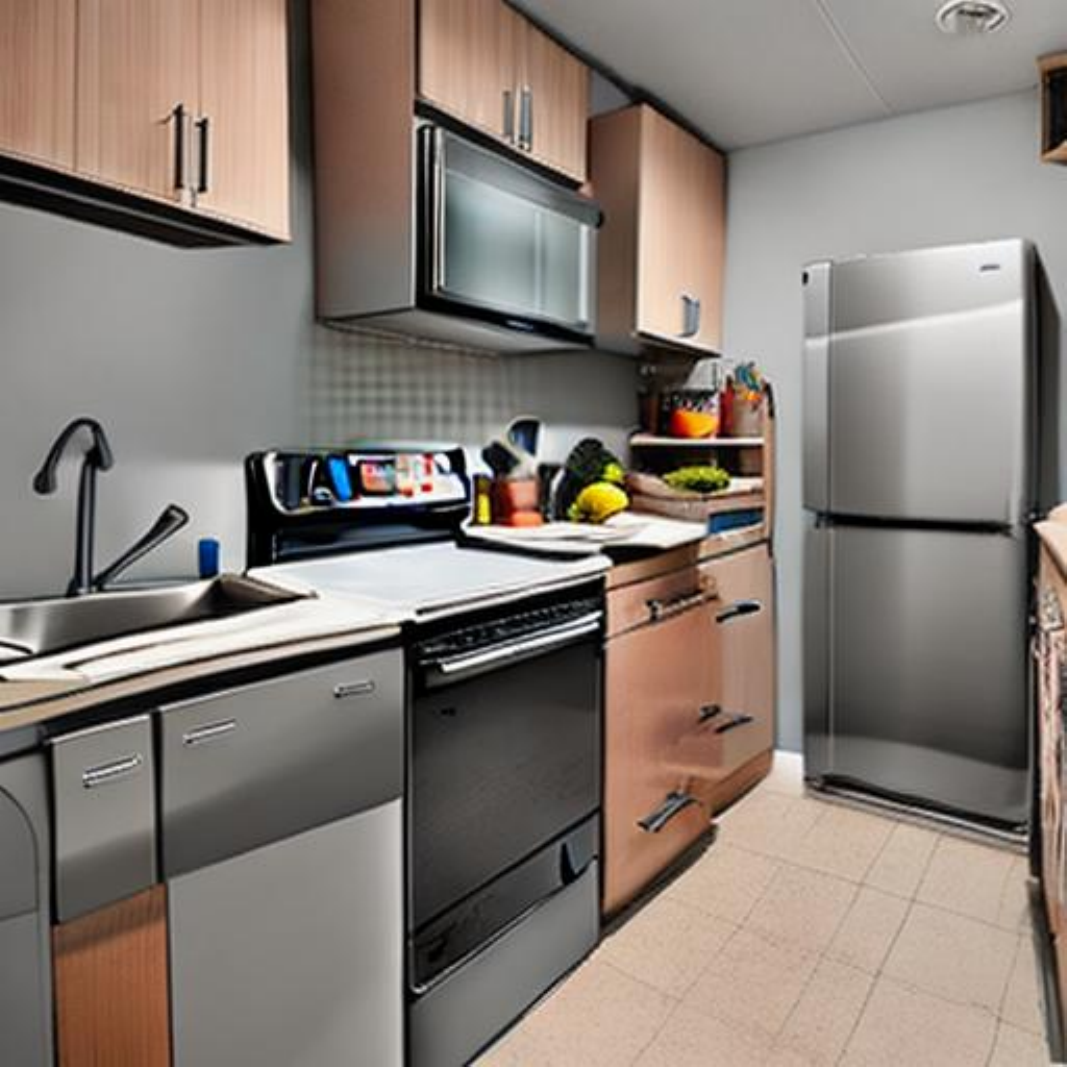} &
    \includegraphics[width=0.10\textwidth]{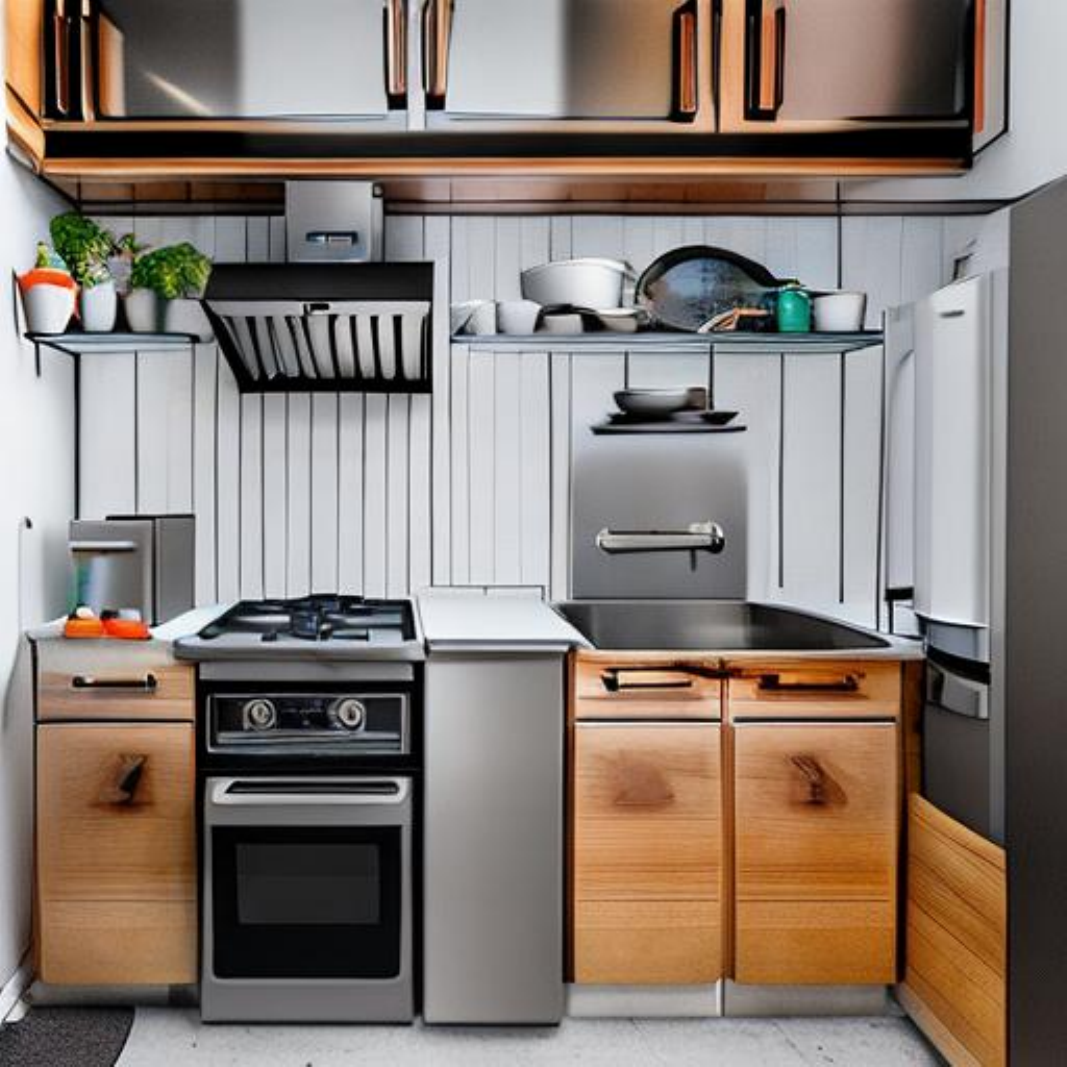} &
    \includegraphics[width=0.10\textwidth]{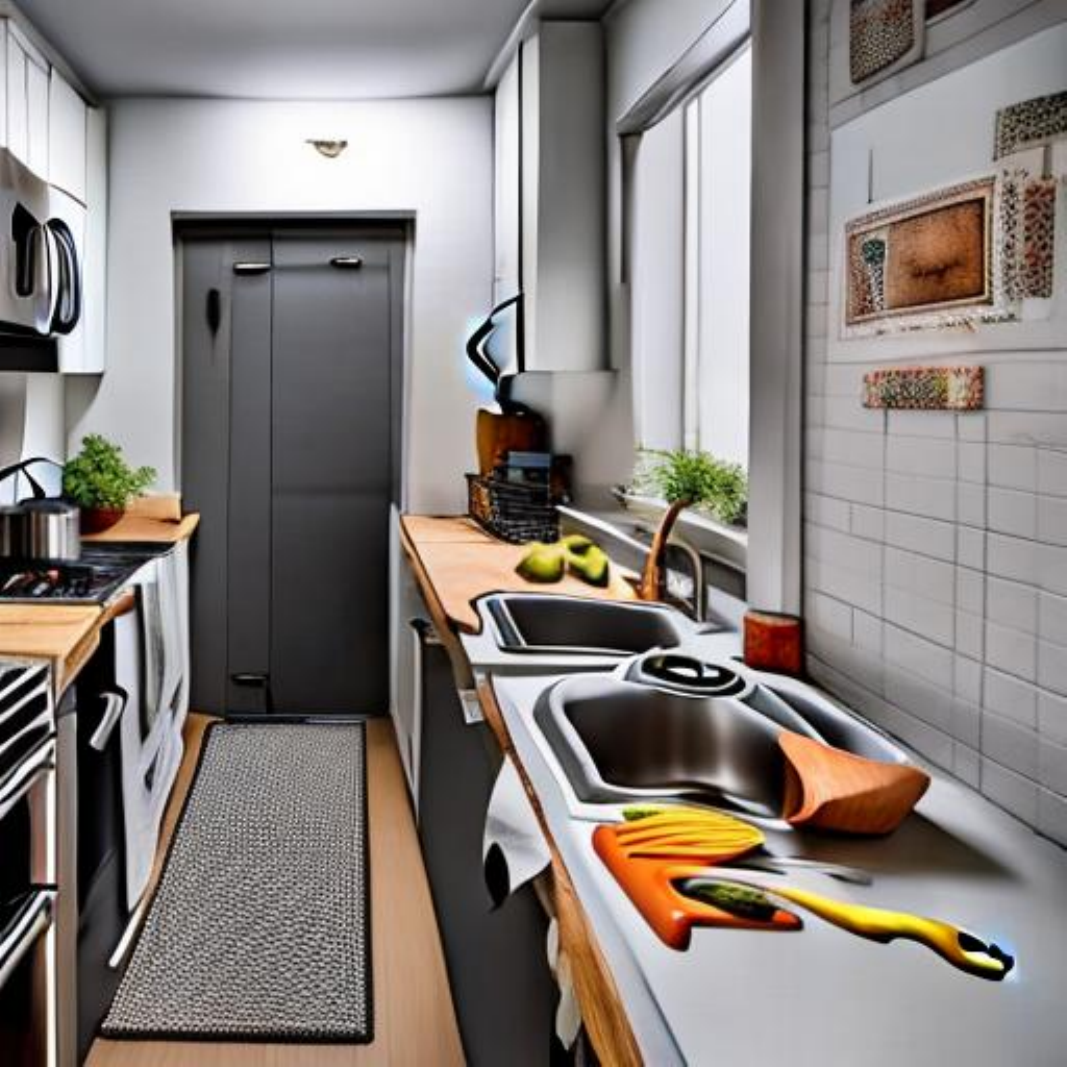} &
    \includegraphics[width=0.10\textwidth]{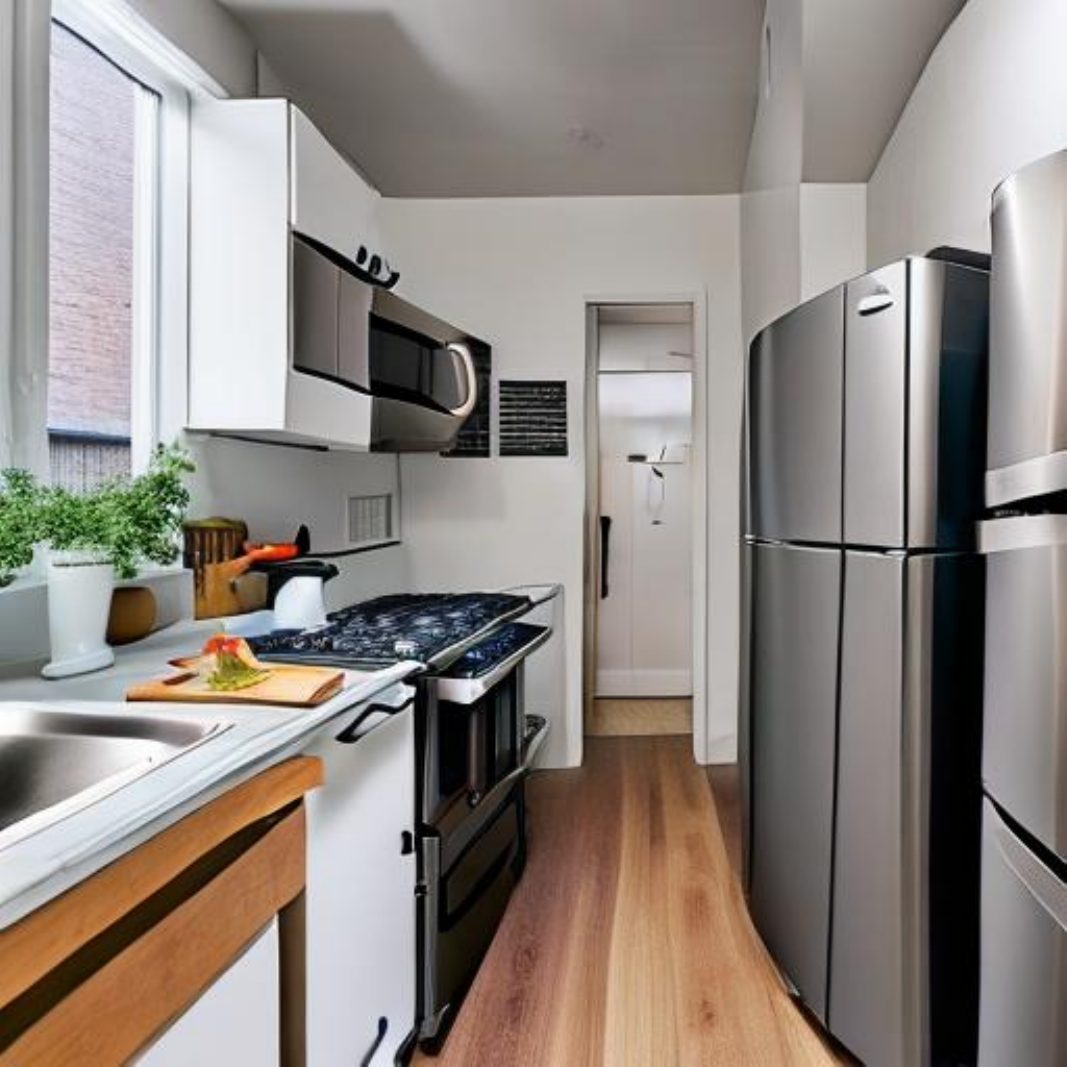} \\
    \includegraphics[width=0.10\textwidth]{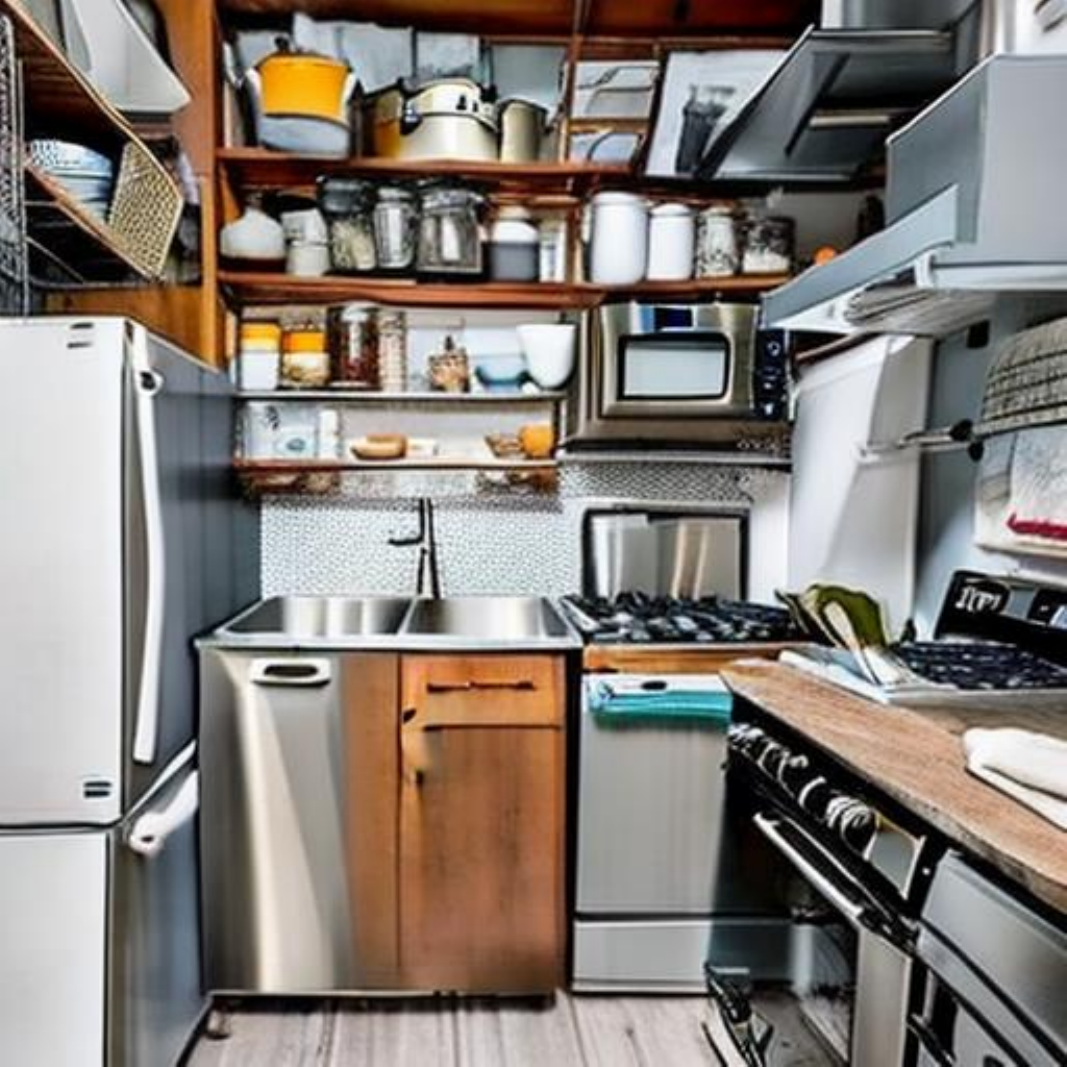} &
    \includegraphics[width=0.10\textwidth]{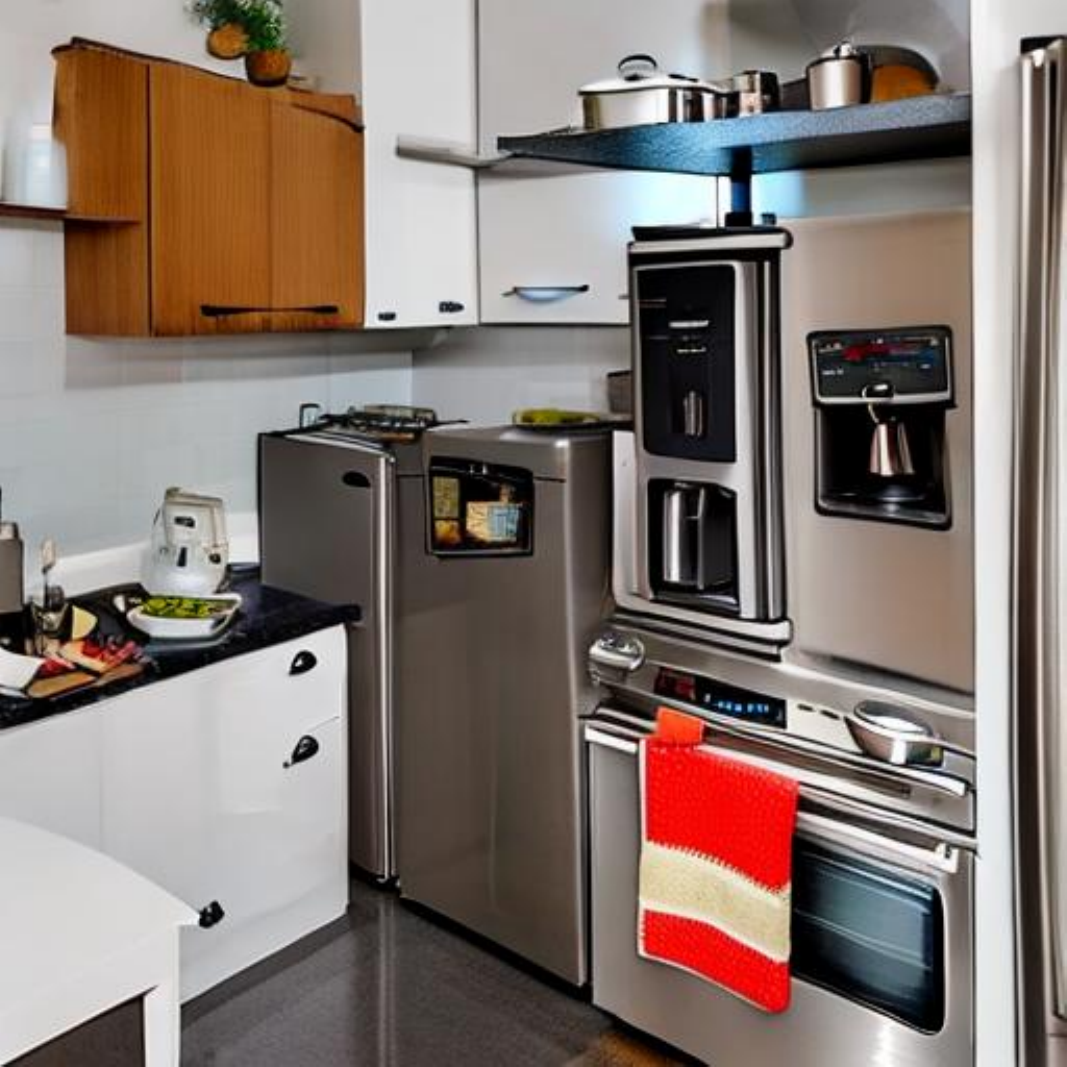} &
    \includegraphics[width=0.10\textwidth]{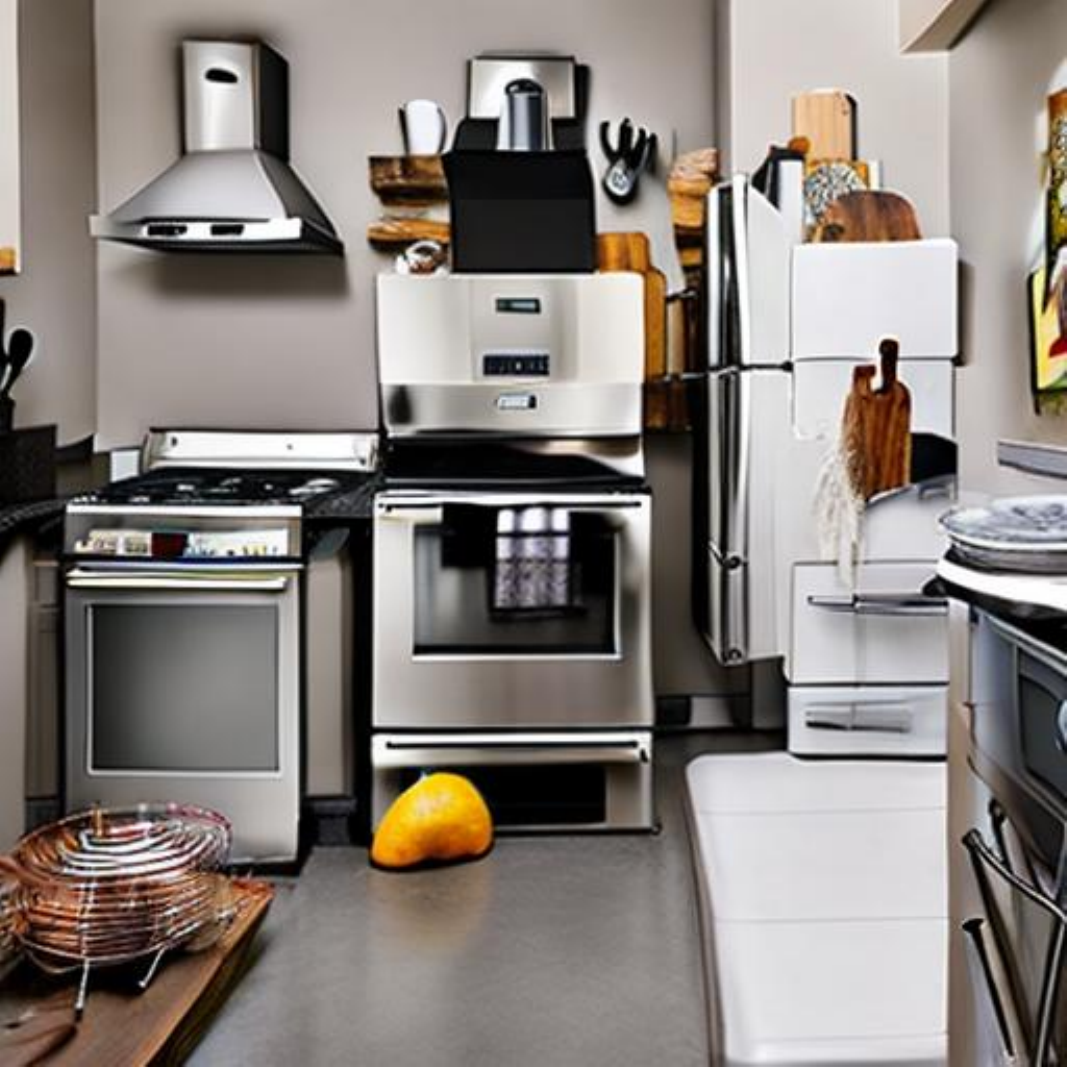} &
    \includegraphics[width=0.10\textwidth]{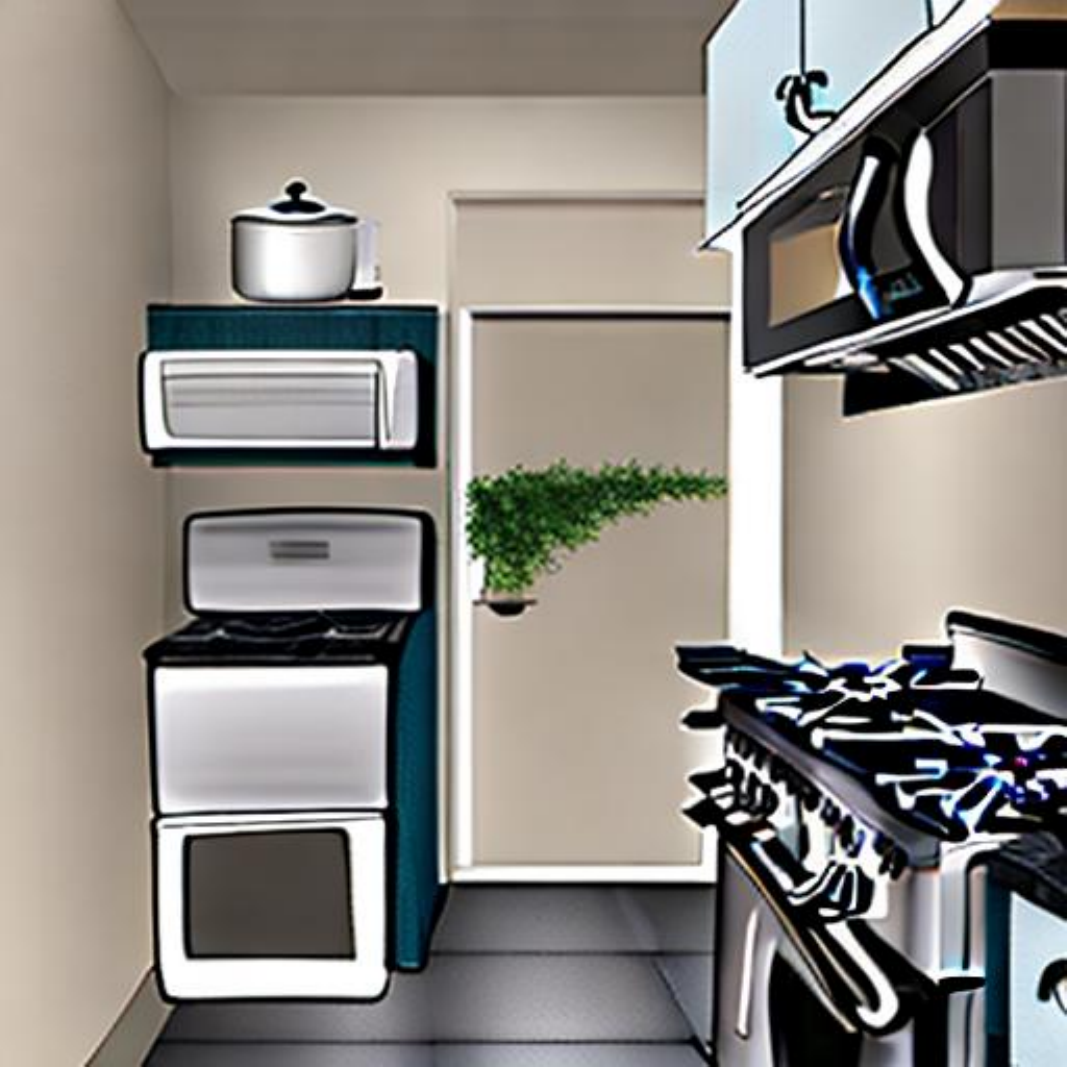} \\
    \includegraphics[width=0.10\textwidth]{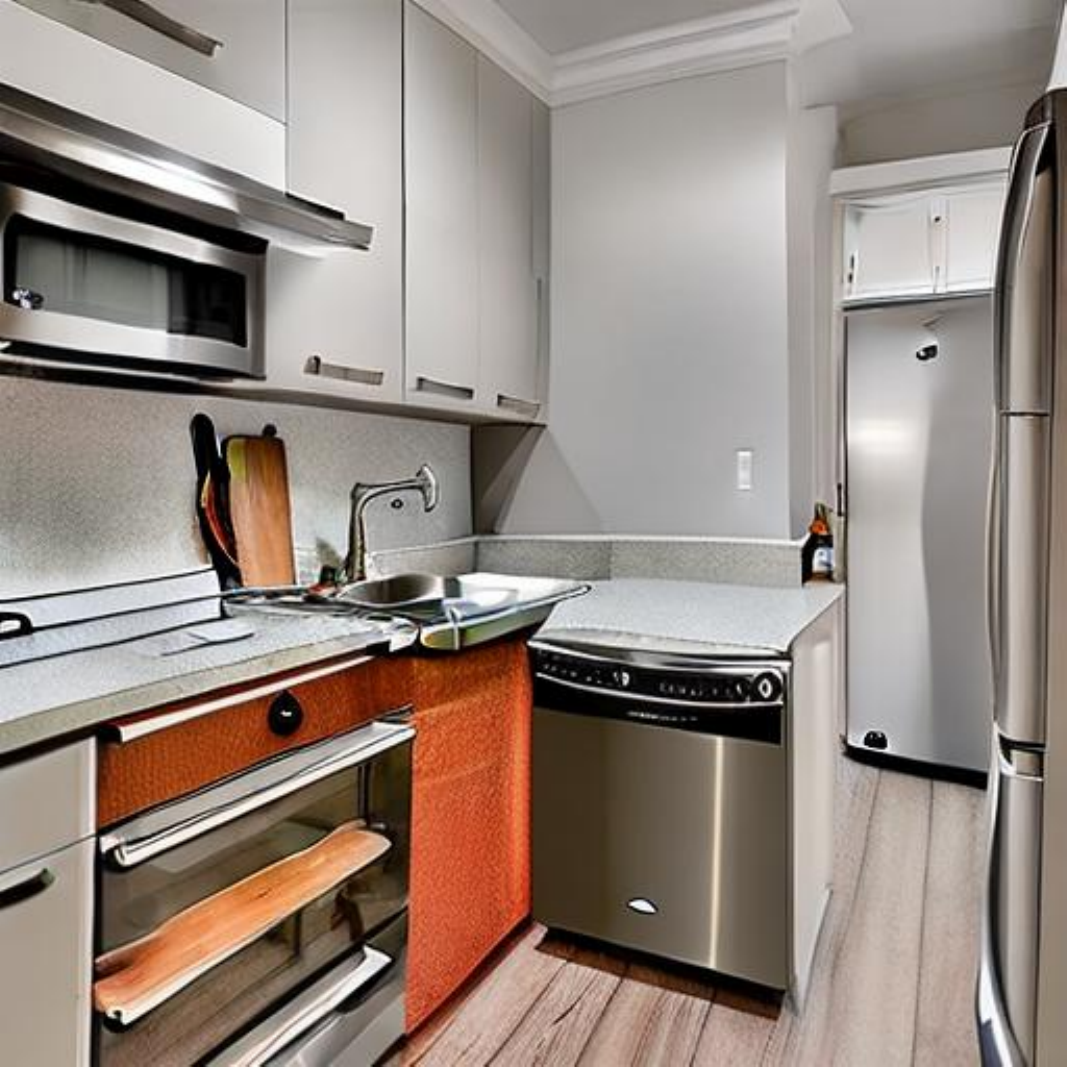} &
    \includegraphics[width=0.10\textwidth]{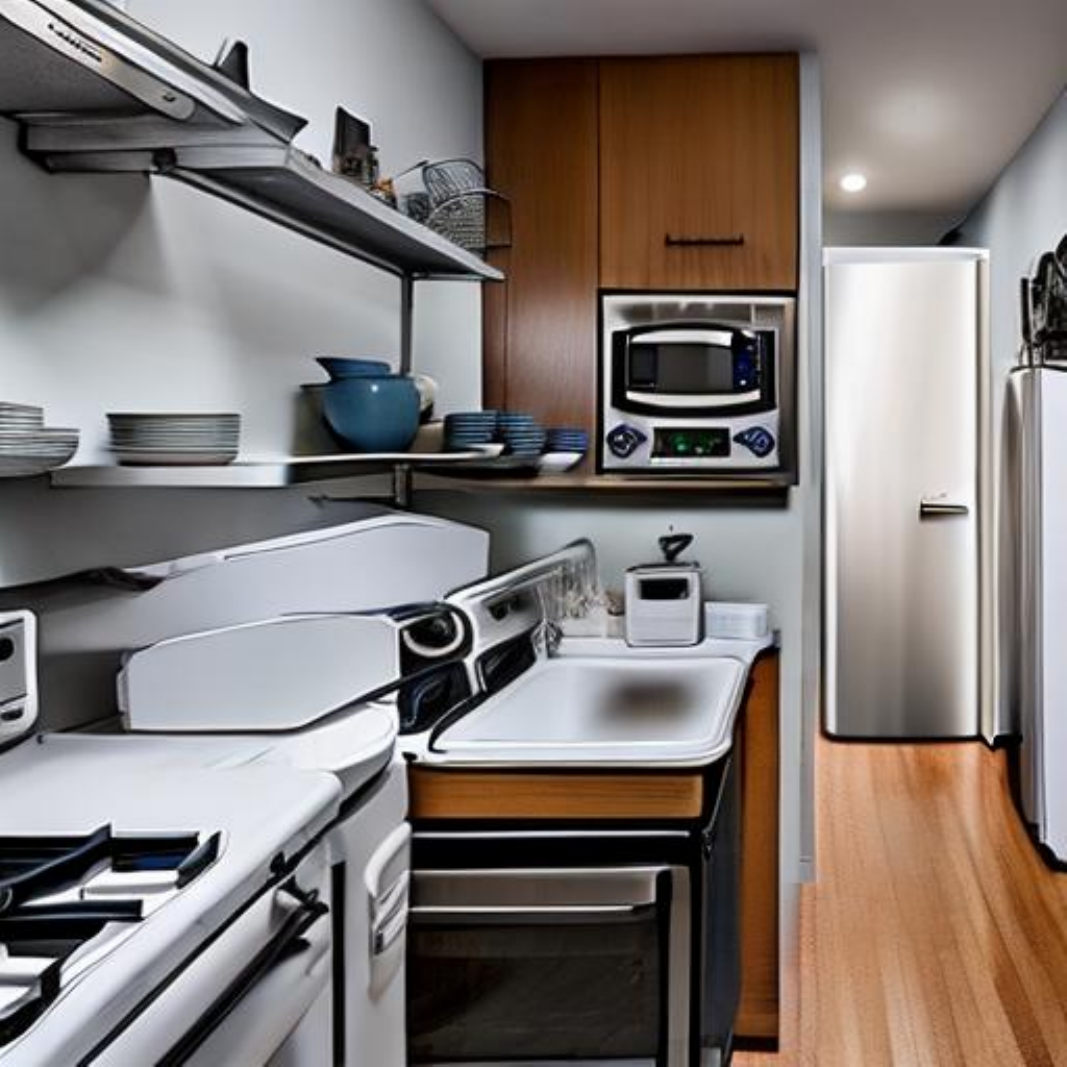} &
    \includegraphics[width=0.10\textwidth]{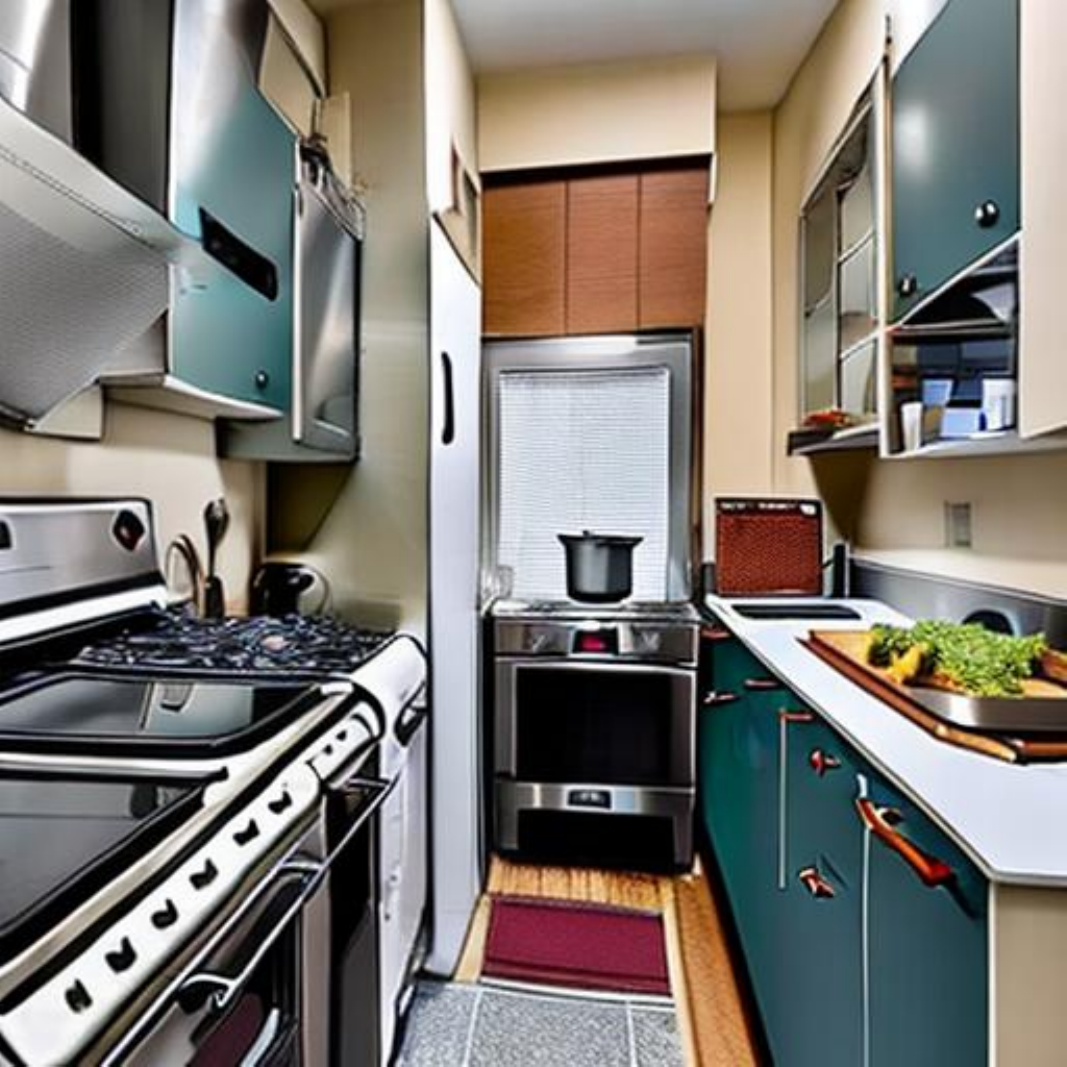} &
    \includegraphics[width=0.10\textwidth]{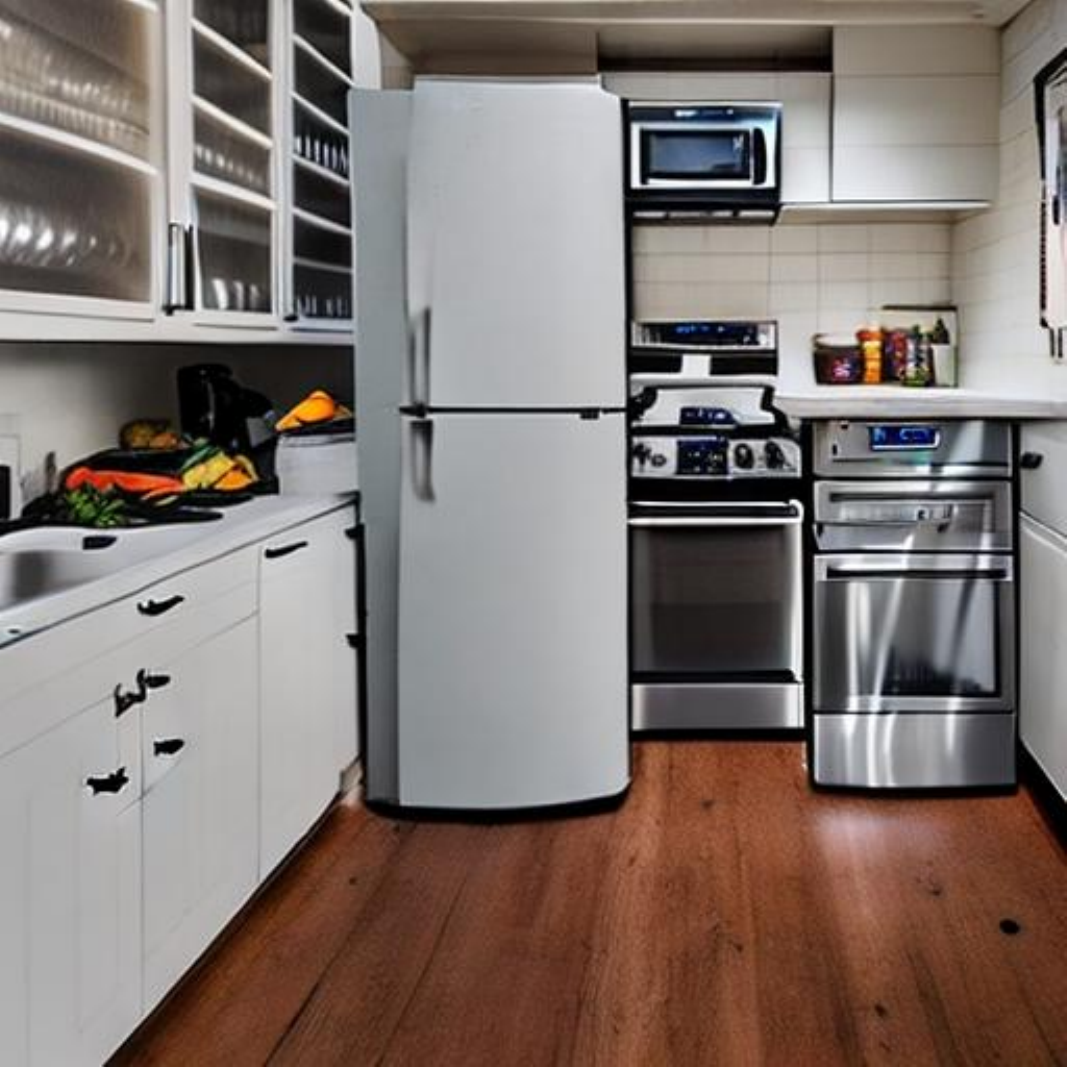} \\
};
\node[below=0pt of br] {DG-CFG (Ours)};

\end{tikzpicture}
\caption{\textbf{Qualitative ablation on Stable Diffusion~2.1 at $\bar{\omega}=23$.}
Each quadrant shows 16 samples generated from the same prompt with different random seeds. \textbf{Top-left:} constant CFG. \textbf{Top-right:} Medium 1 (time balancing). \textbf{Bottom-left:} Medium 2 (time balancing plus signal-content weighting). \textbf{Bottom-right:} DG-CFG (all three components). The prompt is \textit{a narrow kitchen filled with appliances and cooking utensils}. The prompt is \textit{a narrow kitchen filled with appliances and cooking utensils}.}
\label{fig:ablation-sd21}
\end{figure}

Taken together, the ablation results on SD1.5 and SD2.1 demonstrate that time balancing, signal-content weighting, and score-error mitigation provide complementary benefits that transfer across model generations rather than reflecting checkpoint-specific behavior.

\subsubsection{Evaluation Across NFE Settings}
\label{sec:nfe-sweep}

We next examine whether the advantage of DG-CFG persists as the numerical integration becomes more accurate and whether the trend transfers across backbones. Given the cost of repeated sampling, we focus this analysis on SD1.5 and SD2.1. We fix $\bar{\omega}=23$, the high-guidance regime where color over-saturation is pronounced, and evaluate all four methods across $\text{NFE}\in\{4,8,16,32\}$ on both backbones.

Tables~\ref{tab:nfe-sweep} and~\ref{tab:nfe-sweep-sd21} summarize the results. At $\text{NFE}=4$, substantial discretization error affects all methods. Even in this regime, DG-CFG achieves the highest CLIP and IR scores and the lowest saturation on both backbones. At $8$ NFE, DG-CFG leads in CLIP, IR, and FID on both backbones. It retains the strongest overall performance at $16$ and $32$ NFE. These consistent gains across NFE settings show that the advantage of DG-CFG is not tied to a particular discretization level.

The raw Div values further underscore the need for a quality-conditioned interpretation of diversity. Several baselines attain higher Div scores than DG-CFG despite substantially worse FID and IR. As NFE increases and these artifacts recede, their Div scores often decline, indicating that low-NFE distortions in global structure and fine detail can artificially disperse feature embeddings and inflate the Vendi score.

The saturation trend is consistent across backbones. At every NFE, DG-CFG suppresses the increase induced by strong guidance and keeps Sat close to the real-image reference on both SD1.5 and SD2.1. Constant CFG remains markedly over-saturated, whereas Interval CFG exhibits the most severe saturation at the lowest NFE. The qualitative comparisons in Figure~\ref{fig:nfe-ablation} corroborate these results: increasing NFE improves structural quality for all methods, yet the baselines remain more saturated and require additional sampling steps to approach the visual quality achieved by DG-CFG. This disparity motivates the fixed-quality analysis in Section~\ref{sec:nfe-threshold}.

\begin{table}[H]
\centering
\caption{Evaluation across NFE settings on Stable Diffusion~1.5 at $\bar{\omega}=23$. All methods are evaluated with DDIM sampling at $\text{NFE}\in\{4,8,16,32\}$. Metrics are computed over 1K MS-COCO validation prompts, except Div (Vendi score), which is averaged over 100 prompts with 16 samples per prompt. The real-image reference values are CLIP=30.57, IR=0.38, and Sat=0.33.}
\label{tab:nfe-sweep}
\small
\begin{tabular}{llcccc}
\toprule
Metric & Method & NFE$=4$ & NFE$=8$ & NFE$=16$ & NFE$=32$ \\
\midrule
\multirow{4}{*}{CLIP$\uparrow$}
& Constant          & 26.79 & 30.64 & 31.50 & 31.58 \\
& Interval CFG      & 24.84 & 29.10 & 31.00 & 31.45 \\
& $\beta$--CFG      & 26.11 & 30.88 & 31.57 & 31.62 \\
& \textbf{DG-CFG (ours)} & \textbf{27.89} & \textbf{31.39} & \textbf{31.60} & \textbf{31.64} \\
\midrule
\multirow{4}{*}{IR$\uparrow$}
& Constant          & $-$1.618 & $-$0.330 & +0.180 & +0.287 \\
& Interval CFG      & $-$1.948 & $-$0.986 & $-$0.128 & +0.185 \\
& $\beta$--CFG      & $-$1.738 & $-$0.239 & +0.227 & +0.299 \\
& \textbf{DG-CFG (ours)} & \textbf{$-$1.314} & \textbf{+0.029} & \textbf{+0.269} & \textbf{+0.338} \\
\midrule
\multirow{4}{*}{Sat}
& Constant          & 0.408 & 0.469 & 0.455 & 0.450 \\
& Interval CFG      & 0.500 & 0.462 & 0.413 & 0.395 \\
& $\beta$--CFG      & 0.439 & 0.438 & 0.415 & 0.412 \\
& \textbf{DG-CFG (ours)} & 0.348 & 0.332 & 0.341 & 0.356 \\
\midrule
\multirow{4}{*}{FID$\downarrow$}
& Constant          & 137.16 & 82.34  & 73.00  & 72.79 \\
& Interval CFG      & 165.71 & 110.89 & 81.61  & 75.43 \\
& $\beta$--CFG      & 150.37 & 85.38  & 73.71  & 73.45 \\
& \textbf{DG-CFG (ours)} & \textbf{133.54} & \textbf{74.53} & \textbf{70.60} & \textbf{69.95} \\
\midrule
\multirow{4}{*}{Div$\uparrow$}
& Constant          & 7.20 & 5.67 & 4.80 & 4.45 \\
& Interval CFG      & 7.16 & 6.80 & 5.77 & 5.19 \\
& $\beta$--CFG      & 7.16 & 5.57 & 4.65 & 4.53 \\
& \textbf{DG-CFG (ours)} & 6.79 & 5.24 & 4.76 & 4.72 \\
\bottomrule
\end{tabular}
\end{table}

\begin{table}[H]
\centering
\caption{Evaluation across NFE settings on Stable Diffusion~2.1 at $\bar{\omega}=23$. All methods are evaluated with DDIM sampling at $\text{NFE}\in\{4,8,16,32\}$. Metrics are computed over 1K MS-COCO validation prompts, except Div (Vendi score), which is averaged over 100 prompts with 16 samples per prompt. The real-image reference values are CLIP=30.57, IR=0.38, and Sat=0.33.}
\label{tab:nfe-sweep-sd21}
\small
\begin{tabular}{llcccc}
\toprule
Metric & Method & NFE$=4$ & NFE$=8$ & NFE$=16$ & NFE$=32$ \\
\midrule
\multirow{4}{*}{CLIP$\uparrow$}
& Constant          & 27.854 & 30.977 & \textbf{31.664} & 31.656 \\
& Interval CFG      & 26.460 & 29.802 & 31.058 & 31.444 \\
& $\beta$--CFG      & 27.475 & 31.141 & 31.613 & 31.639 \\
& \textbf{DG-CFG (ours)} & \textbf{28.348} & \textbf{31.427} & 31.655 & \textbf{31.657} \\
\midrule
\multirow{4}{*}{IR$\uparrow$}
& Constant          & $-$1.325 & $-$0.092 & +0.409 & +0.469 \\
& Interval CFG      & $-$1.687 & $-$0.630 & +0.082 & +0.364 \\
& $\beta$--CFG      & $-$1.466 & +0.060 & +0.416 & +0.516 \\
& \textbf{DG-CFG (ours)} & \textbf{$-$1.095} & \textbf{+0.244} & \textbf{+0.488} & \textbf{+0.536} \\
\midrule
\multirow{4}{*}{Sat}
& Constant          & 0.369 & 0.431 & 0.413 & 0.408 \\
& Interval CFG      & 0.454 & 0.422 & 0.382 & 0.366 \\
& $\beta$--CFG      & 0.395 & 0.398 & 0.381 & 0.381 \\
& \textbf{DG-CFG (ours)} & 0.322 & 0.315 & 0.327 & 0.340 \\
\midrule
\multirow{4}{*}{FID$\downarrow$}
& Constant          & \textbf{125.43} & 79.94 & 72.12 & 72.60 \\
& Interval CFG      & 147.68 & 105.19 & 82.19 & 76.67 \\
& $\beta$--CFG      & 135.48 & 83.11 & 74.38 & 73.91 \\
& \textbf{DG-CFG (ours)} & 129.59 & \textbf{76.43} & \textbf{71.93} & \textbf{72.38} \\
\midrule
\multirow{4}{*}{Div$\uparrow$}
& Constant          & 6.88 & 5.33 & 4.44 & 4.14 \\
& Interval CFG      & 7.15 & 6.38 & 5.39 & 4.77 \\
& $\beta$--CFG      & 6.91 & 5.24 & 4.40 & 4.27 \\
& \textbf{DG-CFG (ours)} & 6.46 & 4.83 & 4.41 & 4.35 \\
\bottomrule
\end{tabular}
\end{table}

\begin{figure}[t]
\centering
\small
\begin{tikzpicture}[
    img/.style={inner sep=0pt, outer sep=0pt, anchor=center},
    rowgrid/.style={
        matrix of nodes,
        nodes=img,
        inner sep=0pt, outer sep=0pt,
        column sep=1.2pt,
        row sep=1.2pt
    }
]

\matrix (nfe) [rowgrid]
{
\node[name=l1-1]{\includegraphics[width=0.10\textwidth]{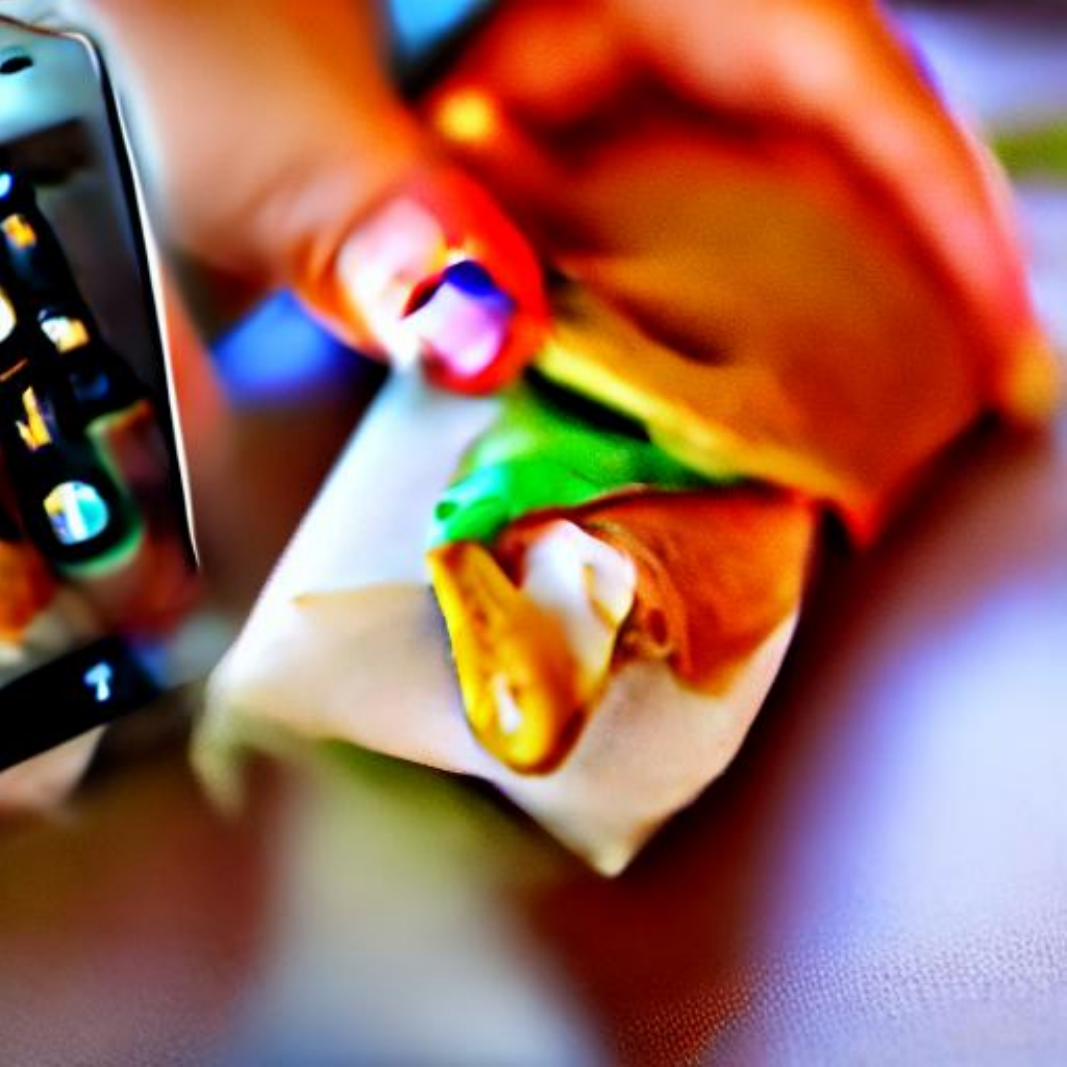}}; &
\node[name=l1-2]{\includegraphics[width=0.10\textwidth]{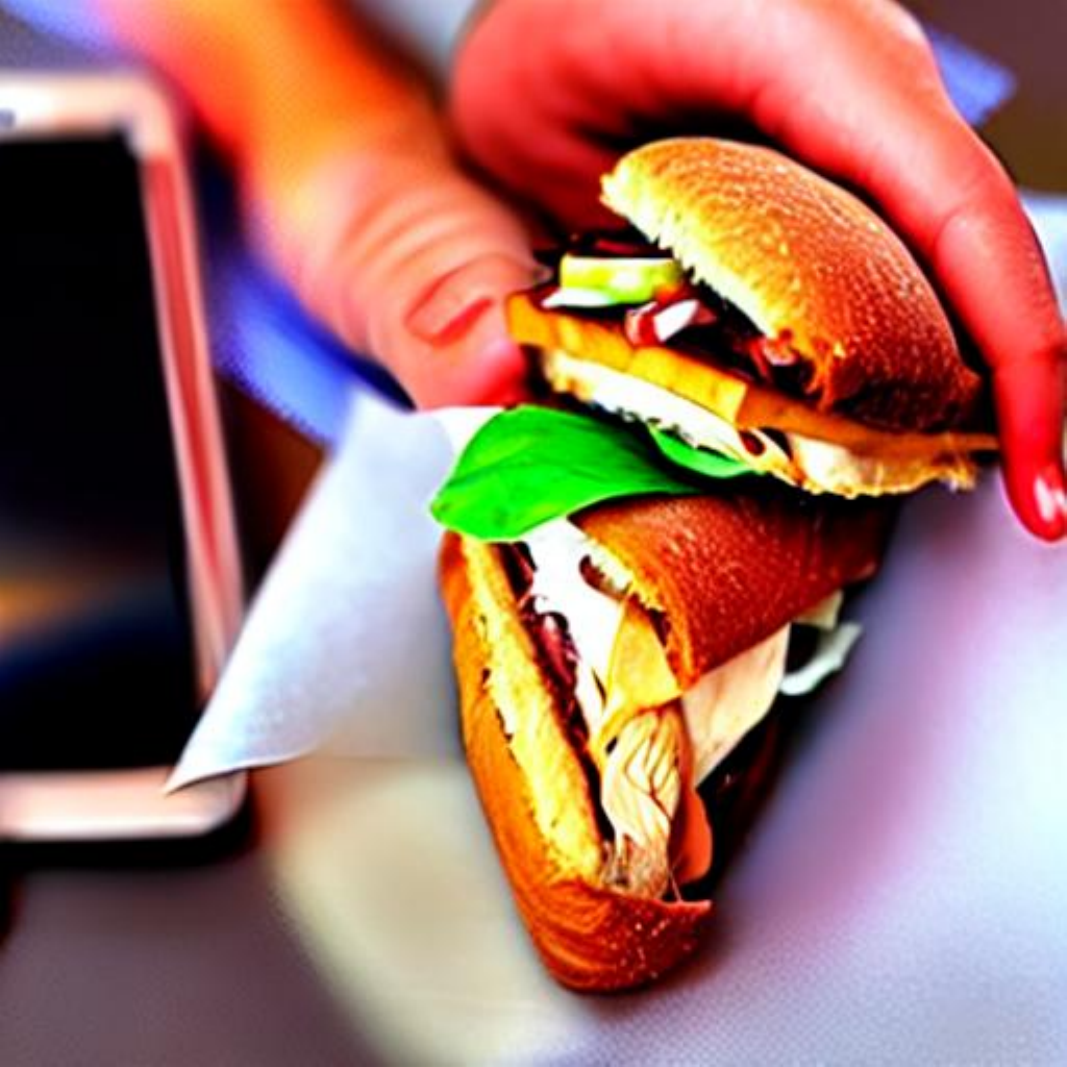}}; &
\node[name=l1-3]{\includegraphics[width=0.10\textwidth]{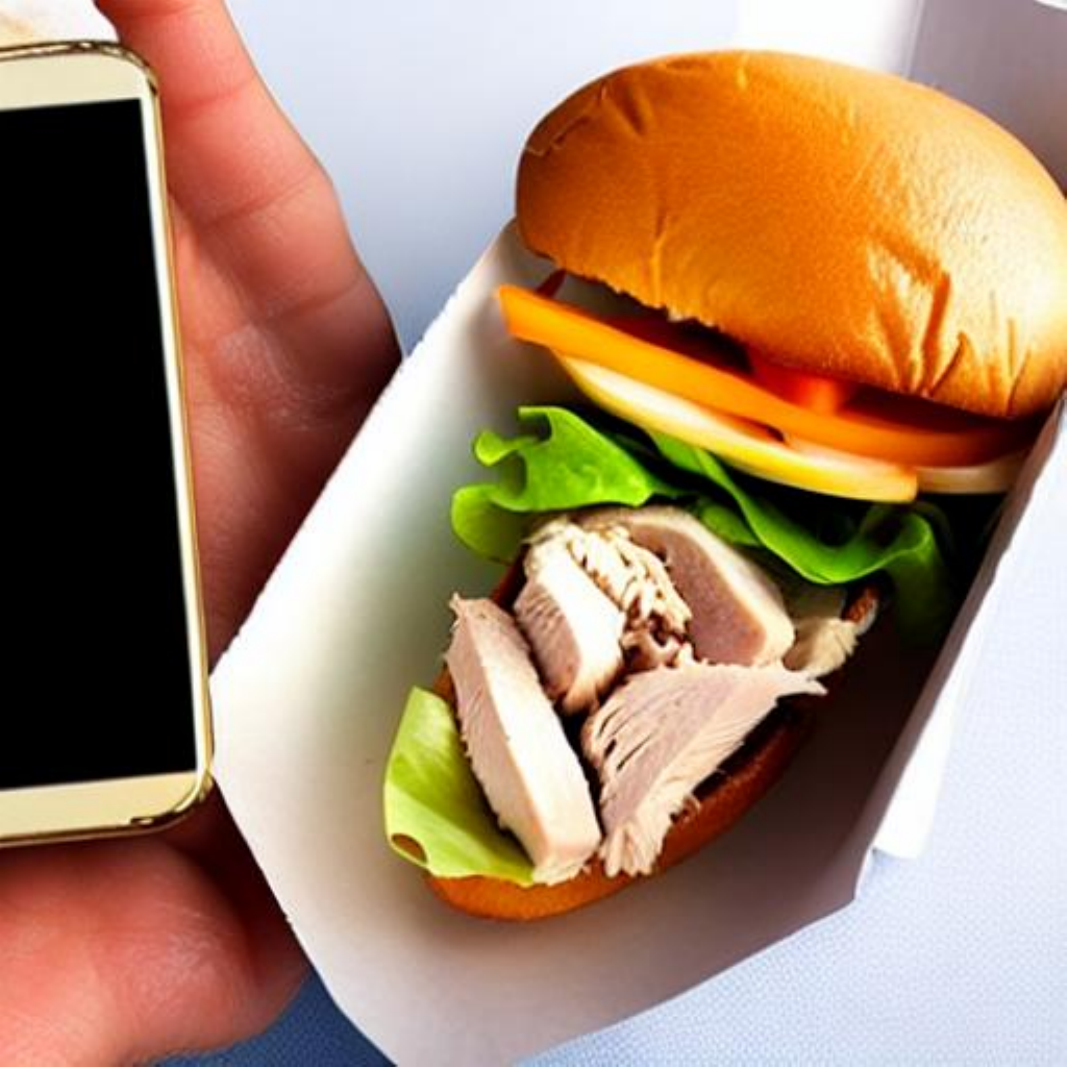}}; &
\node[name=l1-4]{\includegraphics[width=0.10\textwidth]{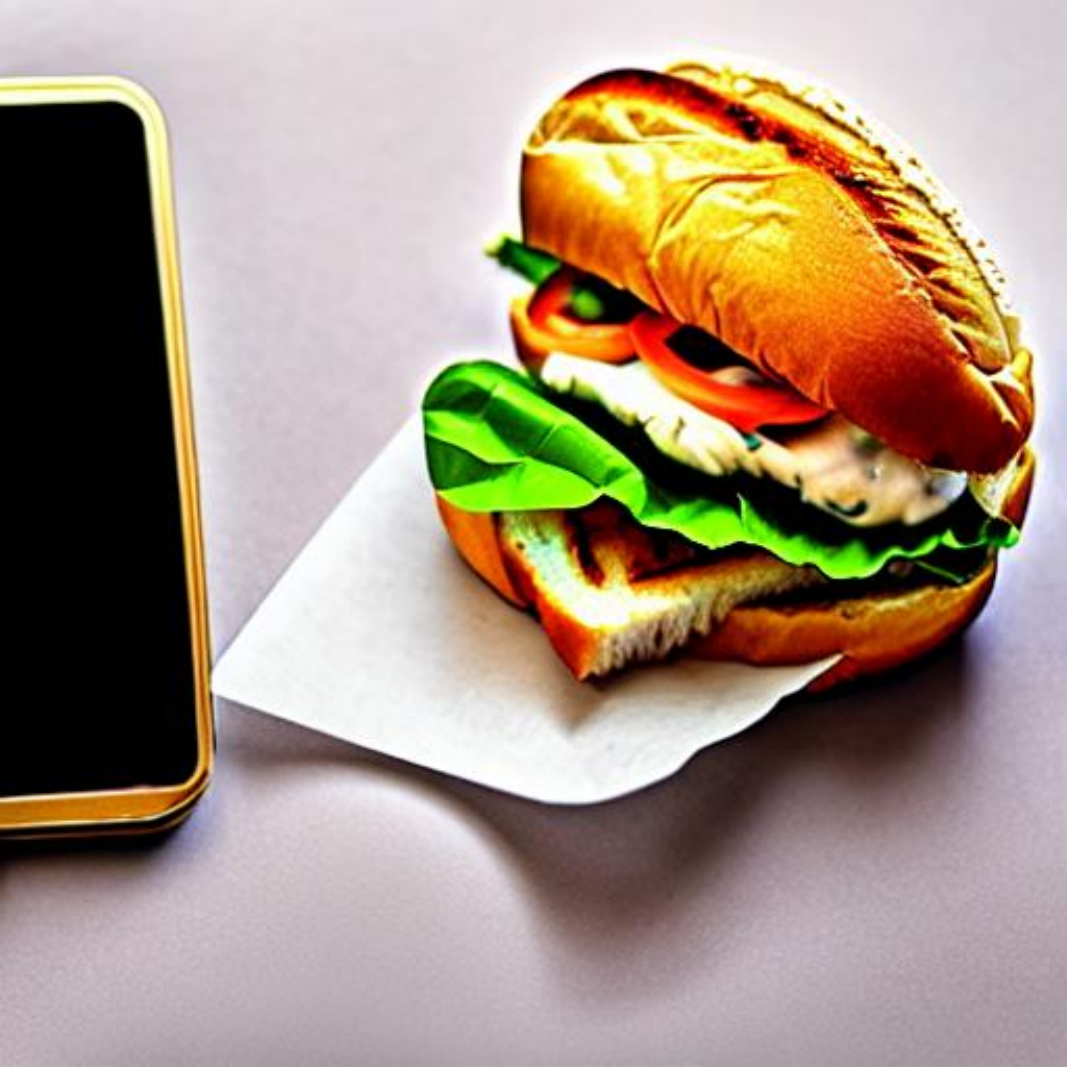}}; &
\node[minimum width=14pt]{}; &
\node[name=r1-1]{\includegraphics[width=0.10\textwidth]{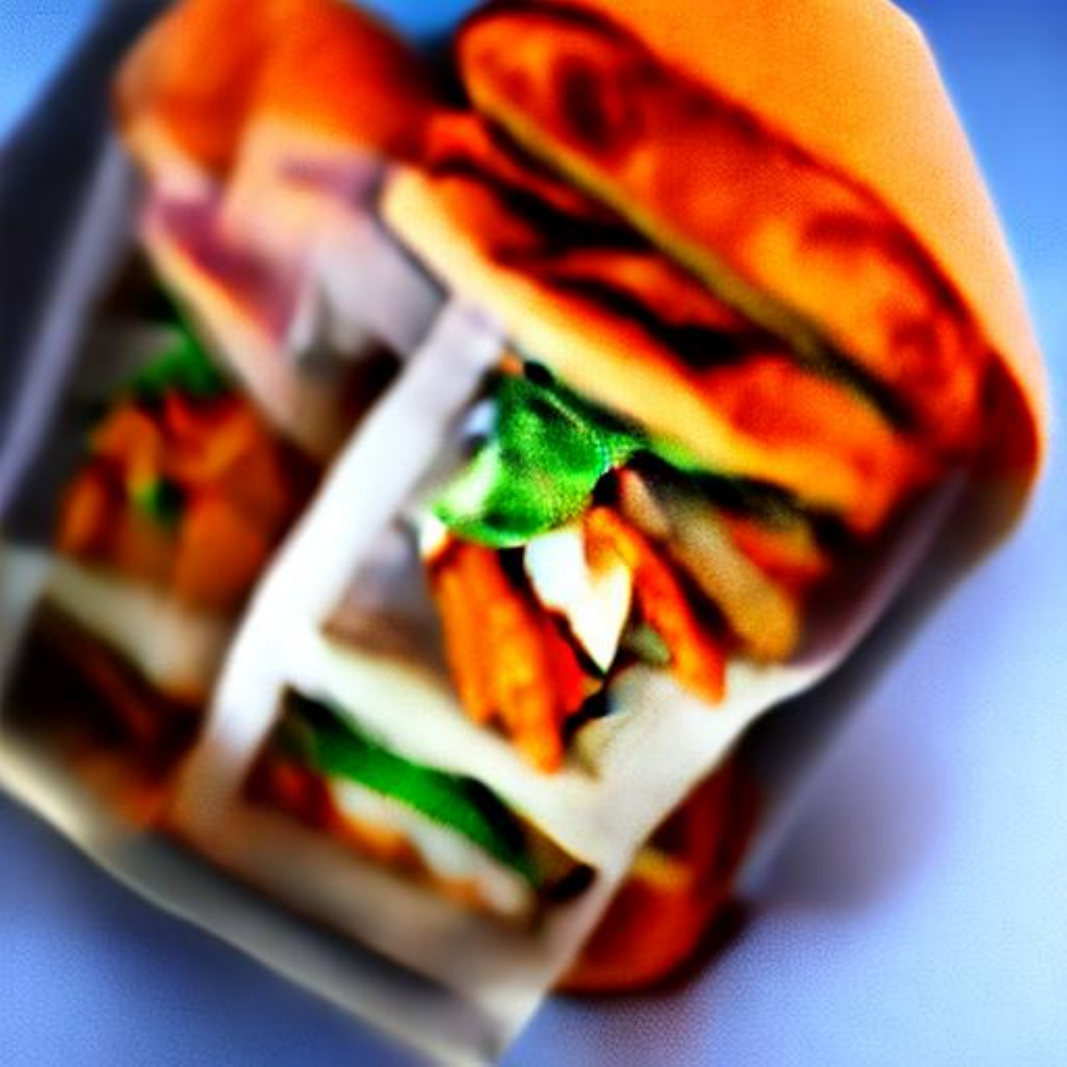}}; &
\node[name=r1-2]{\includegraphics[width=0.10\textwidth]{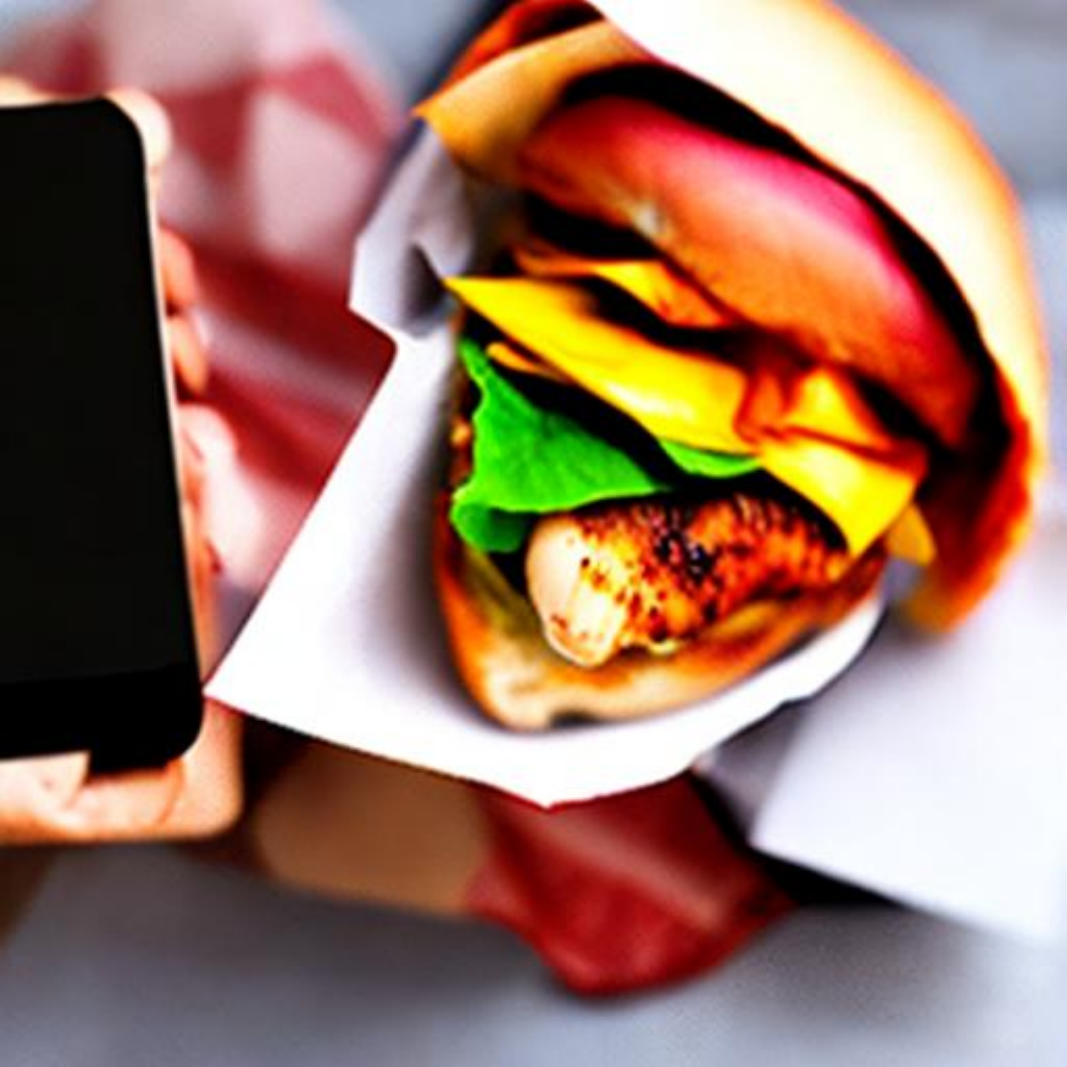}}; &
\node[name=r1-3]{\includegraphics[width=0.10\textwidth]{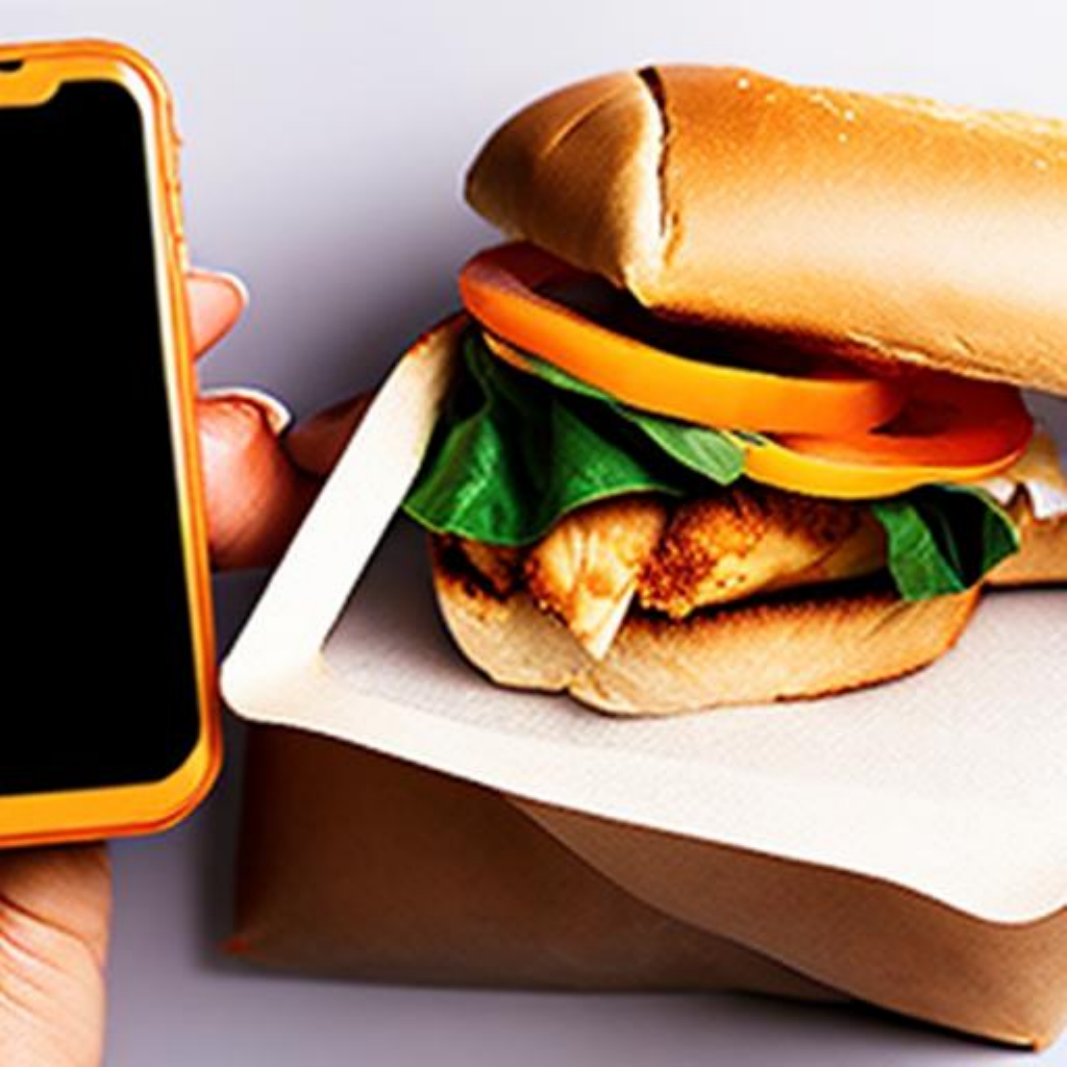}}; &
\node[name=r1-4]{\includegraphics[width=0.10\textwidth]{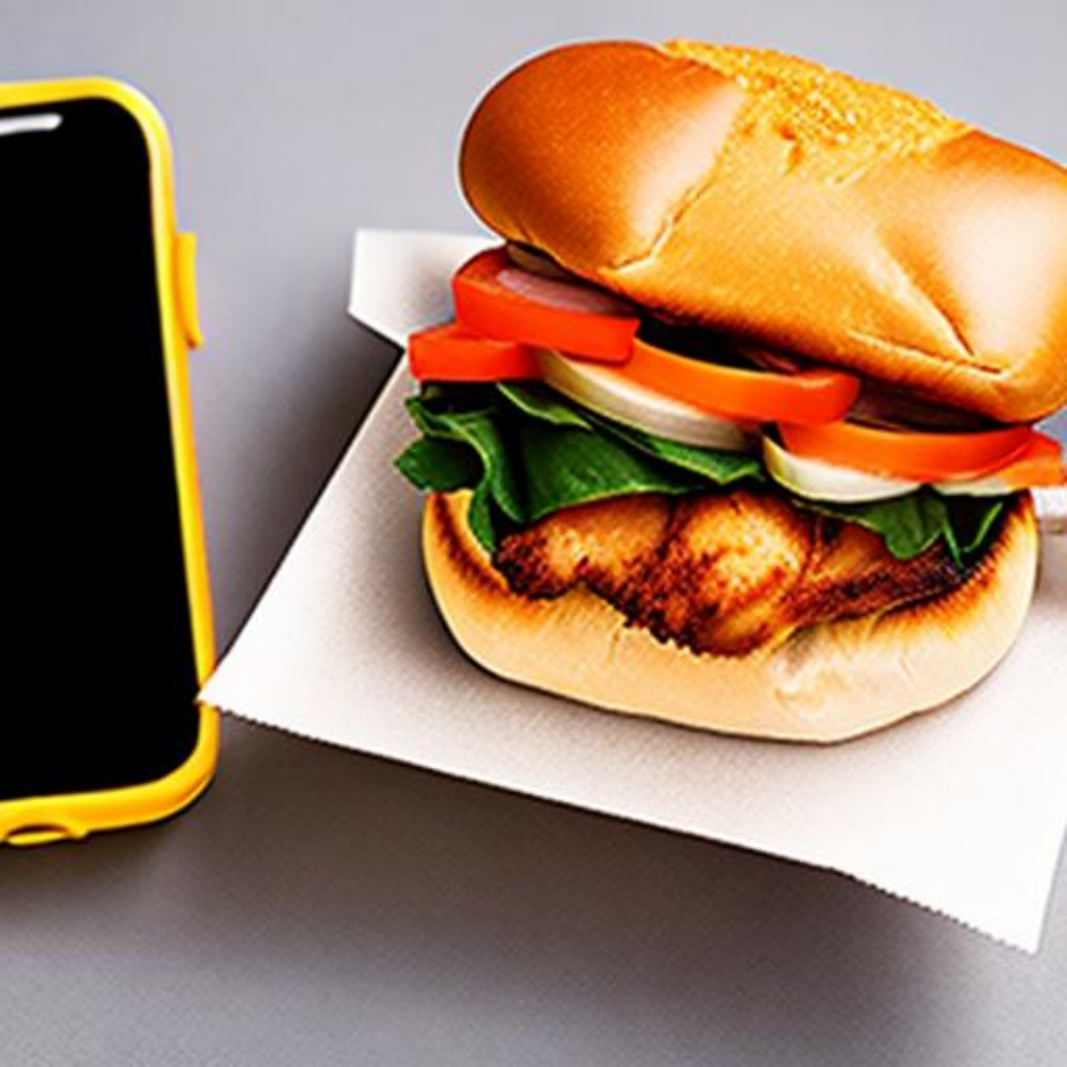}}; \\
\node[name=l2-1]{\includegraphics[width=0.10\textwidth]{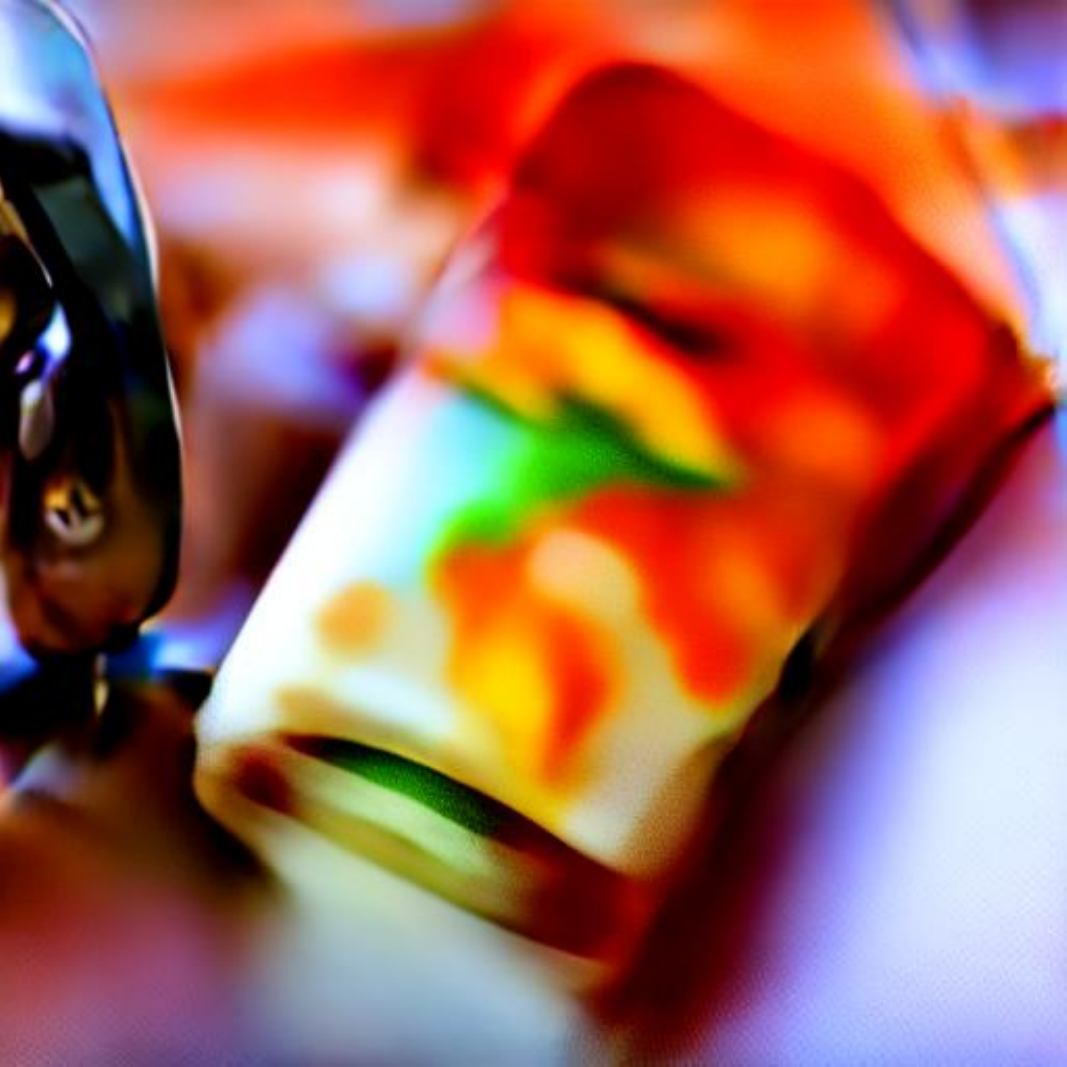}}; &
\includegraphics[width=0.10\textwidth]{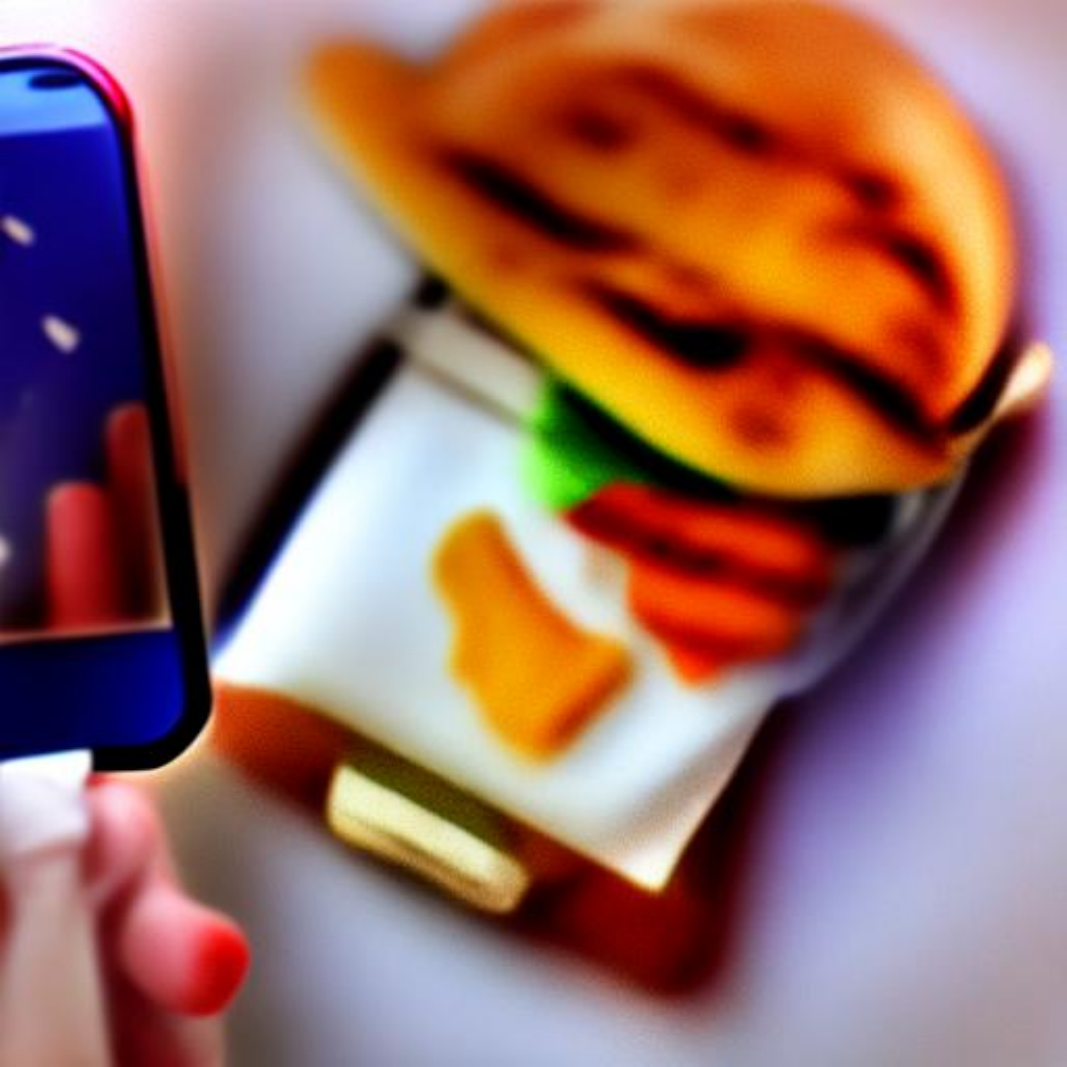} &
\includegraphics[width=0.10\textwidth]{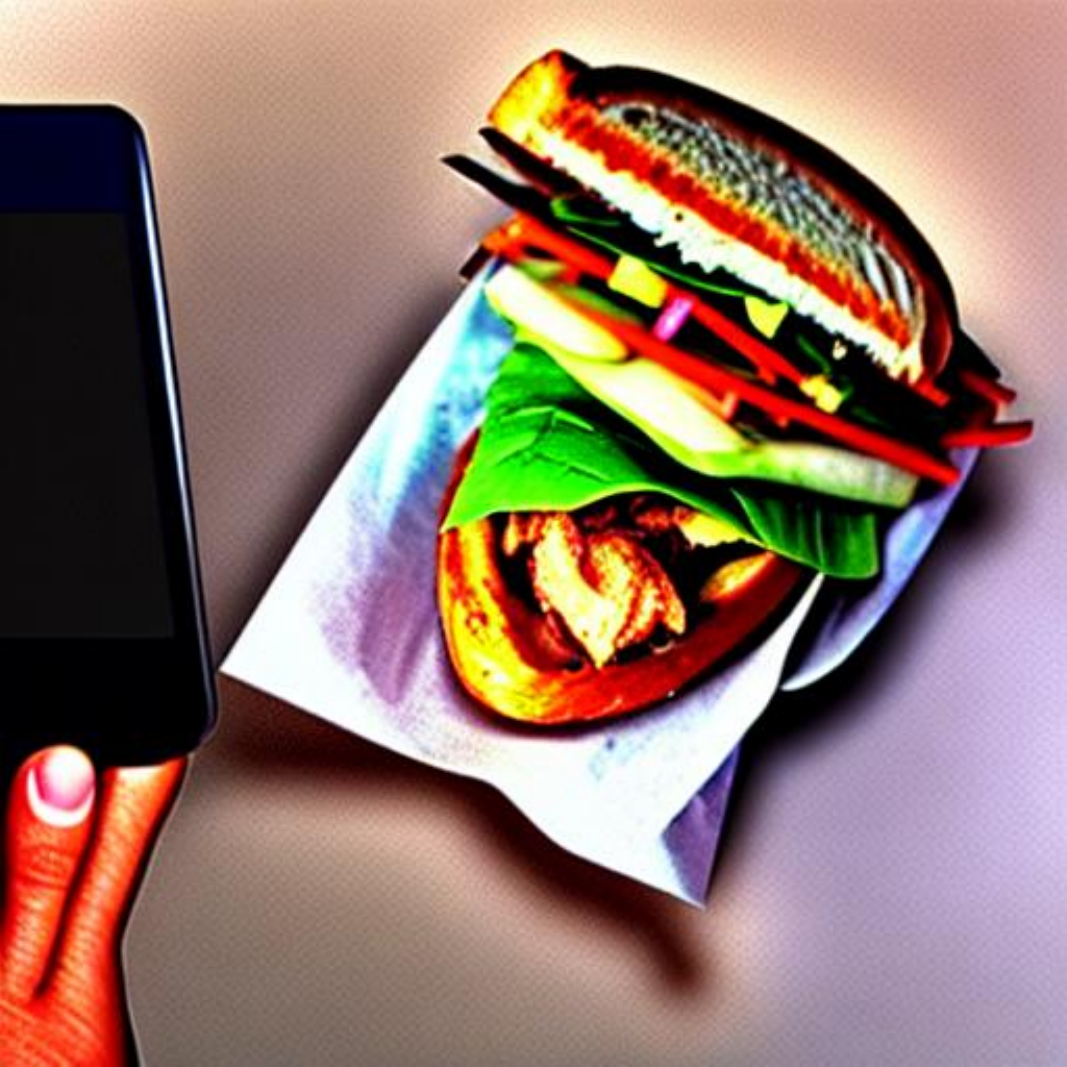} &
\includegraphics[width=0.10\textwidth]{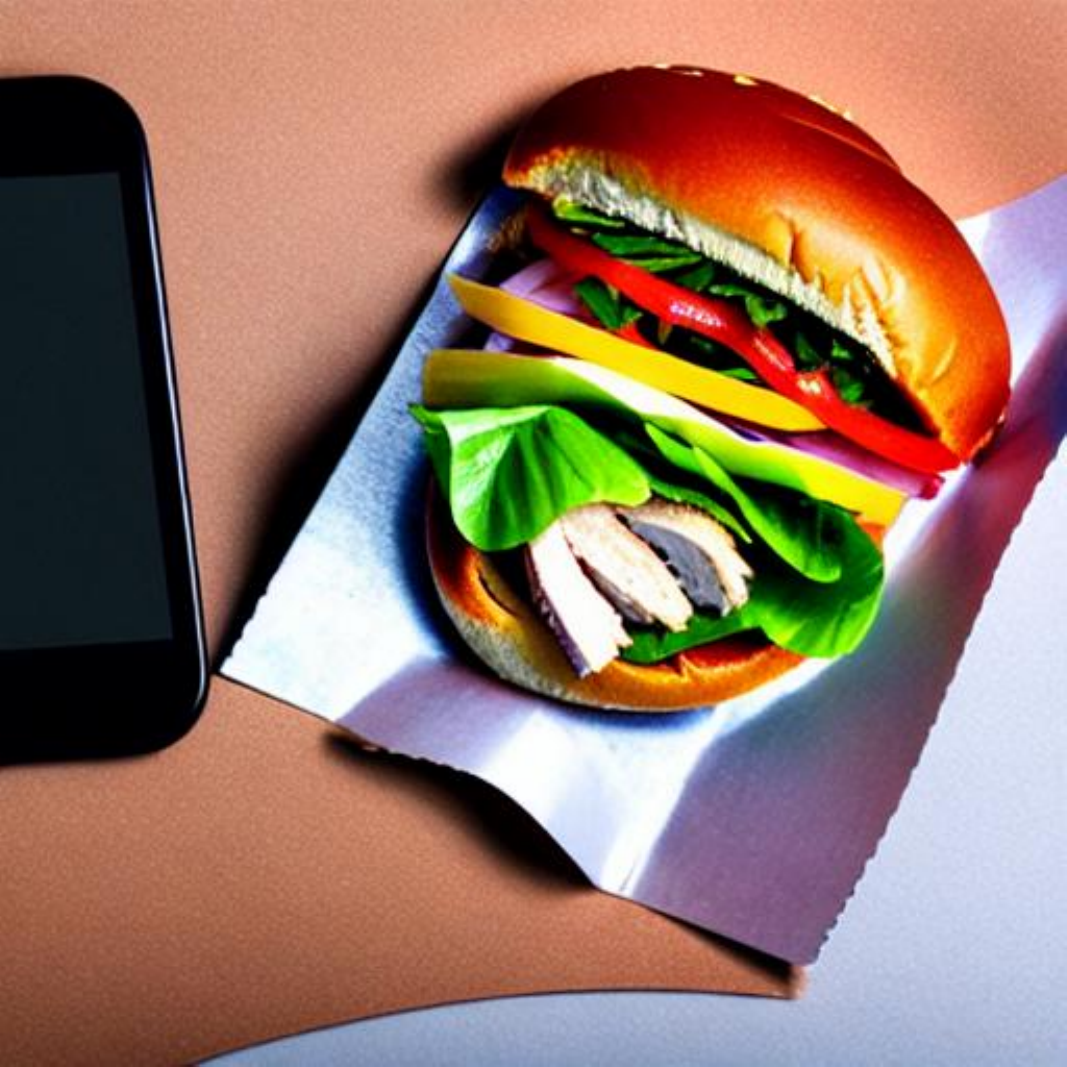} &
\node[minimum width=14pt]{}; &
\includegraphics[width=0.10\textwidth]{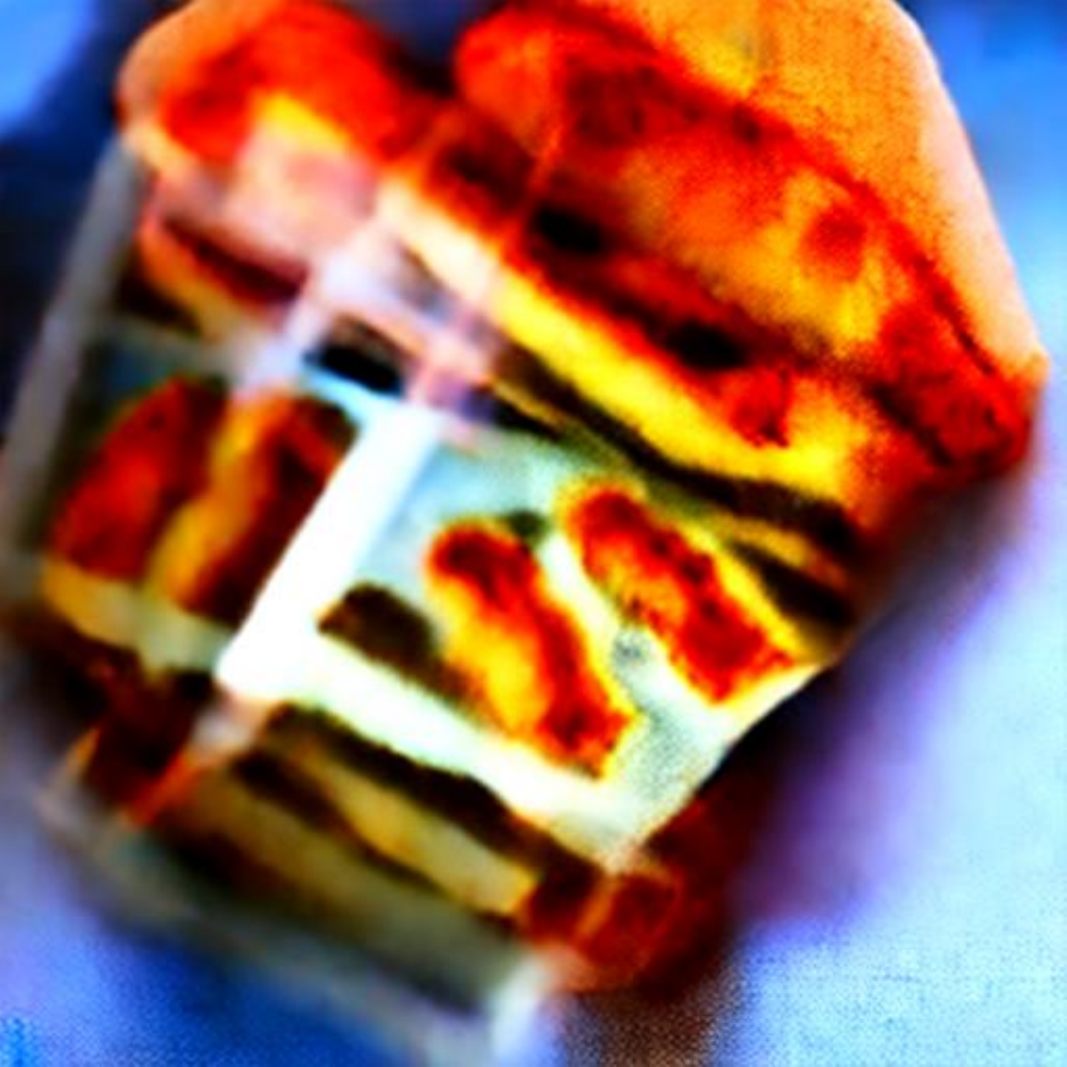} &
\includegraphics[width=0.10\textwidth]{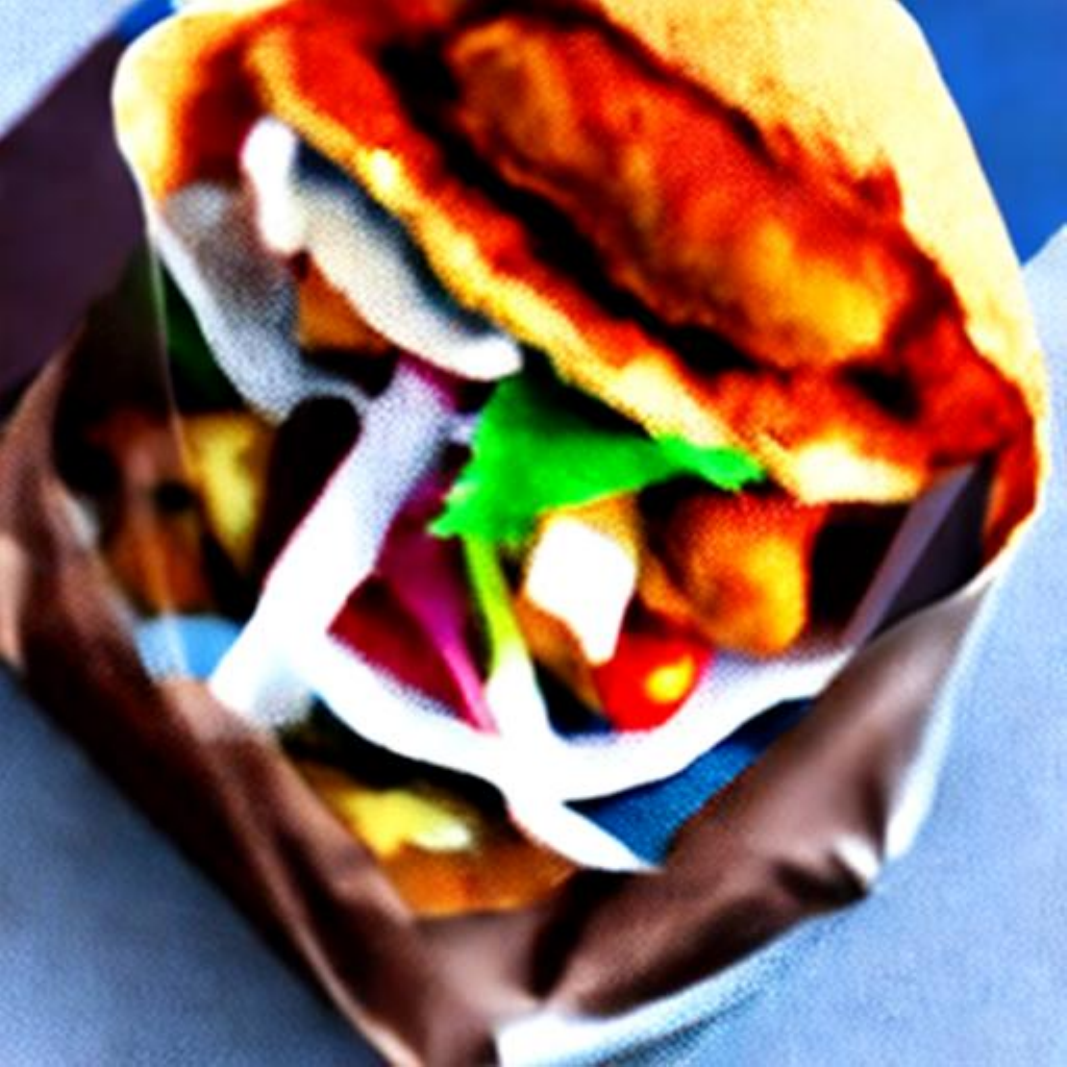} &
\includegraphics[width=0.10\textwidth]{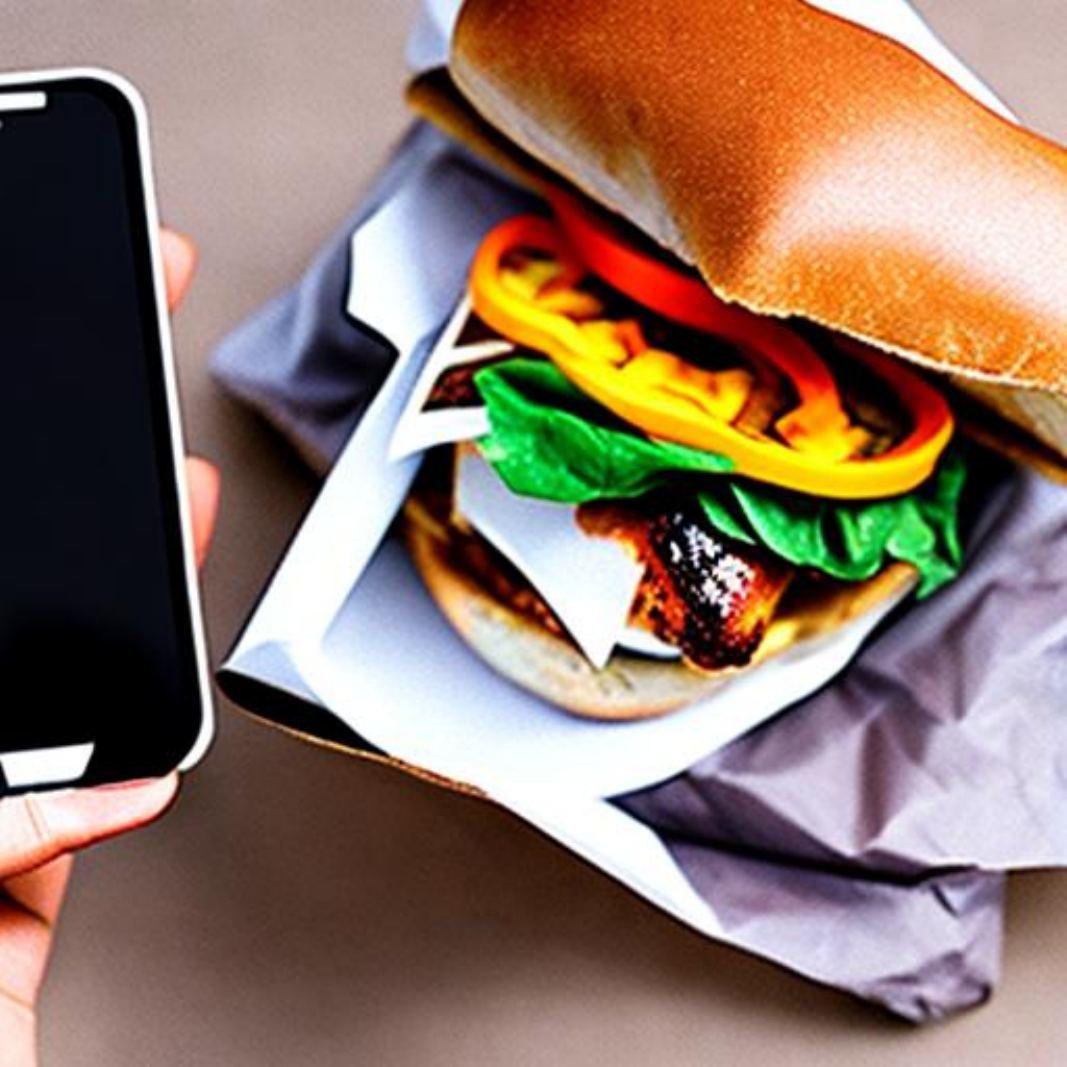} &
\includegraphics[width=0.10\textwidth]{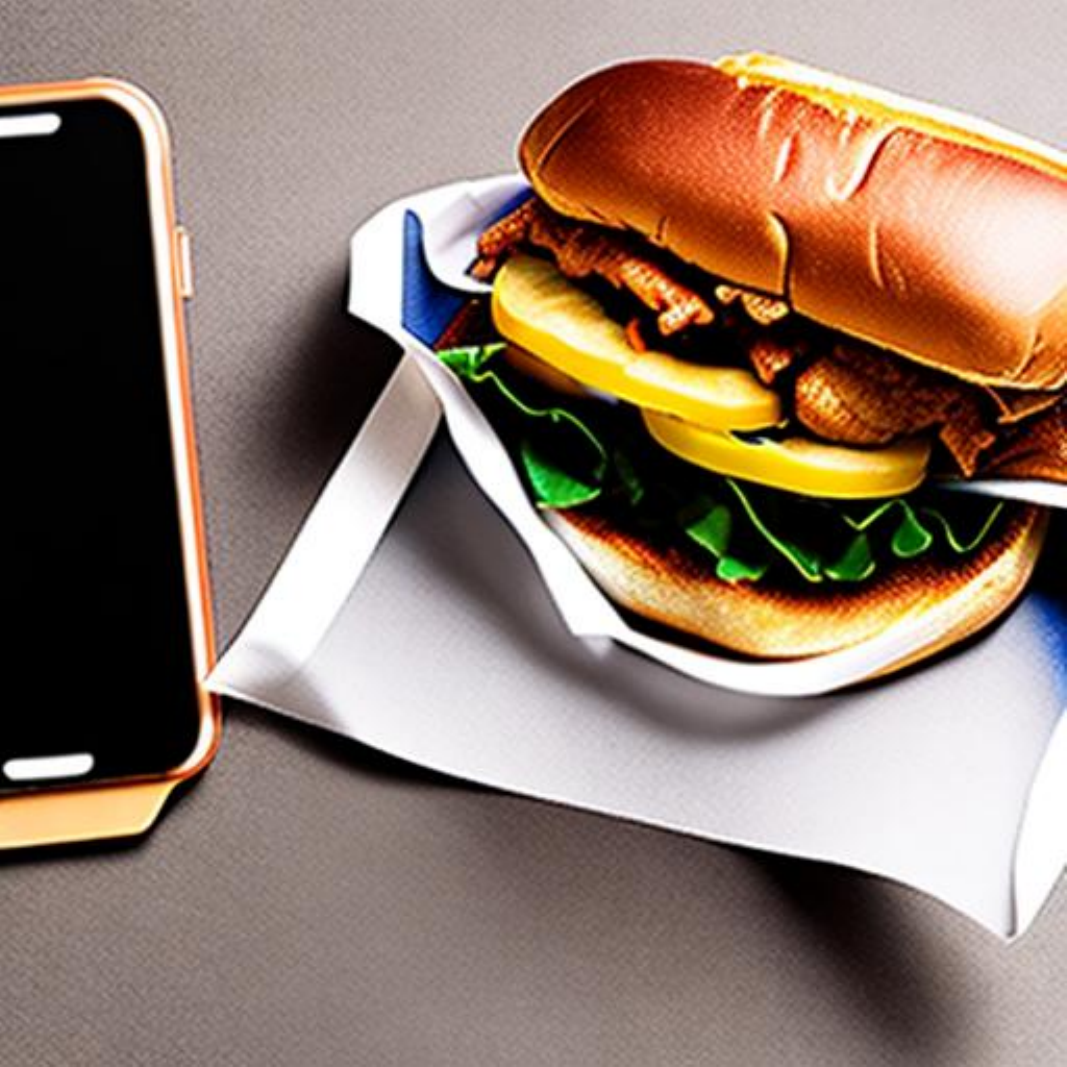} \\
\node[name=l3-1]{\includegraphics[width=0.10\textwidth]{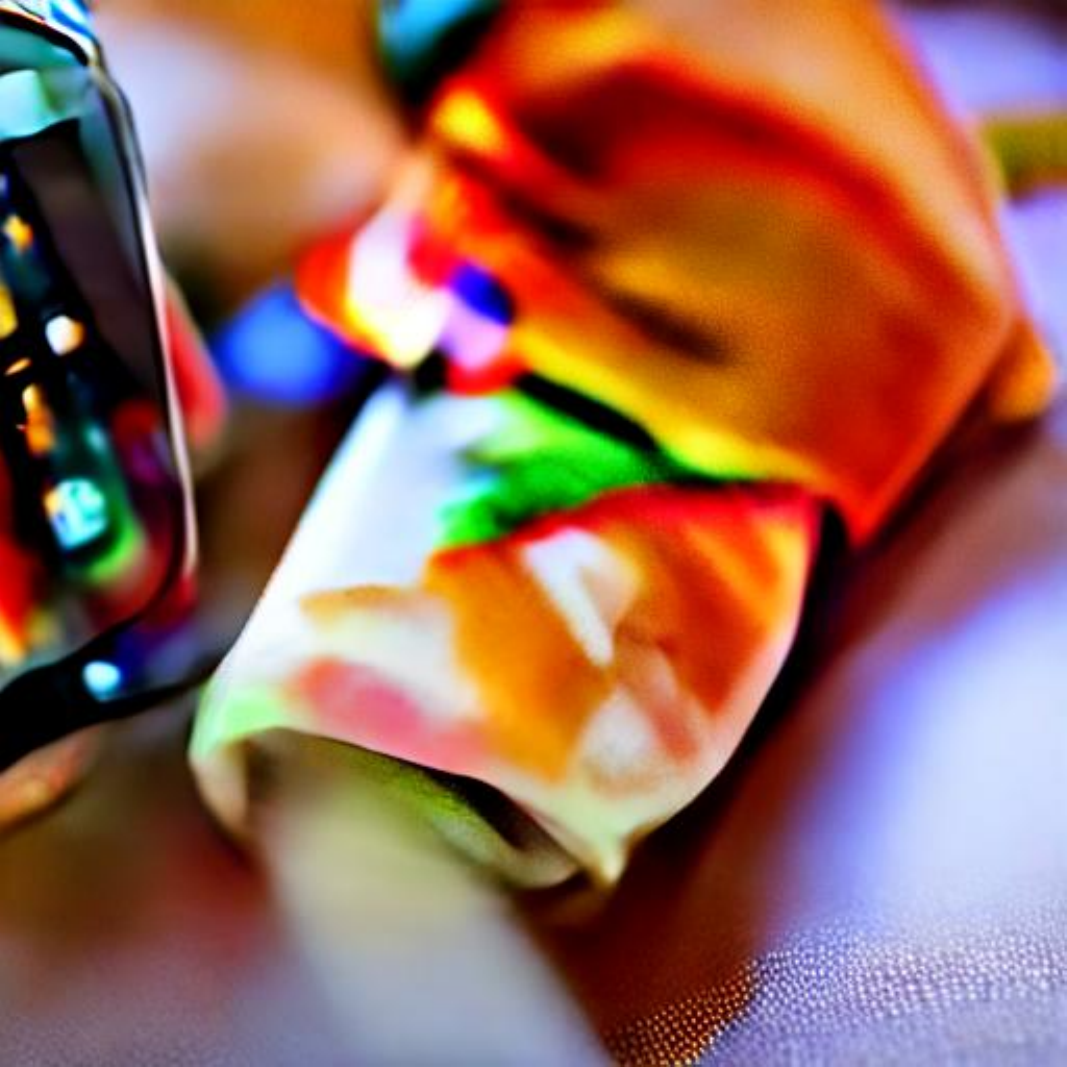}}; &
\includegraphics[width=0.10\textwidth]{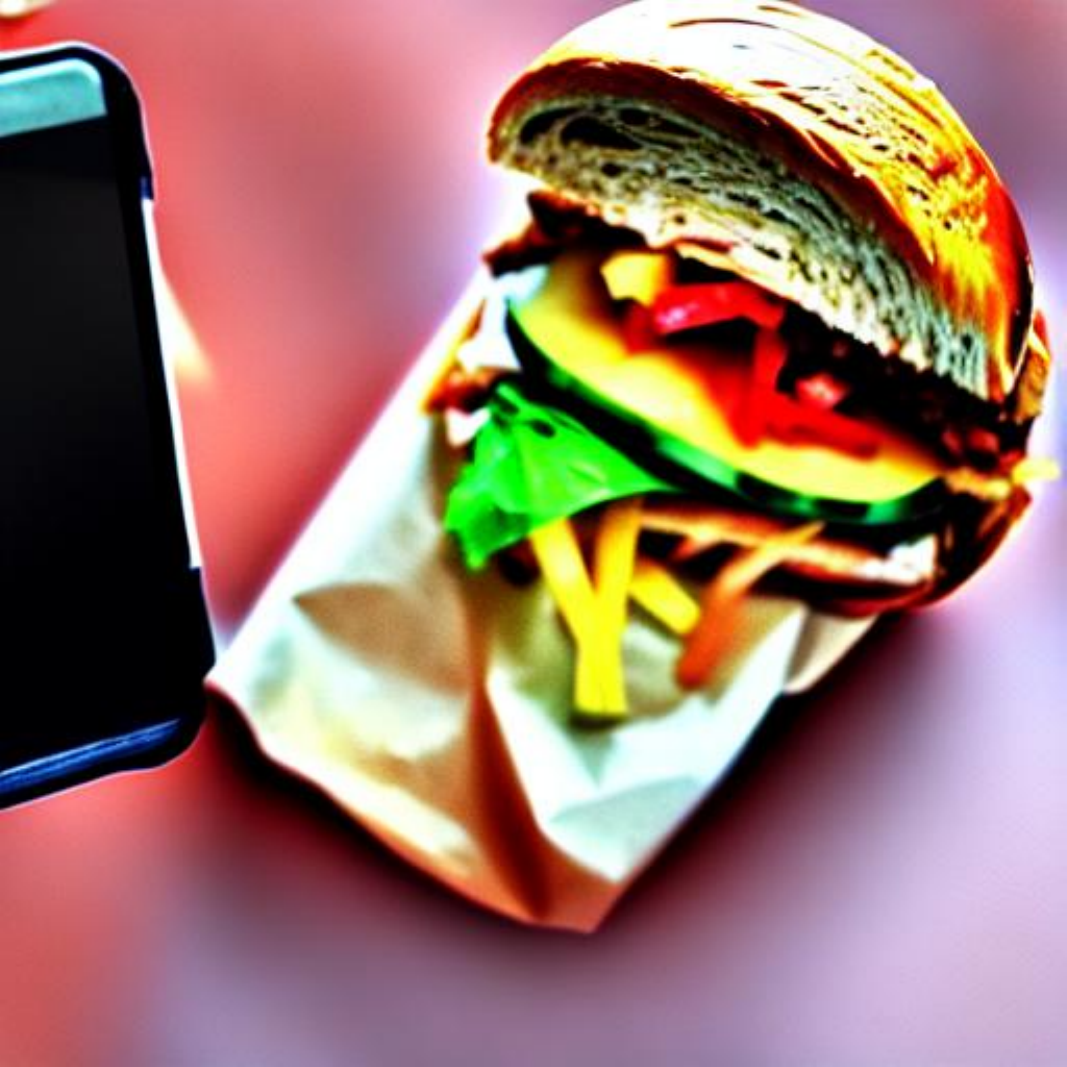} &
\includegraphics[width=0.10\textwidth]{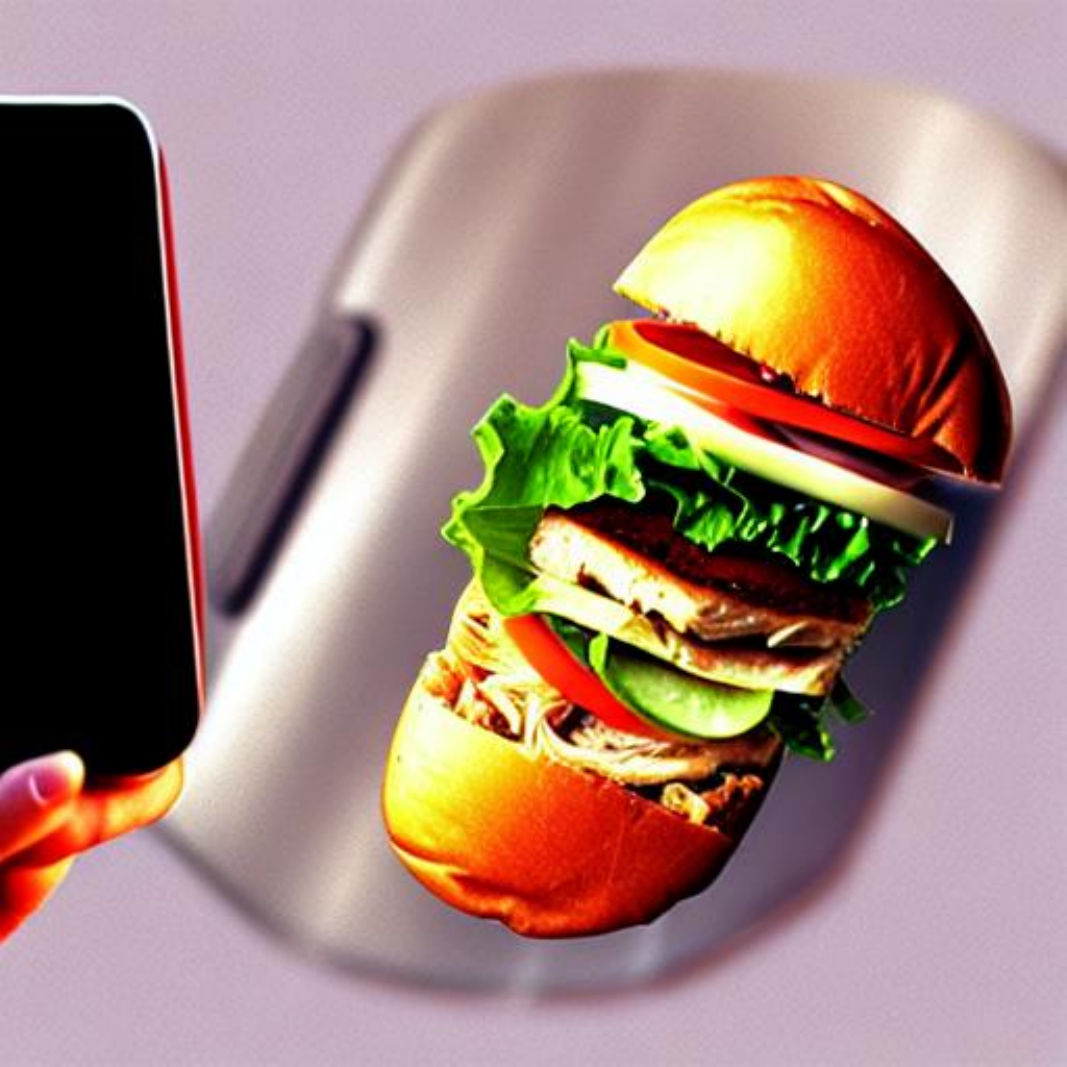} &
\includegraphics[width=0.10\textwidth]{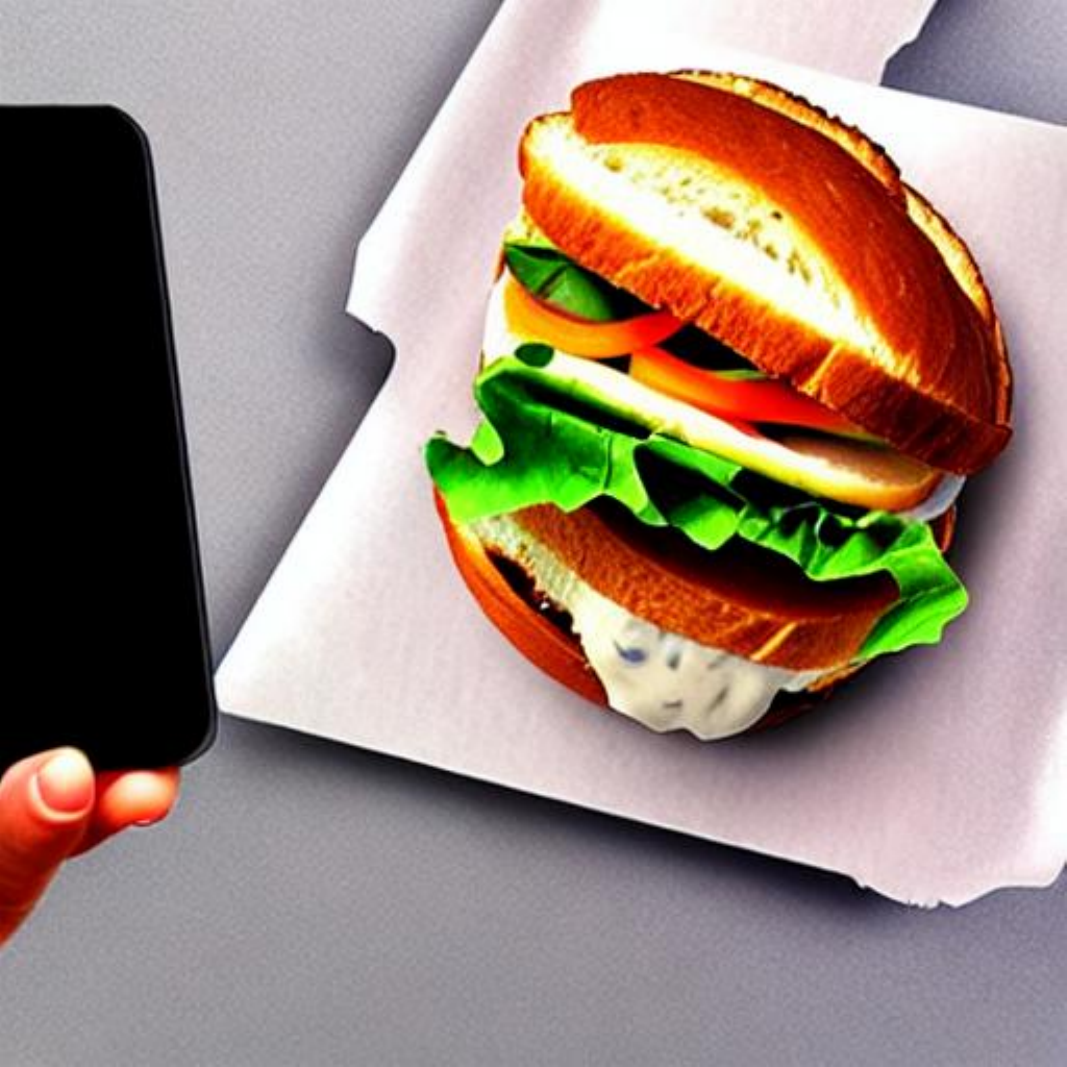} &
\node[minimum width=14pt]{}; &
\includegraphics[width=0.10\textwidth]{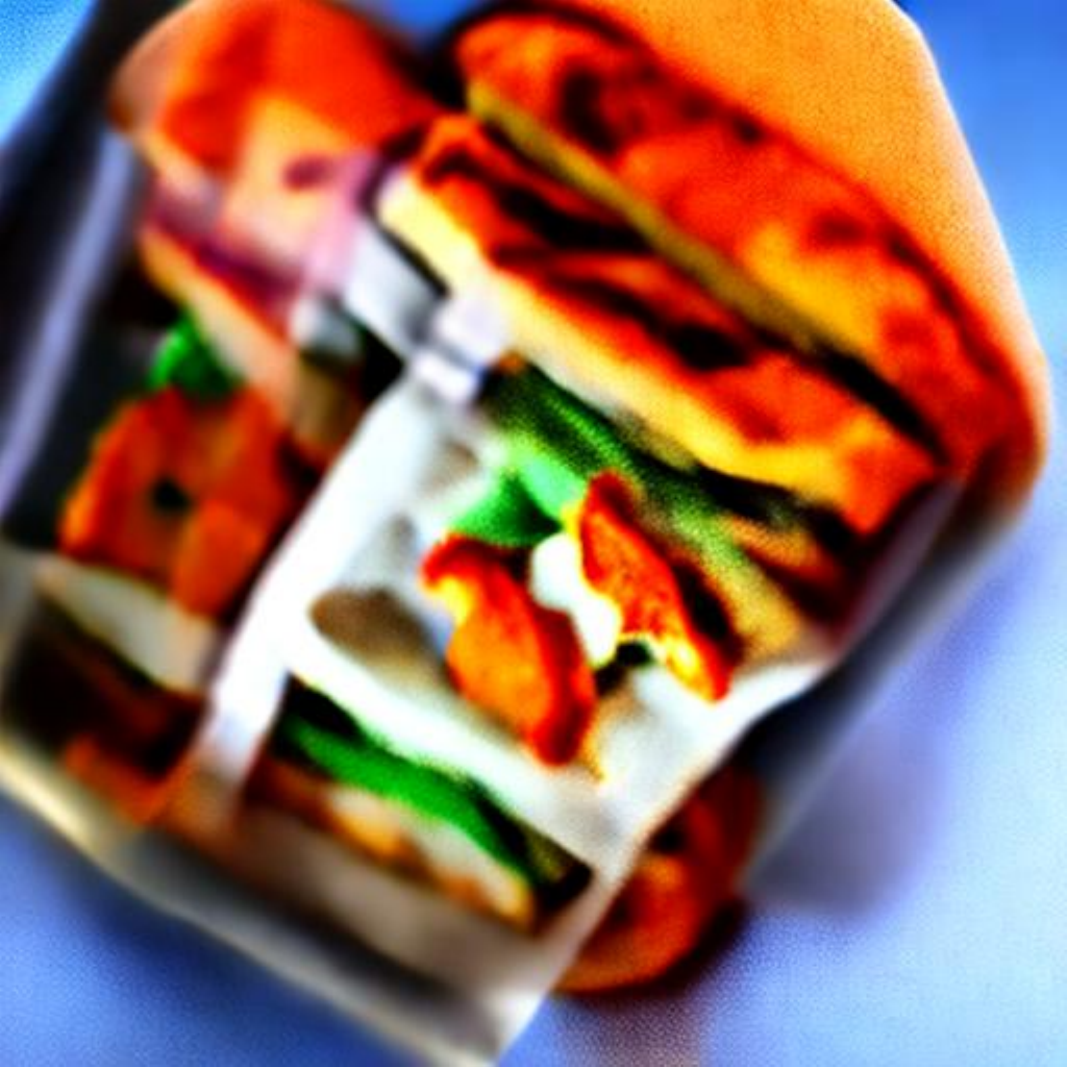} &
\includegraphics[width=0.10\textwidth]{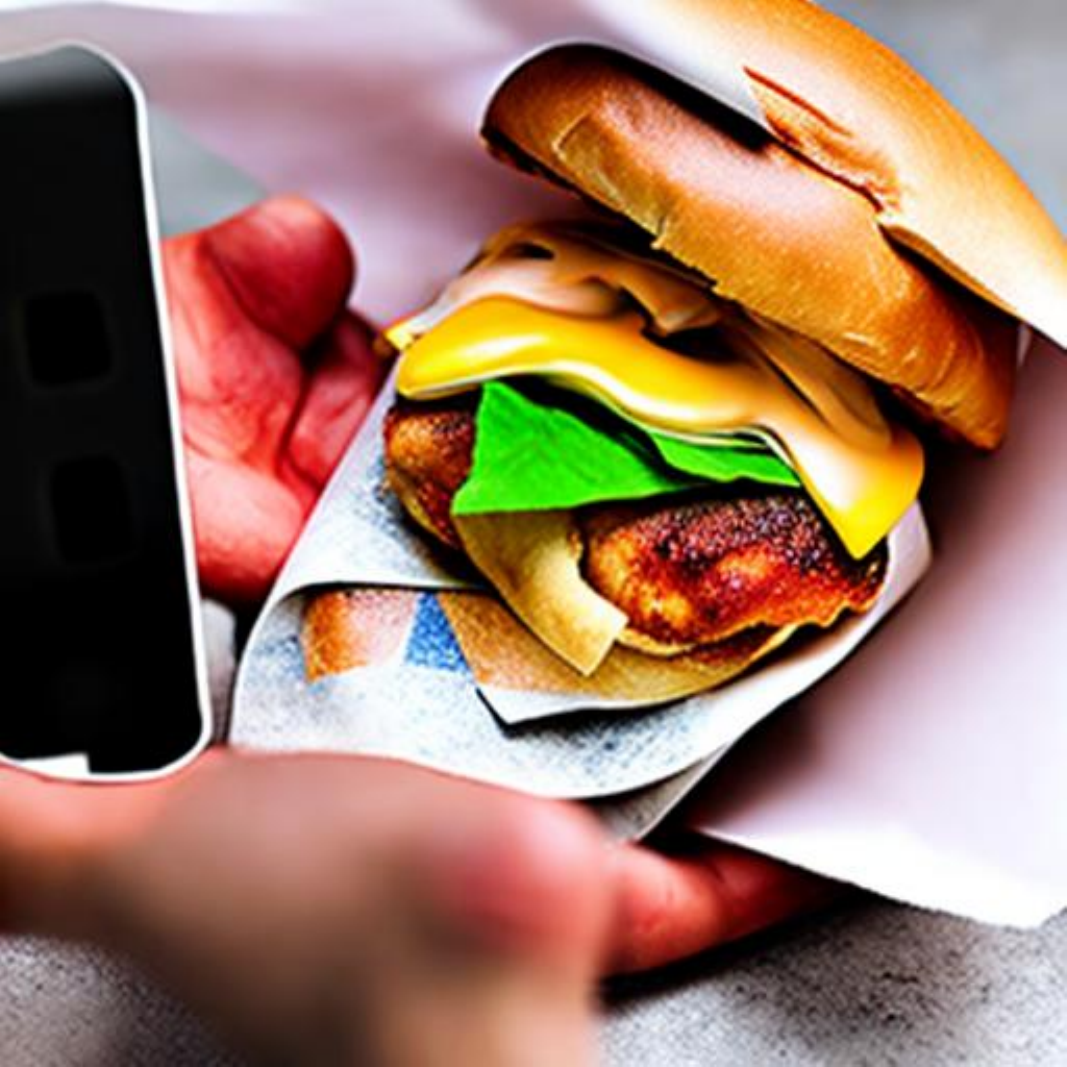} &
\includegraphics[width=0.10\textwidth]{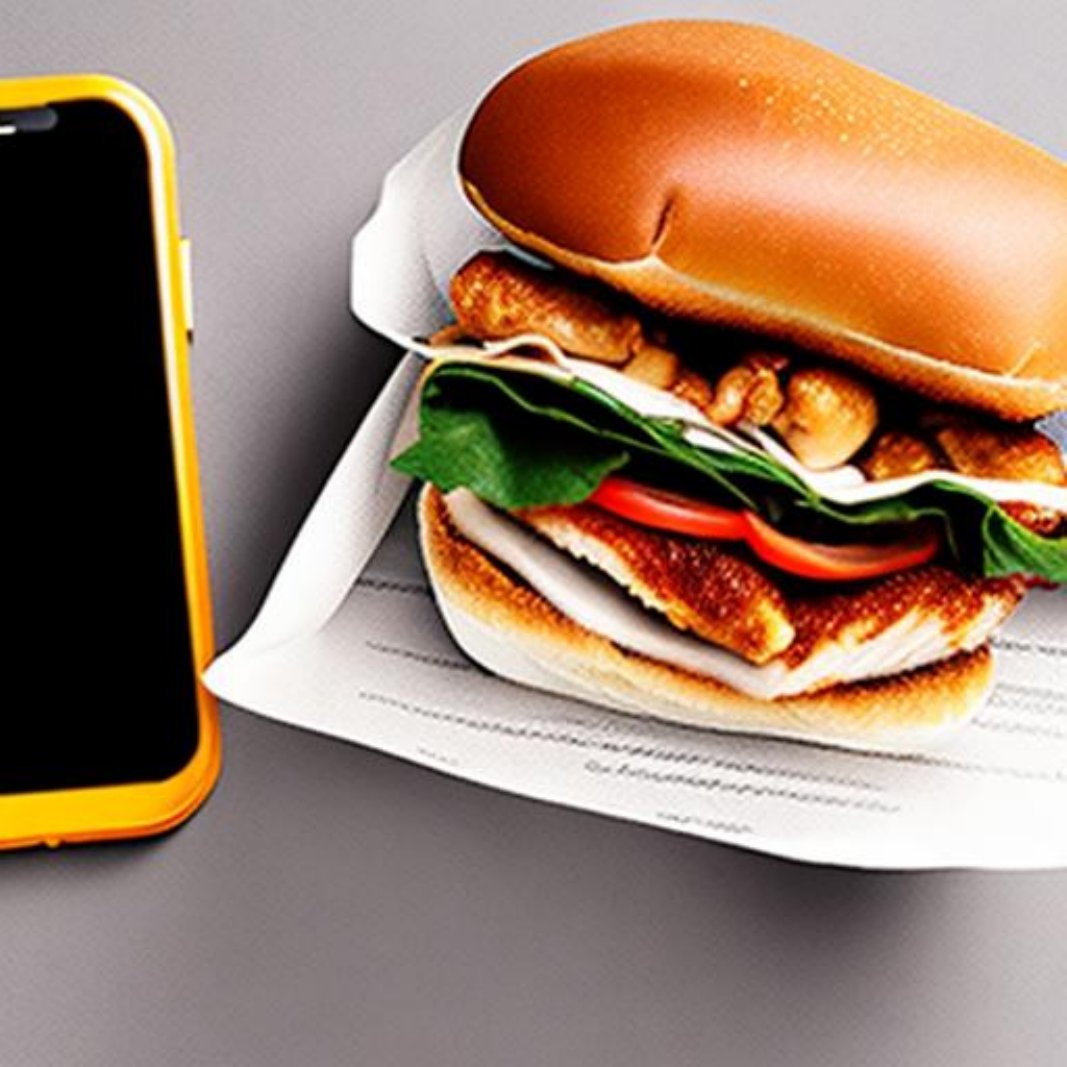} &
\includegraphics[width=0.10\textwidth]{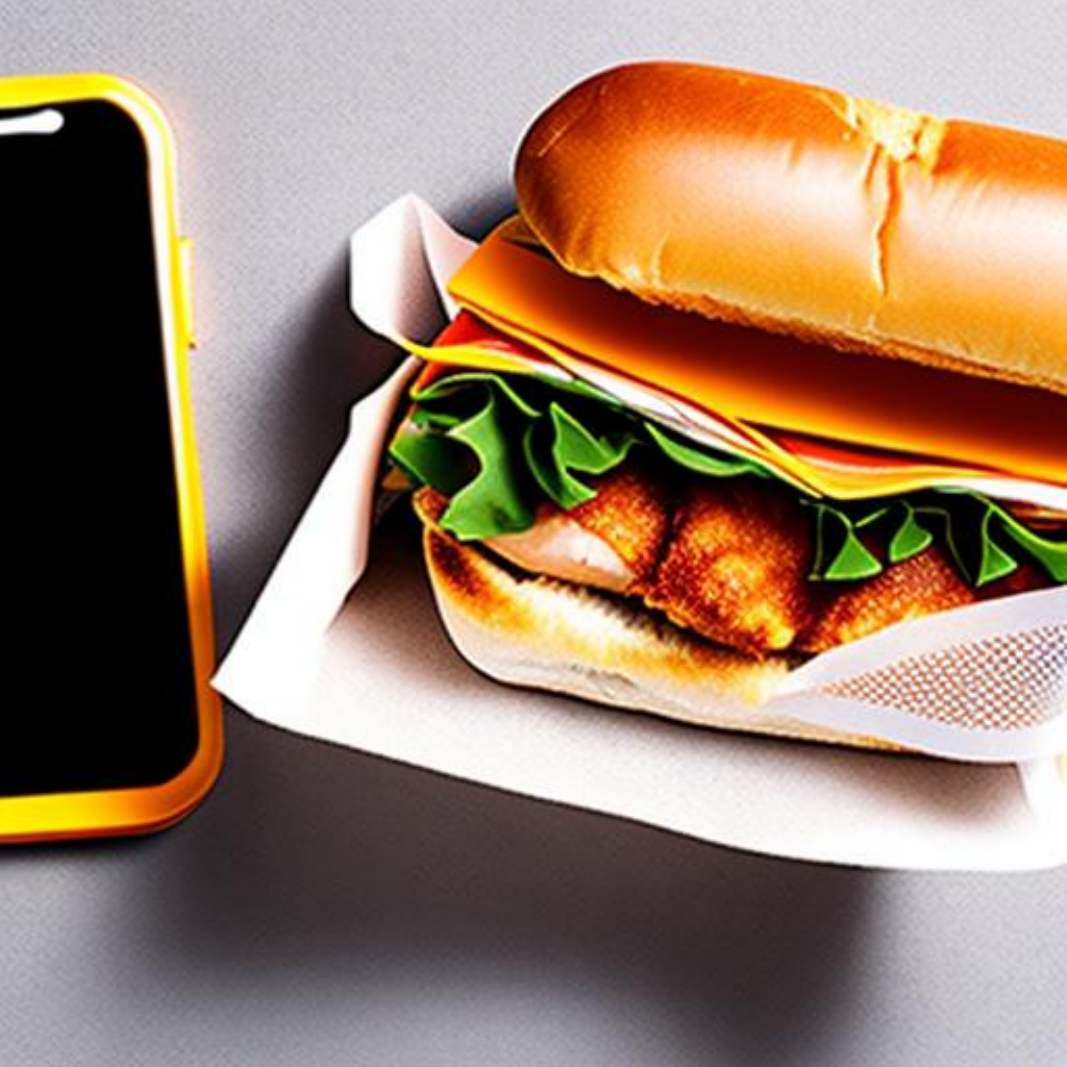} \\
\node[name=l4-1]{\includegraphics[width=0.10\textwidth]{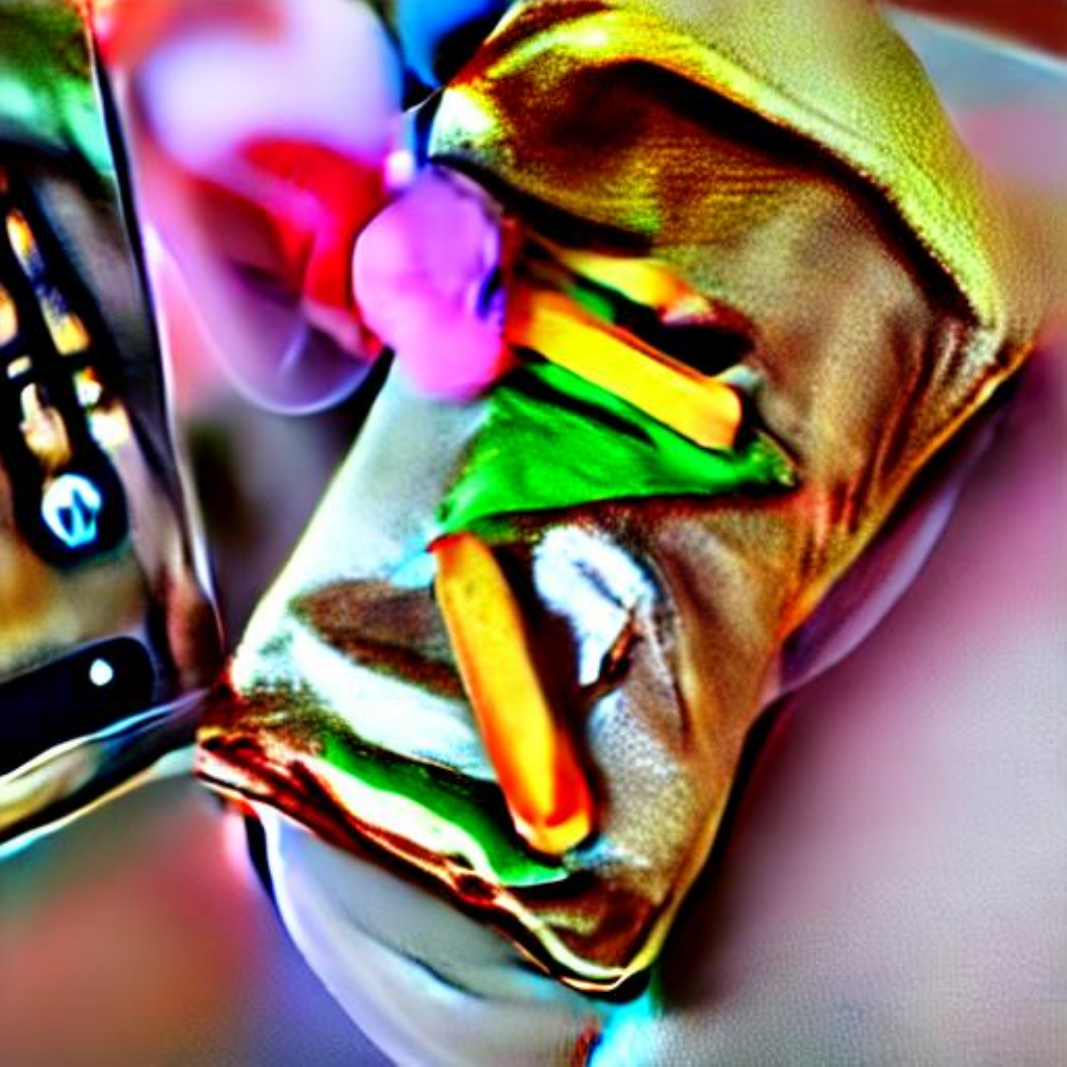}}; &
\node[name=l4-2]{\includegraphics[width=0.10\textwidth]{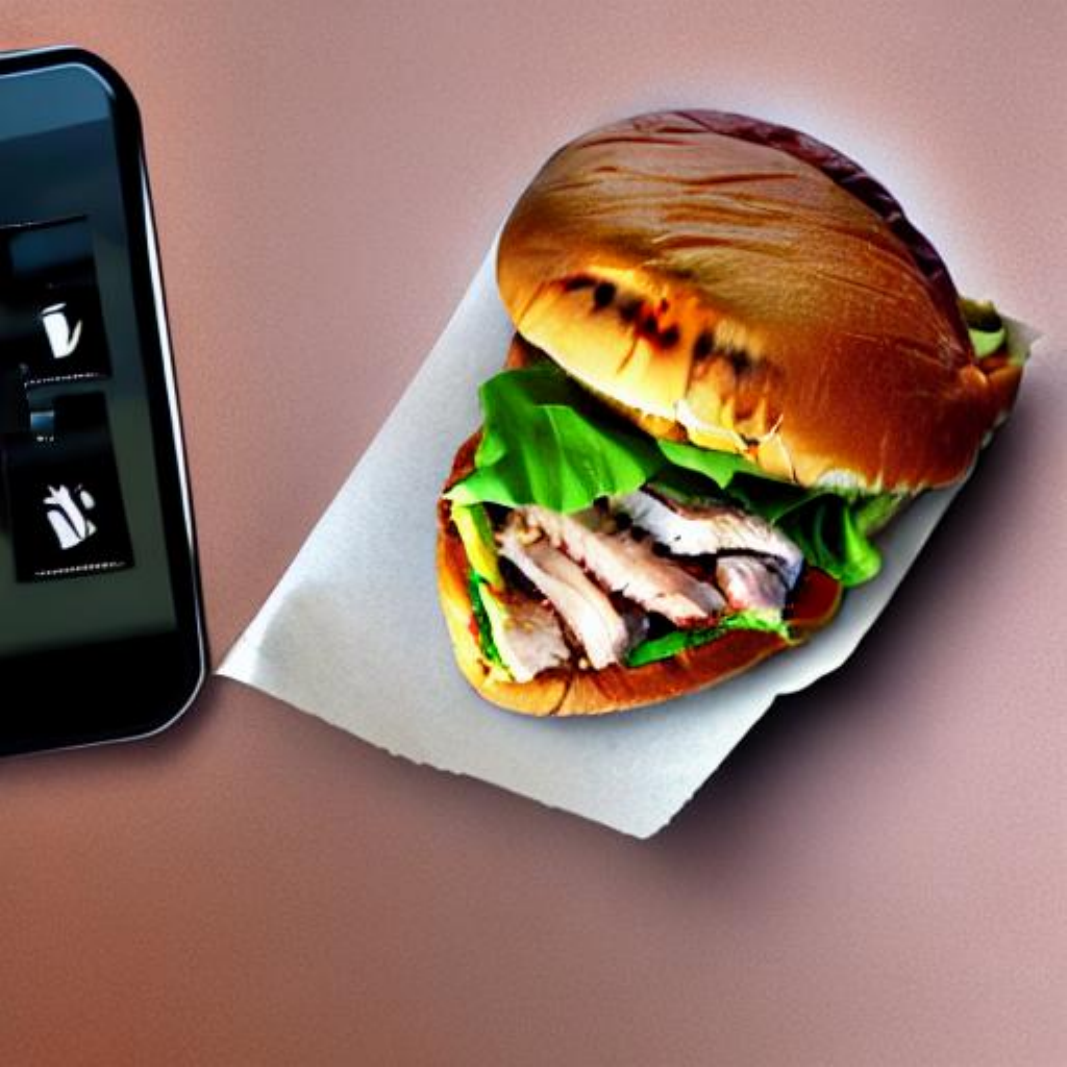}}; &
\node[name=l4-3]{\includegraphics[width=0.10\textwidth]{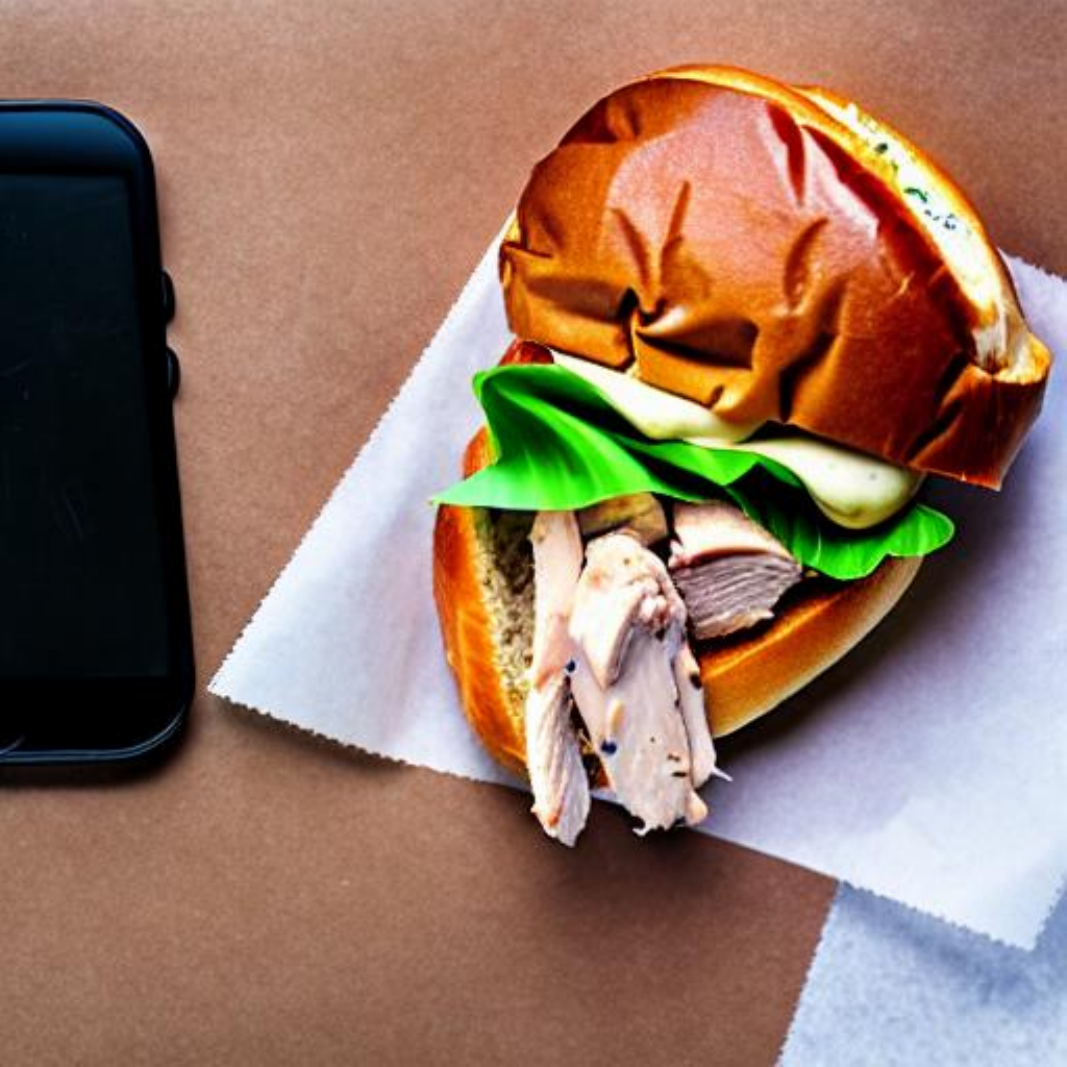}}; &
\node[name=l4-4]{\includegraphics[width=0.10\textwidth]{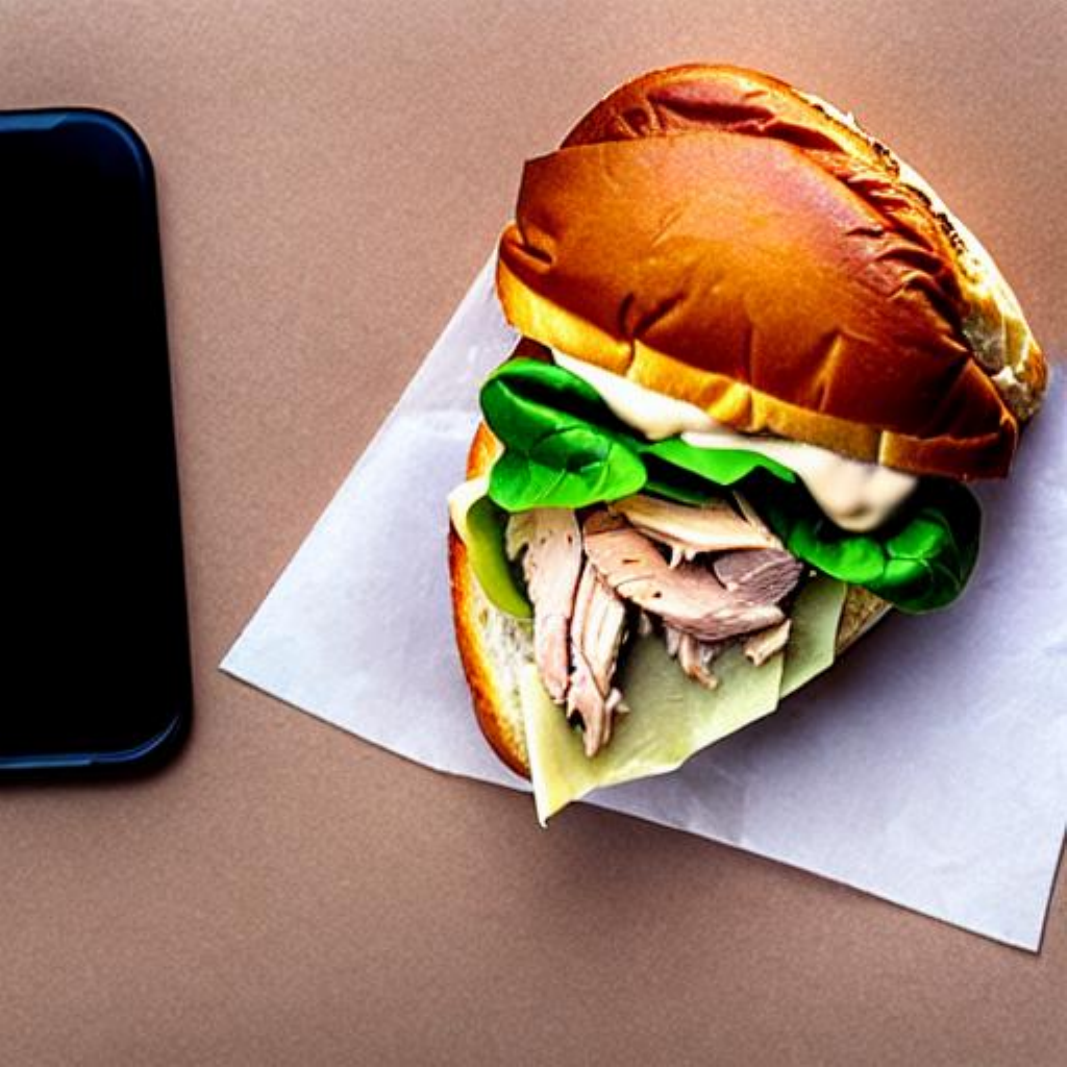}}; &
\node[minimum width=14pt]{}; &
\node[name=r4-1]{\includegraphics[width=0.10\textwidth]{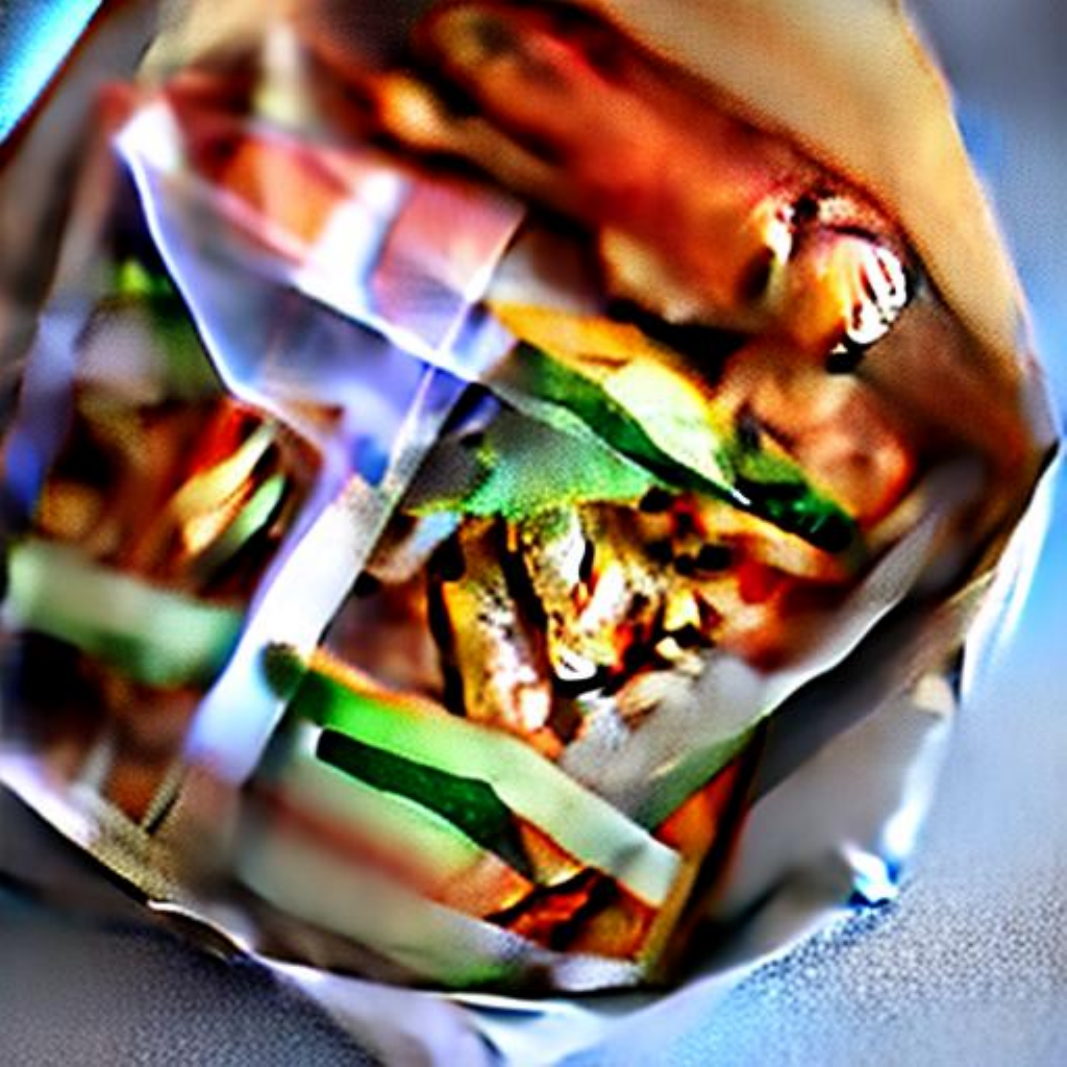}}; &
\node[name=r4-2]{\includegraphics[width=0.10\textwidth]{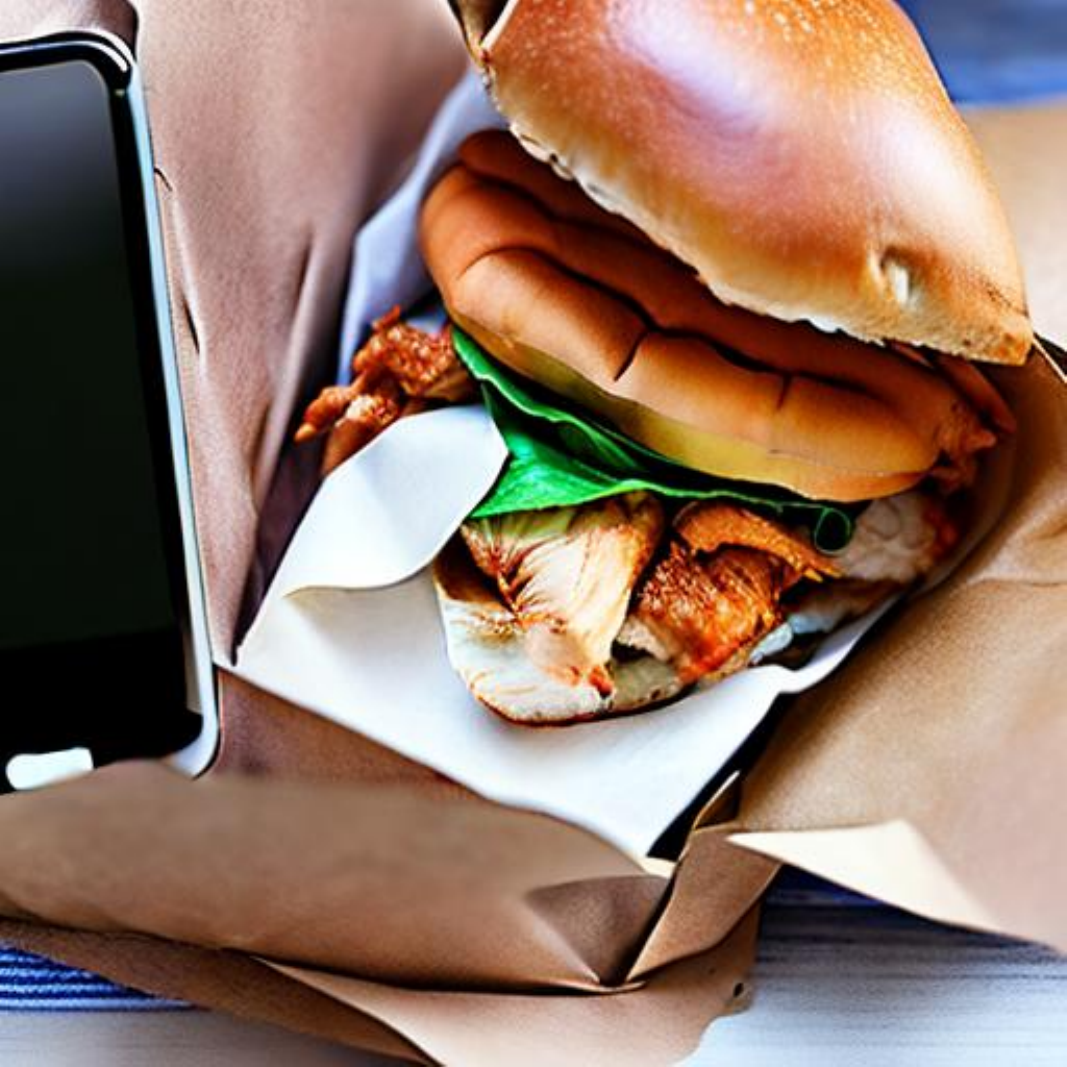}}; &
\node[name=r4-3]{\includegraphics[width=0.10\textwidth]{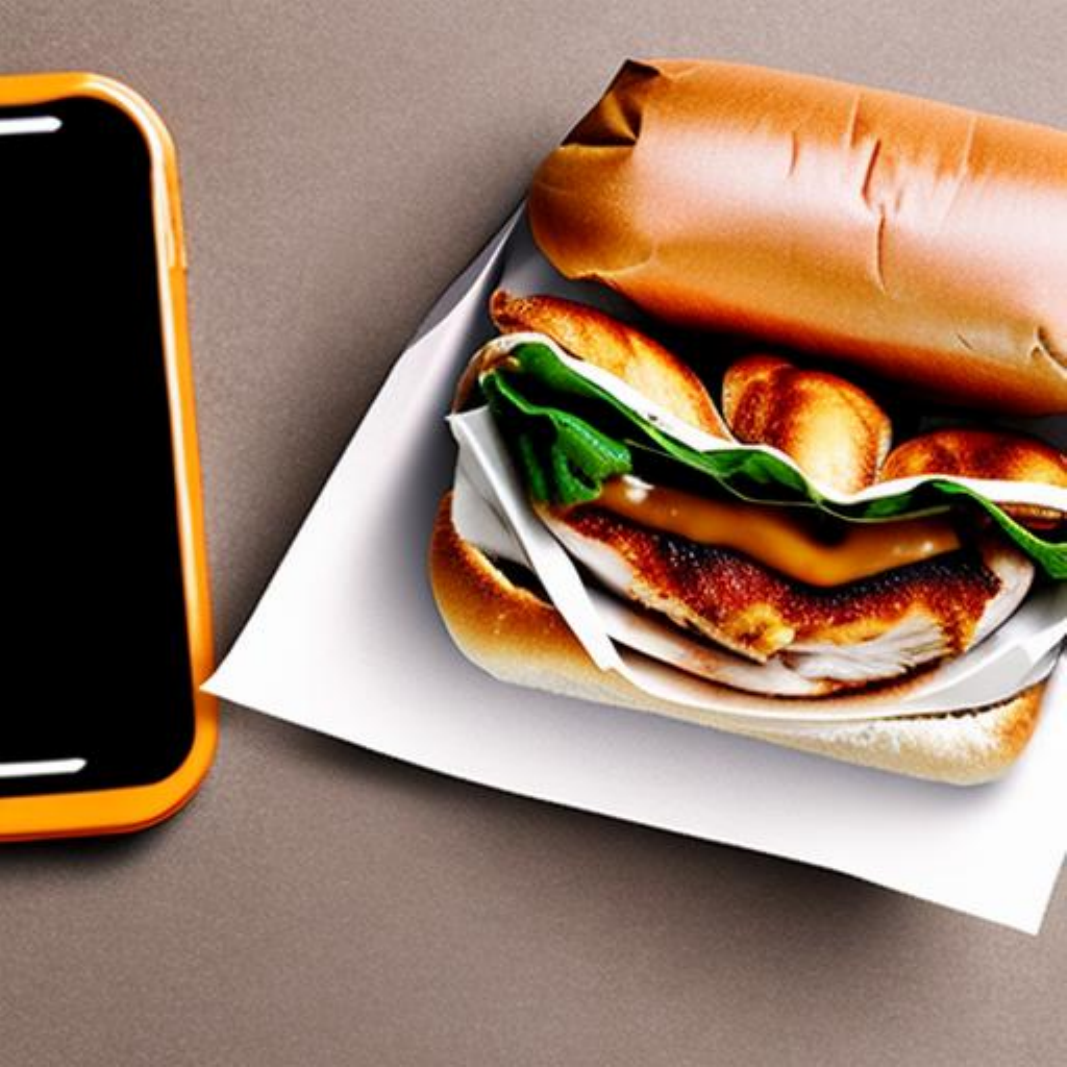}}; &
\node[name=r4-4]{\includegraphics[width=0.10\textwidth]{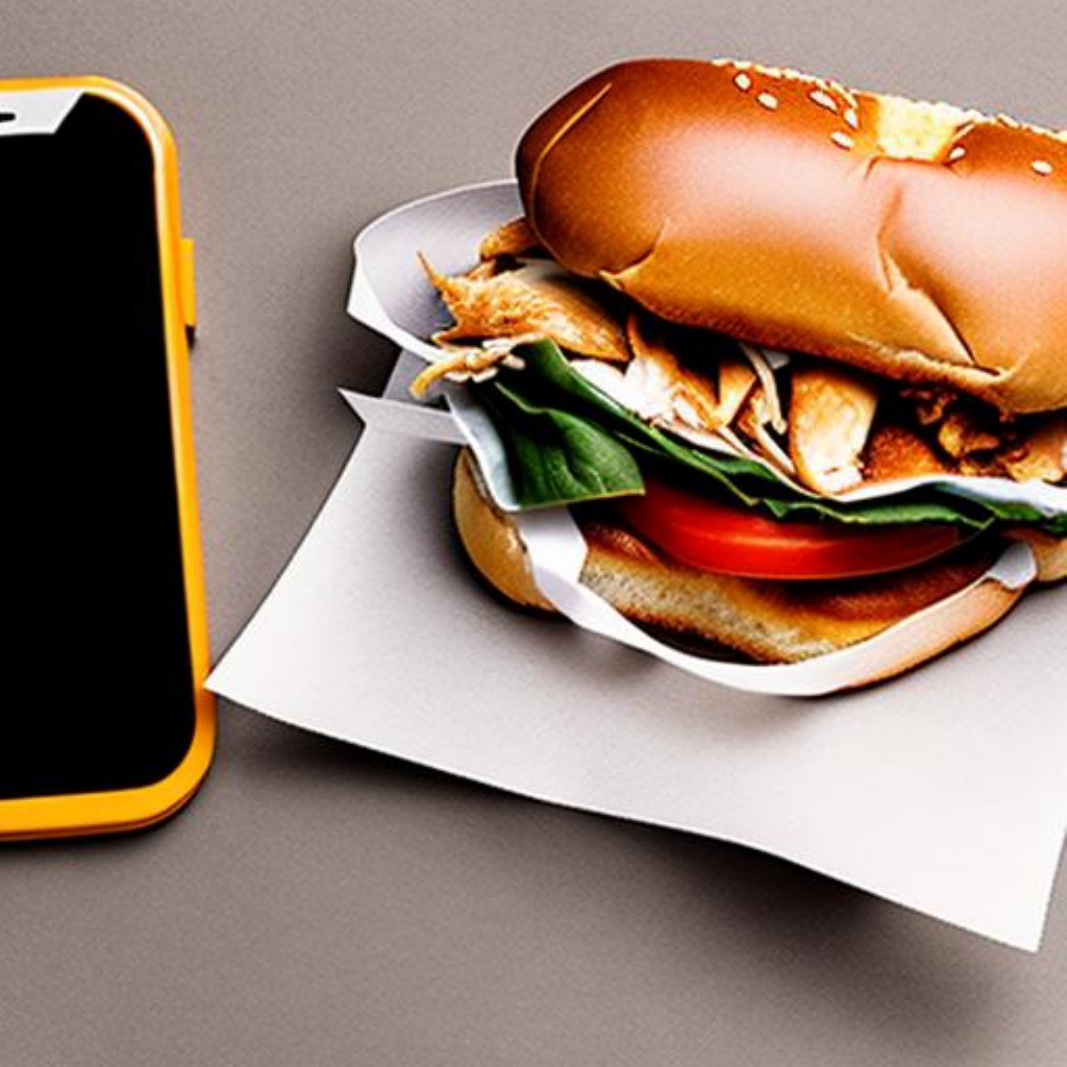}}; \\
};

\node[rotate=90, anchor=center, yshift=8pt, font=\small\bfseries] at (l1-1.west) {Constant CFG};
\node[rotate=90, anchor=center, yshift=8pt, font=\small\bfseries] at (l2-1.west) {Interval CFG};
\node[rotate=90, anchor=center, yshift=8pt, font=\small\bfseries] at (l3-1.west) {$\beta$--CFG};
\node[rotate=90, anchor=center, yshift=8pt, font=\small\bfseries] at (l4-1.west) {DG-CFG (ours)};

\node[above=3pt] at (l1-1.north) {$\text{NFE}=4$};
\node[above=3pt] at (l1-2.north) {$\text{NFE}=8$};
\node[above=3pt] at (l1-3.north) {$\text{NFE}=16$};
\node[above=3pt] at (l1-4.north) {$\text{NFE}=32$};
\node[above=3pt] at (r1-1.north) {$\text{NFE}=4$};
\node[above=3pt] at (r1-2.north) {$\text{NFE}=8$};
\node[above=3pt] at (r1-3.north) {$\text{NFE}=16$};
\node[above=3pt] at (r1-4.north) {$\text{NFE}=32$};

\node[below=6pt, font=\small\bfseries] at ($(l4-1.south)!0.5!(l4-4.south)$) {SD 1.5};
\node[below=6pt, font=\small\bfseries] at ($(r4-1.south)!0.5!(r4-4.south)$) {SD 2.1};

\end{tikzpicture}
\caption{Qualitative comparison across NFE settings on Stable Diffusion~1.5 and Stable Diffusion~2.1 at $\bar{\omega}=23$. Rows correspond to guidance methods and columns correspond to $\text{NFE}\in\{4,8,16,32\}$ within each backbone block. The left block shows SD1.5 and the right block shows SD2.1. The prompt is \textit{a chicken sandwich in a wrapper near a cell phone}.}
\label{fig:nfe-ablation}
\end{figure}

\subsubsection{Minimal NFE to Reach Fixed Quality Targets}
\label{sec:nfe-threshold}

The preceding evaluation compares methods at equal sampling budgets. We now ask the converse question on SD1.5 and SD2.1: given a fixed quality target, how many sampling steps does each method require to reach it? This fixed-quality view is practically relevant because it measures the inference cost needed to attain a prescribed output quality.

We set $\bar{\omega}=23$, a challenging yet informative high-guidance regime. At this strength, color over-saturation and discretization errors become pronounced, making differences across NFE settings easier to resolve; when these artifacts are controlled, strong guidance can also yield substantial gains in prompt alignment and perceptual quality. For SD1.5, we define five quality targets: three ImageReward thresholds ($\text{IR}\geq0$, $\text{IR}\geq0.1$, and $\text{IR}\geq0.2$), one text-alignment threshold ($\text{CLIP}\geq30.5$), and one distributional-quality threshold ($\text{FID}\leq75$). For SD2.1, we use four targets: $\text{IR}\geq0.2$, $\text{IR}\geq0.4$, $\text{CLIP}\geq30.5$, and $\text{FID}\leq75$. For each method--target pair, we evaluate a range of NFE values and report the minimum NFE required to meet the corresponding criterion.

\begin{table}[ht]
\centering
\caption{Minimum NFE required to reach fixed quality targets on Stable Diffusion~1.5 and Stable Diffusion~2.1 at $\bar{\omega}=23$. Lower values indicate fewer sampling steps. Best results are shown in \textbf{bold}; $>100$ indicates that the target is not reached within 100 NFE.}
\label{tab:nfe-threshold}
\small
\begin{tabular}{llcccc}
\toprule
Backbone & Target & Constant CFG & Interval CFG & $\beta$--CFG & \textbf{DG-CFG (ours)} \\
\midrule
\multirow{5}{*}{SD1.5}
& IR $\geq 0$      & 12 & 19 & 10 & \textbf{8} \\
& IR $\geq 0.1$    & 15 & 24 & 12 & \textbf{9} \\
& IR $\geq 0.2$    & 18 & 36 & 16 & \textbf{12} \\
& CLIP $\geq 30.5$ & 8  & 12 & 8  & \textbf{6} \\
& FID $\leq 75$    & 12 & 36 & 13 & \textbf{8} \\
\midrule
\multirow{4}{*}{SD2.1}
& IR $\geq 0.2$    & 11 & 20 & 10 & \textbf{8} \\
& IR $\geq 0.4$    & 16 & 38 & 15 & \textbf{12} \\
& CLIP $\geq 30.5$ & 7  & 11 & 7  & \textbf{6} \\
& FID $\leq 75$    & 12 & $>100$ & 14 & \textbf{9} \\
\bottomrule
\end{tabular}
\end{table}

Table~\ref{tab:nfe-threshold} summarizes the results. Across both backbones, DG-CFG reaches every quality target with the lowest NFE, demonstrating that its fixed-budget gains translate into lower inference cost at matched quality. This advantage holds across perceptual quality, text alignment, and distributional fidelity, and becomes more pronounced at stricter thresholds, where the baselines require substantially more integration steps.

Together with the fixed-budget results in Section~\ref{sec:nfe-sweep}, these findings demonstrate gains in both effectiveness and efficiency: DG-CFG improves generation quality at a given NFE and reaches prescribed quality targets with fewer sampling steps on both SD1.5 and SD2.1. The proposed schedule can therefore reduce inference latency and computational cost at matched output quality.

\clearpage
\section{Conclusion}
\label{sec:conclusion}

We establish an exact analytic characterization of the distribution induced by classifier-free guidance under deterministic probability flow ODE sampling, expressed through path-integral representations. These results show that CFG modifies the target distribution through an exponential correction whose temporal weighting is governed by $\omega(t)-1$, clarifying how the effect of guidance accumulates along the sampling trajectory. This characterization motivates DG-CFG, which redistributes guidance across timesteps while accounting for signal content and amplified score errors in the low-noise regime.

Toy-model experiments closely validate the predicted induced distributions. Across Stable Diffusion~1.5, Stable Diffusion~2.1, and Stable Diffusion~XL, DG-CFG provides a more favorable diversity--fidelity trade-off at lower guidance strengths and delivers higher image quality with better-controlled saturation than the baselines under strong guidance. Component ablations on SD1.5 and SD2.1 confirm the complementary roles of the schedule factors, while evaluations across NFE settings and fixed-quality targets show that these gains persist across sampling budgets and translate into fewer sampling steps at matched quality.

\textbf{Limitations and Future Work.} 
DG-CFG addresses temporal non-uniformity induced by the noise schedule but does not explicitly account for the geometry of the conditional--reference score discrepancy. Incorporating this geometry into schedule design and extending the distributional analysis to flow matching models may yield principled guidance strategies for a broader class of generative models.

\bigskip
\noindent{\bf Acknowledgments}\\
Zheng Ma is supported by NSFC Grant No. 12531016 and  Beijing Institute of Applied
Physics and Computational Mathematics funding HX02023-6. Additionally, we also thank Shanghai Institute for Mathematics and Interdisciplinary Sciences (SIMIS) for their financial support. This research was funded by SIMIS under grant number SIMIS-ID-2025-ST. The authors are grateful for the resources and facilities provided by SIMIS, which were essential for the completion of this work.

\bibliographystyle{unsrt}
\bibliography{references}

\appendix
\section{Proofs of Main Results}
\label{sec:appendix}

This appendix gives the proofs of the three main theorems in Section~\ref{sec:dist-cfg}. The arguments are organized by their logical dependence: we first derive the Fokker--Planck equation for the reference distribution $\tilde{r}_t$ (Theorem~\ref{thm:fp}), then use it to characterize the deterministic CFG-induced distribution (Theorem~\ref{thm:ddim}), and finally extend the result to time-dependent guidance (Theorem~\ref{thm:time-dep}).

\subsection{Proof of Theorem~\ref{thm:fp} (Fokker--Planck Equation for $\tilde{r}_t$)}
\label{sec:proof-fp}

We start from the forward Fokker--Planck equations satisfied by the two marginals:
\begin{align}
    \frac{\partial p_t}{\partial t} &= -\nabla \cdot (f p_t) + \frac{1}{2}g^2 \Delta p_t, \label{ap:fp-p}\\
    \frac{\partial q_t}{\partial t} &= -\nabla \cdot (f q_t) + \frac{1}{2}g^2 \Delta q_t. \label{ap:fp-q}
\end{align}

\medskip\noindent\textbf{Step 1: Evolution of the unnormalized product.}
Let $r_t = p_t^\omega q_t^{1-\omega}$. Applying the chain rule gives
\begin{align}
    \frac{\partial r_t}{\partial t}
    &= \omega p_t^{\omega-1} q_t^{1-\omega}\,\frac{\partial p_t}{\partial t}
    + (1-\omega) p_t^\omega q_t^{-\omega}\,\frac{\partial q_t}{\partial t} \notag\\[4pt]
    &= \omega p_t^{\omega-1} q_t^{1-\omega}\Big[-\nabla\cdot(f p_t) + \frac{1}{2}g^2\Delta p_t\Big]
    + (1-\omega) p_t^\omega q_t^{-\omega}\Big[-\nabla\cdot(f q_t) + \frac{1}{2}g^2\Delta q_t\Big]. \label{ap:rt-raw}
\end{align}

\emph{Drift terms.} Expanding the divergence terms, we obtain
\begin{align}
    &\omega p_t^{\omega-1} q_t^{1-\omega}\big[-\nabla\cdot(f p_t)\big]
    + (1-\omega) p_t^\omega q_t^{-\omega}\big[-\nabla\cdot(f q_t)\big] \notag\\
    &= -\omega p_t^{\omega-1} q_t^{1-\omega}\big[(\nabla\cdot f)p_t + f\cdot\nabla p_t\big]
    - (1-\omega) p_t^\omega q_t^{-\omega}\big[(\nabla\cdot f)q_t + f\cdot\nabla q_t\big] \notag\\
    &= -(\nabla\cdot f)\,p_t^\omega q_t^{1-\omega}
    - f\cdot\big(\omega p_t^{\omega-1} q_t^{1-\omega}\nabla p_t + (1-\omega)p_t^\omega q_t^{-\omega}\nabla q_t\big) \notag\\
    &= -(\nabla\cdot f) r_t - f\cdot \nabla r_t
    = -\nabla\cdot(f r_t). \label{ap:drift-result}
\end{align}

\emph{Diffusion terms.} Define
\begin{equation}
    A_t = \frac{1}{2}g^2\Big[\omega p_t^{\omega-1} q_t^{1-\omega}\Delta p_t
    + (1-\omega) p_t^\omega q_t^{-\omega}\Delta q_t\Big].
    \label{ap:diffusion-contribution}
\end{equation}
The Laplacian of $r_t$ can be decomposed as
\begin{align}
    \Delta r_t &= \omega p_t^{\omega-1} q_t^{1-\omega}\Delta p_t
    + (1-\omega) p_t^\omega q_t^{-\omega}\Delta q_t \notag\\
    &\quad + \sum_{i=1}^n\Big(\frac{\partial}{\partial x_i}(\omega p_t^{\omega-1} q_t^{1-\omega})\frac{\partial p_t}{\partial x_i}
    + \frac{\partial}{\partial x_i}((1-\omega)p_t^\omega q_t^{-\omega})\frac{\partial q_t}{\partial x_i}\Big). \label{ap:delta-rt}
\end{align}
Therefore,
\begin{equation}
    A_t = \frac{1}{2}g^2\Delta r_t
    - \frac{1}{2}g^2\sum_{i=1}^n\Big(\frac{\partial}{\partial x_i}(\omega p_t^{\omega-1} q_t^{1-\omega})\frac{\partial p_t}{\partial x_i}
    + \frac{\partial}{\partial x_i}((1-\omega)p_t^\omega q_t^{-\omega})\frac{\partial q_t}{\partial x_i}\Big).
    \label{ap:diffusion-expanded}
\end{equation}

\emph{Residual term.}
Observe that $\omega p_t^{\omega-1} q_t^{1-\omega} = \omega (p_t/q_t)^{\omega-1}$ and
$(1-\omega)p_t^\omega q_t^{-\omega} = (1-\omega)(p_t/q_t)^\omega$. Therefore,
\begin{align}
    \frac{\partial}{\partial x_i}\big(\omega p_t^{\omega-1} q_t^{1-\omega}\big)
    &= \omega(\omega-1)\Big(\frac{p_t}{q_t}\Big)^{\omega-2}
    \frac{\partial}{\partial x_i}\!\Big(\frac{p_t}{q_t}\Big), \\[4pt]
    \frac{\partial}{\partial x_i}\big((1-\omega)p_t^\omega q_t^{-\omega}\big)
    &= (1-\omega)\omega\Big(\frac{p_t}{q_t}\Big)^{\omega-1}
    \frac{\partial}{\partial x_i}\!\Big(\frac{p_t}{q_t}\Big)
    = -\omega(\omega-1)\Big(\frac{p_t}{q_t}\Big)^{\omega-1}
    \frac{\partial}{\partial x_i}\!\Big(\frac{p_t}{q_t}\Big).
\end{align}

The residual sum is then
\begin{align}
    &\sum_{i=1}^n\Big(\frac{\partial}{\partial x_i}(\omega p_t^{\omega-1} q_t^{1-\omega})\frac{\partial p_t}{\partial x_i}
    + \frac{\partial}{\partial x_i}((1-\omega)p_t^\omega q_t^{-\omega})\frac{\partial q_t}{\partial x_i}\Big) \notag\\
    &= \sum_{i=1}^n \omega(\omega-1)\Big(\frac{p_t}{q_t}\Big)^{\omega-2}
    \frac{\partial}{\partial x_i}\!\Big(\frac{p_t}{q_t}\Big)
    \Big(\frac{\partial p_t}{\partial x_i} - \frac{p_t}{q_t}\frac{\partial q_t}{\partial x_i}\Big).
\end{align}

Using $\frac{\partial}{\partial x_i}(p_t/q_t) = \frac{1}{q_t}(\frac{\partial p_t}{\partial x_i}
- \frac{p_t}{q_t}\frac{\partial q_t}{\partial x_i})$,
\begin{align}
    \text{Residual} &= \sum_{i=1}^n \omega(\omega-1)\Big(\frac{p_t}{q_t}\Big)^{\omega-2}
    \frac{1}{q_t}\Big(\frac{\partial p_t}{\partial x_i}
    - \frac{p_t}{q_t}\frac{\partial q_t}{\partial x_i}\Big)^2 \notag\\
    &= \sum_{i=1}^n \omega(\omega-1) p_t^{\omega-2} q_t^{1-\omega}
    \Big(\frac{\partial p_t}{\partial x_i}
    - \frac{p_t}{q_t}\frac{\partial q_t}{\partial x_i}\Big)^2 \notag\\
    &= \sum_{i=1}^n \omega(\omega-1) p_t^\omega q_t^{1-\omega}
    \Big(\frac{1}{p_t}\frac{\partial p_t}{\partial x_i}
    - \frac{1}{q_t}\frac{\partial q_t}{\partial x_i}\Big)^2 \notag\\
    &= \omega(\omega-1)\,r_t\sum_{i=1}^n\big[(\nabla_x\log p_t)_i - (\nabla_x\log q_t)_i\big]^2 \notag\\
    &= \omega(\omega-1)\,\|\nabla_x\log p_t - \nabla_x\log q_t\|_2^2\,r_t
    = V_t\,r_t.
    \label{ap:V-result}
\end{align}

Combining the drift contribution \eqref{ap:drift-result}, the diffusion decomposition \eqref{ap:diffusion-expanded}, and the residual identity \eqref{ap:V-result} gives
\begin{equation}
    \frac{\partial r_t}{\partial t} = -\nabla\cdot(f r_t) + \frac{1}{2}g^2\Delta r_t - \frac{1}{2}g^2 V_t r_t.
    \label{ap:fp-rt}
\end{equation}

\medskip\noindent\textbf{Step 2: Evolution of the normalizing constant.}
Integrating \eqref{ap:fp-rt} over $\mathbb{R}^n$ and applying the divergence theorem, using the assumed boundary decay $f r_t \to 0$ and $\nabla r_t \to 0$ as $|x| \to \infty$, yields
\begin{equation}
    \frac{dZ_t}{dt} = \int_{\mathbb{R}^n} \frac{\partial r_t}{\partial t}\,dx
    = -\frac{1}{2}g^2\int_{\mathbb{R}^n} V_t r_t\,dx
    = -\frac{1}{2}g^2 Z_t\,\mathbb{E}_{\tilde{r}_t}[V_t].
    \label{ap:dZdt}
\end{equation}

\medskip\noindent\textbf{Step 3: Evolution of the normalized density.}
Finally, since $\tilde{r}_t = r_t/Z_t$,
\begin{align}
    \frac{\partial\tilde{r}_t}{\partial t}
    &= \frac{1}{Z_t}\frac{\partial r_t}{\partial t} - \frac{r_t}{Z_t^2}\frac{dZ_t}{dt} \notag\\
    &= \frac{1}{Z_t}\Big[-\nabla\cdot(f r_t) + \frac{1}{2}g^2\Delta r_t - \frac{1}{2}g^2 V_t r_t\Big]
    - \frac{r_t}{Z_t^2}\Big[-\frac{1}{2}g^2 Z_t\,\mathbb{E}_{\tilde{r}_t}[V_t]\Big] \notag\\
    &= -\nabla\cdot(f\tilde{r}_t) + \frac{1}{2}g^2\Delta\tilde{r}_t
    - \frac{1}{2}g^2\big(V_t - \mathbb{E}_{\tilde{r}_t}[V_t]\big)\tilde{r}_t,
    \label{ap:fp-tilde-final}
\end{align}
which completes the proof. \hfill $\square$

\subsection{Proof of Theorem~\ref{thm:ddim} (Exact CFG Sampling Distribution)}
\label{sec:proof-ddim}

With the evolution equation for $\tilde{r}_t$ established, we now analyze the density
transported by the deterministic guided sampler. The key step is to compare this
density with the reference density through the ratio
$\rho_t = \hat{r}_t / \tilde{r}_t$, and then solve the resulting transport equation
along characteristics.

\medskip\noindent\textbf{Step 1: Fokker--Planck equation for $\hat{r}_t$.}
Under the deterministic ODE $dx = (f - \frac{1}{2}g^2\nabla_x\log\tilde{r}_t)\,dt$,
the density $\hat{r}_t$ satisfies the continuity equation
\begin{align}
    \frac{\partial\hat{r}_t}{\partial t}
    &= -\nabla\cdot\Big[\big(f - \frac{1}{2}g^2\nabla_x\log\tilde{r}_t\big)\hat{r}_t\Big] \notag\\
    &= -\nabla\cdot(f\hat{r}_t) + \frac{1}{2}g^2\,\nabla\cdot(\hat{r}_t\nabla_x\log\tilde{r}_t).
    \label{ap:fp-hat}
\end{align}

\medskip\noindent\textbf{Step 2: Equation for the density ratio.}
Applying the quotient rule to $\rho_t = \hat{r}_t/\tilde{r}_t$ gives
\begin{align}
    \frac{\partial\rho_t}{\partial t}
    &= \frac{1}{\tilde{r}_t}\frac{\partial\hat{r}_t}{\partial t}
    - \frac{\hat{r}_t}{\tilde{r}_t^2}\frac{\partial\tilde{r}_t}{\partial t} \notag\\
    &= \frac{1}{\tilde{r}_t}\Big[-\nabla\cdot(f\hat{r}_t) + \frac{1}{2}g^2\nabla\cdot(\hat{r}_t\nabla_x\log\tilde{r}_t)\Big] \notag\\
    &\quad - \frac{\hat{r}_t}{\tilde{r}_t^2}\Big[-\nabla\cdot(f\tilde{r}_t) + \frac{1}{2}g^2\Delta\tilde{r}_t
    - \frac{1}{2}g^2(V - \mathbb{E}_{\tilde{r}_t}[V])\tilde{r}_t\Big]. \label{ap:rho-raw}
\end{align}

\emph{Drift terms.}
\begin{align}
    -\frac{1}{\tilde{r}_t}\nabla\cdot(f\hat{r}_t) + \frac{\hat{r}_t}{\tilde{r}_t^2}\nabla\cdot(f\tilde{r}_t)
    &= -f\cdot\Big(\frac{\nabla\hat{r}_t}{\tilde{r}_t} - \frac{\hat{r}_t\nabla\tilde{r}_t}{\tilde{r}_t^2}\Big)
    = -f\cdot\nabla\rho_t. \label{ap:rho-drift}
\end{align}

\emph{Diffusion terms.} Since $\hat{r}_t\nabla_x\log\tilde{r}_t = \rho_t\nabla\tilde{r}_t$,
\begin{align}
    \frac{1}{\tilde{r}_t}\cdot\frac{1}{2}g^2\nabla\cdot(\hat{r}_t\nabla_x\log\tilde{r}_t)
    &= \frac{g^2}{2\tilde{r}_t}\nabla\cdot(\rho_t\nabla\tilde{r}_t) \notag\\
    &= \frac{1}{2}g^2\nabla\log\tilde{r}_t\cdot\nabla\rho_t + \frac{1}{2}g^2\frac{\rho_t}{\tilde{r}_t}\Delta\tilde{r}_t,
\end{align}
while
\begin{equation}
    -\frac{\hat{r}_t}{\tilde{r}_t^2}\cdot\frac{1}{2}g^2\Delta\tilde{r}_t
    = -\frac{1}{2}g^2\frac{\rho_t}{\tilde{r}_t}\Delta\tilde{r}_t.
\end{equation}
The $\Delta\tilde{r}_t$ terms cancel, leaving $\frac{1}{2}g^2\nabla\log\tilde{r}_t\cdot\nabla\rho_t$.

\emph{Potential term.}
\begin{equation}
    -\frac{\hat{r}_t}{\tilde{r}_t^2}\cdot\Big[-\frac{1}{2}g^2(V - \mathbb{E}_{\tilde{r}_t}[V])\tilde{r}_t\Big]
    = \frac{1}{2}g^2(V - \mathbb{E}_{\tilde{r}_t}[V])\rho_t.
\end{equation}

Collecting the drift, diffusion, and potential contributions yields
\begin{equation}
    \frac{\partial\rho_t}{\partial t} = -f\cdot\nabla\rho_t
    + \frac{1}{2}g^2\nabla\log\tilde{r}_t\cdot\nabla\rho_t
    + \frac{1}{2}g^2(V - \mathbb{E}_{\tilde{r}_t}[V])\rho_t.
    \label{ap:rho-pde}
\end{equation}

\medskip\noindent\textbf{Step 3: Solution along characteristics.}
Along the characteristic ODE $dX_s/ds = f - \frac{1}{2}g^2\nabla_x\log\tilde{r}_s$, the total derivative of $\rho_s(X_s)$ is
\begin{align}
    \frac{d}{ds}\rho_s(X_s)
    &= \frac{\partial\rho_s}{\partial s} + \nabla\rho_s\cdot\frac{dX_s}{ds} \notag\\
    &= \frac{\partial\rho_s}{\partial s} + \nabla\rho_s\cdot\big(f - \frac{1}{2}g^2\nabla\log\tilde{r}_s\big) \notag\\
    &= \frac{1}{2}g^2(V_s - \mathbb{E}_{\tilde{r}_s}[V_s])\rho_s,
    \label{ap:ode-deriv}
\end{align}
where the last equality follows from \eqref{ap:rho-pde}. Integrating from $t_0$ to $T$ gives
\begin{equation}
    \log\frac{\rho_T(X_T)}{\rho_{t_0}(X_{t_0})}
    = \frac{1}{2}\int_{t_0}^T g^2(s)\big(V_s(X_s) - \mathbb{E}_{\tilde{r}_s}[V_s]\big)\,ds.
\end{equation}
Because the transported density is initialized from $\tilde{r}_T$, we have
$\rho_T = \hat{r}_T/\tilde{r}_T = 1$ at the terminal time. Hence
\begin{equation}
    \rho_{t_0}(x) = \exp\!\Big(-\frac{1}{2}\int_{t_0}^T g^2(s)\big(V_s(X_s^x) - \mathbb{E}_{\tilde{r}_s}[V_s]\big)\,ds\Big).
    \label{ap:rho-solution}
\end{equation}
Here $X_s^x$ denotes the solution of the characteristic ODE
$dX_s^x/ds = f(X_s^x,s)-\frac{1}{2}g^2(s)\nabla_x\log\tilde{r}_s(X_s^x)$
with initial condition $X_{t_0}^x=x$.

\medskip\noindent\textbf{Step 4: Product-reference form.}
Multiplying by $\tilde{r}_{t_0} = Z_{t_0}^{-1}p_{t_0}^\omega q_{t_0}^{1-\omega}$ and absorbing all $x$-independent factors into the normalization constant gives
\begin{equation}
    \hat{r}_{t_0}(x) = \frac{1}{Z'}p_{t_0}^\omega(x)q_{t_0}^{1-\omega}(x)\,
    \exp\!\Big(-\frac{1}{2}\int_{t_0}^T g^2(s)\,V_s(X_s^x)\,ds\Big).
\end{equation}

\medskip\noindent\textbf{Step 5: Equivalent form relative to $p_{t_0}$.}
It remains to rewrite the product-reference form as the equivalent expression in
Theorem~\ref{thm:ddim}. Let
$L_t(x) = \log p_t(x) - \log q_t(x)$, hence $\nabla L_t = \nabla\log p_t - \nabla\log q_t$
and $L_T(X_T^x)=\log p_T(X_T^x)-\log q_T(X_T^x)$.

\emph{Rewriting the prefactor.}
The prefactor $p_{t_0}^\omega q_{t_0}^{1-\omega}$ can be written as
\begin{align}
    p_{t_0}^\omega(x) q_{t_0}^{1-\omega}(x)
    &= p_{t_0}(x) \Big(\frac{p_{t_0}(x)}{q_{t_0}(x)}\Big)^{\omega-1} \notag\\
    &= p_{t_0}(x) e^{(\omega-1)L_{t_0}(x)} \notag\\
    &= p_{t_0}(x) \exp\!\Big((\omega-1)\big(L_T(X_T^x) - \int_{t_0}^T \frac{d}{ds}L_s(X_s^x)\,ds\big)\Big) \notag\\
    &= p_{t_0}(x) \exp\!\Big((\omega-1)L_T(X_T^x)
    -(\omega-1)\int_{t_0}^T \frac{d}{ds}L_s(X_s^x)\,ds\Big),
    \label{ap:prefactor}
\end{align}
where the integral is evaluated along the characteristic ODE \eqref{eq:char-ode}.

\emph{Computing $dL/ds$ along the characteristic.}
The Fokker--Planck equations \eqref{ap:fp-p}--\eqref{ap:fp-q} imply that the log-densities satisfy
\begin{align}
    \frac{\partial\log p_s}{\partial s} &= -\nabla\cdot f - f\cdot\nabla\log p_s
    + \frac{1}{2}g^2\big(\Delta\log p_s + \|\nabla\log p_s\|^2\big), \label{ap:logp-pde}\\
    \frac{\partial\log q_s}{\partial s} &= -\nabla\cdot f - f\cdot\nabla\log q_s
    + \frac{1}{2}g^2\big(\Delta\log q_s + \|\nabla\log q_s\|^2\big). \label{ap:logq-pde}
\end{align}
Subtracting \eqref{ap:logq-pde} from \eqref{ap:logp-pde} gives
\begin{equation}
    \frac{\partial L_s}{\partial s} = -f\cdot\nabla L_s
    + \frac{1}{2}g^2\big(\Delta L_s + \|\nabla\log p_s\|^2 - \|\nabla\log q_s\|^2\big).
    \label{ap:partial-L}
\end{equation}

Along the same characteristic ODE, the total derivative of $L$ is
\begin{align}
    \frac{dL_s}{ds}
    &= \frac{\partial L_s}{\partial s} + \nabla L_s \cdot \frac{dX_s}{ds} \notag\\
    &= \frac{\partial L_s}{\partial s} + \nabla L_s \cdot \big(f - \frac{1}{2}g^2\nabla\log\tilde{r}_s\big).
    \label{ap:dLds-raw}
\end{align}

Substituting \eqref{ap:partial-L} and
$\nabla\log\tilde{r}_s = \omega\nabla\log p_s + (1-\omega)\nabla\log q_s$ gives
\begin{align}
    \frac{dL_s}{ds}
    &= -f\cdot\nabla L_s + \frac{1}{2}g^2\big(\Delta L_s + \|\nabla\log p_s\|^2 - \|\nabla\log q_s\|^2\big) \notag\\
    &\quad + \nabla L_s\cdot f - \frac{1}{2}g^2\nabla L_s\cdot\big(\omega\nabla\log p_s + (1-\omega)\nabla\log q_s\big) \notag\\
    &= \frac{1}{2}g^2\Big[\Delta L_s + \|\nabla\log p_s\|^2 - \|\nabla\log q_s\|^2
        - \omega\|\nabla L_s\|^2 - \nabla L_s\cdot\nabla\log q_s\Big],
    \label{ap:dLds-mid}
\end{align}
where we used $\nabla L_s = \nabla\log p_s - \nabla\log q_s$ and
$\nabla L_s\cdot(\omega\nabla\log p_s + (1-\omega)\nabla\log q_s)
= \omega\|\nabla L_s\|^2 + \nabla L_s\cdot\nabla\log q_s$.

Using $\|\nabla\log p_s\|^2 - \|\nabla\log q_s\|^2
= \|\nabla L_s\|^2 + 2\nabla L_s\cdot\nabla\log q_s$, we obtain
\begin{equation}
    \boxed{\frac{dL_s}{ds}
    = \frac{1}{2}g^2\Big[\big(\Delta L_s + \|\nabla L_s\|^2 + \nabla L_s\cdot\nabla\log q_s\big)
        - \omega\|\nabla L_s\|^2\Big]}.
    \label{ap:dLds}
\end{equation}

\emph{Combining the exponents.}
We now combine the rewritten prefactor \eqref{ap:prefactor} with the $V$-integral from Step~4.
Since $V_s = \omega(\omega-1)\|\nabla L_s\|^2$, the full exponent becomes
\begin{align}
    &(\omega-1)L_T(X_T^x)
     -(\omega-1)\int_{t_0}^T \frac{dL_s}{ds}\,ds
     - \frac{1}{2}\int_{t_0}^T g^2(s)\,\omega(\omega-1)\|\nabla L_s\|^2\,ds \notag\\
    &= (\omega-1)L_T(X_T^x)
    -(\omega-1)\int_{t_0}^T \frac{1}{2}g^2
        \Big[\Delta L_s + \|\nabla L_s\|^2 + \nabla L_s\cdot\nabla\log q_s
        - \omega\|\nabla L_s\|^2\Big]ds \notag\\
    &\quad - \frac{1}{2}\int_{t_0}^T g^2\,\omega(\omega-1)\|\nabla L_s\|^2\,ds \notag\\
    &= (\omega-1)L_T(X_T^x)
    -\frac{1}{2}(\omega-1)\int_{t_0}^T g^2
        \Big[\Delta L_s + \|\nabla L_s\|^2 + \nabla L_s\cdot\nabla\log q_s
        - \omega\|\nabla L_s\|^2 + \omega\|\nabla L_s\|^2\Big]ds \notag\\
    &= (\omega-1)L_T(X_T^x)
    -\frac{1}{2}(\omega-1)\int_{t_0}^T g^2
        \Big[\Delta L_s + \|\nabla L_s\|^2 + \nabla L_s\cdot\nabla\log q_s\Big]ds.
    \label{ap:exponent-combined}
\end{align}

Finally, because $\|\nabla L_s\|^2 + \nabla L_s\cdot\nabla\log q_s
= \nabla L_s\cdot(\nabla\log p_s - \nabla\log q_s + \nabla\log q_s)
= \nabla L_s\cdot\nabla\log p_s$, the exponent simplifies to
\begin{equation}
    (\omega-1)L_T(X_T^x)
    -\frac{1}{2}(\omega-1)\int_{t_0}^T g^2(s)\,
    \Big(\Delta L_s(X_s^x) + \nabla L_s(X_s^x)\cdot\nabla\log p_s(X_s^x)\Big)\,ds.
    \label{ap:laplace-exponent}
\end{equation}

Substituting this exponent into the expression from Step~4 yields the
general Laplacian form \eqref{eq:ddim-laplace-general}, since
$e^{(\omega-1)L_T(X_T^x)}=(p_T(X_T^x)/q_T(X_T^x))^{\omega-1}$.
Under the usual terminal condition $\log p_T=\log q_T$, this boundary factor is
one, giving the simplified form \eqref{eq:ddim-laplace}. The product-reference form
\eqref{eq:ddim-main} follows directly from Step~4 after substituting
$V_s = \omega(\omega-1)\|\nabla L_s\|^2$ and absorbing
$\mathbb{E}_{\tilde{r}_s}[V_s]$ into the normalization constant.

This completes the proof of Theorem~\ref{thm:ddim}. \hfill $\square$

\subsection{Proof of Theorem~\ref{thm:time-dep} (Time-Dependent Guidance)}
\label{sec:proof-time-dep}

We now extend the proof to time-dependent guidance, $\omega=\omega(t)$. The time dependence adds a source term involving $\omega'(t)$ to the evolution of the product reference density. After solving the corresponding ratio equation, an integration-by-parts argument along the characteristic ODE converts this extra term into the same path-integral structure as in the constant-guidance case.

\medskip\noindent\textbf{Step 1: Evolution of the time-dependent product.}
For $r_t = p_t^{\omega(t)} q_t^{1-\omega(t)}$, differentiating with respect to $t$ gives the same terms as in the constant-$\omega$ case, together with an additional derivative of the exponent:
\begin{align}
    \frac{\partial r_t}{\partial t}
    &= \omega p_t^{\omega-1}q_t^{1-\omega}\frac{\partial p_t}{\partial t}
    + (1-\omega)p_t^\omega q_t^{-\omega}\frac{\partial q_t}{\partial t}
    + \omega'(t)\,p_t^\omega q_t^{1-\omega}\log\frac{p_t}{q_t} \notag\\
    &= -\nabla\cdot(f r_t) + \frac{1}{2}g^2\Delta r_t - \frac{1}{2}g^2 V_t r_t
    + \omega'(t)\log\frac{p_t}{q_t}\,r_t,
    \label{ap:rt-time}
\end{align}
where $V_t = \omega(t)(\omega(t)-1)\|\nabla_x\log p_t - \nabla_x\log q_t\|^2$, and $\omega'$ denotes the weak derivative of $\omega$ with respect to $t$.

Normalizing $r_t$ as in the proof of Theorem~\ref{thm:fp}, the Fokker--Planck equation for $\tilde{r}_t$ becomes
\begin{align}
    \frac{\partial\tilde{r}_t}{\partial t}
    &= -\nabla\cdot(f\tilde{r}_t) + \frac{1}{2}g^2\Delta\tilde{r}_t
    - \frac{1}{2}g^2\big(V_t - \mathbb{E}_{\tilde{r}_t}[V_t]\big)\tilde{r}_t \notag\\
    &\quad + \omega'(t)\Big(\log\frac{p_t}{q_t} - \mathbb{E}_{\tilde{r}_t}\big[\log\frac{p_t}{q_t}\big]\Big)\tilde{r}_t,
    \label{ap:fp-tilde-time}
\end{align}
\medskip\noindent\textbf{Step 2: Ratio equation.}
Repeating the density-ratio calculation from Theorem~\ref{thm:ddim} gives
\begin{align}
    \frac{\partial\rho_t}{\partial t}
    &= -f\cdot\nabla\rho_t + \frac{1}{2}g^2\nabla\log\tilde{r}_t\cdot\nabla\rho_t
    + \frac{1}{2}g^2(V_t - \mathbb{E}_{\tilde{r}_t}[V_t])\rho_t \notag\\
    &\quad - \omega'(t)\Big(\log\frac{p_t}{q_t} - \mathbb{E}_{\tilde{r}_t}\big[\log\frac{p_t}{q_t}\big]\Big)\rho_t.
    \label{ap:rho-pde-time}
\end{align}

\medskip\noindent\textbf{Step 3: Solution along characteristics.}
Solving \eqref{ap:rho-pde-time} along the same characteristic ODE yields
\begin{align}
    \rho_{t_0}(x) &= \exp\!\Big(-\frac{1}{2}\int_{t_0}^T g^2(s)\big(V_s(X_s^x) - \mathbb{E}_{\tilde{r}_s}[V_s]\big)\,ds
    + \int_{t_0}^T \omega'(s)\big(\log\frac{p_s(X_s^x)}{q_s(X_s^x)} - \mathbb{E}_{\tilde{r}_s}[\log\frac{p_s}{q_s}]\big)\,ds\Big) \notag\\
    &= \frac{1}{Z''}\exp\!\Big(-\frac{1}{2}\int_{t_0}^T g^2(s)V_s(X_s^x)\,ds
    + \int_{t_0}^T \omega'(s)\log\frac{p_s(X_s^x)}{q_s(X_s^x)}\,ds\Big),
\end{align}
where all terms independent of $x$ are absorbed into $Z''$. Since
$\hat{r}_{t_0}=\tilde{r}_{t_0}\rho_{t_0}$, it remains to simplify the
trajectory-dependent exponent.

\medskip\noindent\textbf{Step 4: Integration by parts.}
Let $L_t(x) = \log p_t(x) - \log q_t(x)$.
Along the characteristic trajectory,
\begin{equation}
    \int_{t_0}^T \omega'(s)L_s(X_s^x)\,ds
    = \big[\omega(s)L_s(X_s^x)\big]_{t_0}^T - \int_{t_0}^T \omega(s)\frac{d}{ds}L_s(X_s^x)\,ds.
    \label{ap:ibp}
\end{equation}

The total derivative of $L$ along the characteristic ODE was computed in
Appendix~\ref{sec:proof-ddim}, Step~5:
\begin{equation}
    \frac{dL_s}{ds}
    = \frac{1}{2}g^2\Big[\big(\Delta L_s + \|\nabla L_s\|^2 + \nabla L_s\cdot\nabla\log q_s\big)
        - \omega(s)\|\nabla L_s\|^2\Big].
\end{equation}

Combining the $\omega'$ integral with the $V$-integral gives
\begin{align}
    &\int_{t_0}^T \omega'(s)L_s\,ds - \frac{1}{2}\int_{t_0}^T \omega(s)(\omega(s)-1)g^2(s)\|\nabla L_s\|^2\,ds \notag\\
    &= \omega(T)L_T(X_T^x) - \omega(t_0)L_{t_0}(x) - \int_{t_0}^T \omega(s)\frac{dL_s}{ds}\,ds
    - \frac{1}{2}\int_{t_0}^T (\omega^2(s)-\omega(s))g^2\|\nabla L_s\|^2\,ds \notag\\
    &= \omega(T)L_T(X_T^x) - \omega(t_0)L_{t_0}(x)
    - \int_{t_0}^T \frac{1}{2}g^2(s)\omega(s)\big(\Delta L_s + \nabla L_s\cdot\nabla\log q_s\big)\,ds \notag\\
    &= (\omega(T)-1)L_T(X_T^x) - (\omega(t_0)-1)L_{t_0}(x) \notag\\
    &\quad - \int_{t_0}^T \frac{1}{2}g^2(s)(\omega(s)-1)\big(\Delta L_s + \nabla L_s\cdot\nabla\log p_s\big)\,ds.
    \label{ap:simplification}
\end{align}

Substituting this identity into $\hat{r}_{t_0}=\tilde{r}_{t_0}\rho_{t_0}$, the boundary term at $t_0$
combines with $p_{t_0}^{\omega(t_0)}q_{t_0}^{1-\omega(t_0)}$ to yield $p_{t_0}$, since
$p_{t_0}^{\omega(t_0)}q_{t_0}^{1-\omega(t_0)}e^{-(\omega(t_0)-1)L_{t_0}(x)} = p_{t_0}$.
The remaining terminal boundary factor is
$e^{(\omega(T)-1)L_T(X_T^x)}=(p_T(X_T^x)/q_T(X_T^x))^{\omega(T)-1}$, which gives the
representation \eqref{eq:time-dep-simple-general}. Under the usual terminal
condition $\log p_T=\log q_T$, this factor is one, yielding \eqref{eq:time-dep-simple}.
\hfill $\square$

\section{Toy Model: Analytic Score Formulas}
\label{sec:appendix-toy}

This appendix records the closed-form quantities used in the toy model experiments
of Section~\ref{sec:exp-toy}. Let $\mathbf{x}^k$ denote the $k$-th support point,
with $K=10$ points equally spaced on a circle of radius $R=2$. The unconditional
distribution is uniform,
$q_0(\mathbf{x})=\sum_{k=0}^{K-1}(1/K)\delta(\mathbf{x}-\mathbf{x}^k)$,
whereas the conditional distribution reweights the same support by likelihoods
$\boldsymbol{\ell}$:
\begin{equation}
    p_0(\mathbf{x})
    = \sum_{k=0}^{K-1}
    \frac{\pi_k\ell_k}{\sum_j \pi_j\ell_j}\,
    \delta(\mathbf{x}-\mathbf{x}^k).
\end{equation}
In the experiments, $\pi_k=1/K$ and
$\boldsymbol{\ell}=[2,2,2,2,1,1,1,1,1,1]$.

Under the VPSDE forward process $\mathbf{x}_t = \sqrt{\bar{\alpha}_t}\mathbf{x}_0
+ \sigma_t \boldsymbol{\epsilon}$ with $\sigma_t^2 = 1 - \bar{\alpha}_t$, the
marginal density at time $t$ is a Gaussian mixture:
\begin{equation}
    p_t(\mathbf{x}) = \sum_{k=0}^{K-1} w_k\,
        \mathcal{N}\bigl(\mathbf{x}; \sqrt{\bar{\alpha}_t}\,\mathbf{x}^k, \sigma_t^2 I\bigr),
\end{equation}
where $w_k=\pi_k\ell_k/\sum_j\pi_j\ell_j$ for $p_t$ and $w_k=1/K$ for $q_t$.

\medskip\noindent
\textbf{Score function.}
\begin{equation}
    \mathbf{score}_t(\mathbf{x}) = \nabla_{\mathbf{x}} \log p_t(\mathbf{x})
    = \frac{\sum_k \gamma_k(\mathbf{x})\,\boldsymbol{\mu}_{t,k} - \mathbf{x}}{\sigma_t^2},
    \label{ap:score-toy}
\end{equation}
where $\boldsymbol{\mu}_{t,k} = \sqrt{\bar{\alpha}_t}\,\mathbf{x}^k$ and
\begin{equation}
    \gamma_k(\mathbf{x}) = \frac{w_k\,
        \mathcal{N}(\mathbf{x}; \boldsymbol{\mu}_{t,k}, \sigma_t^2 I)}
        {\sum_j w_j\,
        \mathcal{N}(\mathbf{x}; \boldsymbol{\mu}_{t,j}, \sigma_t^2 I)}.
\end{equation}

\medskip\noindent
\textbf{Score divergence.}
\begin{equation}
    \nabla_{\mathbf{x}} \cdot \mathbf{score}_t(\mathbf{x})
    = \frac{\mathrm{Tr}\bigl(\mathrm{Cov}_{\gamma}[\boldsymbol{\mu}_t]\bigr)}{\sigma_t^4}
      - \frac{d}{\sigma_t^2},
    \label{ap:div-toy}
\end{equation}
where $d = 2$ and $\mathrm{Cov}_{\gamma}[\boldsymbol{\mu}_t] = \sum_k \gamma_k
\boldsymbol{\mu}_{t,k} \boldsymbol{\mu}_{t,k}^\top - (\sum_k \gamma_k
\boldsymbol{\mu}_{t,k})(\sum_k \gamma_k \boldsymbol{\mu}_{t,k})^\top$.

\medskip\noindent
\textbf{Empirical label distribution.}
Sampling is stopped at the small positive time $t_0$ rather than at exactly
$0$. Each generated sample is assigned to the nearest support point
$\mathbf{x}^k$, producing the empirical label distribution
$\hat{P}^{\mathrm{samp}}_k$. Since $t_0$ is small, this nearest-neighbor label
distribution provides a direct discrete summary of the sampled density around
the original support points.

\medskip\noindent
\textbf{Theoretical distribution computation.}
For each support point $\mathbf{x}^k$, we compute the theoretical label weight
$\hat{P}^{\mathrm{th}}_k$ by averaging the path-integral correction over local
noise at time $t_0$. Specifically, draw $N_{\mathrm{epoch}}=1000$ perturbations
$\mathbf{x}_{t_0}^{(k,\ell)} = \sqrt{\bar{\alpha}_{t_0}}\mathbf{x}^k +
\sigma_{t_0}\boldsymbol{\epsilon}^{(\ell)}$. For each perturbation, integrate
the DDIM characteristic ODE from $t_0$ to $T$ using $1000$ steps. At each step
$s_j$, evaluate the conditional score $\mathbf{score}^{(p)}$, the unconditional
score $\mathbf{score}^{(q)}$, their difference
$\Delta\mathbf{score}=\mathbf{score}^{(p)}-\mathbf{score}^{(q)}$, and the corresponding
divergence difference using~\eqref{ap:div-toy}. The discretized path integral is
\begin{equation}
    I^{(k,\ell)} = \frac{1}{2}\sum_{j} (\omega(s_j)-1)\,\beta(s_j)\,
        \Bigl(\nabla\cdot\Delta\mathbf{score}
              + \Delta\mathbf{score} \cdot \mathbf{score}^{(p)}\Bigr)\Big|_{\mathbf{X}_{s_j}^{(k,\ell)}}
        \,\Delta s,
\end{equation}
where $\Delta s=T/1000=10^{-3}$ is the uniform step size, since $T=1$. The resulting theoretical label
probability is
\begin{equation}
    \hat{P}^{\mathrm{th}}_k \propto
    p_{t_0}(\mathbf{x}^k)\,
    \frac{1}{N_{\mathrm{epoch}}}\sum_{\ell=1}^{N_{\mathrm{epoch}}}
    \exp(-I^{(k,\ell)}),
\end{equation}
normalized over $k=0,\dots,K-1$. This is the discrete counterpart of the
continuous theoretical density at time $t_0$, averaged over a small neighborhood
of each support point, and is the quantity compared with
$\hat{P}^{\mathrm{samp}}_k$ in Section~\ref{sec:exp-toy}.

\end{document}